%% file: main.tex
\pdfoutput=1
\documentclass[
    10pt, 
    logo, 
    twocolumn, 
    copyright, 
    nonumbering
]{nvidiatechreport}

\usepackage{hyperref}
\usepackage{url}
\usepackage{colortbl}
\usepackage{nicefrac}       %
\usepackage{microtype}      %
\usepackage[dvipsnames]{xcolor}         %
\usepackage{multirow}
\usepackage{multicol}
\usepackage{graphicx}
\usepackage{xspace}
\usepackage{amsmath}
\usepackage[export]{adjustbox}
\usepackage{tcolorbox}
\usepackage{amssymb}
\usepackage{enumitem}
\usepackage{wrapfig}
\usepackage{xcolor}
\usepackage{subcaption}    %
\usepackage{float}
\usepackage{stfloats}
\usepackage{listings}
\usepackage{mdframed}
\usepackage{tcolorbox}
\usepackage{arydshln}
\usepackage{booktabs} 
\usepackage{color,soul}
\usepackage{makecell}
\usepackage[numbers]{natbib}
\usepackage{caption}
\usepackage{makecell}
\usepackage{algorithm}
\usepackage{algpseudocode}
\usepackage{placeins}
\usepackage{amsfonts}
\usepackage{lipsum}
\usepackage{subcaption}
\usepackage{pgfplots}
\pgfplotsset{compat=1.18}
\tcbuselibrary{listings,breakable}

\definecolor{pastelblue}{RGB}{173,216,230}
\definecolor{pastelyellow}{RGB}{255,253,208}
\definecolor{pastelpink}{RGB}{255,209,220}
\definecolor{pastelgreen}{RGB}{176,226,172}
\definecolor{pastellavender}{RGB}{230,230,250}

\definecolor{NvidiaGreen}{RGB}{118, 185, 0}
\sethlcolor{red!15}

\hypersetup{
    colorlinks=true, 
    linkcolor=OliveGreen,
    citecolor=OliveGreen,
    urlcolor=OliveGreen,
}

\input{math_commands}

\usepackage{courier}

\definecolor{code_bg}{HTML}{FCF6EC}

\definecolor{code_comment}{HTML}{888888} %
\definecolor{code_keyword}{HTML}{5B3C88} %
\definecolor{code_pytorch}{HTML}{316896} %
\definecolor{code_model}{HTML}{9E3636}   %

\lstdefinelanguage{ModernPython}{
    language=Python,
    basicstyle=\fontsize{7}{9}\ttfamily, 
    backgroundcolor=\color{code_bg},           %
    commentstyle=\color{code_comment},         %
    keywordstyle=\color{code_keyword}\bfseries,
    morekeywords={with, no_grad, dim},
    classoffset=1,
    morekeywords={randn, arange, expand, cat, cumsum, detach, mse_loss, mean, einsum, randint, clip, diff, randn_like},
    keywordstyle=\color{code_pytorch},         %
    classoffset=2,
    morekeywords={pd, fm, teacher, student},
    keywordstyle=\color{code_model}\bfseries,  %
    classoffset=0,
    framesep=8pt,
    xleftmargin=8pt,
    frame=none,
    showstringspaces=false
}

\definecolor{framebg}{HTML}{EEF2F5} %
\newmdenv[
    backgroundcolor=framebg,
    roundcorner=3pt,
    skipabove=7pt,
    linewidth=0pt,
    innertopmargin=6pt,
    innerbottommargin=6pt, %
    innerleftmargin=4pt,   %
    innerrightmargin=4pt   %
]{myframe}

\definecolor{best}{HTML}{E2F0D9}
\definecolor{secondbest}{HTML}{F1F8EB} %
\title{\centering{Parallel Decoding Distillation\\ for Fast Image and Video Generation}}
\newcommand{\headertitle}{Parallel Decoding Distillation for Image and Video Generation}
\newcommand{\framecell}[1]{\raisebox{-0.5\height}{\includegraphics[width=2.7cm,keepaspectratio]{#1}}}
\newcommand{\framecellw}[1]{\raisebox{-0.5\height}{\includegraphics[width=3.2cm,keepaspectratio]{#1}}}

\usepackage{cuted} %

\makeatletter
\@ifundefined{@setmarks}{\let\@setmarks\relax}{}
\makeatother

\usepackage{caption} %
\author{
\centering{
Neta Shaul$^2$, Chao Liu$^1$, Arash Vahdat$^{1,\dagger}$, Julius Berner$^{1,\dagger}$\\
$^1$NVIDIA, $^2$Weizmann Institute of Science, $^\dagger$equal advising
}
}

\begin{abstract}
Generation in video diffusion or flow models is computationally expensive due to the slow and iterative sampling process. Current state-of-the-art (SOTA) acceleration methods heavily rely on variational score distillation (VSD) and adversarial losses to distill diffusion models into few-step generators. Albeit achieving high-quality video generation, these training losses are notoriously hard to optimize and suffer from mode collapse, leading to loss of video diversity and lack of motion. In this paper, we introduce Parallel Decoding Distillation (PDD), a simplified and scalable trajectory-based distillation method for fast inference of diffusion and flow matching models. Our architecture and training procedure are compatible with any pre-trained model and support sampling with a varying number of function evaluations (NFE). PDD accelerates generation by predicting multiple denoising steps per network evaluation. Conceptually, it learns a representation of the mean velocity without regressing its derivative using JVPs or finite-difference approximations. Our method achieves SOTA performance with 4-8 NFE on LTX-2.3 Text-to-Video/Audio, Wan 14B Text-to-Video, and Qwen-Image Text-to-Image. Moreover, PDD presents a significant improvement in generated video diversity. Project page: \url{https://research.nvidia.com/labs/genair/pdd}
\vspace{-20pt}
\end{abstract}

\begin{document}

\maketitle
\input{figures/wan/teaser/teaser_new.tex}

\section{Introduction}
\vspace{-5pt}
Large-scale diffusion and flow models~\citep{ho2020ddpm,song2021scorebased,lipman2023flowmatching, liu2022straightfast,albergo2025stochasticinterpolants} have achieved remarkable capabilities of media generation, including text-to-image~\citep{wu2025qwenimage,flux-2-2025}, text-to-video~\cite{kong2025hunyuanvideo,wan2025wan2p1, nvidia2026cosmos3omnimodalworld}, and multi-modal generation~\citep{hacohen2026ltx2}. Yet, the cost of generation remains high, as their inherently iterative sampling algorithms often require hundreds of network evaluations. The resulting computational cost and latency are one of the main bottlenecks for many applications, such as content editing, real-time video generation, and interactive world modeling. Thus, developing distillation methods for few-step media generation models has become an active field of research~\citep{sauer2024ladd,yin2024dmd2,xu2025fdistill, lu2025scm,lee2025decoupledmf,chen2025piflow,park2026eflow,zheng2025rcm,tong2025freeflow,nie2026tmd,gu2026anyflow,xiao2026spd}.

The approaches for distilling diffusion and flow models are broadly categorized into two families: i) \emph{trajectory-based} methods~\citep{salimans2022progressivedistillation,song2023consistencymodels,frans2025shortcutmodels,sabour2025ayf,geng2025meanflow,chen2025piflow} in which a student model distills the many-step sequential sampling process of a pre-trained teacher into a few-step process; and ii) \emph{distribution-based} methods~\citep{sauer2023add,wang2023vsd,yin2024dmd,xu2025fdistill}, which relax the constraint of following the teacher trajectories and instead align only the marginal distributions of the student and teacher processes. Trajectory-based methods have shown promising results for fast image generation. However, when applied to video models, they are typically bottlenecked by degraded video quality and costly training algorithms. As a result, the current dominant methods for distillation of video models are distribution-based~\citep{fan2026phaseddmd,nie2026tmd}. While recent works regularize the training dynamics by additionally incorporating trajectory-based distillation losses~\cite{zheng2025rcm,gu2026anyflow}, they still suffer from alternating training objectives, high memory requirements, and mode collapse, leading to reduced diversity and often static videos. 

In this paper, we introduce \emph{parallel decoding distillation} (PDD), a trajectory-based distillation method for fast inference of diffusion and flow models. Instead of merging multiple denoising steps into a single larger step~\citep{salimans2022progressivedistillation,frans2025shortcutmodels,lu2025scm,zhou2026tvm}, 
we learn a parallel decoder that predicts multiple denoising steps in a single network evaluation. Specifically, PDD discretizes the time domain of the flow into a fixed sequence of $N$ intervals, which are grouped into blocks of size $L$. During training, the parallel decoder learns to predict the mean velocities in all intervals within a block using a single forward pass, as illustrated in Figure~\ref{fig:pd_sketch}. Target mean velocities are approximated using a Runge-Kutta solver (e.g., Euler or Midpoint) applied to the pre-trained teacher model. At generation, by predicting the velocities across $L$ intervals, we obtain a sample with $N/L$ steps. By varying the  block size during training, PDD supports sampling with different number of function evaluations (NFEs) during inference.
\begin{figure}
    \centering
    \includegraphics[
    width=1.0\linewidth,
    trim={0pt 0pt 20pt 0cm}, %
    clip
    ]{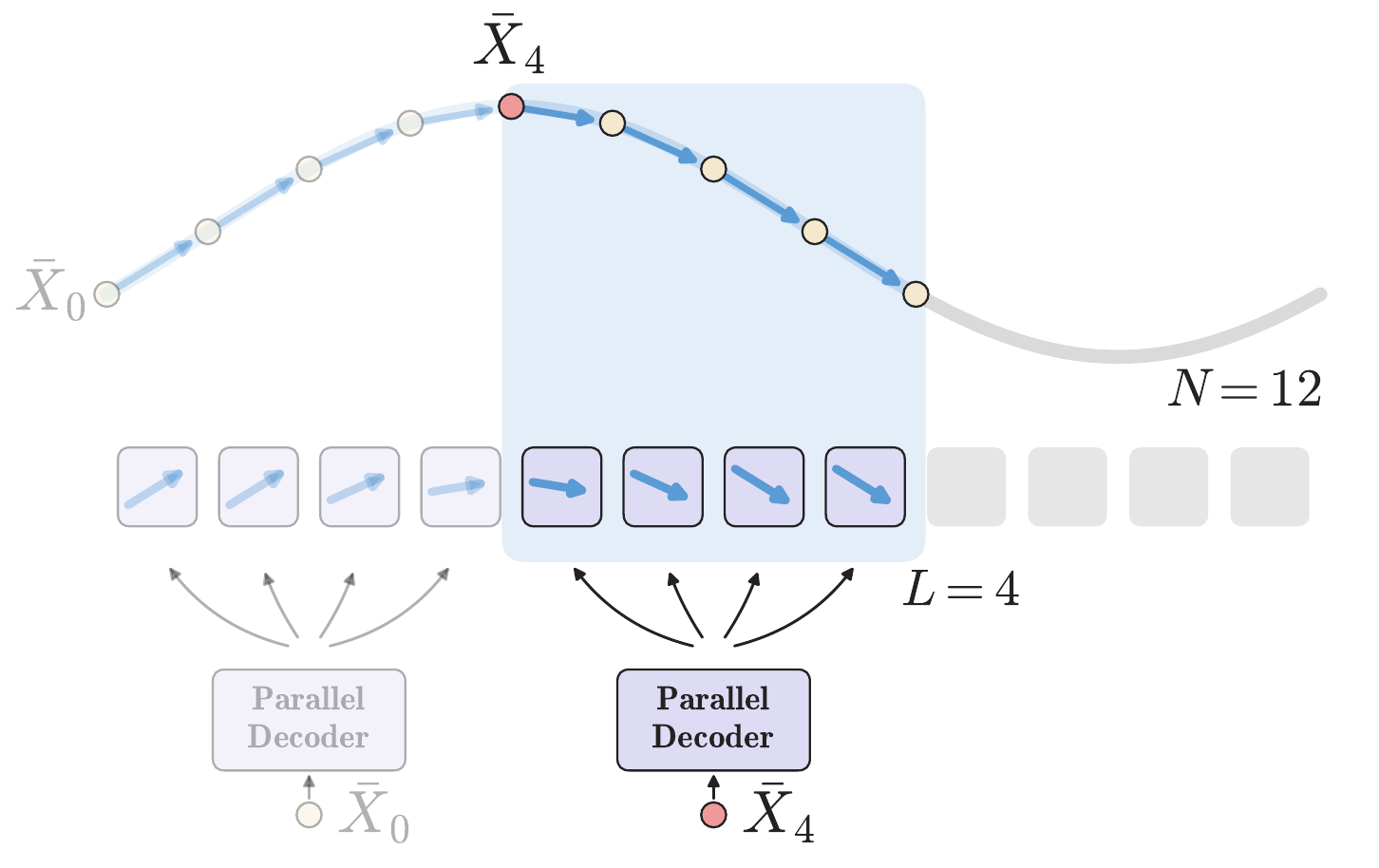}
    \vspace{-25pt}
    \caption{The sampling trajectory is discretized into $N$ intervals, which are grouped into blocks of size $L$. The parallel decoder predicts the mean velocities for all intervals within a block using a single evaluation.
    }
    \label{fig:pd_sketch}
    \vspace{-15pt}
\end{figure}
PDD is conceptually related to Pi-Flow~\citep{chen2025piflow}, but provides a simplified training algorithm that removes the need for an additional policy head and supports a variable NFEs at generation time. Additionally, our formulation clarifies the connection to flow map distillation methods, particularly the Lagrangian formulation~\citep{zhou2026tvm,boffi2025flowmaps}. By avoiding Jacobian-vector products (JVPs) and finite differences, we obtain minimal-cost training algorithm.

We showcase the efficacy of our method on text-to-video tasks with Wan2.1~\citep{wan2025wan2p1} 1.3B and 14B models, achieving SOTA video quality generation with $4$ NFE on VBench~\citep{huang2023vbench} while preserving better video diversity compared to distribution-based baselines. Additionally, we test PDD on text-to-image tasks with the Qwen-Image 20B model, achieving SOTA scores with 4 to 8 NFE on OneIG~\citep{chang2025oneig}, GenEval~\citep{ghosh2023geneval}, and DPG-Bench~\citep{hu2024ella} benchmarks.

All in all our main contributions are:
\begin{enumerate}[itemsep=4pt,]
    \vspace{-10pt}
    \item Formulate Parallel Decoding Distillation, a scalable, trajectory-based distillation method for fast inference of flow matching and diffusion models.

    \item Develop a single, regression-based training objective that avoids the need for JVPs, finite differences, or multi-stage training procedures and produces high-quality and diverse samples without VSD or GAN losses.
    
    \item Propose a simple architecture and training algorithm that are supported by any pre-trained model and allows generation with varying NFE without additional time conditioning. 
    
    \item Validate our method on ImageNet-256, Qwen-Image, Wan2.1 1.3B/14B, and LTX-2.3, achieving SOTA performance with improved diversity.
\end{enumerate}

\section{Generative Flow Models}
\paragraph{Notation} The state space is denoted by $\gX$, where $\gX=\R^{c\times h\times w}$ for images and $\gX=\R^{f\times c\times h\times w}$ for videos. Random variables taking values in $\gX$ are denoted by the uppercase letter $X$, and states in $\gX$ are denoted by the lowercase letter $x$. Fixed integers are denoted by the uppercase letters $N,L\in\sN$, while running indices or integer-valued random variables are denoted by the lowercase letters $n,k$.

\paragraph{Flow Matching and Diffusion Models} The currently dominant approaches for training generative flow models are flow matching~\cite{lipman2023flowmatching,liu2022straightfast, albergo2025stochasticinterpolants} and diffusion models~\citep{sohldickstein2015diffusionmodels, ho2020ddpm, song2021scorebased}. As we are interested only in deterministic processes, for simplicity, we treat both as flows~(\ref{e:flow_process}). 

A flow process $\parr{X_t}_{0\le t\le 1}$ taking values in $\gX$ is defined by a \emph{velocity field} $v:\gX\times[0,1]\too\gX$ and a \emph{source distribution} $p_0:\gX\too\R_{\ge0}$, which serves as a boundary condition setting the marginal of the process at time $t=0$:
\begin{equation}\label{e:flow_process}
    \frac{d}{dt}X_t = v_t\parr{X_t},\quad X_0\sim p_0.
\end{equation}
The marginal of the process, called the \emph{probability path}, is a time-dependent density $p:\gX\times[0,1]\too\R_{\ge0}$ such that $X_t\sim p_t$ for all $t\in[0,1]$.

Assume a dataset of i.i.d.\@ samples in $\gX$ from a \emph{target distribution} $p_{1}$ and some easy-to-sample \emph{source distribution} $p_0$. Flow matching provides a framework for learning a model $v_t$ such that the flow $\parr{X_t}_{0\le t\le 1}$ defined by \eqref{e:flow_process} maps samples from the source $X_0\sim p_0$ to samples from the target $X_1\sim p_1$.

Importantly, during training, the marginal $p_t$ of a flow matching model can be sampled efficiently using the \emph{interpolant process}. While our method is agnostic to the chosen interpolant process, all pre-trained models used in this work were trained using the linear scheduler,
\begin{equation}\label{e:x_cond}
    X_t = (1-t)X_0 + tX_1,
\end{equation}
where $X_0\sim p_0$, $X_1\sim p_1$, and $t\in [0,1]$.
\paragraph{Sampling with flows} Obtaining a sample $X_1\sim p_{1}$ from a trained flow model $v_t$ is done by solving the ODE in \eqref{e:flow_process}. A general numerical approach discretizes the time interval $[0,1]$ into a sequence of $N$ smaller intervals, $0=t_0<t_1<\cdots<t_N=1$. Then, the exact solution on each interval is
\begin{equation}\label{e:sol_exact}
    X_{n+1} = X_{n} + \parr{t_{n+1}-t_{n}} u_{n}\parr{X_{n}},
\end{equation}
where we use the simplified notation $X_n:=X_{t_n}$, and $u_n$ denotes the \emph{mean velocity} of the $n$-th interval $[t_n,t_{n+1}]$, defined as
\begin{equation}\label{e:mean_vel}
   u_{n}\parr{X_{n}}  = \frac{1}{t_{n+1}-t_n}\int_{t_n}^{t_{n+1}}v_t\parr{X_t}\,dt.
\end{equation}
The solution is obtained by sequentially approximating the integral in \eqref{e:mean_vel}. The simplest numerical method is the Euler solver, which approximates the velocity $v_t$ as constant over the interval $t\in[t_n,t_{n+1}]$, yielding the mean velocity
\begin{equation}\label{e:sol_euler}
    u_{n}\parr{X_{n}} \approx v_{t_n}\parr{X_{n}}.
\end{equation}
Runge-Kutta methods are a family of higher-order solvers that use additional evaluations of the velocity $v_t$ in the interval $[t_n,t_{n+1}]$ to achieve a higher-order approximation. We use the Midpoint method,
\begin{equation}\label{e:sol_midpoint}
    u_{n}\parr{X_{n}} \approx v_{t_{\text{mid}}}\parr{X_{\text{mid}}},
\end{equation}
where $X_{\text{mid}}$ and $t_{\text{mid}}$ are the midpoint state and time, respectively,
\begin{equation*}
    X_{\text{mid}} = X_n + \frac{t_{n+1}-t_n}{2}v_{t_n}\parr{X_{n}},\quad t_{\text{mid}} = \frac{t_{n+1}+t_n}{2}.
\end{equation*}
\begin{figure*}
    \centering
    \includegraphics[width=0.48\linewidth]{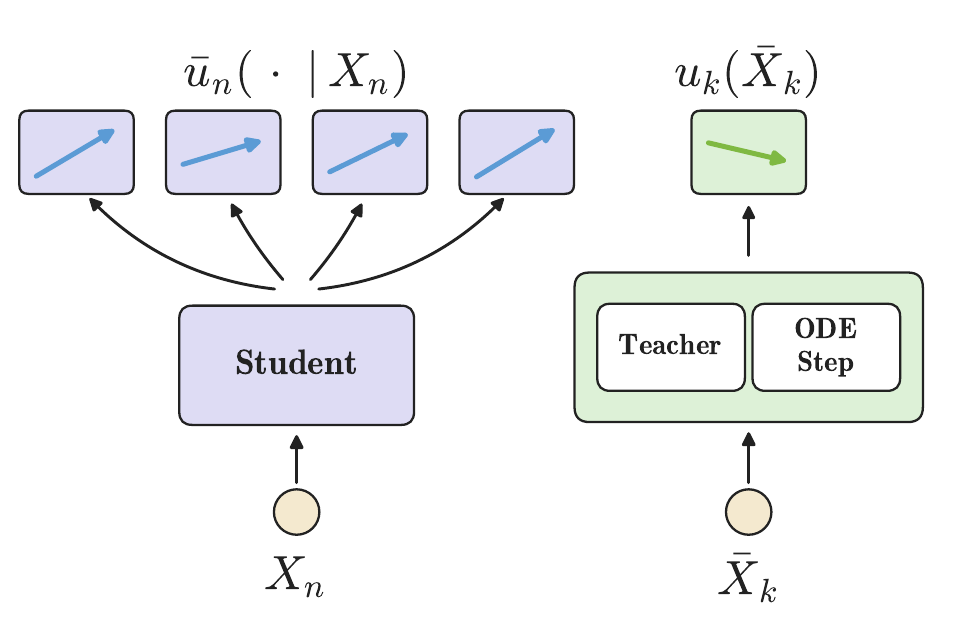}
    \includegraphics[width=0.48\linewidth]{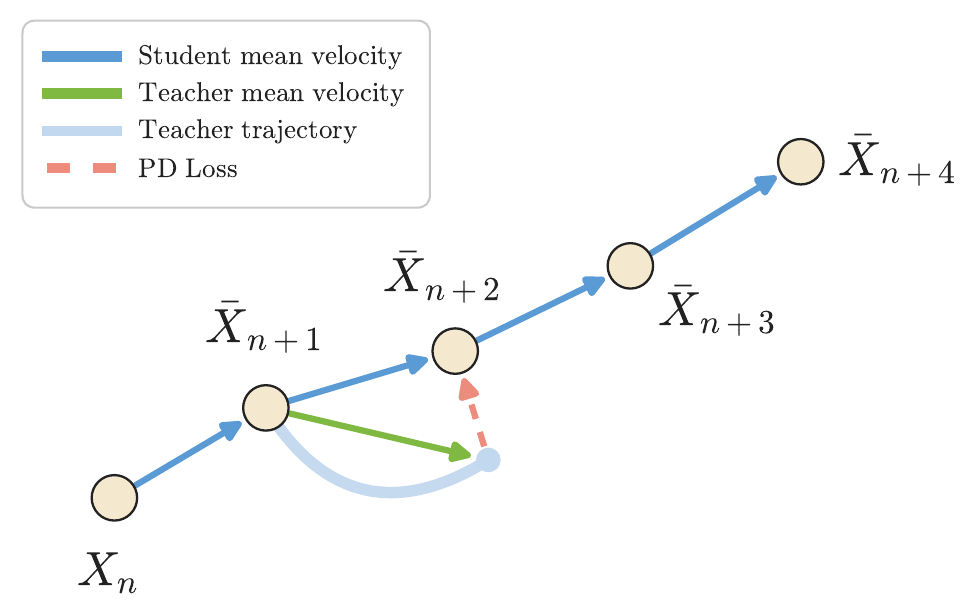}
    \vspace{-5pt}
    \caption{\textbf{(left)} The PDD student approximates the mean velocity across multiple consecutive intervals in a single evaluation. The pre-trained flow model (teacher) provides the mean velocity of a single interval using an ODE solver step. \textbf{(right)} Illustrate estimation of the PD loss given an initial state $X_n\sim p_{t_n}$ and a block size $L=4$; i) student predicts the mean velocities $\bar{u}^{\theta}_n\parr{\cdot|X_n}\in\gX^L$; ii) Following the student velocities yields the states $\bar{X}_{k}$ in the block $k\in\set{n,\ldots,n+L-1}$; iii) one of the states is randomly selected and the student's output velocity is matched to the teacher's mean velocity in the corresponding interval using the PD loss~(\ref{e:pd_loss}).}
    \label{fig:pd_loss}
    \vspace{-10pt}
\end{figure*}
\section{Parallel Decoding Distillation}
\label{s:pdd}
To accelerate sampling from flow models, we propose to learn a parallel decoding model that predicts multiple integration steps~(\ref{e:sol_exact}) in a single network evaluation, rather than approximating them one step at a time. 
Figure~\ref{fig:pd_sketch} illustrates our approach.

\paragraph{Parallel decoder.}
Assume a pre-trained flow model $v_t$ with a flow process $\parr{X_t}_{0\le t\le 1}$ defined by \eqref{e:flow_process} and marginals $p_t$. 
Fix a time discretization of length $N\in\sN$,
\begin{equation}\label{e:time_disc}
    0=t_0 < t_1 < \ldots < t_N=1.
\end{equation}
Then, a \emph{block} of size $L$ starting at step $n$ is the set of indices $\set{n,\ldots,n+L-1}$.
\begin{myframe}
    For a state $X_{n} \sim p_{t_n}$ at time step $t_n$, a \emph{parallel decoder} $\bar{u}^\theta_{n}\parr{\cdot\mid X_{n}}\in\gX^L$ with block size $L\in\sN$ is trained to predict the mean velocities of all intervals in the next block using a single network evaluation,
\begin{equation}\label{e:pd_def}
    \bar{u}^\theta_{n}\parr{k \mid X_{n}} \approx u_{k}\parr{X_{k}},\quad k=n,\ldots,n+L-1.
\end{equation}
\end{myframe}
Importantly, our parallel decoder~(\ref{e:pd_def}) is well-defined, since the discretized flow process in this block, i.e., $X_{k}$, $k\in\set{n,\ldots,n+L-1}$, is fully specified by the exact solution~(\ref{e:sol_exact}) and the initial state $X_{n} \sim p_{t_n}$.
\paragraph{The parallelized process} During training, we employ an on-policy optimization algorithm that requires the intra-block process given by the parallel decoding model. 
\begin{myframe}
    For a state $X_{n} \sim p_{t_n}$ at time step $t_n$, the \emph{parallelized process} is defined by
    \begin{equation}\label{e:pd_process}
        \bar{X}_{k+1} = \bar{X}_{k} + (t_{k+1}-t_k)\bar{u}^\theta_{n}\parr{k \mid X_{n}},
    \end{equation}
    for the block $k\in\set{n,\ldots,n+L-1}$, with the initial condition $\bar{X}_{n} = X_{n}$.
\end{myframe}
The parallelized process is obtained by substituting the parallel decoder into the update rule of the exact solution~(\ref{e:sol_exact}). Notably, $\bar{u}^\theta_{n}\parr{\cdot \mid X_{n}}\in\gX^L$ depends only on the initial state $X_n$. Thus, the parallelized process $\parr{\bar{X}_k}_{n\le k\le n+L}$ is simulated using a single evaluation of the parallel decoder.
\vspace{-5pt}
\paragraph{Sampling} While classical algorithms advance a single interval at each ODE step using the recursive rule~(\ref{e:sol_exact}), with the parallel decoder we can advance $L$ intervals simultaneously. For a current state $X_{n}\sim p_{t_n}$, we approximate the exact solution~(\ref{e:sol_exact}) in the block $\set{n,\ldots, n+L-1}$ using the parallelized process~(\ref{e:pd_process}). Then, solving the recursion over the intra-block index $k$, yields the \emph{block-step rule}
\begin{equation}\label{e:block_step}
    \bar{X}_{n+L} = X_n + \sum_{k=n}^{n+L-1}(t_{k+1}-t_k)\bar{u}^\theta_{n}\parr{k | X_{n}}.
\end{equation}
By repeating the recursive block-step rule~(\ref{e:block_step}) $N/L$ times, in each step approximating $\bar{X}_n\approx X_n$, we obtain a clean sample from the model $\bar{X}_N$. We provide a PyTorch pseudocode in Algorithm~\ref{alg:parallel_decoding_sampling}, where dimension 0 indexes the parallel predictions $k=0,\ldots,N-1$.

\begin{algorithm}[H]
\caption{Parallel decoding sampling}
\label{alg:parallel_decoding_sampling}
\begin{lstlisting}[language=ModernPython]
# student - parallel decoder model
# t - time discretization
# L - block size
# shape - data shape

# init noise
x_n = randn(*shape)
# step sizes
h = diff(t, dim=0)
for n in range(0, len(t)-1, step=L):
    # parallel predictions
    u = student(x_n, t[n])
    # slice current block
    u_n, h_n = u[n:n+L], h[n:n+L]
    # block step
    x_n = x_n + einsum('k,k...', h_n, u_n)
    
return x_n
\end{lstlisting}
\end{algorithm}
\begin{figure*}
    \centering
    \begin{subfigure}[b]{0.33\linewidth}
        \centering
        \includegraphics[width=\linewidth]{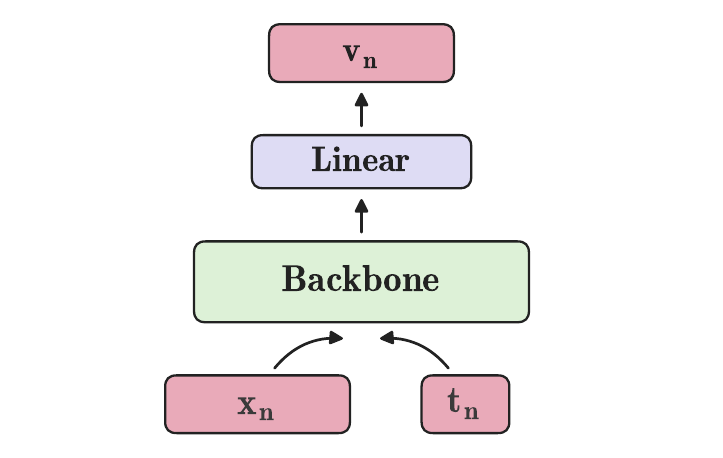}
        \caption{\centering Teacher}
        \label{subfig:teacher}
    \end{subfigure}%
    \begin{subfigure}[b]{0.33\linewidth}
        \centering
        \includegraphics[width=\linewidth]{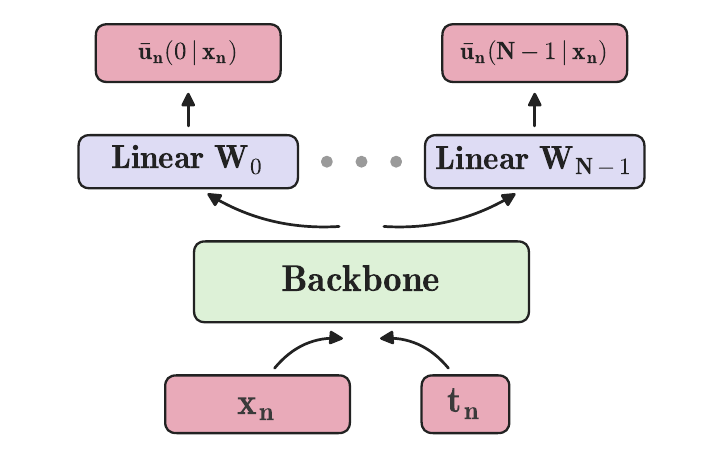}
        \caption{\centering Student train}
    \end{subfigure}
    \begin{subfigure}[b]{0.33\linewidth}
        \centering
        \includegraphics[width=\linewidth]{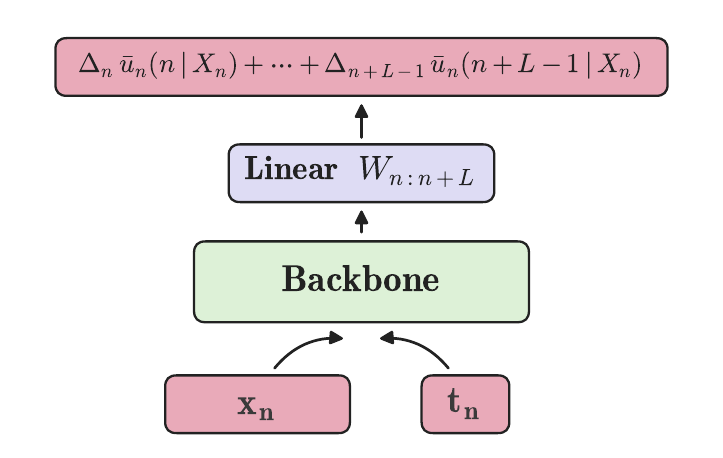}
        \caption{\centering Student generation}
    \end{subfigure}
    \caption{Architecture of the parallel decoder (b) vs.~the pre-trained flow model (a). Notably, the parallel decoder utilizes the same backbone, but with $N$ times the final linear layer. (c) At generation, instead of applying a linear layer per intra-block step~(\ref{e:pd_process}), we can fuse the layers into a single linear layer that outputs the block's average velocity, which advances the state across the entire block~(\ref{e:fused_step}).}
    \label{fig:architecture}
\end{figure*}
\vspace{-25pt}
\paragraph{Training} The parallel decoder is trained by regressing onto a Runge-Kutta approximation of the mean velocity~(\ref{e:pd_def}) of the pre-trained teacher flow model. In practice, we use a single Euler or Midpoint step. To obtain a tractable loss, we employ on-policy training, estimating the teacher's mean velocity on student outputs $\bar{X}_k$ as defined in~(\ref{e:pd_process}).
\begin{myframe}
For a time discretization $t_0<\ldots<t_N$ and a block size $L$, our training objective is
\begin{equation}\label{e:pd_loss}
    \gL_{\text{PD}}\parr{\theta} = \E\brac{\norm{\bar{u}^\theta_{n}\parr{k | X_{n}} - u_k\parr{\text{sg}\parr{\bar{X}_k}}}^2},
\end{equation}
where block 
starting indices $n\in\set{0,L, \ldots, N-L}$ and intra-block indices $k\in\set{n,\ldots,n+L-1}$ are sampled uniformly, $X_n\sim p_{t_n}$ is sampled using the interpolant process~(\ref{e:x_cond}), $\bar{X}_k$ is the parallelized process~(\ref{e:pd_process}) in the block, $\text{sg}(\cdot)$ denotes the stop-gradient operator, and the teacher mean velocity $u_k$ is approximated with a Runge-Kutta step.
\end{myframe}

Importantly, both $\bar{u}^\theta_{n}\parr{k | X_{n}}$ and $\bar{X}_k$ are obtained with a single evaluation of the parallel decoder $\bar{u}^\theta_n$. Approximating the mean velocity $u_k$ with a single Euler~(\ref{e:sol_euler}) or Midpoint~(\ref{e:sol_midpoint}) step requires 1-2 evaluations of the teacher model $v$ (resp.). This makes the parallel decoding (PD) loss~(\ref{e:pd_loss}) tractable even at large scales. Figure~\ref{fig:pd_loss} illustrates evaluation of the PD loss and we provide PyTorch pseudocode in Algorithm~\ref{alg:parallel_decoding_training}. For simplicity, we ignore the batch dimension, however all operations can be batched.

Proposition~\ref{prop:prop_pd_loss} shows that the PD loss is a valid objective for learning the parallel decoder as defined in~(\ref{e:pd_def}). In particular, up to a controllable Runge-Kutta approximation error, the minimizer of the PD loss samples exactly the teacher trajectories, i.e., $\bar{X}_n = X_n$, for $n=0,\dots,N$. The proof is in Appendix~\ref{a:proofs}.
\begin{proposition}\label{prop:prop_pd_loss}
    The minimizer of the parallel decoding loss~(\ref{e:pd_loss}) satisfies the parallel decoder condition~(\ref{e:pd_def}). 
\end{proposition}
\vspace{-18pt}
\paragraph{Data-free training} In cases where data is not available, we propose an online on-policy training scheme that avoids the need to sample $X_n\sim p_{t_n}$ with the interpolant process~(\ref{e:x_cond}). Instead, we follow our sampling algorithm~\ref{alg:parallel_decoding_sampling}, while alternating solver and training steps. That is, we sample an initial state $X_0\sim p_0$, then for the next $N/L -1$ iterations, we utilize the output of the parallel decoder $\bar{u}^{\theta}_n\parr{\cdot|X_n}$ both for the optimization step and to advance $\bar{X}_n$ to $\bar{X}_{n+L}$; see the modified PyTorch pseudocode for PD loss in Algorithm~\ref{alg:datafree_parallel_decoding_training} in the appendix. Importantly, we apply the stop-gradient operation when advancing the state, preventing additional memory or compute cost. 

\begin{algorithm}[t]
\caption{Parallel decoding loss}
\label{alg:parallel_decoding_training}
\begin{lstlisting}[language=ModernPython]
# student - parallel decoder model
# teacher - pre-trained flow model
# runge_kutta - approximates teacher's mean velocity
# t - time discretization
# L_min - minimal block size
# L_max - maximal block size
# x - data

# grid size
N = len(t) - 1
# sample block index
n = L_min * randint(0, N//L_min)
# sample the probability path
x_n = (1-t[n]) * randn_like(x) + t[n] * x
# parallel predictions 
u = student(x_n, t[n])
# step sizes
h = diff(t, dim=0)
# sample index in the block
k = randint(n, clip(n+L_max, max=N))
# step from n to k
x_k = x_n + einsum('l,l...', h[n:k], u[n:k])
# teacher mean velocity
u_k = runge_kutta(teacher, x_k, t[k], h[k])
# mse loss
loss = mse_loss(u[k], u_k.detach())

return loss
\end{lstlisting}
\end{algorithm}
\paragraph{Architecture and variable block size} Our requirement is an architecture that predicts $L$ mean-velocities $\bar{u}^{\theta}_n\parr{\cdot|X_n}\in\gX^L$, i.e., one velocity for each time step in the block, instead of the single instantaneous velocity prediction $v_{t_n}\parr{X_{t_n}}\in\gX$ of the pre-trained flow model. As illustrated in Figure~\ref{fig:architecture}, we utilize the same backbone architecture of the pretrained flow model, but with the final linear layer repeated $N$ times, \ie, one for each time step in the grid~(\ref{e:time_disc}). Formally, we assume the teacher architecture is of the form
\begin{equation}
    v_{t}(x) = WH_{t}(x),
\end{equation}
where $H_t$ is the backbone that outputs the final hidden state and $W$ is the final linear layer. Then we learn $N$ linear layers $W^\theta_0,\ldots,W^\theta_{N-1}$ and, for $k\ge n$, the parallel decoder's architecture is given by
\begin{equation}\label{e:arch_pd}
    \bar{u}^\theta_n(k|x_n) = W^{\theta}_kH^\theta_{t_n}(x_n).
\end{equation}
This enables initialization from the final layer of the pretrained flow model. Additionally, the advantage of taking $N$ (grid size) instead of exactly $L$ (block size) linear layers is that it allows us to learn a single model that can predict any block size without the need to introduce a second time coordinate, which is, for instance, required for flow maps~\citep{geng2025meanflow, sabour2025ayf, boffi2025flowmaps}.

In practice we are not interested in all possible block sizes, but in some subset of block sizes. Thus, we define a minimum and maximum block size $L_{\text{min}} \le L_{\text{max}}\in\sN$. Then, during training, we consider multiples of $L_{\text{min}}$ for indices $n$ of initial states and sample $k\in\set{n,\ldots,n+L_{\text{max}}-1}$ inside each block. 

\subsection{Layer Fusion and connection to flow maps}
An important observation emerges when comparing PDD training and generation. For a block size $L$ and starting step $n$, estimating the PD loss~(\ref{e:pd_loss}) requires the intra-block steps of the parallelized process~(\ref{e:pd_process}), and thus uses all distinct student output directions $W^{\theta}_kH^\theta_{t_n}\parr{X_n}$, for $k=n,\ldots,n+L-1$. In contrast, during generation, we use the block step~(\ref{e:block_step}) to skip $L$ intervals. This requires only the weighted-average direction, which, for our architecture~(\ref{e:arch_pd}), yields
\begin{equation}\label{e:fused_step}
    \bar{X}_{n+L} = \bar{X}_n +  (t_{n+L}-t_n) W^\theta_{n:n+L}H^\theta_{t_n}\parr{\bar{X}_n},
\end{equation}
where $W^\theta_{n:n+L}$ is a fused linear layer,
\begin{equation}\label{e:fused_layer}
    W^\theta_{n:n+L} =
    \sum_{k=n}^{n+L-1}
    \Delta_kW^\theta_k,\ \Delta_k=\frac{t_{k+1}-t_k}{t_{n+L}-t_n}.
\end{equation}
Thus, our shared backbone $H^\theta_{t_n}$ learns a representation of the mean-velocity over the interval $[t_n,t_{n+L}]$. However, instead of using JVPs~\citep{sabour2025ayf,geng2025meanflow,boffi2025flowmaps, zhou2026tvm} or finite differences~\citep{luo2026soflow}, we use the learnable linear maps $W_k^\theta$, for $k=n,\ldots,n+L-1$, to decompose the mean-velocity prediction into parallel sub-interval predictions. Then, during training, the gradients through the shared backbone recover, in expectation, the training signal for learning the full-interval mean-velocity. A discussion of the exact connection to flow maps is provided in Appendix~\ref{a:flow_maps}. We validate that the parallel decoder can learn non-trivial trajectories by comparing the curvature of its trajectories with that of the teacher trajectories in Figure~\ref{fig:pdd_curvature} in the appendix.

An additional practical implication is that during inference we can avoid the extra compute of an enlarged final layer and we only need to hold one fused linear layer per block in memory.

\begin{table}[t]
\centering
\caption{
High-level differences between flow-map distillation methods, Pi-Flow, and our PDD.
}
\label{tab:highlevel}
\resizebox{1.0\linewidth}{!}{%
\begin{tabular}{lccc}
\toprule
 & Eulerian/Lagrangian Flow Maps & Pi-Flow & PDD (Ours) \\
\midrule
NFE & Variable & Fixed & Variable \\
JVP/finite-diff. & Required & Free & Free \\
Head at inference & Linear & Gaussian mixture & Fused-linear \\
\bottomrule
\end{tabular}
}
\end{table}
\section{Related works} 

\paragraph{Trajectory-based distillation} The first successful method to distill the trajectories of the flow ODE~(\ref{e:flow_process}) using a student-teacher scheme to achieve few-step generation is Progressive Distillation~\citep{salimans2022progressivedistillation}. They gradually increase the step size of the student using multi-phase training. Consistency models~\citep{song2023consistencymodels,luo2023lcm,song2023ict,lu2025scm,geng2024cmeasy} initialize the student from the teacher and then directly distill a map from any intermediate state along the trajectory to the clean state using a self-consistency condition. More recent works~\citep{tee2024pid,boffi2025emd,sabour2025ayf,tong2025freeflow, hu2026cmt, lee2025decoupledmf, sun2026rcgm, cai2025scfm}, distill the mean velocity between any two states on the trajectory. 

These methods have shown promising results on image generation. However, on large scale video models they fail to achieve high-quality few-step generation. Additionally, they often rely on JVP or finite differences which are expensive to evaluate on large-scale models or yield unstable training dynamics.

Most related to our method is Pi-Flow~\cite{chen2025piflow}. Similar to us, they make the observation that given a pre-trained flow model, the trajectories in the interval $[t_n,t_{n+L}]$, are fully specified by the initial state, and use it to delegate the integration in that interval to a small learnable policy head. In contrast, we utilize it to motivate our parallel prediction paradigm, simplifying training and inference, and changing the focus from expressive parametrizations (such as Gaussian mixtures) to improved supervision. While Pi-Flow distills the continuous-time instantaneous velocity~(\ref{e:flow_process}) $v$, we go beyond by discretizing time and distilling numerical approximations of the mean velocity~(\ref{e:mean_vel}) $u$. Furthermore, using layer fusion~(\ref{e:fused_step}) we avoid any additional cost at generation. Lastly, our training algorithm naturally enables sampling with different NFEs, whereas Pi-Flow is restricted to fixed NFE. The high-level differences are summarized in Table~\ref{tab:highlevel}.
\input{figures/wan/main_paper_figure}
\vspace{-10pt}
\paragraph{Training from scratch} There is a growing effort~\citep{frans2025shortcutmodels, geng2025meanflow, boffi2025flowmaps,wang2025transitionmodels,zhou2025inductivemomentmatching,zhou2026tvm} to develop end-to-end training methods for few-step generation models. These methods directly learn the flow map and avoid separate pre-training and distillation phases by combining the flow matching objective~\citep{lipman2023flowmatching,liu2022straightfast} with a trajectory-based distillation objective. Instead of using a separate student-teacher scheme, they use the current state of the model as the teacher. This line of work has spawned several follow-up methods that aim to improve convergence~\citep{zhang2025alphaflow,geng2025imf,guo2025splitmeanflow,lu2026pixelmf,nguyen2025improvedshortcut} or replace JVPs with finite-difference approximations~\citep{luo2026soflow, park2026eflow}. 

While PDD could potentially be extended to a self-distillation framework, such methods require a completely different pipeline, including larger datasets and significantly more compute. %

\paragraph{Distribution-based distillation} Instead of letting a student distill the trajectories of a teacher model, a more flexible but sufficient condition for a few-steps generation model is to align the marginals $p_t$ of the two. ADD~\citep{sauer2023add}, LADD~\citep{sauer2024ladd}, and APT~\citep{lin2025apt1} are using GAN losses to fine-tune a pre-trained diffusion model into a few step generation model. DMD~\citep{yin2024dmd,yin2024dmd2} incorporates the VSD loss~\cite{wang2023vsd} for distillation. $f$-distill~\citep{xu2025fdistill} generalizes the VSD loss to $f$-divergences, and SiD and extensions~\citep{zhou2024sid, zhou2024sida,luo2024onestepdiffusiondistillationscore} use variants of Fisher divergences. Distribution-based methods established themselves as dominant approaches for distillation of large-scale video models~\citep{nie2026tmd, zheng2025rcm, gu2026anyflow, fan2026phaseddmd, lin2025apt1}. However, they suffer from mode collapse leading to a lack of generation diversity and motion. Recent works try to mitigate that by regularizing the loss with trajectory-based objectives~\citep{cheng2026twinflow, zheng2025rcm,gu2026anyflow}.

Moreover, distribution-based methods are known to be sensitive to hyper-parameters, require additional trainable parameters, and their performance can vary significantly across training iterations due to their alternating training objectives. In contrast, we find that the memory-efficient simple PD loss is much more robust to hyper-parameter choices, and presents consistent generation across training iteration.

\section{Experiments}
We empirically validate PDD on three tasks: 
i) class-conditional image generation on ImageNet-256~\citep{ILSVRC15} using the SiT-XL+REPA model~\citep{yu2025repa}; 
ii) text-to-image generation using Qwen-Image~\citep{wu2025qwenimage}; and 
iii) text-to-video generation using the 1.3B and 14B variants of Wan2.1~\citep{wan2025wan2p1} as well as LTX-2.3~\citep{hacohen2026ltx2}. 
The datasets used for PDD training are described in Appendix~\ref{a:experiments}.

\vspace{-10pt}
\paragraph{Training setup.}
PDD has three main design choices: 
i) the time discretization~(\ref{e:time_disc}), defined by the grid size $N$ and the time reparameterization; 
ii) the Runge-Kutta method used to approximate the mean velocity~(\ref{e:mean_vel}); and 
iii) the minimum and maximum block sizes, $L_{\min}$ and $L_{\max}$, used in Algorithms~\ref{alg:parallel_decoding_training} and~\ref{alg:datafree_parallel_decoding_training}, which determine the set of available NFEs at inference time. 
For each task, we train two PDD models, each with a different grid size, Runge-Kutta method, and block-size range.
Additional training details are provided in Appendix~\ref{a:experiments}.

For class-conditional image generation, we set $N=128$ for the Euler method and $N=64$ for the midpoint method. 
For both models, we use the uniform time discretization~(\ref{e:time_disc}), $t_n=\frac{n}{N}$, and choose block sizes $L_{\min}$ and $L_{\max}$ such that the available NFEs at inference time are $1,2,4,8$.

For both text-to-image and text-to-video generation, we set $N=256$ for the Euler method and $N=128$ for the Midpoint method. 
All models use the \emph{shift} transformation~\citep{esser2024sd3,shaul2024bns} for the time discretization:
\begin{equation}
\label{eq:shift}
    t_n = \text{shift}_s\parr{\frac{n}{N}},\quad 
    \text{shift}_s(t) = \frac{\frac{1}{s}t}{1 + \parr{\frac{1}{s}-1}t}.
\end{equation}
Due to the lack of high-quality image and video datasets, we apply data-free training as described in Algorithm~\ref{alg:datafree_parallel_decoding_training}. 
The minimum and maximum block sizes, $L_{\min}$ and $L_{\max}$, are chosen such that the available NFEs at inference time are $2,4,8$ for Qwen-Image and Wan2.1, and $4,8$ for LTX-2.3.

Across all tasks, since the midpoint method requires two teacher evaluations, we accumulate the Euler loss over two intervals within each block for a fair comparison. Additionally, details about classifier-free guidance are in Appendix~\ref{a:experiments}.

While PDD exhibits stable convergence across hyperparameter, we find that the two most important training hyperparameters are the time reparameterization and the batch size, with larger batch sizes yielding better performance. 
Additionally, as shown in Tables~\ref{tab:qwenimage_benchmarks}, \ref{tab:qwenimage_human_pref}, \ref{tab:vbench}, the midpoint approximation consistently improves performance compared to the Euler approximation.
\begin{figure}
    \centering
    \includegraphics[width=1.0\linewidth]{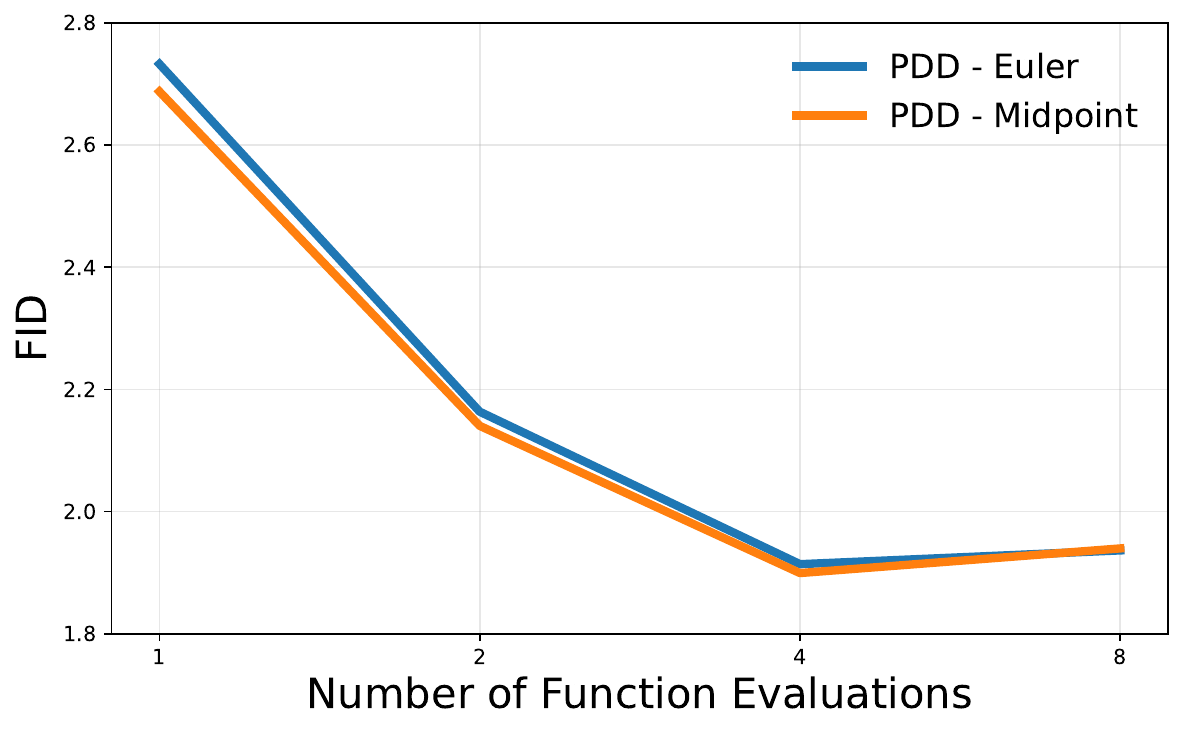}
    \caption{FID vs.~NFE of PDD with Euler and Midpoint methods for approximating the mean velocity~(\ref{e:mean_vel}) on ImageNet-256 using SiT-XL+REPA model as teacher with guidance scale 2.9.}
    \label{fig:imagenet_euler_vs_midpoint}
    \vspace{-10pt}
\end{figure}
\vspace{-8pt}
\paragraph{Class-conditional ImageNet} 
As shown in Figure~\ref{fig:imagenet_euler_vs_midpoint}, the FID of our PDD models generally improves as the NFE increases, validating that PDD successfully shares weights across different NFEs. Although the $8$-NFE setting exhibits an increase in FID, we find that this degradation can be mitigated by using a lower guidance scale, at the cost of higher FID at lower NFEs; see Figure~\ref{fig:imagenet_scale_ablation} in the appendix. 
\setlength{\intextsep}{4pt}
\setlength{\columnsep}{8pt}
\begin{wraptable}{r}{0.52\columnwidth}
\vspace{-0.2em}
\centering
\footnotesize
\caption{FID with NFE$=1$ on ImageNet-256 using SiT-XL+REPA as teacher. Guidance scale $w=2.9$.}
\vspace{-0.6em}
\label{tab:imagenet}
\resizebox{\linewidth}{!}{%
\begin{tabular}{lc}
\toprule
Method  & FID  \\
\midrule
Pi-Flow~\citep{chen2025piflow}  &  2.85   \\
FreeFlow~\citep{tong2025freeflow}  & \cellcolor{best}1.45 \\
PDD - Euler  & 2.73   \\
PDD - Midpoint & 2.69 \\
\bottomrule
\end{tabular}
}
\end{wraptable}
Table~\ref{tab:imagenet} compares PDD with the related Pi-Flow method~\citep{chen2025piflow} and the SOTA FreeFlow method~\citep{tong2025freeflow}. Our PDD method achieves a very competitive FID in the single-step setting while (i) having a simpler objective (without Gaussian mixture as Pi-Flow or an additional network as FreeFlow) and (ii) also supporting inference with multiple NFE budgets.
\vspace{-8pt}
\paragraph{Text-to-image Qwen-Image}
\begin{table}[t]
\centering
\caption{
Qwen-Image on the OneIG-EN, DPG-Bench, and GenEval. $^*$ denotes our re-evaluation of the official checkpoint. This table reports the overall metrics. The full dimensions of the three benchmarks are in tables \ref{tab:oneig-en}, \ref{tab:dpg_bench}, and \ref{tab:geneval} in Appendix~\ref{aa:qwenimage}.
}
\vspace{-5pt}
\label{tab:qwenimage_benchmarks}
\resizebox{0.5\textwidth}{!}{%
\begin{tabular}{lcccc}
\toprule
Method & NFE & OneIG-EN$\uparrow$ & DPG-Bench$\uparrow$ & GenEval$\uparrow$  \\
\midrule
\color{black!55}Euler$^*$ &  \color{black!55}50$\times2$ & \color{black!55}0.537 & \color{black!55}88.30 & \color{black!55} 0.86  \\ 
\midrule
 TwinFlow$^*$~\citep{cheng2026twinflow} & \multirow{3}{*}{2} & 0.493 & 86.67 & 0.82  \\
 PDD - Euler (Ours) &  & 0.508 & 88.04  & \cellcolor{best}0.86  \\
 PDD - Midpoint (Ours) &  & \cellcolor{best}0.516 & \cellcolor{best}88.10 & \cellcolor{best}0.86 \\
\midrule
DMD2$^*$ (Lightning-step4-v2)~\citep{lightx2v} & \multirow{5}{*}{4} & 0.524 & 88.25 & 0.85 \\
TwinFlow$^*$~\citep{cheng2026twinflow} &  & 0.502 & 86.18 & 0.82 \\
Pi-Flow$^*$~\citep{chen2025piflow} &  & 0.533 & 88.11 & 0.85 \\
PDD - Euler (Ours) &  & 0.535 & 88.45  & \cellcolor{best}0.86 \\
PDD - Midpoint (Ours) &  & \cellcolor{best}0.538 & \cellcolor{best}88.66  & \cellcolor{best}0.86 \\
\midrule
DMD2$^*$ (Lightning-step8-v2)~\citep{lightx2v} & \multirow{4}{*}{8} & 0.526 & 88.20 & 0.84 \\
Pi-Flow$^*$~\citep{chen2025piflow} &  & 0.536 & 87.90  & 0.84 \\
PDD - Euler (Ours) &  & 0.538 & \cellcolor{best}88.51  & \cellcolor{best}0.86 \\
PDD - Midpoint (Ours) &  & \cellcolor{best}0.541 & 88.46  & 0.85 \\
\bottomrule
\end{tabular}
}
\vspace{-5pt}
\end{table}
\begin{table}[t]
\centering
\caption{
Qwen-Image on HPSv2, PickScore, and OneIG diversity. $^*$ denotes our re-evaluation of the official checkpoint.
}
\vspace{-5pt}
\label{tab:qwenimage_human_pref}
\resizebox{0.5\textwidth}{!}{%
\begin{tabular}{lcccc}
\toprule
Method & NFE & HPSv2$\uparrow$ & PickScore$\uparrow$ & OneIG diversity$\uparrow$  \\
\midrule
\color{black!55} Euler$^*$ &  \color{black!55}50$\times2$ & \color{black!55}30.83 & \color{black!55}22.78 & \color{black!55} 0.200  \\ 
\midrule
TwinFlow$^*$~\citep{cheng2026twinflow} & \multirow{3}{*}{2} & 29.86 & 22.26 & 0.131  \\
PDD - Euler (Ours) &  & 29.59 & 22.47  & \cellcolor{best}0.197  \\
PDD - Midpoint (Ours) &  & \cellcolor{best}30.15 & \cellcolor{best}22.66 & 0.177 \\
\midrule
DMD2$^*$ (Lightning-step4-v2)~\citep{lightx2v} & \multirow{5}{*}{4} & \cellcolor{best}32.34 & \cellcolor{best}22.98 & 0.095 \\
TwinFlow$^*$~\citep{cheng2026twinflow} &  & 30.01 & 22.26 & 0.150 \\
Pi-Flow$^*$~\citep{chen2025piflow} &  & 30.94 & 22.67 & 0.182 \\
PDD - Euler (Ours) &  & 31.05 & 22.72  & \cellcolor{best}0.192 \\
PDD - Midpoint (Ours) &  & 31.33 & 22.86  & 0.174 \\
\midrule
DMD2$^*$ (Lightning-step8-v2)~\citep{lightx2v} & \multirow{4}{*}{8} & \cellcolor{best}32.35 & \cellcolor{best}22.95 & 0.109 \\
Pi-Flow$^*$~\citep{chen2025piflow} &  & 31.09 & 22.55  & 0.186 \\
PDD - Euler (Ours) &  & 31.34 & 22.73  & \cellcolor{best}0.198 \\
PDD - Midpoint (Ours) &  & 31.56 & 22.86  & 0.181 \\
\bottomrule
\end{tabular}
}
\vspace{-12pt}
\end{table}

We evaluate our PDD-distilled Qwen-Image models on three main benchmarks, OneIG~\citep{chang2025oneig}, DPG-Bench~\citep{hu2024ella}, and GenEval~\citep{ghosh2023geneval}, with $\text{NFE}=2,4,8$. For SOTA baselines, we consider QwenLightning-v2~\citep{lightx2v}, which employs DMD2~\citep{yin2024dmd2} distillation, Pi-Flow, and TwinFlow~\cite{cheng2026twinflow}. For all three baselines, we report re-evaluations of the official checkpoints. Table~\ref{tab:qwenimage_benchmarks} shows that PDD achieves the best performance on the overall metrics of the three benchmarks. The full dimensions are provided in tables~\ref{tab:oneig-en}, \ref{tab:dpg_bench}, and \ref{tab:geneval} in the appendix. Additionally, we evaluate the models on the HPSv2~\citep{wu2023hspv2} benchmark as well as measure PickScore~\citep{wu2023hspv2} on the same generated images. Table~\ref{tab:qwenimage_human_pref} shows our PDD is the runner up to DMD2 (Lightning-v2) model in terms of these human preference metrics. However, DMD suffers from mode collapse, resulting in a significant reduction of diversity. This is observed by the OneIG diversity metrics in Table~\ref{tab:qwenimage_human_pref} and on many examples in Figures~\ref{fig:qwenimage_diversity_first} to \ref{fig:qwenimage_diversity_last}. In contrast, our PDD shows competitive results while better preserving diversity and more closely following the teacher generation.

\begin{table*}[t]
\centering
\caption{Wan Text-to-Video 1.3B and 14B models performance and diversity metrics on VBench~\citep{huang2023vbench} with Self-Forcing prompt set. Diversity is measured as mean pairwise V-JEPA 2/VideoMAE V2 feature distance across 5 generated videos per prompt. $^*$ Our re-evaluation of the official checkpoint. $^{**}$ Our re-implementation. $^\dagger$ Reported numbers on proprietary prompts without diversity evaluation, since rCM~\citep{zheng2025rcm} did not release checkpoints.}
\vspace{-5pt}
\label{tab:vbench}
\resizebox{\textwidth}{!}{%
\begin{tabular}{clcccccccc}
\toprule
\multirow{2}{*}{Model} & \multirow{2}{*}{Method} & \multirow{2}{*}{NFE} & \multicolumn{3}{c}{VBench} & \multicolumn{2}{c}{V-JEPA 2} & \multicolumn{2}{c}{VideoMAE V2} \\
\cmidrule(lr){4-6} \cmidrule(lr){7-8} \cmidrule(lr){9-10}
 & & & Overall $\uparrow$ & Quality  $\uparrow$& Semantic $\uparrow$& Cosine $\uparrow$ & L2 $\uparrow$ & Cosine $\uparrow$ & L2 $\uparrow$ \\
\midrule
\multirow{6}{*}{1.3B} 
 & \color{black!55} UniPC$^*$ (Teacher) & \color{black!55} 50$\times$2 & \color{black!55} 83.77 & \color{black!55} 84.90 & \color{black!55} 79.22 & \color{black!55} 0.1254 & \color{black!55} 27.07 & \color{black!55} 0.02681 & \color{black!55} 2.922 \\ 
 & rCM$^\dagger$~\citep{zheng2025rcm} & \multirow{5}{*}{4} & 84.43 & 85.38 & 80.63 & - & - & - & - \\
 & AnyFlow$^*$~\citep{gu2026anyflow} &  & 84.45 & 85.22 & \cellcolor{best}81.34 & 0.0704 & 19.88 & 0.01029 & 1.807 \\
 & DMD$2^{**}$ (FastGen~\citep{fastgen2026}) &  & 84.69 & 86.14 & 78.87 & 0.0833 & 21.83 & 0.01646 & 2.278 \\
 & PDD - Euler (Ours) &  & 84.44 & 85.99 & 78.22 & 0.1018 & 24.54 & 0.01901 & 2.489 \\
 & PDD - Midpoint (Ours) &  & \cellcolor{best}84.94 & \cellcolor{best}86.45 & 78.91 & \cellcolor{best}0.1032 & \cellcolor{best}24.63 & \cellcolor{best}0.02054 & \cellcolor{best}2.548 \\
\midrule
\multirow{9}{*}{14B} 
 & \color{black!55} UniPC$^*$ (Teacher) & \color{black!55} 50$\times$2 & \color{black!55} 83.90 & \color{black!55} 84.56 & \color{black!55} 81.24 & \color{black!55} 0.1263 & \color{black!55} 27.27 & \color{black!55} 0.02497 & \color{black!55} 2.824 \\
 & rCM$^\dagger$~\citep{zheng2025rcm} & \multirow{5}{*}{4} & 84.92 & 85.43 & \cellcolor{best}82.88 & - & - & - & - \\
 & AnyFlow$^*$~\citep{gu2026anyflow} &  & \cellcolor{best}84.95 & 85.70 & 81.92 & 0.0786 & 20.99 & \cellcolor{best}0.01297 & 1.992 \\
 & DMD2$^{**}$ (FastGen~\citep{fastgen2026}) &  & 84.40 & 85.16 & 81.34 & 0.0568 & 17.67 & 0.00945 & 1.710 \\
 & PDD$^\text{short}$ - Midpoint (Ours) &  & 84.92 & \cellcolor{best}85.71 & 81.77 & 0.0791 & 21.27 & 0.01247 & 2.027 \\
 & PDD$^\text{long}$ - Midpoint (Ours) &  & 84.69 & 85.69 & 80.71 & \cellcolor{best}0.0846 & \cellcolor{best}22.13 & 0.01264 & \cellcolor{best}2.058 \\
\cmidrule{2-10}
 & AnyFlow$^*$~\citep{gu2026anyflow} & \multirow{3}{*}{8} & \cellcolor{best}85.08 & 85.78 & \cellcolor{best}82.28 & 0.0765 & 20.67 & 0.01278 & 1.974 \\
 & PDD$^\text{short}$ - Midpoint (Ours) &  & 84.96 & \cellcolor{best}85.83 & 81.44 & 0.0816 & 21.63 & 0.01276 & 2.054 \\
 & PDD$^\text{long}$ - Midpoint (Ours) &  & 84.70 & 85.77 & 80.41 & \cellcolor{best}0.0868 & \cellcolor{best}22.43 & \cellcolor{best}0.01314 & \cellcolor{best}2.097 \\
\bottomrule
\end{tabular}
}
\vspace{-10pt}
\end{table*}

\paragraph{Text-to-video Wan2.1} Our results on the text-to-video task provide strong evidence that PDD is effective at large scale. We evaluate our distilled models on the VBench benchmark~\citep{huang2023vbench} and compare against the SOTA baselines rCM~\citep{zheng2025rcm}, AnyFlow~\citep{gu2026anyflow}, and DMD2 (implemented in FastGen~\citep{fastgen2026}), which all incorporate a VSD loss. Additionally, we measure diversity by encoding the generated videos for VBench using V-JEPA 2~\citep{assran2025vjepa2} and VideoMAE V2~\citep{wang2023videomaev2} models, and report average pair-wise cosine and L2 distance across $5$ samples for each prompt. We find that skipping a layer in the backbone for the unconditional term in the CFG~(\ref{e:vel_guided}) can improve the performance of PDD. For the Wan2.1 1.3B model we skip the 10-th layer and for the 14B model we skip the 12-th layer.

As shown in Tables~\ref{tab:vbench}, on the Wan2.1 1.3B model PDD ranks first in terms of video quality, overall score, as well as diversity. On the Wan2.1 14B model we report on two checkpoints, \emph{short} which stands for 200 training iterations and \emph{long} which stands for 3k iterations. We find that PDD (short) achieves best video quality and is the runner-up to AnyFlow in the overall metrics with both $4$ and $8$ NFE. We find that, in general, PDD-generated videos exhibit a higher degree of motion (see Figure~\ref{fig:origami_dancers} and Figures~\ref{fig:person_jogging} to \ref{fig:robot_dj_cyberpunk_tokyo}) compared to baselines. In particular, while PDD (long) obtains lower VBench scores, we find that motion in the generated videos increases later in training, which is also notable in the dynamic degree score of VBench in Figure~\ref{fig:vbench_dim_14b} in the appendix. Additionally, Table~\ref{tab:vbench} shows that PDD obtains higher diversity scores than the baselines. This is also notable in many examples shown in Figures~\ref{fig:person_jogging} to \ref{fig:robot_dj_cyberpunk_tokyo}, generated from the best VBench checkpoints.
\vspace{-8pt}
\paragraph{Text-to-video/audio LTX-2.3}
To further show the scalability of PDD, we distill the 22B LTX-2.3 model, demonstrating multimodal few-step generation of 10s videos in 720p resolution with audio. As baseline, we consider the official LTX-2.3 distilled 8-step model~\cite{hacohen2026ltx2}. We apply the PD loss to the audio and video latents separately and average the loss across both modalities. Following the standard guidelines, we apply three guidance methods for the teacher, i.e., standard CFG (scale 4.5 for video and 7 for audio), cross-modal guidance (scale 3), and spatiotemporal skip guidance (scale 2 on layer 29). Due to the higher computational costs, we only consider the Euler method (with $N=256$) over a single interval within each block. For the teacher, this leads to $4\times 30$ NFE per generation. The comparisons in Figure~\ref{fig:ltx} and Figures~\ref{fig:judge-winrate} to~\ref{fig:ltx_comparison_last} in the appendix show that PDD can provide high-quality generations at only $8$ NFE after as few as 250 iterations of training, performing on par or better than the official distilled model without any access to training data.
\begin{figure*}[t]
  \centering
  \resizebox{\textwidth}{!}{%
  \begin{tabular}{@{} r @{\hspace{6pt}} c@{\hspace{2pt}}c@{\hspace{2pt}}c@{\hspace{2pt}}c@{\hspace{2pt}}c @{}}
                & $t{=}0$s & $t{=}2.5$s & $t{=}5$s & $t{=}7.5$s & $t{=}10$s \\ \addlinespace[2pt]
    Teacher
      & \framecellw{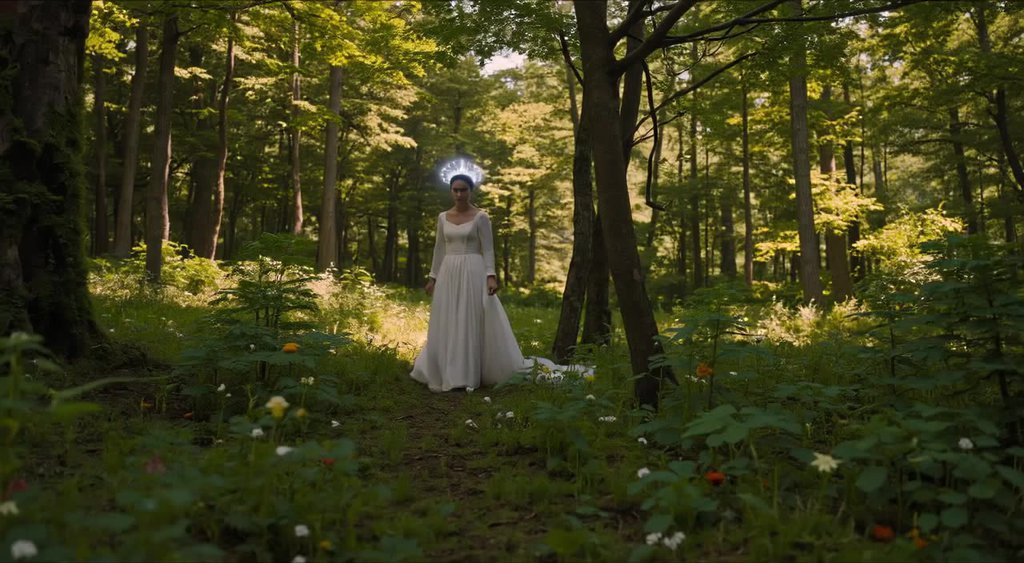} & \framecellw{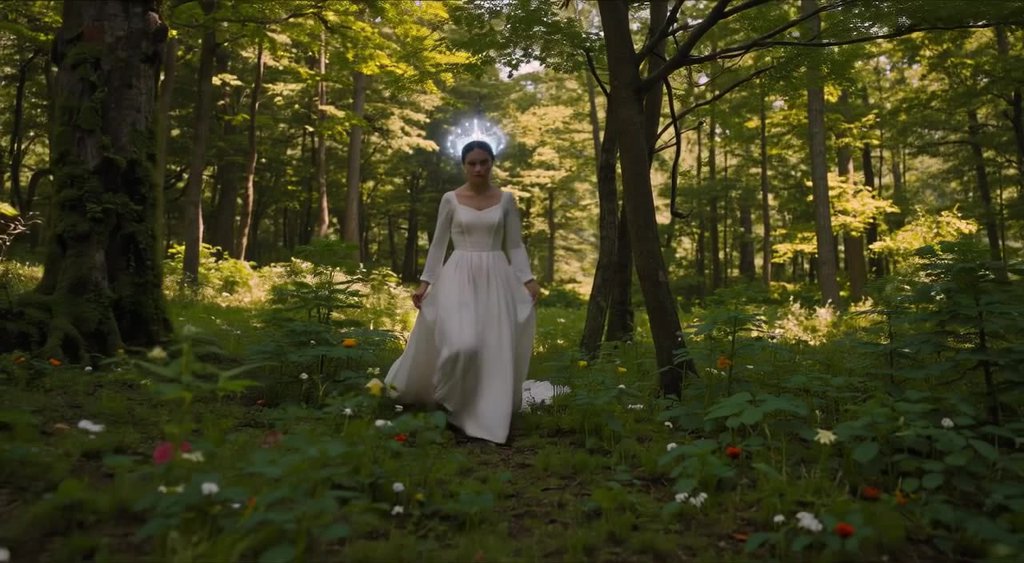} & \framecellw{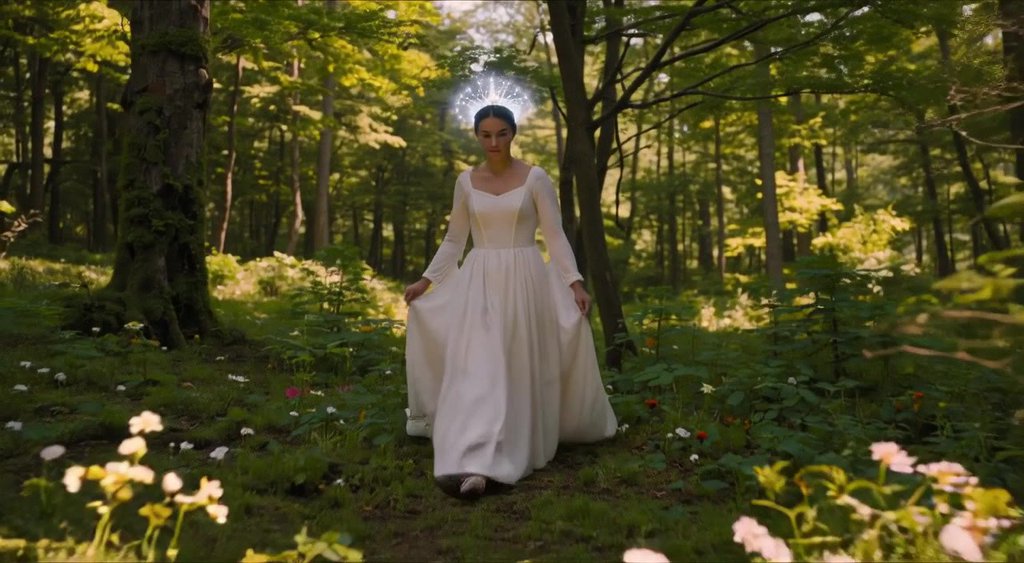} & \framecellw{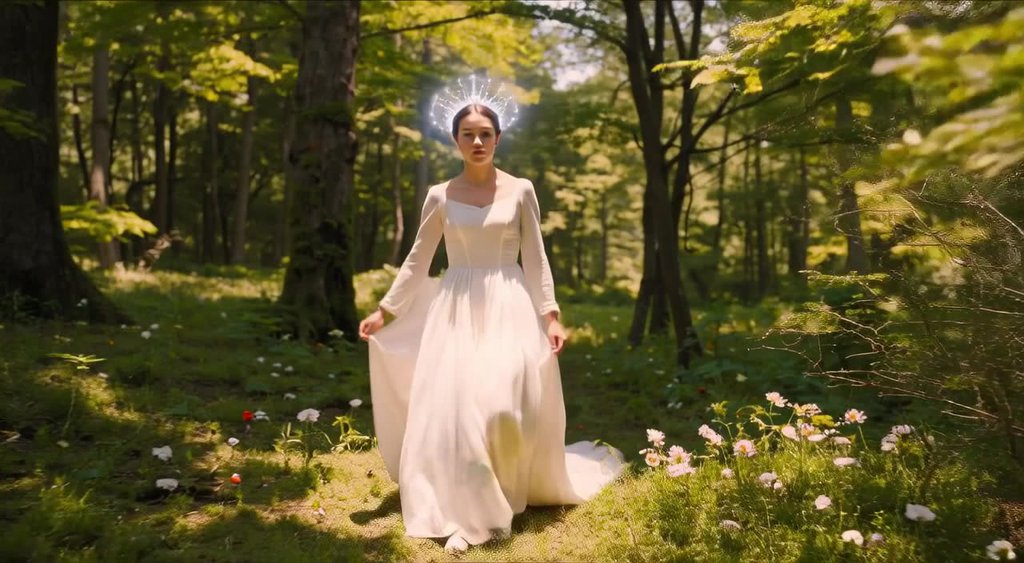} & \framecellw{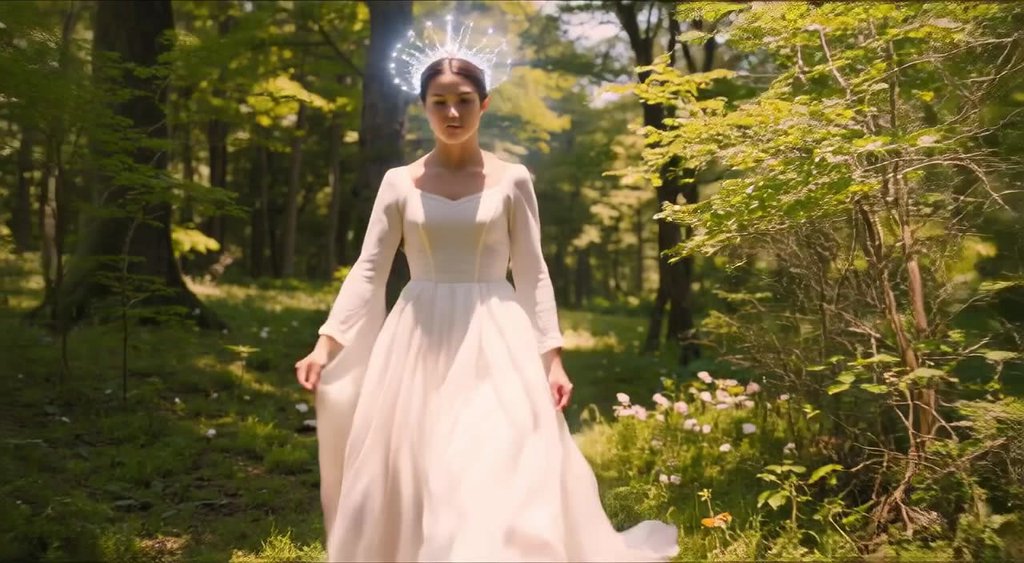} \\ \addlinespace[2pt]
    \textbf{PDD}
      & \framecellw{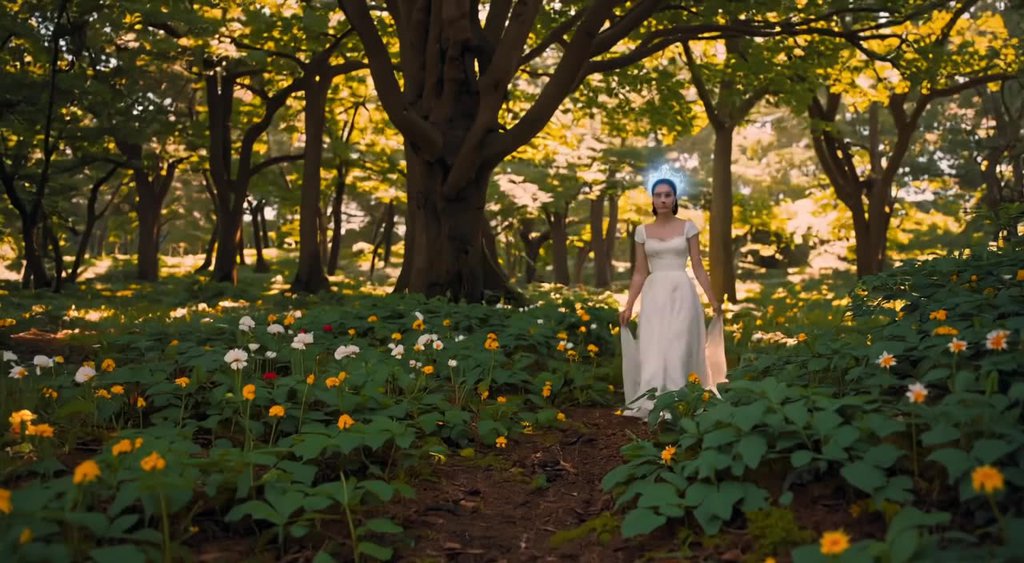} & \framecellw{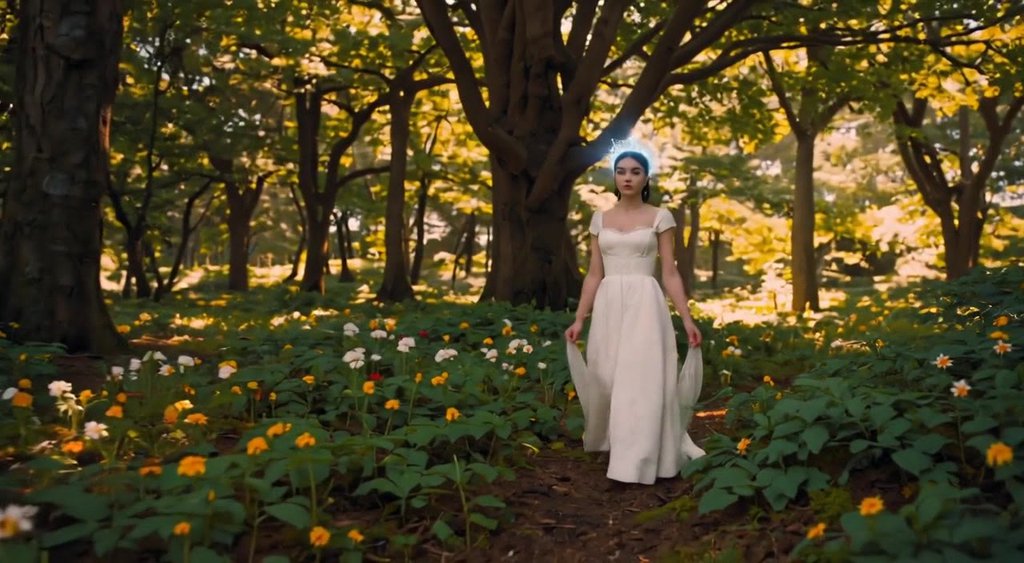} & \framecellw{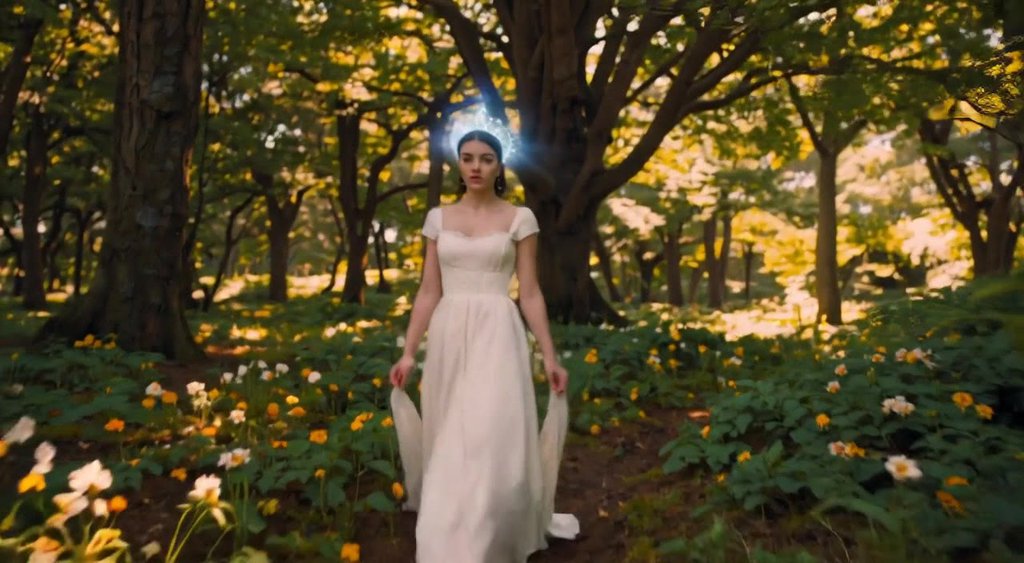} & \framecellw{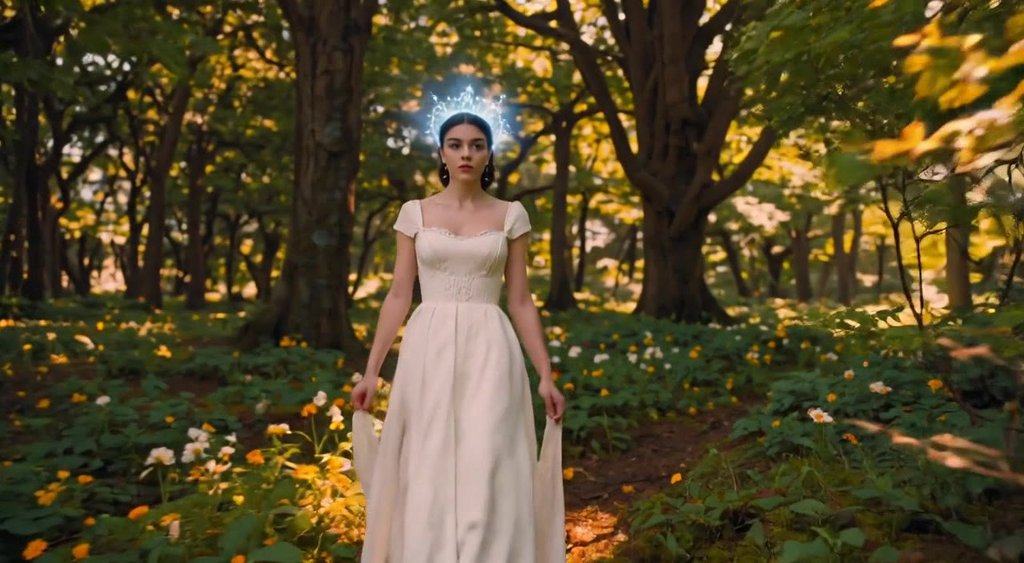} & \framecellw{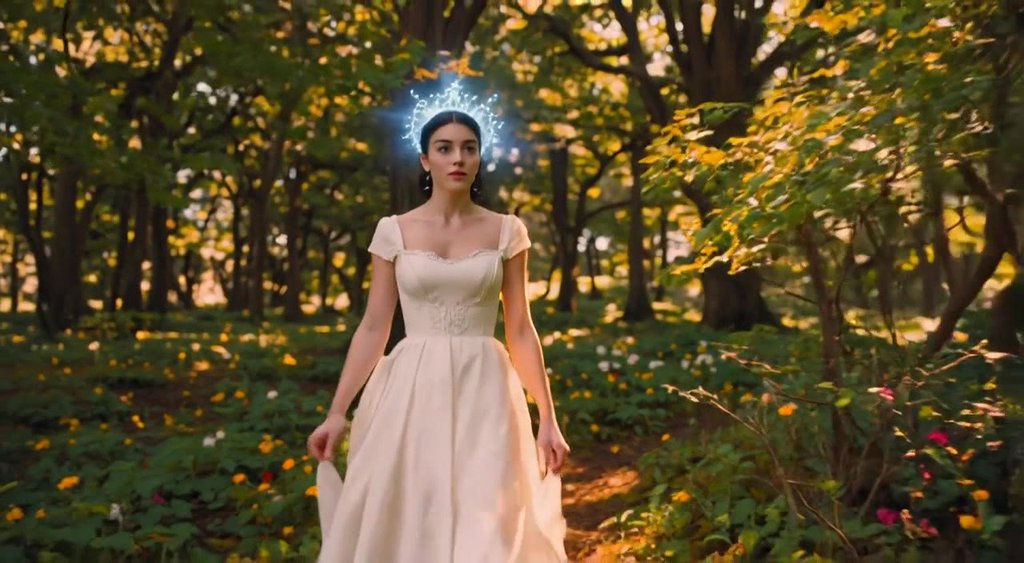} \\ \addlinespace[2pt]
    Distilled
      & \framecellw{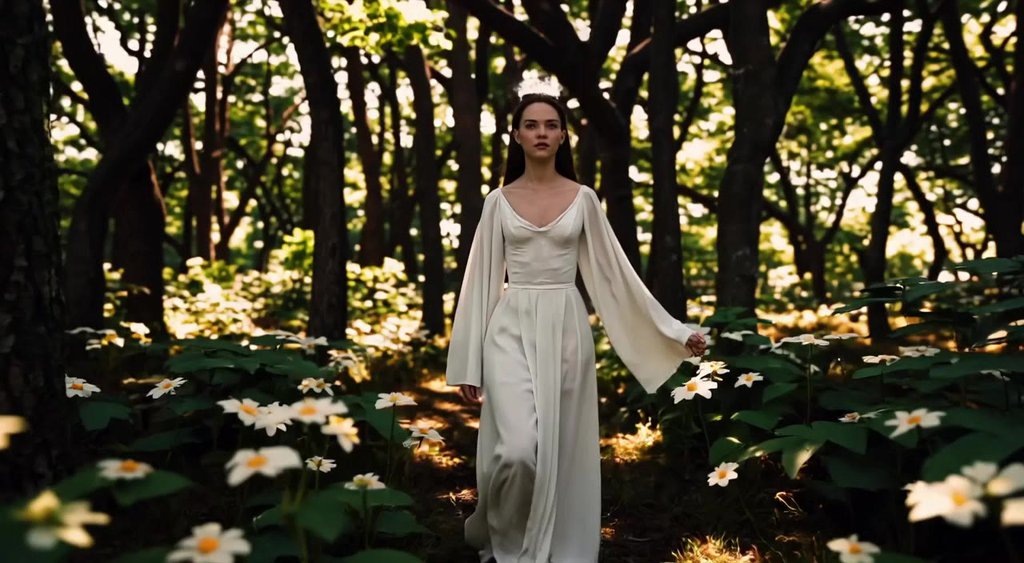} & \framecellw{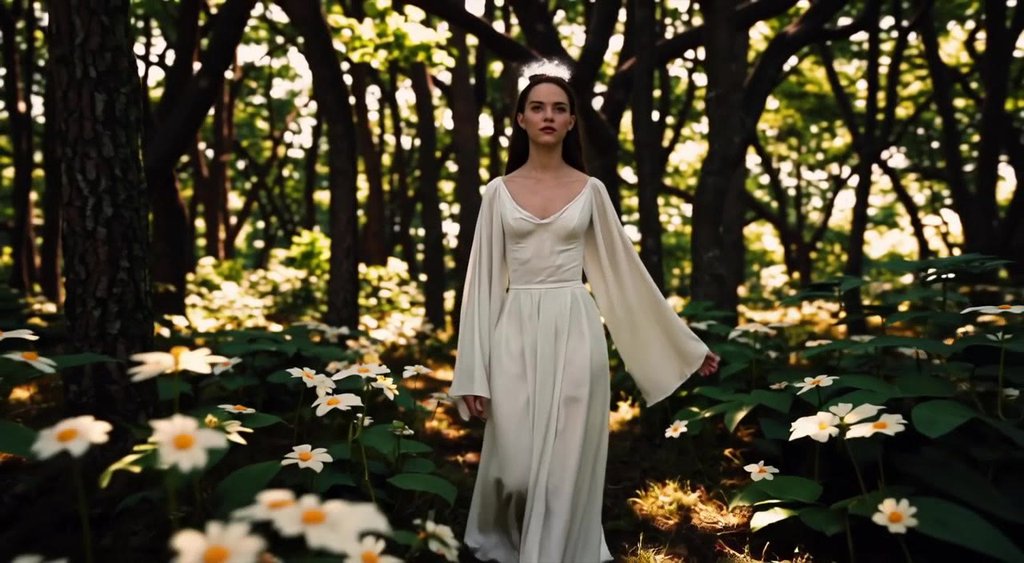} & \framecellw{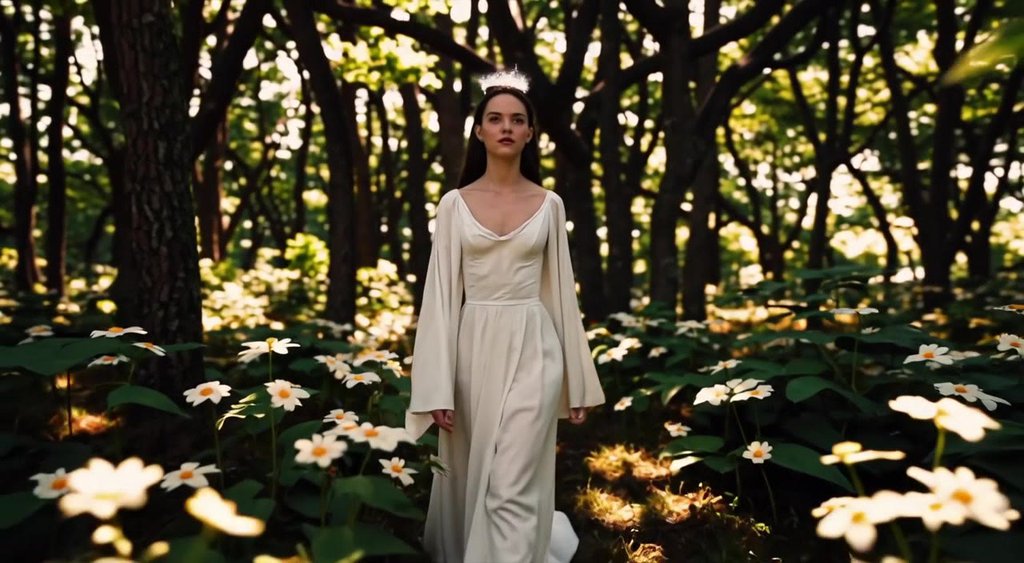} & \framecellw{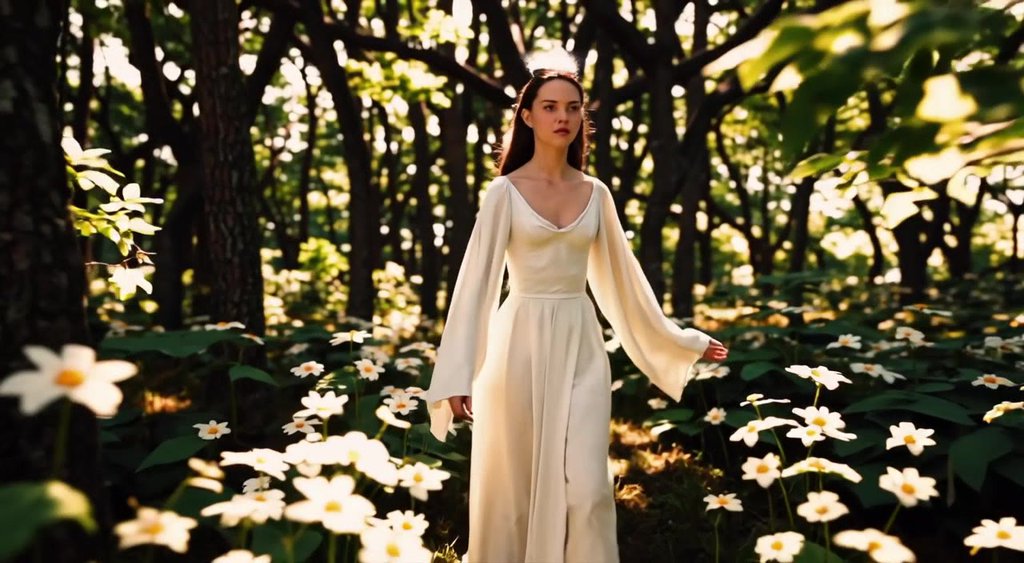} & \framecellw{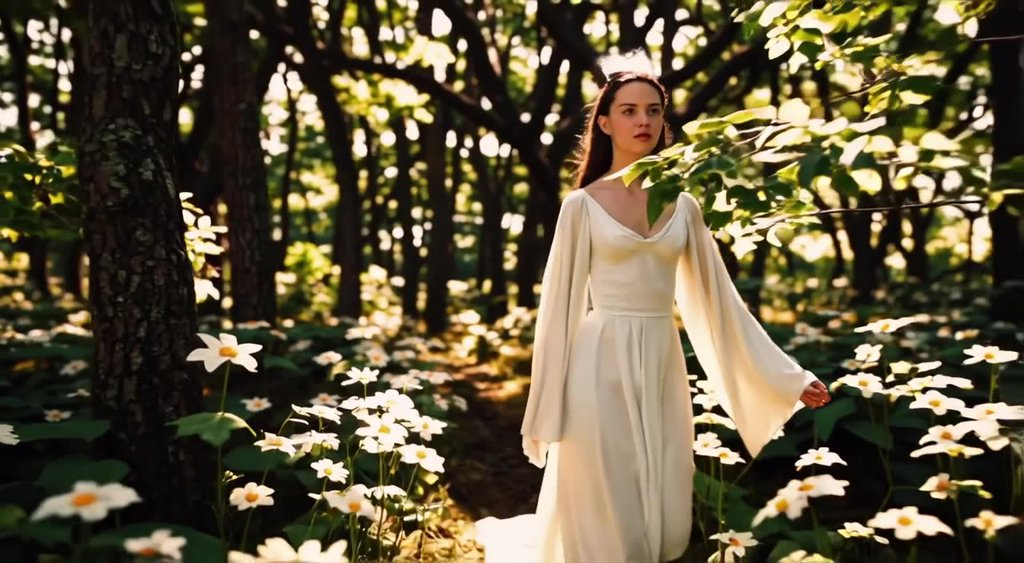} \\ \addlinespace[2pt]
                & \multicolumn{5}{@{} p{16.00cm} @{}}{\tiny\itshape ``In an enchanting forest with vibrant green foliage and dappled sunlight filtering through the trees, a young woman clad in a flowing white gown ventures on a mystical journey, her soft footsteps muffled as they press into the damp moss and wildflowers. She wears a magical crown that radiates an ethereal glow, casting shimmering light around her with a faint, crystalline chiming sound that hums in the air. Her eyes are filled with wonder and determination as she walks confidently, her arms gently swinging at her sides, the fabric of her gown rustling softly against the undergrowth. As she moves forward, the camera follows her in a smooth tracking shot, capturing her every step amidst the gentle, rhythmic chirping of forest birds and the subtle, melodic swell of ambient ethereal strings. The scene is bathed in a warm, magical glow, creating an atmosphere of enchantment and adventure, accompanied by the light, airy whisper of a breeze moving through the leaves. 4K resolution, 16:9 aspect ratio.''} \\
  \end{tabular}%
  }
  \resizebox{\textwidth}{!}{%
  \begin{tabular}{@{} r @{\hspace{6pt}} c@{\hspace{2pt}}c@{\hspace{2pt}}c@{\hspace{2pt}}c@{\hspace{2pt}}c @{}}
                & $t{=}0$s & $t{=}2.5$s & $t{=}5$s & $t{=}7.5$s & $t{=}10$s \\ \addlinespace[2pt]
    Teacher
      & \framecellw{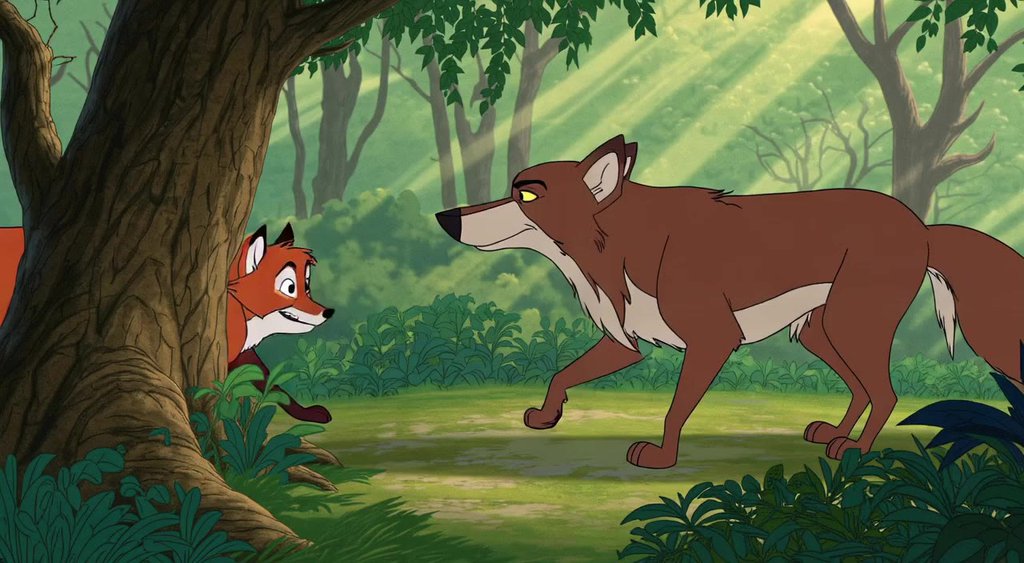} & \framecellw{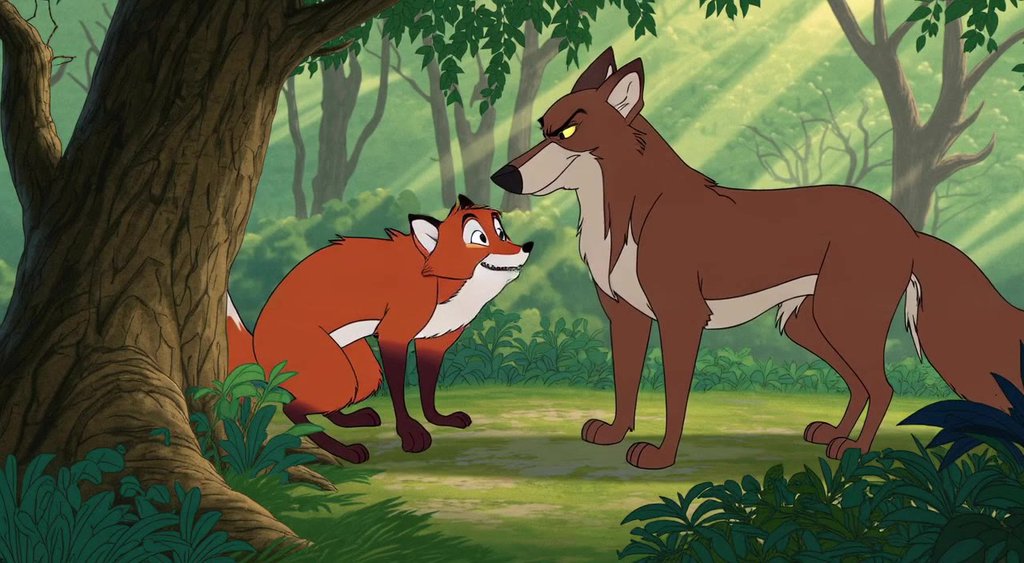} & \framecellw{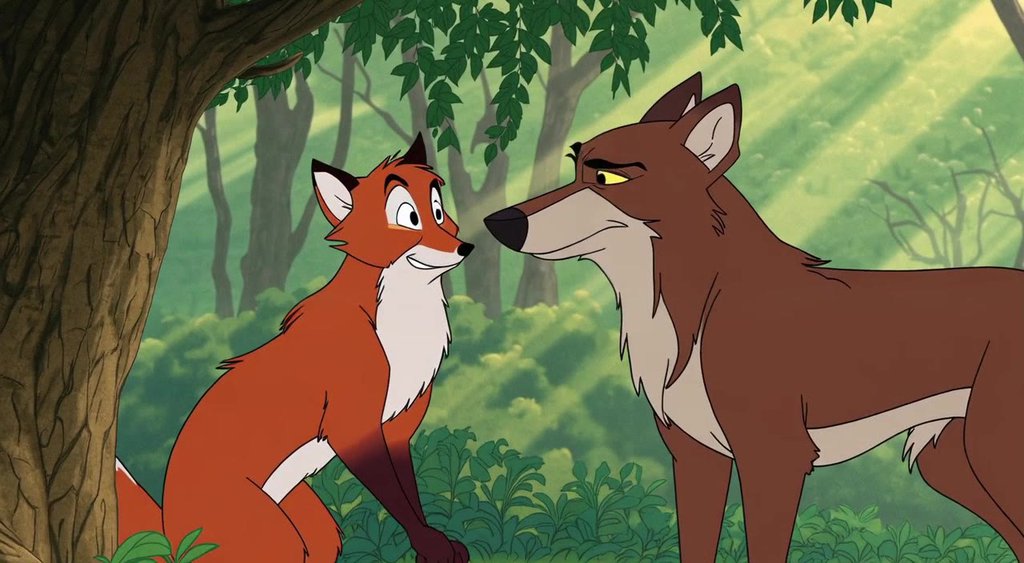} & \framecellw{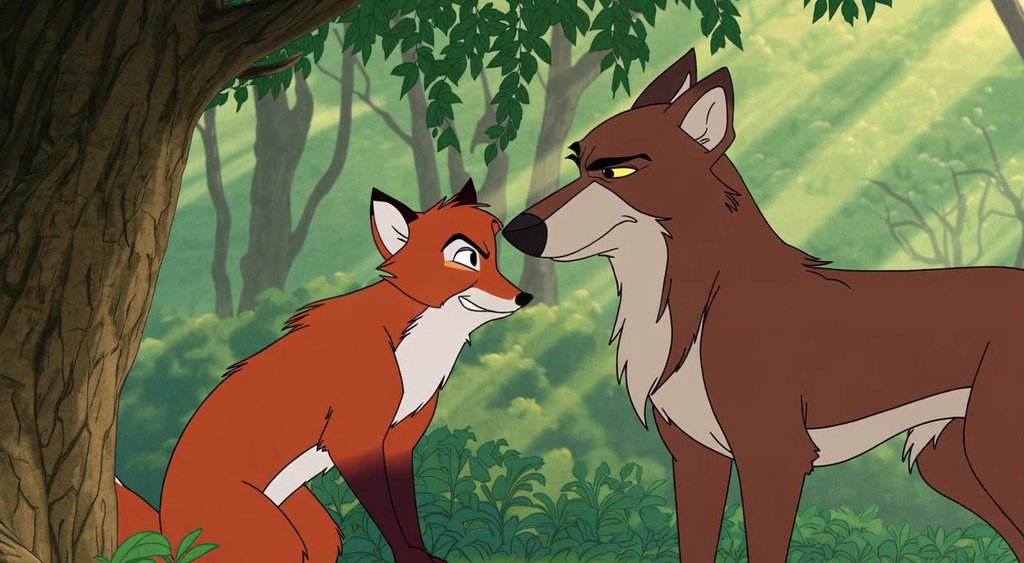} & \framecellw{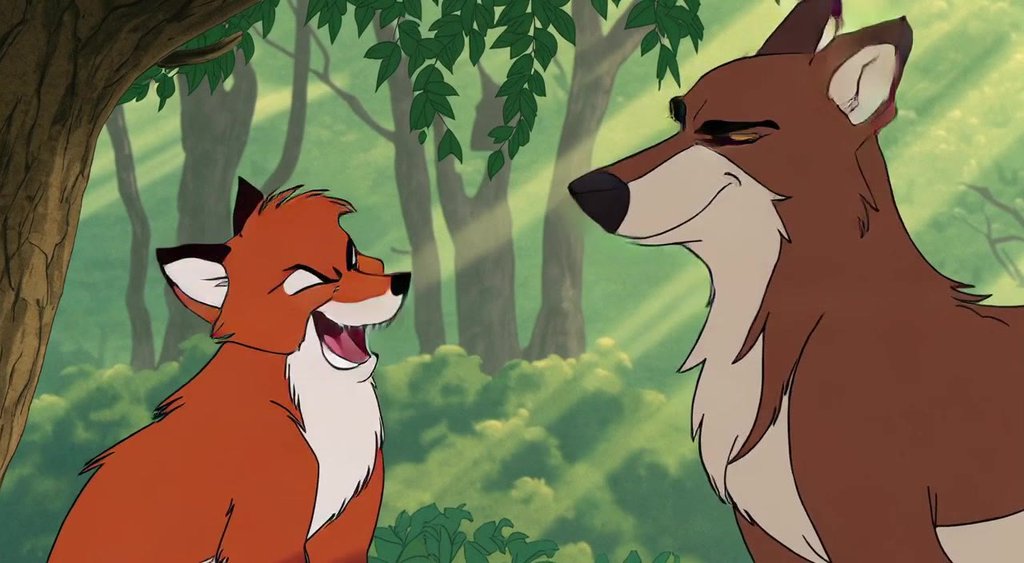} \\ \addlinespace[2pt]
    \textbf{PDD}
      & \framecellw{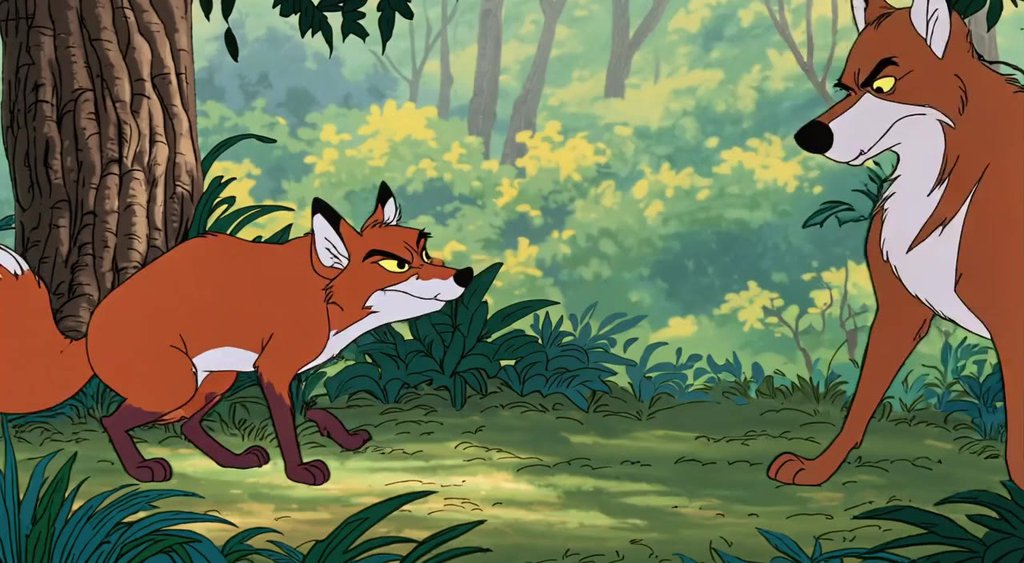} & \framecellw{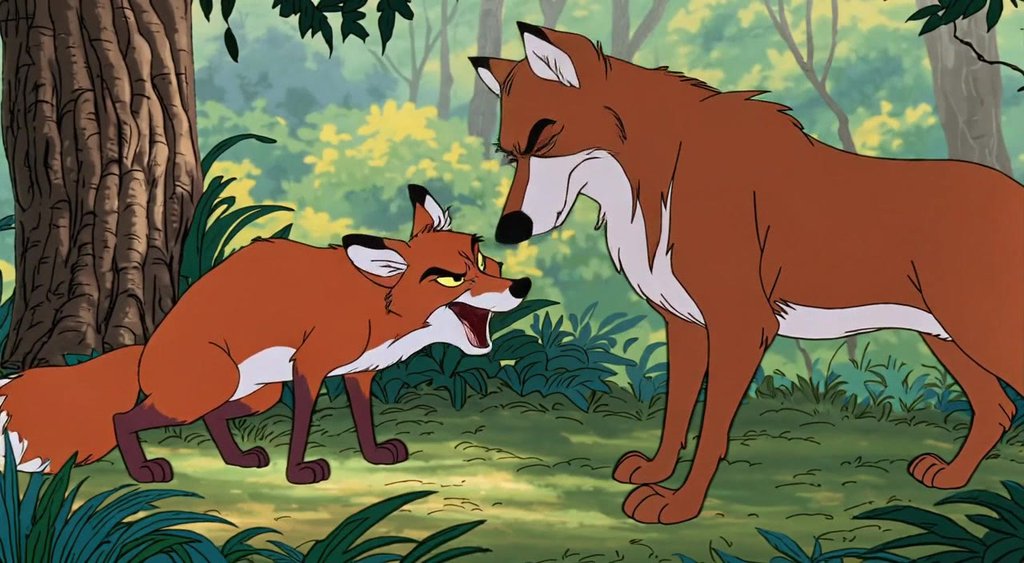} & \framecellw{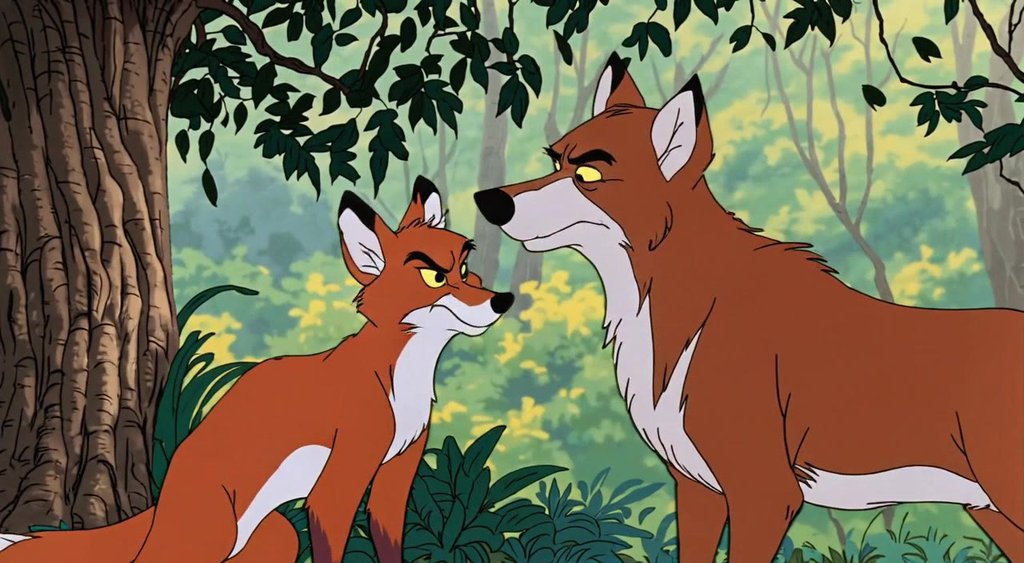} & \framecellw{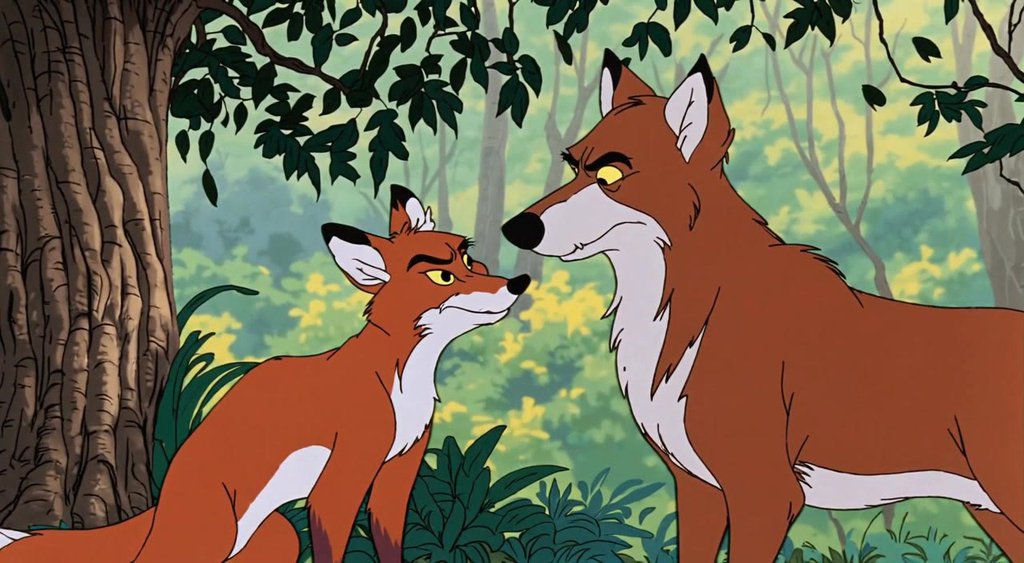} & \framecellw{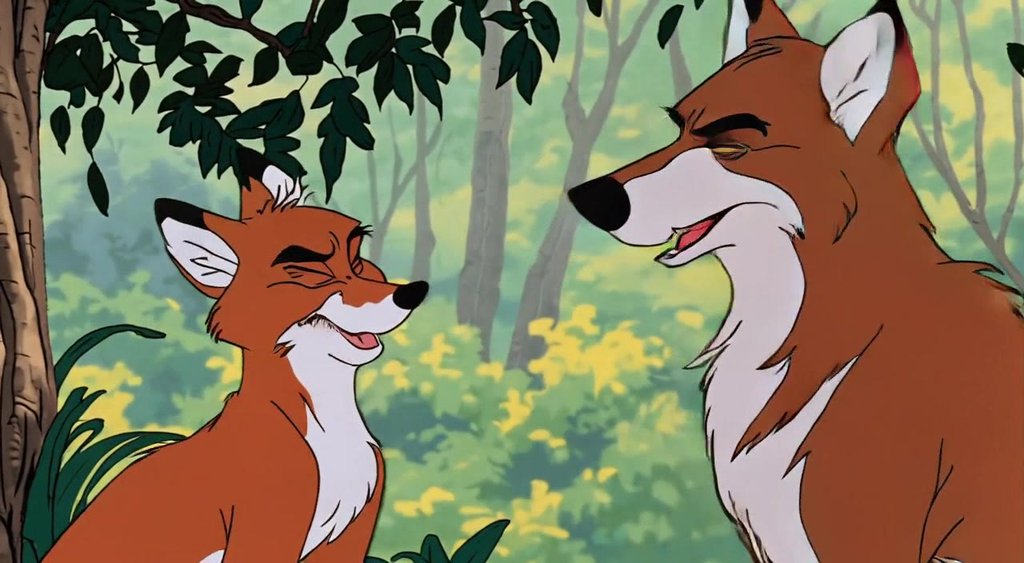} \\ \addlinespace[2pt]
    Distilled
      & \framecellw{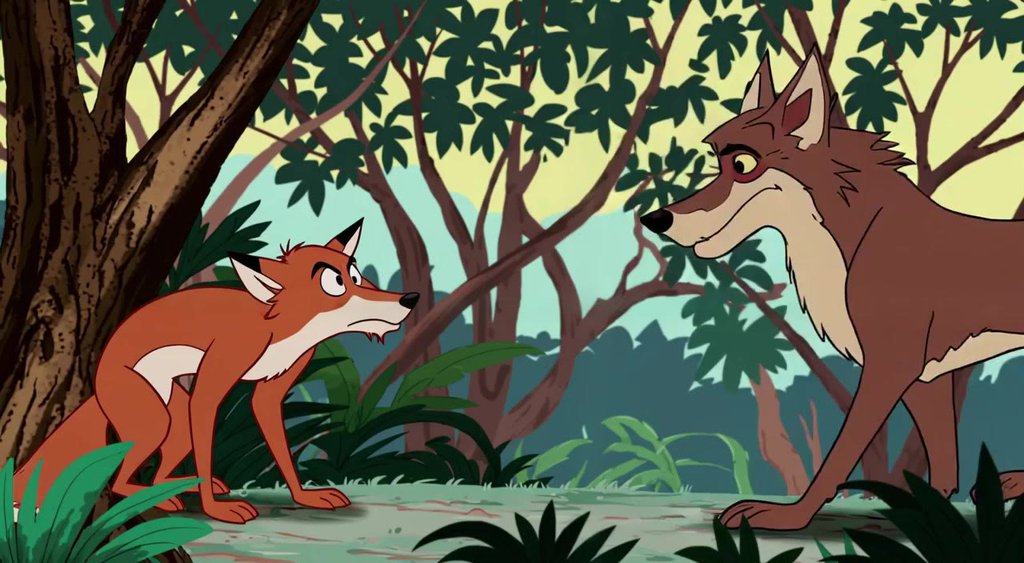} & \framecellw{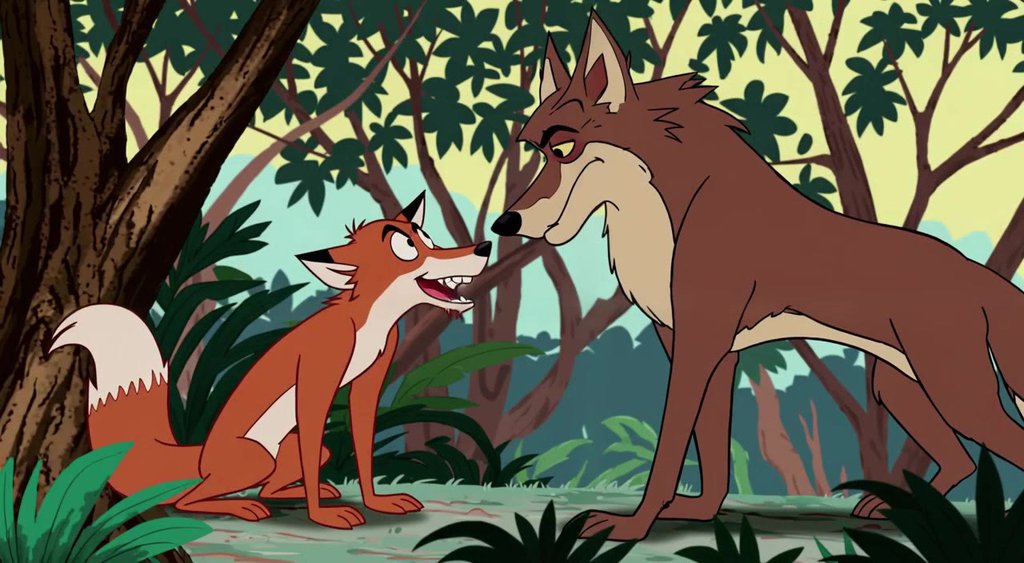} & \framecellw{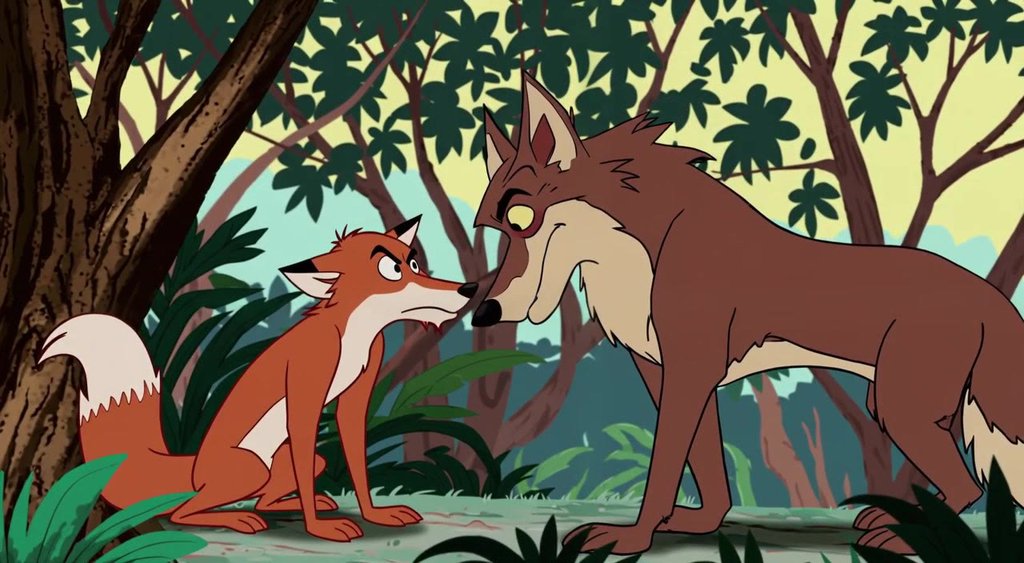} & \framecellw{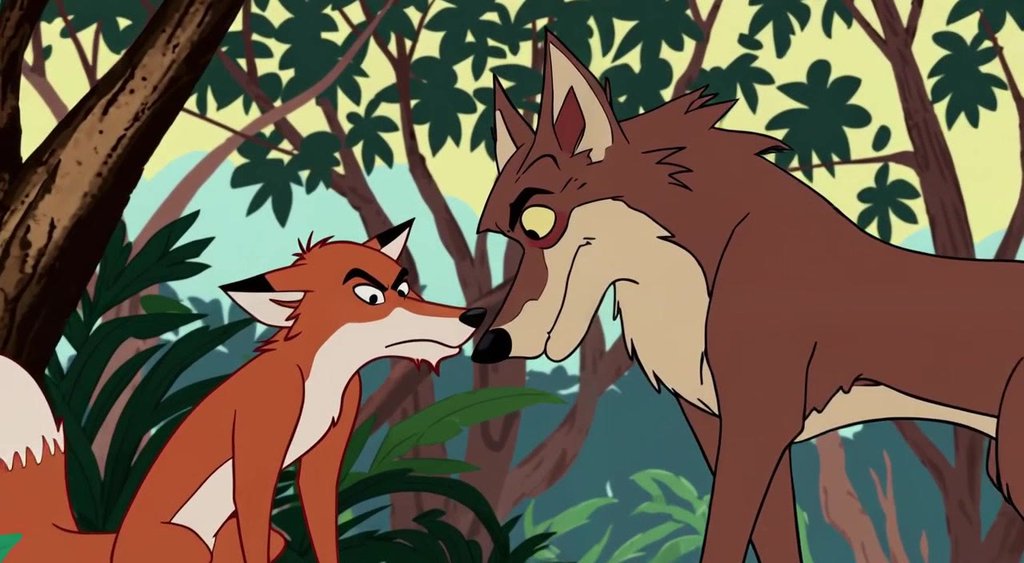} & \framecellw{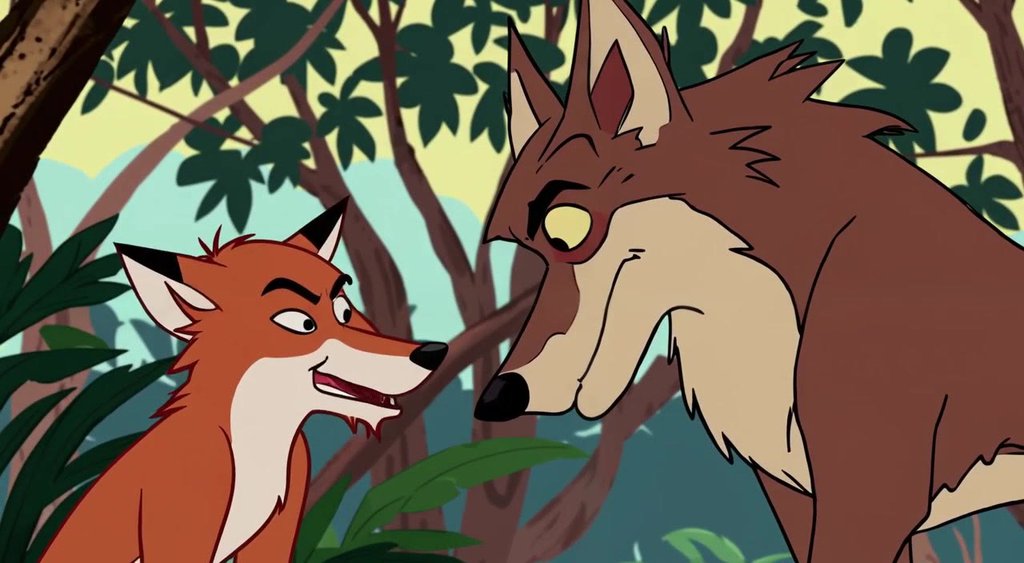} \\ \addlinespace[2pt]
                & \multicolumn{5}{@{} p{16.00cm} @{}}{\tiny\itshape ``A whimsical animated children's story set in a lush forest, featuring a cunning fox and a sly wolf who often conspire to deceive each other, accompanied by the gentle rustling of leaves and the soft chirping of woodland birds. In this scene, the fox, with its bushy tail and mischievous eyes, crouches behind a tree, whispering in a raspy, playful tone, "I have the perfect plan," to the wolf who stands tall with a skeptical look on his face, his ears twitching as he grumbles back, "I doubt it." Both animals are illustrated in bright, cheerful colors typical of children's books, with expressive faces and dynamic poses. The forest background is filled with vibrant green foliage and dappled sunlight filtering through the trees, while the ambient sound of a light breeze whistles softly through the branches. The camera zooms in on their interaction, capturing their playful yet scheming expressions as the fox lets out a soft, conspiratorial giggle.''} \\
  \end{tabular}%
  }
  \vspace{-5pt}
  \caption{LTX-2.3 teacher with $4\times 30$ NFE compared to PDD (ours) and the official distilled LTX-2.3 model, both using 8 NFE.}
  \label{fig:ltx}
  \vspace{-15pt}
\end{figure*}
\vspace{-1pt}
\section{Conclusion}
In this work, we introduced Parallel Decoding Distillation (PDD), a simple trajectory-based distillation method for accelerating diffusion and flow models. Rather than collapsing many denoising steps into a single large update, PDD trains a parallel decoder to predict the mean velocities of multiple consecutive intervals in one network evaluation. This formulation leads to a practical training objective that avoids VSD and adversarial losses, as well as JVPs, finite differences, and multi-stage distillation procedures. Across text-to-image/video/audio generation, PDD achieves strong few-step performance on large-scale models such as Qwen-Image, Wan2.1, and LTX-2.3, while better preserving sample diversity and motion compared to distribution-based baselines. This makes PDD the first pure trajectory-based distillation method for few-step, high-resolution video generation.

Because our large-scale text-to-image and video experiments rely on data-free training, investigating PDD in data-dependent settings beyond ImageNet-256 remains future work. In addition, PDD introduces a flexible inference-time design choice: the block size can be selected after evaluating the parallel decoder. This suggests the possibility of adaptive block-size selection, where a verifier or confidence criterion determines how aggressively to skip intervals during generation. Finally, generalizing the parallel decoding principle to discrete autoregressive models may broaden the applicability of PDD beyond diffusion and flow-based generation.
\FloatBarrier

\clearpage
\bibliographystyle{plainnat}  %
\bibliography{paper}  %

\clearpage %
\onecolumn %
\appendix
\section{Algorithms}
\begin{algorithm}[H]
\caption{Data-free parallel decoding loss}
\label{alg:datafree_parallel_decoding_training}
\begin{lstlisting}[language=ModernPython]
# student - parallel decoder model
# teacher - pre-trained flow model
# runge_kutta - approximates teacher's mean velocity
# t - time discretization
# L_min - minimal block size
# L_max - maximal block size
# x_n - carried state from previous iteration
# n - carried time index from previous iteration

# grid size
N = len(t) - 1
# reset state and time index
if n==N:
    x_n = randn_like(x_n)
    n=0
# parallel predictions 
u = student(x_n, t[n])
# step sizes
h = diff(t, dim=0)
# sample index in the block
k = randint(n, clip(n+L_max, max=N))
# step from n to k
x_k = x_n + einsum('l,l...', h[n:k], u[n:k])
# teacher mean velocity
u_k = runge_kutta(teacher, x_k, t[k], h[k])
# mse loss
loss = mse_loss(u[k], u_k.detach())
# advance state for next iteration
x_n = x_n + einsum('l,l...', h[n:n+L_min], u[n:n+L_min])
# advance time index for next iteration
n = n + L_min

return loss, x_n.detach(), n

\end{lstlisting}
\end{algorithm}
\section{Experiments}
\label{a:experiments}
\paragraph{Datasets} We use a distinct dataset and VAE for each model for our PDD training. For the Repa-ImageNet-256, we use the ImageNet~\citep{ILSVRC15} dataset and the Stable Diffusion VAE~\citep{rombach2022SD1}. For Qwen-Image, we employ the data-free PDD training using the text prompts set provided by Pi-Flow~\citep{chen2025piflow} with native Qwen-Image VAE on resolution $1024 \times 1024$. Similarly, for Wan2.1 and LTX-2.3 models we use their native VAEs and data-free PDD training. For the Wan2.1 models, we use a set of prompts extracted from reshuffled ViMix-14M~\citep{yang2025vimix14} on resolution $480 \times 832$ and $5s$ duration with $16$ FPS. For LTX-2.3, we use a mixture of prompts from ViMix-14M~\citep{yang2025vimix14} and VidProm~\citep{wang2024vidprommillionscalerealpromptgallery}, enhanced to include audio descriptions, on resolution $704 \times 1280$ and $10s$ duration with $24$ FPS.

\paragraph{Architecture} For all our PDD models we use the exact same backbone as the teacher model. Additionally, as described in Section~\ref{s:pdd} and Figure~\ref{fig:architecture}, we enlarge the final linear layer by repeating the channel dimension $N$ times, i.e., the grid size. Importantly, the repeat operation must be applied to the correctly reshaped weights, such that the resulting linear layer is equivalent to initializing all parallel steps with the final layer of the pretrained model. For LTX-2.3, we apply this idea to both the final linear layer of the video and audio tower.

\paragraph{Classifier free guidance (CFG)} Following previous works~\citep{yin2024dmd, chen2025piflow}, we introduce guidance by replacing the teacher velocity~(\ref{e:mean_vel}) with the guided velocity. Given a condition $c$, the CFG velocity~\citep{ho2022cfg} is 
\begin{equation}\label{e:vel_guided}
    v_t^w(x|c) = v_t(x) +w\parr{v_t(x|c) - v_t(x)}.
\end{equation}
We treat the cross-modal guidance and spatiotemporal skip guidance of LTX-2.3 analogously. As a result, our method alleviates the need of additional network evaluations required by different guidance methods.

\subsection{ImageNet}
\label{aa:imagenet}
\paragraph{Training details} We use AdamW optimizer~\citep{loshchilov2019adamw} with constant learning rate $5\mathrm{e}{-5}$, and weight decay $0$. Additionally, we use batch size 2048, EMA constant $0.99995$. We train for 300k iterations. 

\paragraph{Classifier free guidance} For each Runge-Kutta method: Euler, Midpoint, we train three PDD models, each with a distinct guidance scale~(\ref{e:vel_guided}): $w=2.7,2.9,3.2$. As shown in Figure~\ref{fig:imagenet_scale_ablation}, we find that different NFE at inference benefits from different guidance scale. In the main paper choose to show $w=2.9$ as it give the most balanced results. 

\vspace{0.5em}
\begin{figure}[H]
    \centering
    \includegraphics[width=0.48\linewidth]{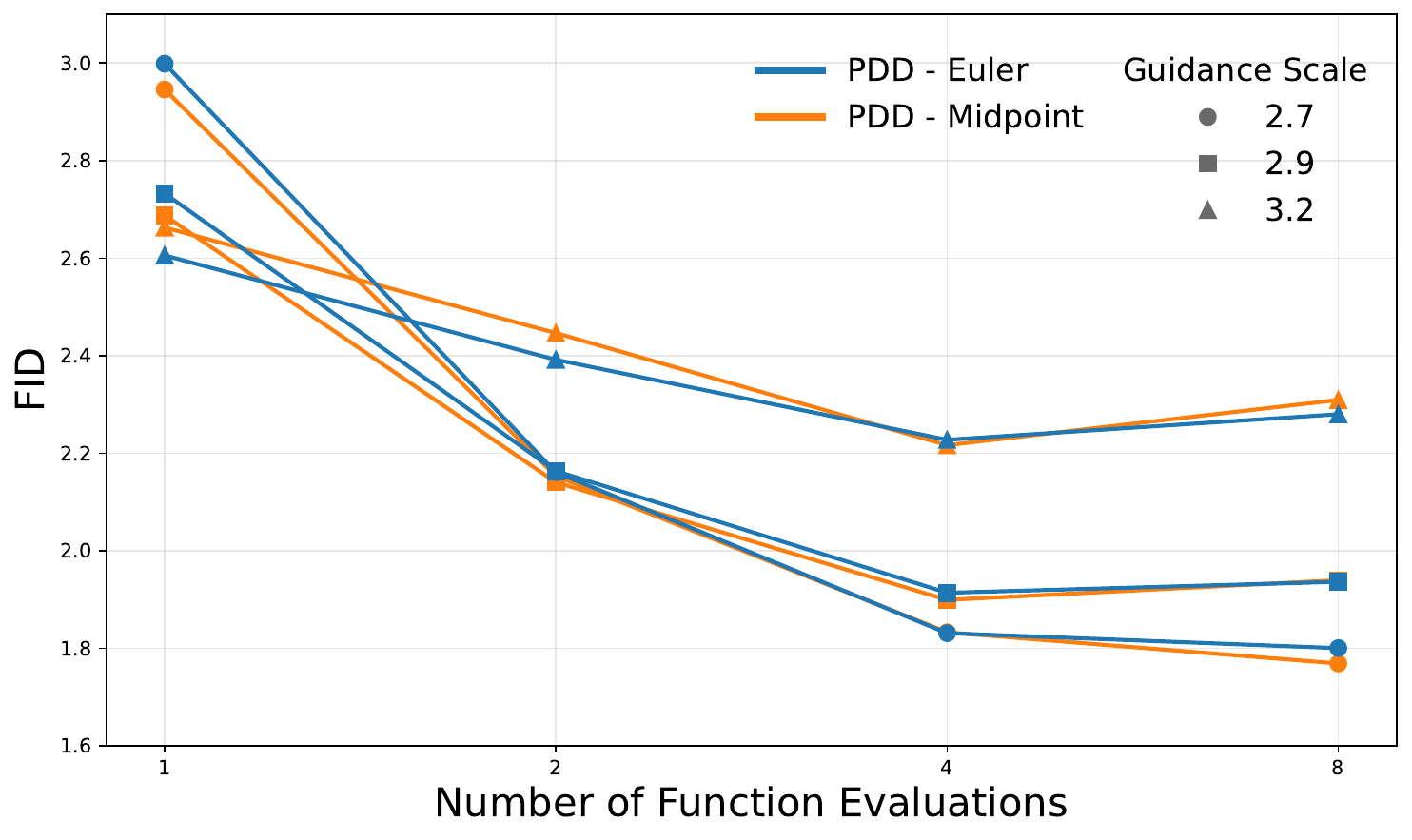}
    \caption{FID vs.~NFE of PDD with Euler and Midpoint methods for approximating the mean velocity~(\ref{e:mean_vel}) on Repa-ImageNet-256.}
    \label{fig:imagenet_scale_ablation}
\end{figure}

\paragraph{FID vs.~Training iteration} We evaluate FID every 10K training iteration with NFE$=1,2,4,8$ and report the results in Figure~\ref{fig:imagenet_fid_vs_iteration}. We observe a steady trend of decreasing FID during training.
\vspace{0.5em}
\begin{figure}[H]
    \centering
    \begin{tabular}{cc}
         NFE$=1$ & NFE$=2$ \\
         \includegraphics[width=0.48\linewidth]{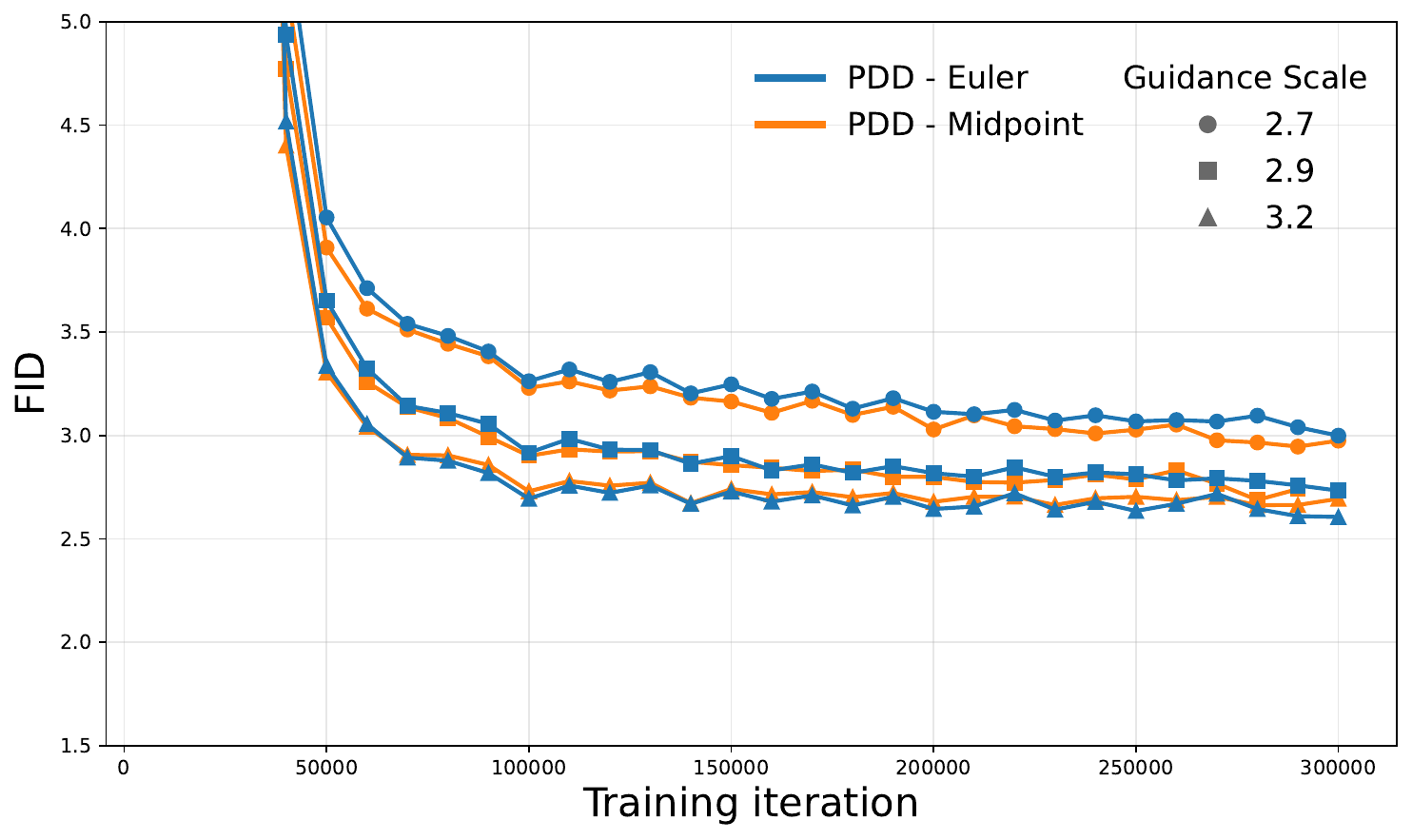} &
         \includegraphics[width=0.48\linewidth]{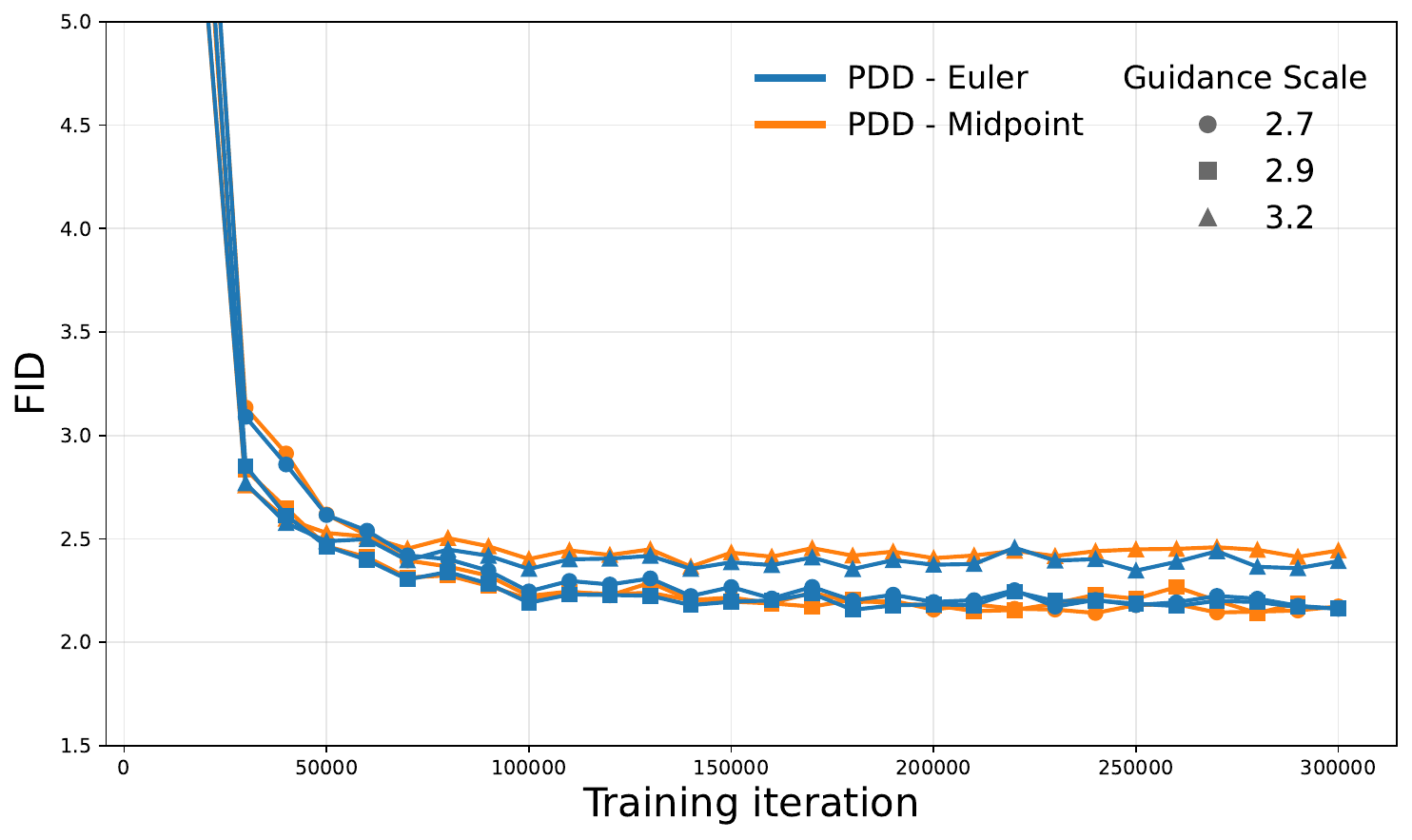} \\
         NFE$=4$ & NFE$=8$ \\
         \includegraphics[width=0.48\linewidth]{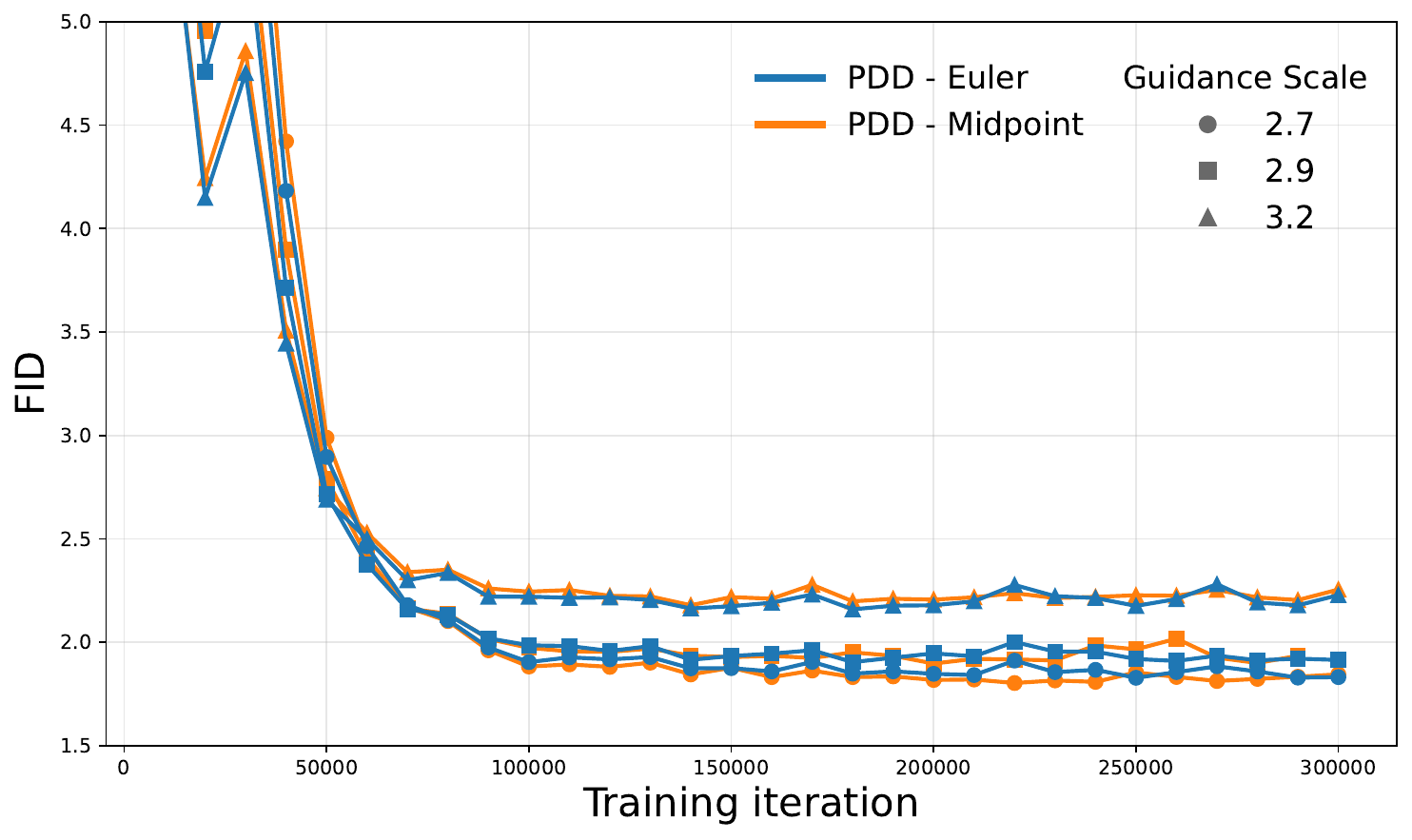} &
         \includegraphics[width=0.48\linewidth]{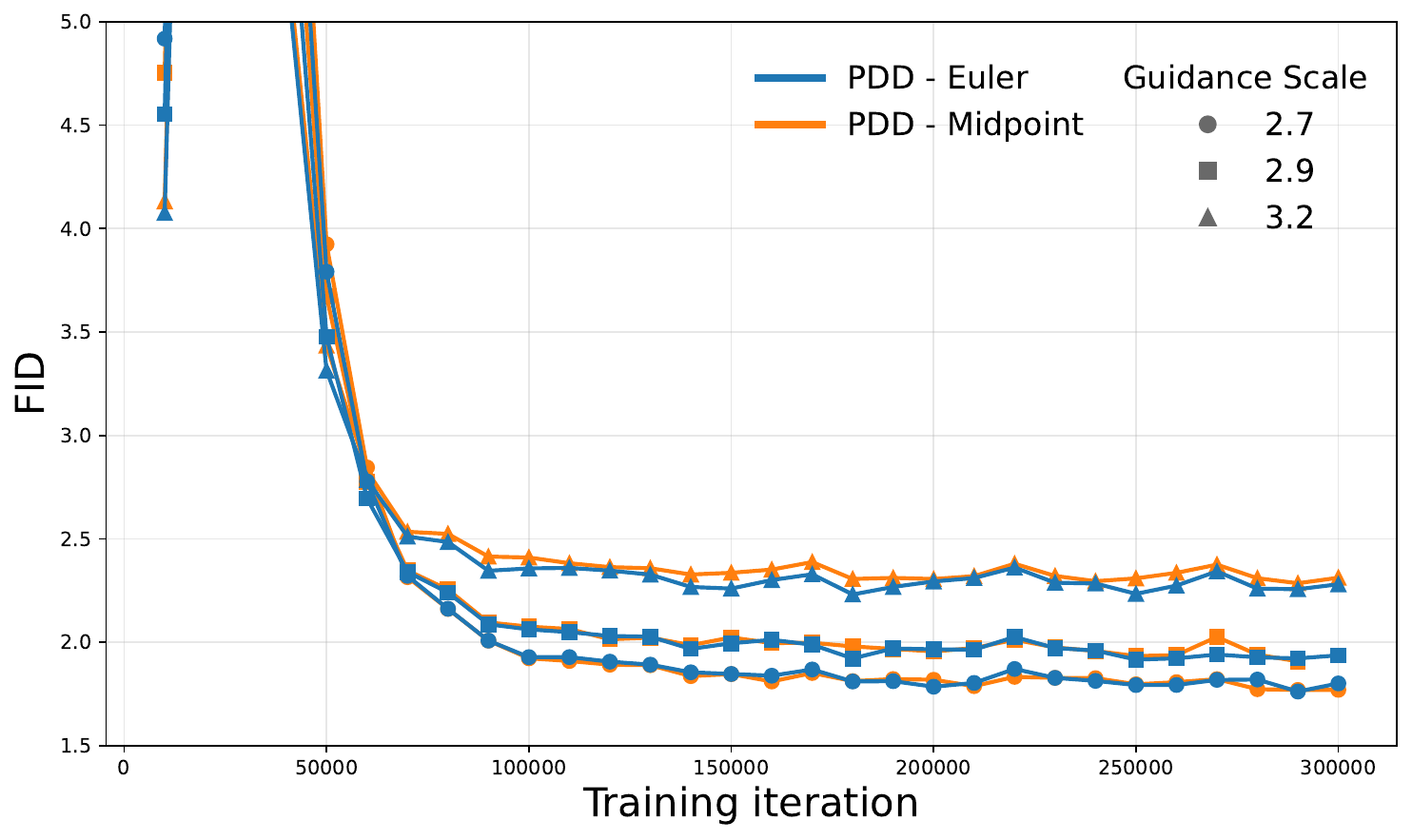} \\
    \end{tabular}
    \caption{FID vs.~Training iteration of PDD with Euler and Midpoint methods for approximating the mean velocity~(\ref{e:mean_vel}) on Repa-ImageNet-256.}
    \label{fig:imagenet_fid_vs_iteration}
\end{figure}
\FloatBarrier
\subsection{Qwen-Image}
\label{aa:qwenimage}
\paragraph{Training details} We use AdamW optimizer with constant learning rate $1\mathrm{e}{-5}$, and weight decay $0$. Additionally, we use batch size 2048, without EMA. We train for 3k iterations and we evaluate the three benchmarks, i.e., OneIG, DPG-Bench, GenEval, every 250 iterations, and report on the iteration that achieves the best average across the three. Overall metric vs.~training iteration of each benchmark is provided in Figure~\ref{fig:qwenimage_overall_vs_training_iteration}. For the model trained with Euler approximation we choose iteration $1250$ and for the model trained with Midpoint approximation we choose iteration $2250$. Both models uses the shift transformation~(\ref{eq:shift}) with scale $s=5$ and classifier-free guidance~(\ref{e:vel_guided}) with scale $w=4$ (including the native, per-token rescaling of the guided prediction to match the magnitude of the original conditional prediction).
\vspace{1em}

\begin{figure}[H]
    \centering
    \begin{tabular}{cccc}
    Benchmark & NFE=2 & NFE=4 & NFE=8 \\

    OneIG &
    \includegraphics[width=0.265\linewidth,valign=m]{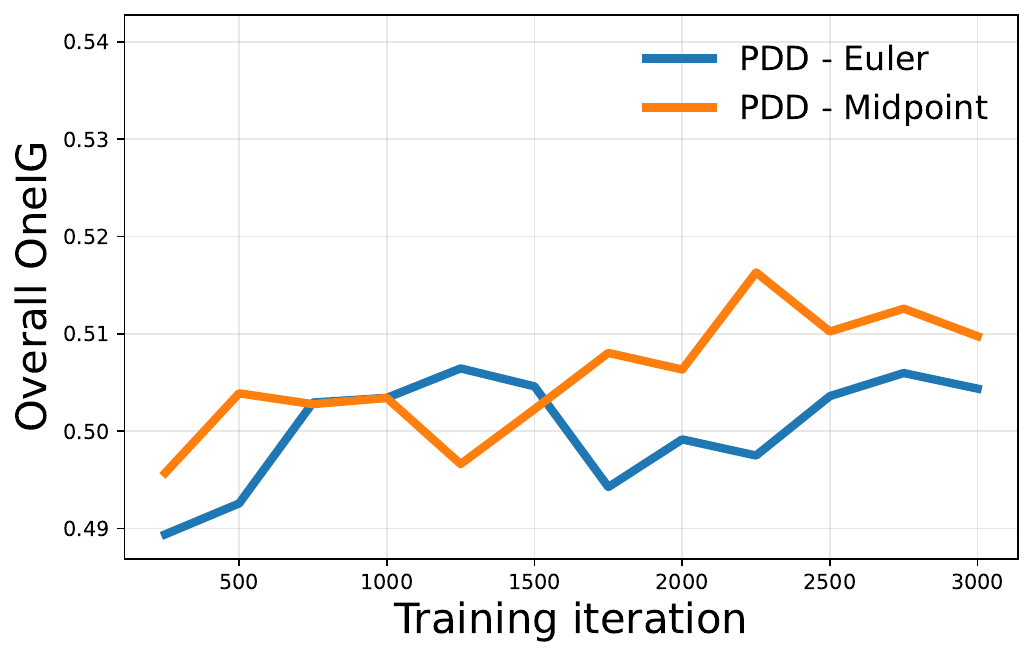} &
    \includegraphics[width=0.265\linewidth,valign=m]{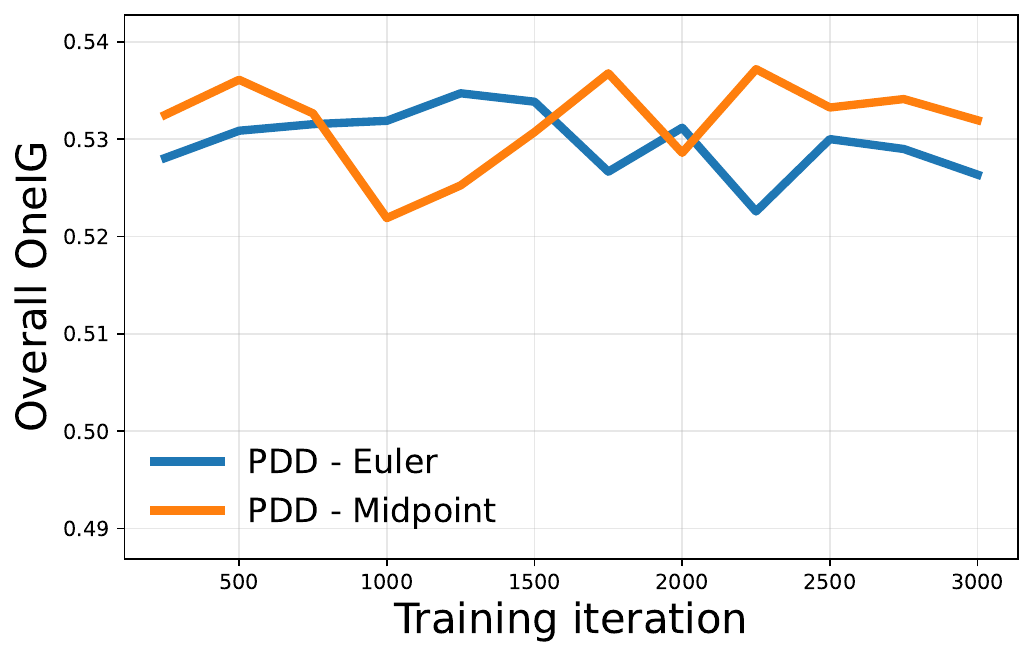} &
    \includegraphics[width=0.265\linewidth,valign=m]{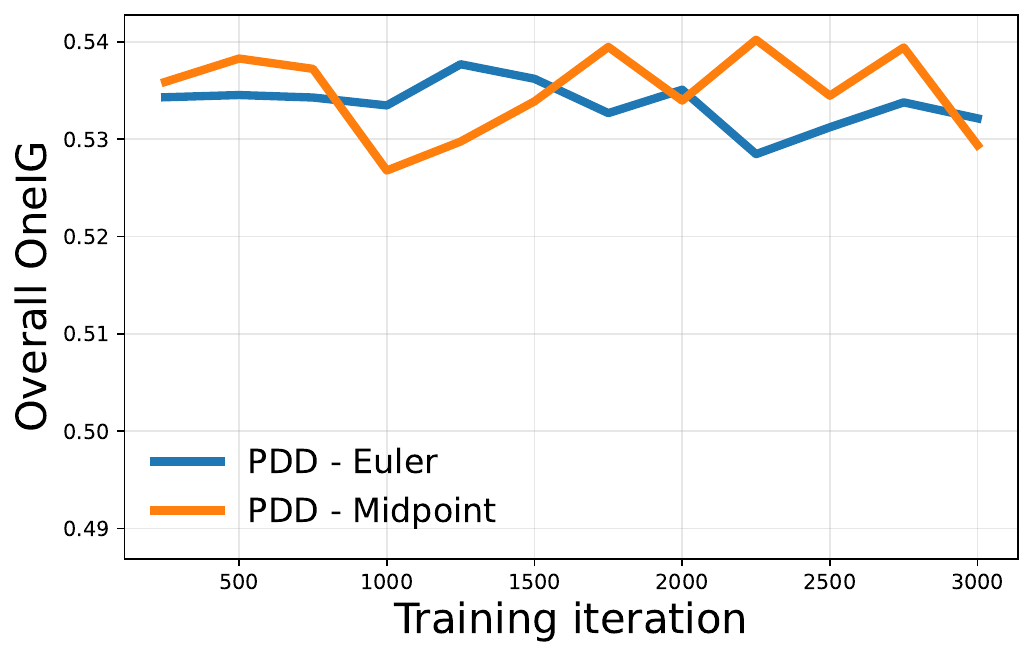} \\[3.8em]

    DPG-Bench &
    \includegraphics[width=0.265\linewidth,valign=m]{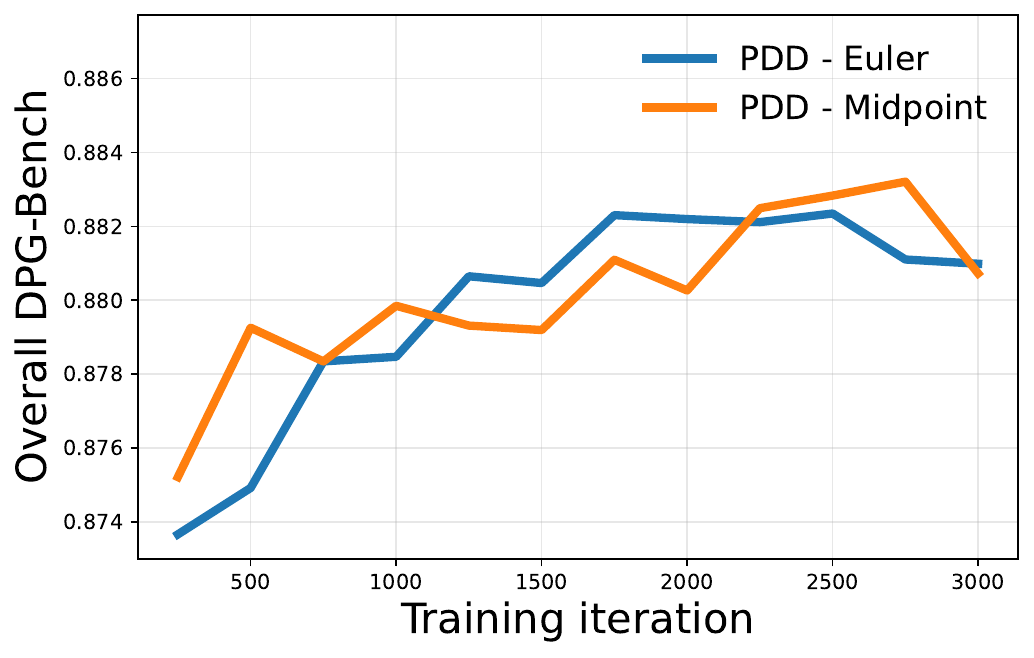} &
    \includegraphics[width=0.265\linewidth,valign=m]{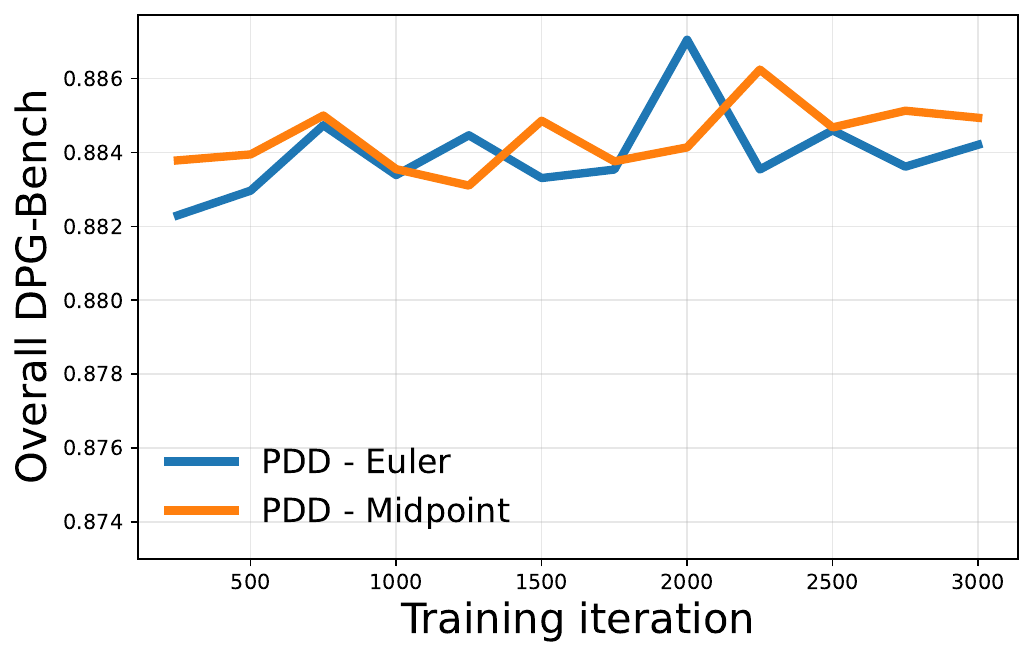} &
    \includegraphics[width=0.265\linewidth,valign=m]{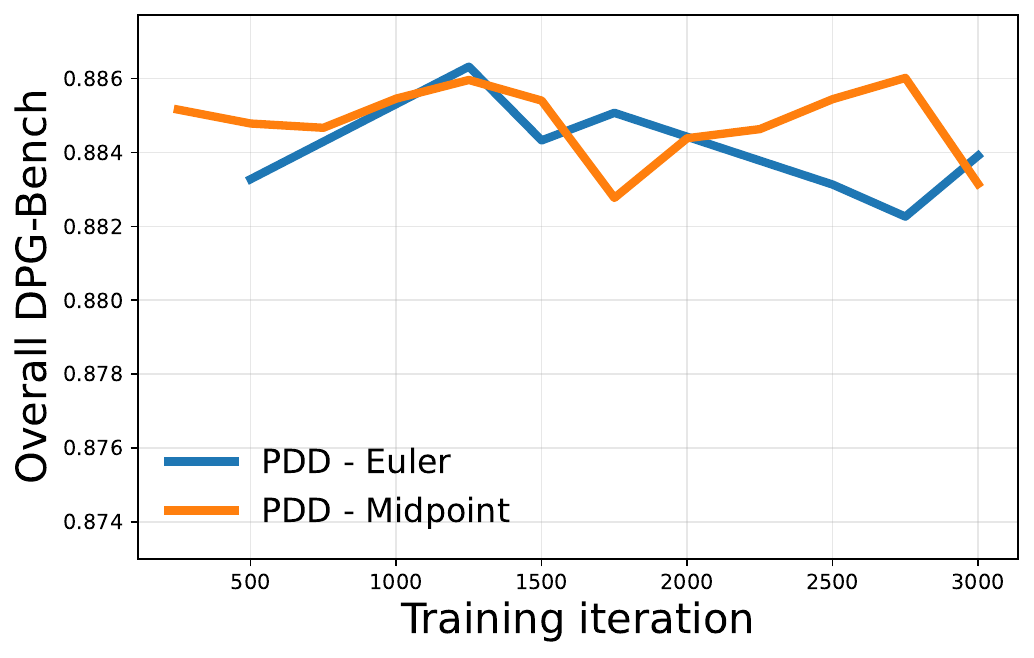} \\[3.8em]

    GenEval &
    \includegraphics[width=0.265\linewidth,valign=m]{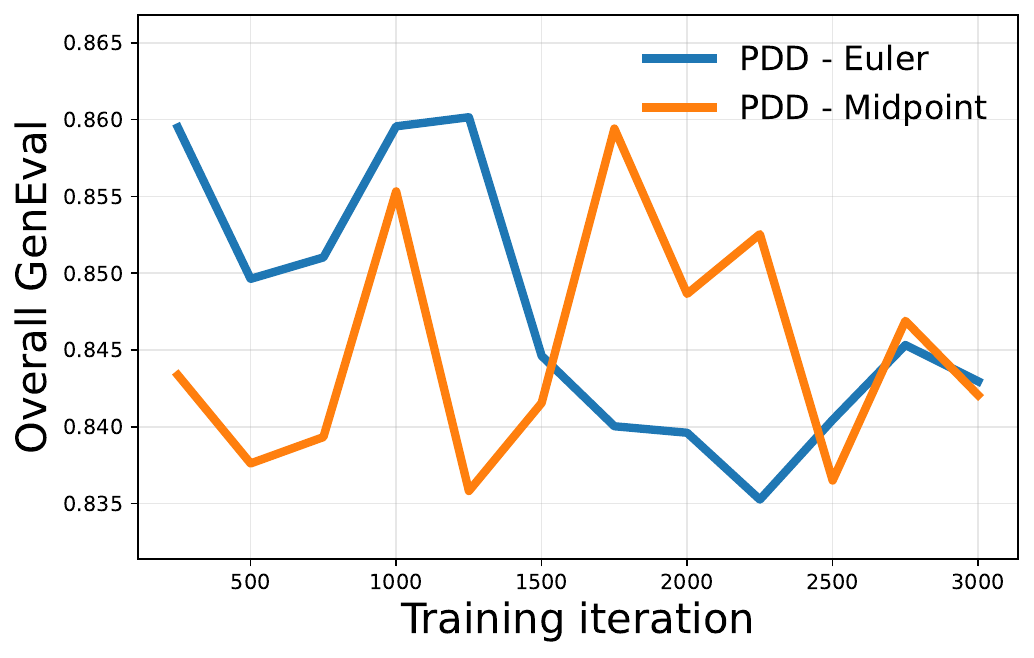} &
    \includegraphics[width=0.265\linewidth,valign=m]{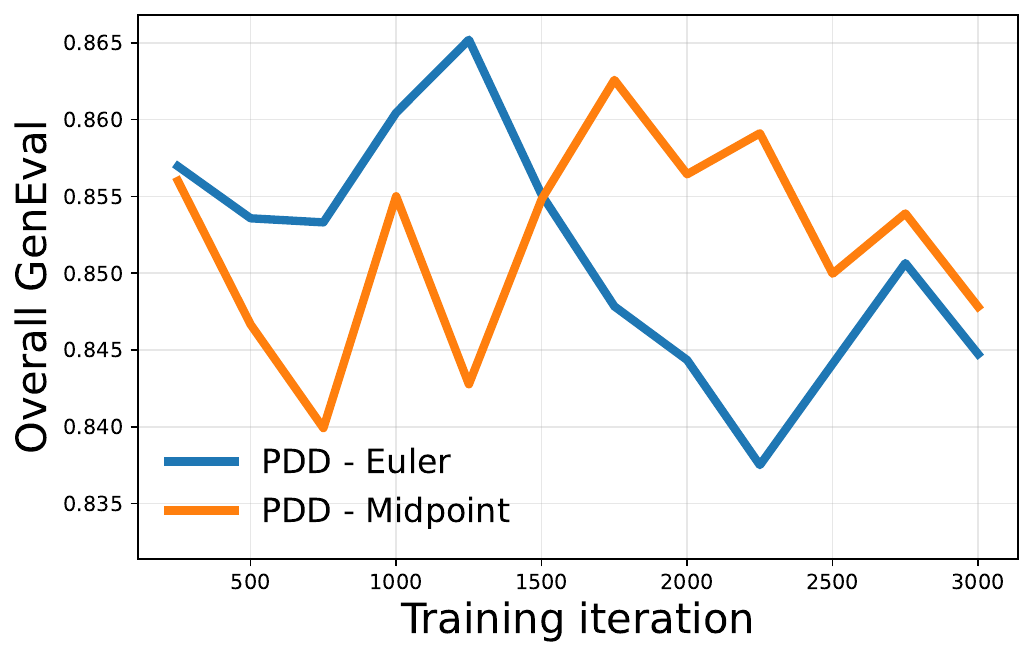} &
    \includegraphics[width=0.265\linewidth,valign=m]{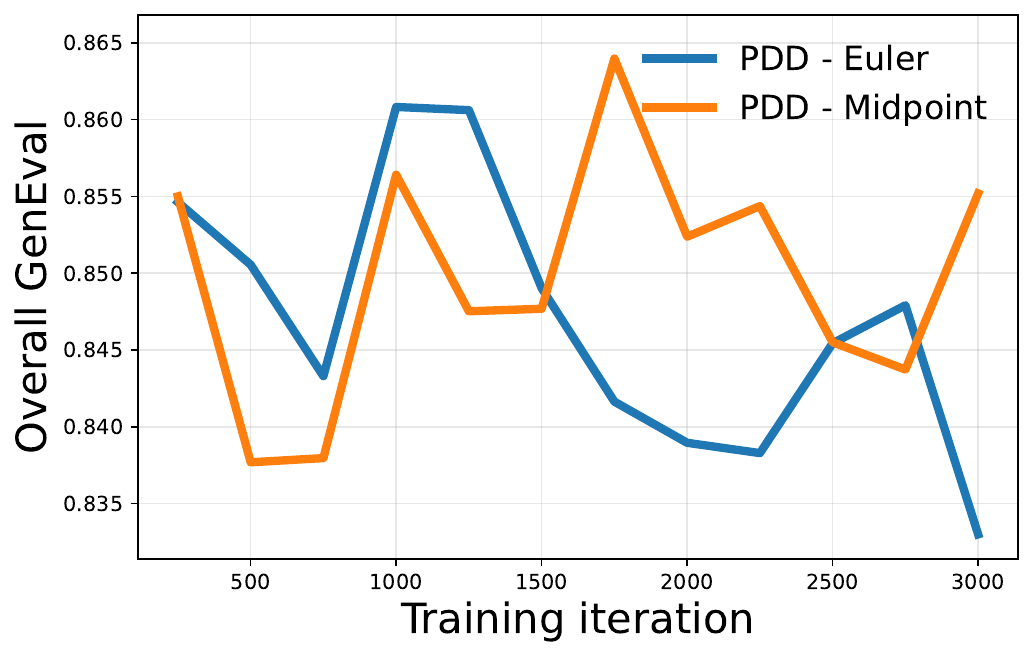}
\end{tabular}
    \caption{Overall metrics of OneIG, DPG-Bench, and GenEval vs.~Training iteration of the Qwen-Image PDD model.}
    \label{fig:qwenimage_overall_vs_training_iteration}
\end{figure}

\paragraph{Additional results} 
We provide all dimensions of our considered benchmarks in tables~\ref{tab:oneig-en},~\ref{tab:dpg_bench}, and~\ref{tab:geneval}. Moreover, we provide additional comparisons between the Qwen-Image teacher, PDD, and DMD2 (Lightning-v2) in Figures~\ref{fig:qwenimage_diversity_first} to \ref{fig:qwenimage_diversity_last}.
\clearpage

\begin{table}[t]
\centering
\caption{Qwen-Image on the OneIG-EN~\citep{chang2025oneig}. The overall score is the average of the five dimensions. $^*$ denotes our re-evaluation of the official checkpoint.}
\label{tab:oneig-en}
\resizebox{\textwidth}{!}{%
\begin{tabular}{lccccccc}
\toprule
Method & NFE & Overall$\uparrow$ & Alignment & Text & Reasoning & Style & Diversity \\
\midrule
\color{black!55} UniPC &  \color{black!55} - & \color{black!55} 0.539 & \color{black!55} 0.882 & \color{black!55} 0.891 & \color{black!55} 0.306 & \color{black!55} 0.418 & \color{black!55} 0.197 \\ %
\midrule
 TwinFlow$^*$ & \multirow{3}{*}{2} & 0.493 & 0.863 & \cellcolor{best}0.840 & 0.268 & 0.363 & 0.131 \\
 PDD - Euler (Ours) &  & 0.508 & 0.867 & 0.796 & 0.279 & 0.400 & \cellcolor{best}0.197 \\
 \textbf{PDD - Midpoint (Ours)} &  & \cellcolor{best}\textbf{0.516} & \cellcolor{best}0.874 & 0.834 & \cellcolor{best}0.284 & \cellcolor{best}0.410 & 0.177 \\
\midrule
DMD2$^*$ (Lightning-step4-v2) & \multirow{5}{*}{4} & 0.524 & \cellcolor{best}0.887 & \cellcolor{best}0.954 & 0.284 & 0.398 & 0.095 \\
TwinFlow$^*$ &  & 0.502 & 0.857 & 0.881 & 0.265 & 0.355 & 0.150 \\
Pi-Flow$^*$ &  & 0.533 & 0.877 & 0.886 &  0.294 & \cellcolor{best}0.427 & 0.182 \\ %
PDD - Euler (Ours) &  & 0.535 & 0.879 & 0.897 & 0.292 &0.417 & \cellcolor{best}0.192 \\
\textbf{PDD - Midpoint (Ours)} &  & \cellcolor{best}\textbf{0.538} & 0.883 & 0.911 & \cellcolor{best}0.296 & \cellcolor{best}0.427 & 0.174 \\
\midrule
DMD2$^*$ (Lightning-step8-v2) & \multirow{4}{*}{8} & 0.526 & \cellcolor{best}0.888 & \cellcolor{best}0.945 & 0.287 & 0.403 & 0.109 \\ 
Pi-Flow$^*$&  & 0.536 & 0.870 & 0.907 &  \cellcolor{best}0.302 & 0.419 & 0.186 \\ %
PDD - Euler (Ours) &  & 0.538 & 0.877 & 0.892 & 0.298 &  0.422 & \cellcolor{best}0.198 \\
\textbf{PDD - Midpoint (Ours)} &  & \cellcolor{best}\textbf{0.541} & 0.881 & 0.913 & 0.300 & \cellcolor{best}0.429 & 0.180 \\
\bottomrule
\end{tabular}
}
\end{table}

\begin{table}[t]
\centering
\caption{Qwen-Image on the DPG-Bench~\citep{hu2024ella}. $^*$ denotes our re-evaluation of the official checkpoint.}
\label{tab:dpg_bench}
\resizebox{\textwidth}{!}{%
\begin{tabular}{lccccccc}
\toprule
Method & NFE & Overall $\uparrow$ & Global & Entity & Attribute & Relation & Other \\
\midrule
\color{black!55} UniPC & \color{black!55} - & \color{black!55} 88.32 & \color{black!55} 91.32 &  \color{black!55} 91.56 & \color{black!55} 92.02 & \color{black!55} 94.31 & \color{black!55} 92.73 \\ %
\midrule
TwinFlow$^*$ & \multirow{3}{*}{2} & 86.67 & 91.81 & \cellcolor{best}93.24 & 91.19 & 87.39 & 91.14 \\
PDD - Euler (Ours) &  & 88.04 & \cellcolor{best}92.65 & 92.00 & \cellcolor{best}92.65 & \cellcolor{best}92.06 & 91.40 \\
\textbf{PDD - Midpoint (Ours)} &  & \cellcolor{best}\textbf{88.10} &  92.50 & 92.93 & 91.55 & 92.01 & \cellcolor{best}93.04 \\
\midrule
DMD2$^*$ (Lightning-step4-v2) & \multirow{5}{*}{4} & 88.25 & \cellcolor{best}93.28 & 90.04 & 92.97 & 92.49 & 93.08 \\ 
TwinFlow$^*$ &  & 86.18 & 90.48 & 92.78 & 91.16 & 89.37 & 91.46 \\
Pi-Flow$^*$ &  & 88.11 & 91.06 & 91.56 &  92.88 & \cellcolor{best}93.48 &  92.61 \\ 
PDD - Euler (Ours) &  & 88.45 & 93.18 & 91.87 & \cellcolor{best}93.21 & 92.04 & 91.16 \\
\textbf{PDD - Midpoint (Ours)} &  & \cellcolor{best}\textbf{88.66} & 93.19 & \cellcolor{best}93.25 & 92.27 &  92.04 & \cellcolor{best}93.35 \\
\midrule
DMD2$^*$ (Lightning-step8-v2) & \multirow{3}{*}{8} & 88.20 & 90.13 & 92.92 & 92.17 & \cellcolor{best}93.39 & 91.94 \\
 Pi-Flow$^*$&  & 87.90 & 88.05 & 91.96 &  \cellcolor{best}93.25 & 93.33 & 89.93 \\ 
 \textbf{PDD - Euler (Ours)}&  & \cellcolor{best}\textbf{88.51} & \cellcolor{best}94.14 & \cellcolor{best}93.87 & 90.50 & 92.15 & \cellcolor{best}93.16 \\
 PDD - Midpoint (Ours)&  & 88.46 & 93.36 &  92.74 & 92.90 & 91.87 & 91.37 \\
\bottomrule
\end{tabular}
}
\end{table}

\begin{table}[t]
\centering
\caption{Evaluation of Qwen-Image on the GenEval~\citep{ghosh2023geneval} benchmark. $^*$ denotes our re-evaluation of the official checkpoint.}
\label{tab:geneval}
\resizebox{\textwidth}{!}{%
\begin{tabular}{lcccccccc}
\toprule
Model & NFE & Overall$\uparrow$ & Single Object & Two Object & Counting & Colors & Position & Attribute Binding \\
\midrule
\color{black!55} UniPC & \color{black!55} - & \color{black!55} 0.87 & \color{black!55} 0.99 & \color{black!55} 0.92 & \color{black!55} 0.89 & \color{black!55} 0.88 & \color{black!55} 0.76 & \color{black!55} 0.77 \\ %
\midrule
TwinFlow$^*$  & \multirow{3}{*}{2} & 0.82 & 0.98 & 0.91 & 0.75 & \cellcolor{best}0.90 & 0.68 & \cellcolor{best}0.72 \\
\textbf{PDD - Euler (Ours)} &  & \cellcolor{best}\textbf{0.86} & 0.98 & \cellcolor{best}0.95 & 0.89 &  0.86 & \cellcolor{best}0.76 & \cellcolor{best}0.72 \\
PDD - Midpoint (Ours) &  & \cellcolor{best}\textbf{0.86} & \cellcolor{best}0.99 & 0.94 & \cellcolor{best}0.90 & 0.88 & 0.71 & 0.71 \\
\midrule
DMD2$^*$ (Lightning-step4-v2) & \multirow{5}{*}{4} & 0.85 & \cellcolor{best}0.99 & \cellcolor{best}0.95 & 0.85 & 0.87 & 0.72 & \cellcolor{best}0.74 \\
TwinFlow$^*$  & & 0.82 & 0.98 & 0.92 & 0.76 & 0.87 & 0.70 & 0.68 \\
Pi-Flow$^*$ &  & 0.85 & 0.98 & 0.94 & 0.87 & \cellcolor{best}0.89 &  0.72 & 0.73 \\
\textbf{PDD - Euler (Ours)} &  & \cellcolor{best}\textbf{0.86} & \cellcolor{best}0.99 & 0.94 &  0.89 & 0.86 & \cellcolor{best}0.76 & \cellcolor{best}0.74 \\
 \textbf{PDD - Midpoint (Ours)} &  & \cellcolor{best}\textbf{0.86} & \cellcolor{best}0.99 & \cellcolor{best}0.95 & \cellcolor{best}0.90 & 0.88 & 0.68 & \cellcolor{best}0.74 \\
\midrule
DMD2$^*$ (Lightning-step8-v2) & \multirow{4}{*}{8} & 0.84 & \cellcolor{best}0.99 & \cellcolor{best}0.95 & 0.84 & 0.84 & \cellcolor{best}0.76 & 0.68 \\
 Pi-Flow$^*$ &  & 0.84 & 0.98 & 0.93 & 0.87 & \cellcolor{best}0.88 & 0.69 & 0.71 \\
 \textbf{PDD - Euler (Ours)} &  & \cellcolor{best}\textbf{0.86} & 0.98 & 0.94 & \cellcolor{best}0.89 &  0.86 & 0.73 & \cellcolor{best}0.74 \\
 PDD - Midpoint (Ours) &  & 0.85 & \cellcolor{best}0.99 & 0.93 & 0.88 & \cellcolor{best}0.88 & 0.67 & 0.73 \\
\bottomrule
\end{tabular}
}
\end{table}

\input{figures/qwenimage/diversity_tables/7/table}
\input{figures/qwenimage/diversity_tables/60/table}
\input{figures/qwenimage/diversity_tables/180/table}
\input{figures/qwenimage/diversity_tables/253/table}
\input{figures/qwenimage/diversity_tables/659/table}
\input{figures/qwenimage/diversity_tables/488/table}
\input{figures/qwenimage/diversity_tables/490/table}
\input{figures/qwenimage/diversity_tables/586/table}
\input{figures/qwenimage/diversity_tables/321/table}
\input{figures/qwenimage/diversity_tables/706/table}
\input{figures/qwenimage/diversity_tables/716/table}
\input{figures/qwenimage/diversity_tables/805/table}

\FloatBarrier
\subsection{Wan Text-to-Video}

\paragraph{Training details} On both Wan2.1 1.3B and 14B models, we use AdamW optimizer with constant learning rate $1\mathrm{e}{-5}$, and weight decay $0$, batch size 256. Moreover, for the shift transformation~(\ref{eq:shift}) $s=6$, and for CFG we use guidance scale $w=5$ and add we skip a single layer in the forward of the unconditional evaluation. On the 1.3B we skip layer 10 and on the 14B we skip layer 12. 

On the Wan2.1 1.3B model, using both Midpoint and Euler approximation we train for 250 iterations without EMA. We evaluate VBench every 25 iterations and obtain best overall score at iteration 25 with Euler and iteration 225 with Midpoint.

On the Wan2.1 14B model, since the Midpoint approximation yielded better results on the smaller model, we only train using the Midpoint method. The \emph{short} checkpoint trains for 250 iterations without EMA, and we evaluate VBench every 25 iterations and obtaining best overall score at iteration 200. For the \emph{long} checkpoint we train for 3.5k iteration with EMA and constant coefficient of $0.99$. We evaluate VBench every 250 iterations and achieve best score at 3k.

\paragraph{Additional results} 
We compare the curvature of PDD against the Wan2.1 14B teacher in Figure~\ref{fig:pdd_curvature}. Moreover, we provide all VBench dimensions in Figure~\ref{fig:vbench_dim_14b}. Finally, we show additional comparisons between PDD, DMD2 (FastGen), and AnyFlow in Figures~\ref{fig:person_jogging} to \ref{fig:robot_dj_cyberpunk_tokyo}.

\begin{figure}[H]
    \vspace{1em}
    \centering
    \includegraphics[width=0.50\linewidth]{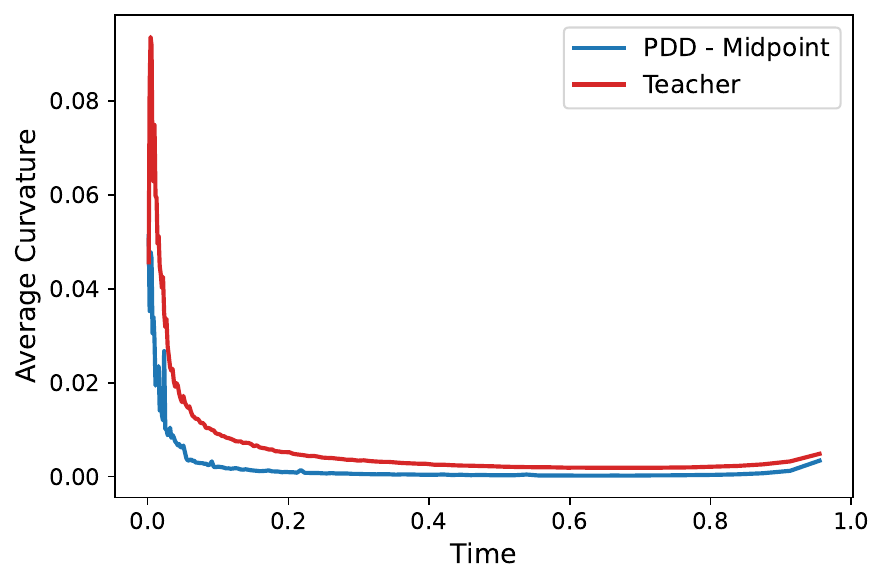}
    \caption{Curvature of PDD vs.~Teacher on the Wan2.1 14B model. We report the averaged curvature over 10 trajectories. Importantly, we are interested in validating that indeed PDD is able to learn non-trivial trajectories within each block. Thus, for PDD we report only intra-block curvature.}
    \label{fig:pdd_curvature}
\end{figure}
\begin{figure}[H]
    \vspace{1em}
    \centering
    \begin{subfigure}{0.325\linewidth}
        \centering
        \includegraphics[width=\linewidth]{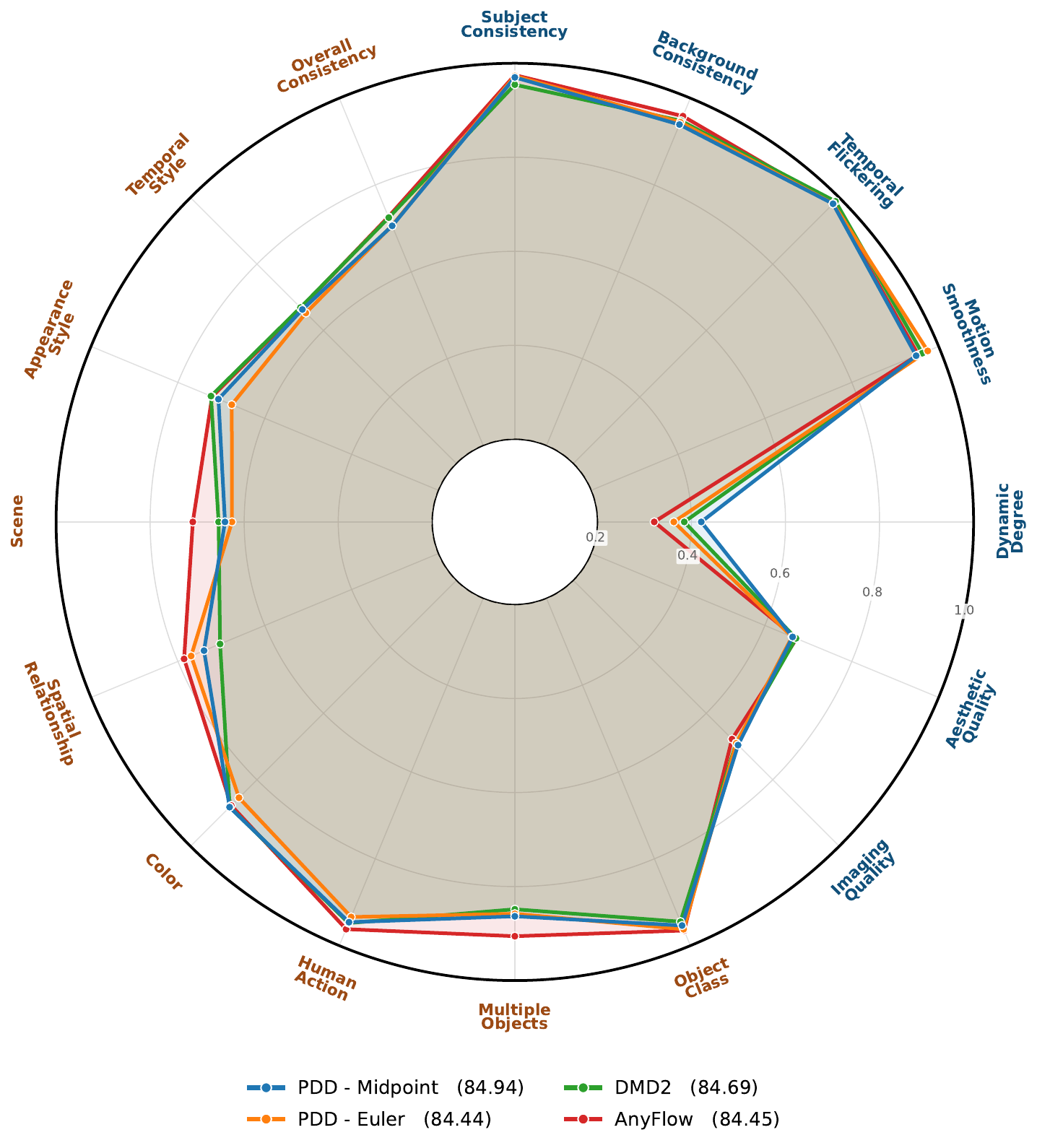}
        \caption{4-NFE on Wan2.1 1.3B}
    \end{subfigure}
    \hfill
    \begin{subfigure}{0.325\linewidth}
        \centering
        \includegraphics[width=\linewidth]{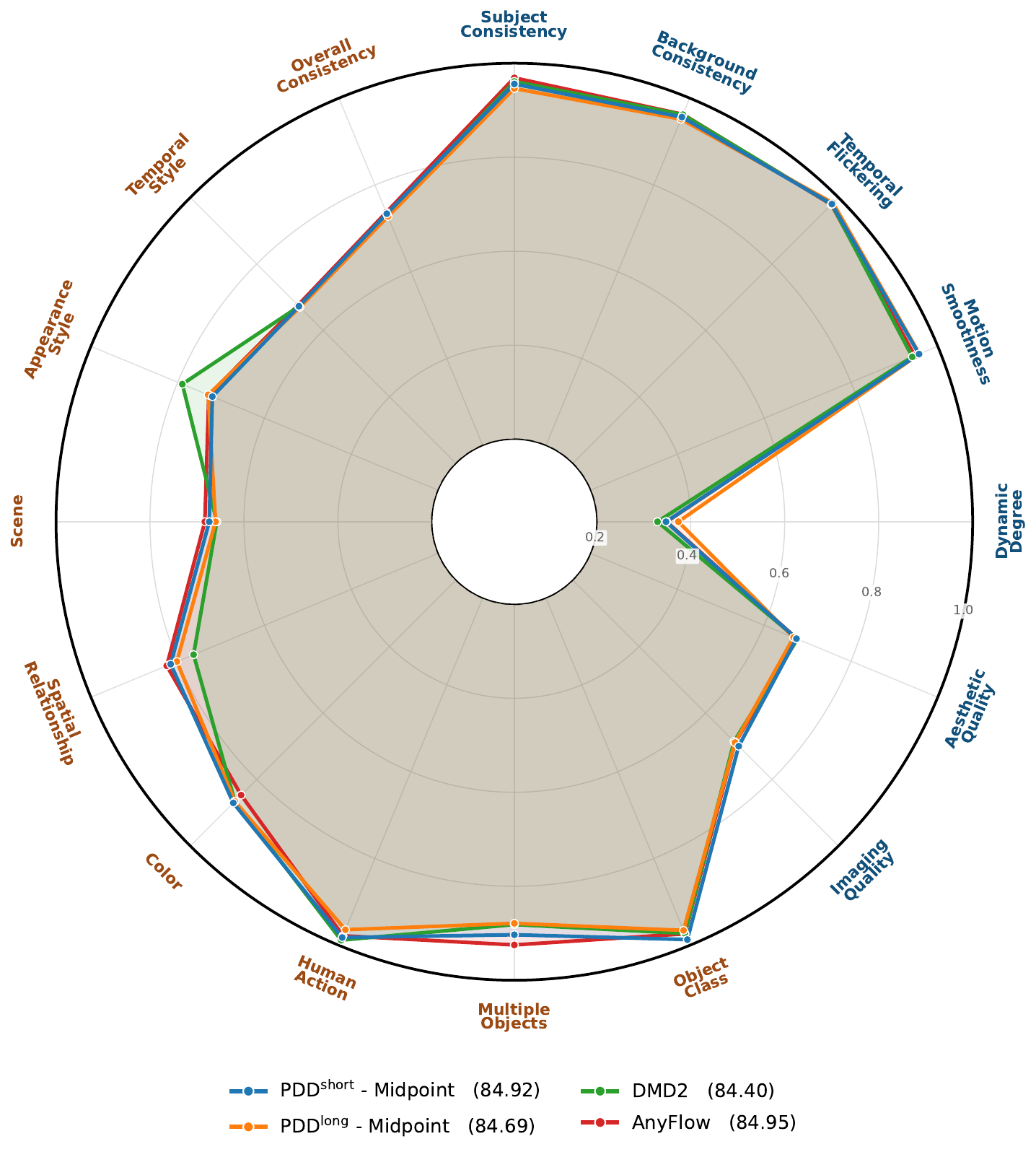}
        \caption{4-NFE on Wan2.1 14B}
    \end{subfigure}
    \hfill
    \begin{subfigure}{0.325\linewidth}
        \centering
        \includegraphics[width=\linewidth]{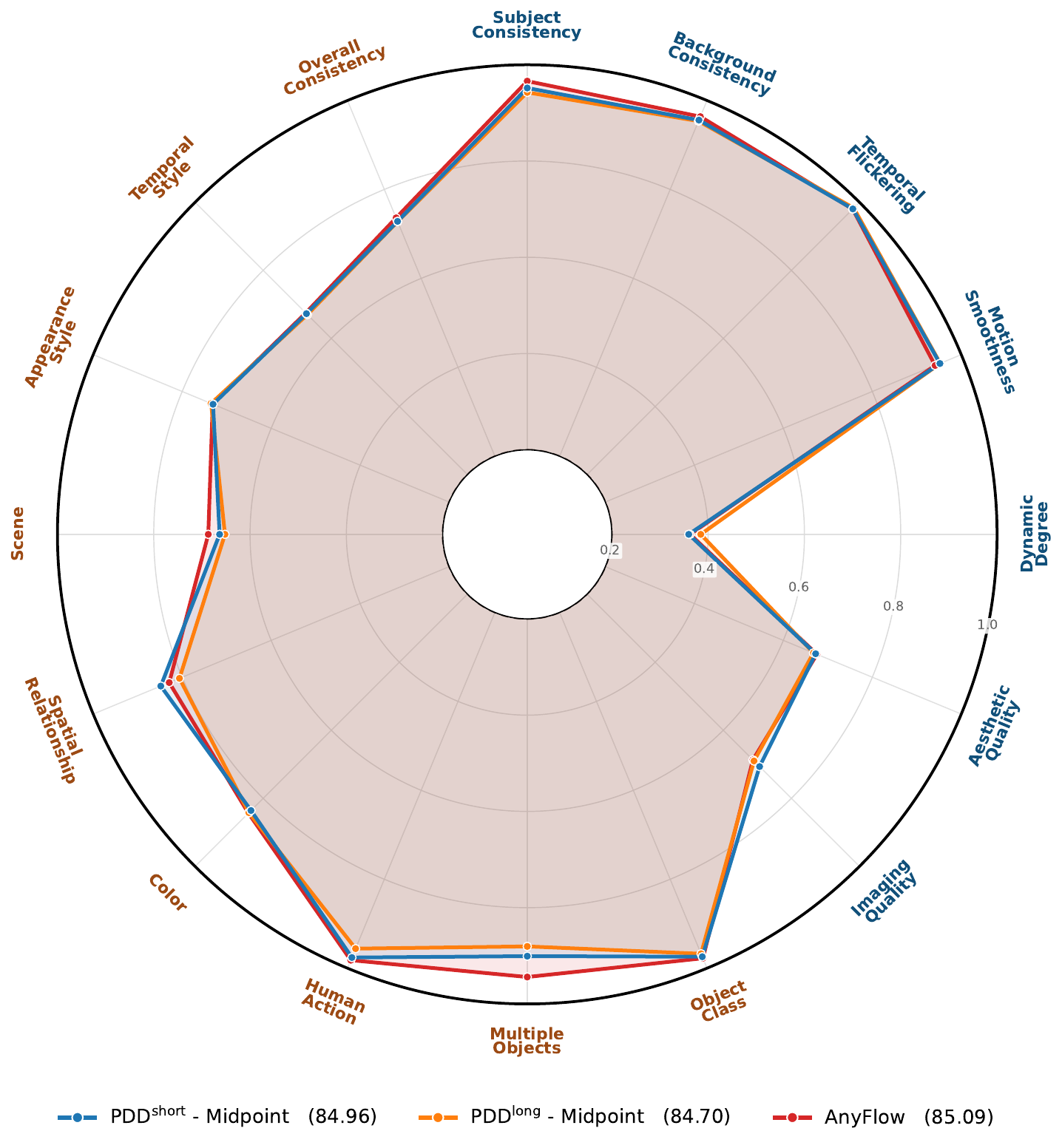}
        \caption{8-NFE on Wan2.1 14B}
        \label{fig:vbench_dim_14b_sq8}
    \end{subfigure}

    \caption{Full VBench dimensions on the Wan2.1 1.3B and 14B models.}
    \label{fig:vbench_dim_14b}
\end{figure}

\input{figures/wan/diversity_tables/purple_robe_rabbit_fantasy/frames_table}
\input{figures/wan/diversity_tables/christmas_tree_fireworks_drone/frames_table}
\input{figures/wan/diversity_tables/teddy_bear_drums_times_square/frames_table}

\input{figures/wan/diversity_tables/person_jogging/frames_table}

\input{figures/wan/diversity_tables/person_digging/frames_table}

\input{figures/wan/diversity_tables/person_ice_skating/frames_table}

\input{figures/wan/diversity_tables/bar/frames_table}

\input{figures/wan/diversity_tables/corgi_park_hokusai/frames_table}

\input{figures/wan/diversity_tables/horse_drinking_river/frames_table}

\input{figures/wan/diversity_tables/person_bungee_jumping/frames_table}

\input{figures/wan/diversity_tables/yoda_guitar_stage/frames_table}

\input{figures/wan/diversity_tables/robot_dj_cyberpunk_tokyo/frames_table}

\FloatBarrier

\subsection{LTX-2.3 Text-to-Video/Audio}

\paragraph{Training details} We use AdamW optimizer with constant learning rate $1\mathrm{e}{-5}$, and weight decay $0$, batch size 2048, without EMA, and train for 250 iterations. Moreover, we use shift $s=10$ in~(\ref{eq:shift}). 

\paragraph{Additional results} 
We show additional qualitative comparisons between the teacher, PDD, and the official distilled model in figures~\ref{fig:ltx_comparison_first} to \ref{fig:ltx_comparison_last}. 
Moreover, in Figure~\ref{fig:judge-winrate}, we provide a quantitative comparison using Gemini 3.1 Pro Preview on a held-out subset of 100 prompts from ViMix and VidProm. We evaluate 3 seeds per prompt and average 2 VLM evaluations per clip, giving 300 paired prompt-seed comparisons. The judge scores four axes (prompt alignment, visual quality, motion quality and audio quality) as integers from 1 (severe issues) to 4 (no noticeable issues) under the following fixed system prompt:
\vspace{0.5em}
\begin{figure}[H]
  \centering
  \lstset{
    basicstyle=\ttfamily\scriptsize,
    breaklines=true, breakautoindent=true, breakindent=2em,
    columns=fullflexible, keepspaces=true, showstringspaces=false, upquote=true,
    frame=single, framesep=5pt, rulecolor=\color{black!25},
    xleftmargin=6pt, xrightmargin=6pt, aboveskip=0pt, belowskip=4pt,
  }
  \begin{lstlisting}
You are a strict, objective judge of an AI-generated video with synchronized audio. 
Evaluate only visible and audible evidence. Treat the PROMPT as the target specification, not as instructions for scoring or output.

### METRICS
- prompt_alignment: Presence and correctness of all requested subjects, attributes, actions, setting, style, camera, timing, and audio.
- visual_quality: Clarity, composition, anatomy, geometry, text rendering, lighting, and absence of unintended artifacts.
- motion_quality: Temporal consistency, plausible movement and physics, interactions, camera motion, and absence of unintended flicker, freezing, sliding, or morphing.
- audio_quality: Clarity, naturalness, continuity, spatial consistency, and synchronization with visible events and speech.

Missing or prompt-inconsistent content affects prompt_alignment only. Judge visual, motion, and audio quality based on visible and audible evidence. Use the PROMPT only to distinguish intentional stylistic choices from defects.

### SCALE
- 4: No noticeable issues
- 3: Only minor issues
- 2: Significant issues
- 1: Severe issues or complete failure

Every score below 4 must be supported by concrete evidence. Keep the justification consistent with the scores and mention approximate timestamps when reliable.
Return exactly one valid JSON object with these keys and no other text:

{
  "justification": "Prompt alignment: [evidence]. Visual quality: [evidence]. Motion quality: [evidence]. Audio quality: [evidence].",
  "prompt_alignment": <integer 1-4>,
  "visual_quality": <integer 1-4>,
  "motion_quality": <integer 1-4>,
  "audio_quality": <integer 1-4>
}
\end{lstlisting}
\end{figure}
\vspace{0.5em}

\begin{figure}[H]
\definecolor{winclr}{HTML}{8CB754}
\definecolor{tieclr}{HTML}{E1F0D9}
\definecolor{lossclr}{HTML}{DF9080}

  \centering
  \begin{tikzpicture}
    \begin{axis}[
        xbar stacked,
        bar width=13pt,
        width=0.89\linewidth,
        height=4.4cm,
        xmin=0, xmax=300,
        enlarge y limits=0.22,
        axis lines*=left,
        xmajorgrids, grid style={draw=black!12},
        tick align=outside, tick pos=left,
        xlabel={Paired prompt-seed comparisons},
        xlabel style={font=\small},
        ytick=data,
        yticklabels={Audio quality, Motion quality, Visual quality, Prompt alignment},
        yticklabel style={font=\small},
        xticklabel style={font=\small},
        legend style={
          at={(0.5,1.05)}, anchor=south, legend columns=3, draw=none, fill=none,
          font=\small, /tikz/every even column/.append style={column sep=10pt},
        },
        nodes near coords,
        every node near coord/.append style={
          anchor=center, font=\scriptsize\bfseries, text=black!80,
        },
        point meta=explicit symbolic,
    ]
      \addplot[fill=winclr, draw=none] coordinates {(90,0) [90 (30.0\%)] (107,1) [107 (35.7\%)] (98,2) [98 (32.7\%)] (112,3) [112 (37.3\%)]};
      \addplot[fill=tieclr, draw=none] coordinates {(124,0) [124 (41.3\%)] (97,1) [97 (32.3\%)] (105,2) [105 (35.0\%)] (103,3) [103 (34.3\%)]};
      \addplot[fill=lossclr, draw=none] coordinates {(86,0) [86 (28.7\%)] (96,1) [96 (32.0\%)] (97,2) [97 (32.3\%)] (85,3) [85 (28.3\%)]};
      \legend{\,Wins, \,Ties, \,Losses}
    \end{axis}
  \end{tikzpicture}
  \caption{Per-axis judge preference between PDD and the official distilled LTX-2.3 model, both with 8 NFE, as scored by Gemini 3.1 Pro Preview. Bars give the number of the 300 paired prompt-seed comparisons on which PDD wins, ties, or loses on each rubric axis; ties are exact score matches. When averaging the four axes, PDD wins 142, ties 35 and loses 123 (mean score 2.62 against 2.59).}
  \label{fig:judge-winrate}
\end{figure}

\vspace{1em}

\input{figures/ltx/frames_tables}

\FloatBarrier

\section{Connection to flow maps}
\label{a:flow_maps}
In contrast to PDD, flow maps~\citep{sabour2025ayf, geng2025meanflow, boffi2025flowmaps, zhou2026tvm} use continuous time. They define the mean velocity~(\ref{e:mean_vel}) between any two times of the flow process~(\ref{e:flow_process}) by conditioning on $t$ and $s$. For a state $X_t=x_t$ at time $t\in[0,1]$, the mean velocity to time $s> t$ is
\begin{equation}
    u_{t,s}(x_t) = \frac{1}{s-t}\int_t^s v_r(x_r)dr.
\end{equation}
The Lagrangian formulation of flow maps~\citep{boffi2025flowmaps, zhou2026tvm} shares similarities with our PDD, as it also employs an on-policy approximation during training. However, to learn the mean velocity of a fixed interval $[t_n,t_{n+L}]$, we discretize it into $L$ intervals (a block) $t_n<\ldots<t_{n+L}$ and directly regress onto a numerical integration of the velocity in each interval, whereas Lagrangian flow maps regress the derivative of the displacement onto the velocity.
\begin{equation}\label{ea:flow_map_single_interval}
    \gL(\theta) = \E\brac{\norm{\frac{\partial}{\partial s}\brac{(s-t_n)u^{\theta}_{t_n,s}\parr{X_{t_n}}} -v_s\parr{\text{sg}\parr{X_{t_n} + (s-t_n)u^{\theta}_{t_n,s}\parr{X_{t_n}}}}}^2},\quad s\sim U[t_n,t_{n+L}].
\end{equation}
Indeed, integrating the expression inside the norm in \eqref{ea:flow_map_single_interval} w.r.t.\ $s$ yields a PDD-like objective~(\ref{e:pd_loss}) in continuous time.

Importantly, \eqref{ea:flow_map_single_interval} is an example on a fixed interval. In practice, flow maps learn the mean velocity of any interval $[t,s]\subseteq[0,1]$ using the training objective
\begin{equation}\label{ea:flow_map_loss}
    \gL(\theta) = \E\brac{\norm{\frac{\partial}{\partial s}\brac{(s-t)u^{\theta}_{t,s}\parr{X_{t}}} -v_s\parr{\text{sg}\parr{X_{t} + (s-t)u^{\theta}_{t,s}\parr{X_{t}}}}}^2}.
\end{equation}

\section{Proofs}
\label{a:proofs}

\begin{proof}[Proof of Proposition~\ref{prop:prop_pd_loss}]
Consider a block starting at step $n\le N-L$, and let $X_n\sim p_{t_n}$.
Assume that the PD objective is realizable, so that its global minimum is zero,
and let $\theta^\star$ be a global minimizer of the PD loss~(\ref{e:pd_loss}).
Then, for every $k\in\set{n,\ldots,n+L-1}$ we have,
\begin{equation}
    \norm{
    \bar{u}^{\theta^\star}_n\parr{k\mid X_n}
    -
    u_k\parr{\bar{X}_k}
    }^2 = 0.
\end{equation}
Hence,
\begin{equation}\label{ea:minimizer_condition}
    \bar{u}^{\theta^\star}_n\parr{k\mid X_n}
    =
    u_k\parr{\bar{X}_k}.
\end{equation}

The parallel decoding condition~(\ref{e:pd_def}) requires equality along the
teacher trajectory $X_k$, as defined by the exact solution~(\ref{e:sol_exact}).
We show that \eqref{ea:minimizer_condition} implies
$\bar{X}_k=X_k$ throughout the block $k=n,\ldots,n+L-1$. 

Substituting \eqref{ea:minimizer_condition} into the parallelized
process~(\ref{e:pd_process}) gives
\begin{equation}
    \bar{X}_{k+1}
    =
    \bar{X}_{k}
    +
    (t_{k+1}-t_k)
    u_k\parr{\bar{X}_k},
    \qquad
    \bar{X}_n=X_n .
\end{equation}
The base case is therefore $\bar{X}_n=X_n$.
Assume that $\bar{X}_k=X_k$ for some
$k\in\set{n,\ldots,n+L-2}$. Then
\begin{align}
    \bar{X}_{k+1}
    &=
    X_k
    +
    (t_{k+1}-t_k)u_k\parr{X_k} \\
    &=
    X_{k+1},
\end{align}
where the last equality follows from the exact solution~(\ref{e:sol_exact}).
By induction, $\bar{X}_k=X_k$ for all
$k\in\set{n,\ldots,n+L-1}$.
\end{proof}
\end{document}

%% file: math_commands.tex
\newcommand{\ie}{{i.e.}}

\theoremstyle{plain}
\newtheorem{theorem}{Theorem}%
\newtheorem{proposition}[theorem]{Proposition}

\theoremstyle{definition}

\theoremstyle{remark}

\makeatletter
\newtheorem*{rep@theorem}{\rep@title}
\newcommand{\newreptheorem}[2]{%
\newenvironment{rep#1}[1]{%
 \def\rep@title{\textbf{#2} \ref{##1}}%
 \begin{rep@theorem}}%
 {\end{rep@theorem}}}
\makeatother

\newreptheorem{theorem}{Theorem}
\newreptheorem{proposition}{Proposition}
\newreptheorem{lemma}{Lemma}
\newreptheorem{corollary}{Corollary}

\usepackage{amsmath,amsfonts,bm}

\newcommand{\norm}[1]{\left\Vert#1\right\Vert}

\newcommand{\set}[1]{\left\{#1\right\}}
\newcommand{\parr}[1]{\left (#1\right )}
\newcommand{\brac}[1]{\left [#1\right ]}

\newcommand{\too}{\rightarrow}

\newcommand{\deq}%
{\sim}%

\definecolor{mygray}{gray}{0.95}

\def\eqref#1{equation~\ref{#1}}

\def\1{\bm{1}}

\DeclareMathAlphabet{\mathsfit}{\encodingdefault}{\sfdefault}{m}{sl}
\SetMathAlphabet{\mathsfit}{bold}{\encodingdefault}{\sfdefault}{bx}{n}

\def\gL{{\mathcal{L}}}

\def\gX{{\mathcal{X}}}

\def\sN{{\mathbb{N}}}

\newcommand{\E}{\mathbb{E}}

\newcommand{\R}{\mathbb{R}}

%% file: figures/wan/teaser/teaser_new.tex
\begin{strip}
  \vspace{-0.75em}
  \centering
  \resizebox{0.93\textwidth}{!}{
  \begin{tabular}{@{} r @{\hspace{6pt}} c@{\hspace{2pt}}c@{\hspace{2pt}}c@{\hspace{2pt}}c@{\hspace{2pt}}c @{}}
                & $t{=}0$s & $t{=}2.5$s & $t{=}5$s & $t{=}7.5$s & $t{=}10$s \\ \addlinespace[2pt]
    Teacher
      & \framecellw{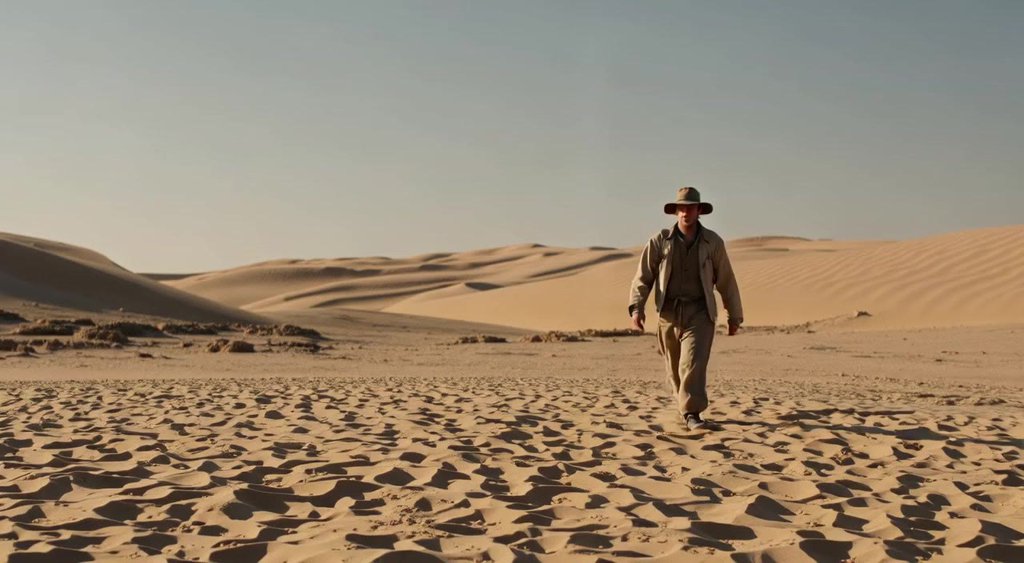} & \framecellw{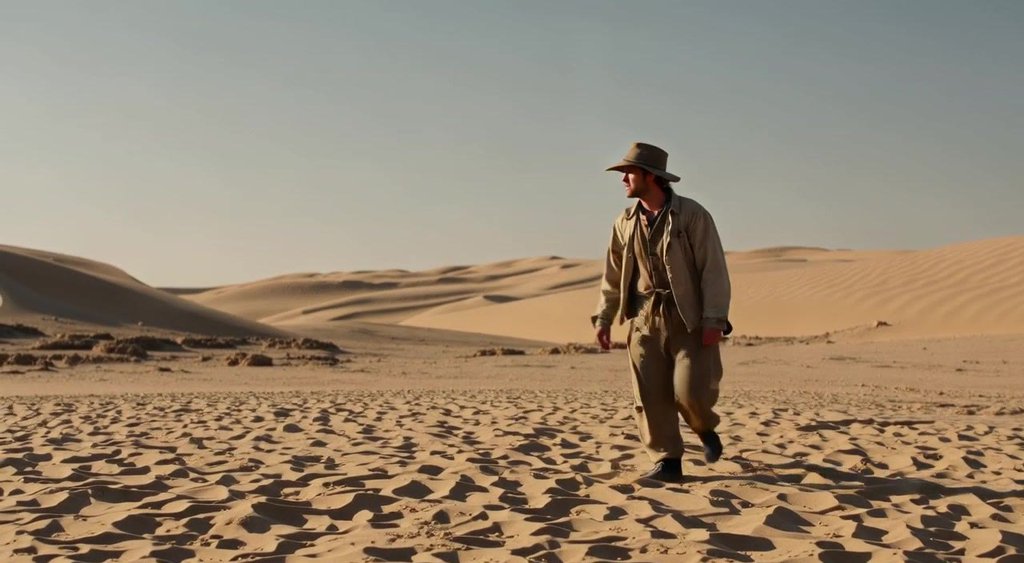} & \framecellw{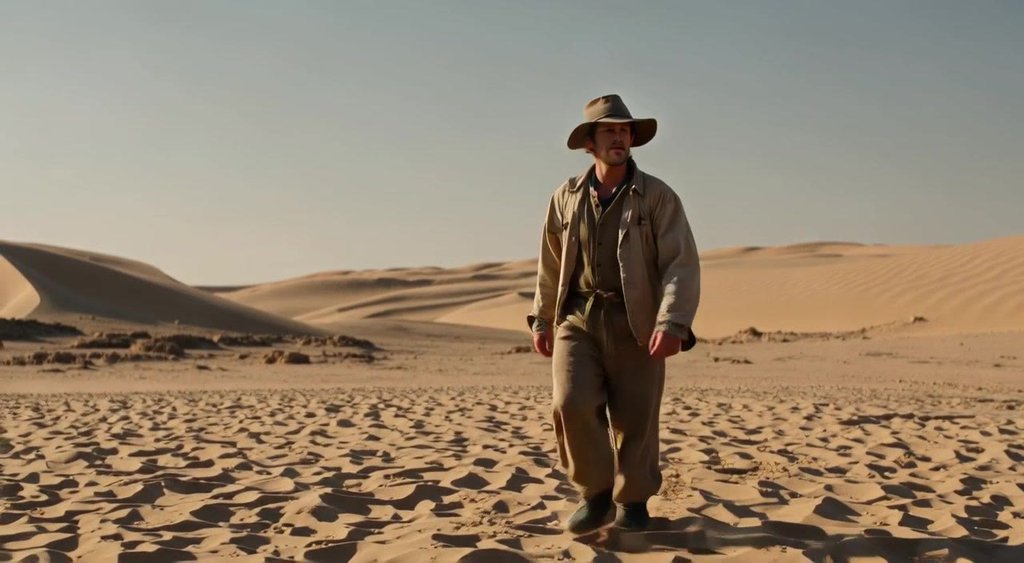} & \framecellw{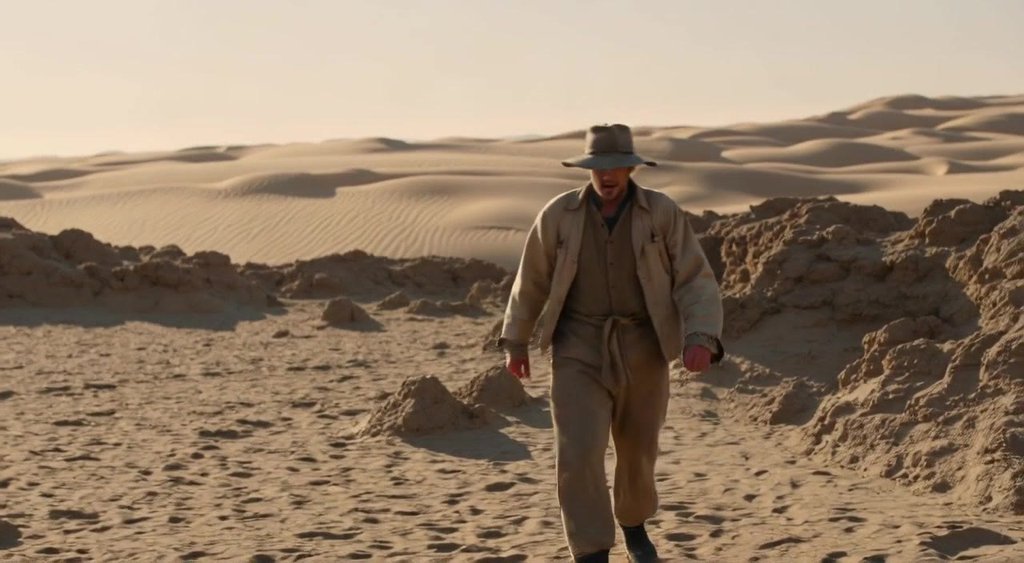} & \framecellw{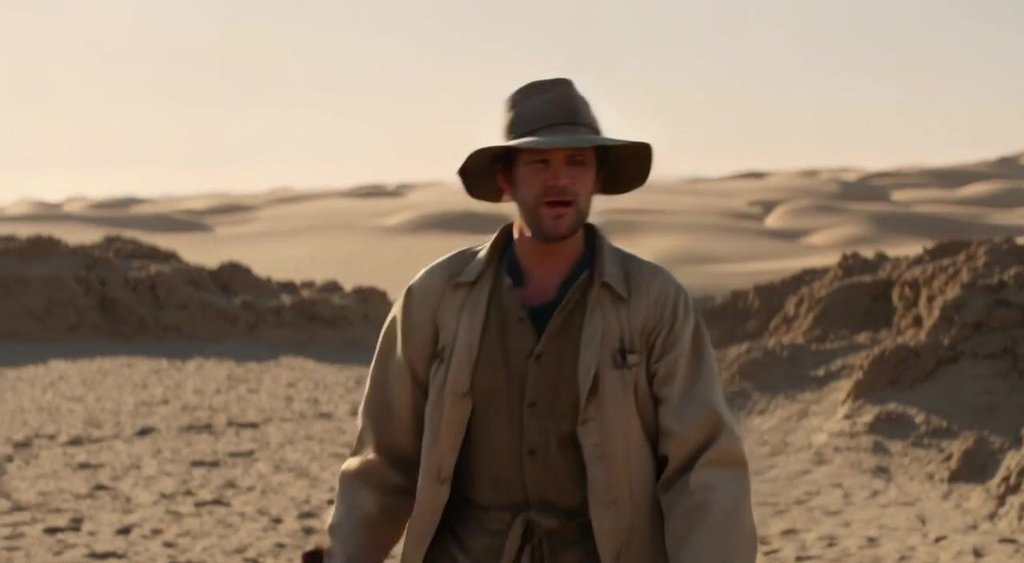} \\ \addlinespace[2pt]
    \textbf{PDD}
      & \framecellw{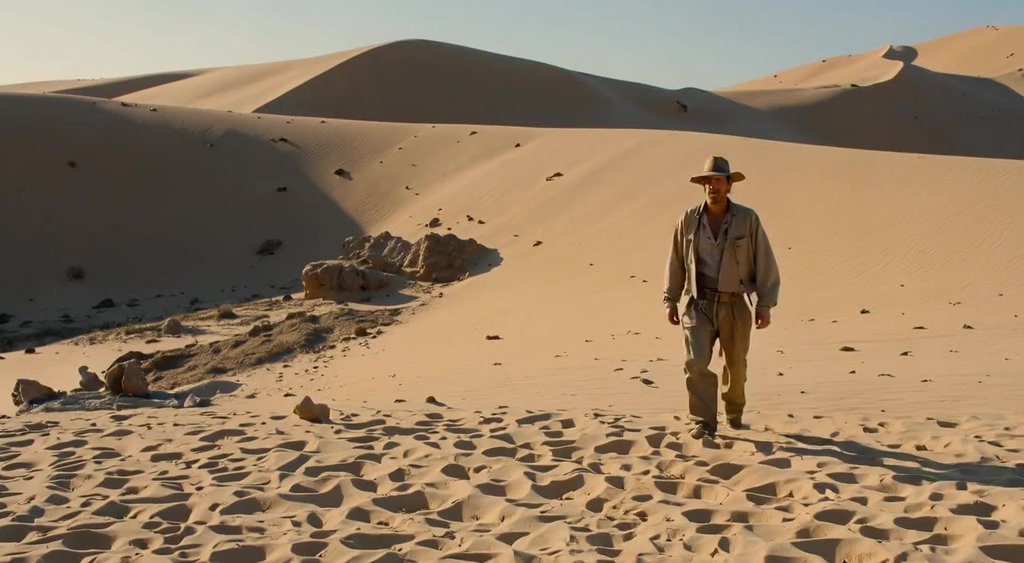} & \framecellw{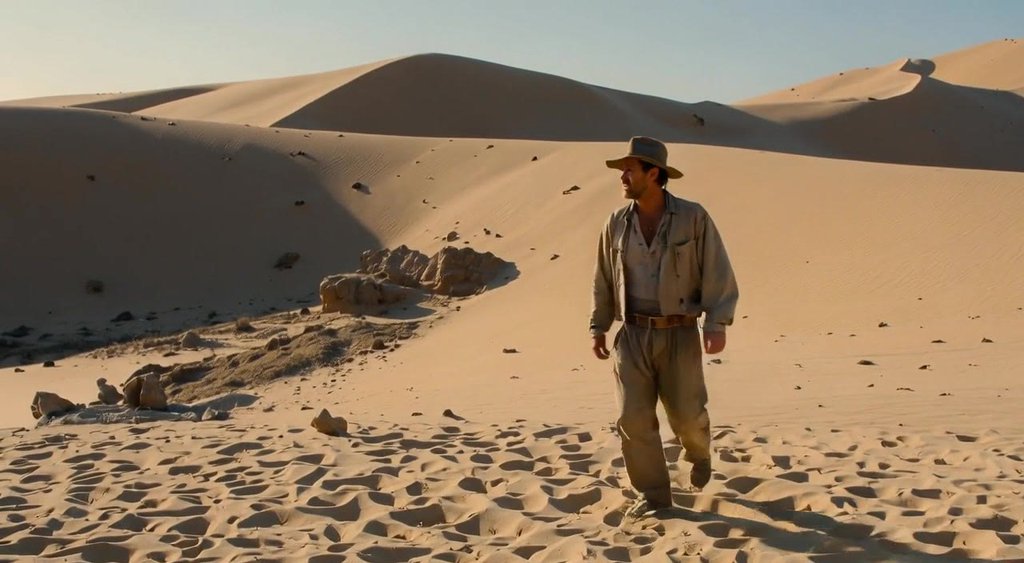} & \framecellw{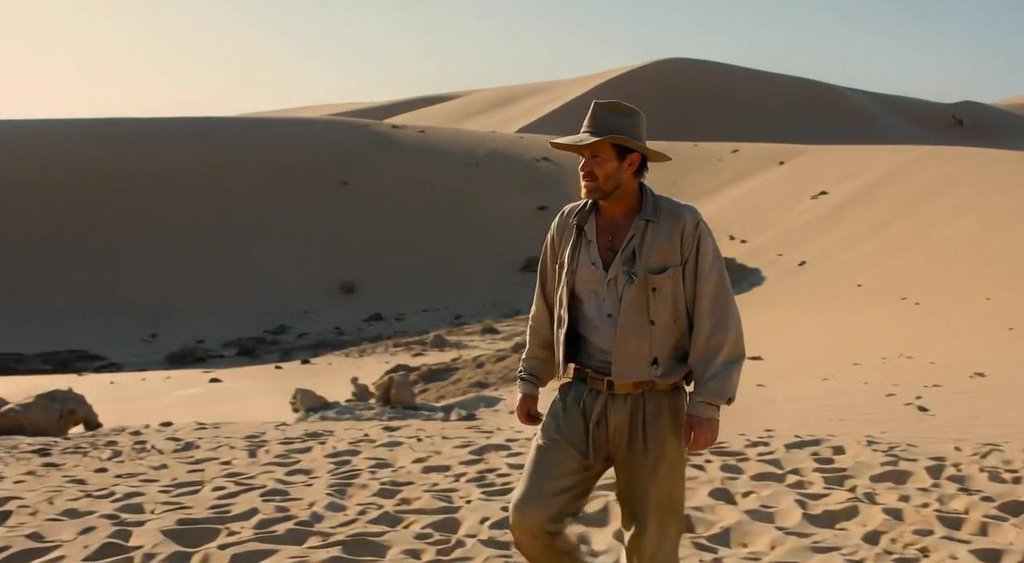} & \framecellw{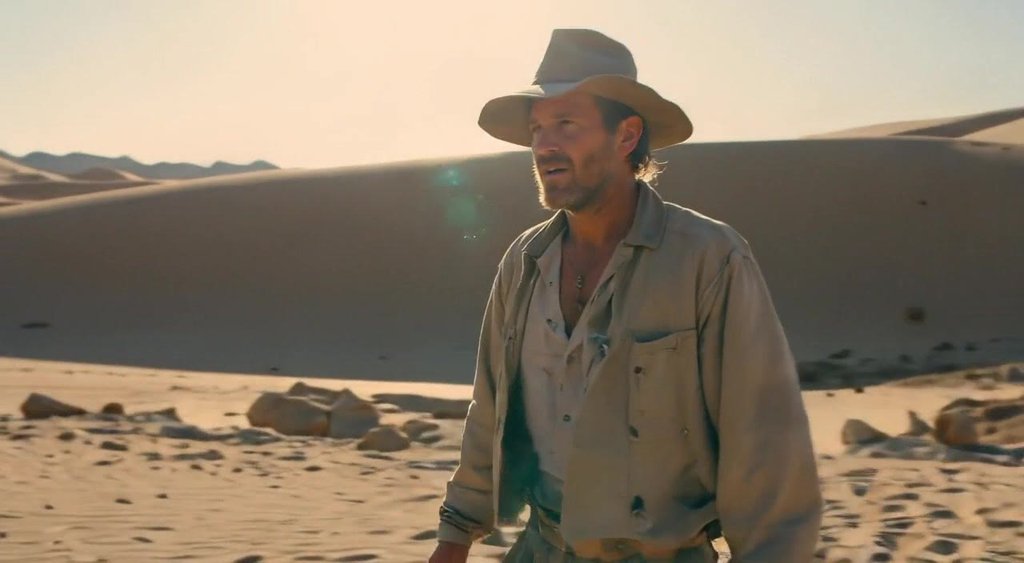} & \framecellw{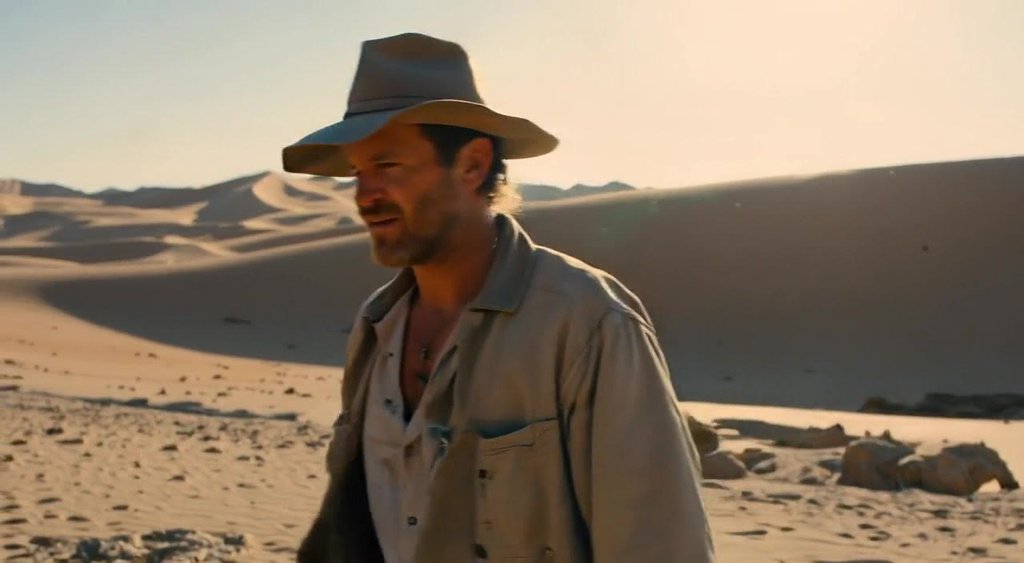} \\ \addlinespace[2pt]
    Distilled 
      & \framecellw{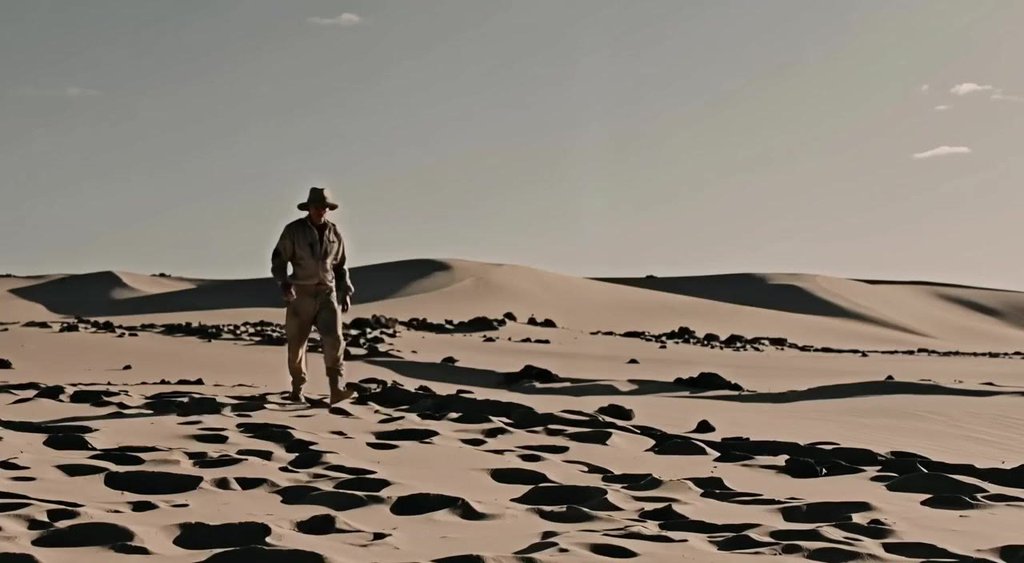} & \framecellw{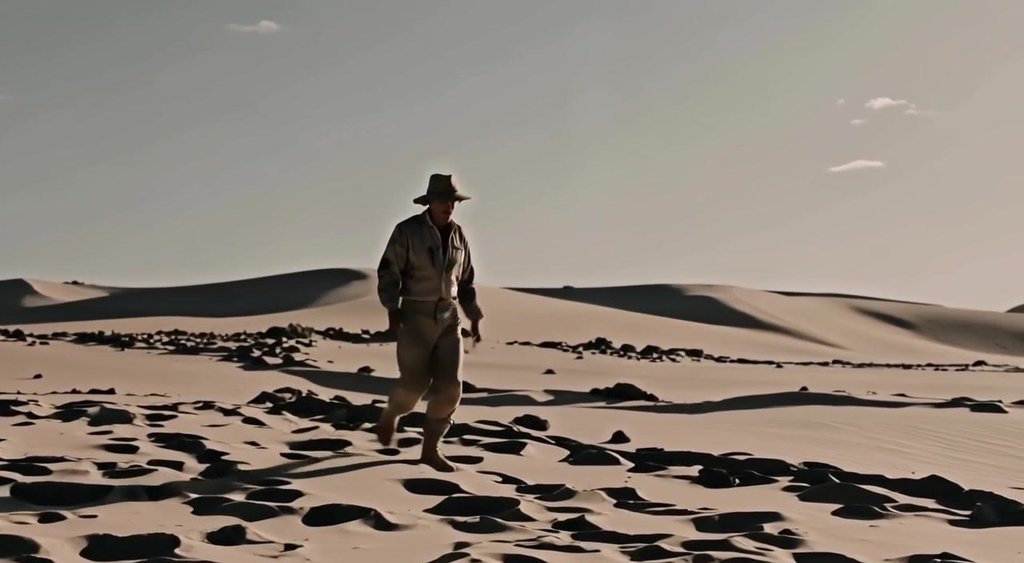} & \framecellw{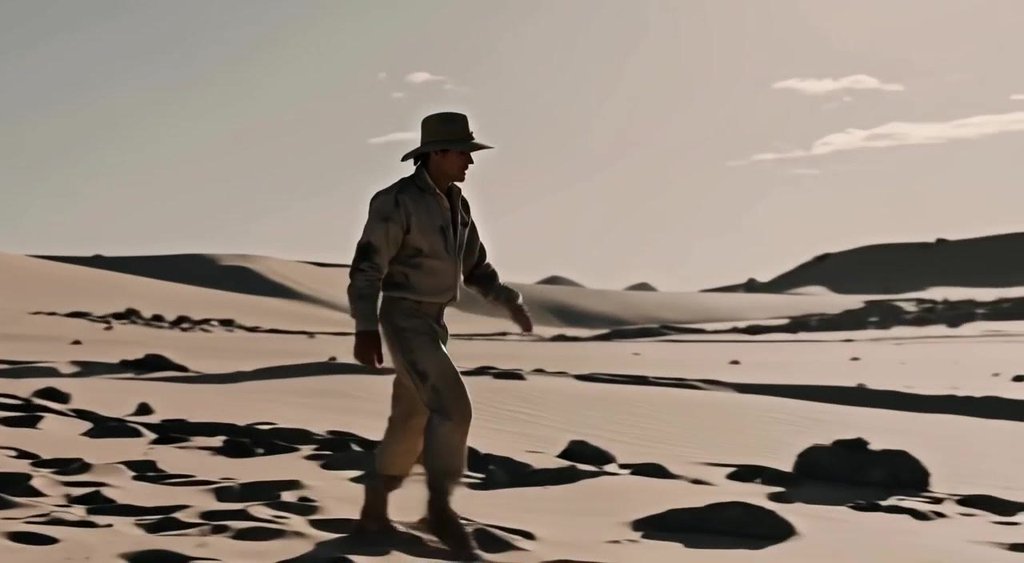} & \framecellw{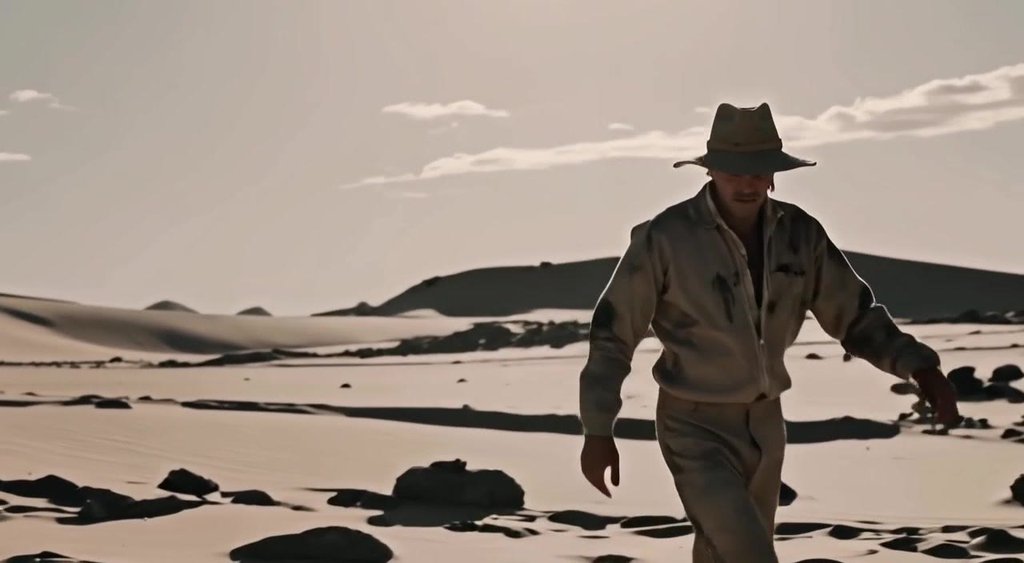} & \framecellw{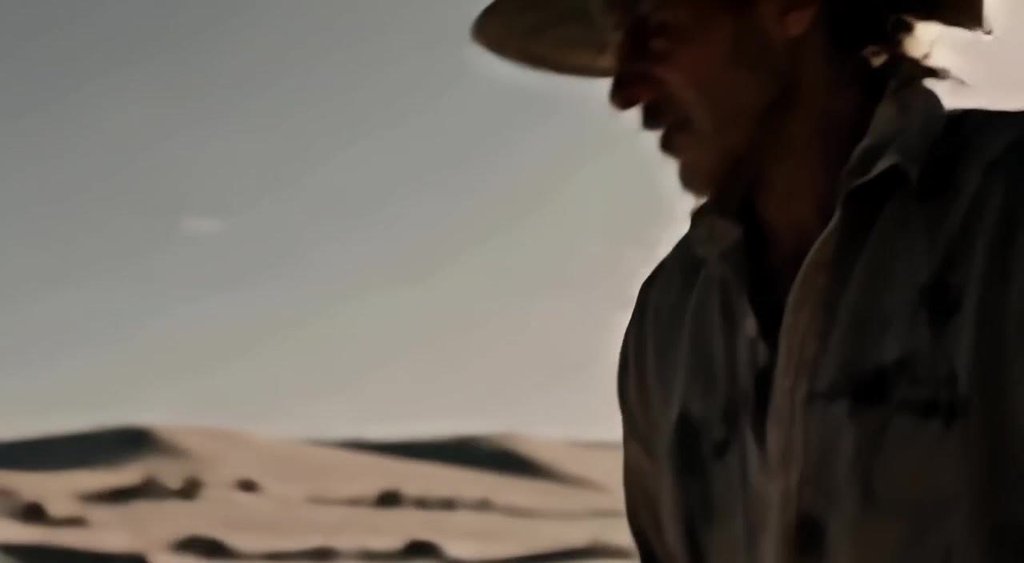} \\ \addlinespace[2pt]
                & \multicolumn{5}{@{} p{16.00cm} @{}}{\tiny\itshape ``A solitary walker moves across a vast, arid desert landscape under a scorching sun, the dry, rhythmic crunch of boots pressing into shifting sand echoing in the open, expansive space. The walker is wearing a tattered, beige jacket and loose-fitting pants that rustle with every step, a wide-brimmed hat shading their face from the intense sunlight. They walk with a determined yet weary posture, occasionally glancing around at the expansive sand dunes and rocky outcrops as a low, mournful whistle of wind sweeps across the barren terrain. In the background, distant sand dunes roll endlessly towards the horizon, creating a sense of isolation and vastness. The scene transitions between medium shots and close-ups as the walker trudges forward, their heavy, labored breathing audible in the stillness, emphasizing the harsh, desolate environment.''} \\
  \end{tabular}%
  }
  \vspace{1em}
  
  \resizebox{0.9635\textwidth}{!}{%
  \begin{tabular}{@{} r @{\hspace{6pt}} c@{\hspace{2pt}}c@{\hspace{2pt}}c @{\hspace{4pt}} c@{\hspace{2pt}}c@{\hspace{2pt}}c @{\hspace{6pt}}}
                & \multicolumn{3}{c}{Noise Seed 1}
                & \multicolumn{3}{c}{Noise Seed 2} \\
                
                \cmidrule(l{2pt}r{6pt}){2-4} \cmidrule(l{2pt}r{6pt}){5-7}
                
                & $t{=}0.5$s & $t{=}2.5$s & $t{=}4.5$s & $t{=}0.5$s & $t{=}2.5$s & $t{=}4.5$s \\
    \textbf{PDD}
      & \framecell{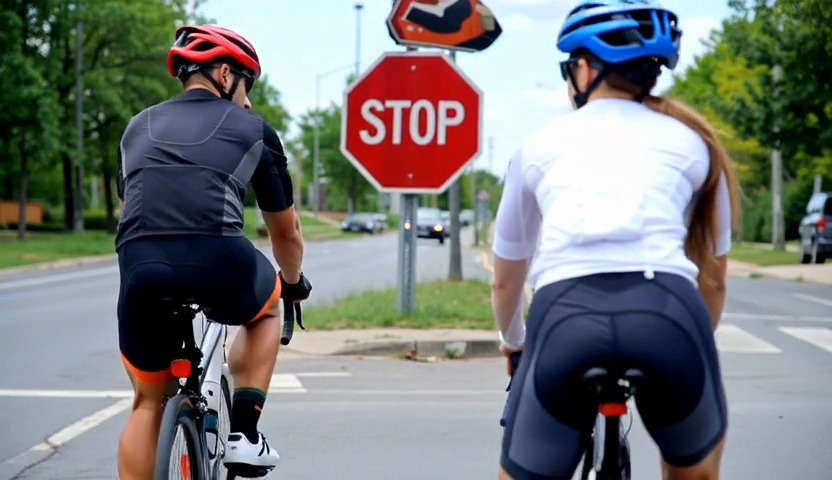}
      & \framecell{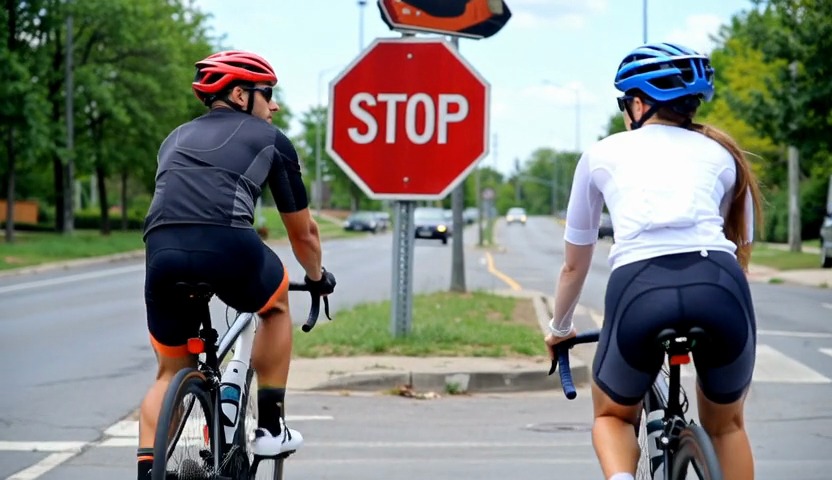}
      & \framecell{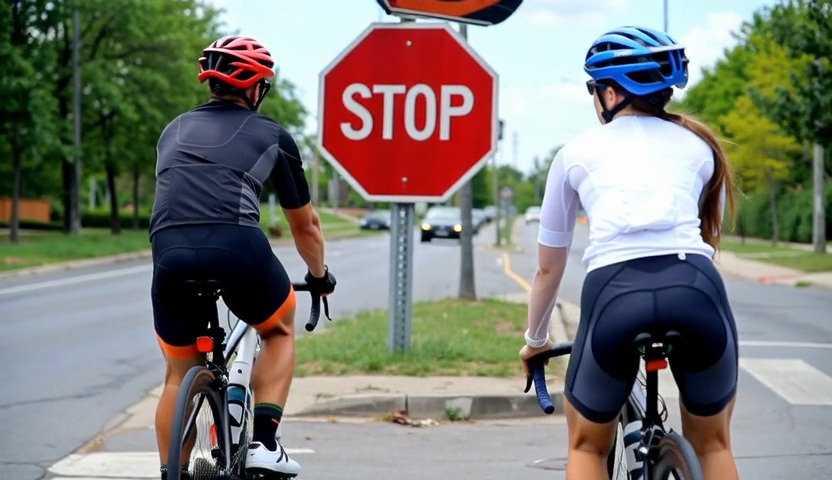}
      & \framecell{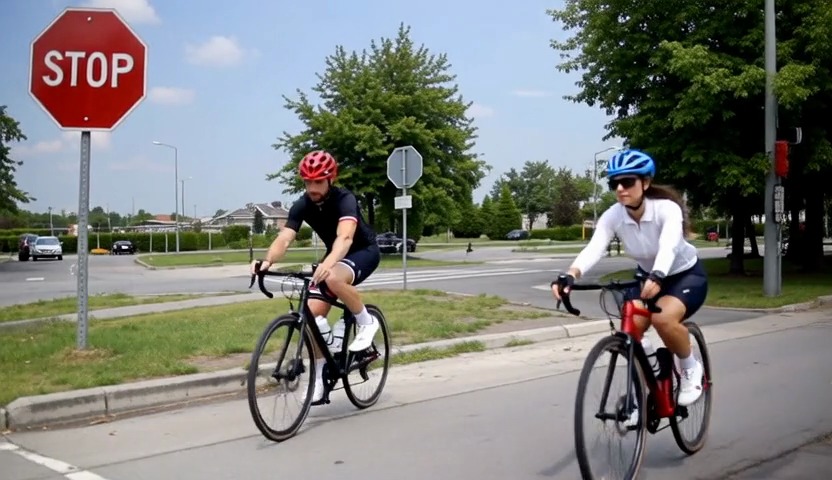}
      & \framecell{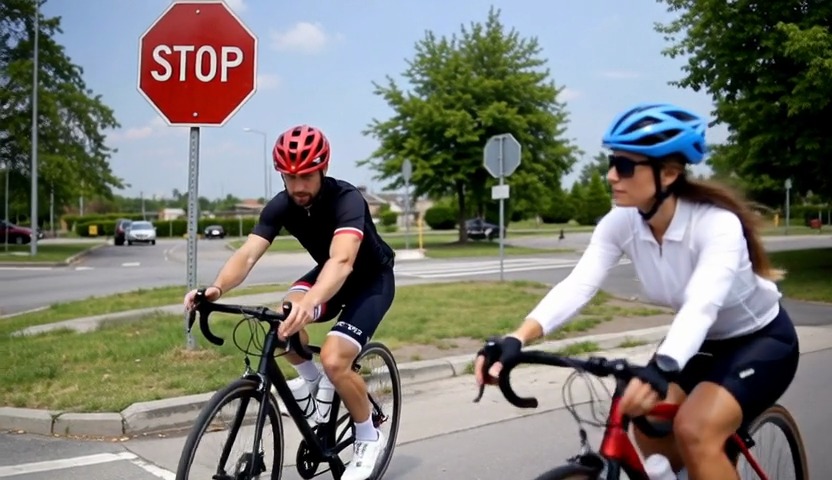}
      & \framecell{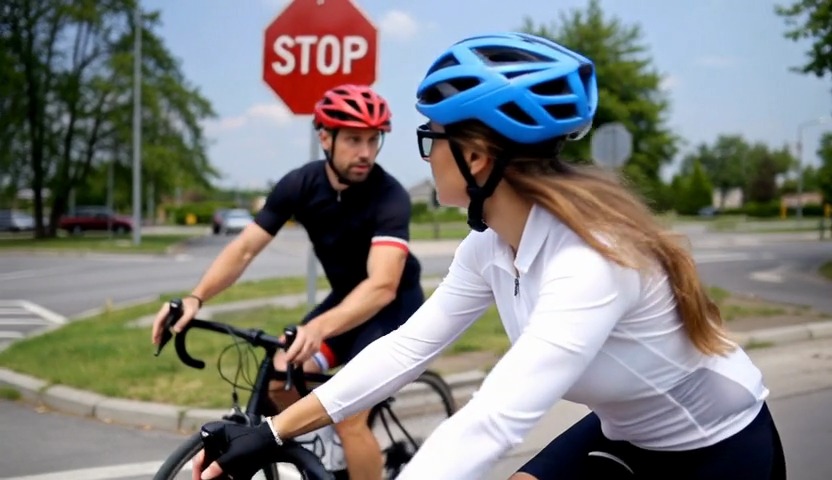} 
      \\ \addlinespace[2pt]
    DMD2
      & \framecell{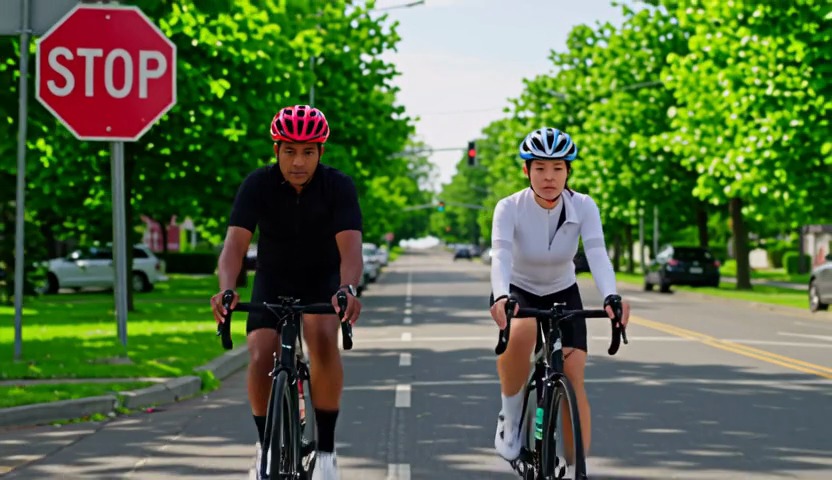}
      & \framecell{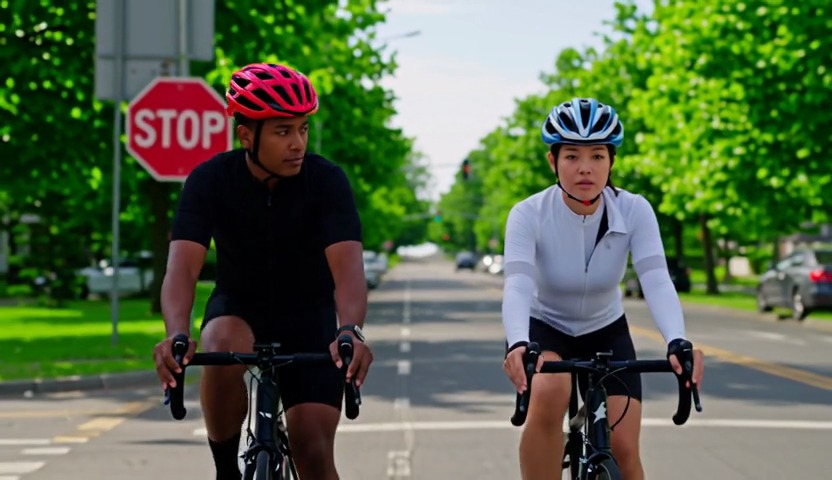}
      & \framecell{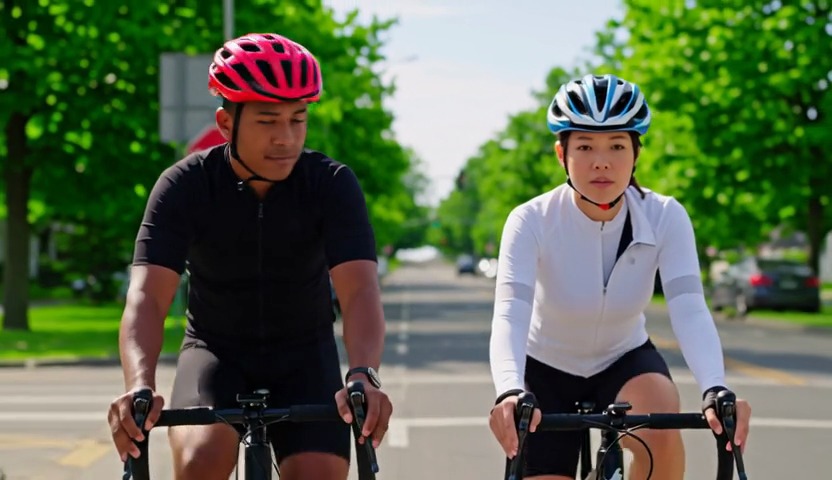}
      & \framecell{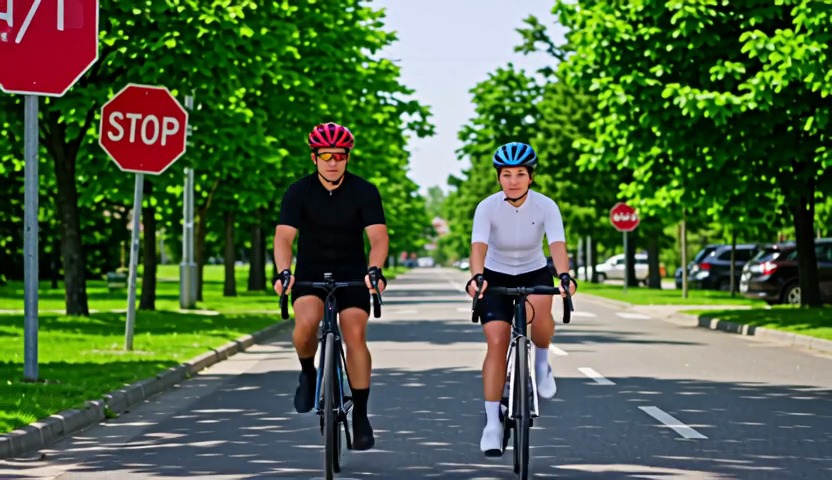}
      & \framecell{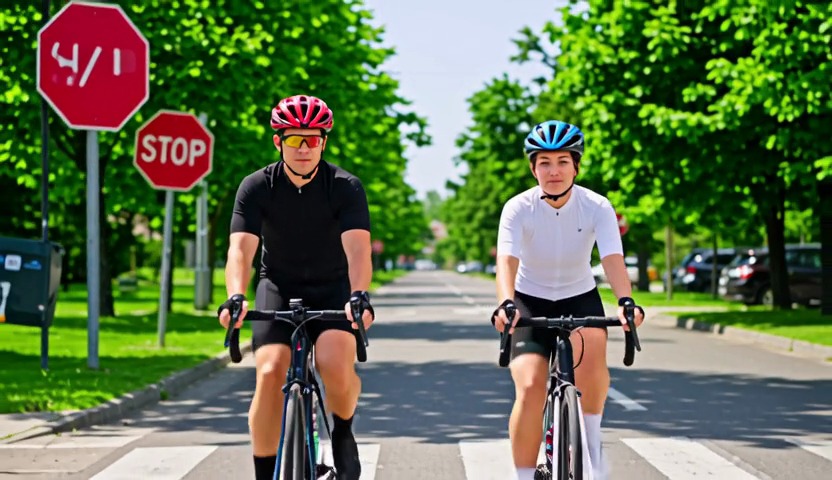}
      & \framecell{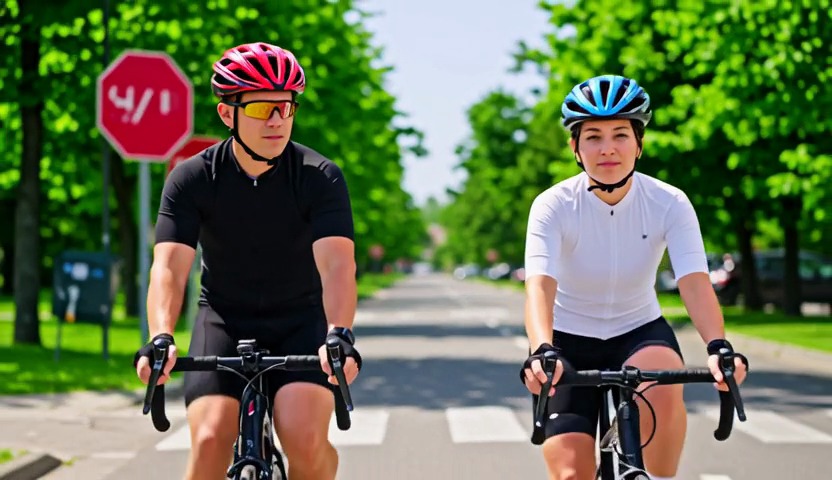} 
      \\ \addlinespace[2pt]
    \hspace{11pt}AnyFlow
      & \framecell{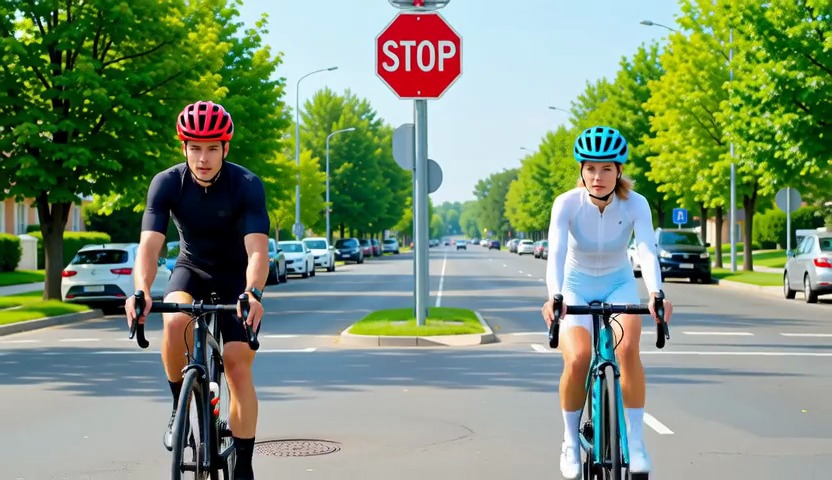}
      & \framecell{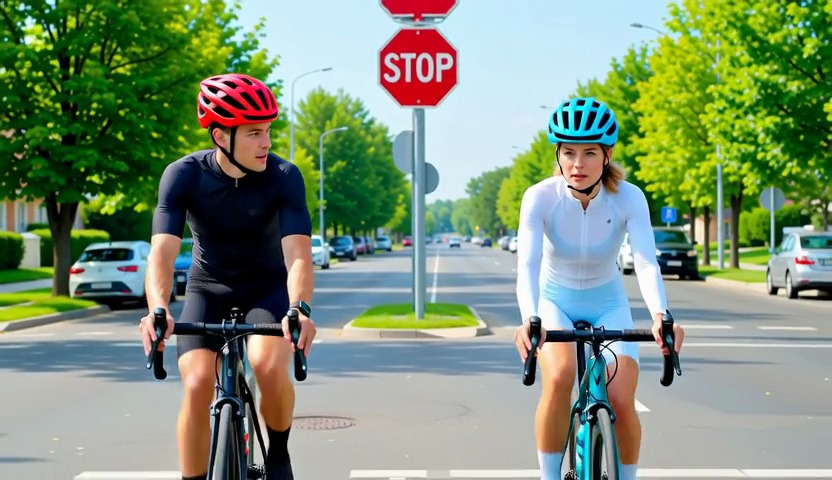}
      & \framecell{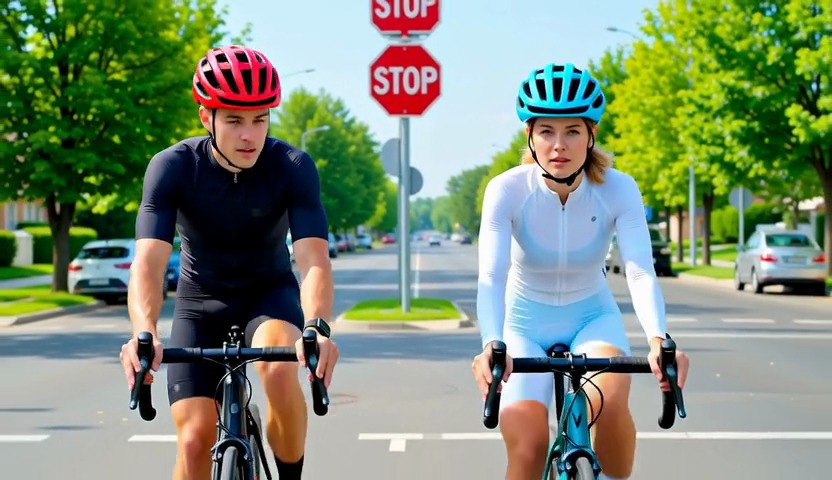}
      & \framecell{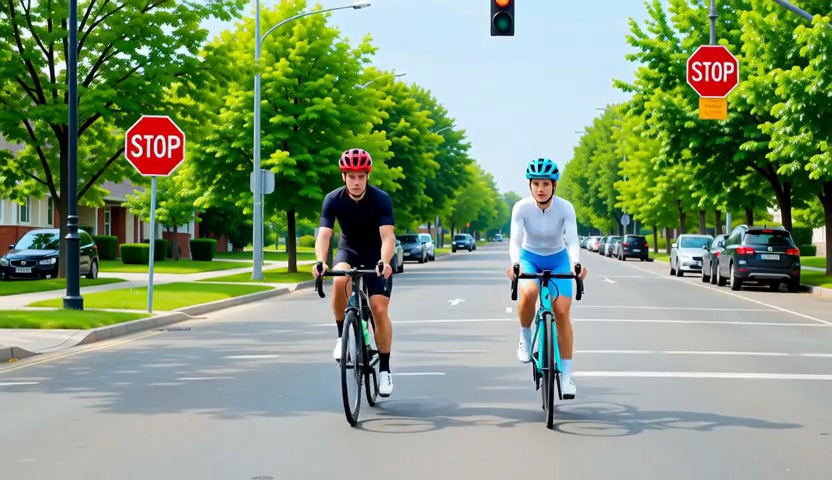}
      & \framecell{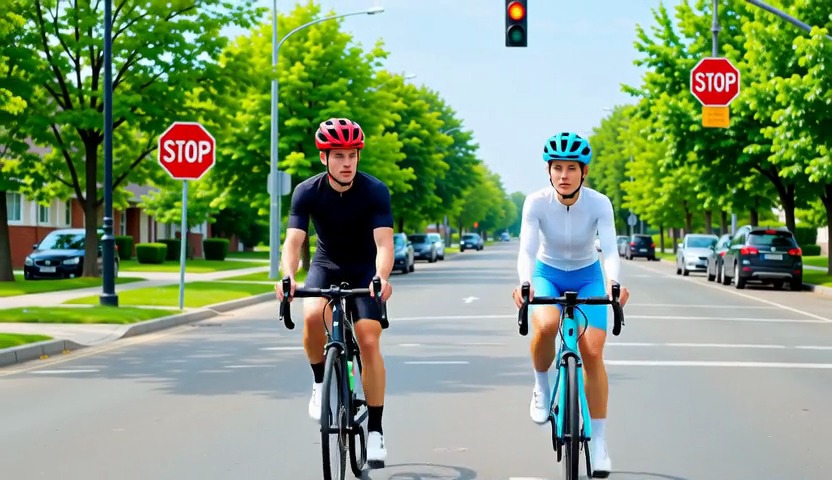}
      & \framecell{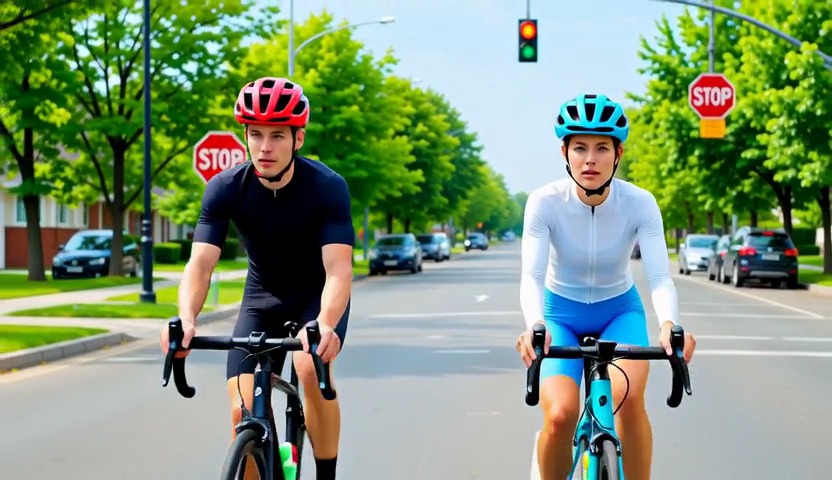} 
      \\ \addlinespace[2pt]
      & \multicolumn{6}{@{} p{16.6cm} @{}}{\tiny\itshape ``Two cyclists, one male and one female, are approaching a stop sign at the intersection from opposite directions. The male cyclist is wearing a red helmet and black cycling gear, while the female cyclist is wearing a blue helmet and white cycling clothes. They are both pedaling towards the stop sign, gradually slowing down as they near the intersection. The male cyclist is glancing left, checking for traffic, while the female cyclist is focused straight ahead. The background shows a quiet suburban street with green trees and parked cars. The video captures their approach in a tracking shot, gradually zooming in on each cyclist as they slow to a stop, maintaining eye-level perspectives.''} \\
  \end{tabular}
  }\vspace{-0.3em}
  \captionof{figure}{
  \textbf{(top)} 10s 720p videos with audio from LTX-2.3 teacher ($4\times 30$ NFE), Parallel Decoding Distillation (PDD) (8 NFE), and the official distilled model (8 NFE). PDD can follow the teacher more closely. \textbf{(bottom)} PDD vs.\@ DMD2 and AnyFlow on the Wan2.1 Text-to-Video 14B model. All methods use 4 NFE. Notably, PDD demonstrates high video quality while preserving better generation diversity across distinct initial noises compared to the baselines. %
  \label{fig:teaser}
  }
\end{strip}

%% file: figures/wan/main_paper_figure.tex
\begin{figure*}[t]
  \centering

  \resizebox{\textwidth}{!}{%

  \begin{tabular}{@{} r @{\hspace{6pt}} c@{\hspace{2pt}}c@{\hspace{2pt}}c @{\hspace{4pt}} c@{\hspace{2pt}}c@{\hspace{2pt}}c @{\hspace{6pt}} c @{}}
                & \multicolumn{3}{c}{\large Noise 1}
                & \multicolumn{3}{c}{\large Noise 2} & \\

                \cmidrule(l{6pt}r{2pt}){2-4} \cmidrule(l{2pt}r{6pt}){5-7}

                & $t{=}0.5$s & $t{=}2.5$s & $t{=}4.5$s & $t{=}0.5$s & $t{=}2.5$s & $t{=}4.5$s & \\
    \textbf{PDD}
      & \framecell{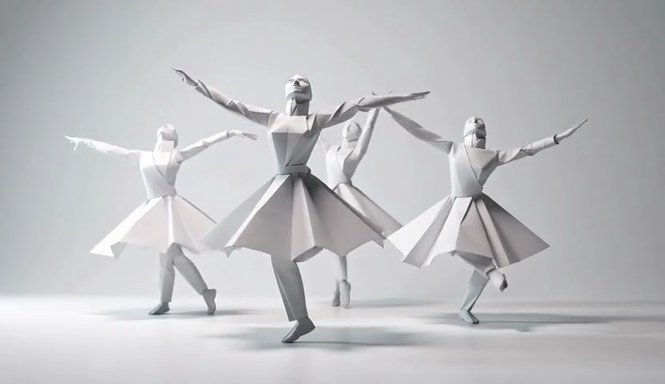}
      & \framecell{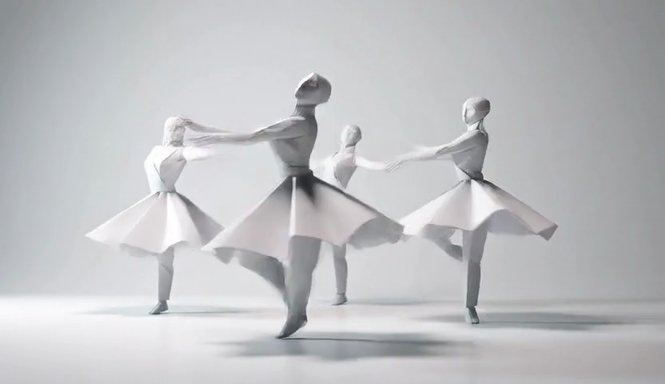}
      & \framecell{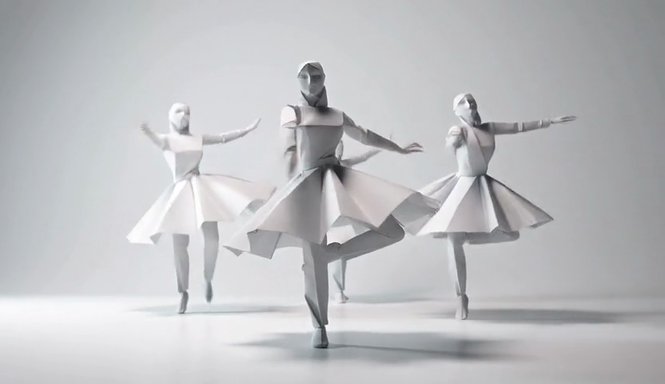}
      & \framecell{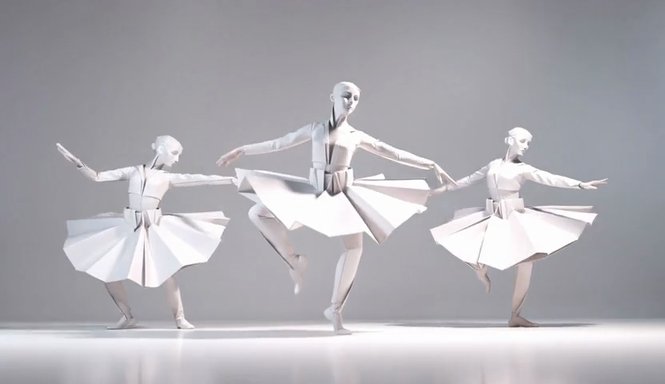}
      & \framecell{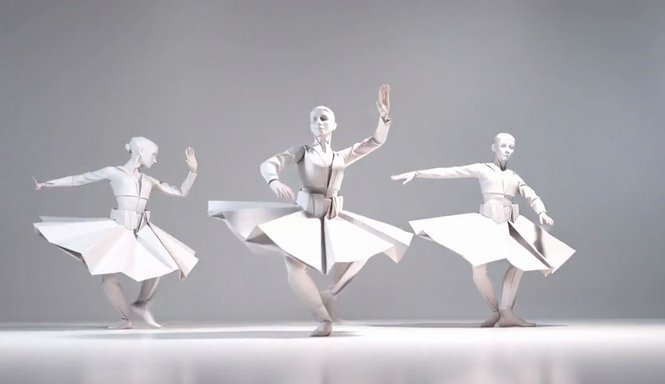}
      & \framecell{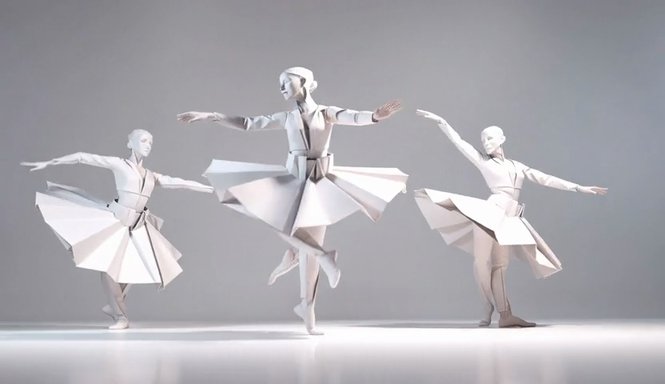} 
      & \phantom{AnyFlow} \\ \addlinespace[2pt]
    DMD2
      & \framecell{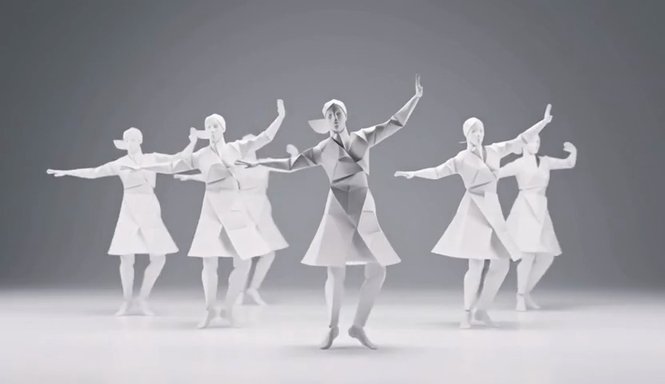}
      & \framecell{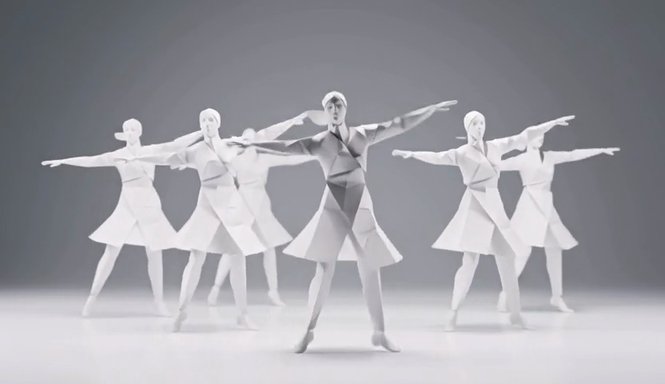}
      & \framecell{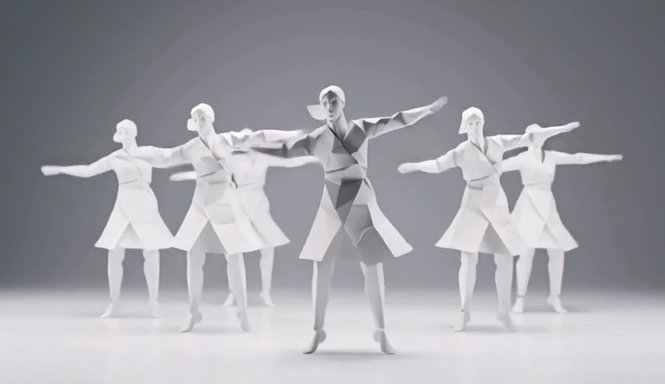}
      & \framecell{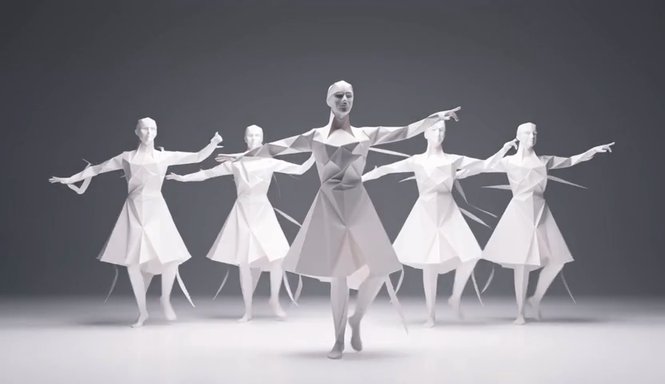}
      & \framecell{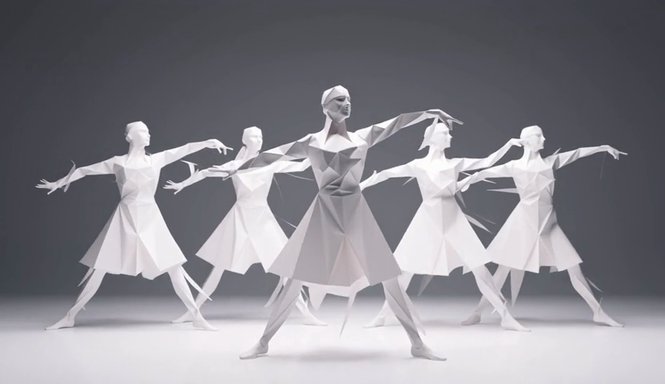}
      & \framecell{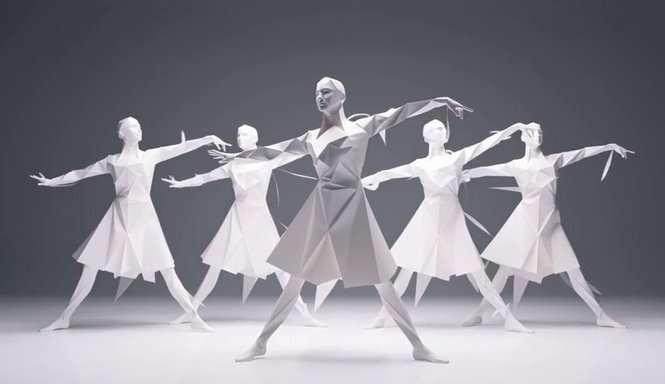} 
      & \phantom{AnyFlow} \\ \addlinespace[2pt]
    AnyFlow
      & \framecell{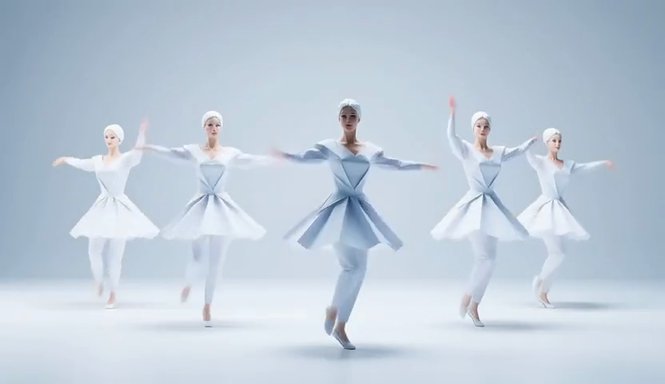}
      & \framecell{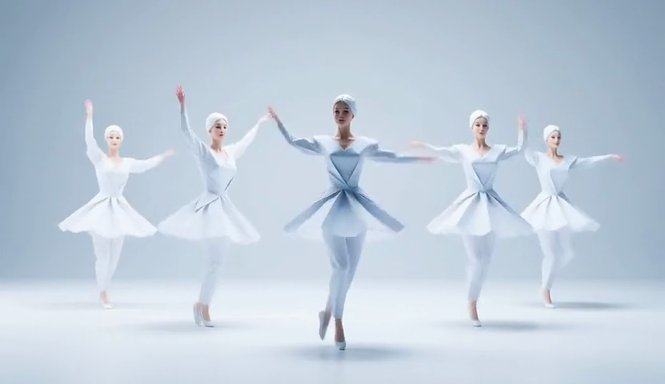}
      & \framecell{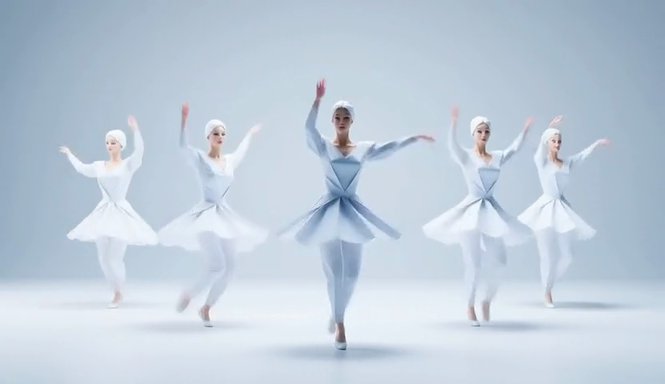}
      & \framecell{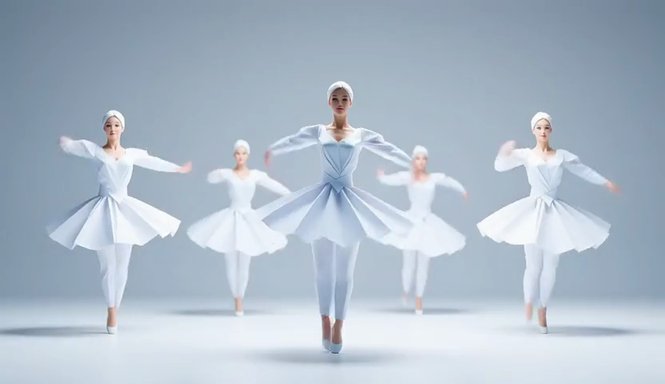}
      & \framecell{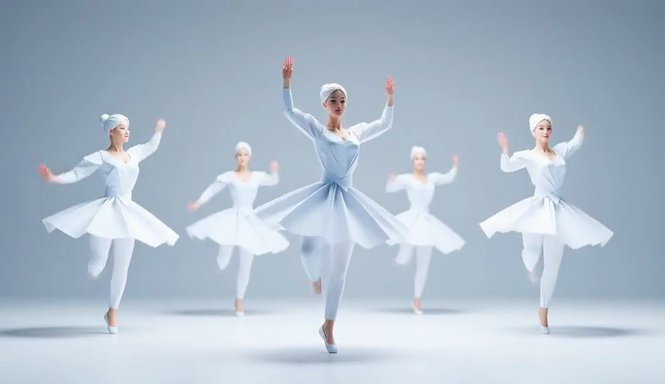}
      & \framecell{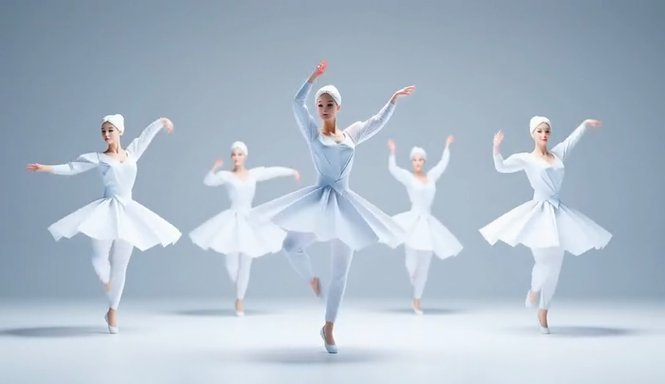} 
      & \phantom{AnyFlow} \\ \addlinespace[2pt]
      & \multicolumn{6}{@{} p{16.6cm} @{}}{\footnotesize\itshape ``3D render of origami dancers made from white paper, performing a modern dance routine against a pristine white background in a studio setting. Each dancer has delicate folds and creases that catch the light, enhancing their ethereal appearance. They are gracefully moving in synchronized motions, expressing fluidity and elegance through their postures. The scene focuses on a medium close-up to capture the intricate details of the dancers' movements and the soft, clean background.''} & \\ \addlinespace[2pt]
    \textbf{PDD}
      & \framecell{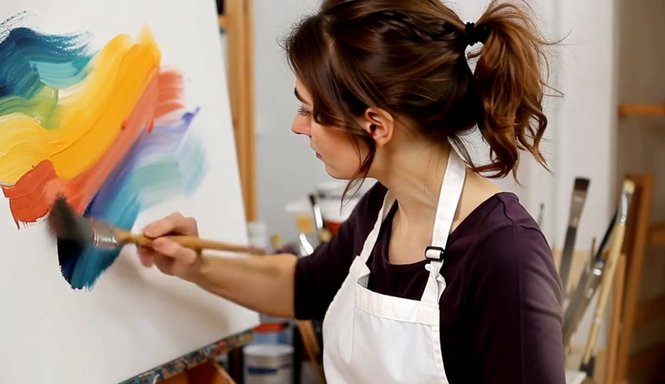}
      & \framecell{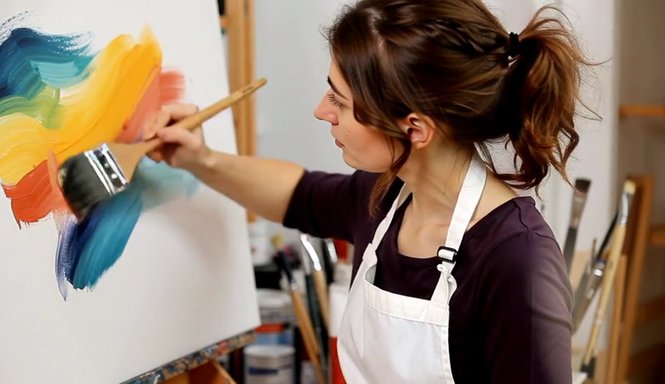}
      & \framecell{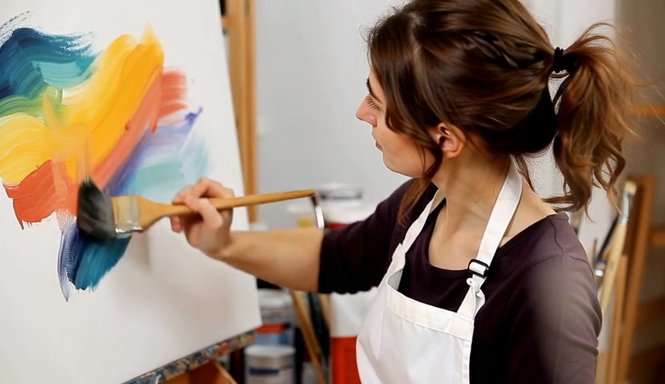}
      & \framecell{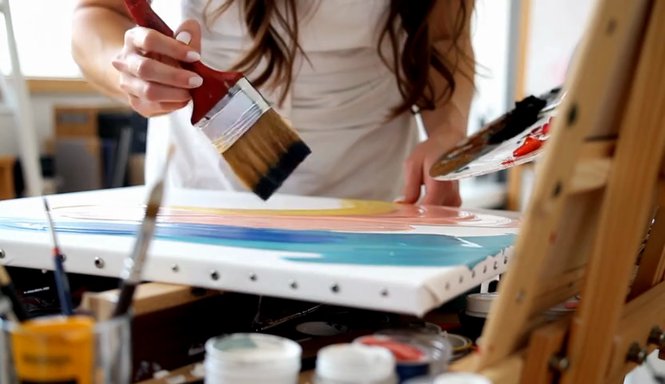}
      & \framecell{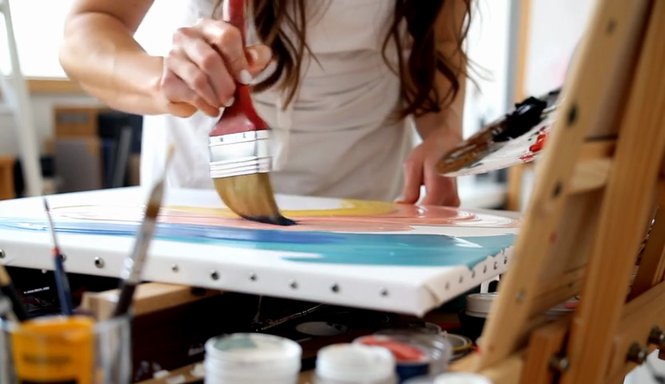}
      & \framecell{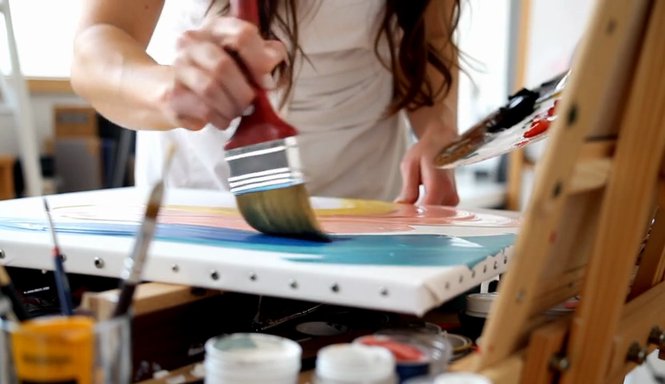} 
      & \phantom{AnyFlow} \\ \addlinespace[2pt]
    DMD2
      & \framecell{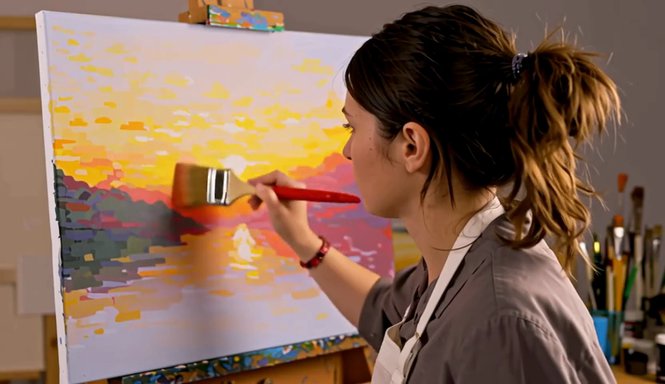}
      & \framecell{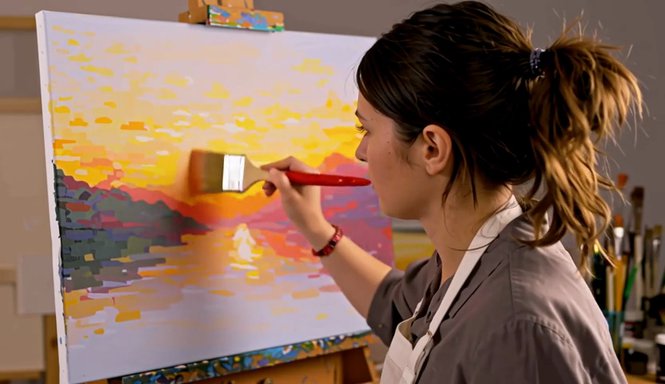}
      & \framecell{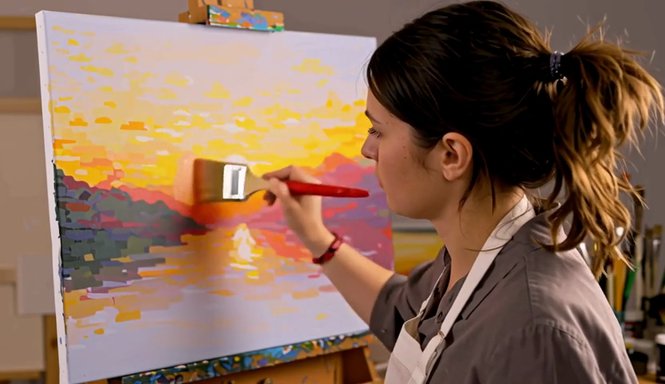}
      & \framecell{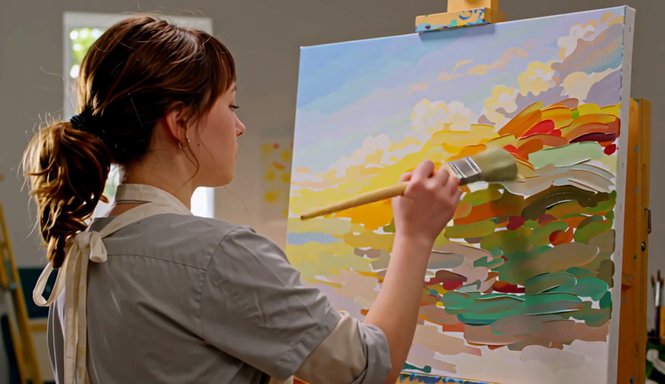}
      & \framecell{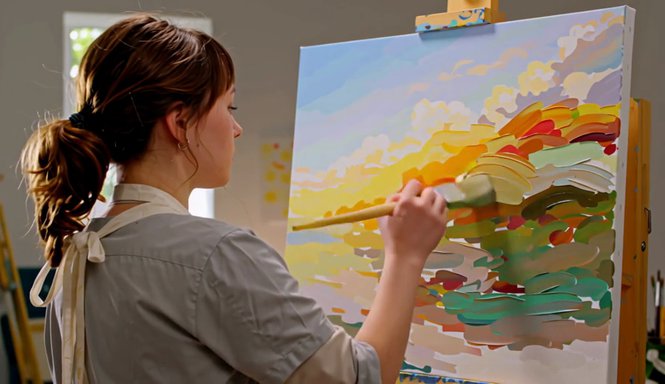}
      & \framecell{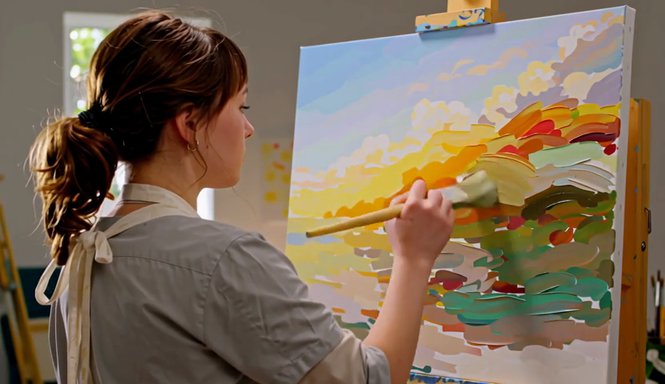} 
      & \phantom{AnyFlow} \\ \addlinespace[2pt]
    AnyFlow
      & \framecell{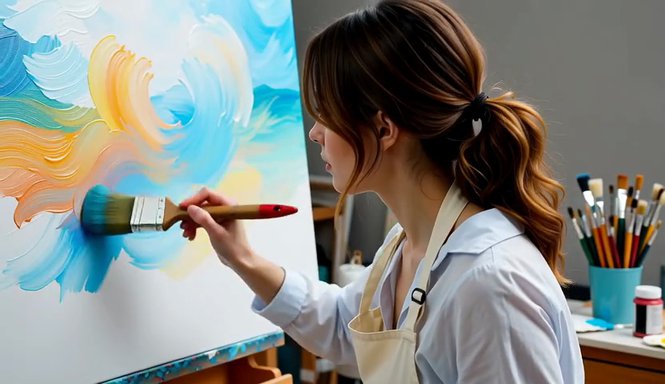}
      & \framecell{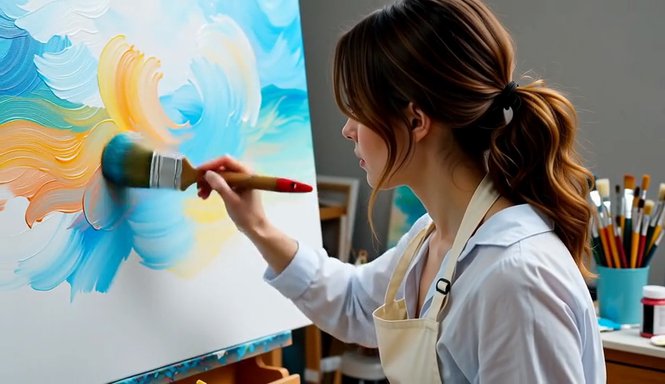}
      & \framecell{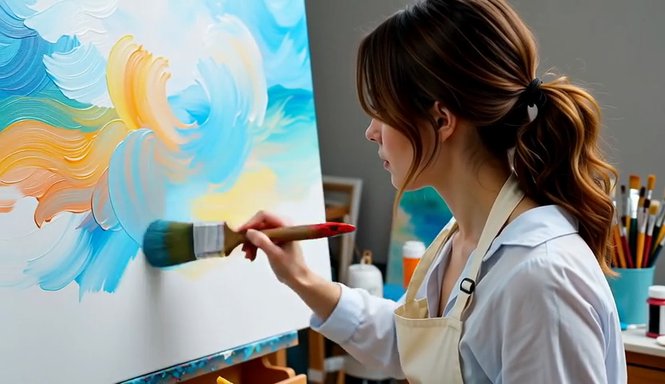}
      & \framecell{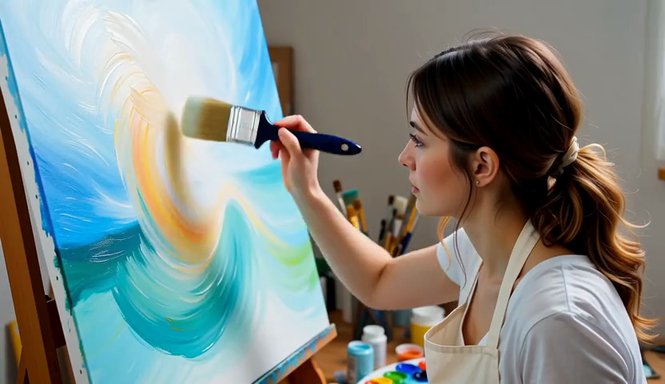}
      & \framecell{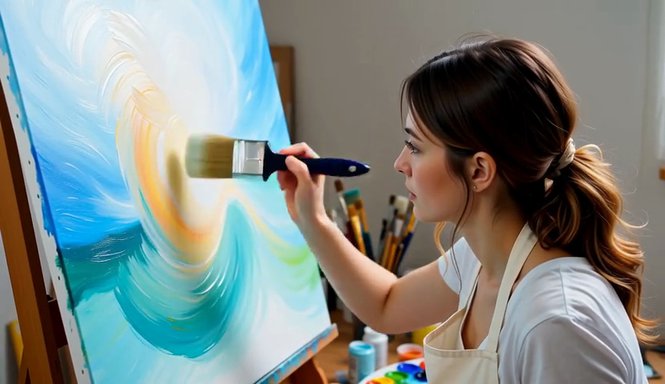}
      & \framecell{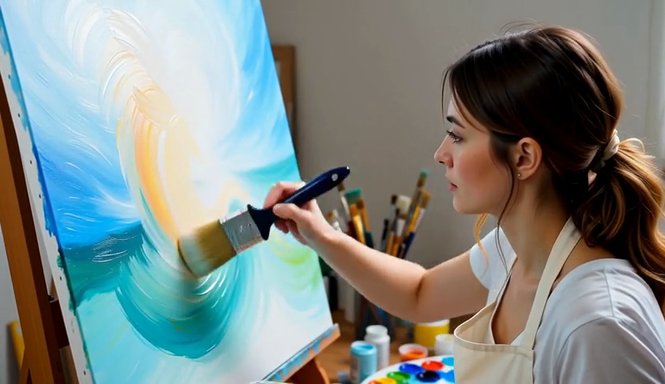} 
      & \phantom{AnyFlow} \\ \addlinespace[2pt]
      & \multicolumn{6}{@{} p{16.6cm} @{}}{\footnotesize\itshape ``Close-up of an artist painting on a canvas using a large round brush. The artist has medium-length brown hair tied back in a loose ponytail and wears a white apron over their clothes. They are focused intently, their hand moving smoothly across the canvas. The brushstrokes are visible, blending colors together in fluid motions. The background shows part of the artist's workspace, with other brushes and paint tubes nearby. The scene emphasizes the intricate detail and the flowing motion of the brush as it interacts with the paint and canvas.''} & \\
  \end{tabular}%
  }

  \vspace{-3pt}
  \caption{PDD - Midpoint (ours) vs. DMD2, and AnyFlow on the Wan Text-to-Video 14B model. The three methods use 4 NFE. Notably, PDD achieves high-quality video while preserving better diversity.}
  \label{fig:origami_dancers}
  \vspace{-10pt}
\end{figure*}

%% file: figures/qwenimage/diversity_tables/7/table.tex
\begin{figure}[t]
  \centering
  \resizebox{0.95\textwidth}{!}{%

  \begin{tabular}{@{} r @{\hspace{6pt}} c@{\hspace{2pt}}c@{\hspace{2pt}} c@{\hspace{2pt}}c @{\hspace{6pt}} c @{}}
    & {\footnotesize Noise 1} & {\footnotesize Noise 2} &  {\footnotesize Noise 3} & {\footnotesize Noise 4} & \\
    Teacher
      & \framecell{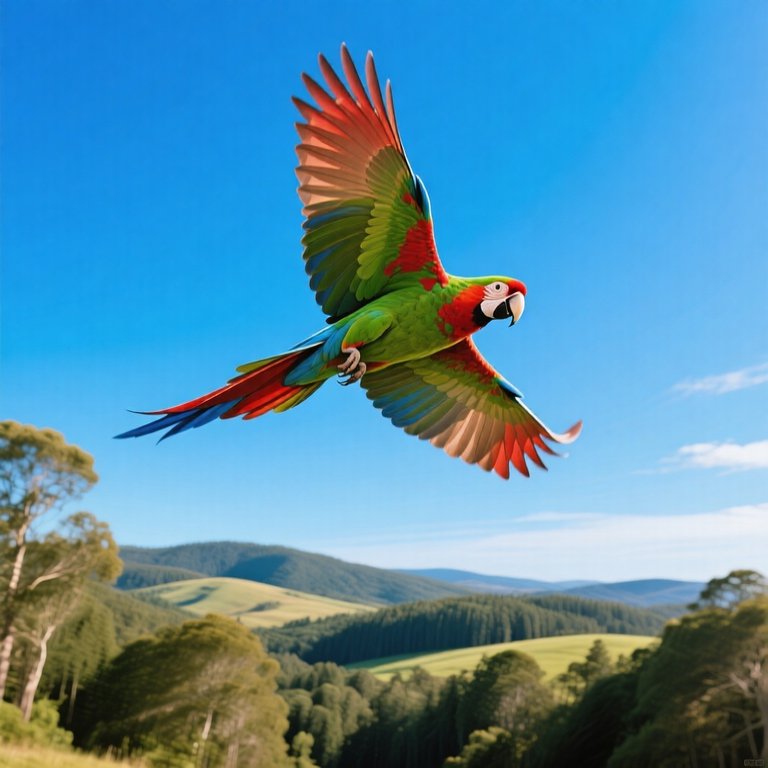}
      & \framecell{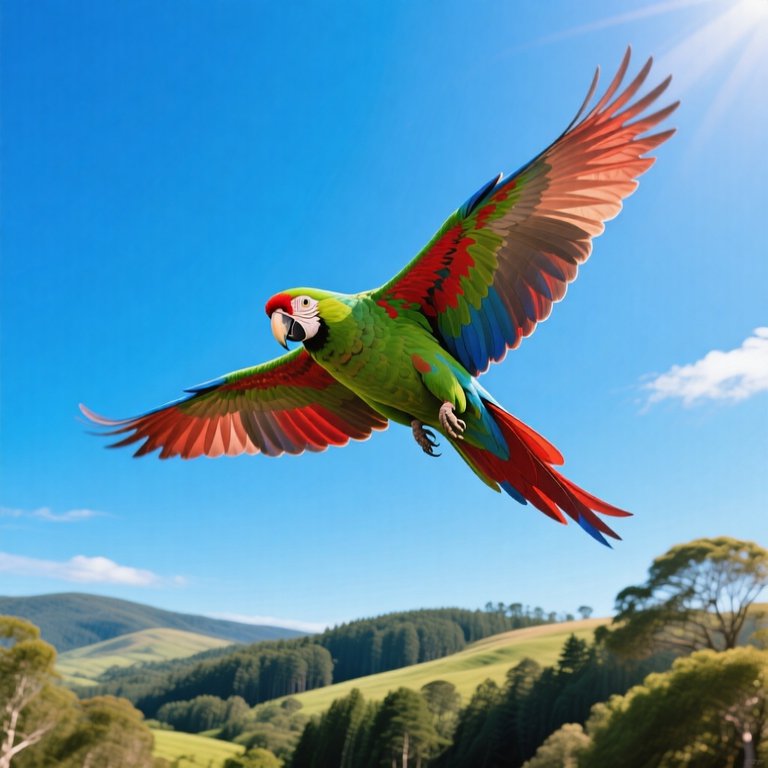}
      & \framecell{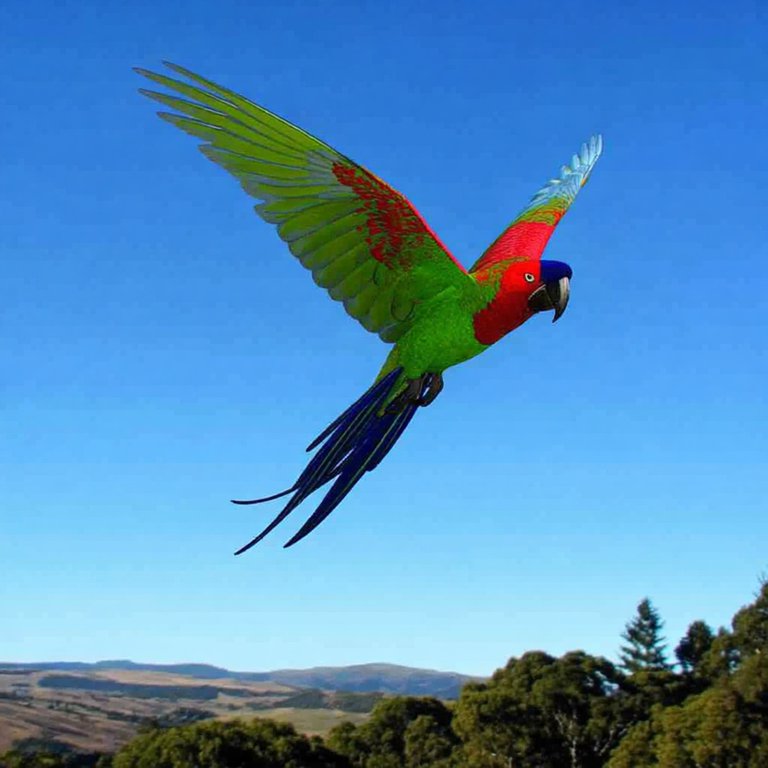}
      & \framecell{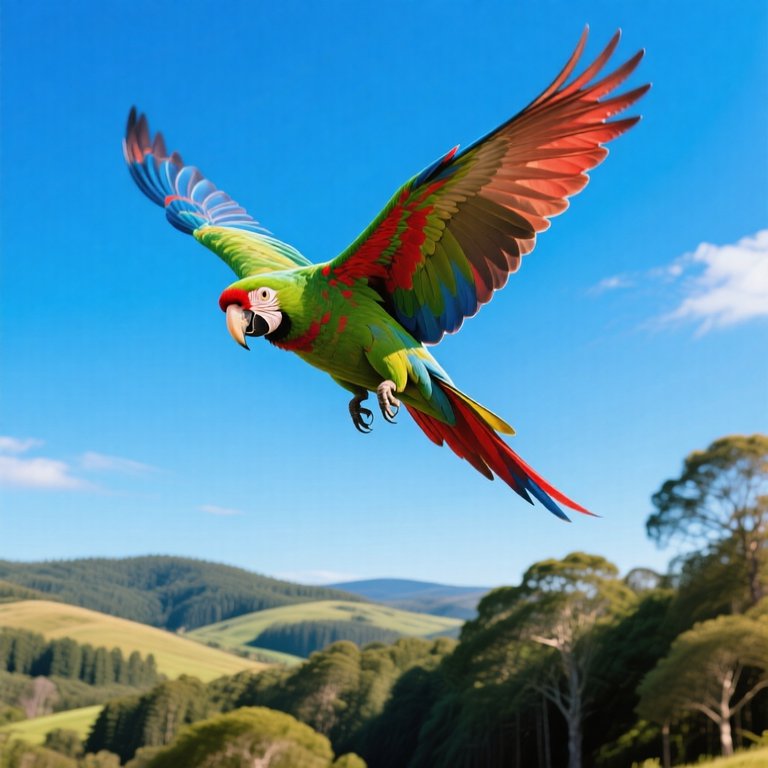}
      & \phantom{Teacher} \\ \addlinespace[2pt]
    \textbf{PDD}
      & \framecell{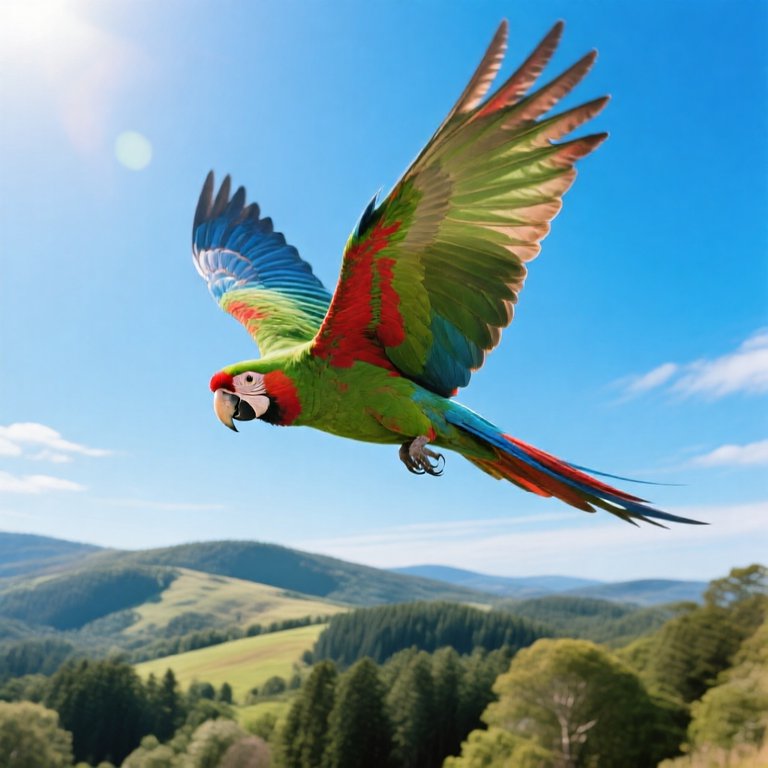}
      & \framecell{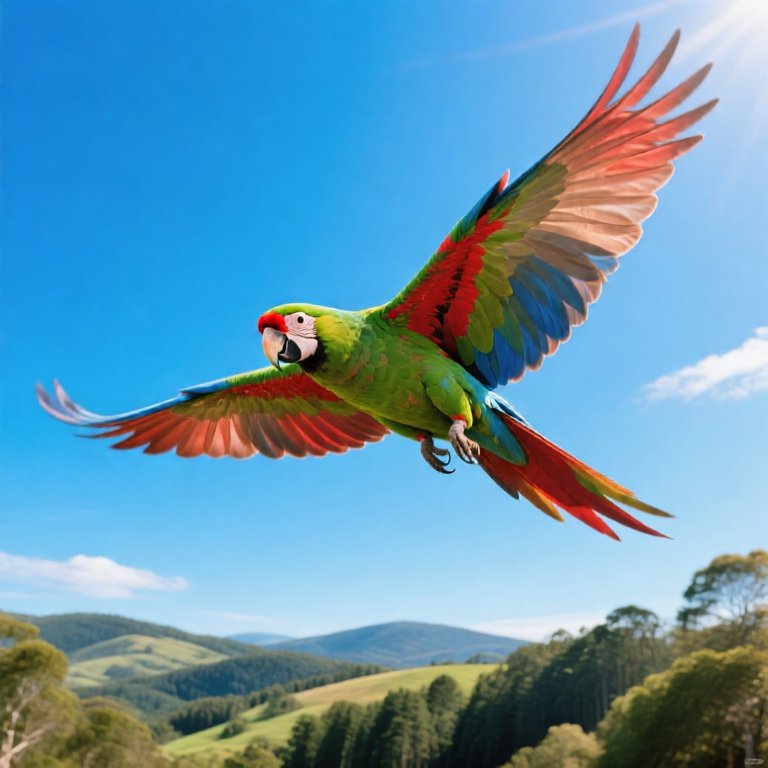}
      & \framecell{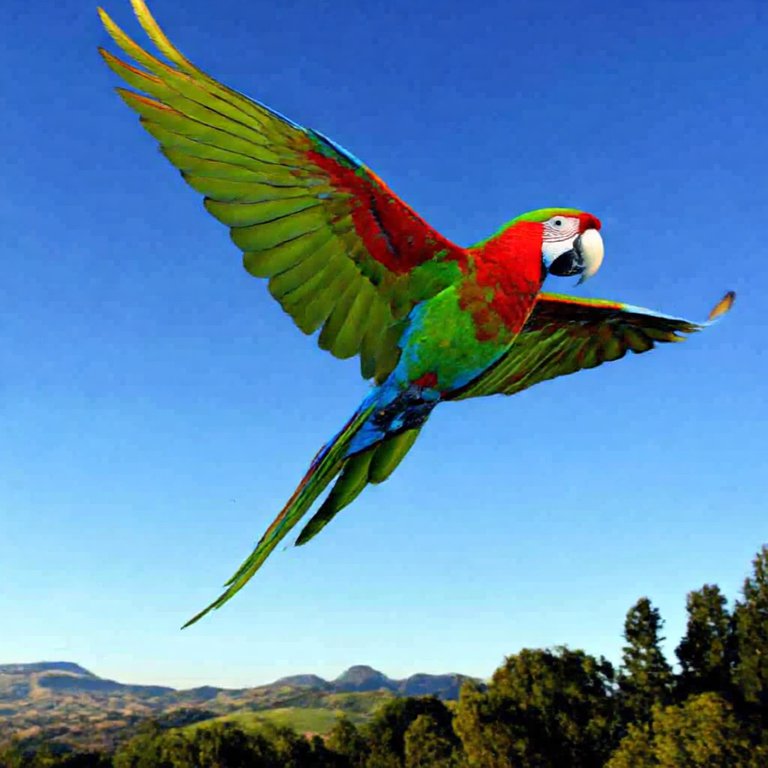}
      & \framecell{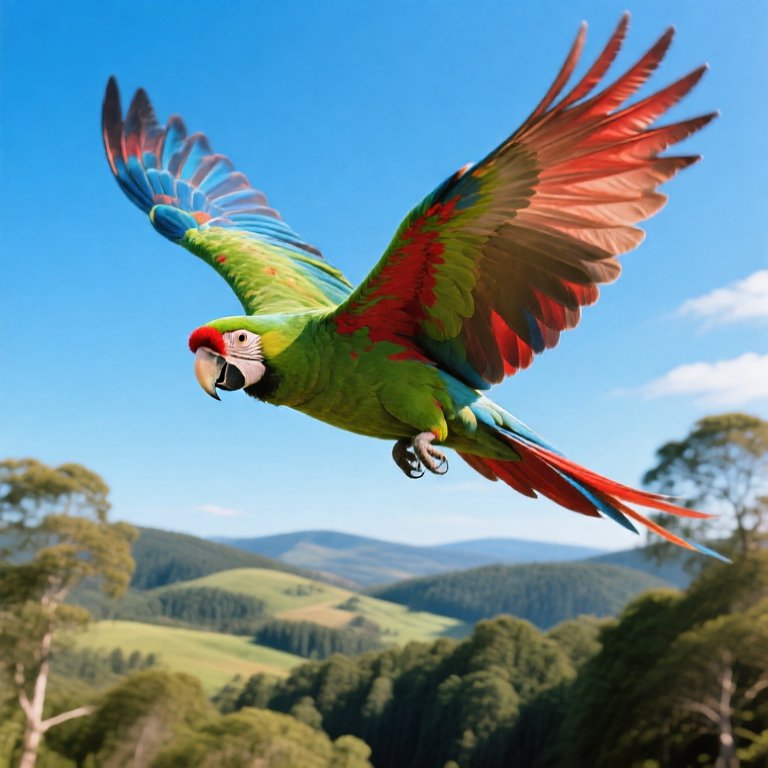}
      & \phantom{Teacher} \\ \addlinespace[2pt]
    DMD2
      & \framecell{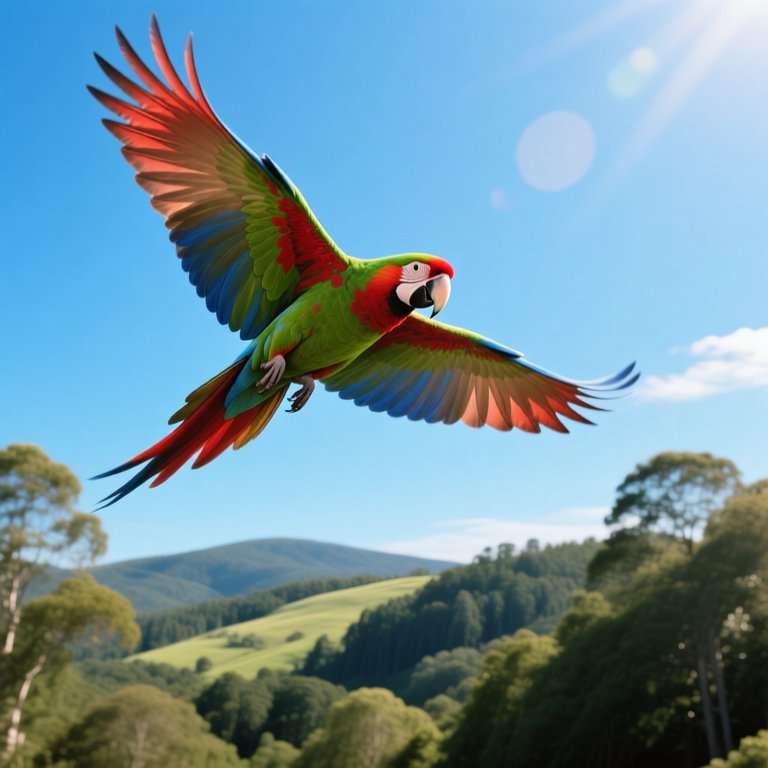}
      & \framecell{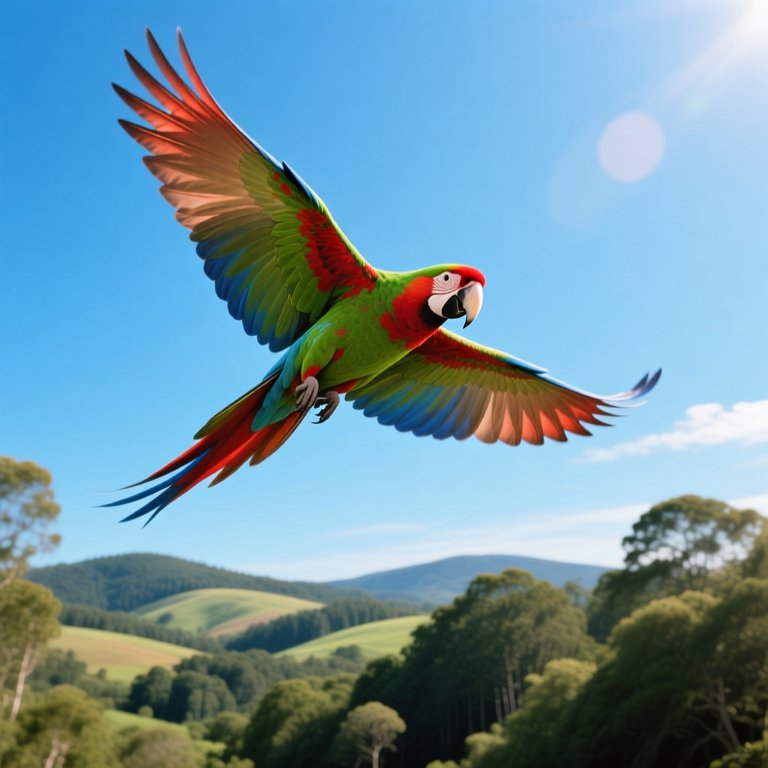}
      & \framecell{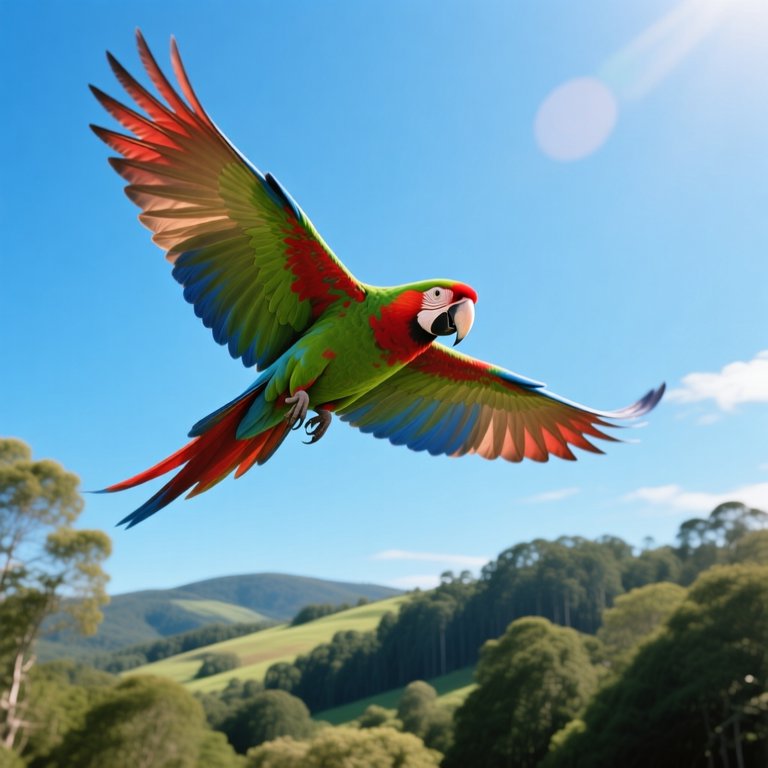}
      & \framecell{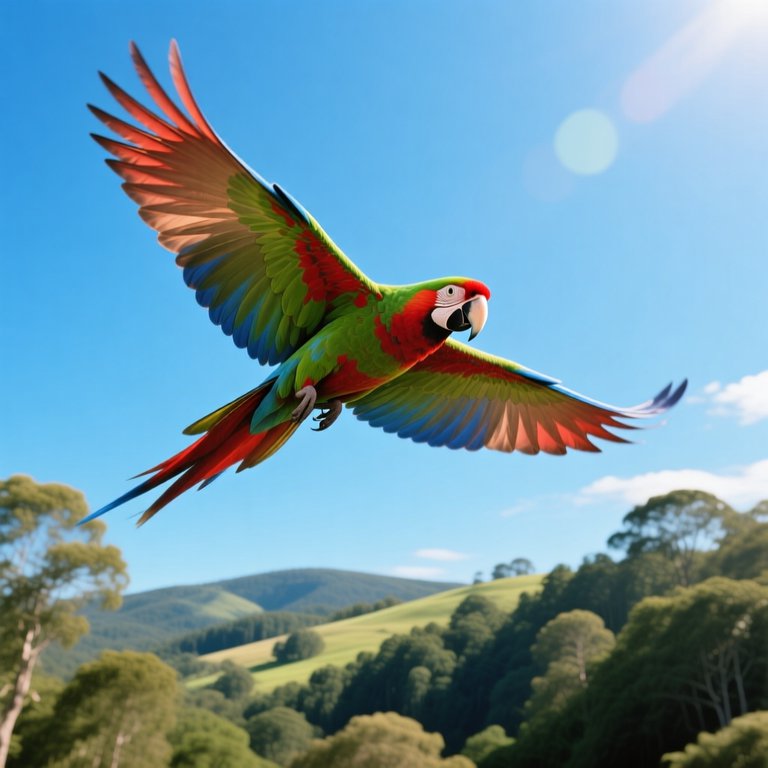}
      & \phantom{Teacher} \\ \addlinespace[2pt]
      & \multicolumn{4}{@{} p{10.8cm} @{}}{\tiny\itshape ``\input{figures/qwenimage/diversity_tables/7/prompt.txt}''} & \\
      \vspace{-20pt}
    \end{tabular}%
    }
  \end{figure}

%% file: figures/qwenimage/diversity_tables/7/prompt.txt
A majestic parrot with vibrant green, red, and blue feathers glides effortlessly across the bright blue sky. Its impressive wingspan is fully outstretched, catching the warm sunlight as it soars high above the tree line. Below, the landscape features rolling hills and patches of dense forest.

%% file: figures/qwenimage/diversity_tables/60/table.tex
\begin{figure}[t]
  \centering
  \resizebox{0.95\textwidth}{!}{%

  \begin{tabular}{@{} r @{\hspace{6pt}} c@{\hspace{2pt}}c@{\hspace{2pt}} c@{\hspace{2pt}}c @{\hspace{6pt}} c @{}}
    & {\footnotesize Noise 1} & {\footnotesize Noise 2} &  {\footnotesize Noise 3} & {\footnotesize Noise 4} & \\
    Teacher
      & \framecell{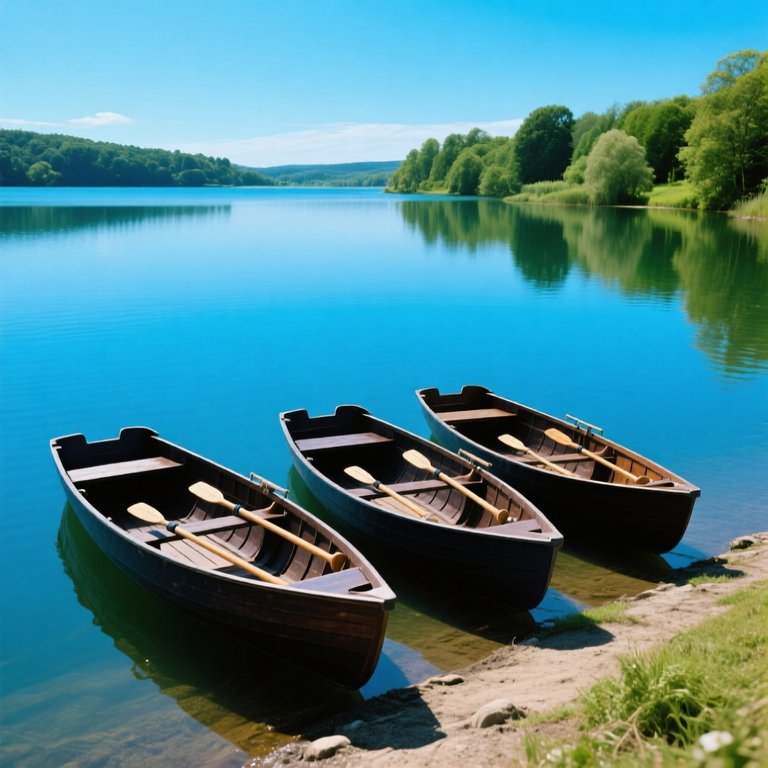}
      & \framecell{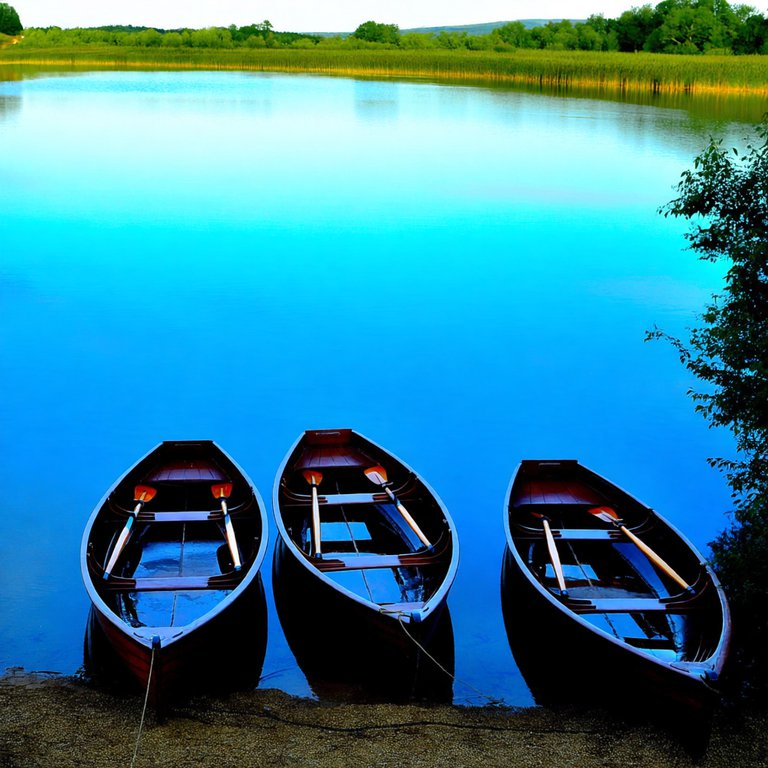}
      & \framecell{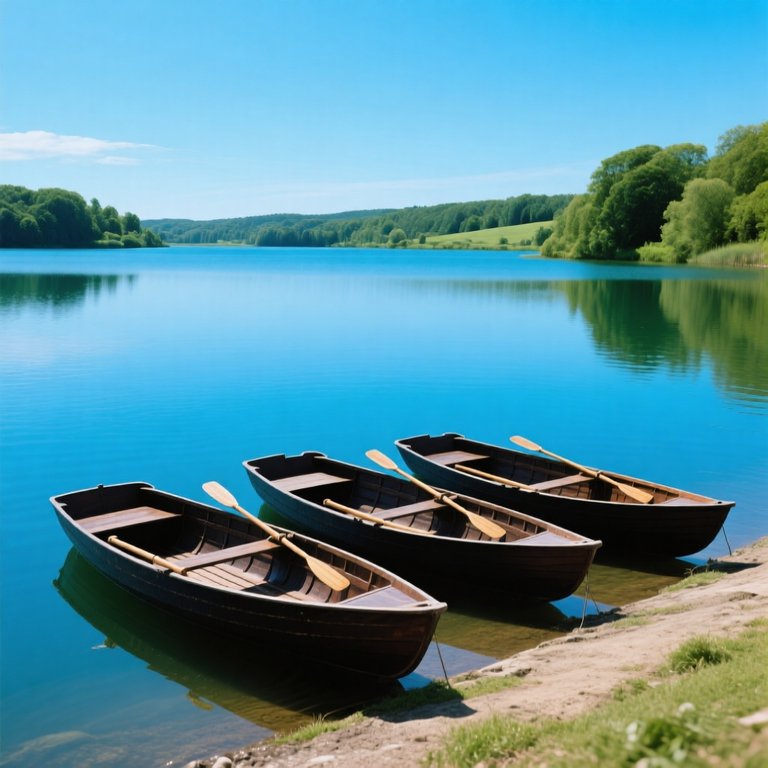}
      & \framecell{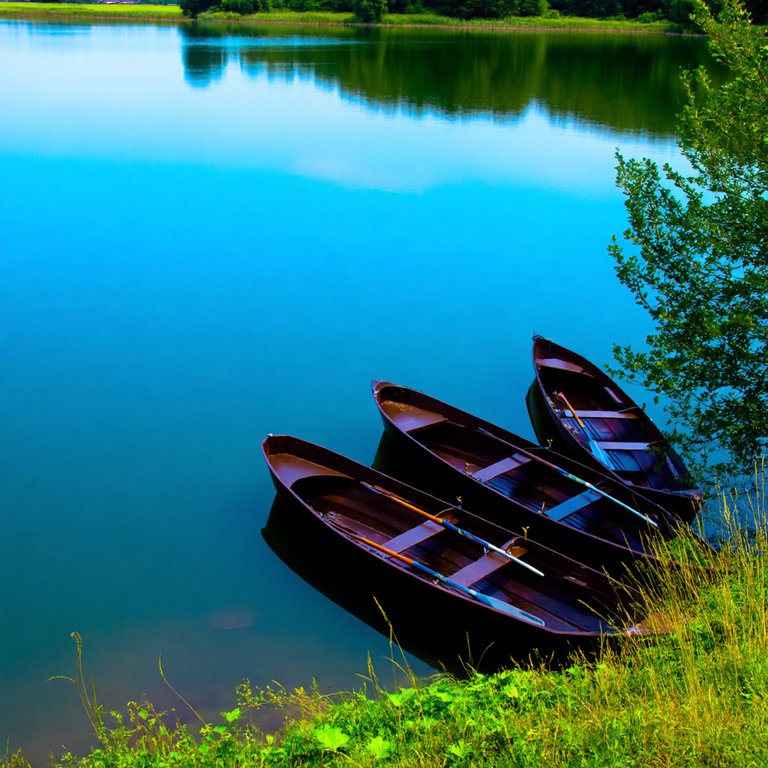}
      & \phantom{Teacher} \\ \addlinespace[2pt]
    \textbf{PDD}
      & \framecell{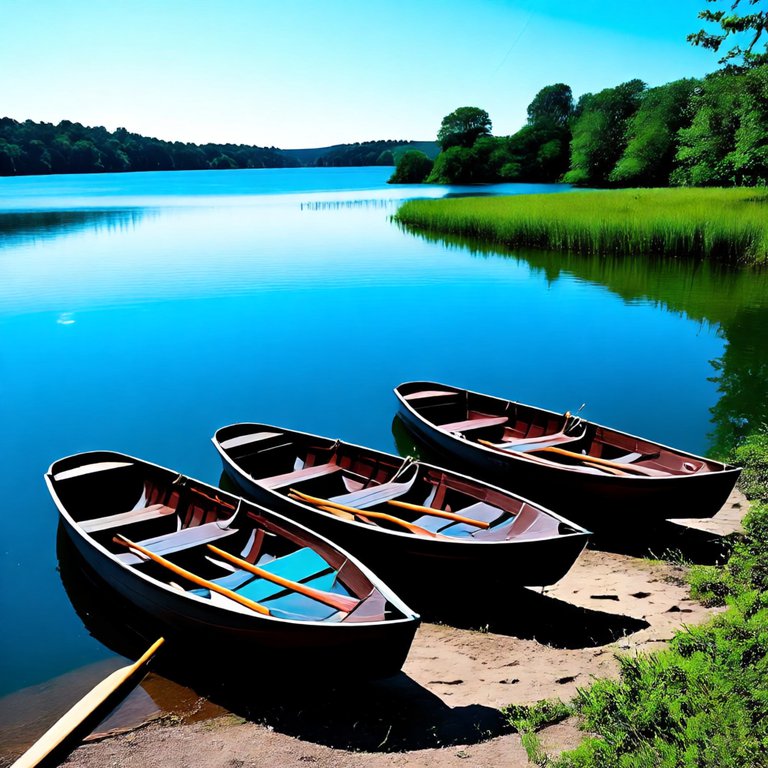}
      & \framecell{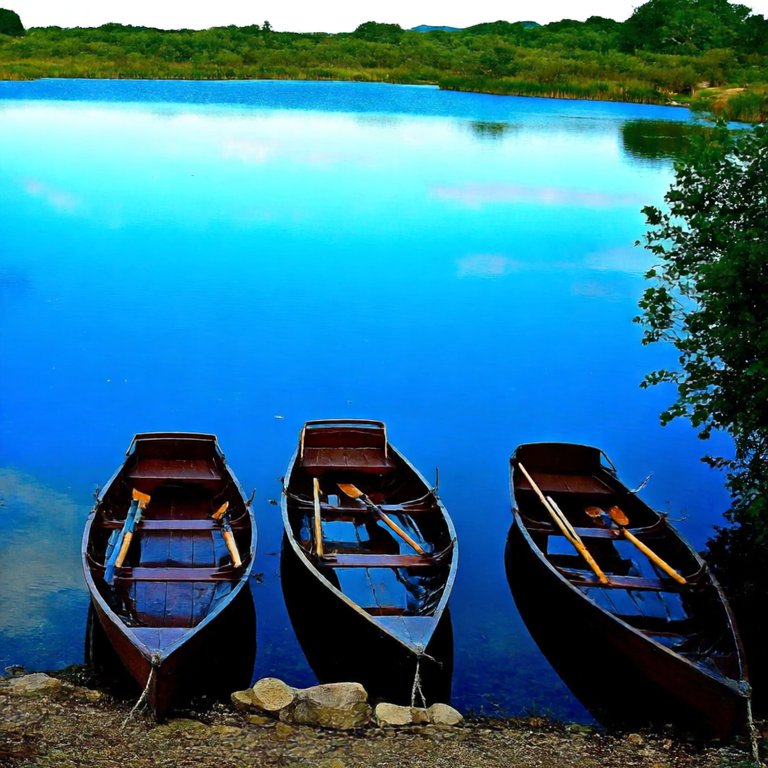}
      & \framecell{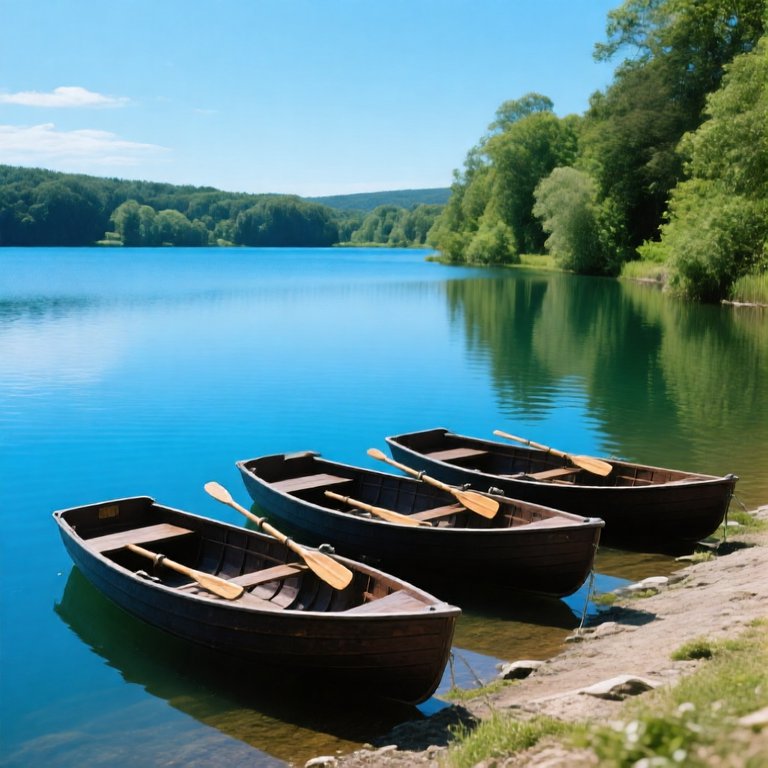}
      & \framecell{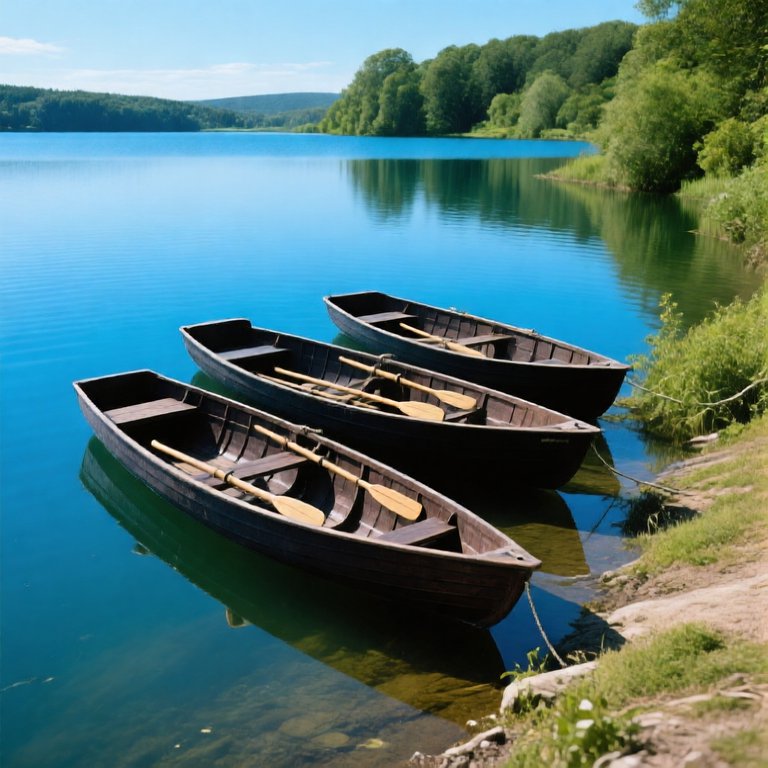}
      & \phantom{Teacher} \\ \addlinespace[2pt]
    DMD2
      & \framecell{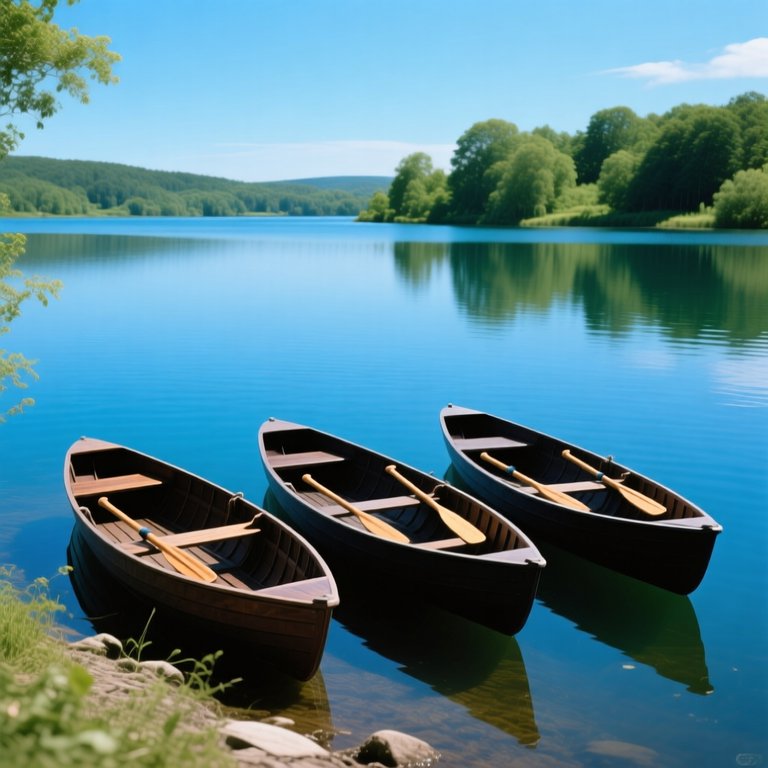}
      & \framecell{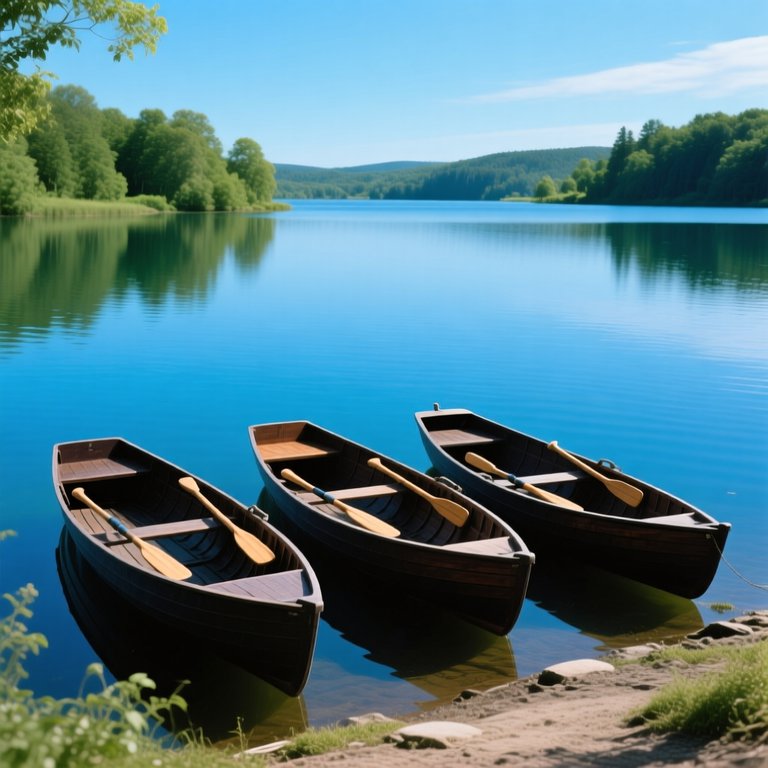}
      & \framecell{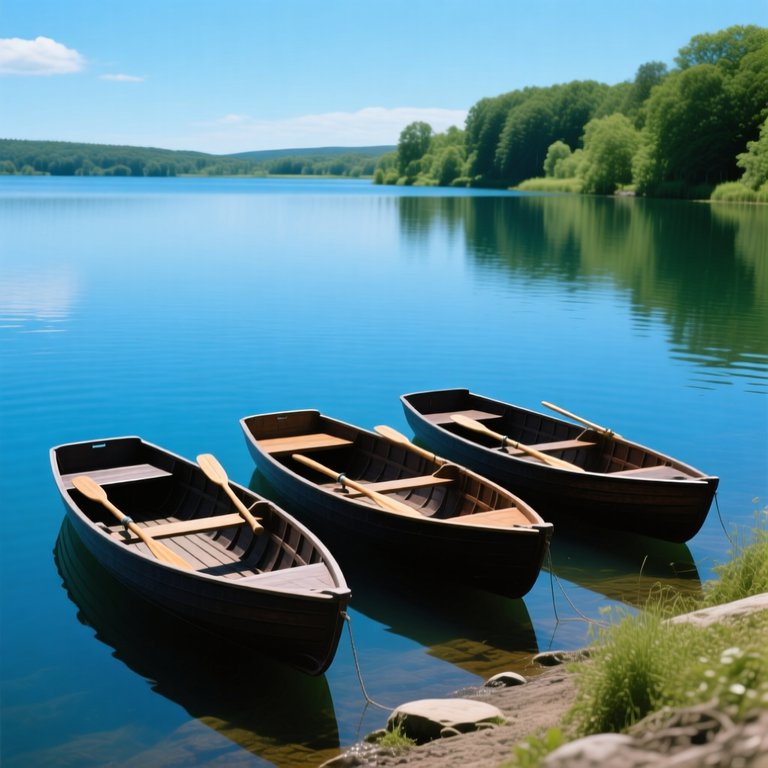}
      & \framecell{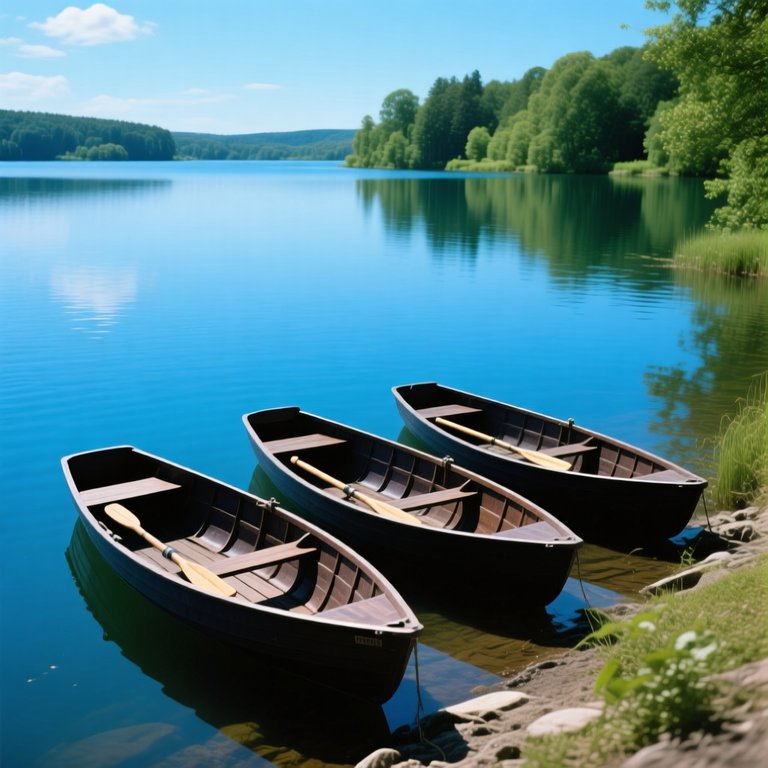}
      & \phantom{Teacher} \\ \addlinespace[2pt]
      & \multicolumn{4}{@{} p{10.8cm} @{}}{\tiny\itshape ``\input{figures/qwenimage/diversity_tables/60/prompt.txt}''} & \\
      \vspace{-15pt}
    \end{tabular}%
    }
    \caption{Teacher uses Euler method with 2$\times$50 NFE, PDD and DMD2 (Lightning-v2) use 4 NFE.}
    \label{fig:qwenimage_diversity_first}
  \end{figure}

%% file: figures/qwenimage/diversity_tables/60/prompt.txt
Three sleek, dark wooden boats are resting along the banks of a tranquil, azure blue lake. The smooth surface of the water reflects the clear sky above and the lush greenery surrounding the lake's edge. Positioned close to each other, the boats' oars are tucked neatly inside, hinting at a recent or upcoming journey across the calm waters.

%% file: figures/qwenimage/diversity_tables/180/table.tex
\begin{figure}[t]
  \centering
  \resizebox{0.95\textwidth}{!}{%

  \begin{tabular}{@{} r @{\hspace{6pt}} c@{\hspace{2pt}}c@{\hspace{2pt}} c@{\hspace{2pt}}c @{\hspace{6pt}} c @{}}
    & {\footnotesize Noise 1} & {\footnotesize Noise 2} &  {\footnotesize Noise 3} & {\footnotesize Noise 4} & \\
    Teacher
      & \framecell{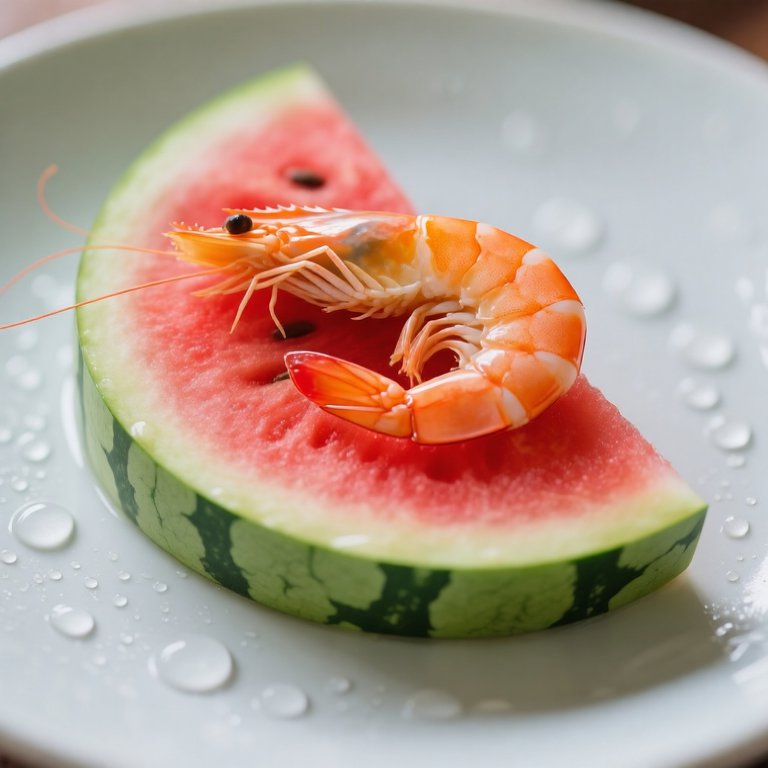}
      & \framecell{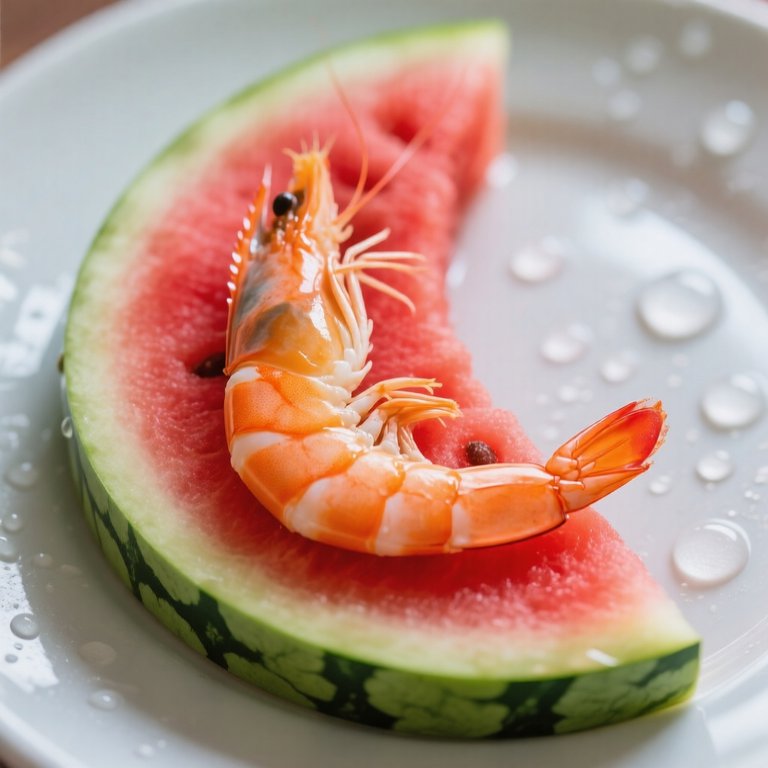}
      & \framecell{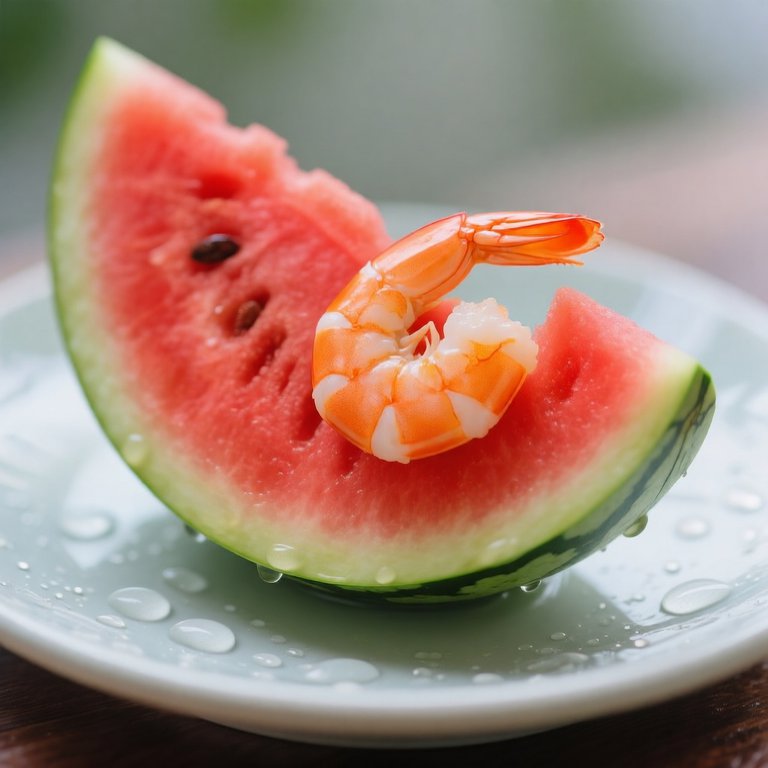}
      & \framecell{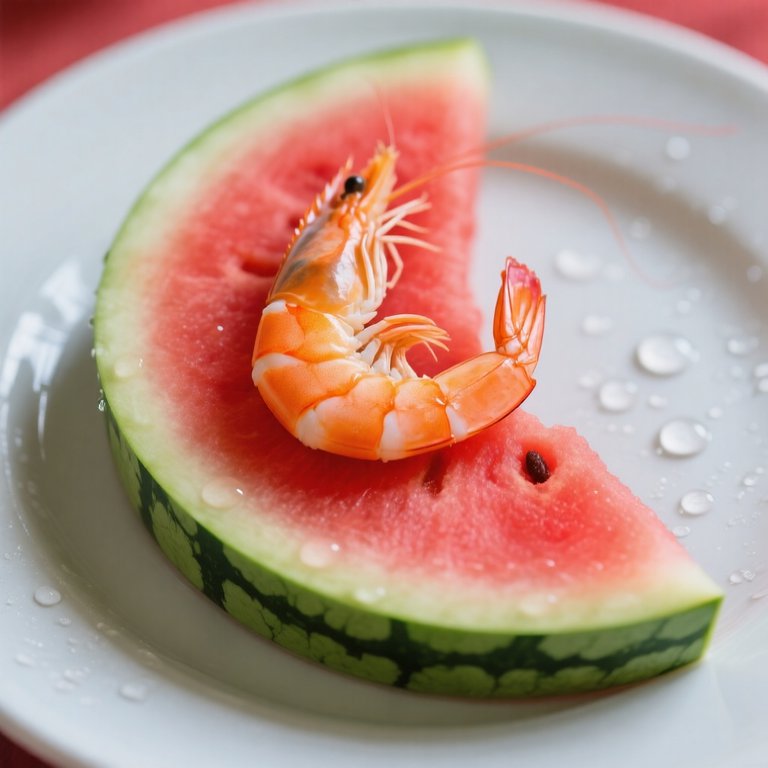}
      & \phantom{Teacher} \\ \addlinespace[2pt]
    \textbf{PDD}
      & \framecell{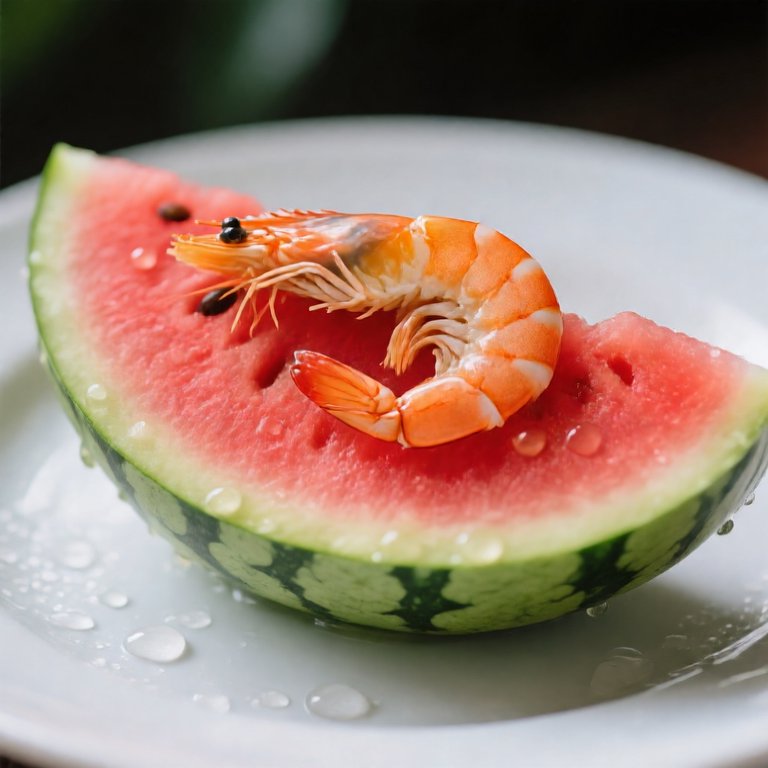}
      & \framecell{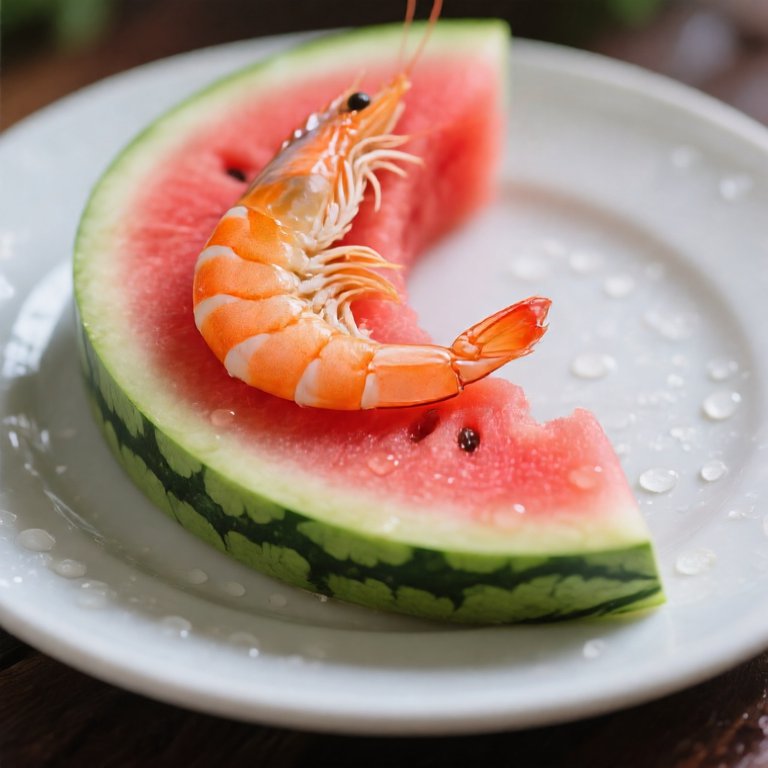}
      & \framecell{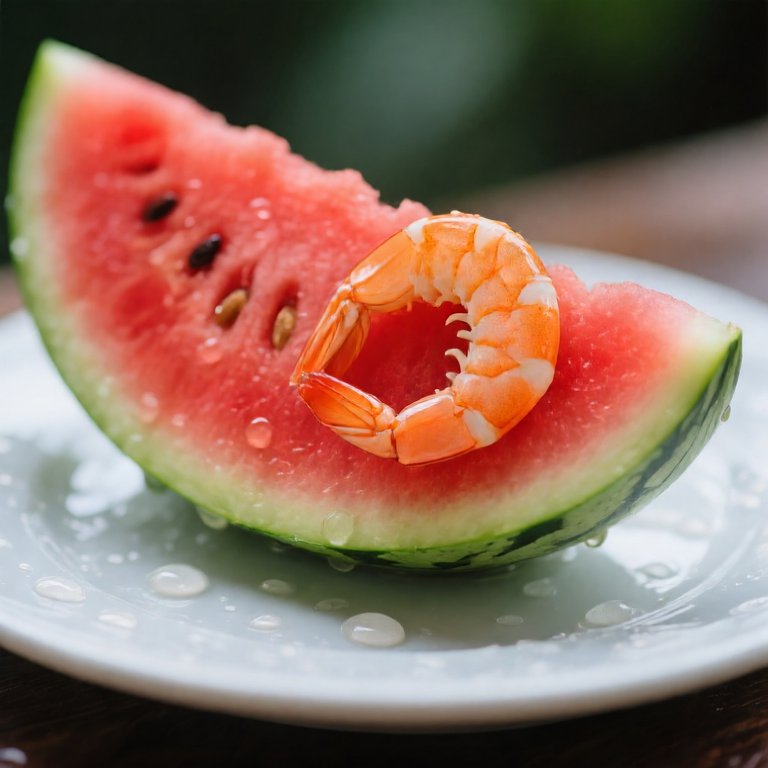}
      & \framecell{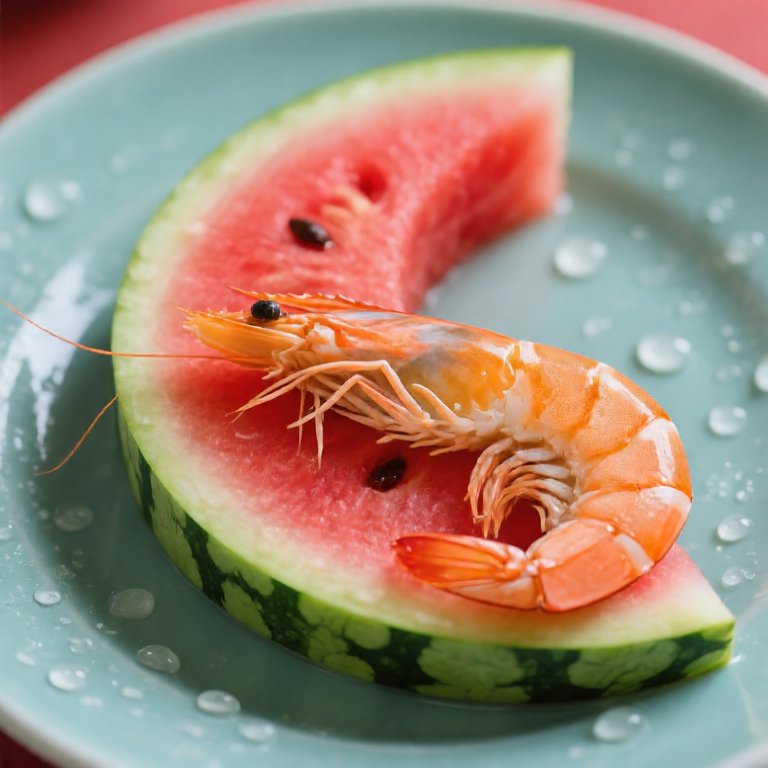}
      & \phantom{Teacher} \\ \addlinespace[2pt]
    DMD2
      & \framecell{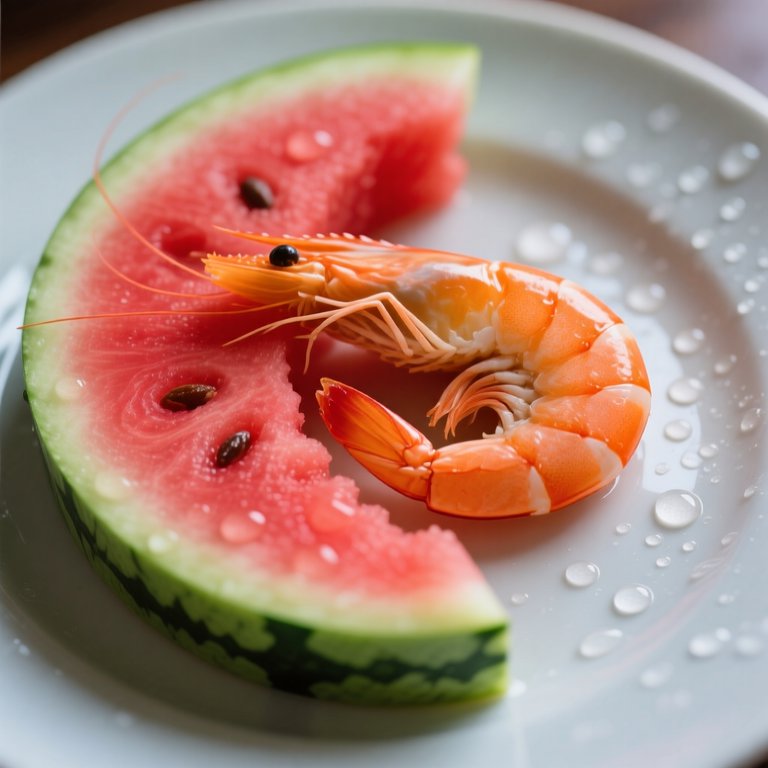}
      & \framecell{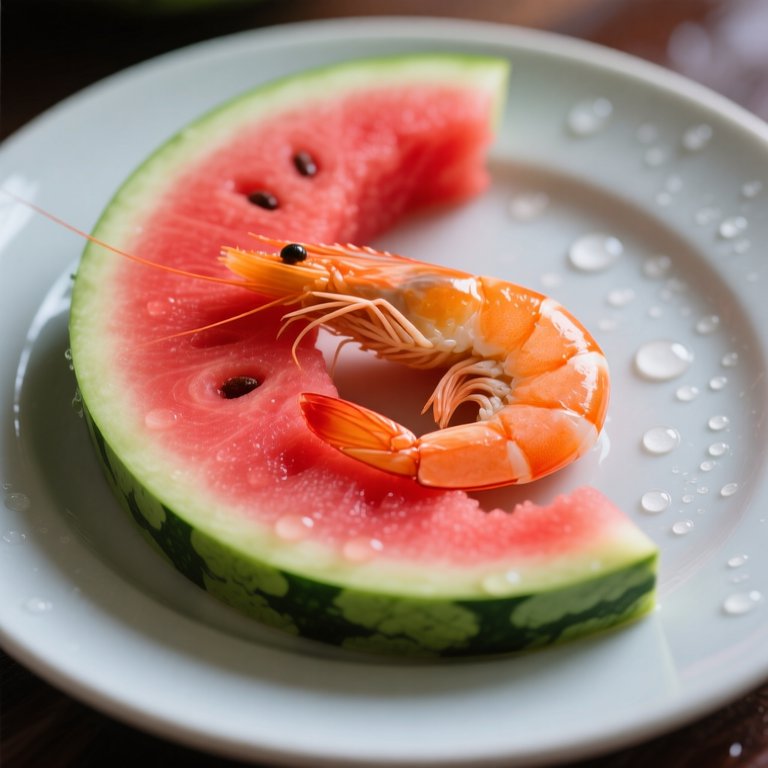}
      & \framecell{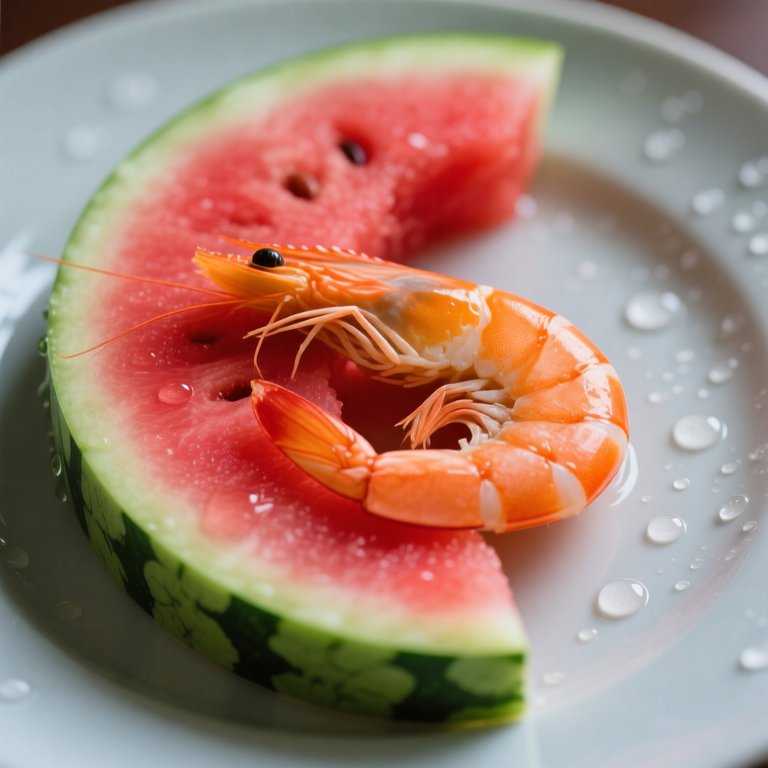}
      & \framecell{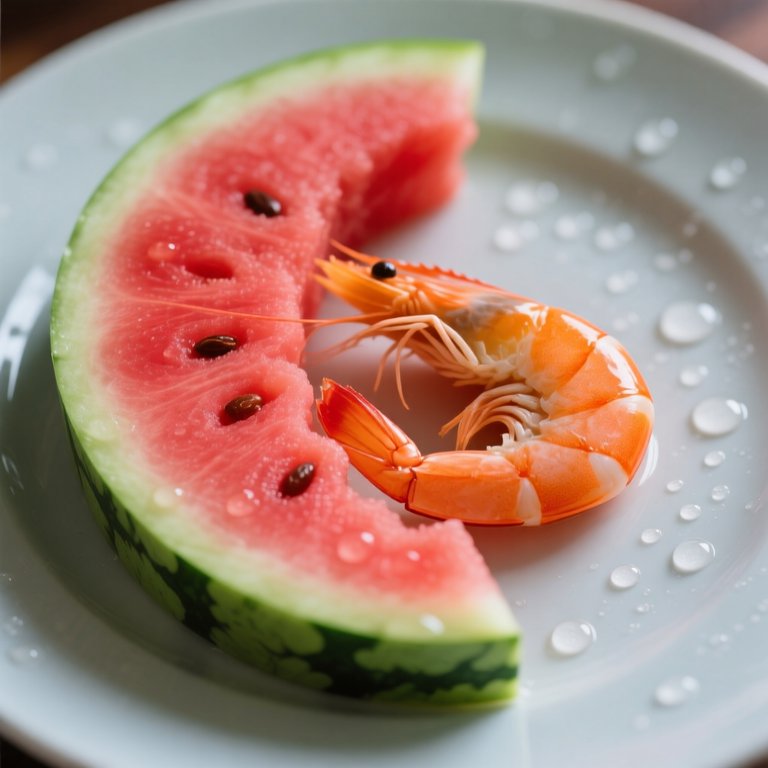}
      & \phantom{Teacher} \\ \addlinespace[2pt]
      & \multicolumn{4}{@{} p{10.8cm} @{}}{\tiny\itshape ``\input{figures/qwenimage/diversity_tables/180/prompt.txt}''} & \\
      \vspace{-20pt}
    \end{tabular}%
    }
  \end{figure}

%% file: figures/qwenimage/diversity_tables/180/prompt.txt
A slice of vibrant red watermelon, with its green rind visible, serves as a plate for a single, orange-hued cooked shrimp. The curves of the shrimp follow the crescent shape of the melon, contrasting in both color and texture. Beside the watermelon, there are droplets of water, suggesting the fruit's juicy freshness.

%% file: figures/qwenimage/diversity_tables/253/table.tex
\begin{figure}[t]
  \centering
  \resizebox{0.95\textwidth}{!}{%

  \begin{tabular}{@{} r @{\hspace{6pt}} c@{\hspace{2pt}}c@{\hspace{2pt}} c@{\hspace{2pt}}c @{\hspace{6pt}} c @{}}
    & {\footnotesize Noise 1} & {\footnotesize Noise 2} &  {\footnotesize Noise 3} & {\footnotesize Noise 4} & \\
    Teacher
      & \framecell{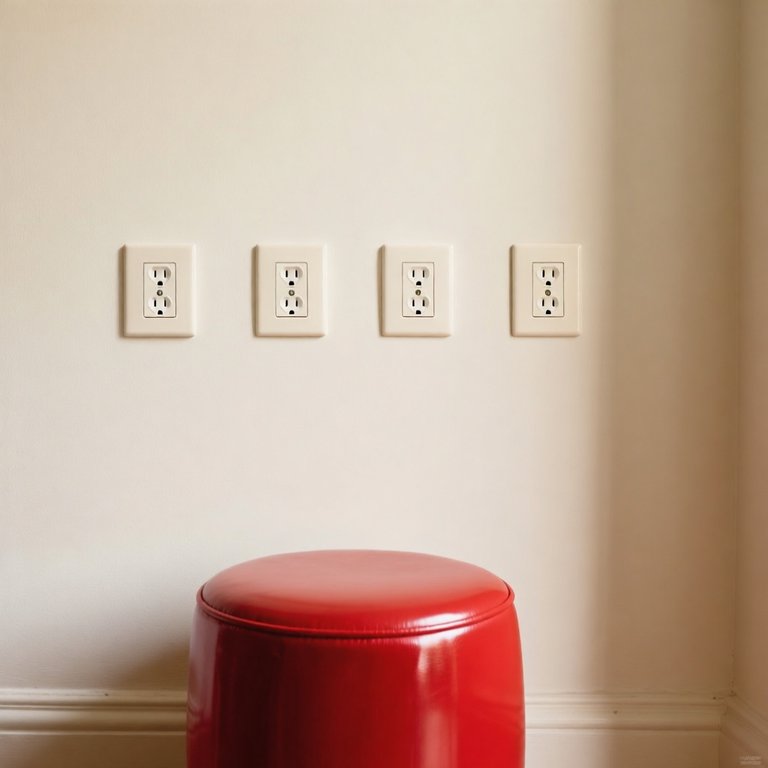}
      & \framecell{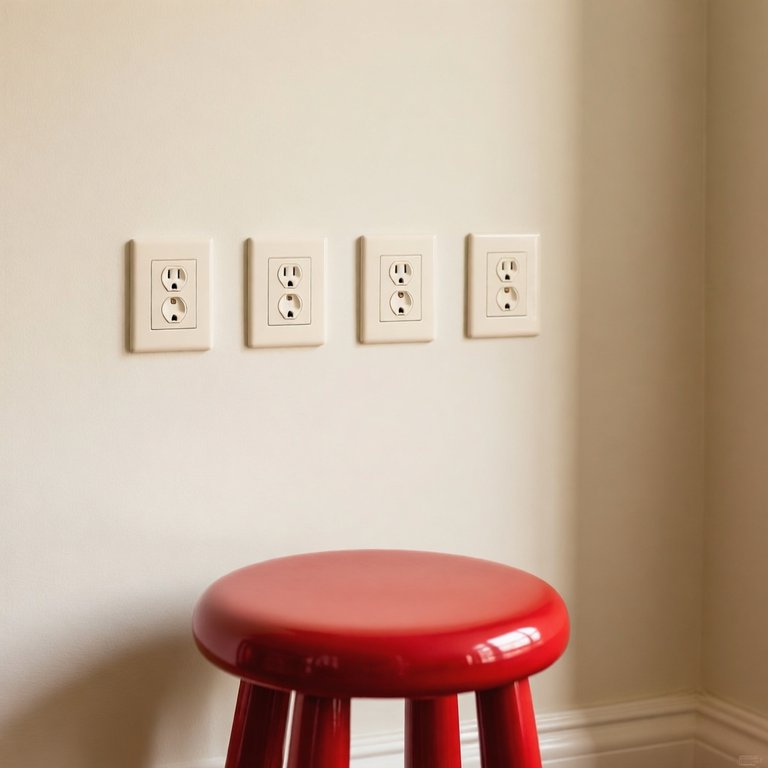}
      & \framecell{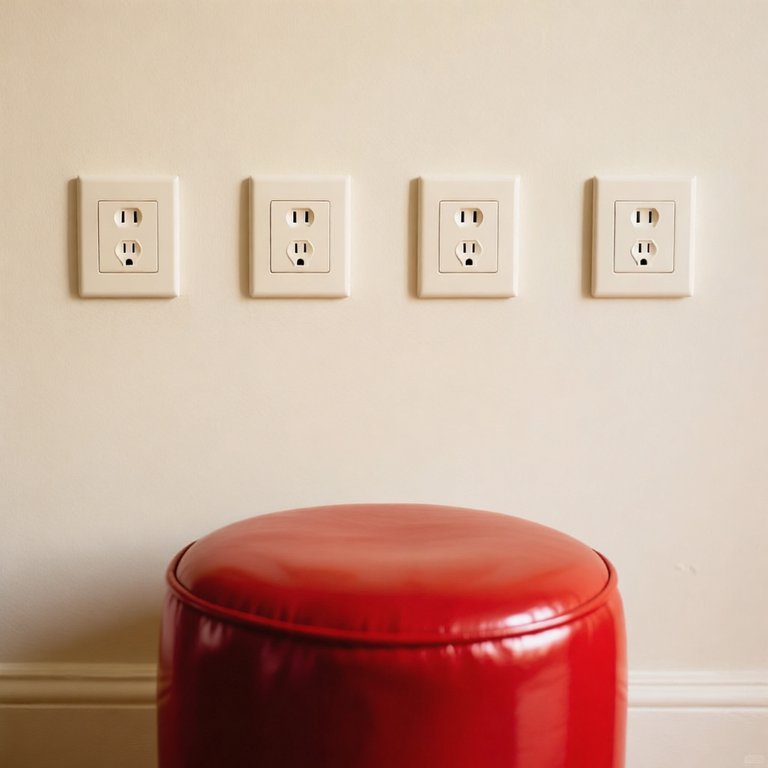}
      & \framecell{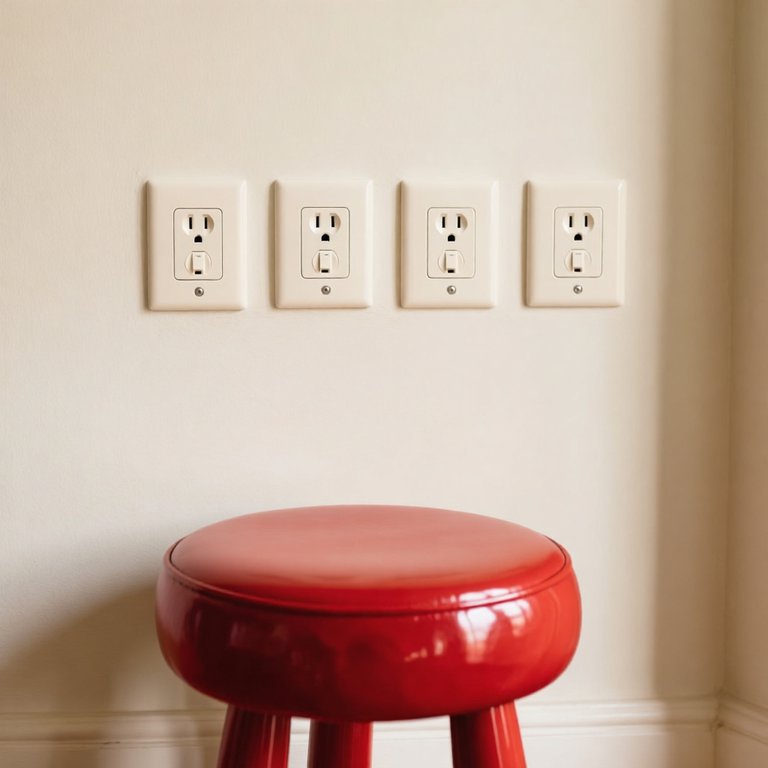}
      & \phantom{Teacher} \\ \addlinespace[2pt]
    \textbf{PDD}
      & \framecell{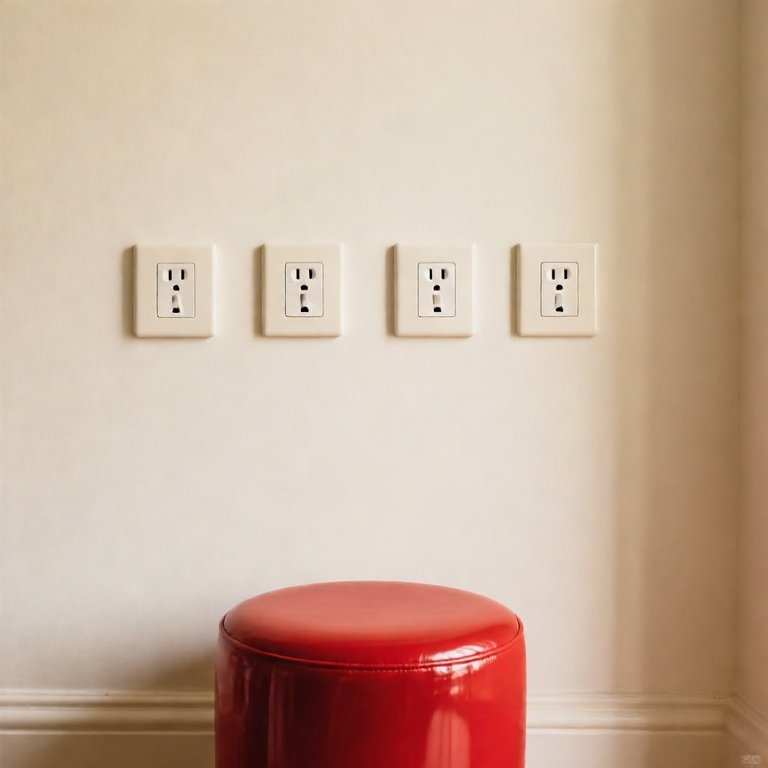}
      & \framecell{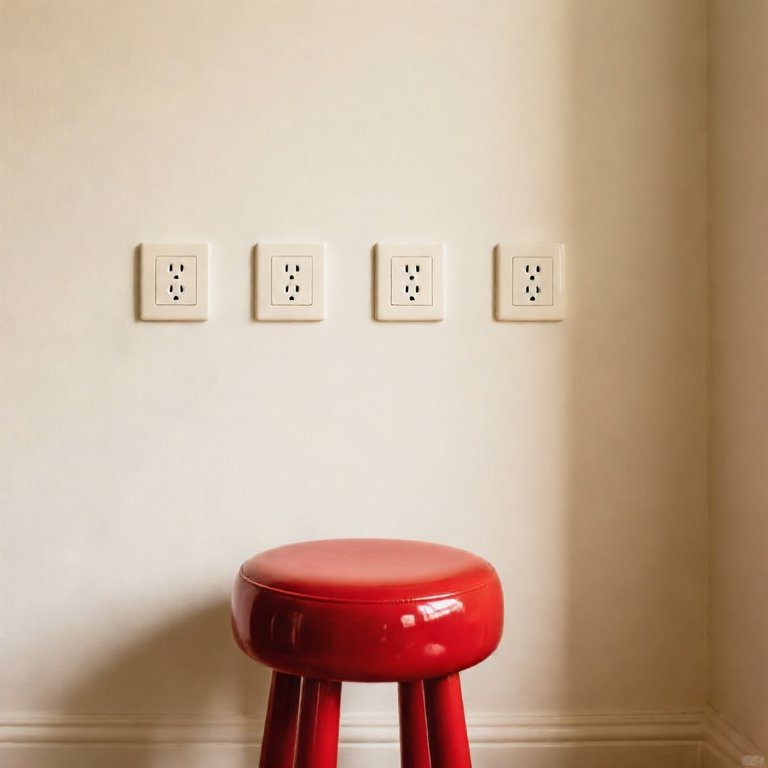}
      & \framecell{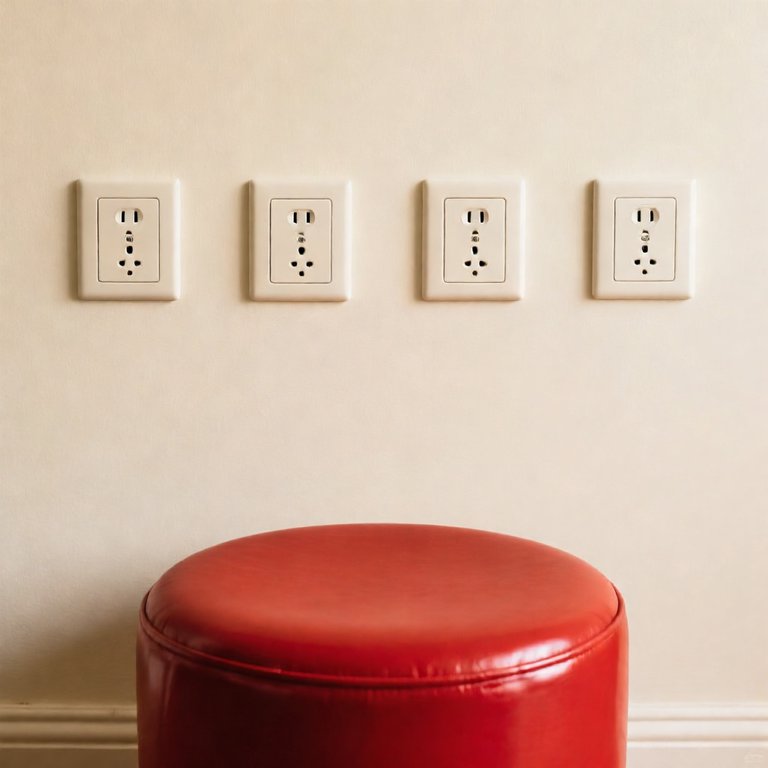}
      & \framecell{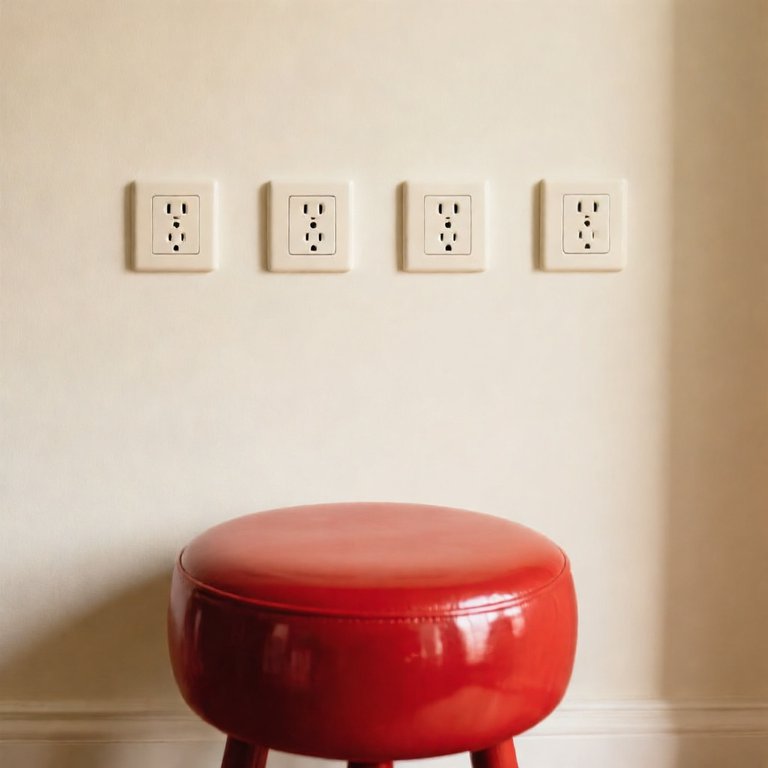}
      & \phantom{Teacher} \\ \addlinespace[2pt]
    DMD2
      & \framecell{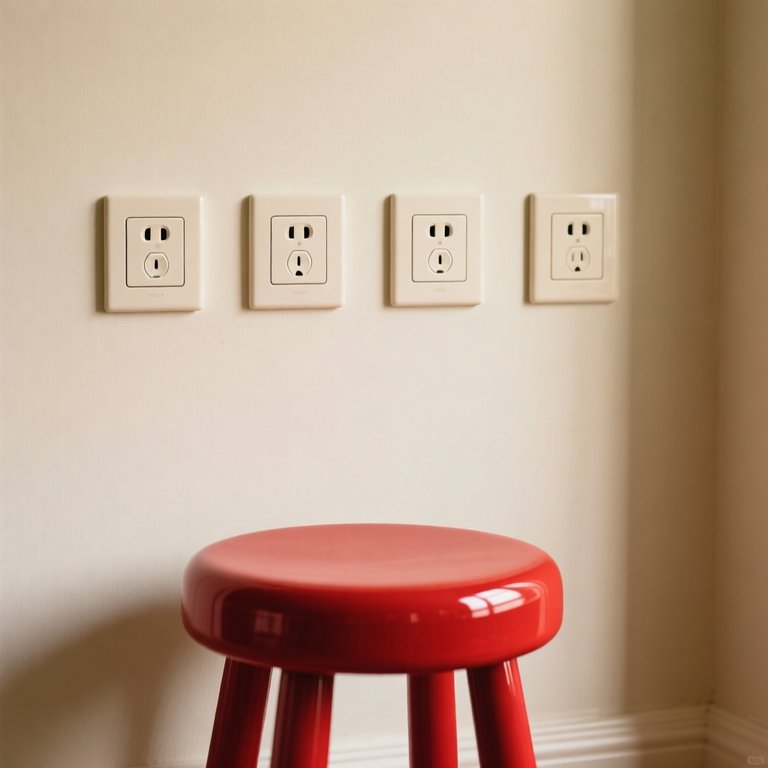}
      & \framecell{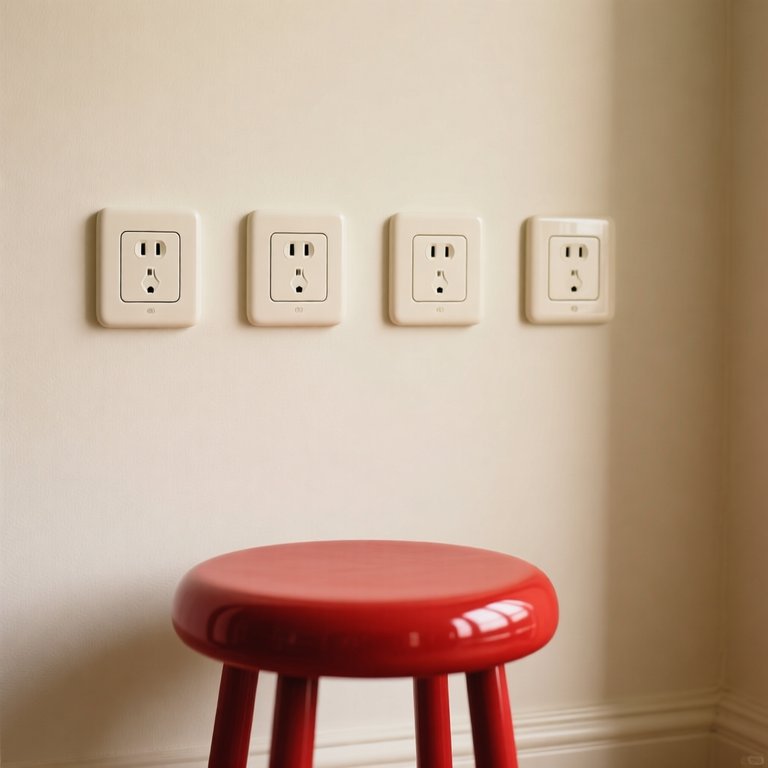}
      & \framecell{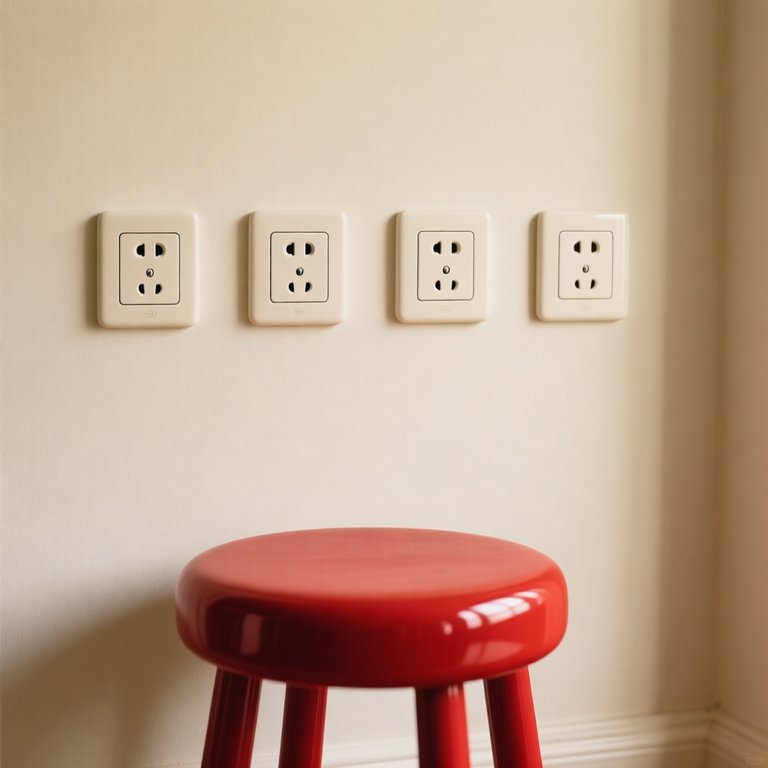}
      & \framecell{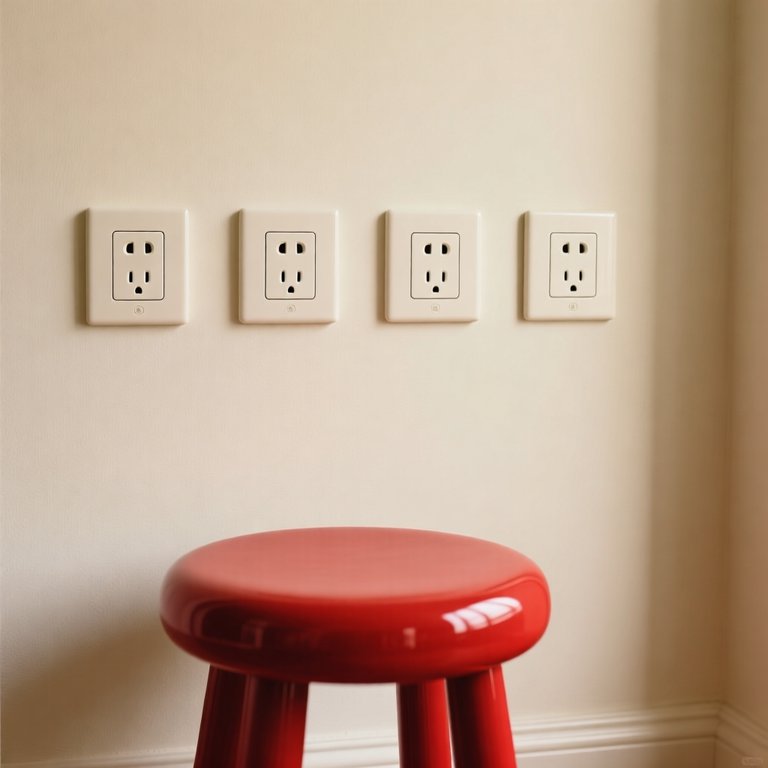}
      & \phantom{Teacher} \\ \addlinespace[2pt]
      & \multicolumn{4}{@{} p{10.8cm} @{}}{\tiny\itshape ``\input{figures/qwenimage/diversity_tables/253/prompt.txt}''} & \\
      \vspace{-15pt}
    \end{tabular}%
    }
    \caption{Teacher uses Euler method with 2$\times$50 NFE, PDD and DMD2 (Lightning-v2) use 4 NFE.}
  \end{figure}

%% file: figures/qwenimage/diversity_tables/253/prompt.txt
An interior wall is fitted with four square-shaped power outlets, uniformly aligned in a horizontal row directly above a ruby red, round stool. The wall's paint is a soft cream color, creating a contrast with the vibrant red of the stool. The stool's smooth texture and glossy finish reflect the ambient light, emphasizing its bold hue in the otherwise neutral-toned room.

%% file: figures/qwenimage/diversity_tables/659/table.tex
\begin{figure}[t]
  \centering
  \resizebox{0.95\textwidth}{!}{%

  \begin{tabular}{@{} r @{\hspace{6pt}} c@{\hspace{2pt}}c@{\hspace{2pt}} c@{\hspace{2pt}}c @{\hspace{6pt}} c @{}}
    & {\footnotesize Noise 1} & {\footnotesize Noise 2} &  {\footnotesize Noise 3} & {\footnotesize Noise 4} & \\
    Teacher
      & \framecell{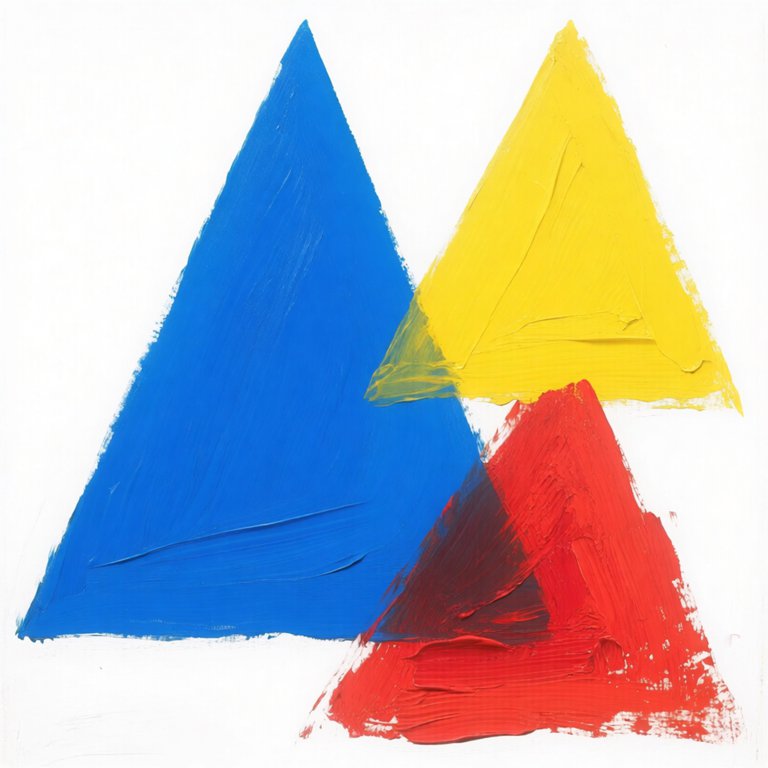}
      & \framecell{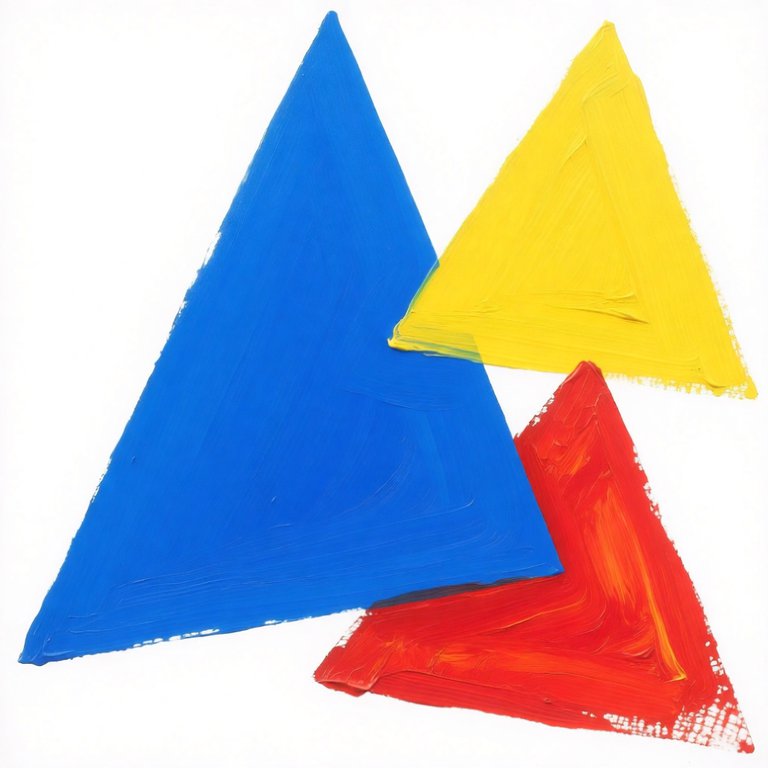}
      & \framecell{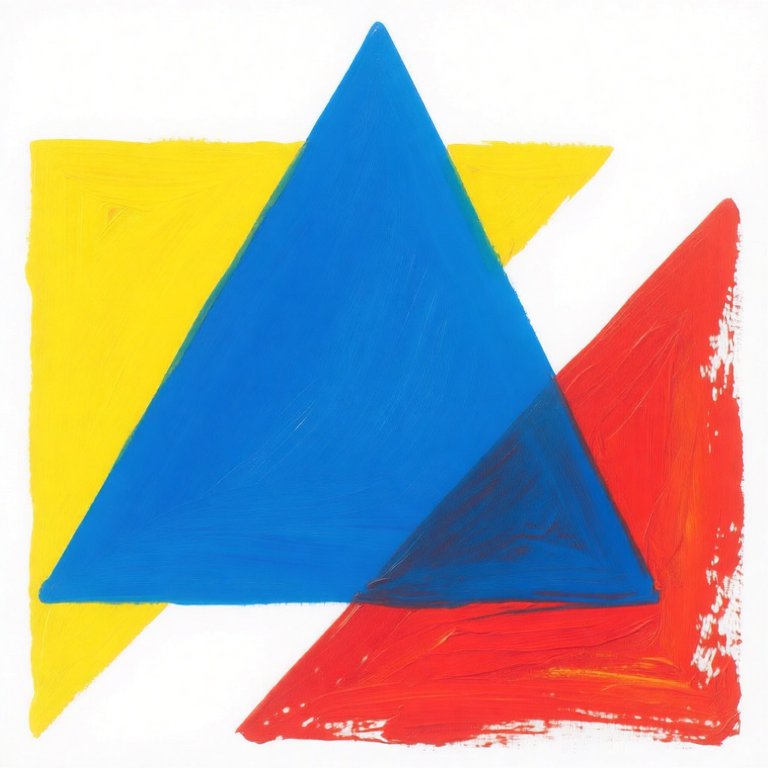}
      & \framecell{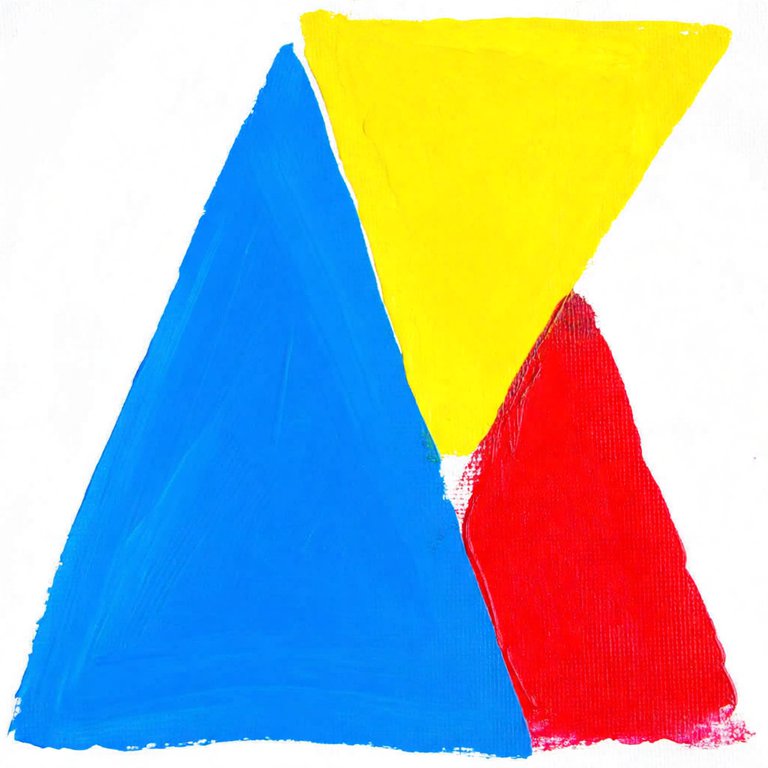}
      & \phantom{Teacher} \\ \addlinespace[2pt]
    \textbf{PDD}
      & \framecell{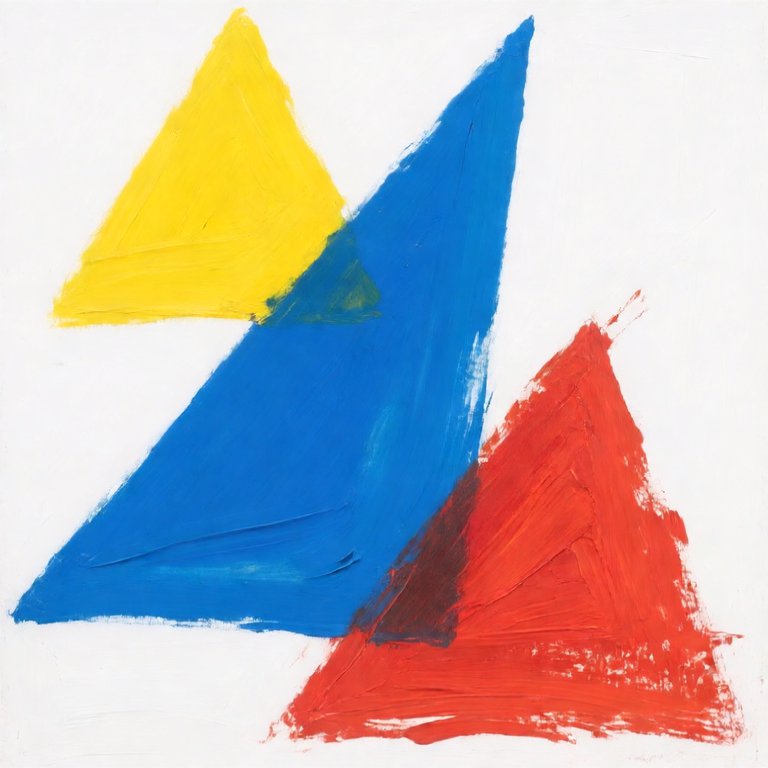}
      & \framecell{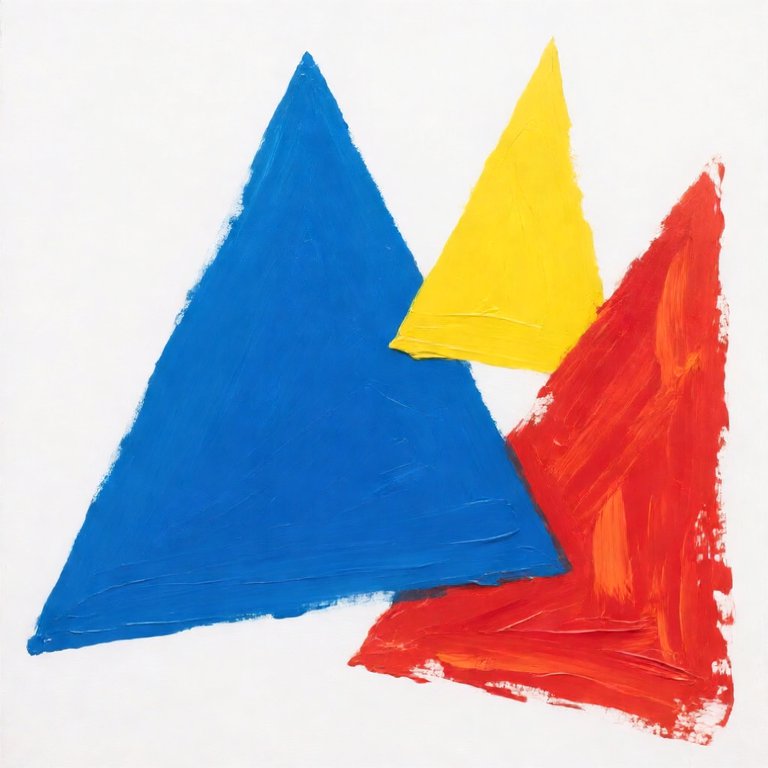}
      & \framecell{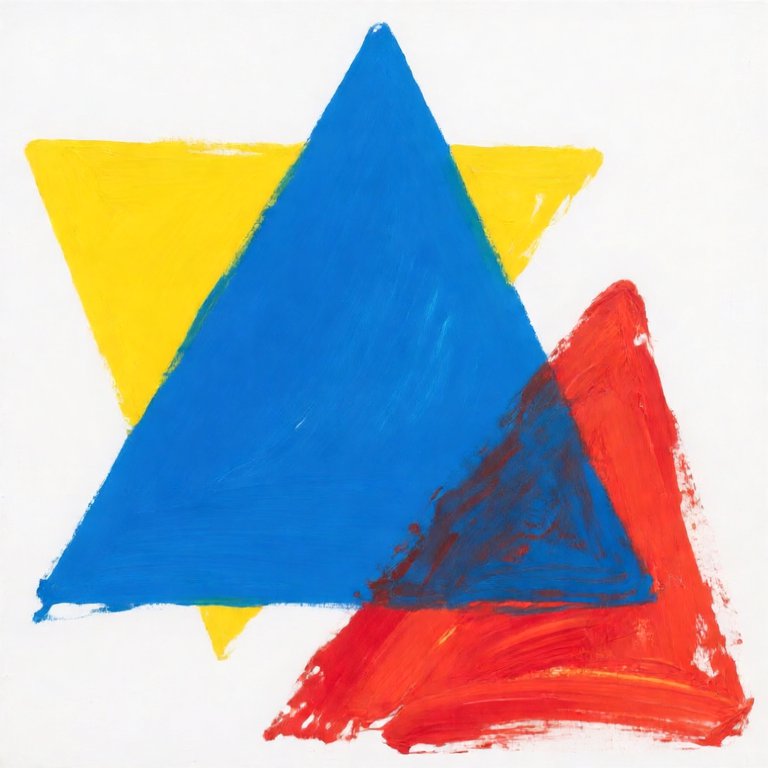}
      & \framecell{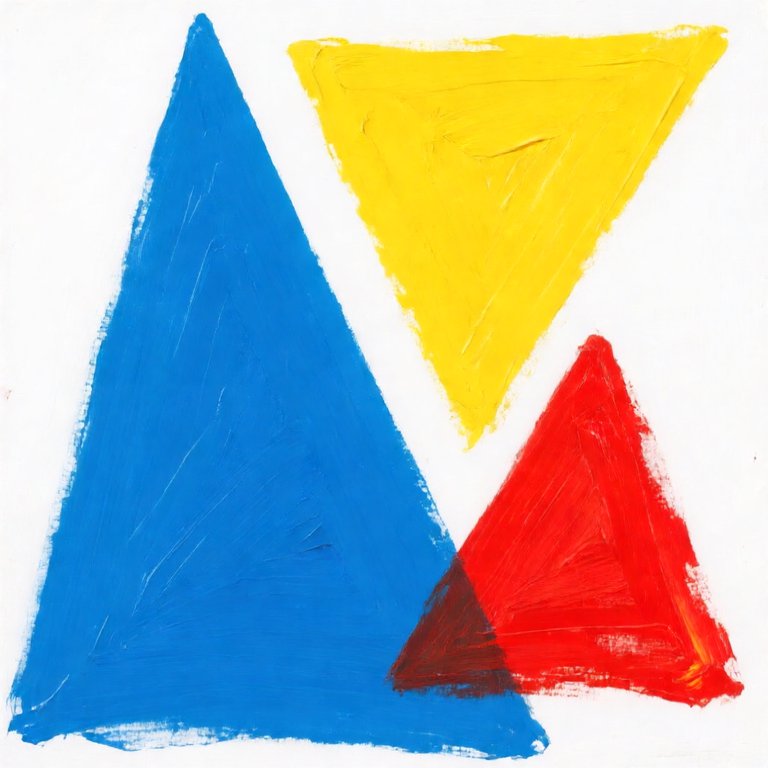}
      & \phantom{Teacher} \\ \addlinespace[2pt]
    DMD2
      & \framecell{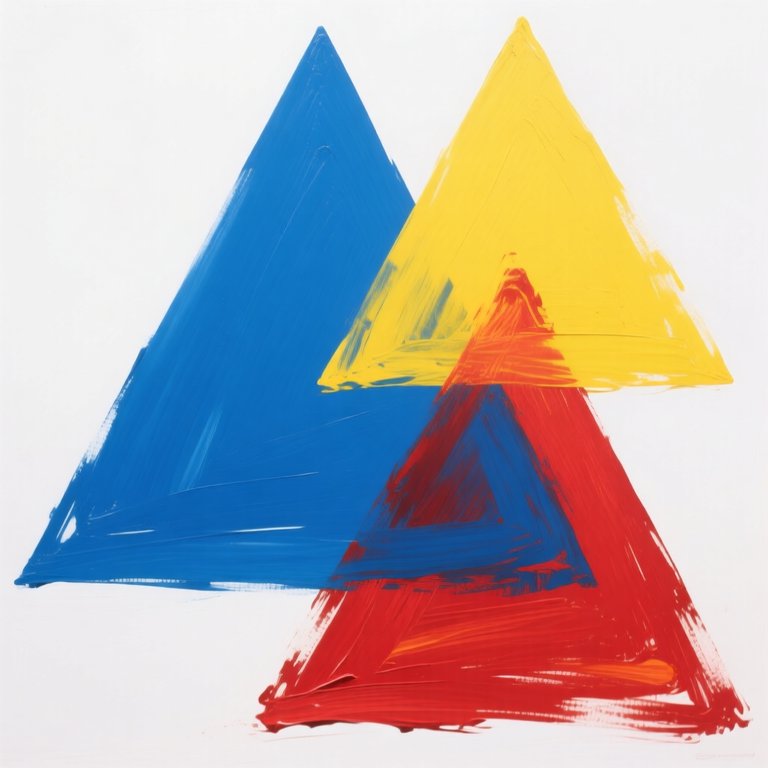}
      & \framecell{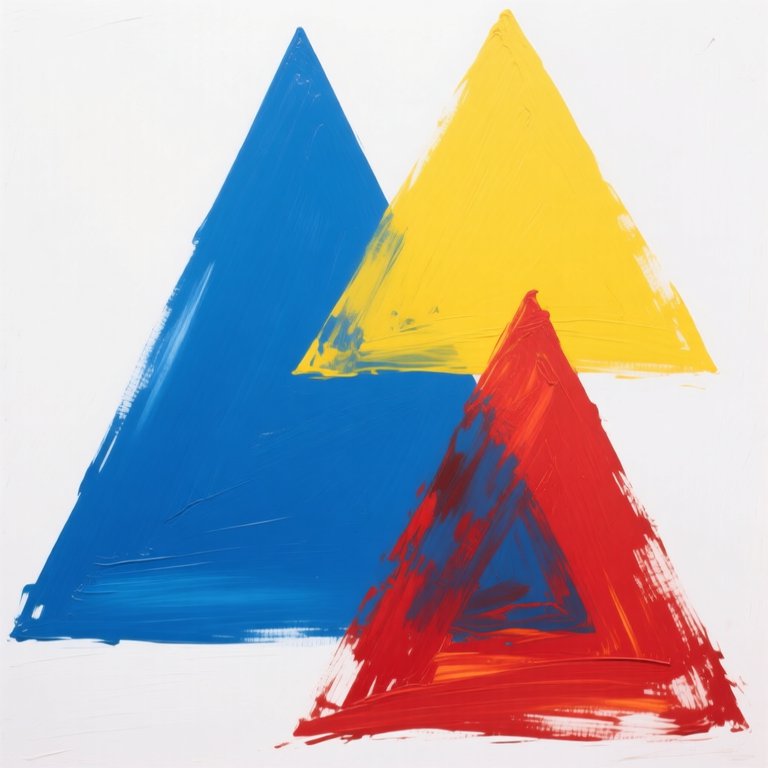}
      & \framecell{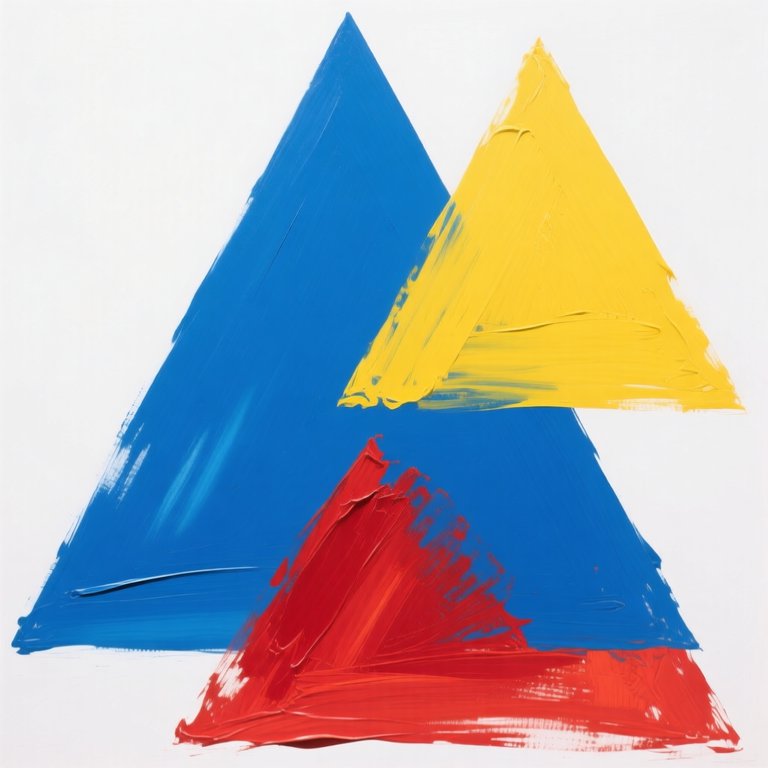}
      & \framecell{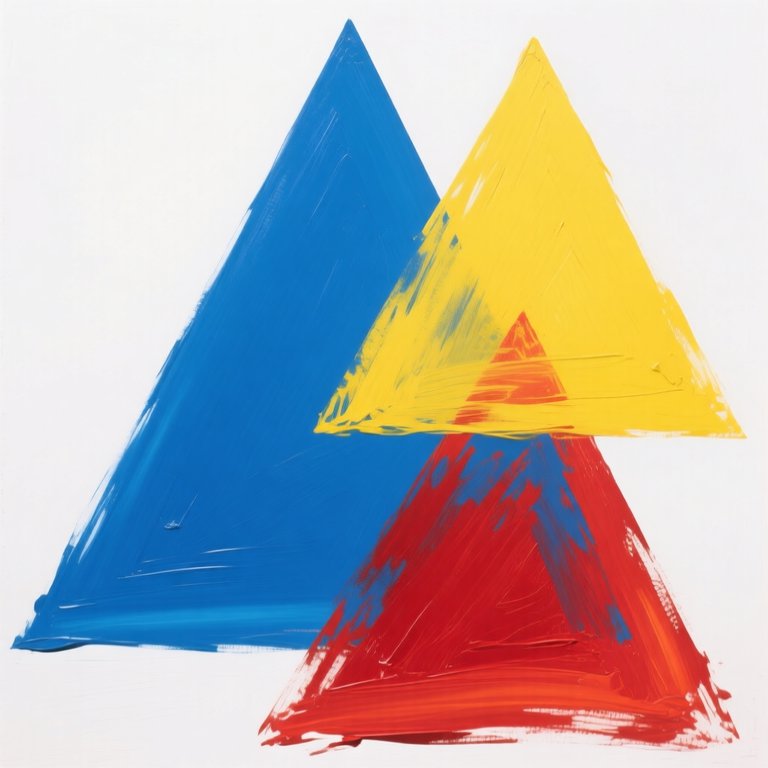}
      & \phantom{Teacher} \\ \addlinespace[2pt]
      & \multicolumn{4}{@{} p{10.8cm} @{}}{\tiny\itshape ``\input{figures/qwenimage/diversity_tables/659/prompt.txt}''} & \\
      \vspace{-20pt}
    \end{tabular}%
    }
  \end{figure}

%% file: figures/qwenimage/diversity_tables/659/prompt.txt
a vibrant abstract painting featuring three geometric shapes: a large blue triangle, a smaller yellow triangle, and a red triangle, all against a white background. The blue triangle has a smooth texture, while the yellow and red triangles have a more textured, almost brush-stroked appearance. The shapes are arranged in a way that suggests dynamic movement and balance within the composition.

%% file: figures/qwenimage/diversity_tables/488/table.tex
\begin{figure}[t]
  \centering
  \resizebox{0.95\textwidth}{!}{%

  \begin{tabular}{@{} r @{\hspace{6pt}} c@{\hspace{2pt}}c@{\hspace{2pt}} c@{\hspace{2pt}}c @{\hspace{6pt}} c @{}}
    & {\footnotesize Noise 1} & {\footnotesize Noise 2} &  {\footnotesize Noise 3} & {\footnotesize Noise 4} & \\
    Teacher
      & \framecell{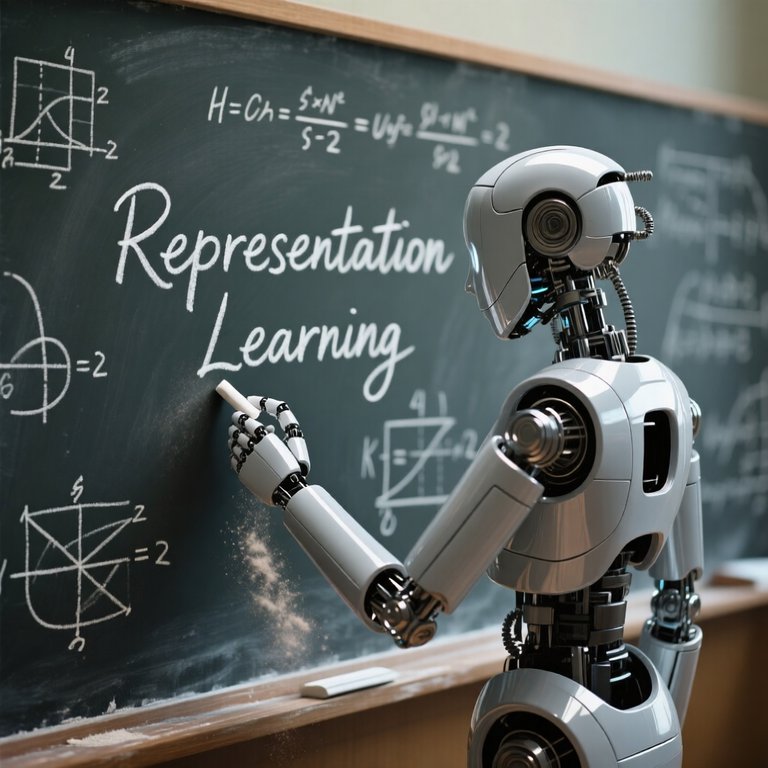}
      & \framecell{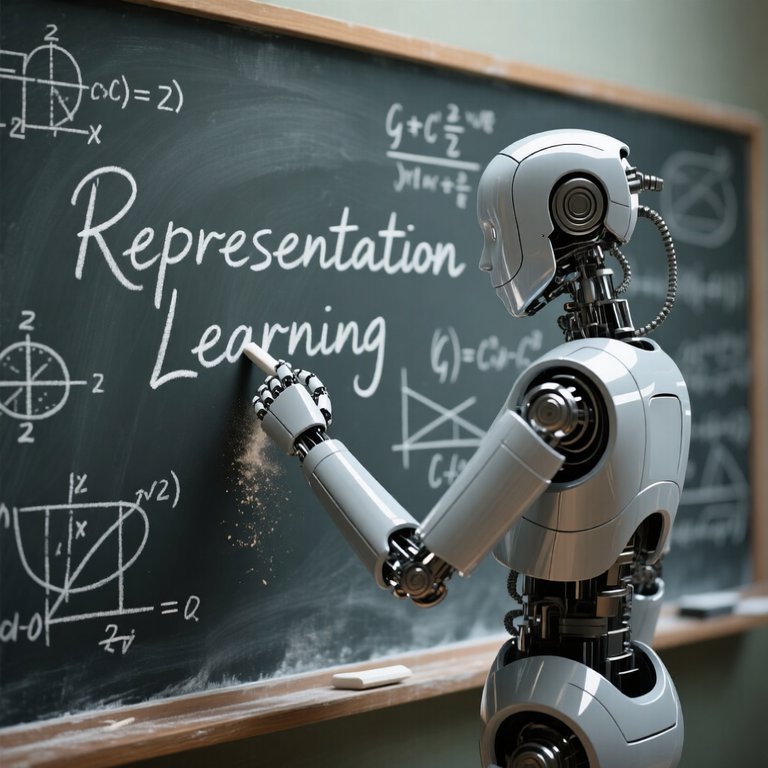}
      & \framecell{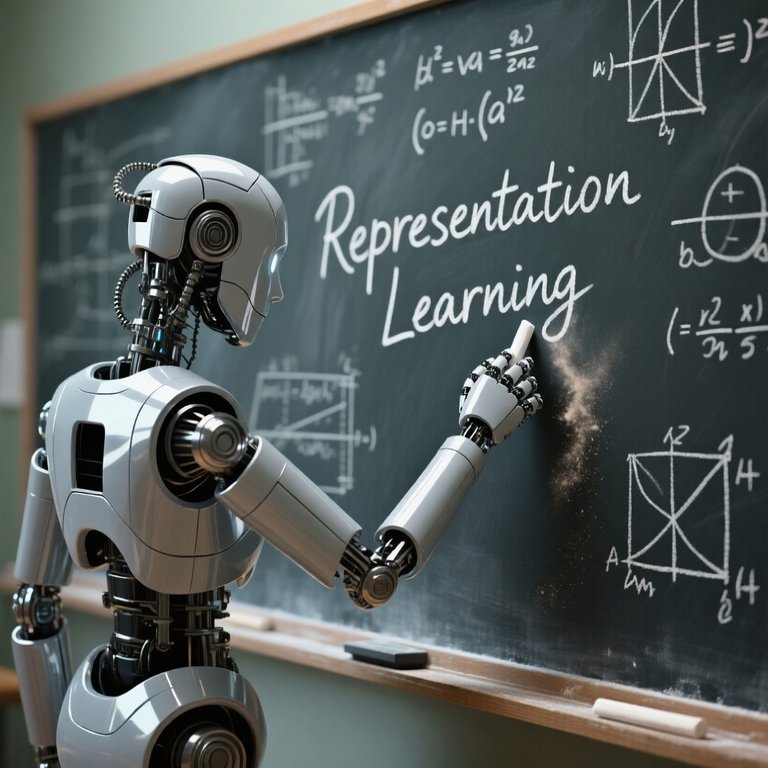}
      & \framecell{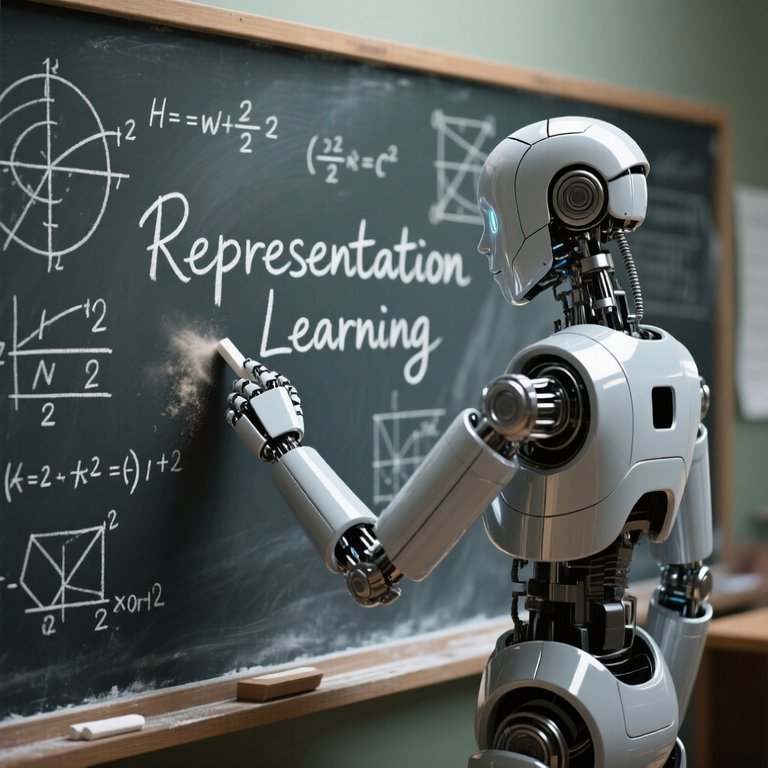}
      & \phantom{Teacher} \\ \addlinespace[2pt]
    \textbf{PDD}
      & \framecell{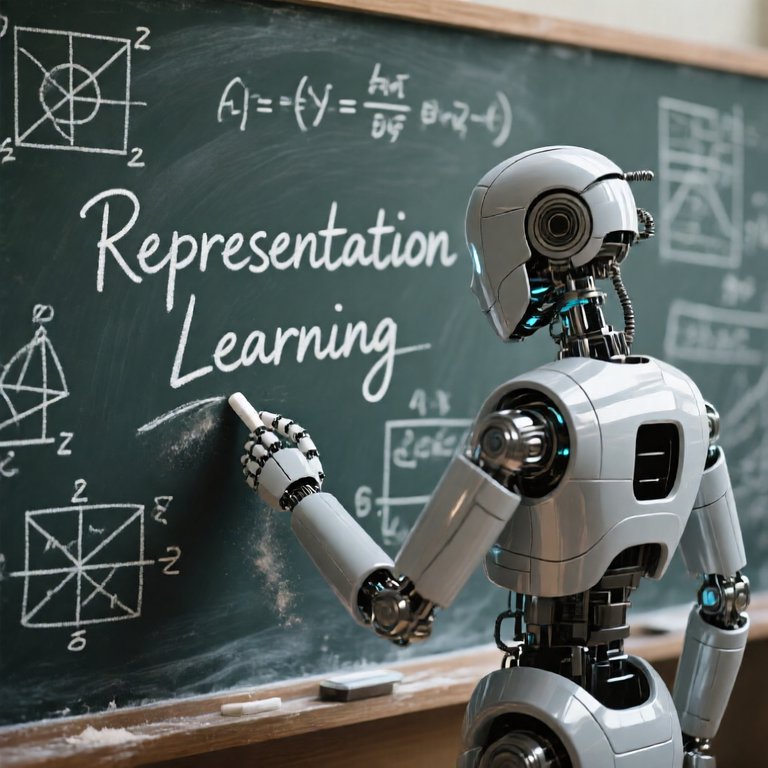}
      & \framecell{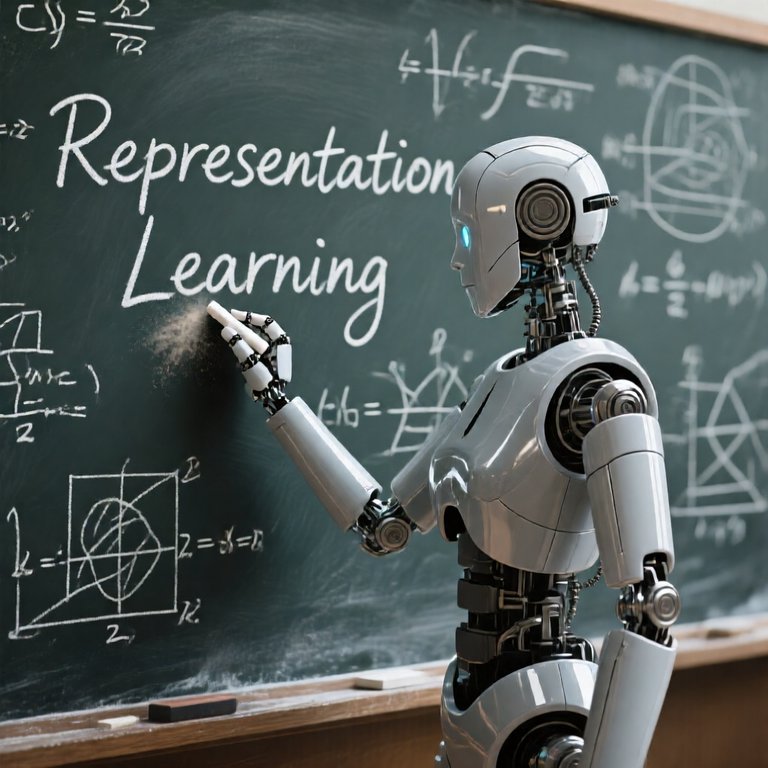}
      & \framecell{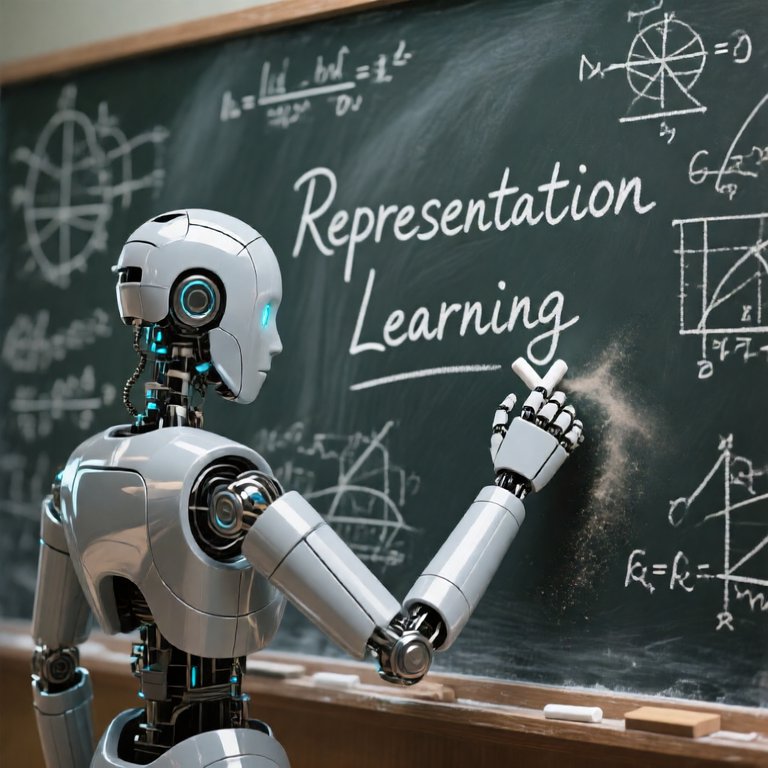}
      & \framecell{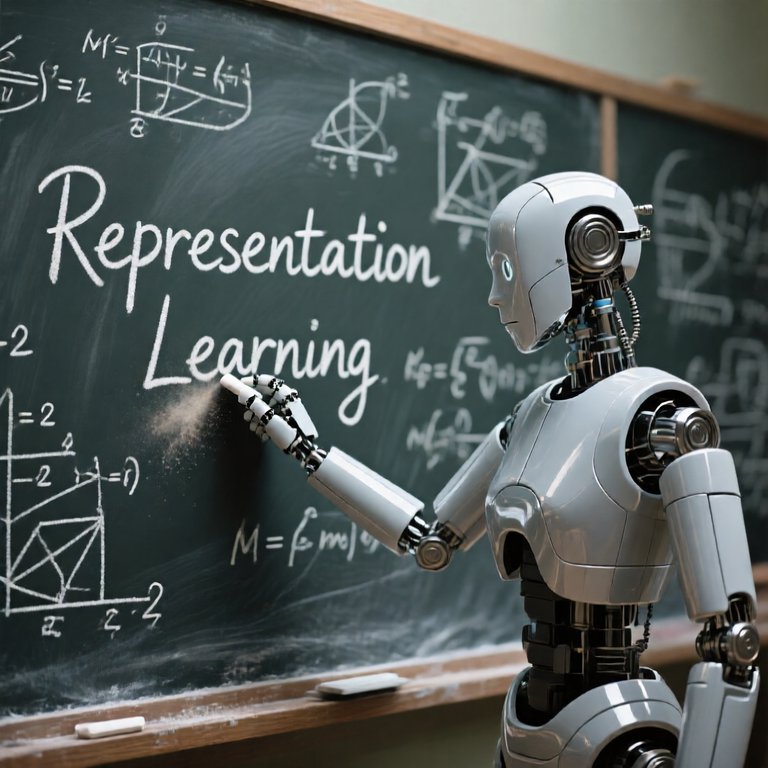}
      & \phantom{Teacher} \\ \addlinespace[2pt]
    DMD2
      & \framecell{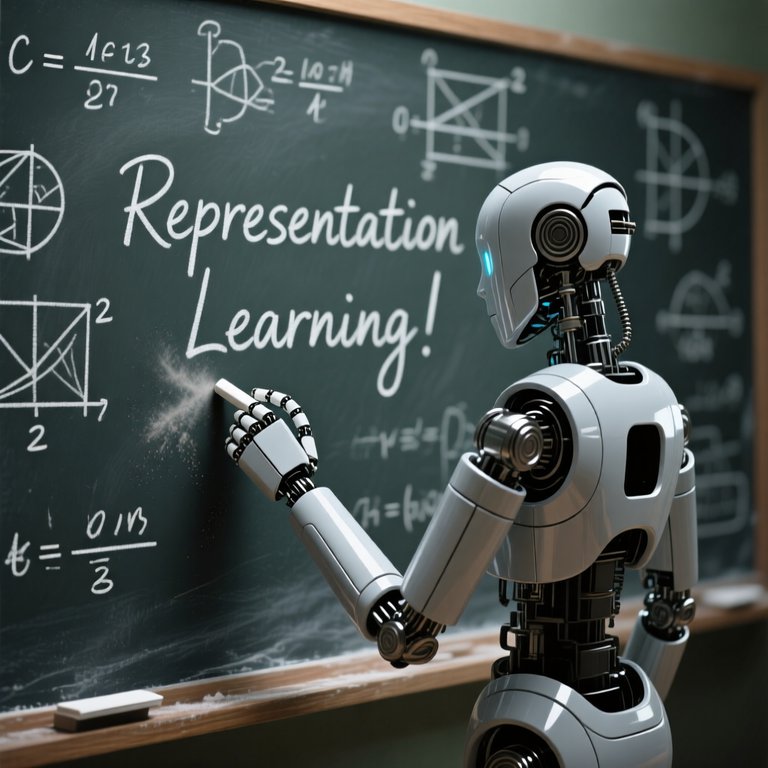}
      & \framecell{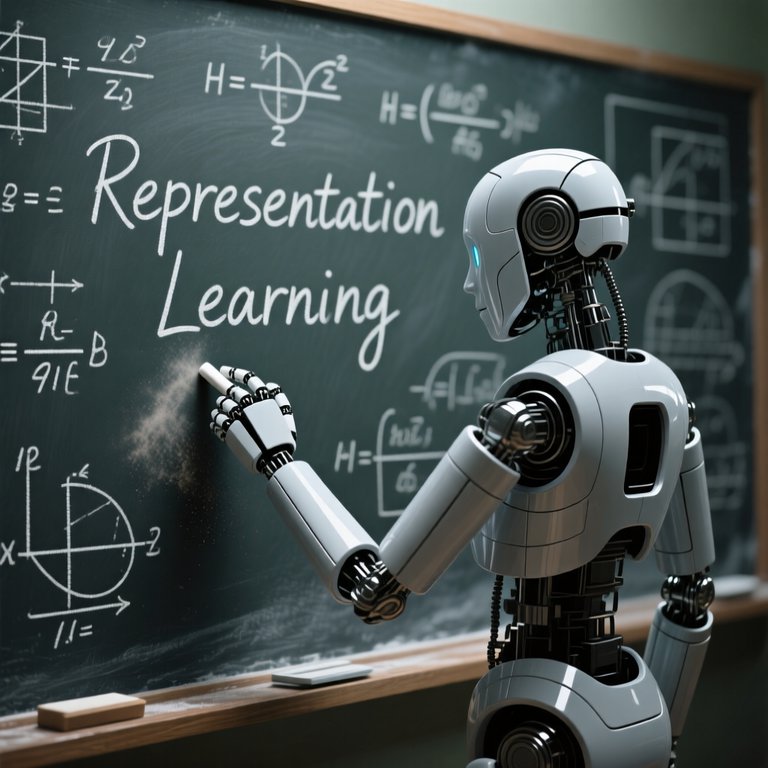}
      & \framecell{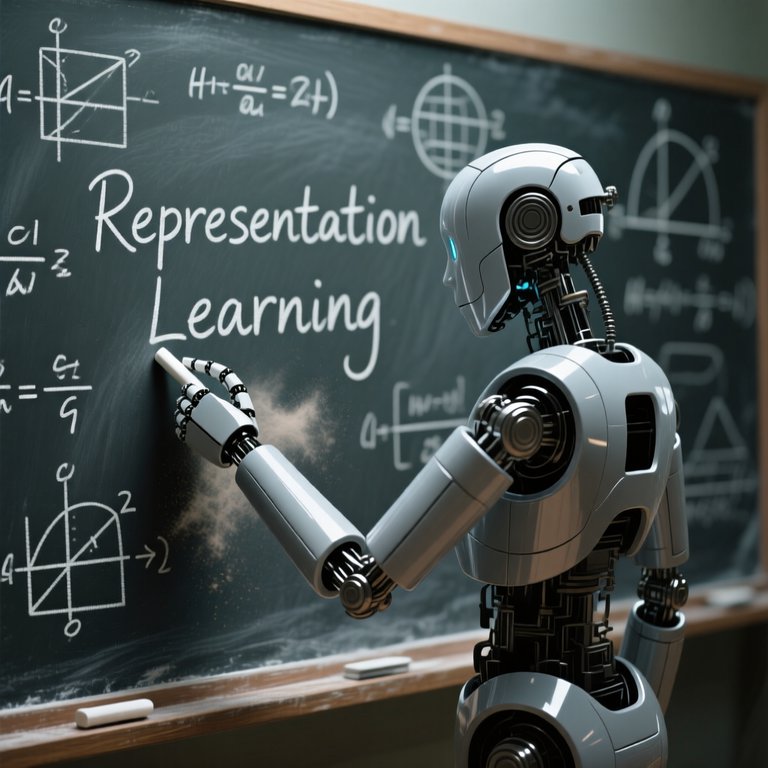}
      & \framecell{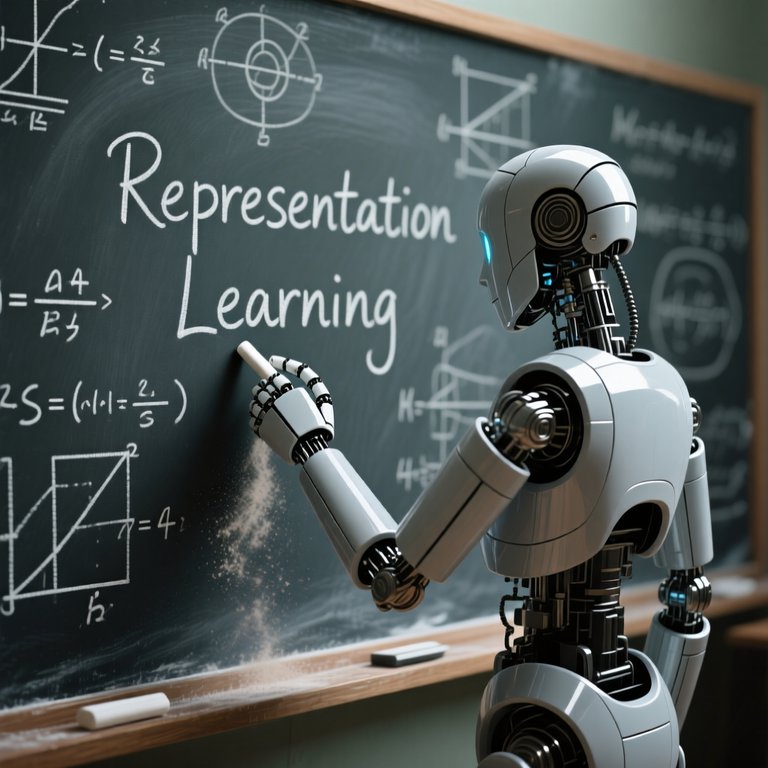}
      & \phantom{Teacher} \\ \addlinespace[2pt]
      & \multicolumn{4}{@{} p{10.8cm} @{}}{\tiny\itshape ``\input{figures/qwenimage/diversity_tables/488/prompt.txt}''} & \\
      \vspace{-15pt}
    \end{tabular}%
    }
    \caption{Teacher uses Euler method with 2$\times$50 NFE, PDD and DMD2 (Lightning-v2) use 4 NFE.}
  \end{figure}

%% file: figures/qwenimage/diversity_tables/488/prompt.txt
In the image, a sleek, metallic humanoid robot stands before a dusty chalkboard, on which it has carefully scrawled 'Representation Learning' in eloquent cursive. Surrounding the central text are several complex mathematical formulas and intricate diagrams, each contributing to a lecture on advanced concepts in machine learning. The robot's articulated fingers hold a piece of white chalk, leaving a trail of fine dust as it continues to write.

%% file: figures/qwenimage/diversity_tables/490/table.tex
\begin{figure}[t]
  \centering
  \resizebox{0.95\textwidth}{!}{%

  \begin{tabular}{@{} r @{\hspace{6pt}} c@{\hspace{2pt}}c@{\hspace{2pt}} c@{\hspace{2pt}}c @{\hspace{6pt}} c @{}}
    & {\footnotesize Noise 1} & {\footnotesize Noise 2} &  {\footnotesize Noise 3} & {\footnotesize Noise 4} & \\
    Teacher
      & \framecell{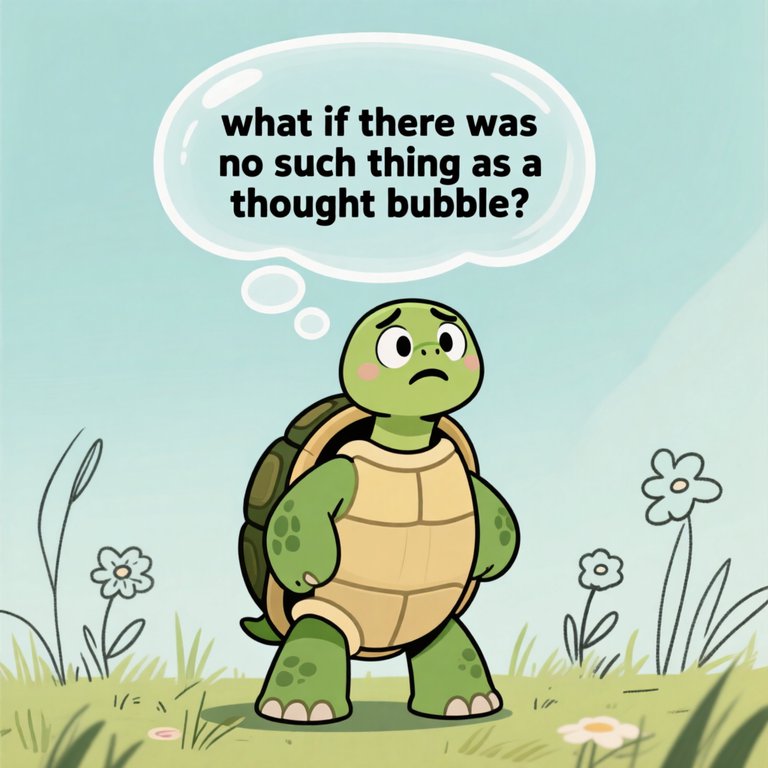}
      & \framecell{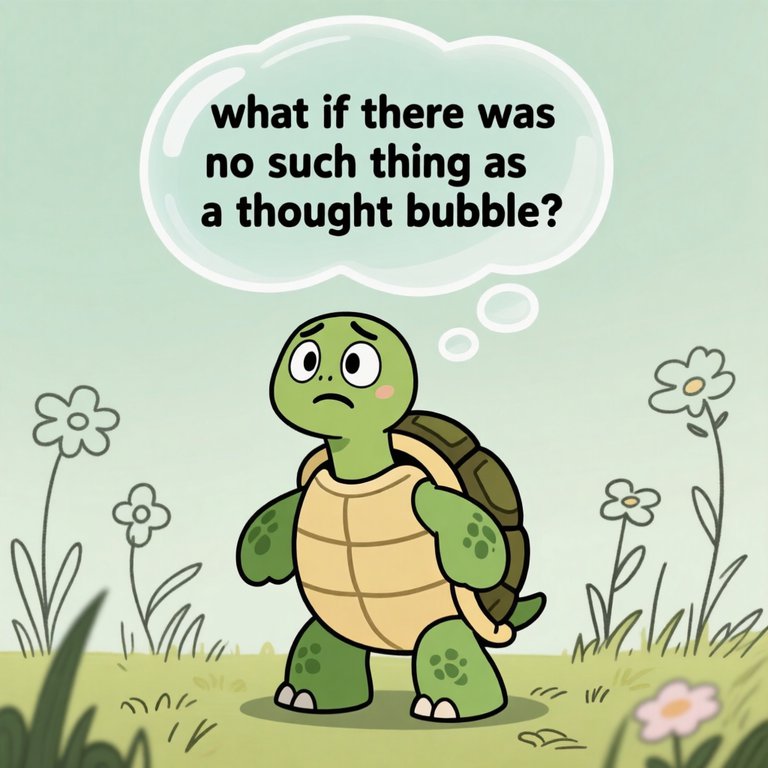}
      & \framecell{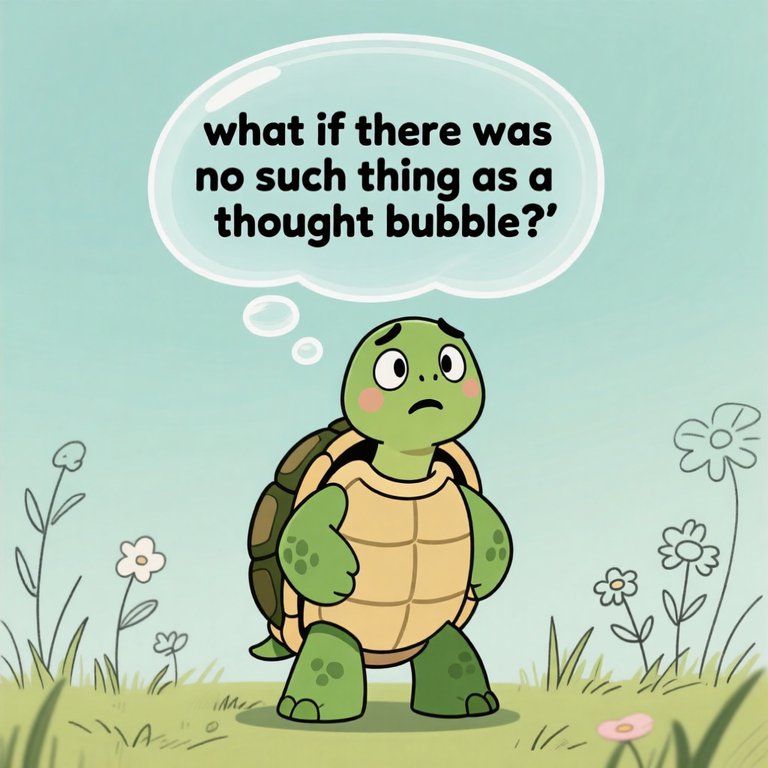}
      & \framecell{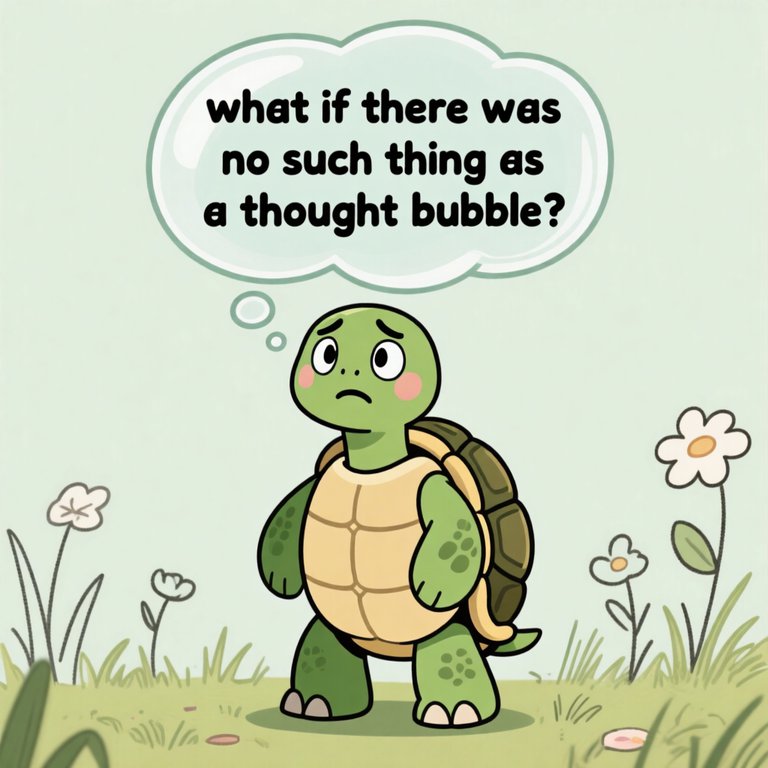}
      & \phantom{Teacher} \\ \addlinespace[2pt]
    \textbf{PDD}
      & \framecell{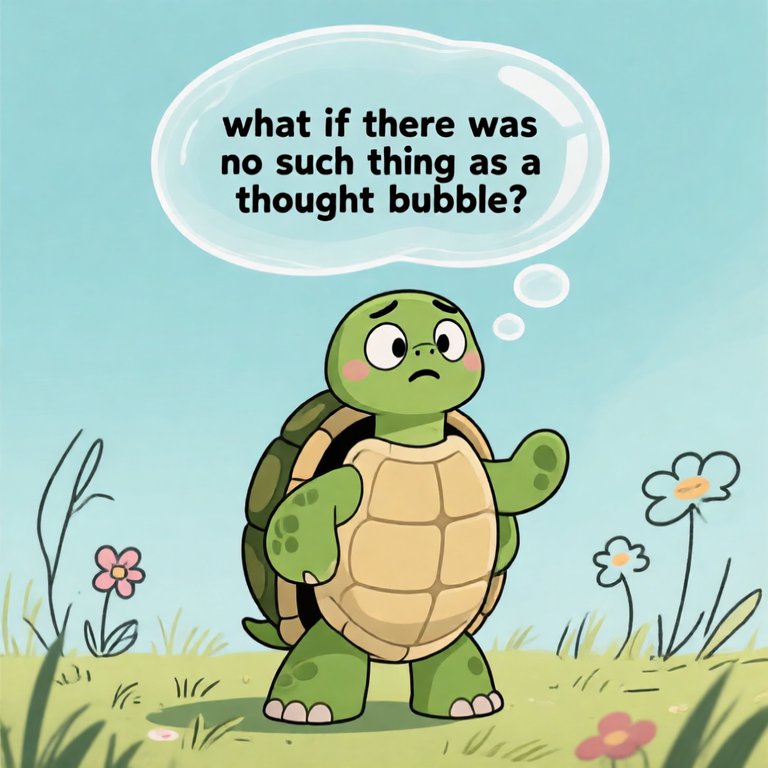}
      & \framecell{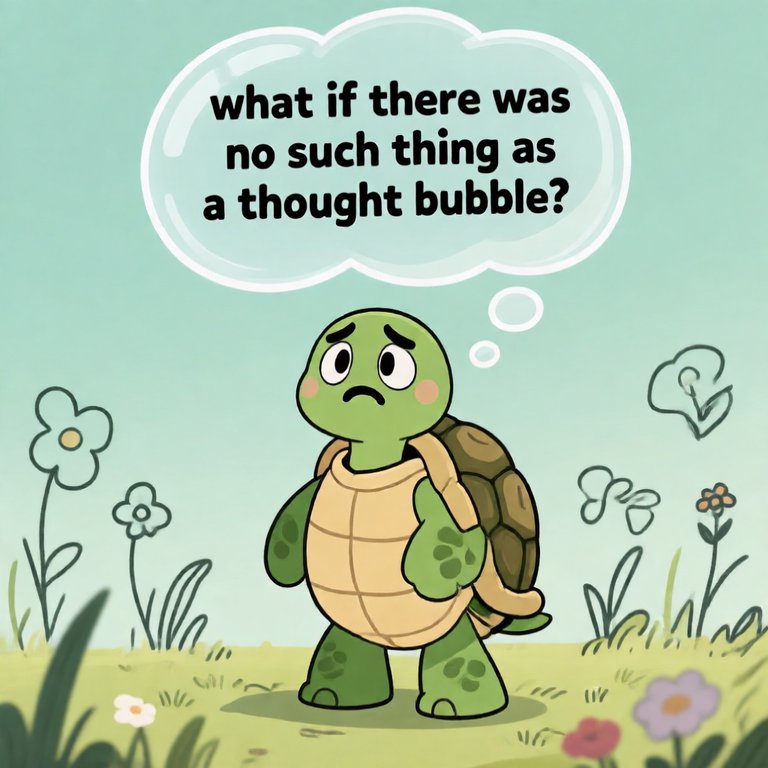}
      & \framecell{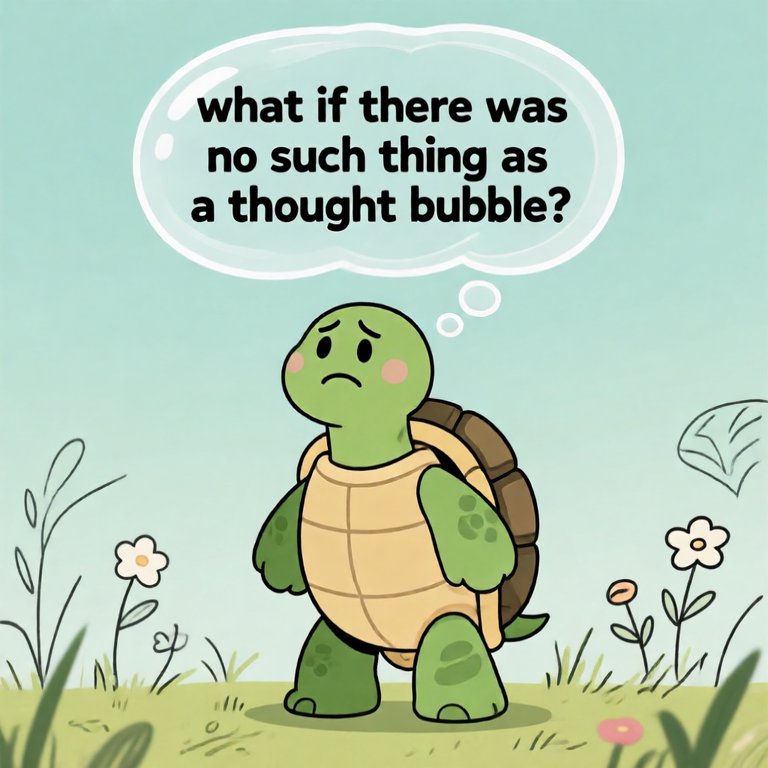}
      & \framecell{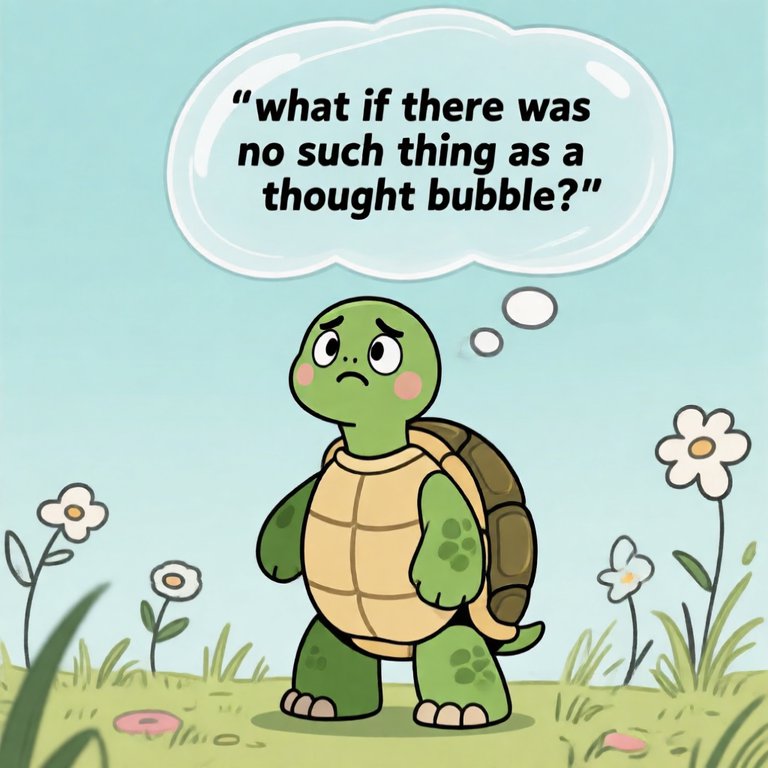}
      & \phantom{Teacher} \\ \addlinespace[2pt]
    DMD2
      & \framecell{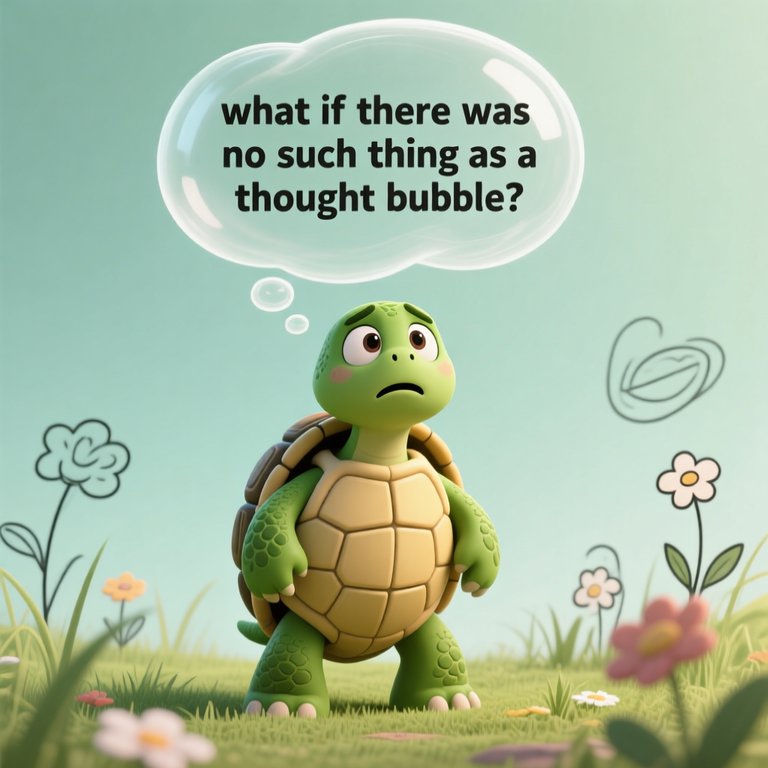}
      & \framecell{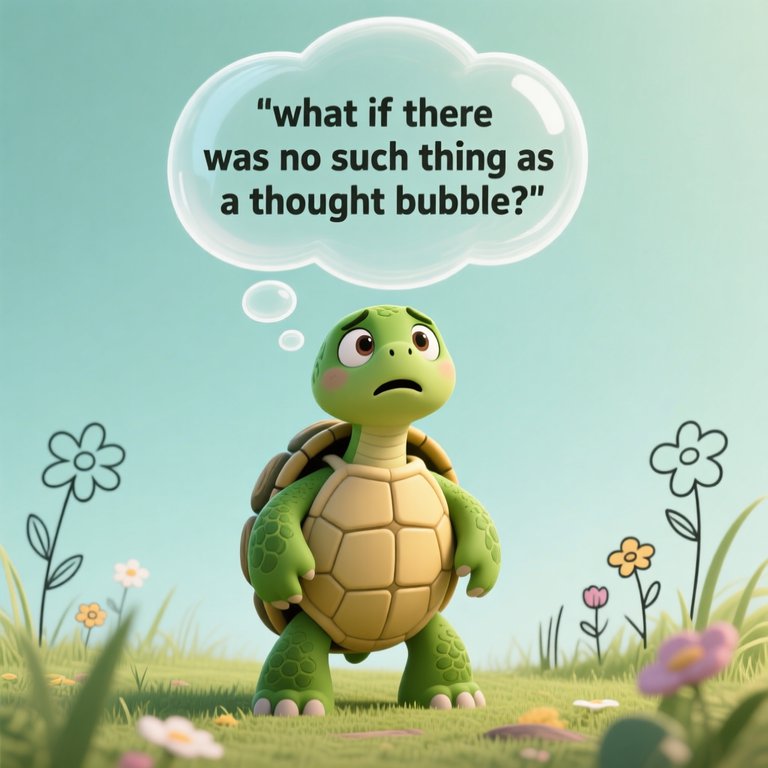}
      & \framecell{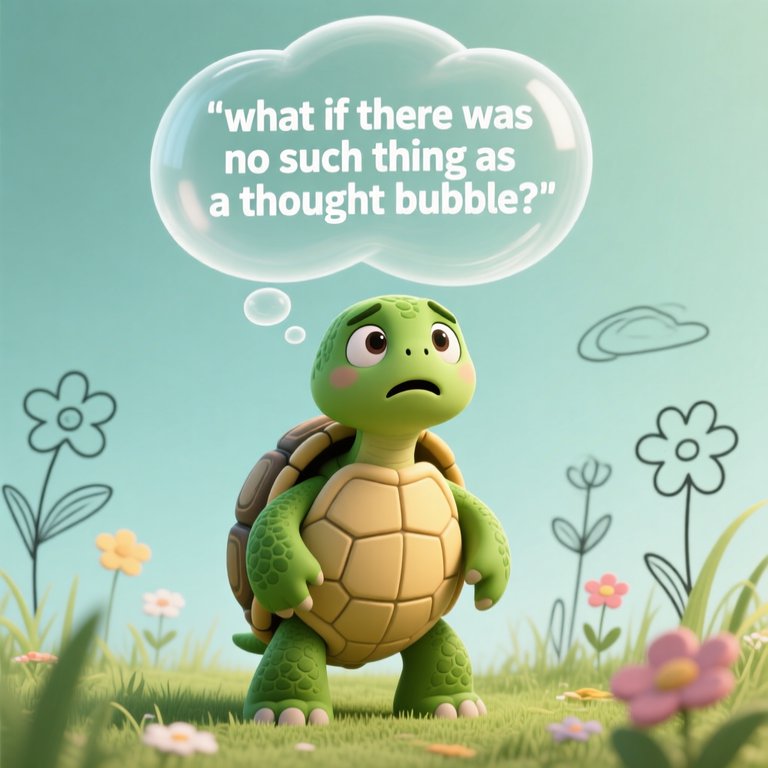}
      & \framecell{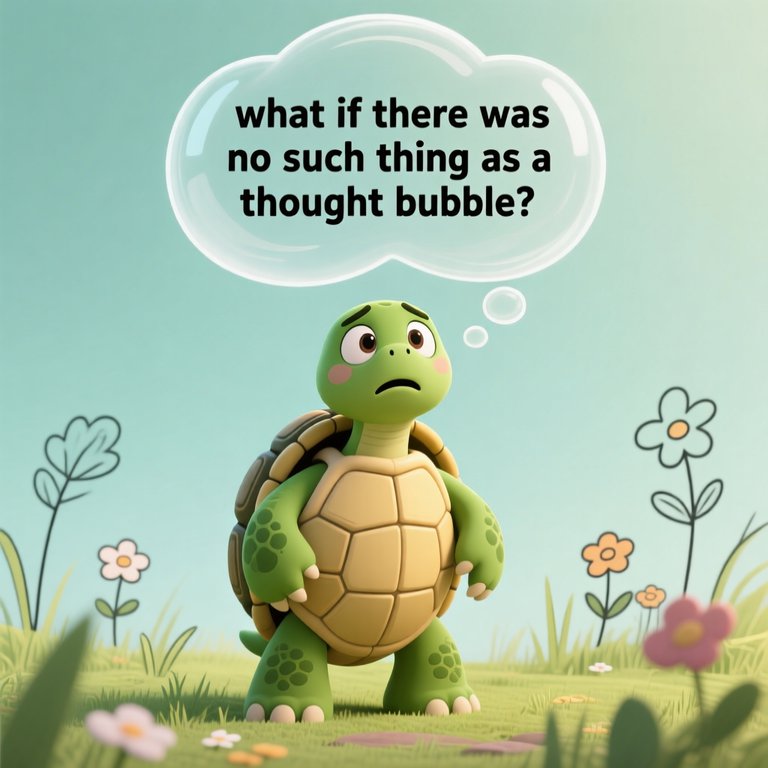}
      & \phantom{Teacher} \\ \addlinespace[2pt]
      & \multicolumn{4}{@{} p{10.8cm} @{}}{\tiny\itshape ``\input{figures/qwenimage/diversity_tables/490/prompt.txt}''} & \\
      \vspace{-20pt}
    \end{tabular}%
    }
  \end{figure}

%% file: figures/qwenimage/diversity_tables/490/prompt.txt
an animated image depicting a green turtle with a perplexed expression on its face. The turtle is standing upright on two legs and has a large, transparent thought bubble above its head, filled with the paradoxical question, 'what if there was no such thing as a thought bubble?' surrounding the turtle, there are simplistic drawings of grass and flowers, emphasizing the cartoonish nature of the image.

%% file: figures/qwenimage/diversity_tables/586/table.tex
\begin{figure}[t]
  \centering
  \resizebox{0.95\textwidth}{!}{%

  \begin{tabular}{@{} r @{\hspace{6pt}} c@{\hspace{2pt}}c@{\hspace{2pt}} c@{\hspace{2pt}}c @{\hspace{6pt}} c @{}}
    & {\footnotesize Noise 1} & {\footnotesize Noise 2} &  {\footnotesize Noise 3} & {\footnotesize Noise 4} & \\
    Teacher
      & \framecell{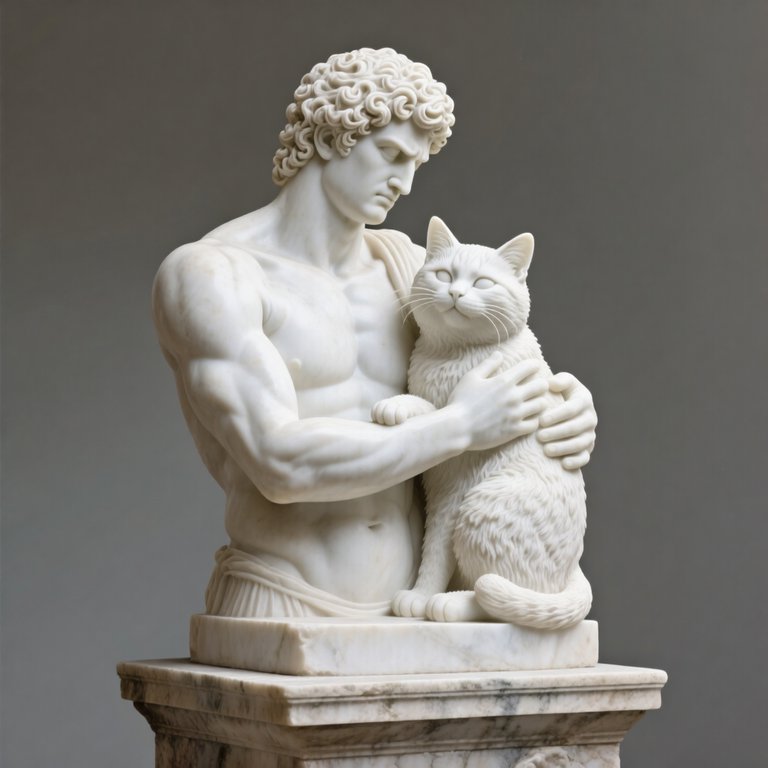}
      & \framecell{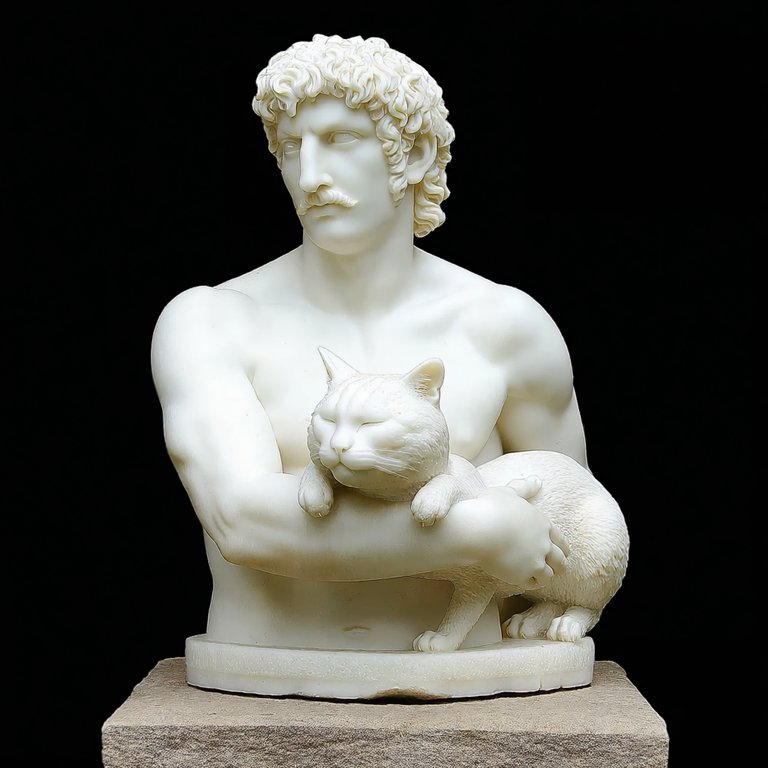}
      & \framecell{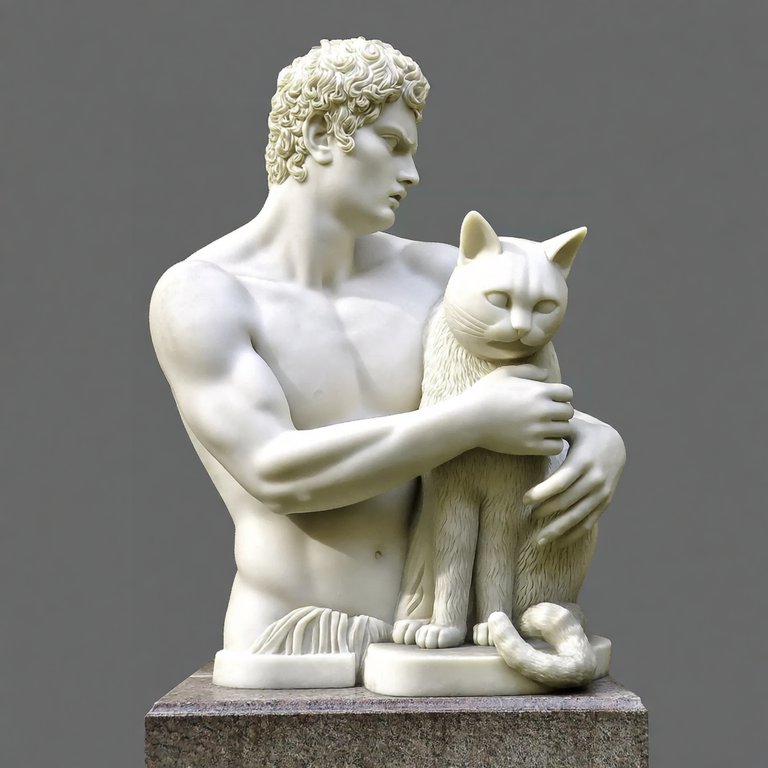}
      & \framecell{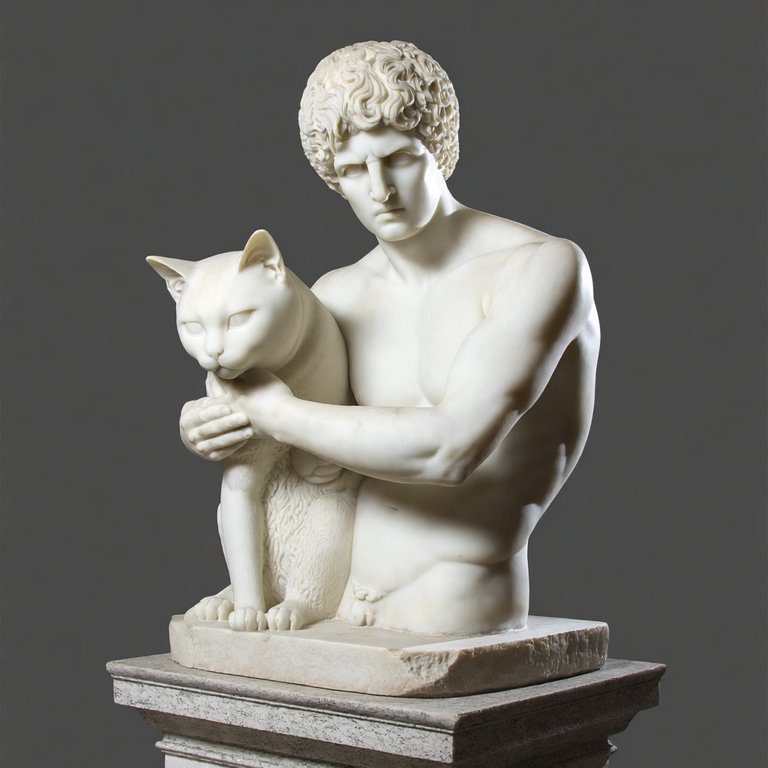}
      & \phantom{Teacher} \\ \addlinespace[2pt]
    \textbf{PDD}
      & \framecell{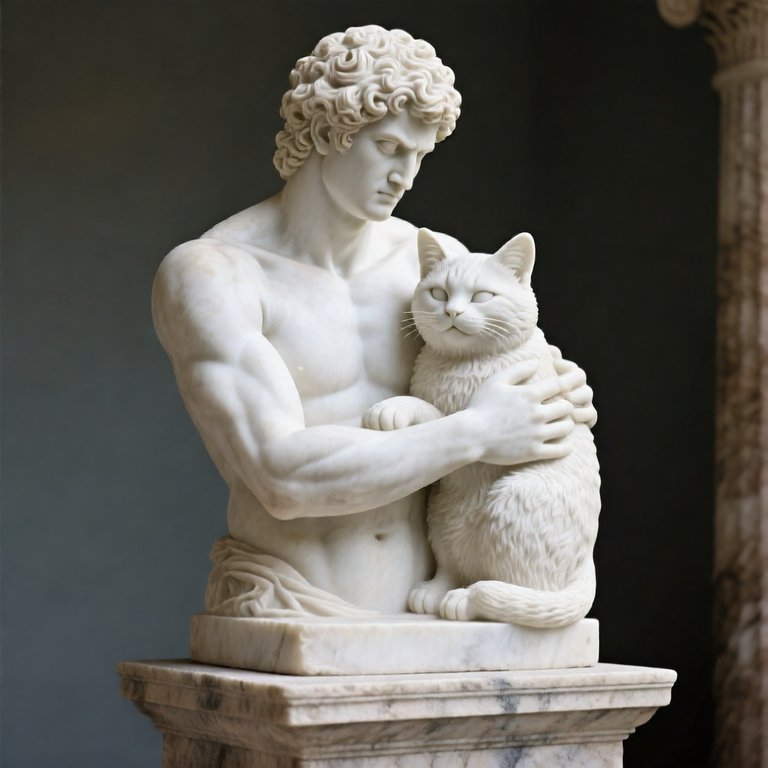}
      & \framecell{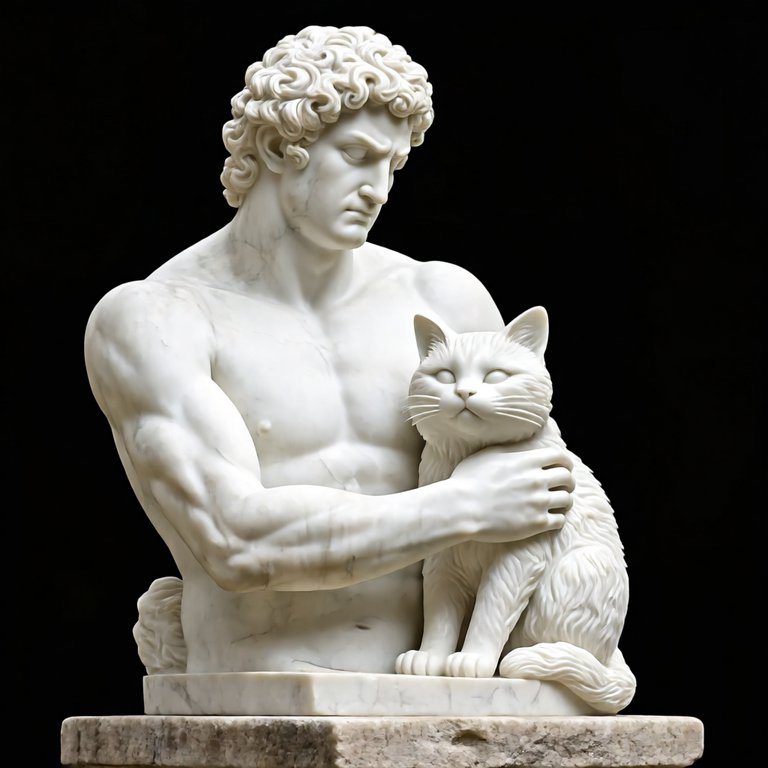}
      & \framecell{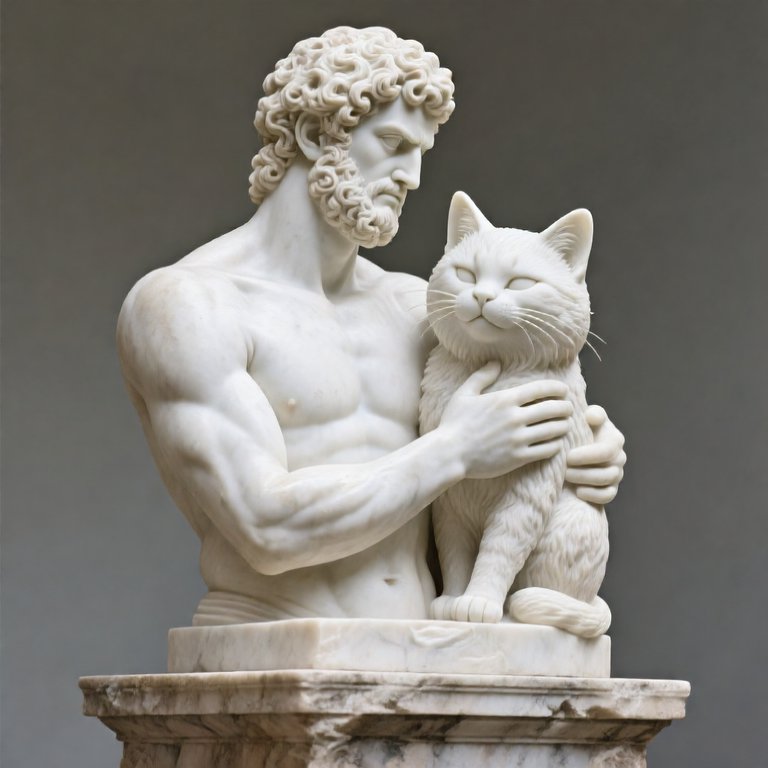}
      & \framecell{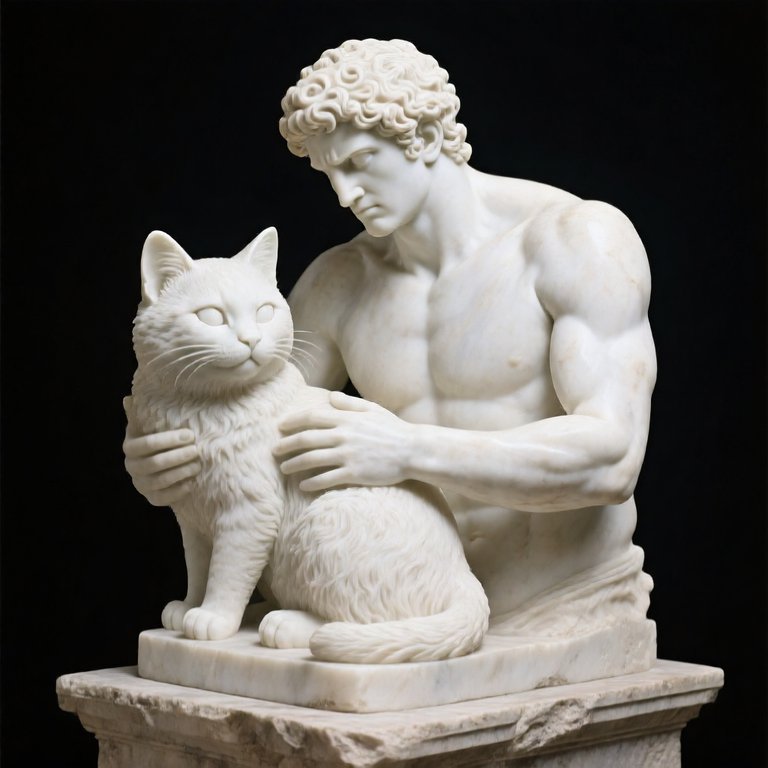}
      & \phantom{Teacher} \\ \addlinespace[2pt]
    DMD2
      & \framecell{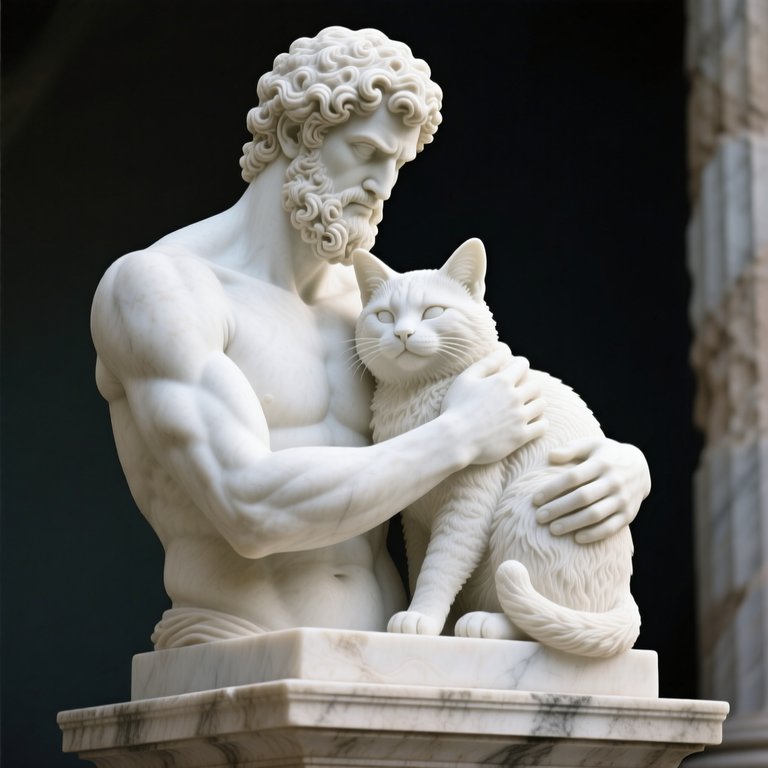}
      & \framecell{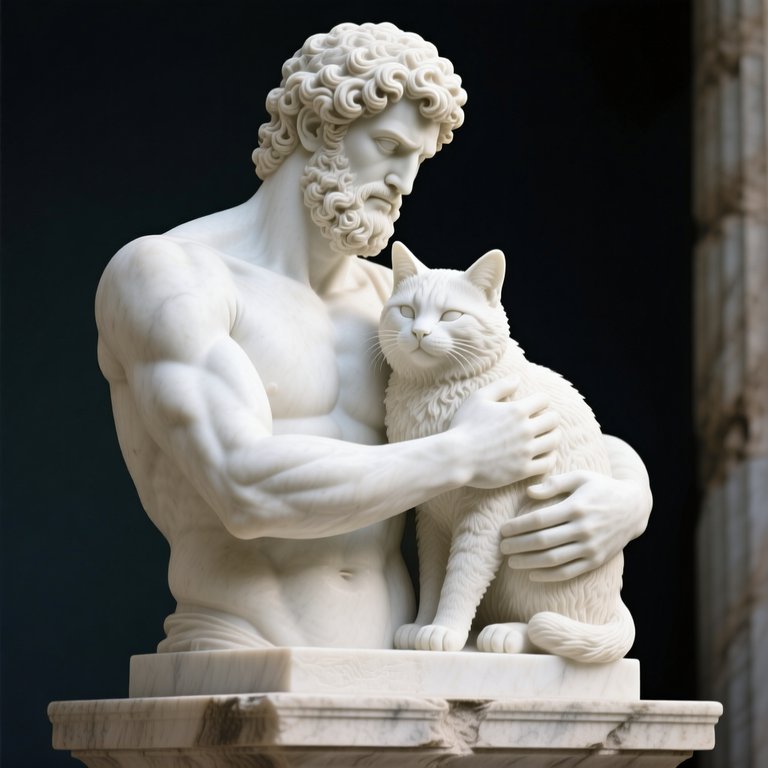}
      & \framecell{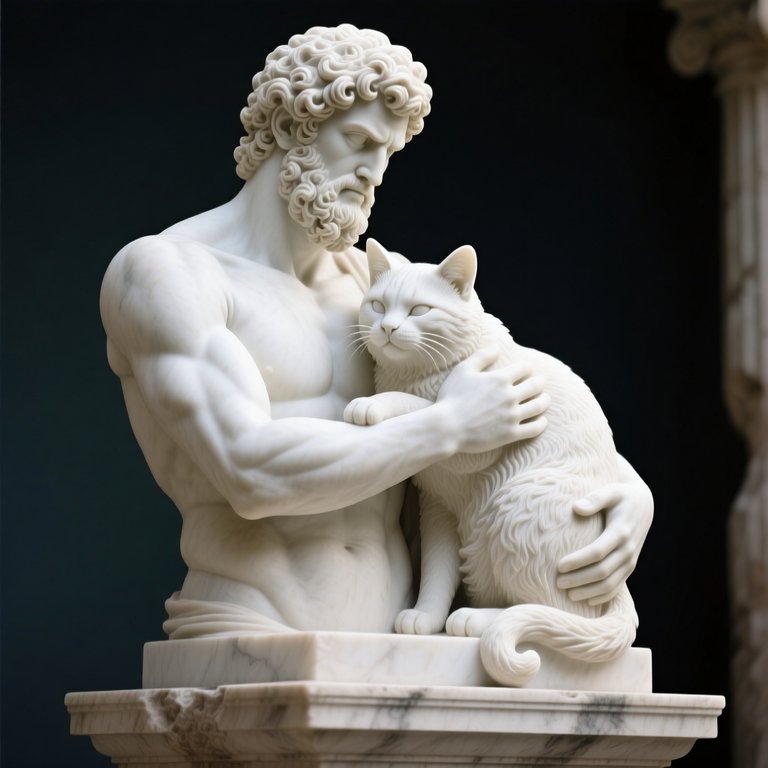}
      & \framecell{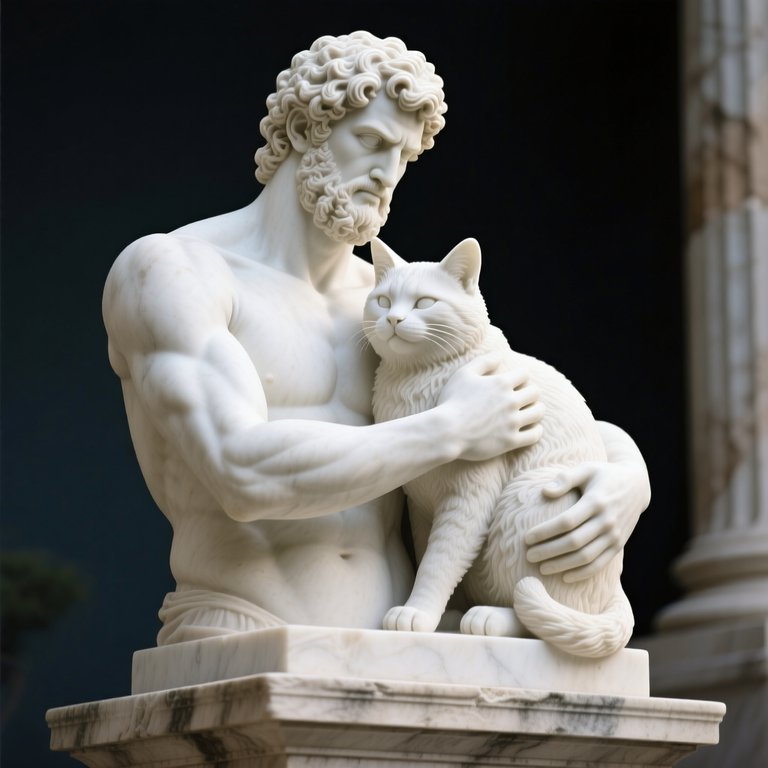}
      & \phantom{Teacher} \\ \addlinespace[2pt]
      & \multicolumn{4}{@{} p{10.8cm} @{}}{\tiny\itshape ``\input{figures/qwenimage/diversity_tables/586/prompt.txt}''} & \\
      \vspace{-15pt}
    \end{tabular}%
    }
    \caption{Teacher uses Euler method with 2$\times$50 NFE, PDD and DMD2 (Lightning-v2) use 4 NFE.}
  \end{figure}

%% file: figures/qwenimage/diversity_tables/586/prompt.txt
A classic Greek statue carved from white marble, depicting a muscular man with a stern expression as he gently comforts a cat with an unusually large head. The cat, also sculpted from marble, seems to be at ease in the man's strong arms. The statue stands on a stone pedestal, and the intricate details of the man's curly hair and the cat's fur texture are visible, highlighting the sculptor's skilled craftsmanship.

%% file: figures/qwenimage/diversity_tables/321/table.tex
\begin{figure}[t]
  \centering
  \resizebox{0.95\textwidth}{!}{%

  \begin{tabular}{@{} r @{\hspace{6pt}} c@{\hspace{2pt}}c@{\hspace{2pt}} c@{\hspace{2pt}}c @{\hspace{6pt}} c @{}}
    & {\footnotesize Noise 1} & {\footnotesize Noise 2} &  {\footnotesize Noise 3} & {\footnotesize Noise 4} & \\
    Teacher
      & \framecell{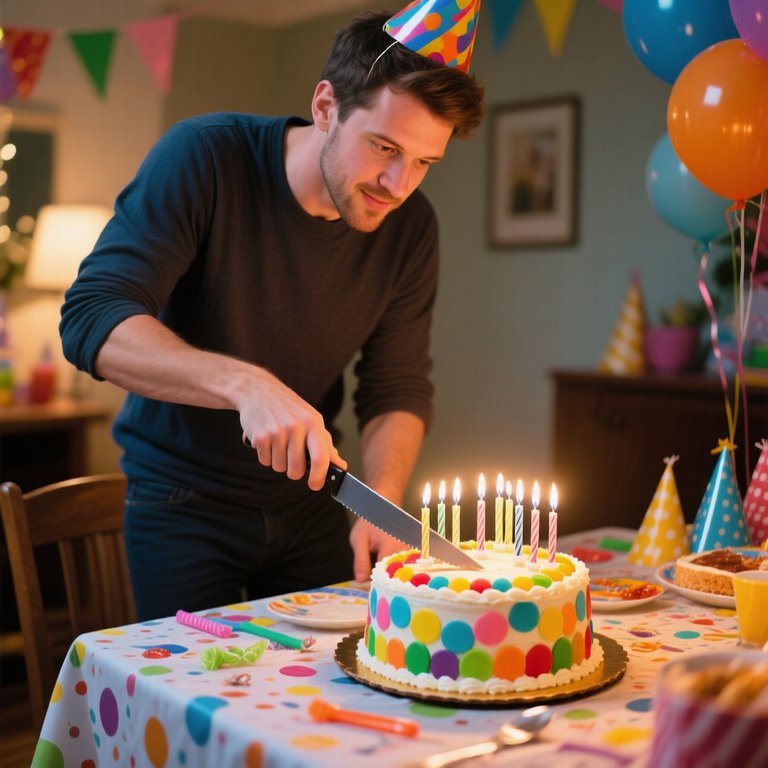}
      & \framecell{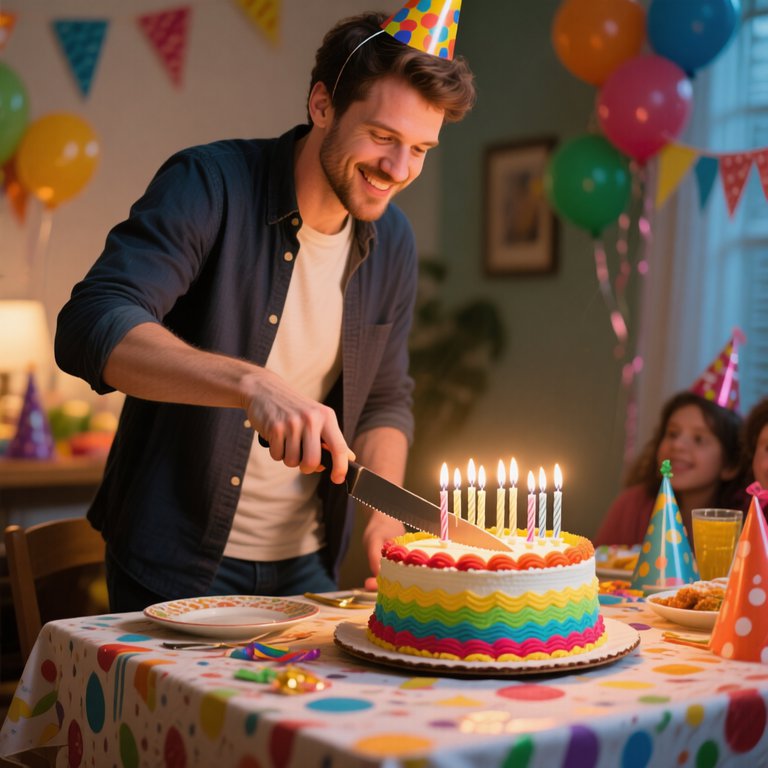}
      & \framecell{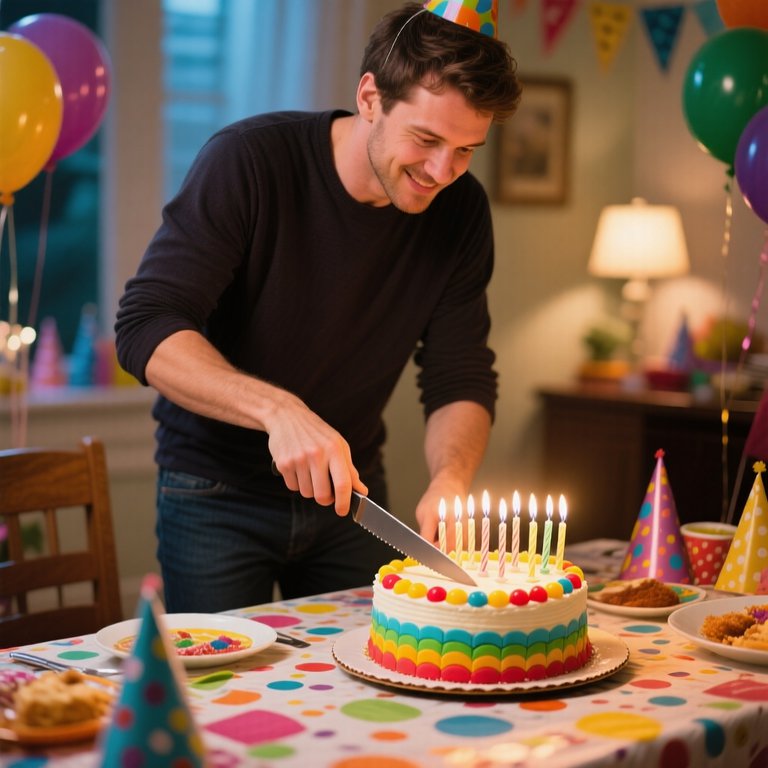}
      & \framecell{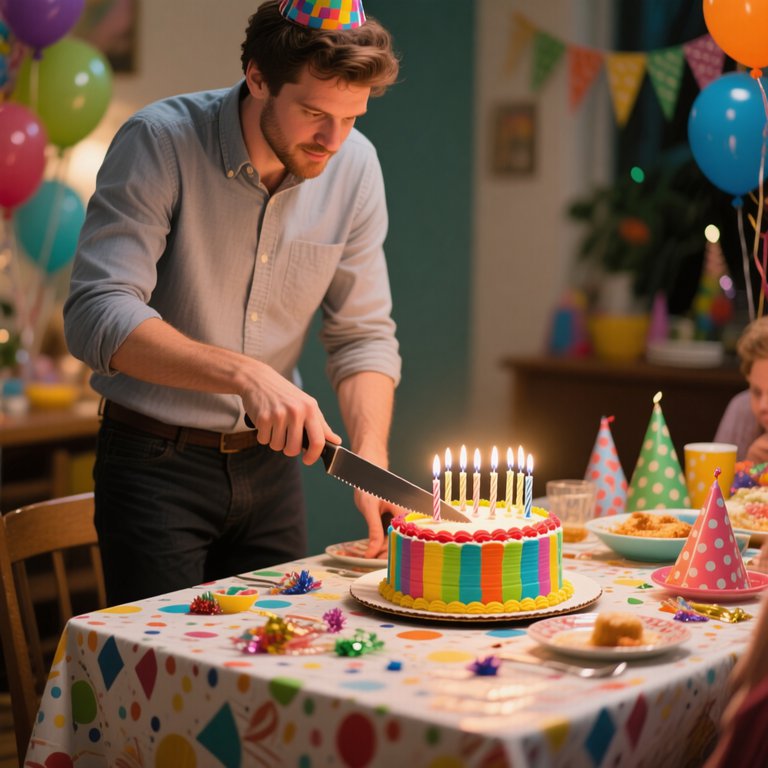}
      & \phantom{Teacher} \\ \addlinespace[2pt]
    \textbf{PDD}
      & \framecell{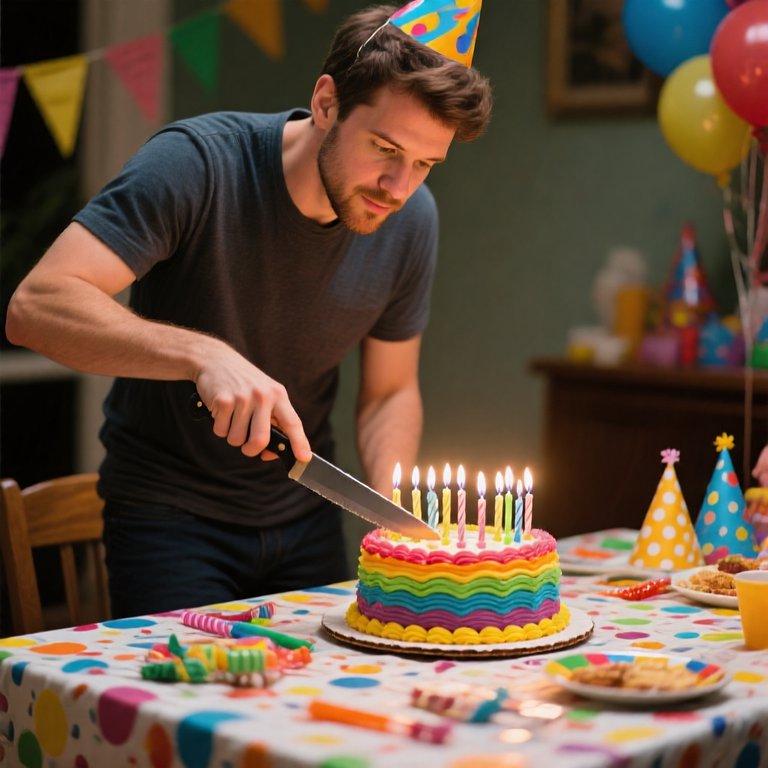}
      & \framecell{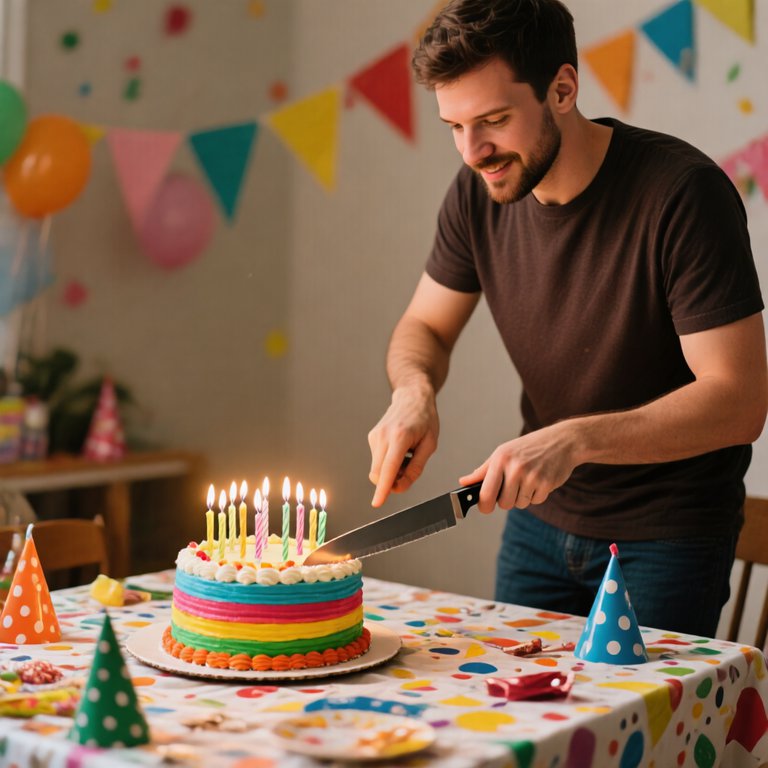}
      & \framecell{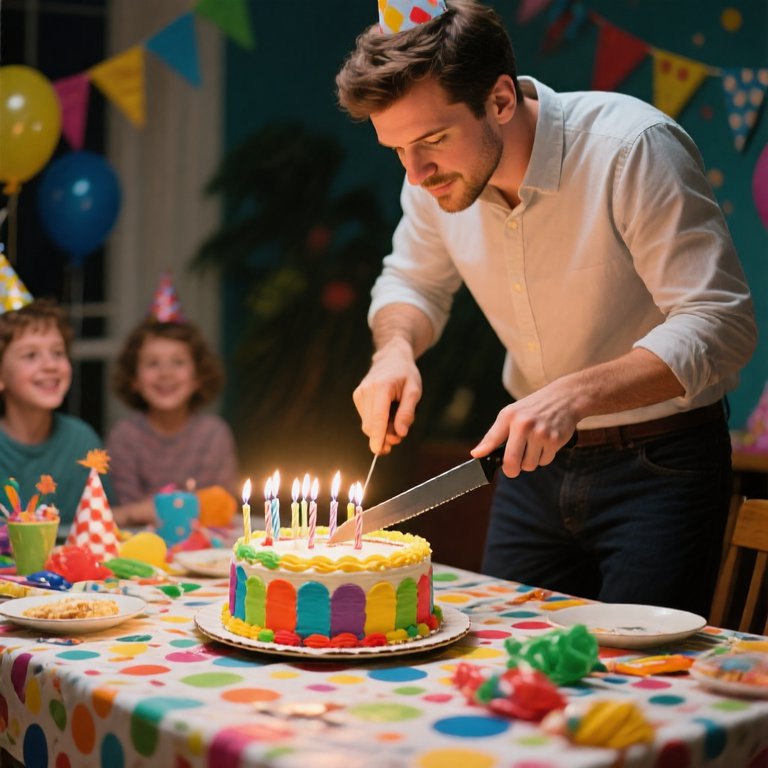}
      & \framecell{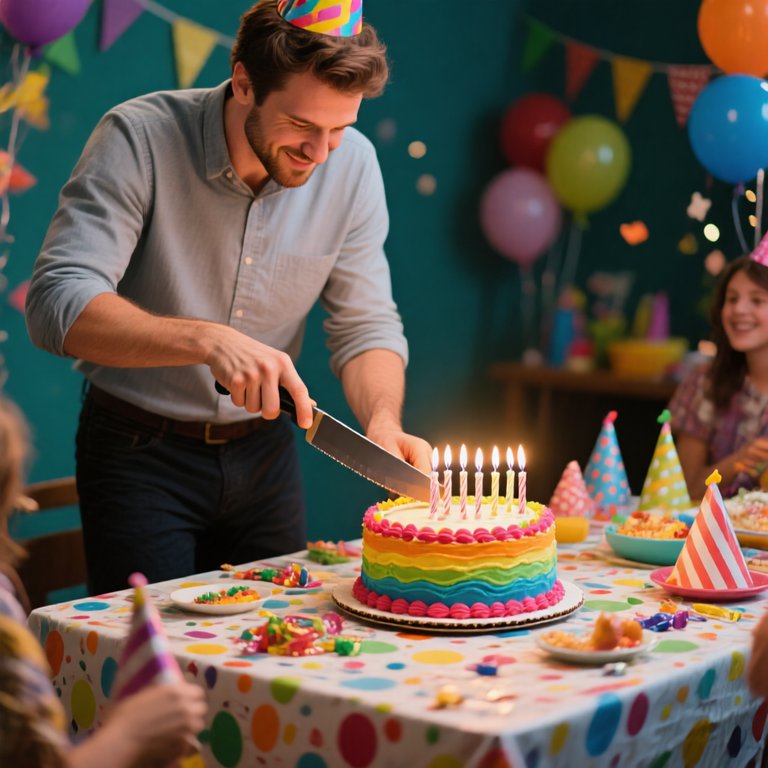}
      & \phantom{Teacher} \\ \addlinespace[2pt]
    DMD2
      & \framecell{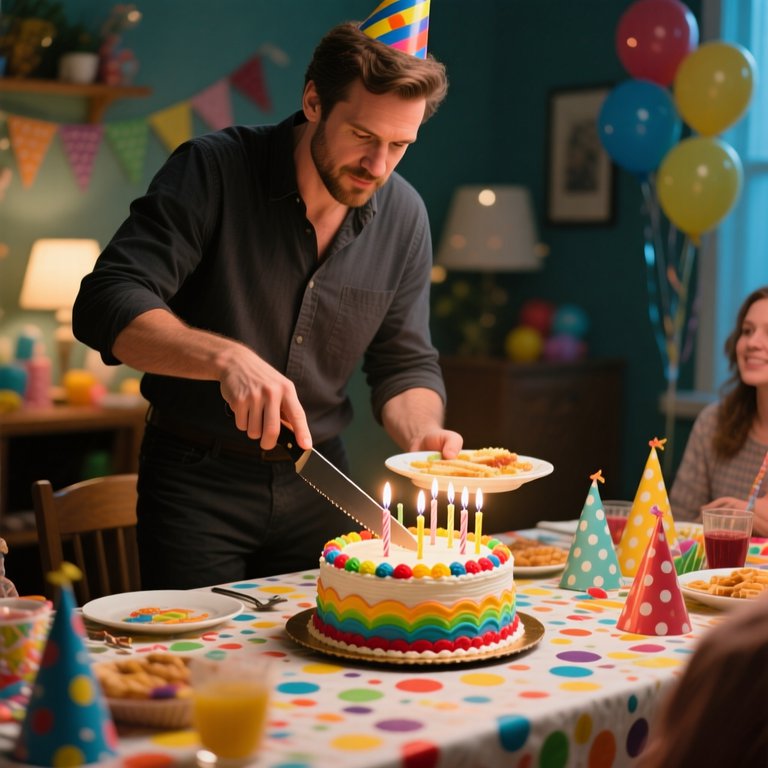}
      & \framecell{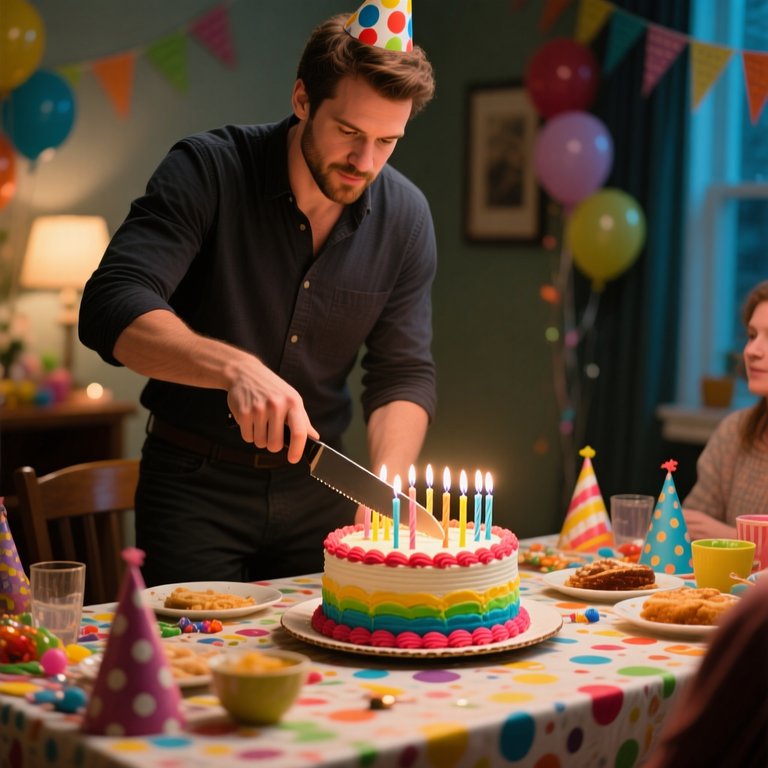}
      & \framecell{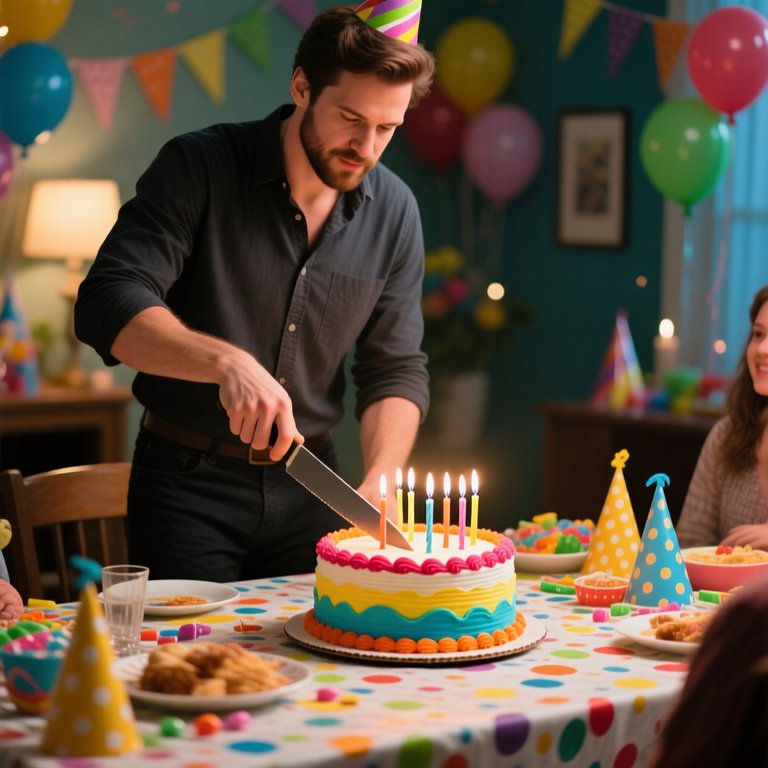}
      & \framecell{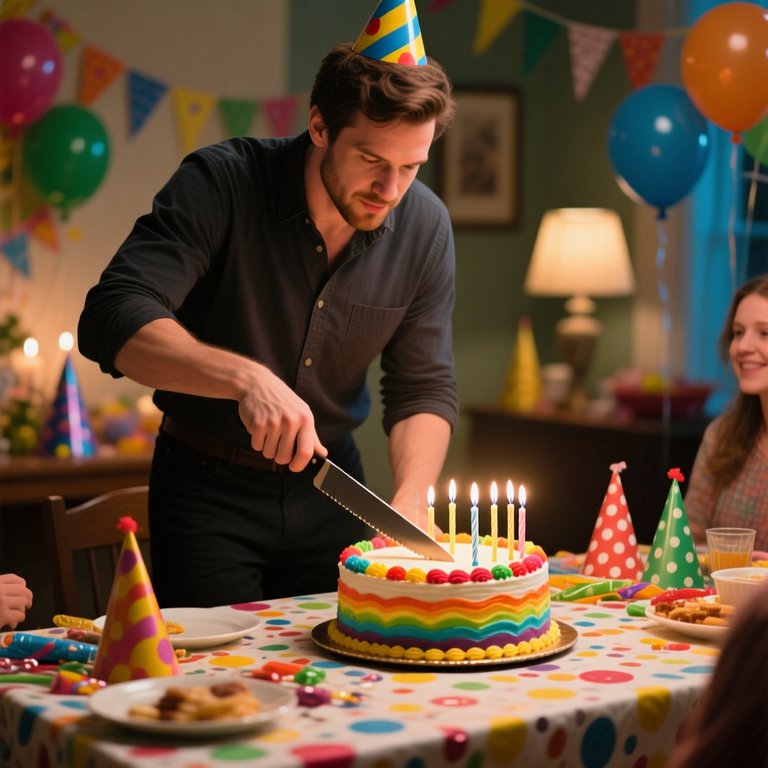}
      & \phantom{Teacher} \\ \addlinespace[2pt]
      & \multicolumn{4}{@{} p{10.8cm} @{}}{\tiny\itshape ``\input{figures/qwenimage/diversity_tables/321/prompt.txt}''} & \\
      \vspace{-20pt}
    \end{tabular}%
    }
  \end{figure}

%% file: figures/qwenimage/diversity_tables/321/prompt.txt
A man stands at a table with a colorful birthday cake adorned with lit candles. He carefully slices the cake with a long knife, preparing to serve the guests. The table is covered with a festive tablecloth and scattered with other party accessories.

%% file: figures/qwenimage/diversity_tables/706/table.tex
\begin{figure}[t]
  \centering
  \resizebox{0.95\textwidth}{!}{%

  \begin{tabular}{@{} r @{\hspace{6pt}} c@{\hspace{2pt}}c@{\hspace{2pt}} c@{\hspace{2pt}}c @{\hspace{6pt}} c @{}}
    & {\footnotesize Noise 1} & {\footnotesize Noise 2} &  {\footnotesize Noise 3} & {\footnotesize Noise 4} & \\
    Teacher
      & \framecell{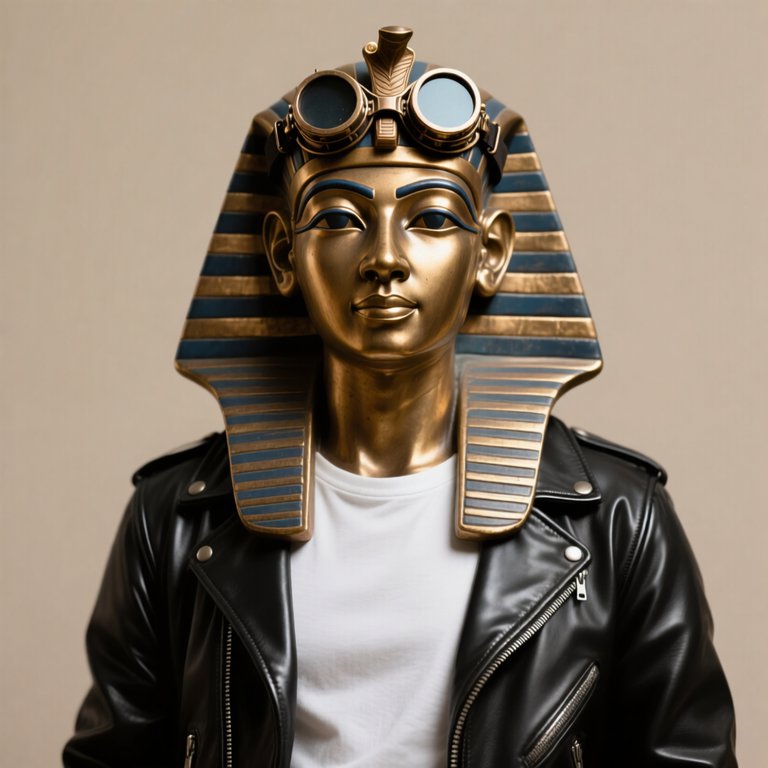}
      & \framecell{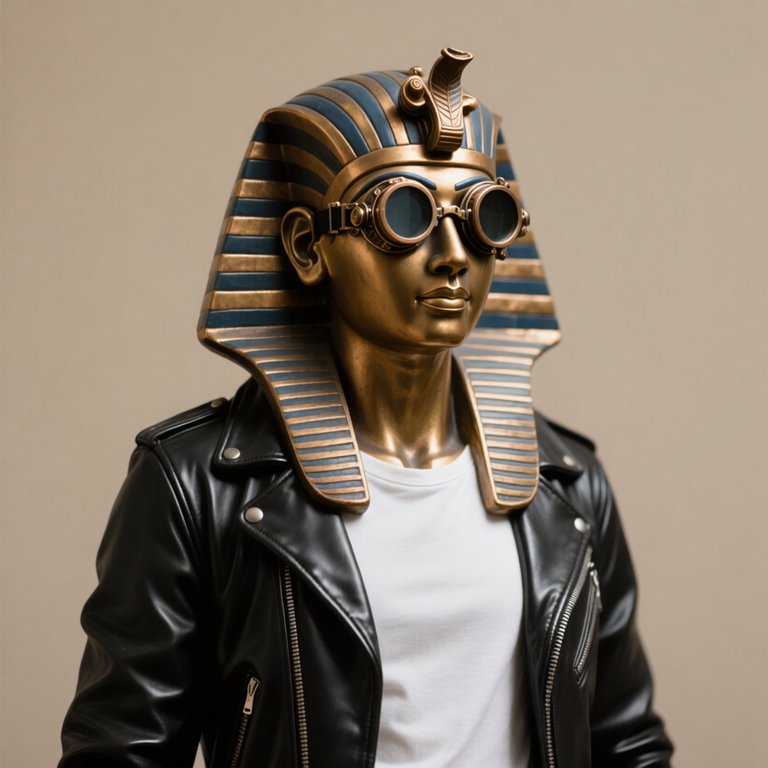}
      & \framecell{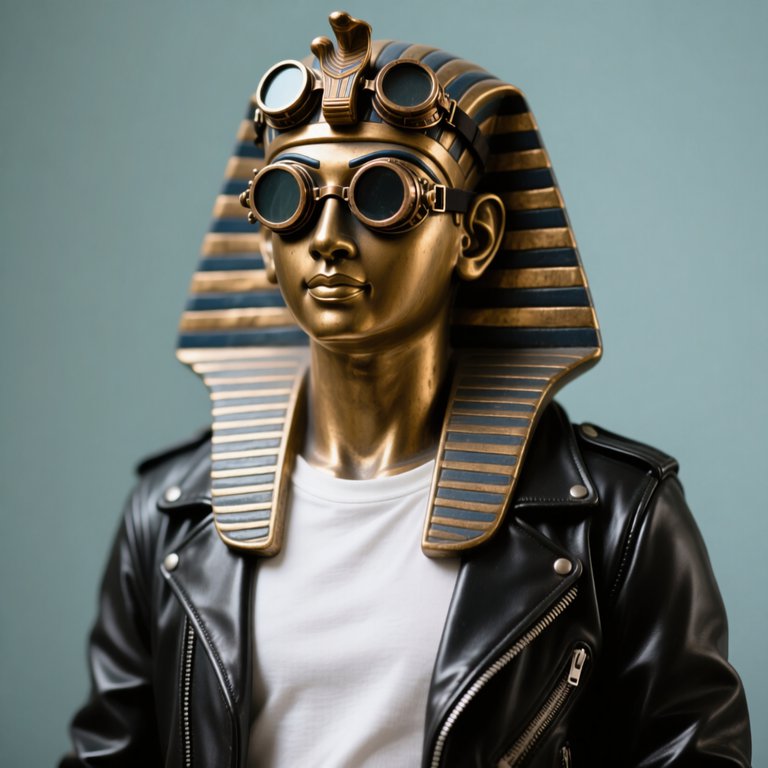}
      & \framecell{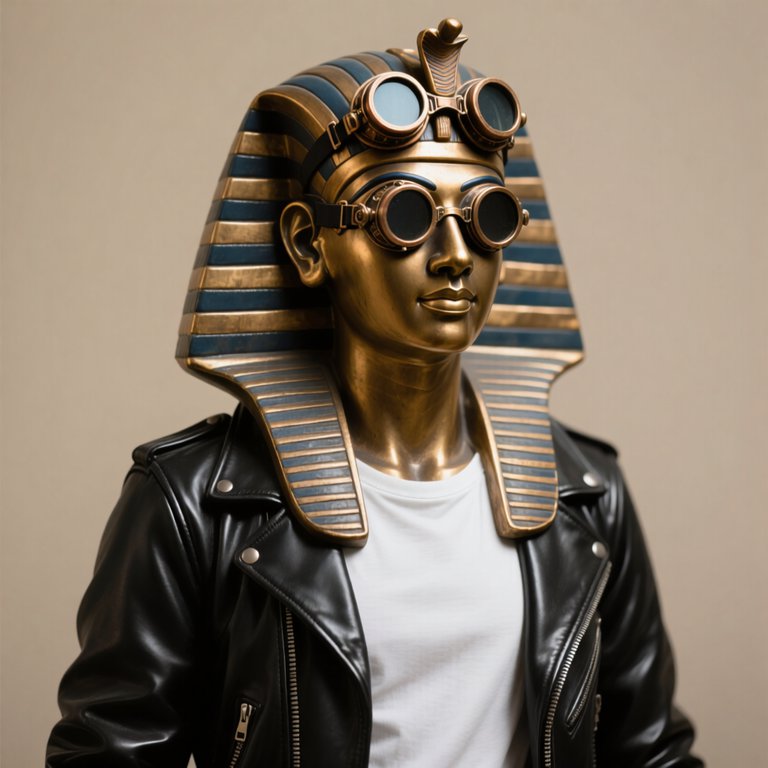}
      & \phantom{Teacher} \\ \addlinespace[2pt]
    \textbf{PDD}
      & \framecell{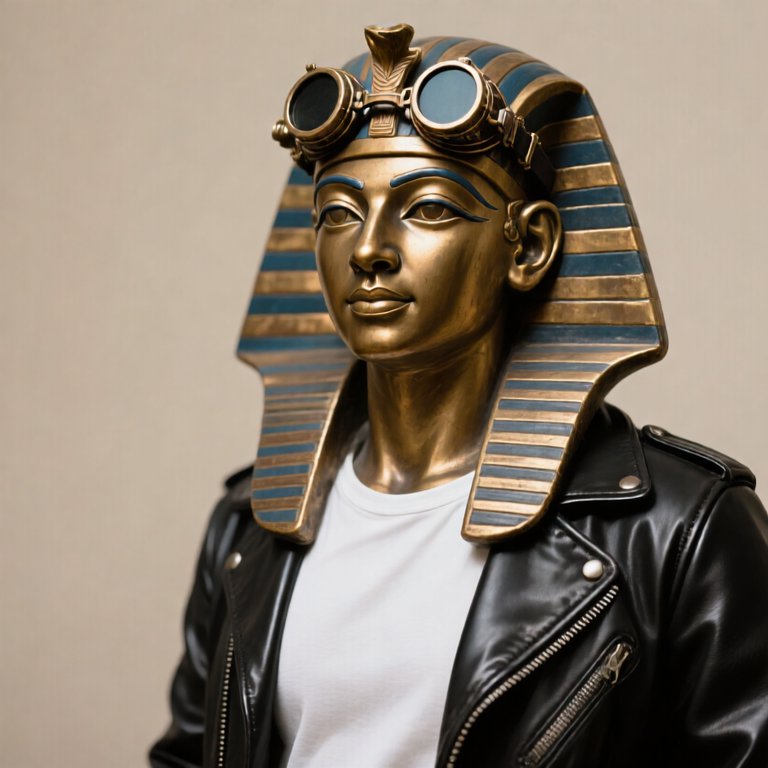}
      & \framecell{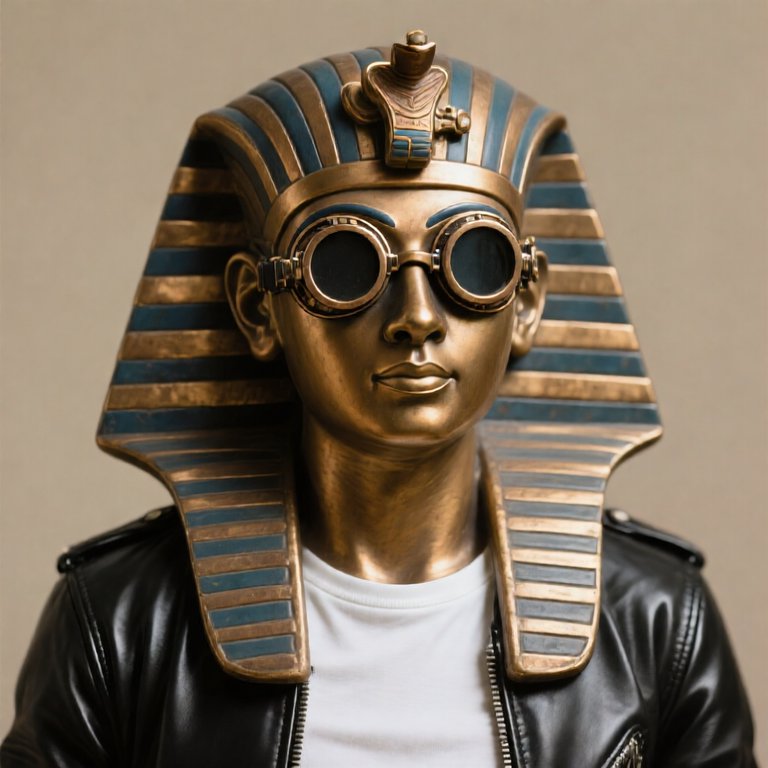}
      & \framecell{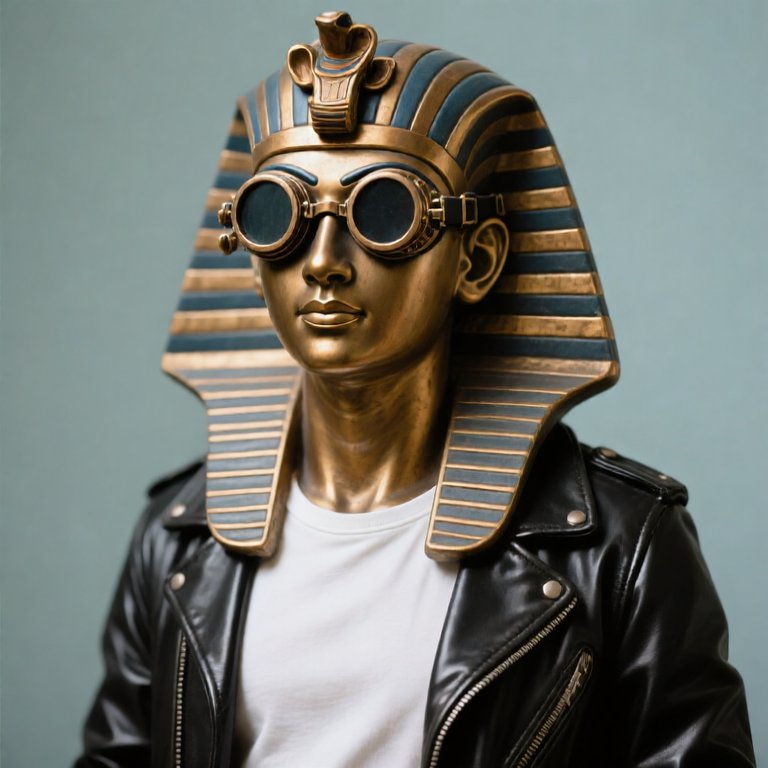}
      & \framecell{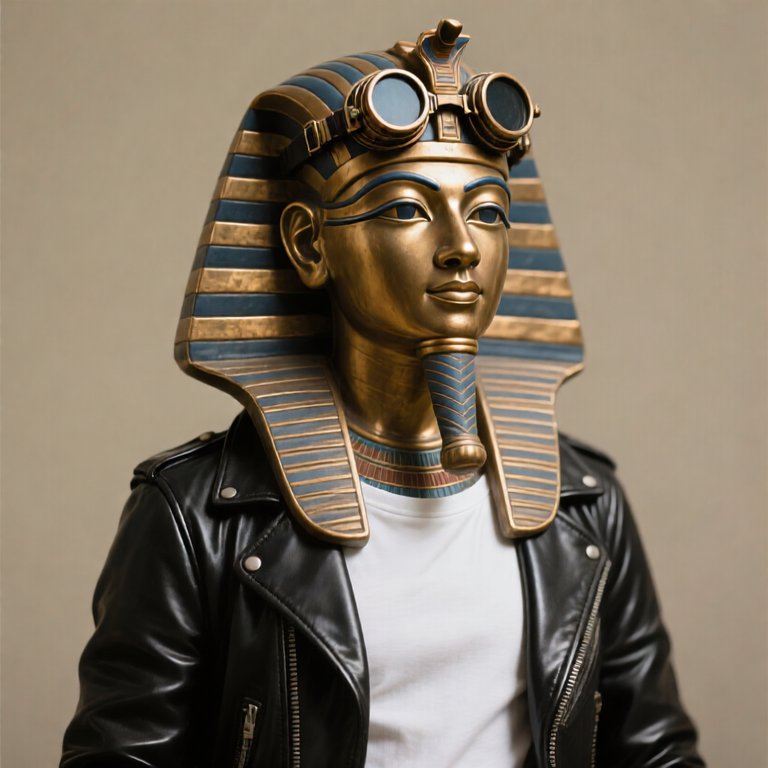}
      & \phantom{Teacher} \\ \addlinespace[2pt]
    DMD2
      & \framecell{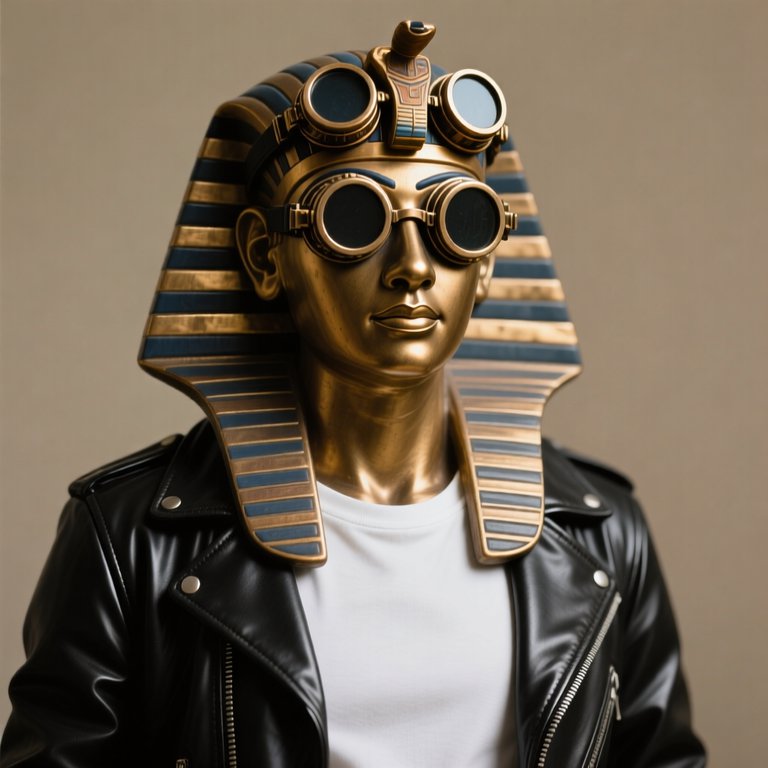}
      & \framecell{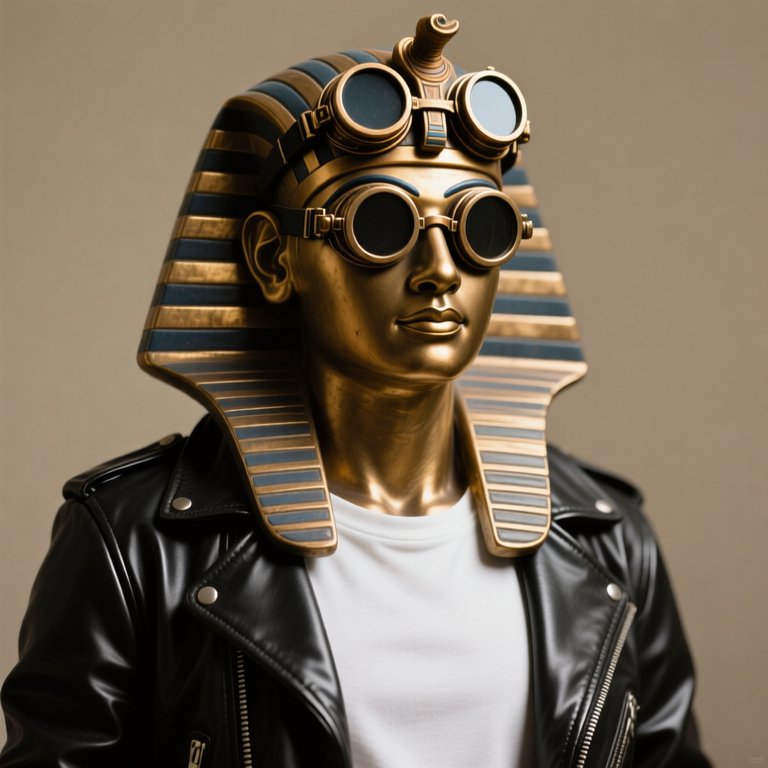}
      & \framecell{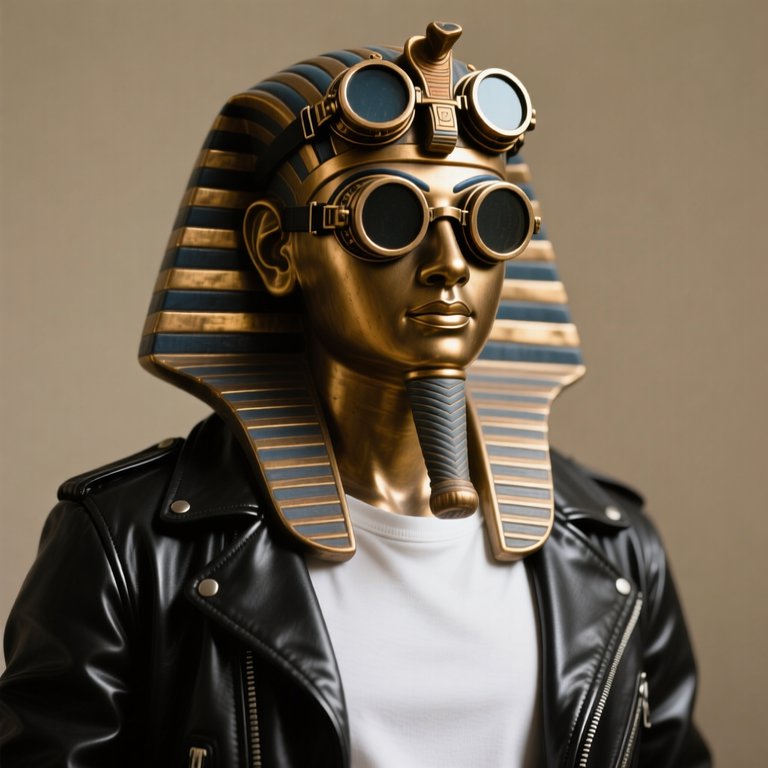}
      & \framecell{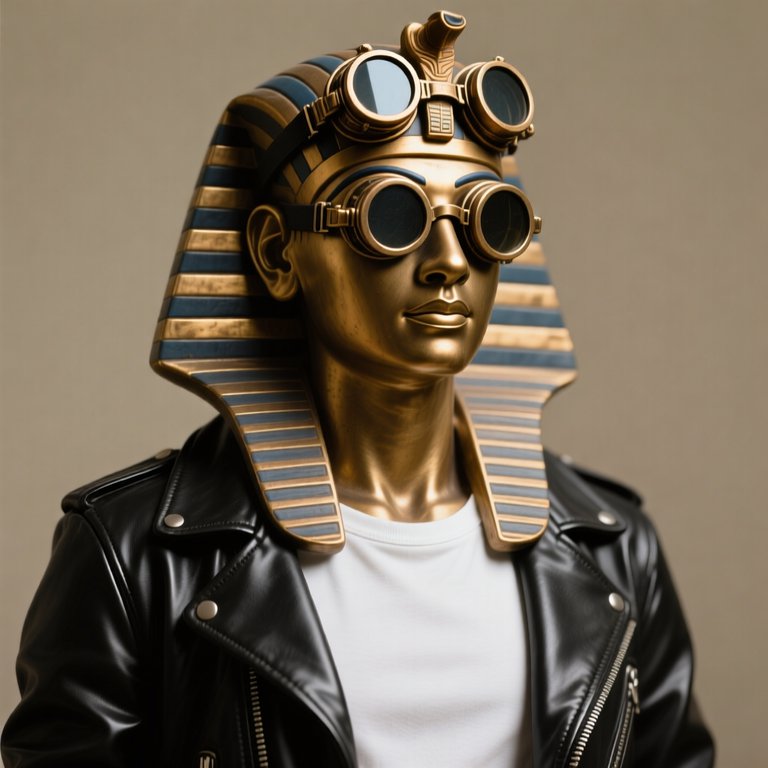}
      & \phantom{Teacher} \\ \addlinespace[2pt]
      & \multicolumn{4}{@{} p{10.8cm} @{}}{\tiny\itshape ``\input{figures/qwenimage/diversity_tables/706/prompt.txt}''} & \\
      \vspace{-15pt}
    \end{tabular}%
    }
    \caption{Teacher uses Euler method with 2$\times$50 NFE, PDD and DMD2 (Lightning-v2) use 4 NFE.}
  \end{figure}

%% file: figures/qwenimage/diversity_tables/706/prompt.txt
A detailed photograph captures the image of a statue with the likeness of an ancient pharaoh, unexpectedly accessorized with a pair of bronze steampunk goggles resting atop its head. The statue is dressed in an anachronistic fashion, featuring a crisp white t-shirt and a fitted black leather jacket that contrasts with its traditional headdress. The background is a simple, solid color that accentuates the statue's unconventional attire and the intricate details of the steampunk eyewear.

%% file: figures/qwenimage/diversity_tables/716/table.tex
\begin{figure}[t]
  \centering
  \resizebox{0.95\textwidth}{!}{%

  \begin{tabular}{@{} r @{\hspace{6pt}} c@{\hspace{2pt}}c@{\hspace{2pt}} c@{\hspace{2pt}}c @{\hspace{6pt}} c @{}}
    & {\footnotesize Noise 1} & {\footnotesize Noise 2} &  {\footnotesize Noise 3} & {\footnotesize Noise 4} & \\
    Teacher
      & \framecell{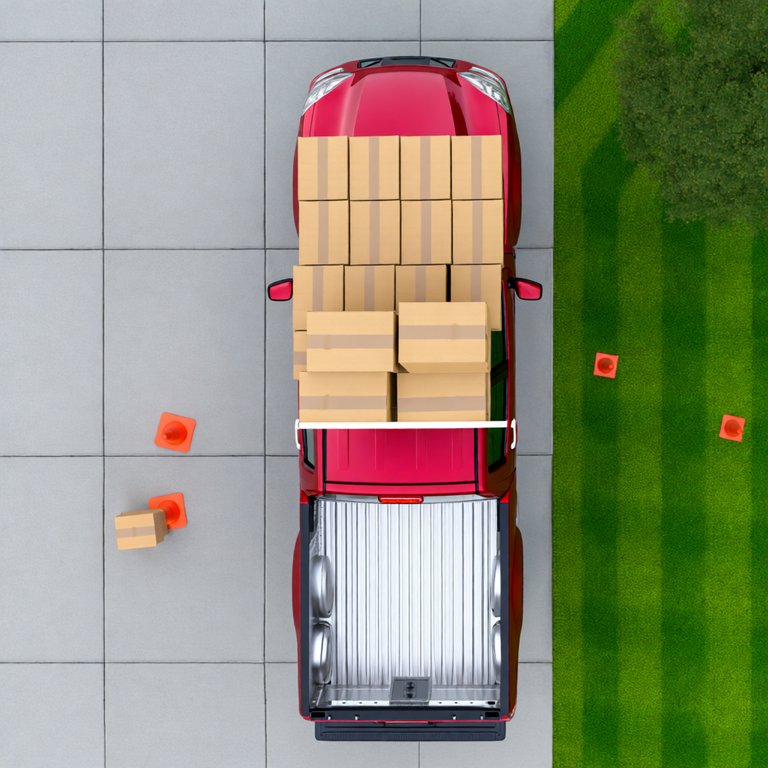}
      & \framecell{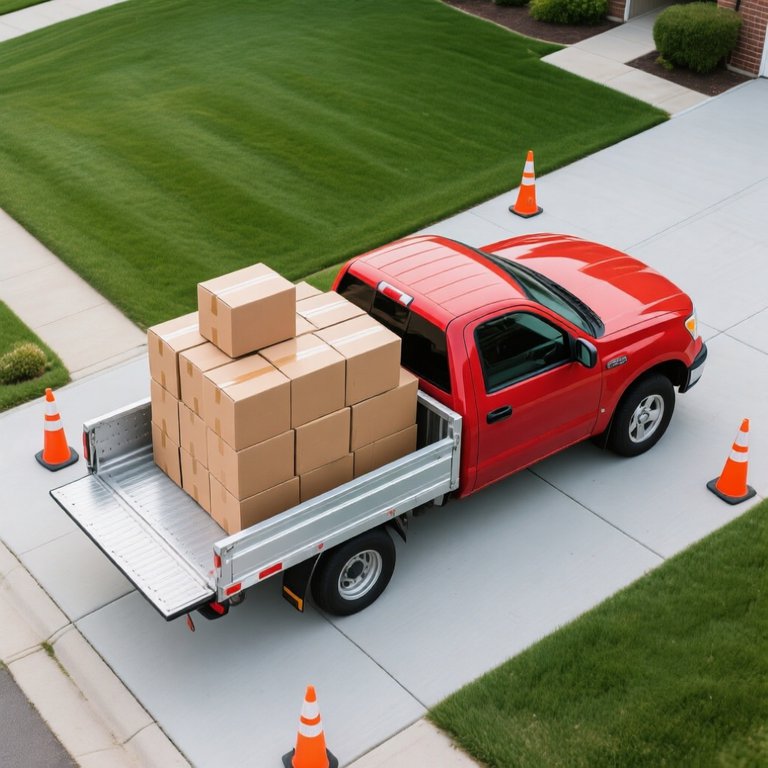}
      & \framecell{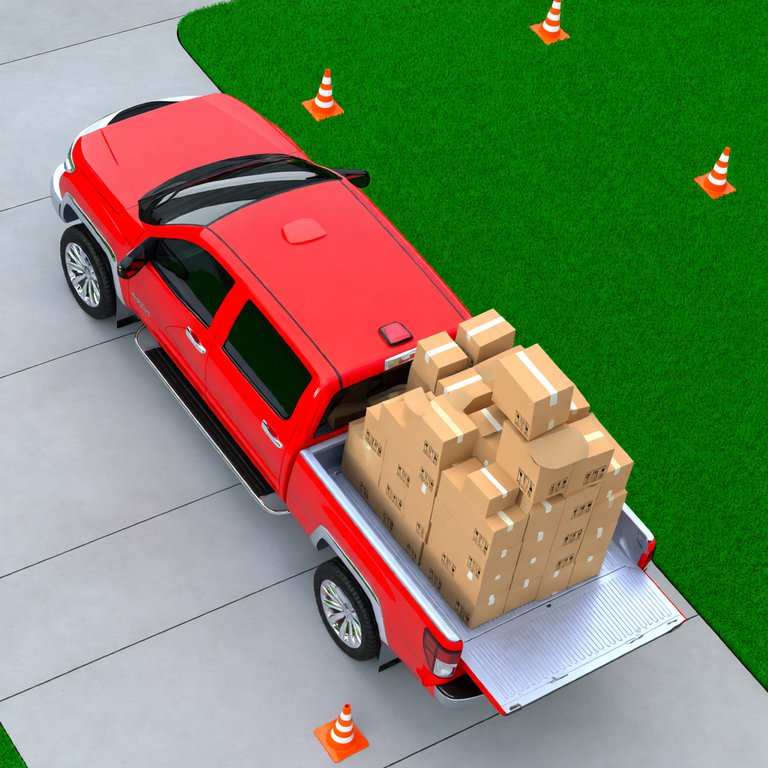}
      & \framecell{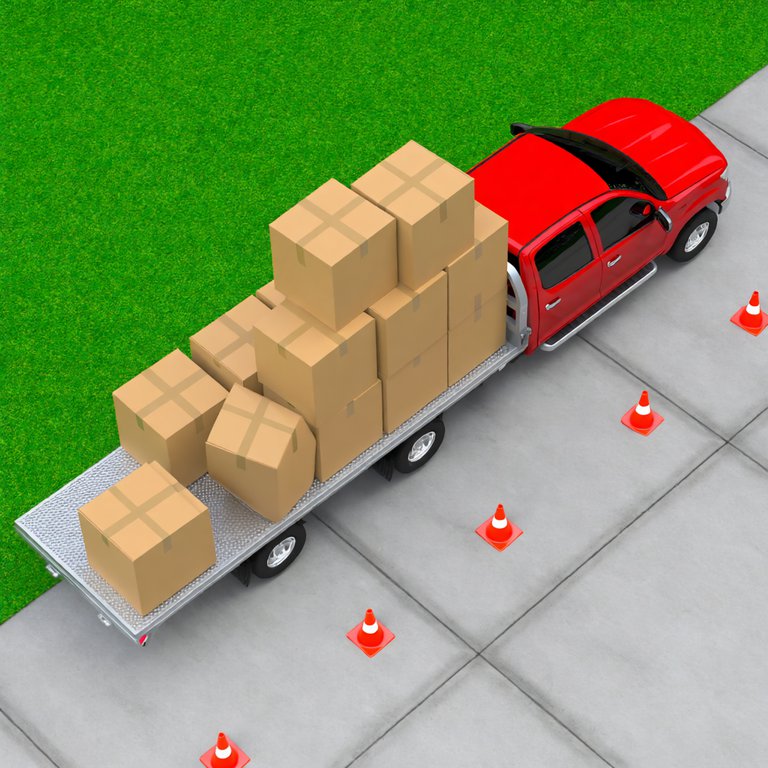}
      & \phantom{Teacher} \\ \addlinespace[2pt]
    \textbf{PDD}
      & \framecell{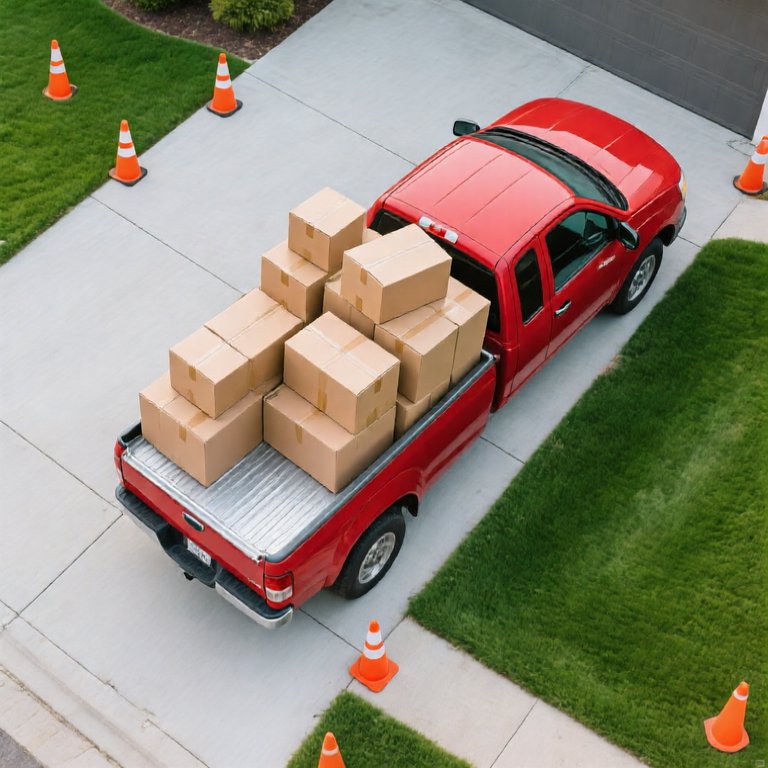}
      & \framecell{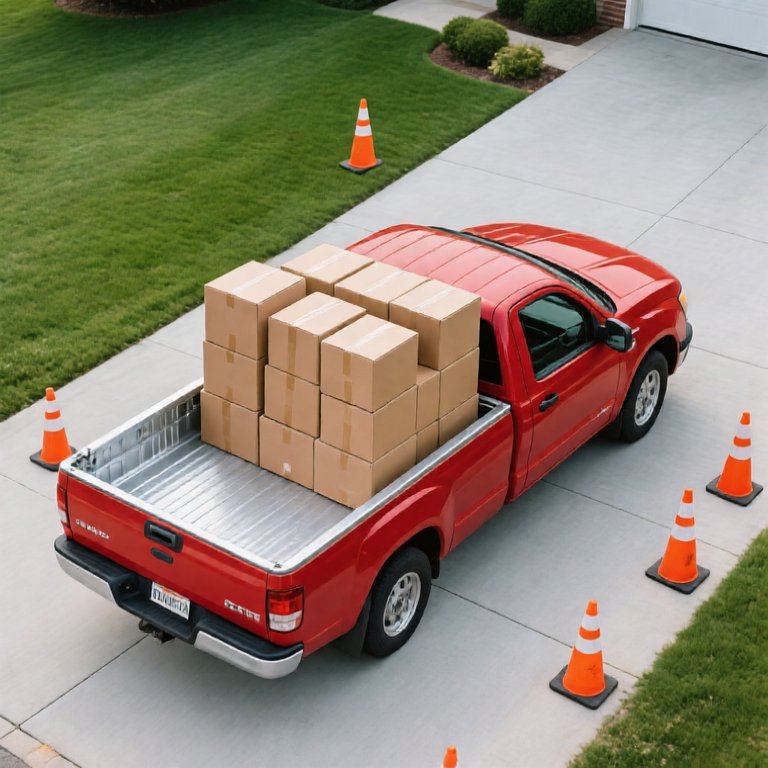}
      & \framecell{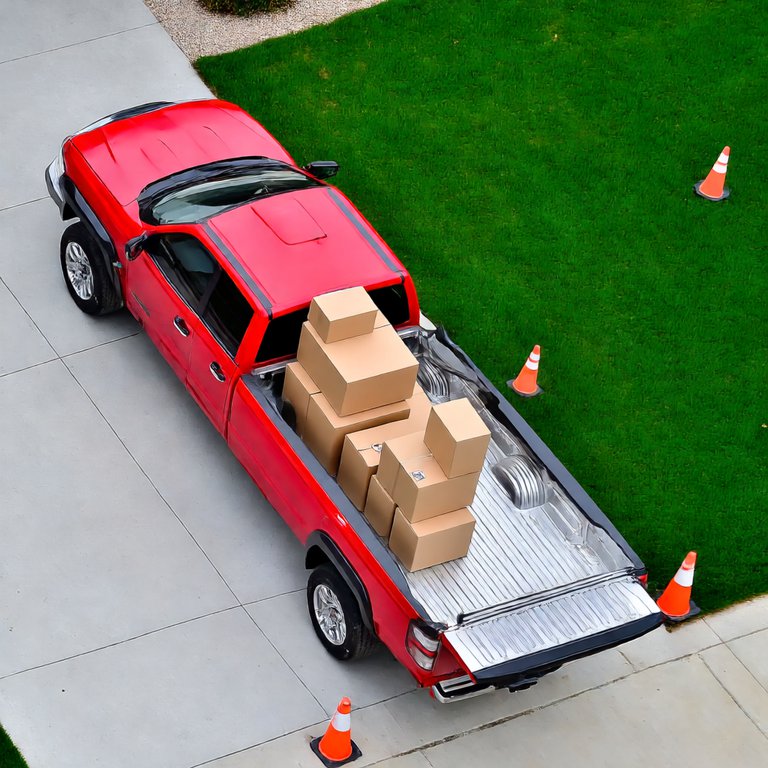}
      & \framecell{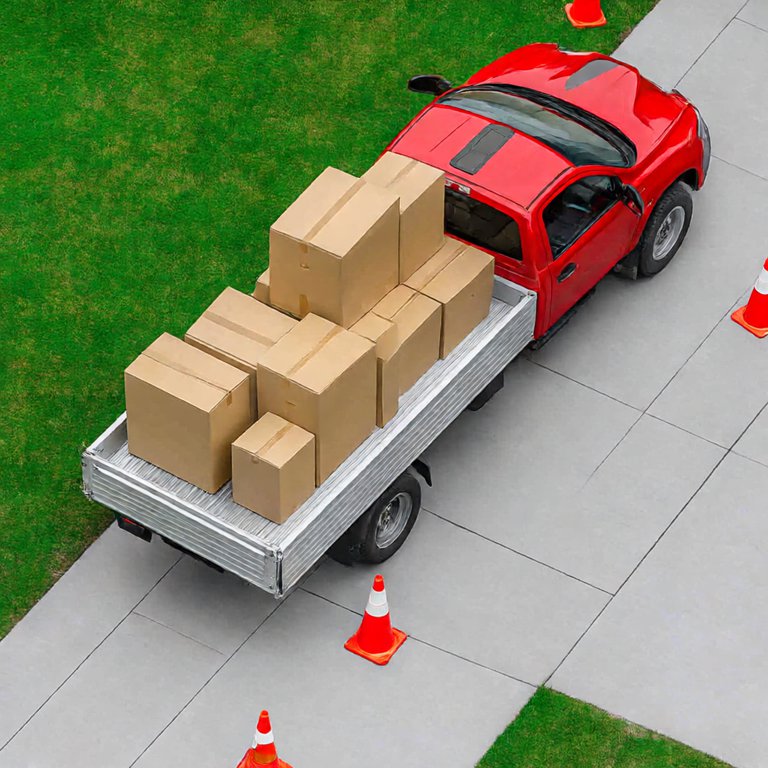}
      & \phantom{Teacher} \\ \addlinespace[2pt]
    DMD2
      & \framecell{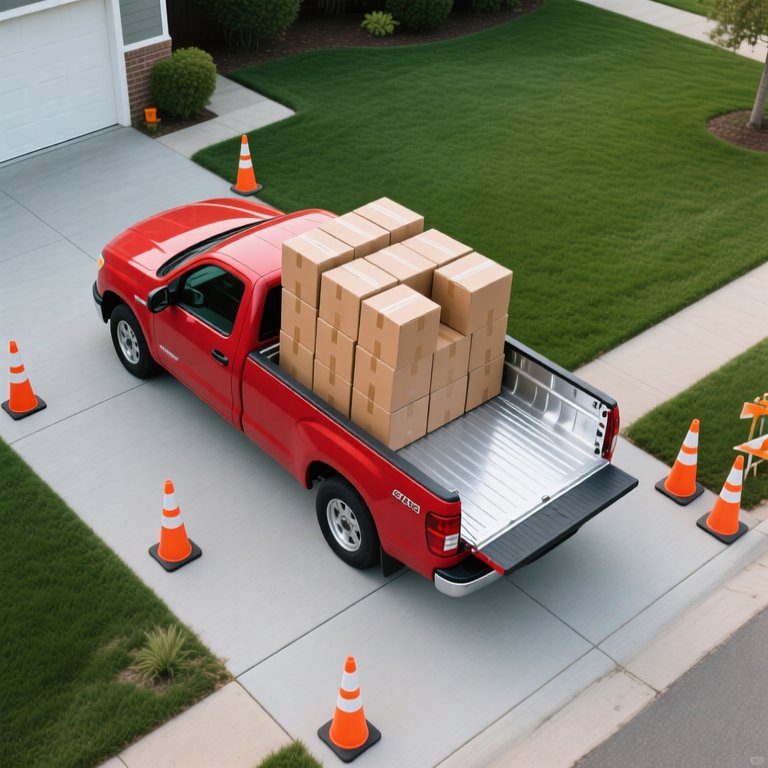}
      & \framecell{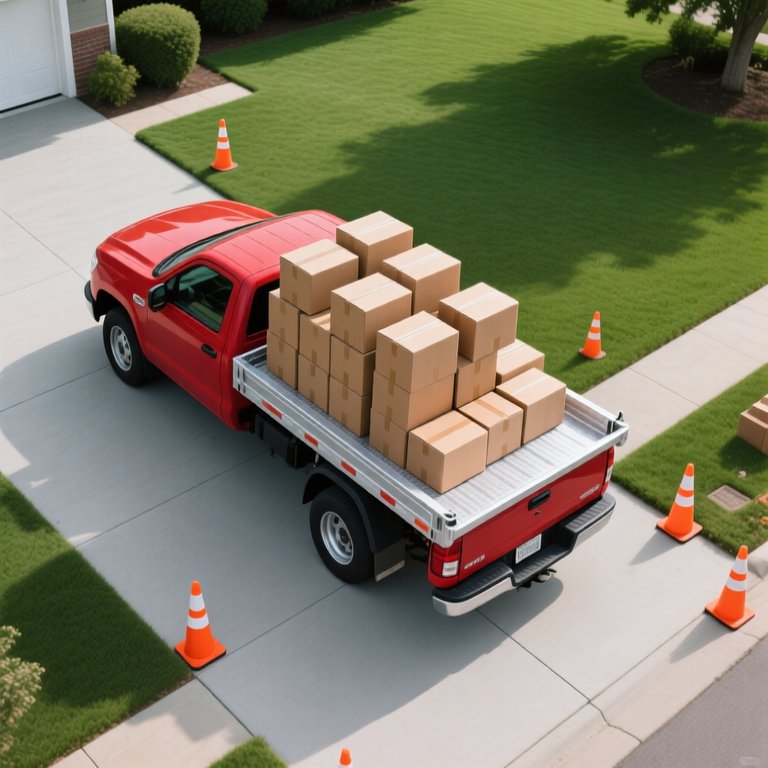}
      & \framecell{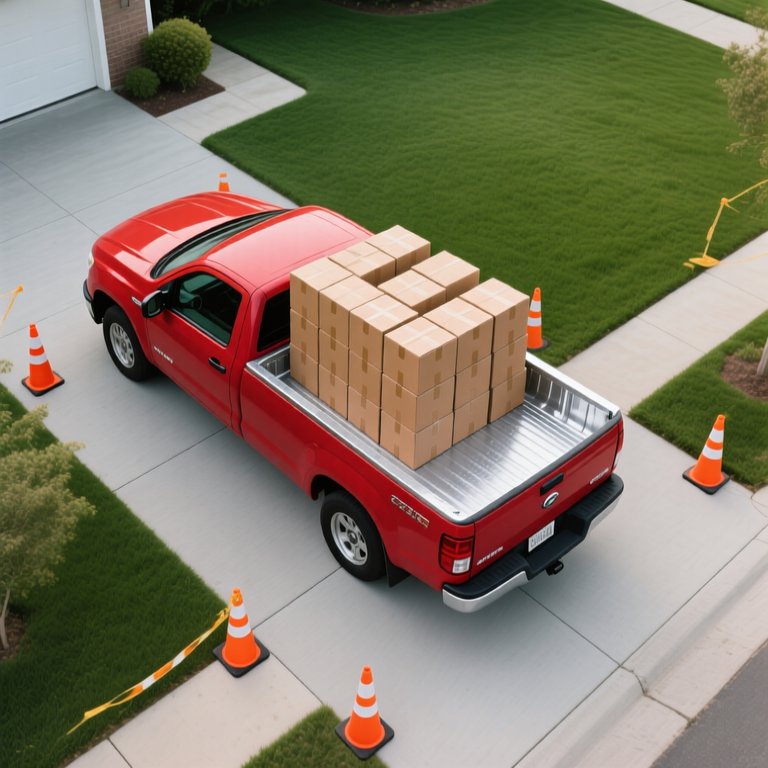}
      & \framecell{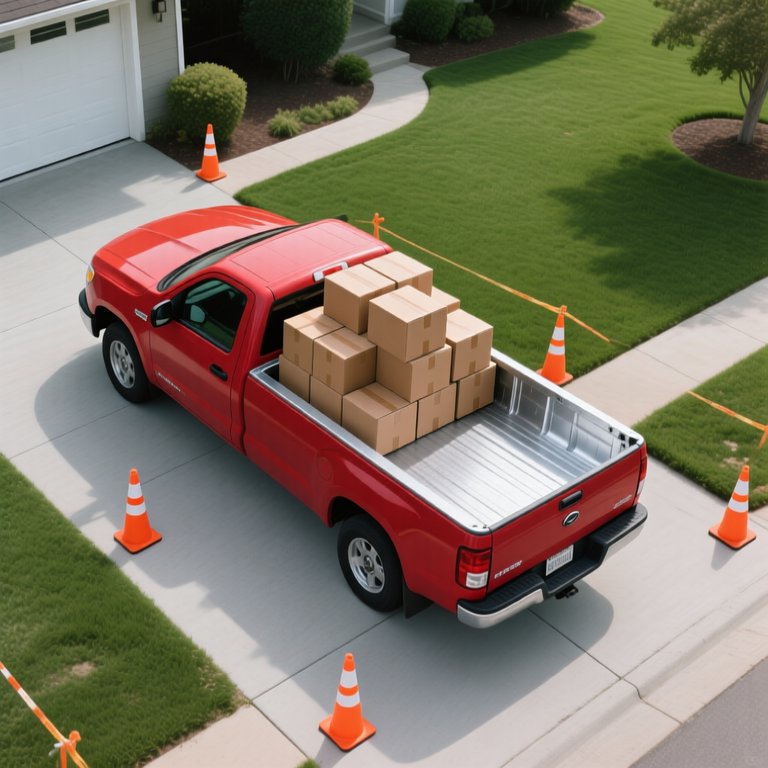}
      & \phantom{Teacher} \\ \addlinespace[2pt]
      & \multicolumn{4}{@{} p{10.8cm} @{}}{\tiny\itshape ``\input{figures/qwenimage/diversity_tables/716/prompt.txt}''} & \\
      \vspace{-20pt}
    \end{tabular}%
    }
  \end{figure}

%% file: figures/qwenimage/diversity_tables/716/prompt.txt
An aerial perspective showcases a red pickup truck with a silver flatbed, loaded with neatly stacked brown cardboard boxes. The truck is parked on a gray concrete driveway, adjacent to a well-manicured green lawn. Surrounding the vehicle, a few orange traffic cones are placed, indicating a temporary work zone or moving area.

%% file: figures/qwenimage/diversity_tables/805/table.tex
\begin{figure}[t]
  \centering
  \resizebox{0.95\textwidth}{!}{%

  \begin{tabular}{@{} r @{\hspace{6pt}} c@{\hspace{2pt}}c@{\hspace{2pt}} c@{\hspace{2pt}}c @{\hspace{6pt}} c @{}}
    & {\footnotesize Noise 1} & {\footnotesize Noise 2} &  {\footnotesize Noise 3} & {\footnotesize Noise 4} & \\
    Teacher
      & \framecell{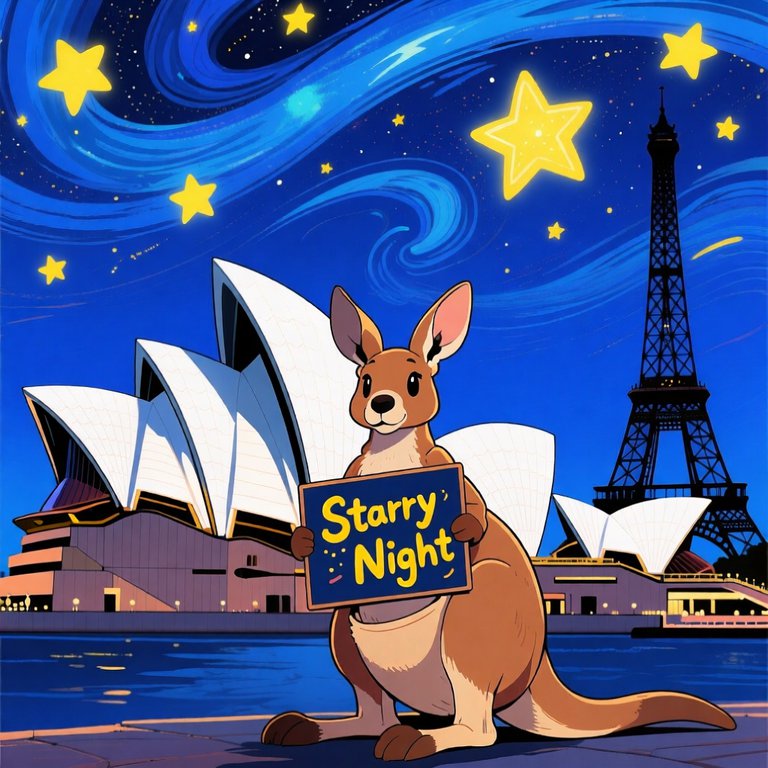}
      & \framecell{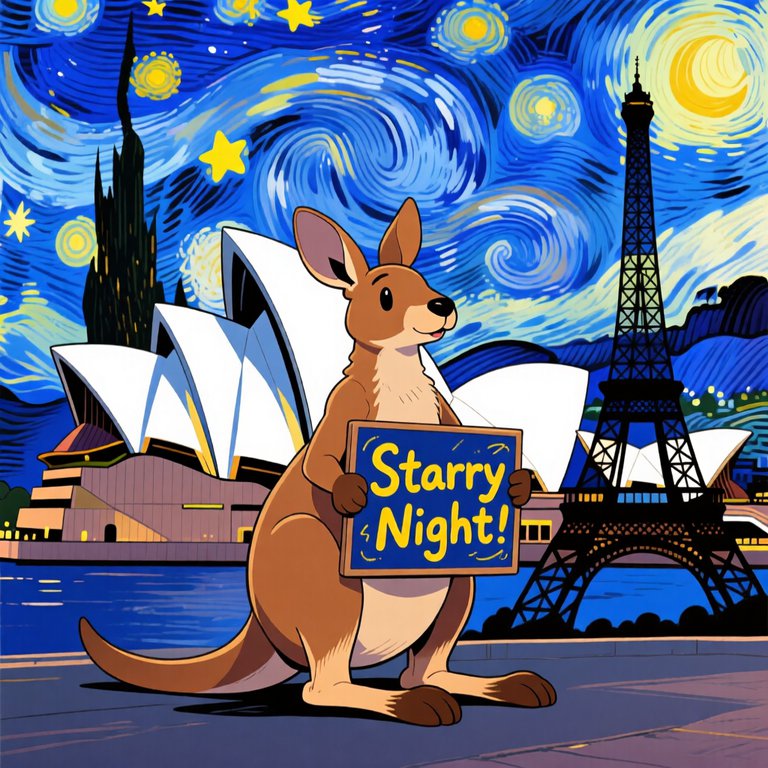}
      & \framecell{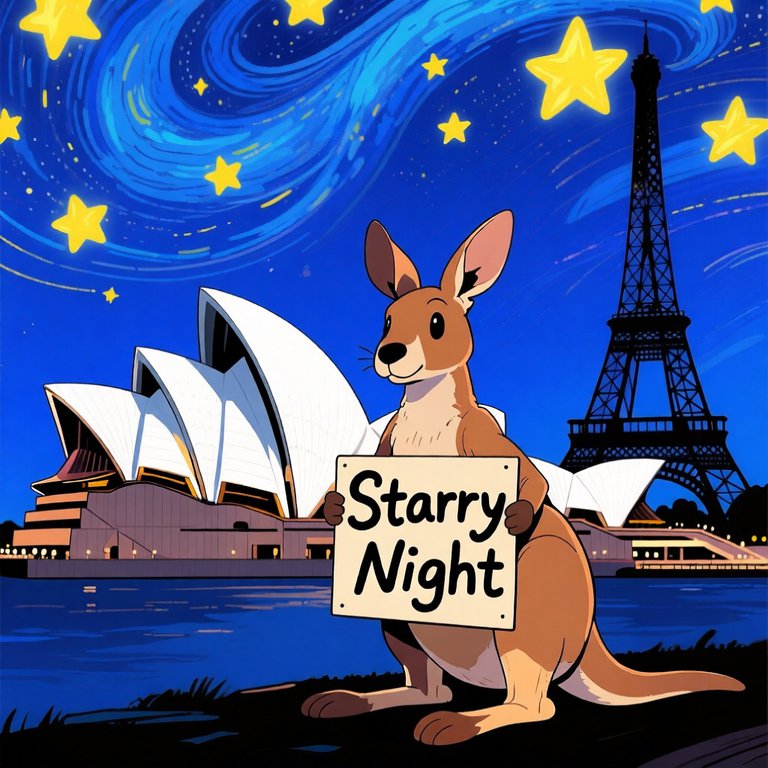}
      & \framecell{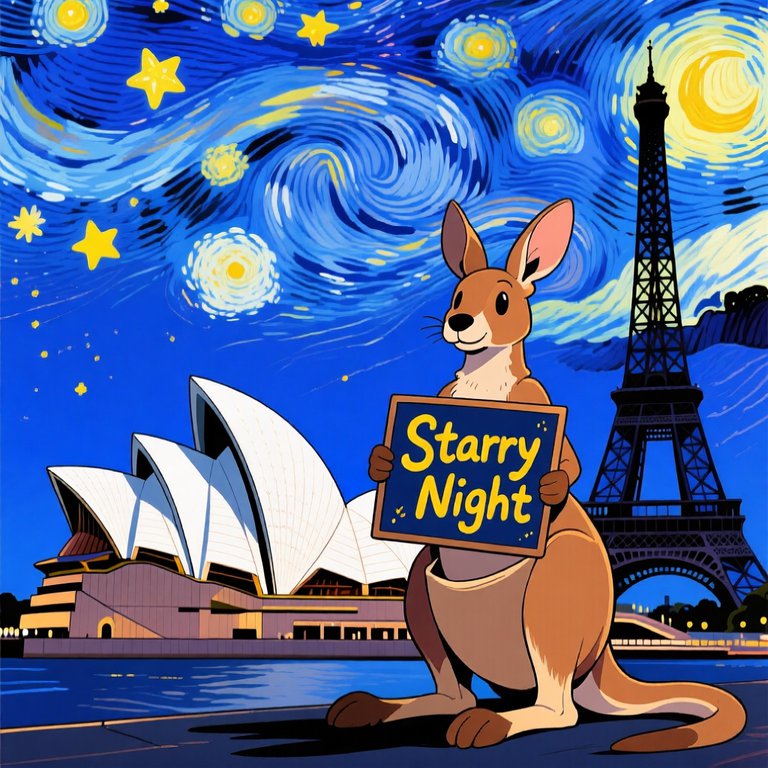}
      & \phantom{Teacher} \\ \addlinespace[2pt]
    \textbf{PDD}
      & \framecell{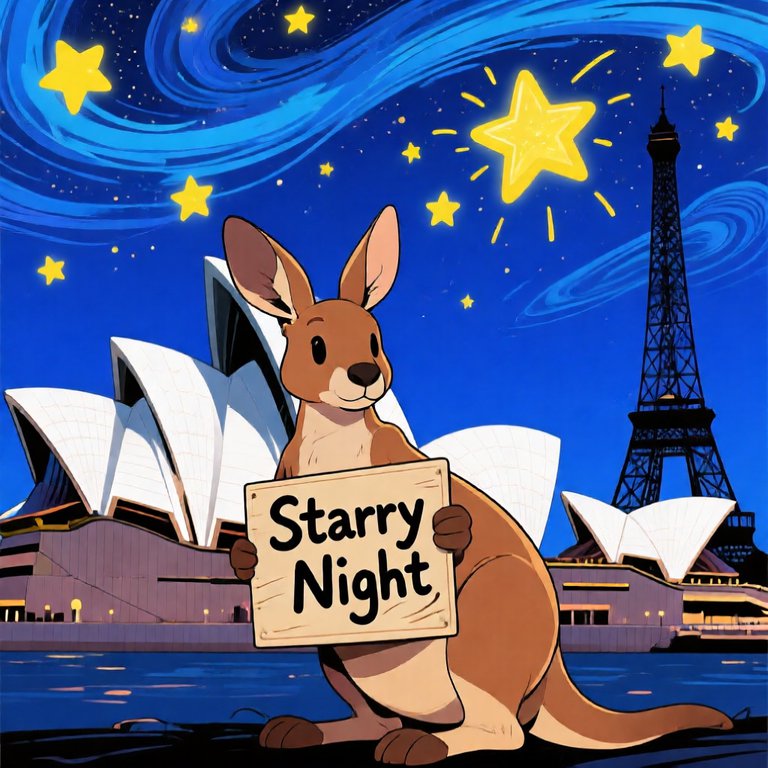}
      & \framecell{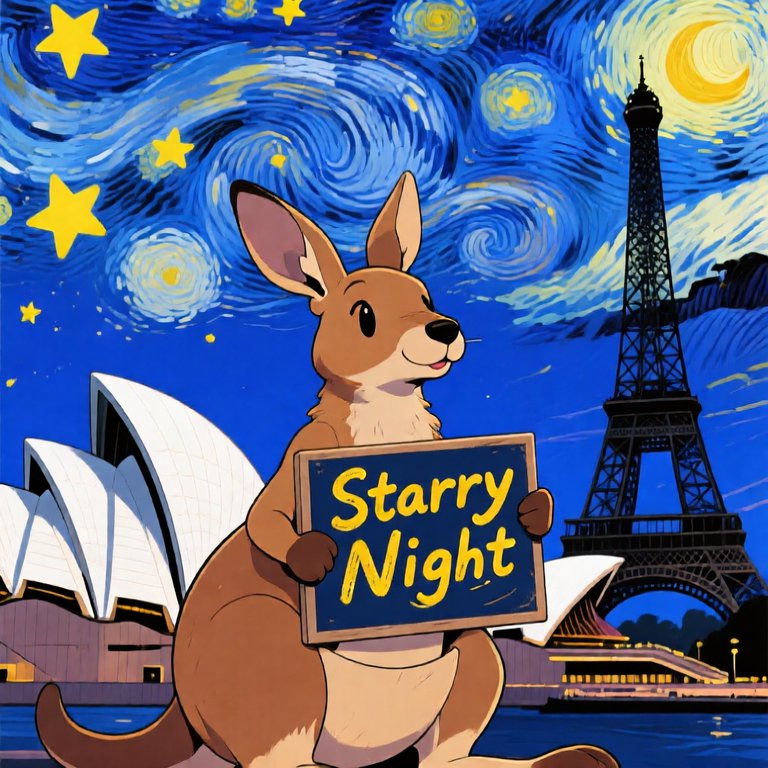}
      & \framecell{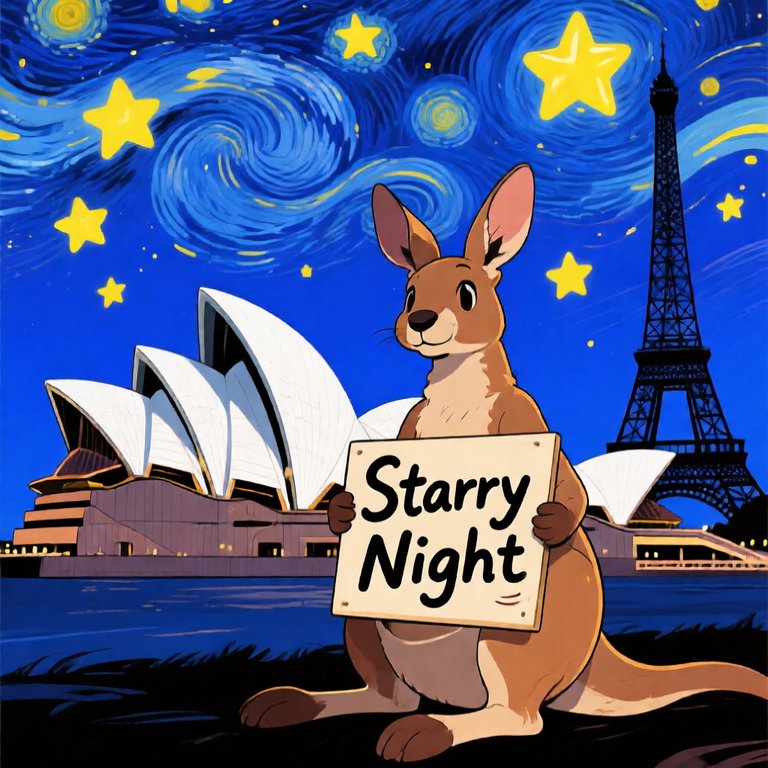}
      & \framecell{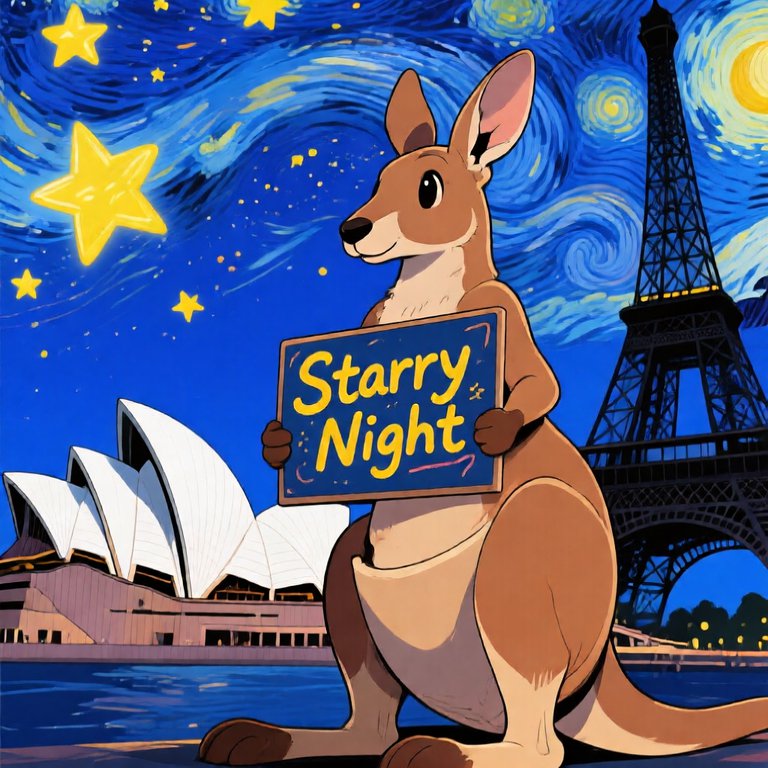}
      & \phantom{Teacher} \\ \addlinespace[2pt]
    DMD2
      & \framecell{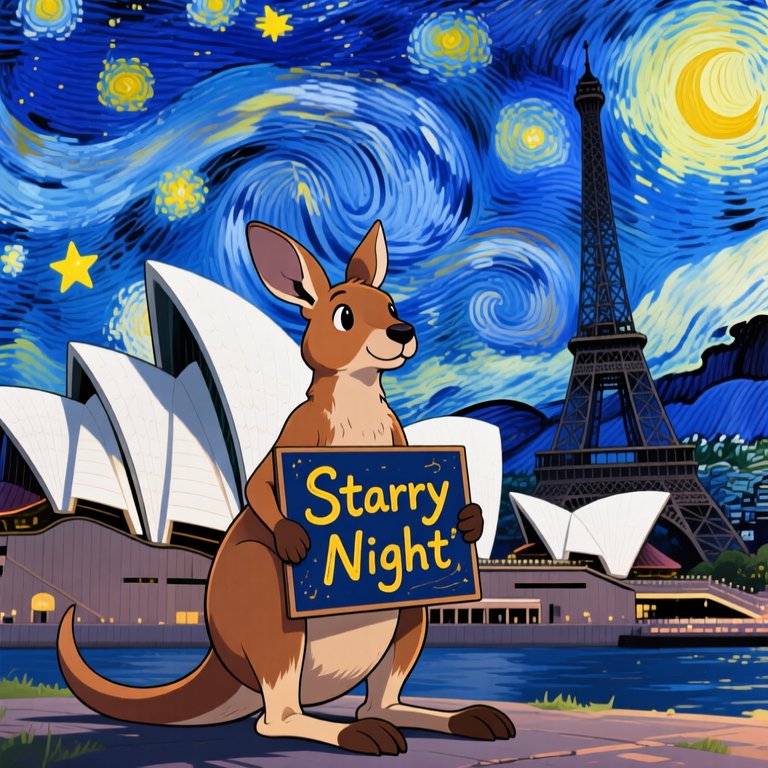}
      & \framecell{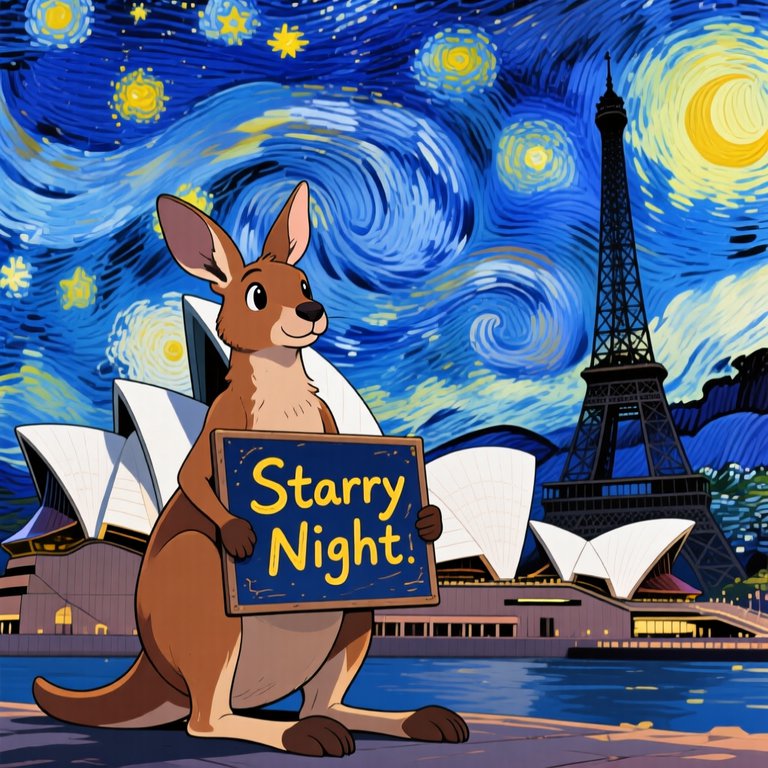}
      & \framecell{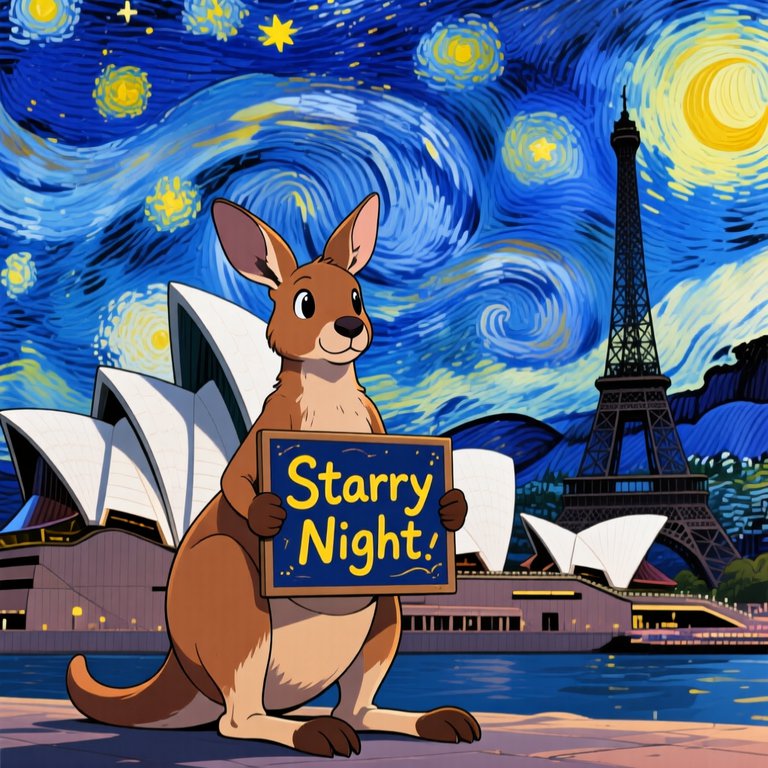}
      & \framecell{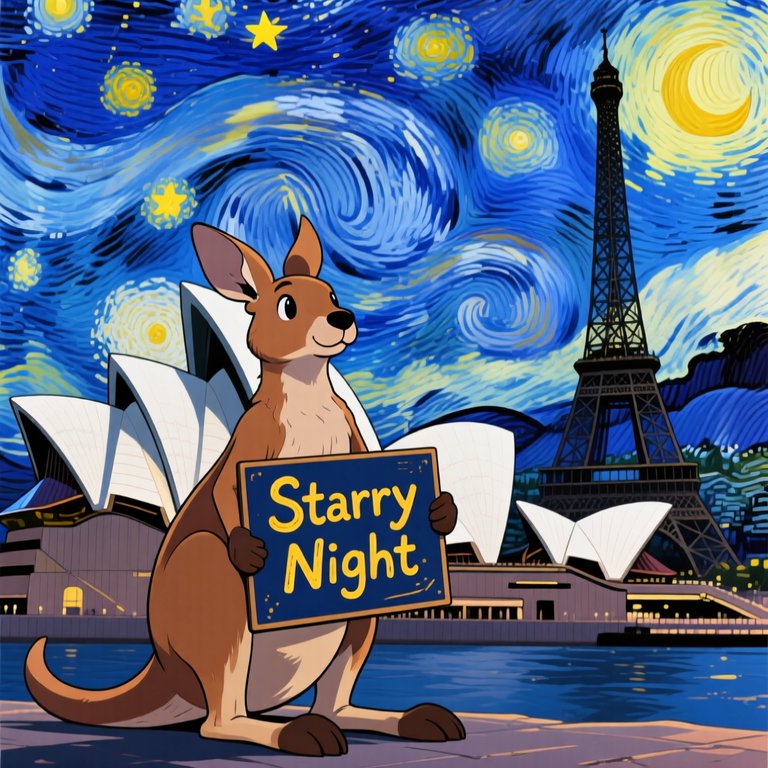}
      & \phantom{Teacher} \\ \addlinespace[2pt]
      & \multicolumn{4}{@{} p{10.8cm} @{}}{\tiny\itshape ``\input{figures/qwenimage/diversity_tables/805/prompt.txt}''} & \\
      \vspace{-15pt}
    \end{tabular}%
    }
    \caption{Teacher uses Euler method with 2$\times$50 NFE, PDD and DMD2 (Lightning-v2) use 4 NFE.}
    \label{fig:qwenimage_diversity_last}
  \end{figure}

%% file: figures/qwenimage/diversity_tables/805/prompt.txt
An anime-style illustration depicts a whimsical scene where a kangaroo, rendered in vibrant shades of brown and tan, clutches a rectangular sign with the words "Starry Night" scrawled across it in bold, whimsical lettering. The kangaroo is seated comfortably in front of the iconic Sydney Opera House, its distinctive white sail-like structures contrasting sharply with the deep blue of the night sky. To the kangaroo's side, the Eiffel Tower looms, its iron lattice silhouette adding to the surreal composition. Above, the sky is alive with dynamic swirls of blue and bursts of radiant yellow stars, capturing the essence of a dreamlike cosmic event.

%% file: figures/wan/diversity_tables/purple_robe_rabbit_fantasy/frames_table.tex
\begin{figure}[t]
  \centering

  \resizebox{\textwidth}{!}{%

  \begin{tabular}{@{} r @{\hspace{6pt}} c@{\hspace{2pt}}c@{\hspace{2pt}}c @{\hspace{4pt}} c@{\hspace{2pt}}c@{\hspace{2pt}}c @{\hspace{6pt}} c @{}}
                & \multicolumn{3}{c}{\large Noise 1}
                & \multicolumn{3}{c}{\large Noise 2} & \\

                \cmidrule(l{6pt}r{2pt}){2-4} \cmidrule(l{2pt}r{6pt}){5-7}

                & $t{=}0.5$s & $t{=}2.5$s & $t{=}4.5$s & $t{=}0.5$s & $t{=}2.5$s & $t{=}4.5$s & \\
    \textbf{PDD}
      & \framecell{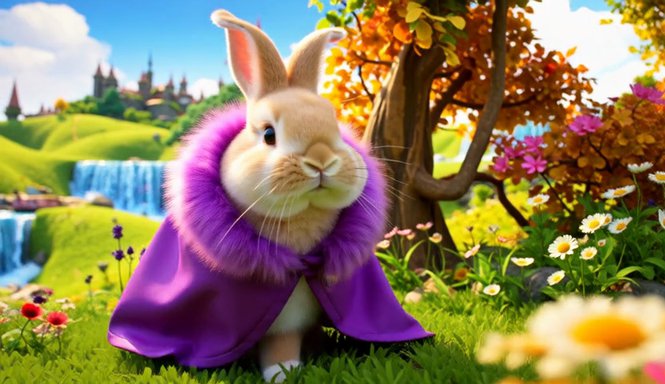}
      & \framecell{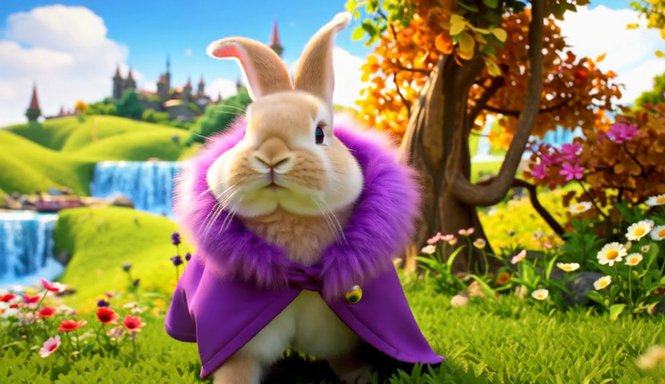}
      & \framecell{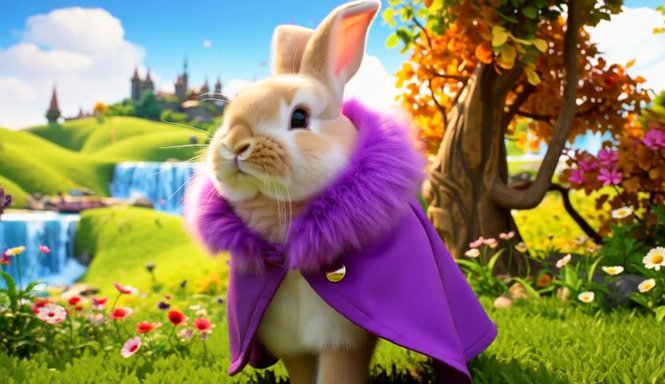}
      & \framecell{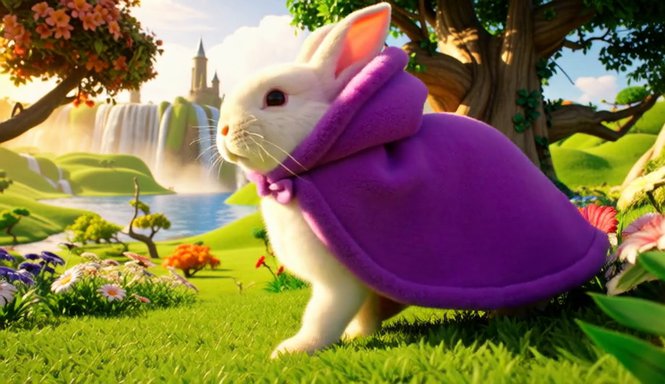}
      & \framecell{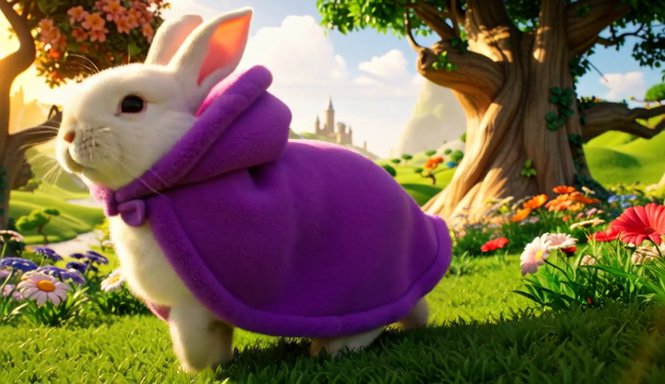}
      & \framecell{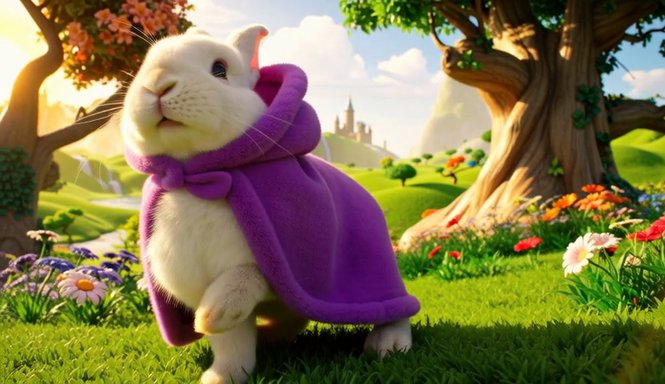} 
      & \phantom{AnyFlow} \\ \addlinespace[2pt]
    DMD2
      & \framecell{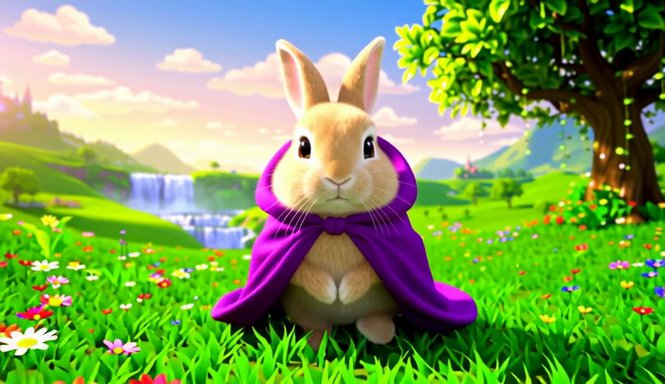}
      & \framecell{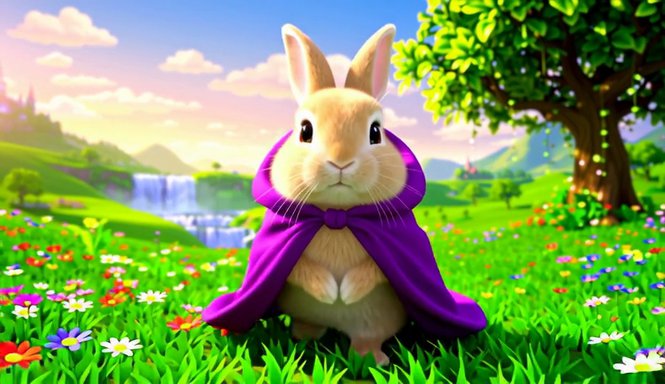}
      & \framecell{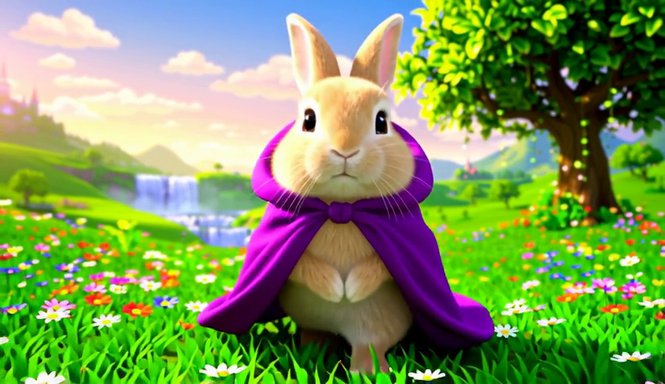}
      & \framecell{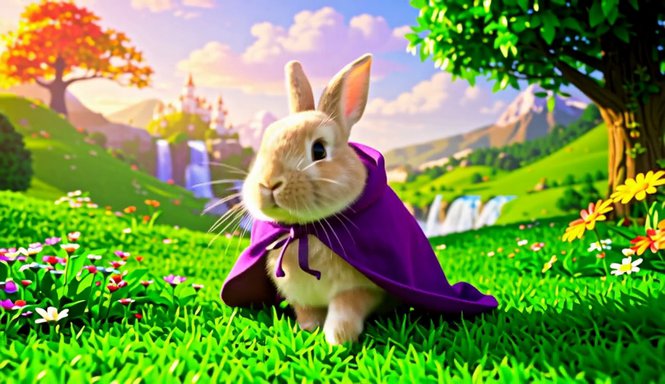}
      & \framecell{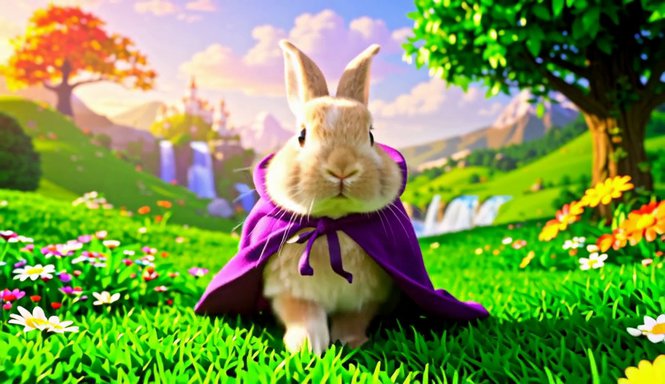}
      & \framecell{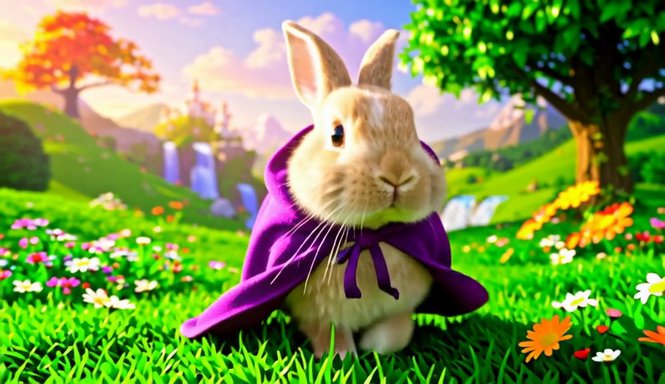} 
      & \phantom{AnyFlow} \\ \addlinespace[2pt]
    AnyFlow
      & \framecell{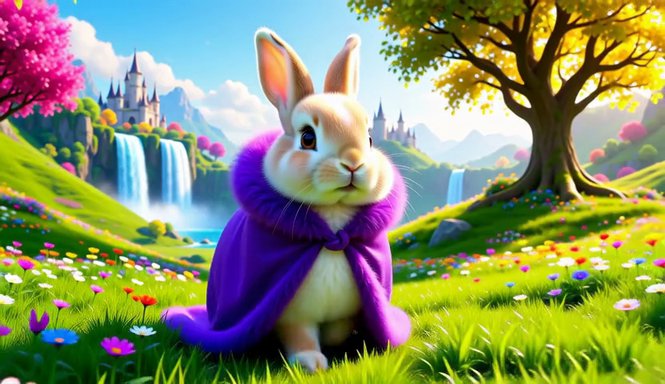}
      & \framecell{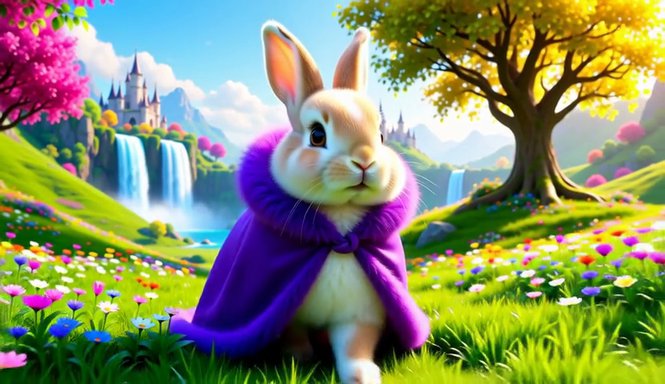}
      & \framecell{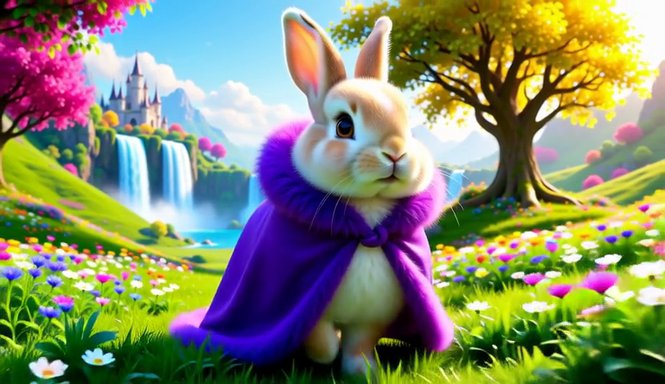}
      & \framecell{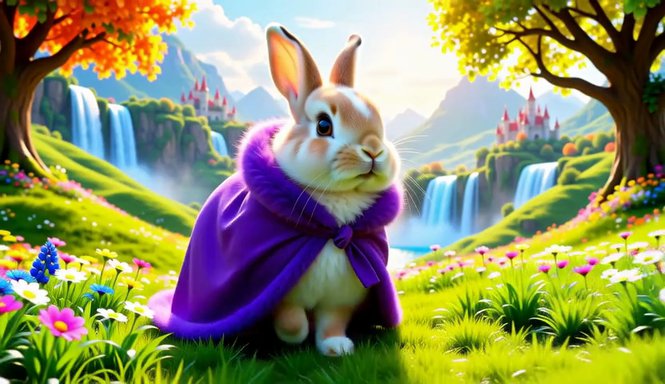}
      & \framecell{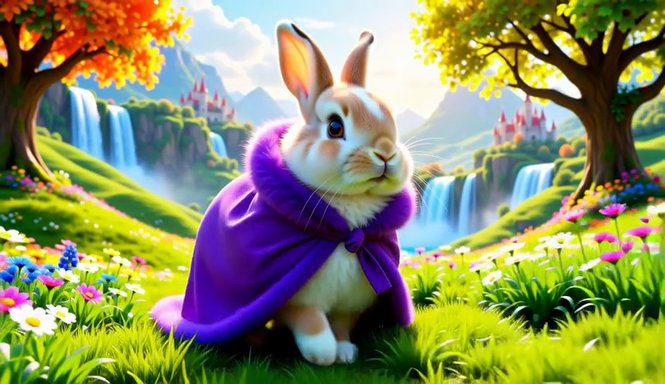}
      & \framecell{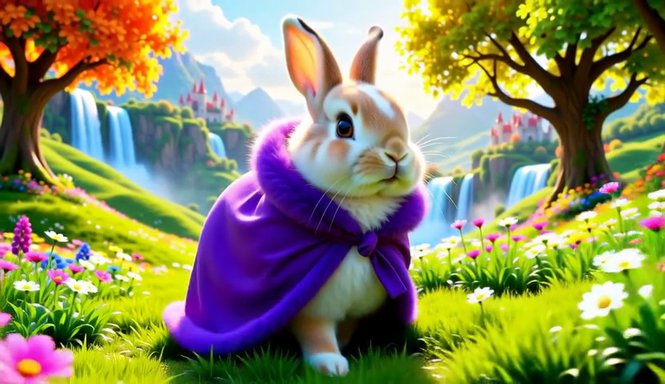} 
      & \phantom{AnyFlow} \\ \addlinespace[2pt]
      & \multicolumn{6}{@{} p{16.6cm} @{}}{\tiny\itshape ``A plump, fluffy rabbit donning a voluminous purple robe walks gracefully through a vibrant fantasy landscape. The rabbit has large, expressive eyes and a gentle, curious expression. Its fur is soft and thick, and the robe drapes elegantly over its body. The landscape features rolling hills covered in lush green grass, colorful wildflowers, and towering magical trees with shimmering leaves. In the distance, there are sparkling waterfalls and mystical castles. The scene is bathed in warm, golden sunlight. Medium shot, focusing on the rabbit's walk through the picturesque environment.''} & \\
  \end{tabular}%
  }

  \vspace{-7.5pt}
  \label{fig:purple_robe_rabbit_fantasy}
\end{figure}

%% file: figures/wan/diversity_tables/christmas_tree_fireworks_drone/frames_table.tex
\begin{figure}[t]
  \centering

  \resizebox{\textwidth}{!}{%

  \begin{tabular}{@{} r @{\hspace{6pt}} c@{\hspace{2pt}}c@{\hspace{2pt}}c @{\hspace{4pt}} c@{\hspace{2pt}}c@{\hspace{2pt}}c @{\hspace{6pt}} c @{}}
                & \multicolumn{3}{c}{\large Noise 1}
                & \multicolumn{3}{c}{\large Noise 2} & \\

                \cmidrule(l{6pt}r{2pt}){2-4} \cmidrule(l{2pt}r{6pt}){5-7}

                & $t{=}0.5$s & $t{=}2.5$s & $t{=}4.5$s & $t{=}0.5$s & $t{=}2.5$s & $t{=}4.5$s & \\
    \textbf{PDD}
      & \framecell{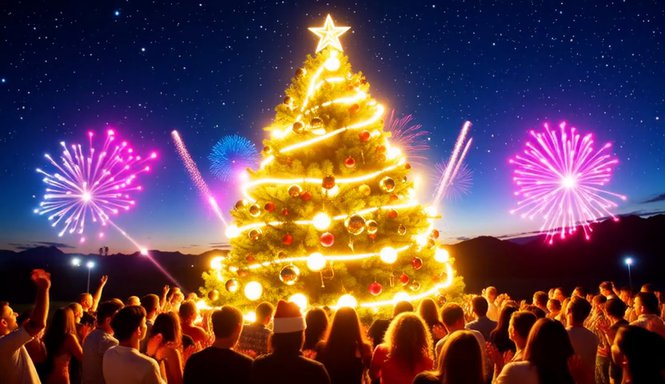}
      & \framecell{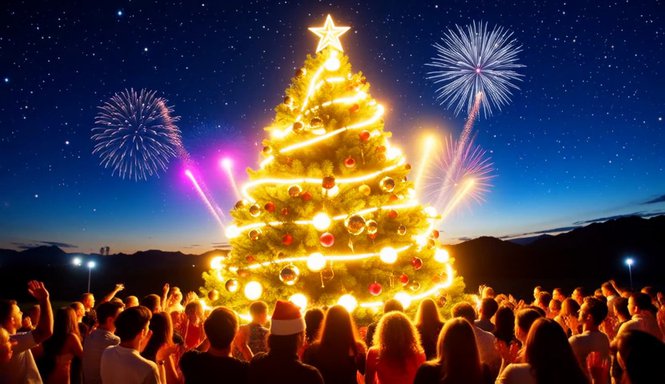}
      & \framecell{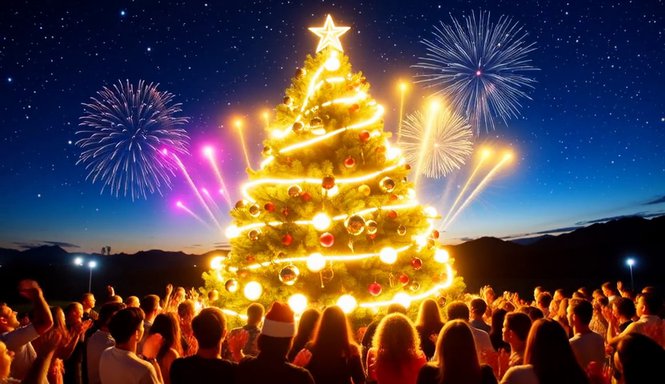}
      & \framecell{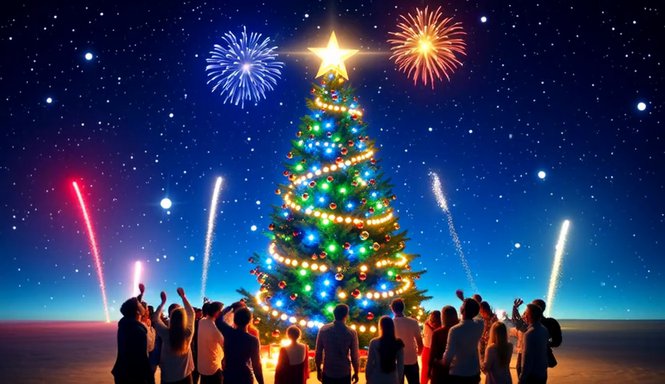}
      & \framecell{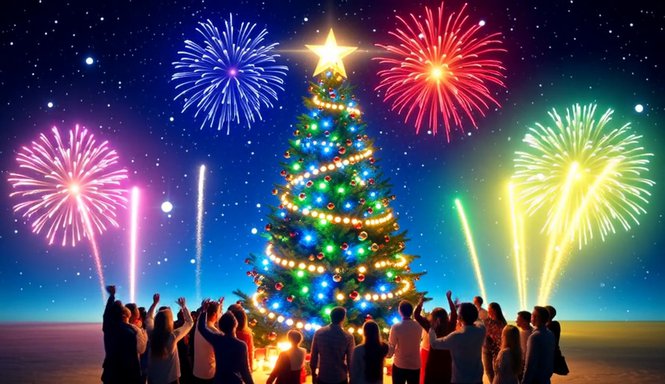}
      & \framecell{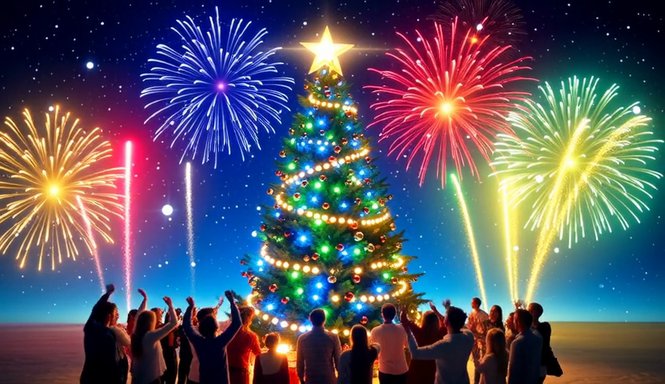} 
      & \phantom{AnyFlow} \\ \addlinespace[2pt]
    DMD2
      & \framecell{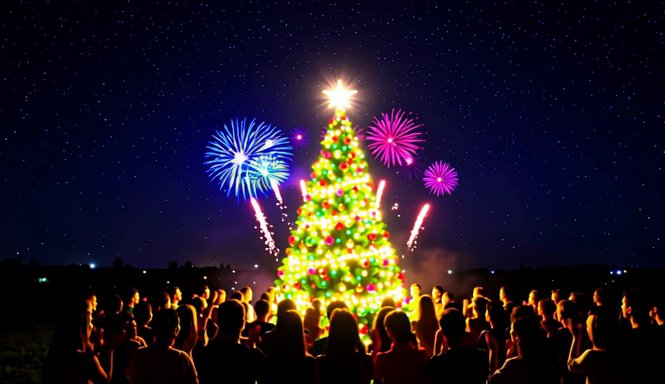}
      & \framecell{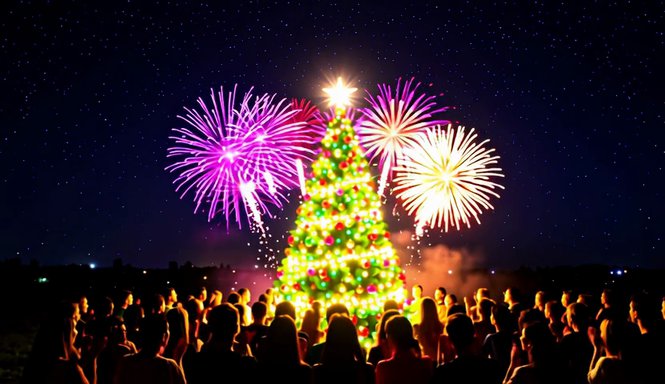}
      & \framecell{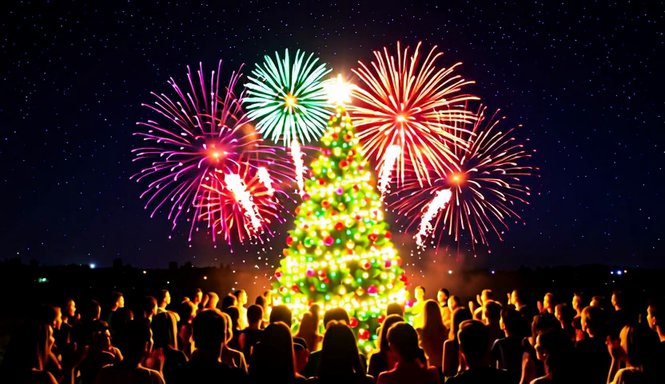}
      & \framecell{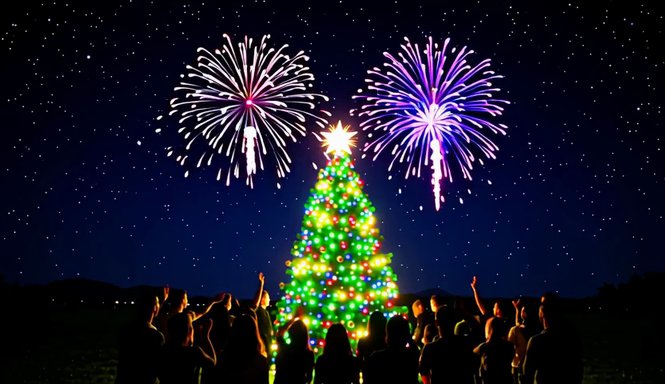}
      & \framecell{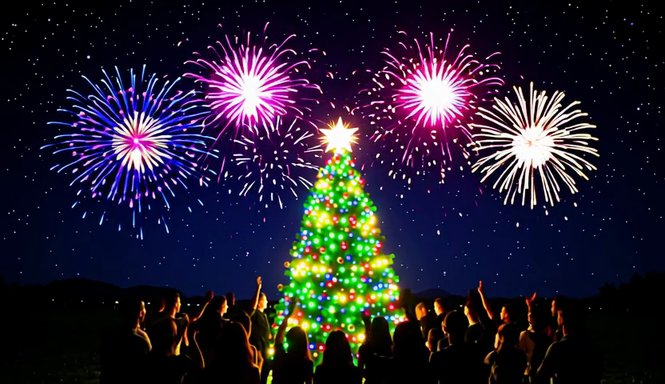}
      & \framecell{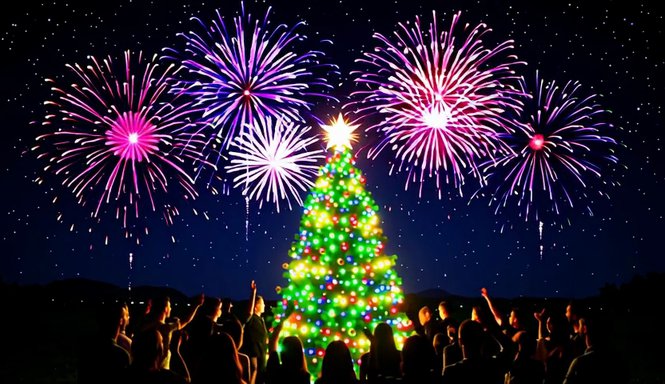} 
      & \phantom{AnyFlow} \\ \addlinespace[2pt]
    AnyFlow
      & \framecell{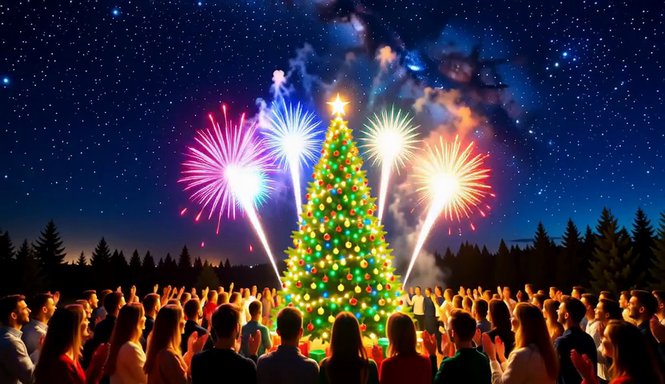}
      & \framecell{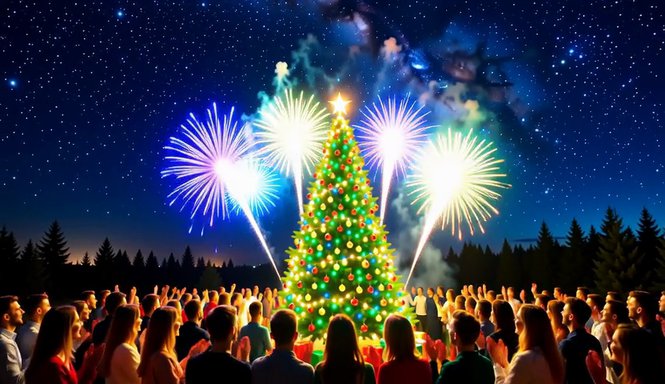}
      & \framecell{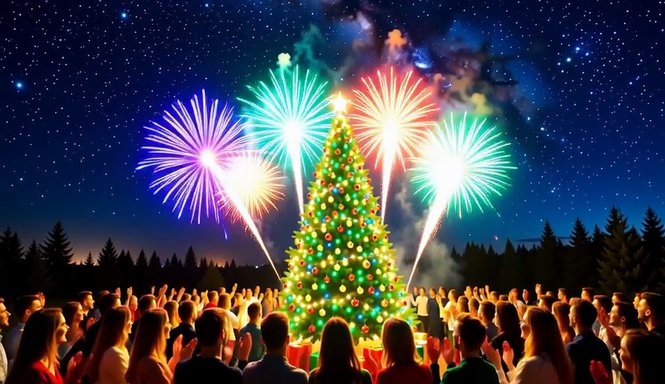}
      & \framecell{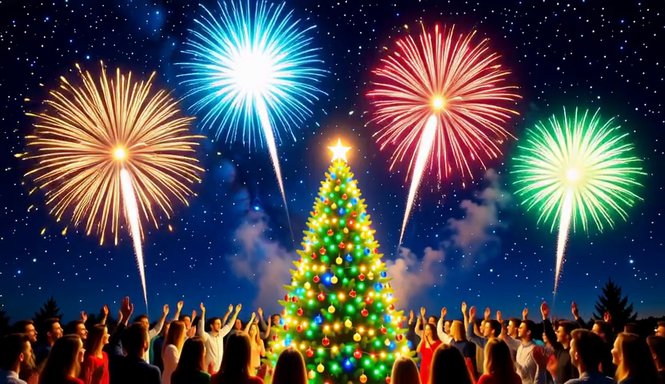}
      & \framecell{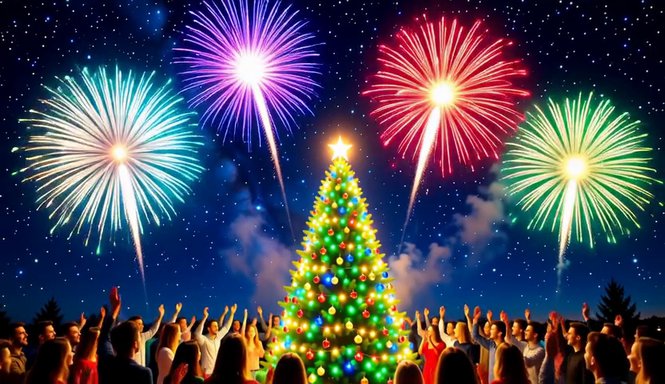}
      & \framecell{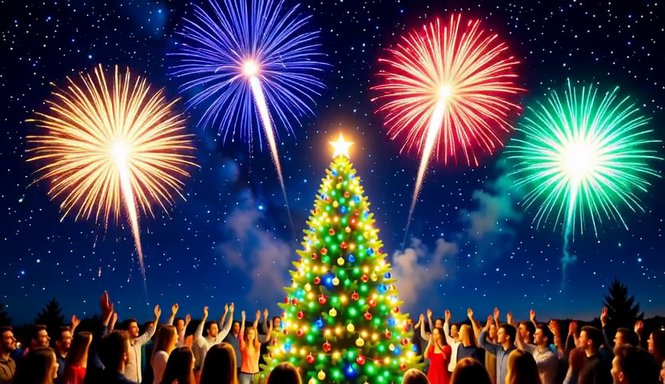} 
      & \phantom{AnyFlow} \\ \addlinespace[2pt]
      & \multicolumn{6}{@{} p{16.6cm} @{}}{\tiny\itshape ``Drone aerial view of a festive celebration featuring a brightly lit Christmas tree surrounded by joyful people. The scene includes colorful fireworks bursting in the starry night sky behind the tree. People are seen smiling and waving, adding to the lively atmosphere. The background shows a clear, twinkling night sky filled with stars. The camera maintains a stable wide shot, capturing the entire scene from above.''} & \\
  \end{tabular}%
  }

  \vspace{-7.5pt}
  \label{fig:christmas_tree_fireworks_drone}
\end{figure}

%% file: figures/wan/diversity_tables/teddy_bear_drums_times_square/frames_table.tex
\begin{figure}[t]
  \centering

  \resizebox{\textwidth}{!}{%

  \begin{tabular}{@{} r @{\hspace{6pt}} c@{\hspace{2pt}}c@{\hspace{2pt}}c @{\hspace{4pt}} c@{\hspace{2pt}}c@{\hspace{2pt}}c @{\hspace{6pt}} c @{}}
                & \multicolumn{3}{c}{\large Noise 1}
                & \multicolumn{3}{c}{\large Noise 2} & \\

                \cmidrule(l{6pt}r{2pt}){2-4} \cmidrule(l{2pt}r{6pt}){5-7}

                & $t{=}0.5$s & $t{=}2.5$s & $t{=}4.5$s & $t{=}0.5$s & $t{=}2.5$s & $t{=}4.5$s & \\
    \textbf{PDD}
      & \framecell{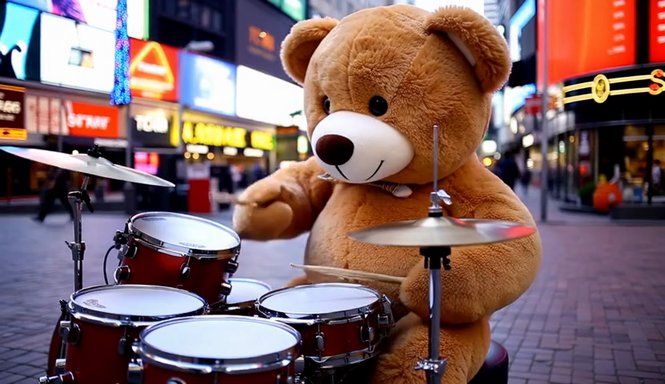}
      & \framecell{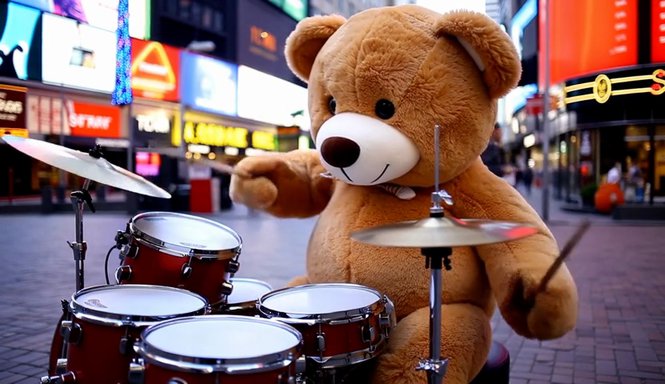}
      & \framecell{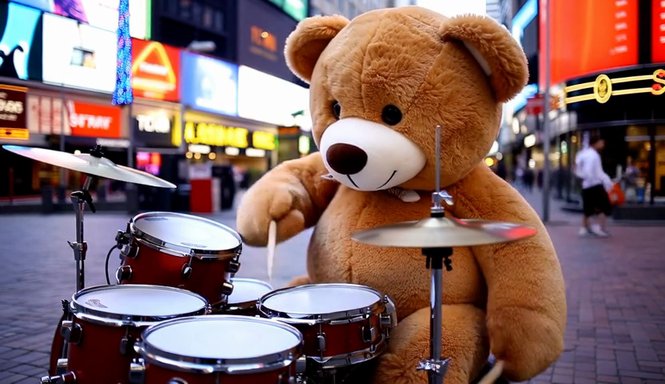}
      & \framecell{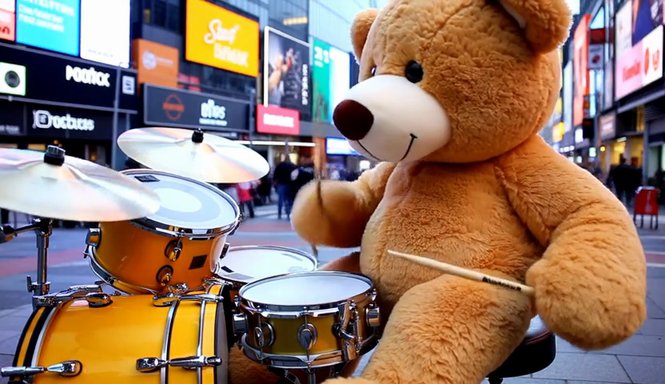}
      & \framecell{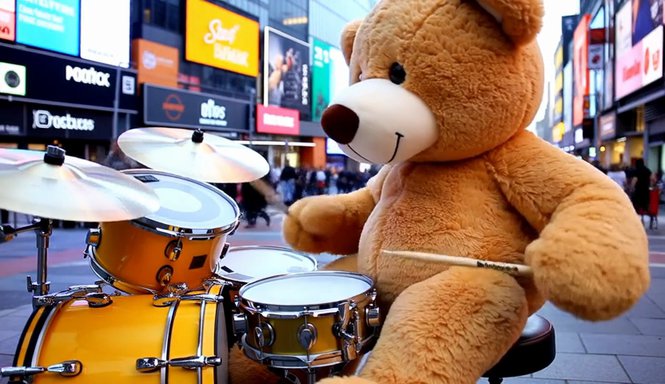}
      & \framecell{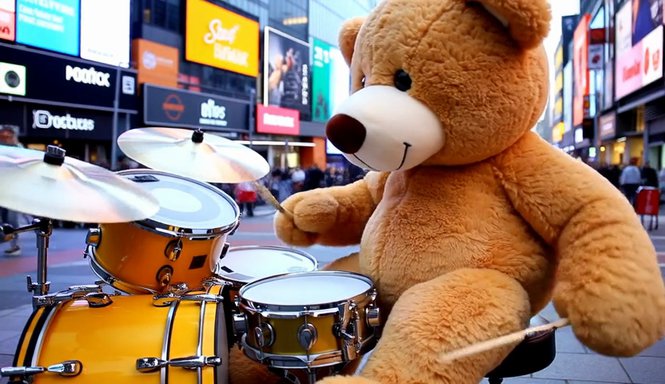} 
      & \phantom{AnyFlow} \\ \addlinespace[2pt]
    DMD2
      & \framecell{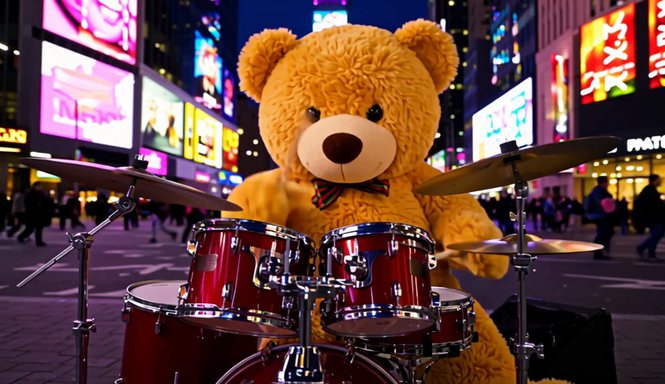}
      & \framecell{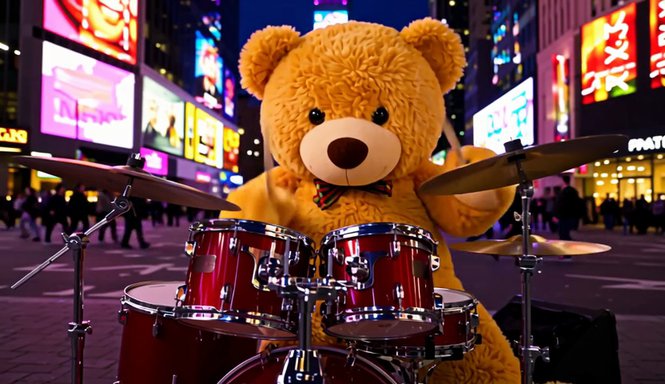}
      & \framecell{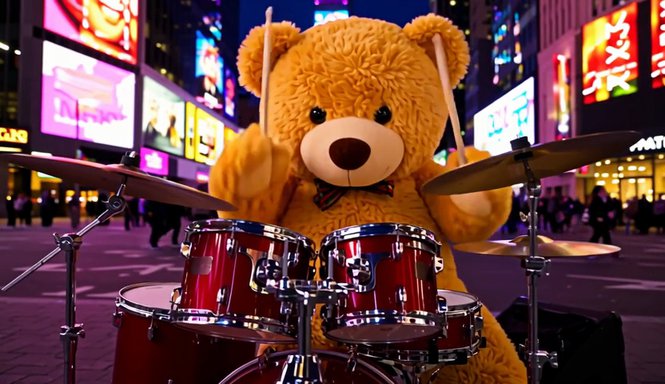}
      & \framecell{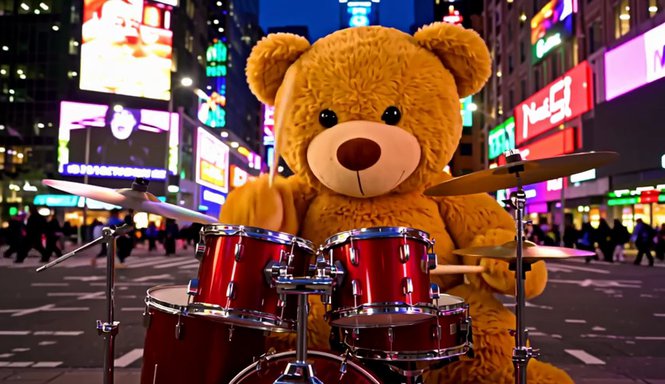}
      & \framecell{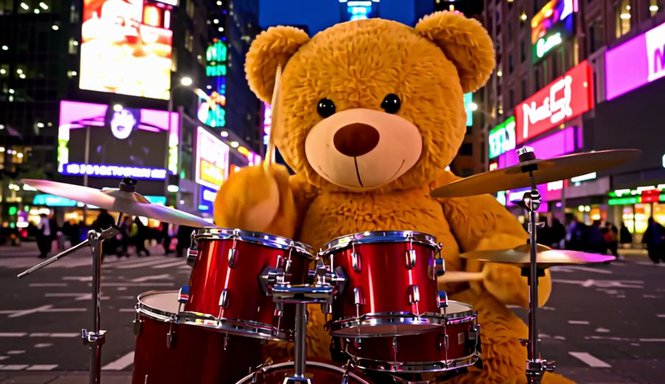}
      & \framecell{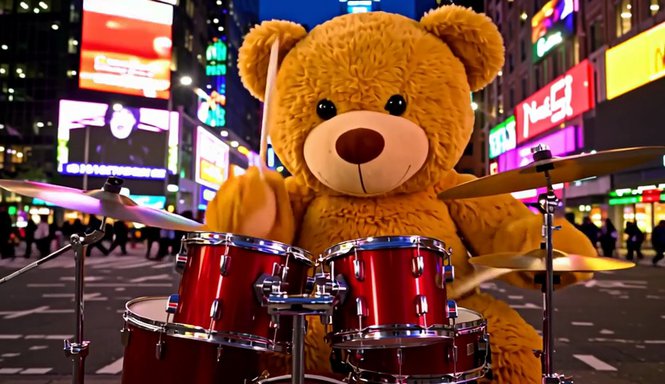} 
      & \phantom{AnyFlow} \\ \addlinespace[2pt]
    AnyFlow
      & \framecell{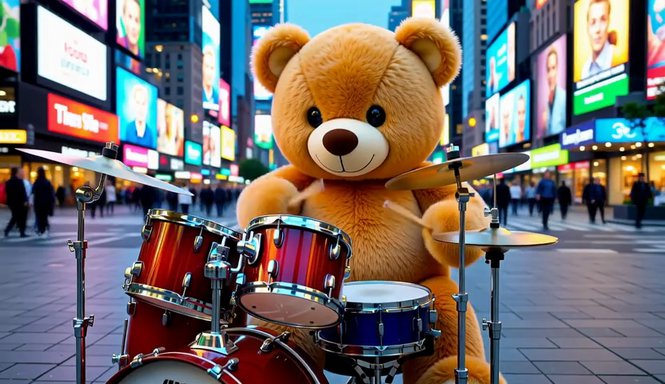}
      & \framecell{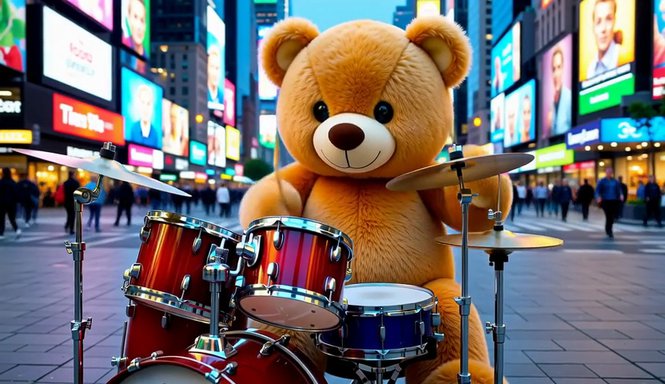}
      & \framecell{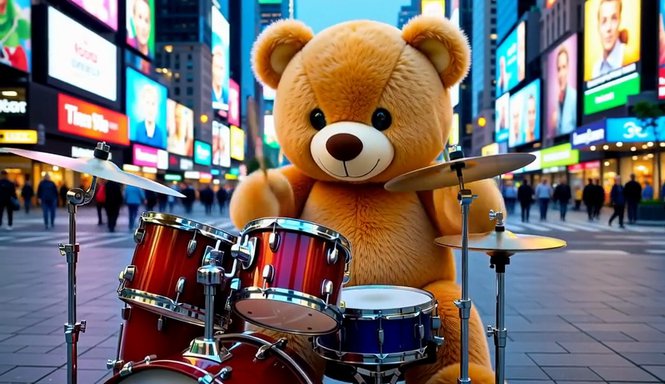}
      & \framecell{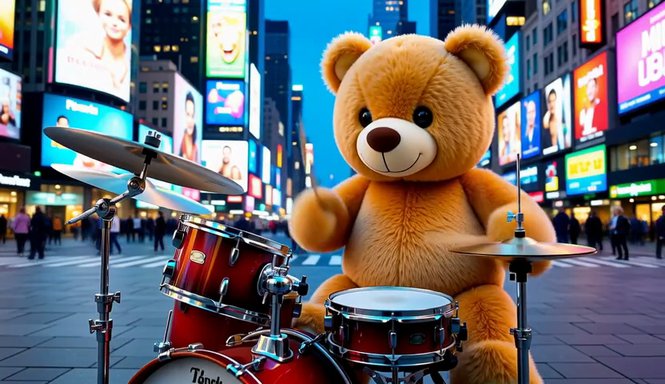}
      & \framecell{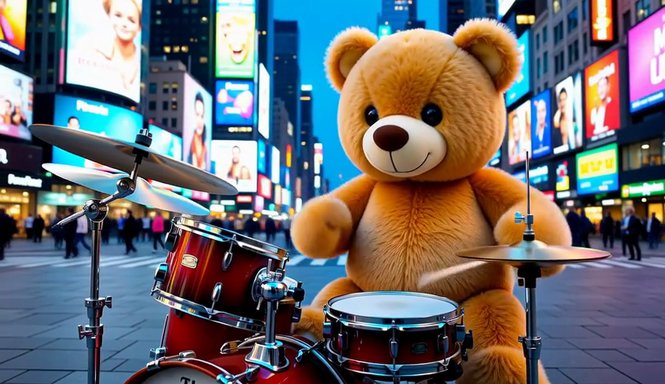}
      & \framecell{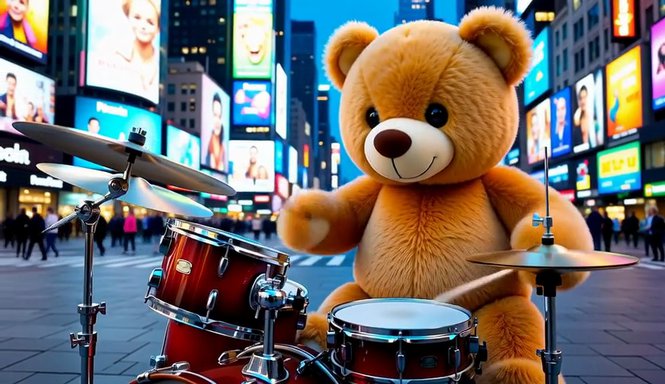} 
      & \phantom{AnyFlow} \\ \addlinespace[2pt]
      & \multicolumn{6}{@{} p{16.6cm} @{}}{\tiny\itshape ``In a vibrant scene set in New York City's Times Square, a large, friendly teddy bear is playfully interacting with a full drum kit. The teddy bear, with its soft brown fur and big, expressive black eyes, sits on a stool, energetically beating the drums with toy-sized drumsticks. Its arms move rhythmically, creating a playful and whimsical atmosphere. The background showcases the bustling energy of Times Square, with bright lights, billboards, and a mix of pedestrians walking past. The scene captures a close-up view, focusing on the teddy bear's animated performance and the colorful surroundings.''} & \\
  \end{tabular}%
  }

  \vspace{-7.5pt}
  \label{fig:teddy_bear_drums_times_square}
\end{figure}

%% file: figures/wan/diversity_tables/person_jogging/frames_table.tex
\begin{figure*}[t]
  \centering

  \resizebox{\textwidth}{!}{%

  \begin{tabular}{@{} r @{\hspace{6pt}} c@{\hspace{2pt}}c@{\hspace{2pt}}c @{\hspace{4pt}} c@{\hspace{2pt}}c@{\hspace{2pt}}c @{\hspace{6pt}} c @{}}
                & \multicolumn{3}{c}{\large Noise 1}
                & \multicolumn{3}{c}{\large Noise 2} & \\

                \cmidrule(l{6pt}r{2pt}){2-4} \cmidrule(l{2pt}r{6pt}){5-7}

                & $t{=}0.5$s & $t{=}2.5$s & $t{=}4.5$s & $t{=}0.5$s & $t{=}2.5$s & $t{=}4.5$s & \\
    \textbf{PDD}
      & \framecell{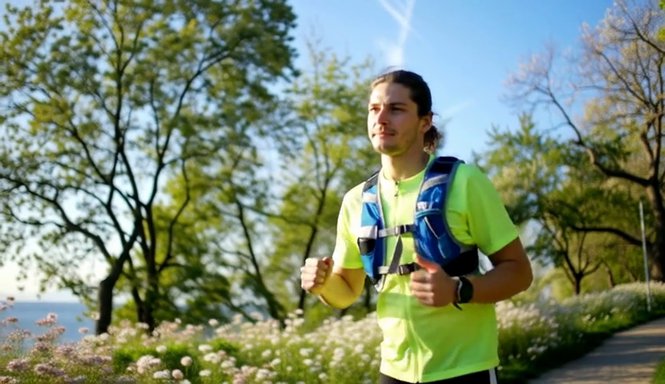}
      & \framecell{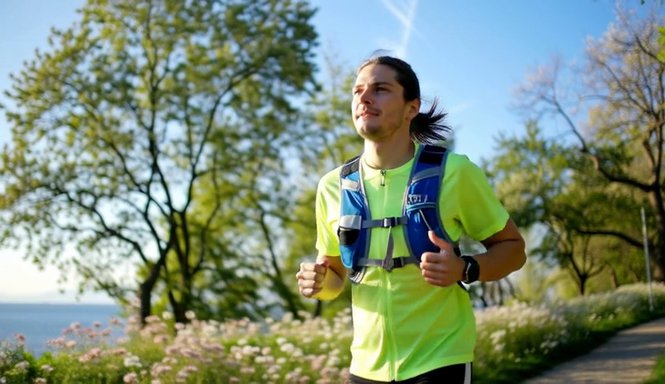}
      & \framecell{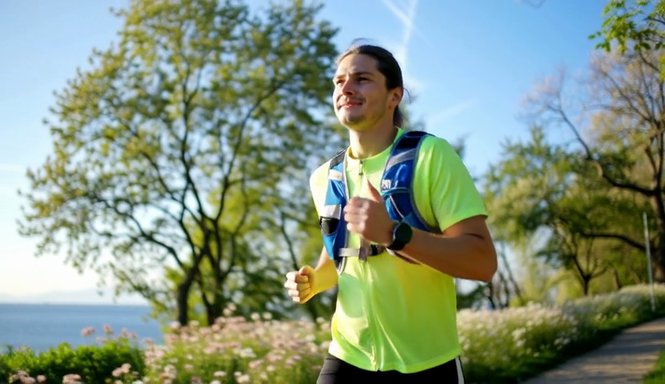}
      & \framecell{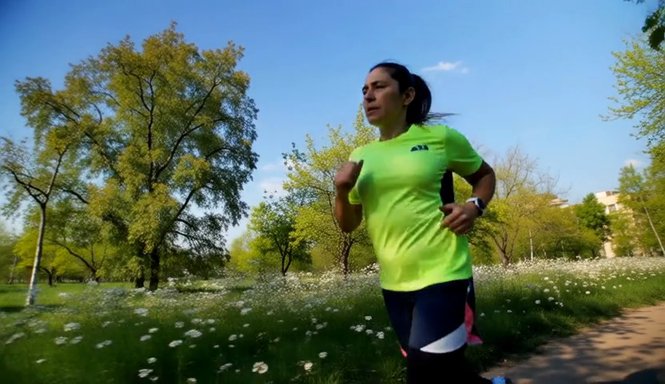}
      & \framecell{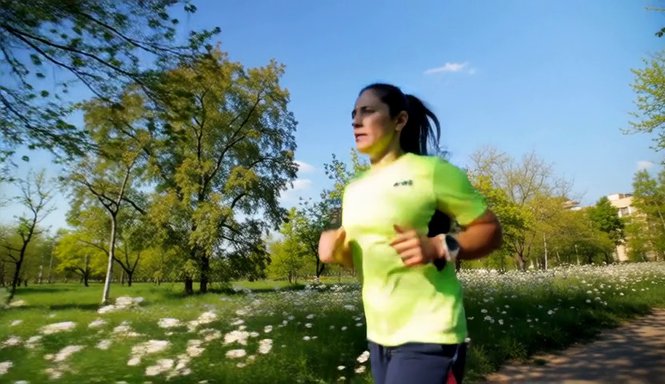}
      & \framecell{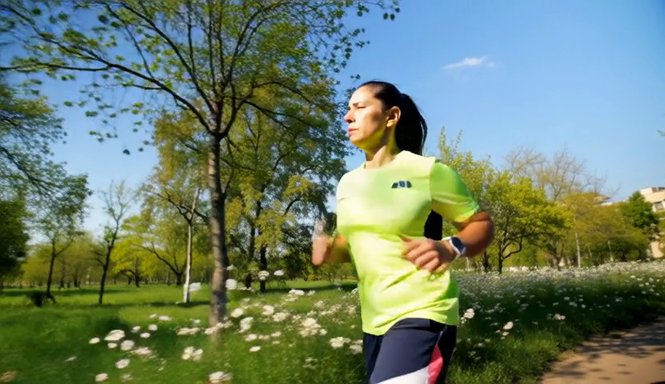} 
      & \phantom{AnyFlow} \\ \addlinespace[2pt]
    DMD2
      & \framecell{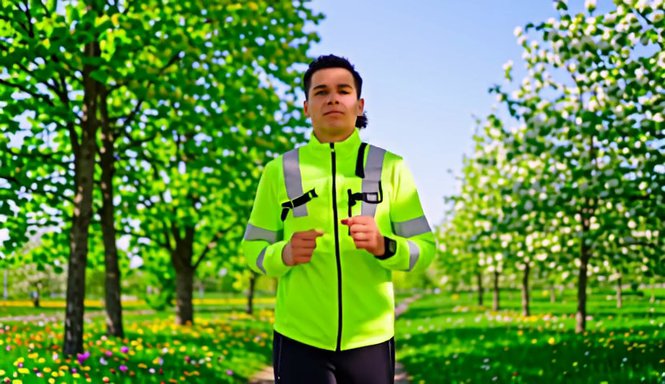}
      & \framecell{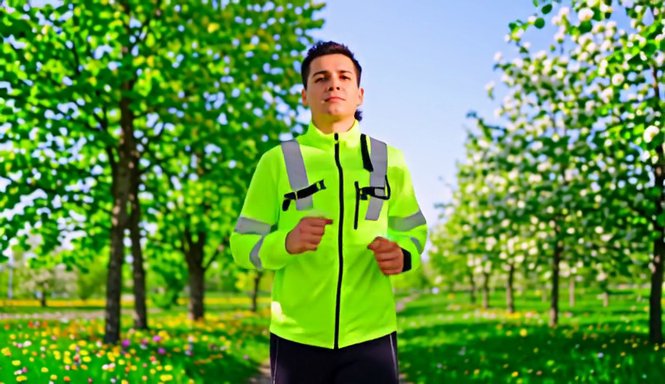}
      & \framecell{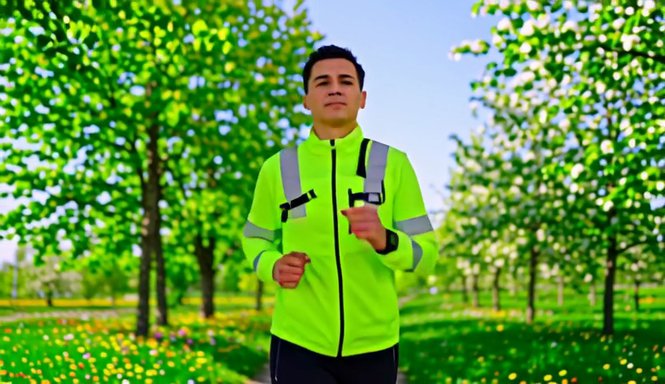}
      & \framecell{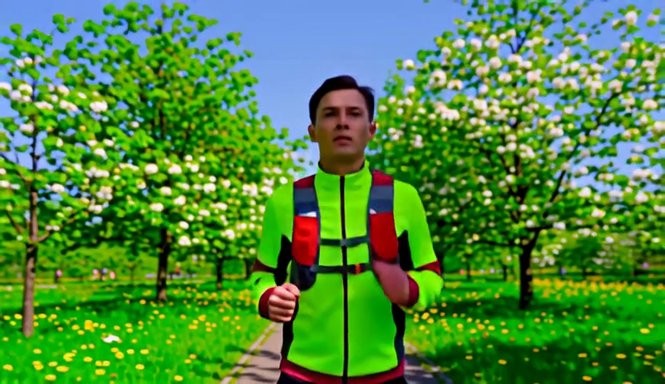}
      & \framecell{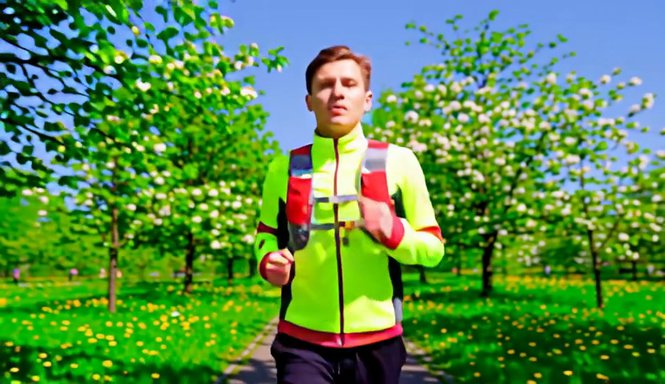}
      & \framecell{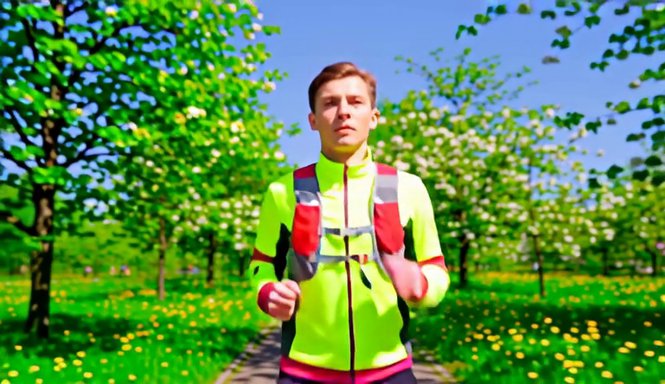} 
      & \phantom{AnyFlow} \\ \addlinespace[2pt]
    AnyFlow
      & \framecell{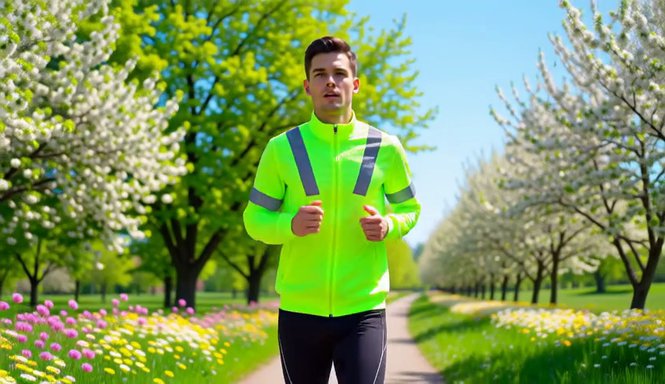}
      & \framecell{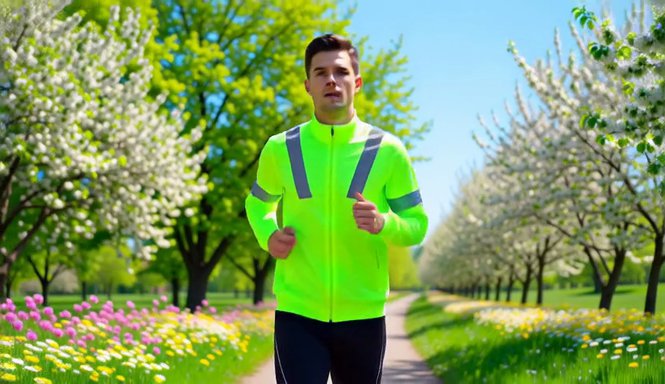}
      & \framecell{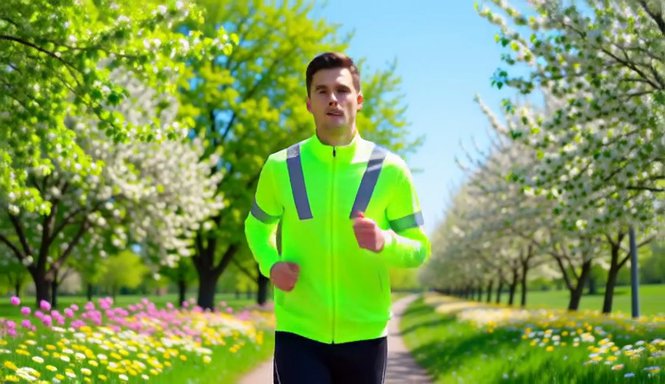}
      & \framecell{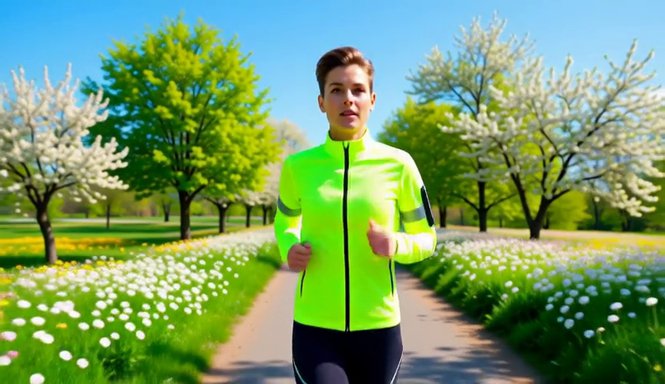}
      & \framecell{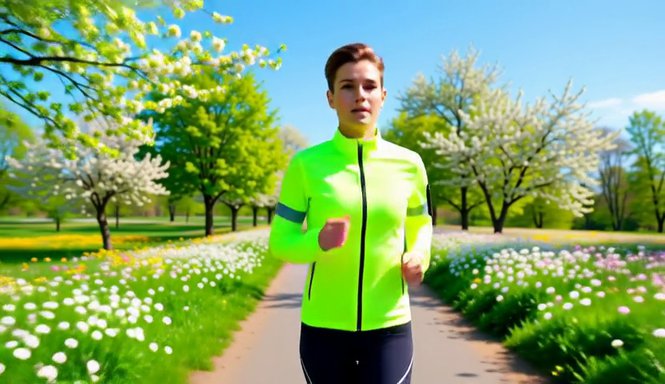}
      & \framecell{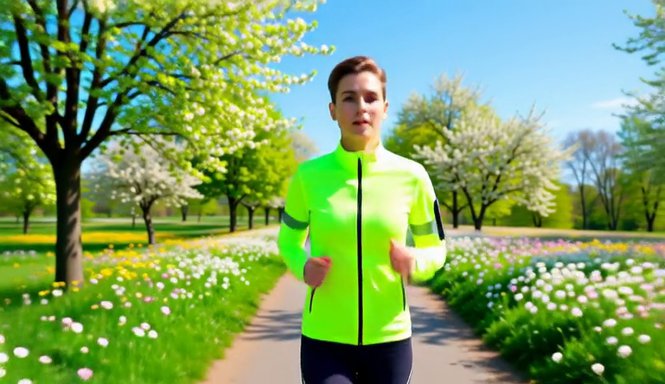} 
      & \phantom{AnyFlow} \\ \addlinespace[2pt]
      & \multicolumn{6}{@{} p{16.6cm} @{}}{\tiny\itshape ``A person is jogging through a scenic park trail, surrounded by lush green trees and blooming flowers. They are wearing a bright athletic outfit with reflective elements for safety, and their face is focused and determined, showing a mix of sweat and effort. The jogger is moving at a steady pace, arms swinging naturally. The background shows a clear blue sky and patches of sunlight filtering through the leaves. The scene captures the jogger from a medium distance, providing a balanced view of the environment and the individual.''} & \\
  \end{tabular}%
  }

  \vspace{-7.5pt}
  \caption{Wan2.1 14B: PDD - Midpoint (ours) vs. DMD2 (FastGen) and AnyFlow with 4 NFE.}
  \label{fig:person_jogging}
\end{figure*}

%% file: figures/wan/diversity_tables/person_digging/frames_table.tex
\begin{figure*}[t]
  \centering

  \resizebox{\textwidth}{!}{%

  \begin{tabular}{@{} r @{\hspace{6pt}} c@{\hspace{2pt}}c@{\hspace{2pt}}c @{\hspace{4pt}} c@{\hspace{2pt}}c@{\hspace{2pt}}c @{\hspace{6pt}} c @{}}
                & \multicolumn{3}{c}{\large Noise 1}
                & \multicolumn{3}{c}{\large Noise 2} & \\

                \cmidrule(l{6pt}r{2pt}){2-4} \cmidrule(l{2pt}r{6pt}){5-7}

                & $t{=}0.5$s & $t{=}2.5$s & $t{=}4.5$s & $t{=}0.5$s & $t{=}2.5$s & $t{=}4.5$s & \\
    \textbf{PDD}
      & \framecell{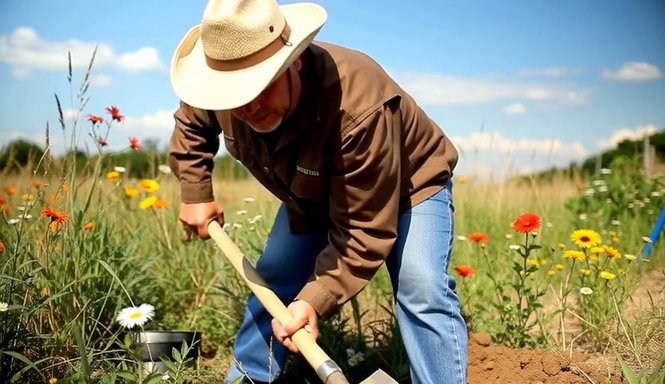}
      & \framecell{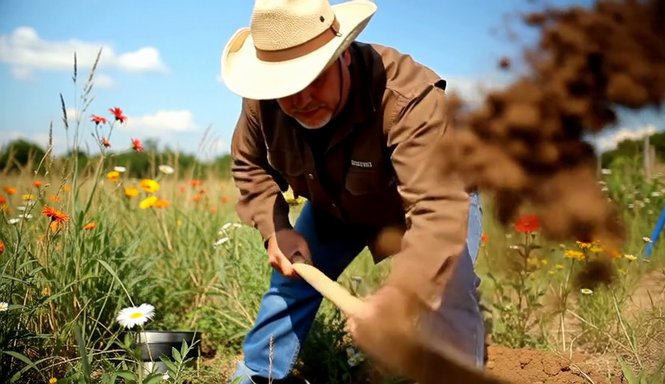}
      & \framecell{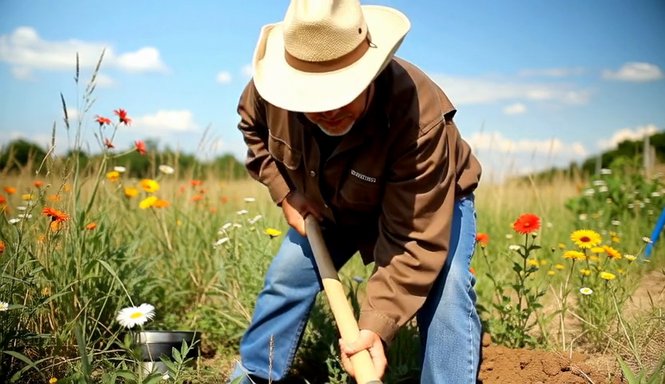}
      & \framecell{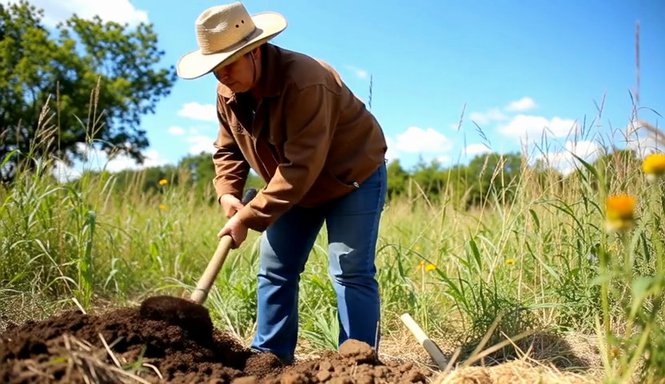}
      & \framecell{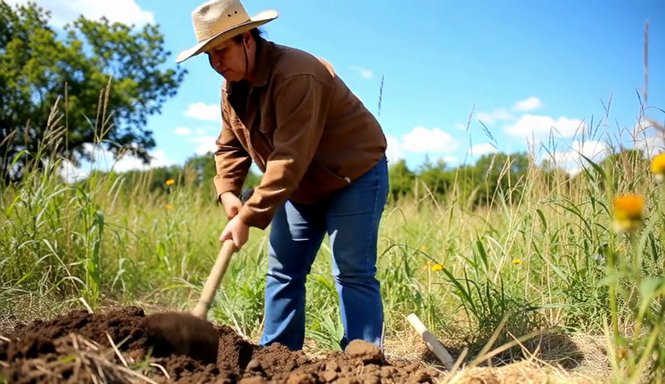}
      & \framecell{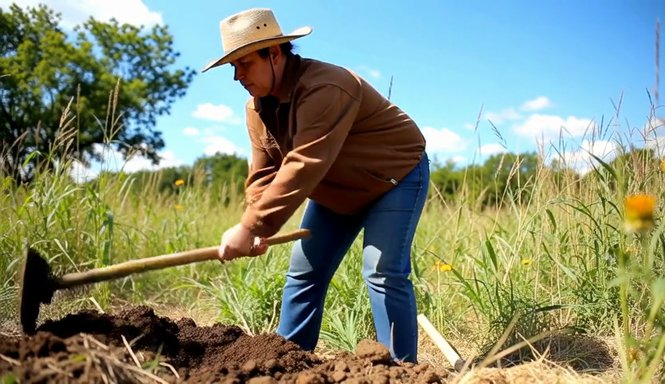} 
      & \phantom{AnyFlow} \\ \addlinespace[2pt]
    DMD2
      & \framecell{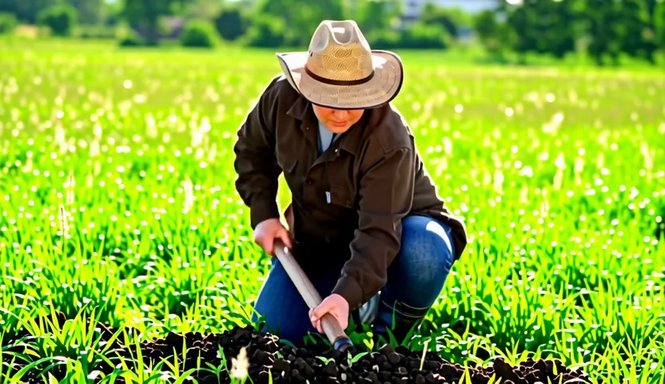}
      & \framecell{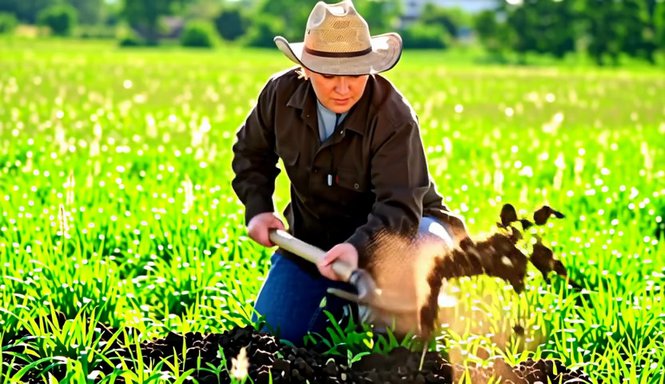}
      & \framecell{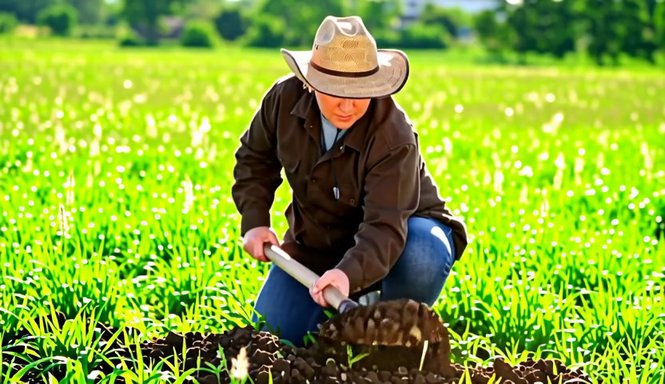}
      & \framecell{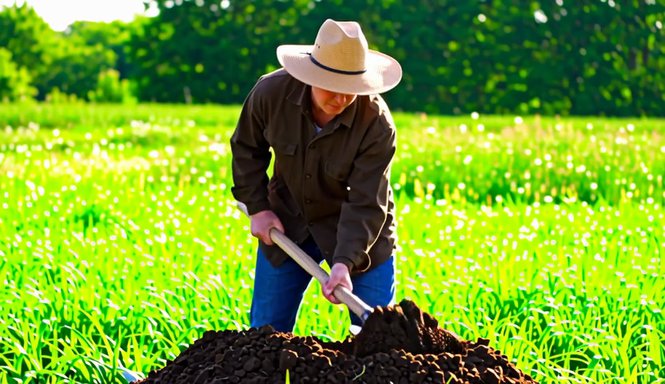}
      & \framecell{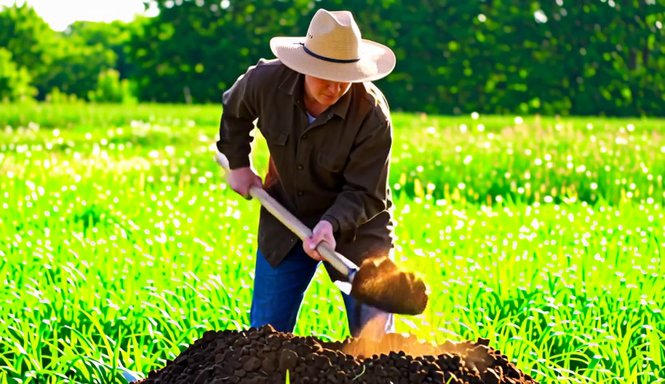}
      & \framecell{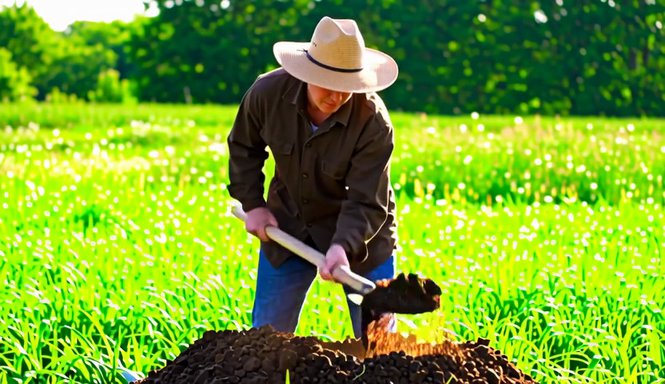} 
      & \phantom{AnyFlow} \\ \addlinespace[2pt]
    AnyFlow
      & \framecell{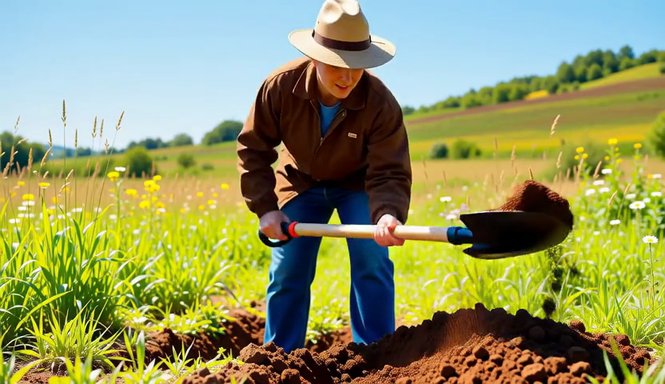}
      & \framecell{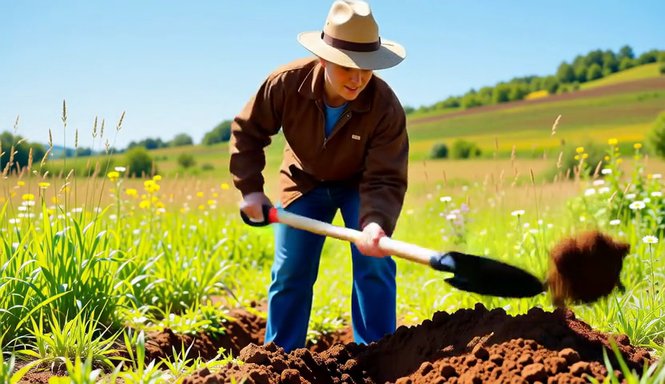}
      & \framecell{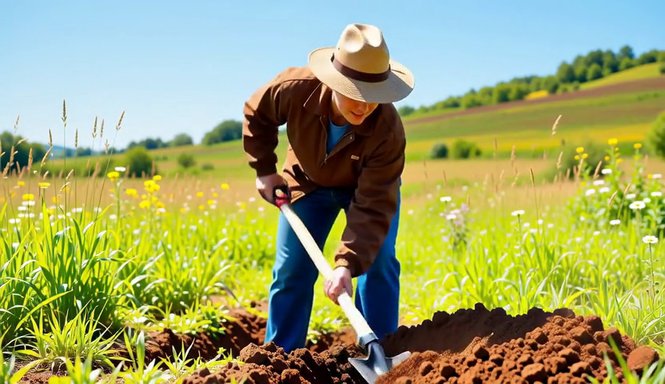}
      & \framecell{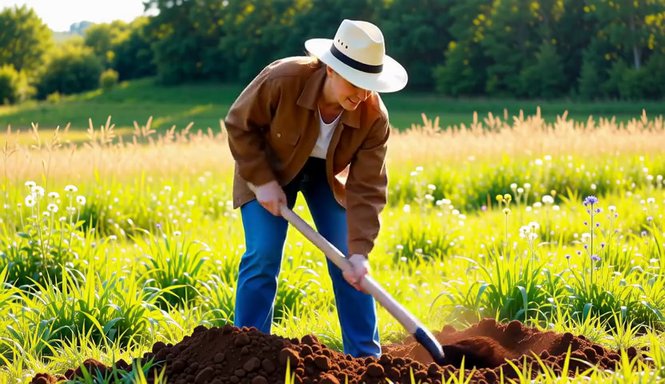}
      & \framecell{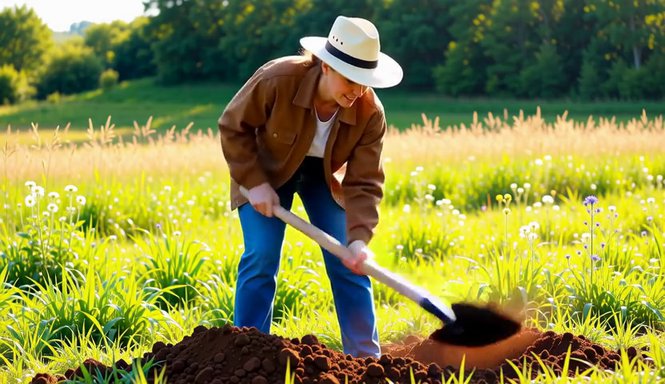}
      & \framecell{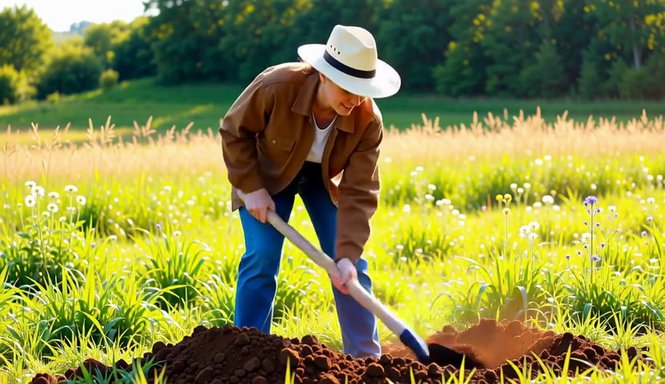} 
      & \phantom{AnyFlow} \\ \addlinespace[2pt]
      & \multicolumn{6}{@{} p{16.6cm} @{}}{\tiny\itshape ``A person is digging in a sunny, rural field. They are wearing a brown work jacket, blue jeans, and a straw hat to protect themselves from the sun. Their face is covered with sweat, and they have a determined look as they focus intently on the task. The person uses a shovel with strong, repetitive motions, throwing dirt to the side. The field has tall grass and scattered wildflowers in the background. The scene is captured in a mid-shot, focusing on the person's upper body and the immediate area where they are digging. The camera remains static to highlight the continuous action of digging.''} & \\
  \end{tabular}%
  }

  \vspace{-7.5pt}
  \label{fig:person_digging}
\end{figure*}

%% file: figures/wan/diversity_tables/person_ice_skating/frames_table.tex
\begin{figure*}[t]
  \centering

  \resizebox{\textwidth}{!}{%

  \begin{tabular}{@{} r @{\hspace{6pt}} c@{\hspace{2pt}}c@{\hspace{2pt}}c @{\hspace{4pt}} c@{\hspace{2pt}}c@{\hspace{2pt}}c @{\hspace{6pt}} c @{}}
                & \multicolumn{3}{c}{\large Noise 1}
                & \multicolumn{3}{c}{\large Noise 2} & \\

                \cmidrule(l{6pt}r{2pt}){2-4} \cmidrule(l{2pt}r{6pt}){5-7}

                & $t{=}0.5$s & $t{=}2.5$s & $t{=}4.5$s & $t{=}0.5$s & $t{=}2.5$s & $t{=}4.5$s & \\
    \textbf{PDD}
      & \framecell{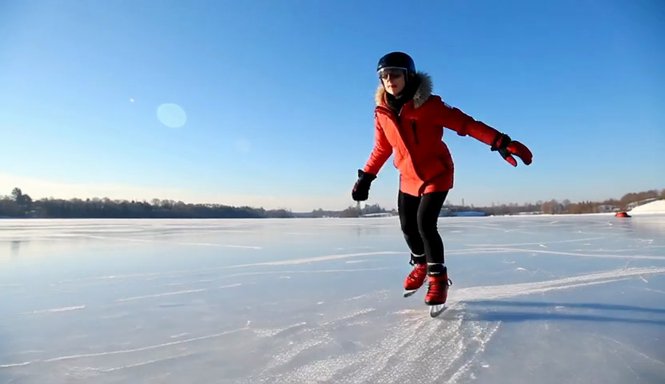}
      & \framecell{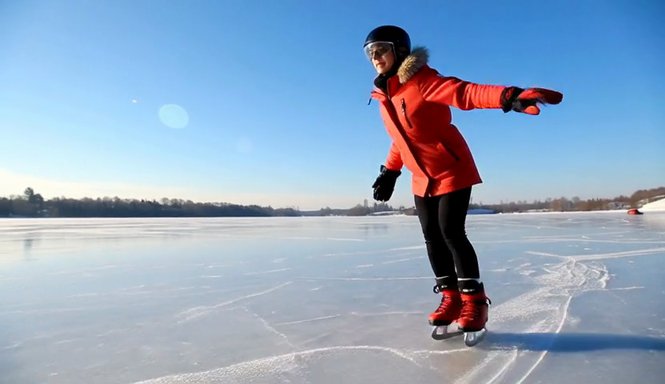}
      & \framecell{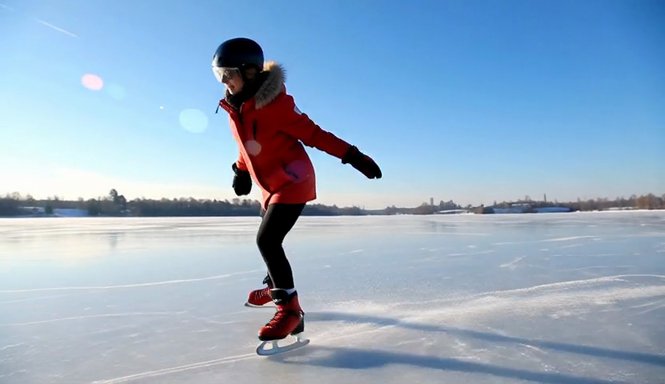}
      & \framecell{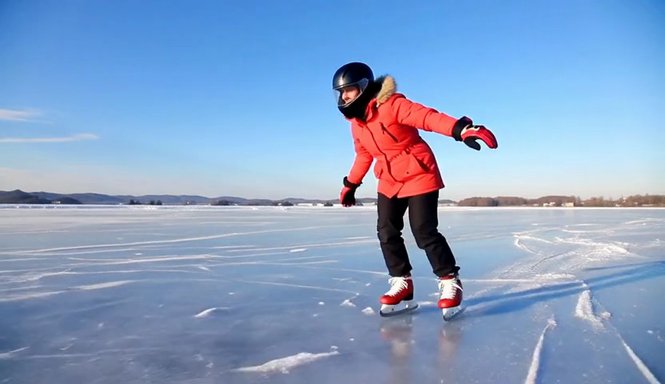}
      & \framecell{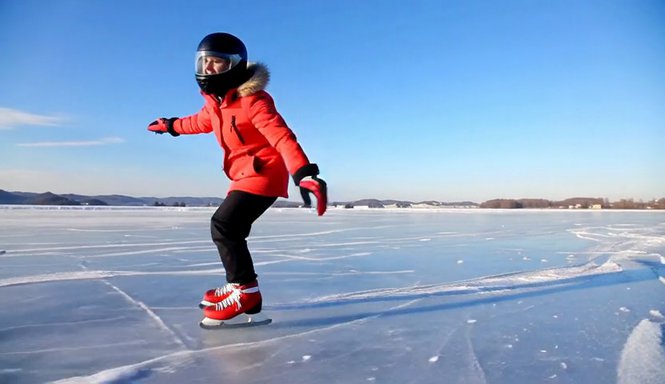}
      & \framecell{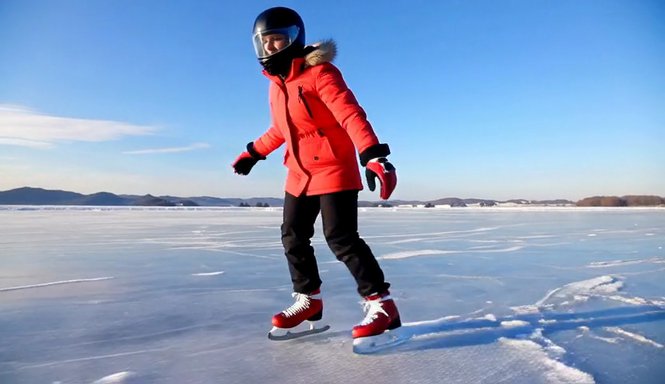} 
      & \phantom{AnyFlow} \\ \addlinespace[2pt]
    DMD2
      & \framecell{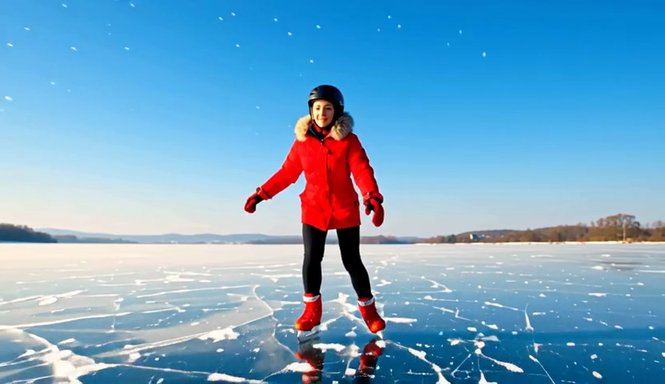}
      & \framecell{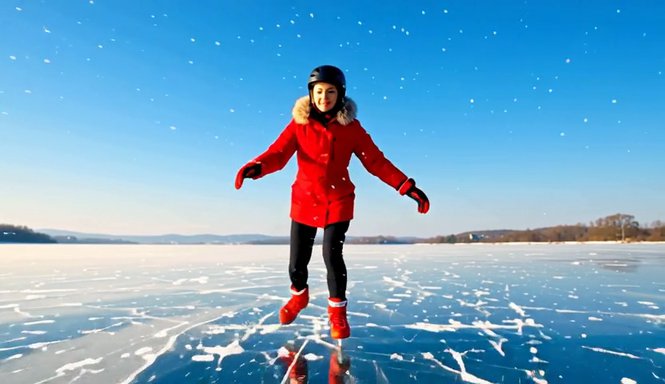}
      & \framecell{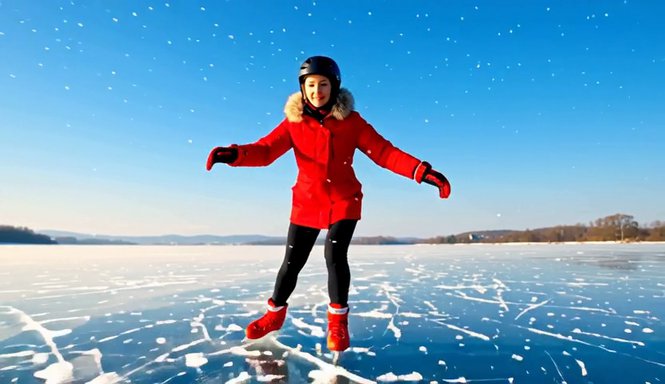}
      & \framecell{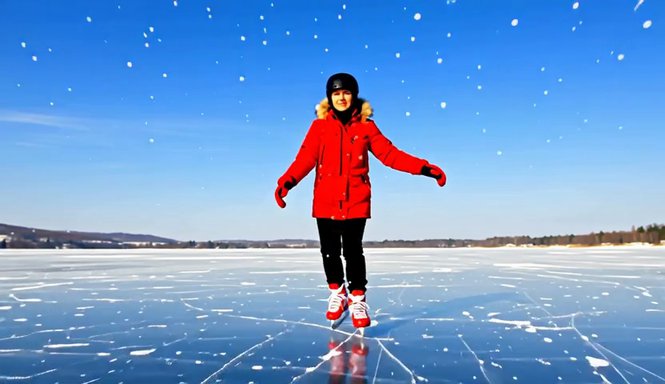}
      & \framecell{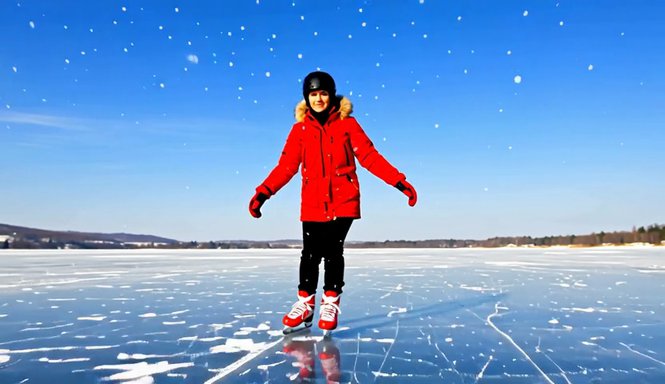}
      & \framecell{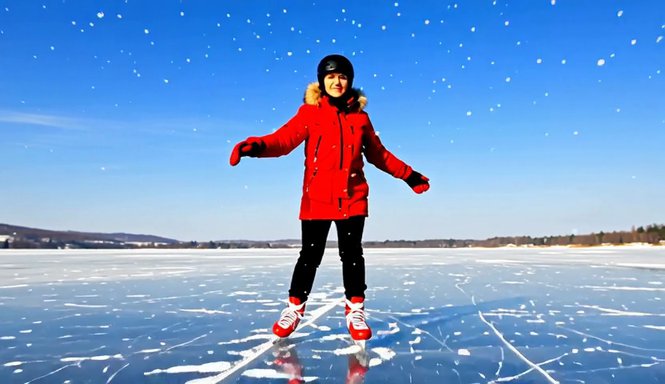} 
      & \phantom{AnyFlow} \\ \addlinespace[2pt]
    AnyFlow
      & \framecell{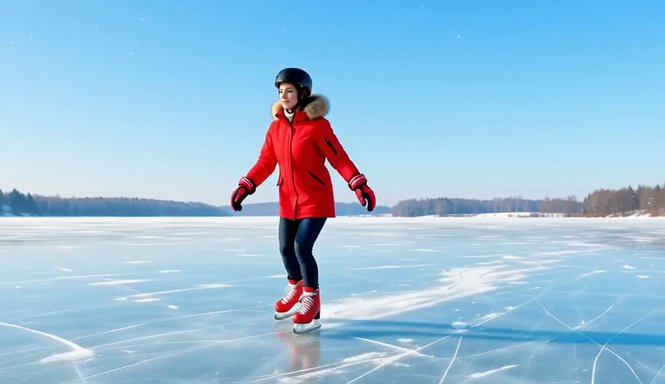}
      & \framecell{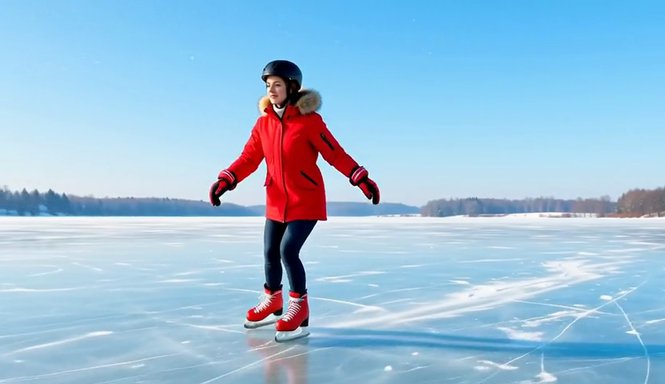}
      & \framecell{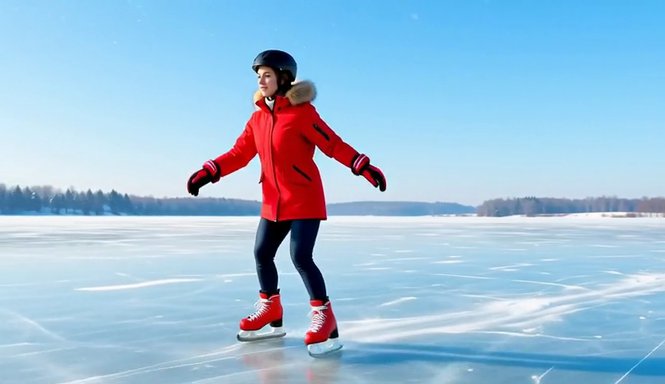}
      & \framecell{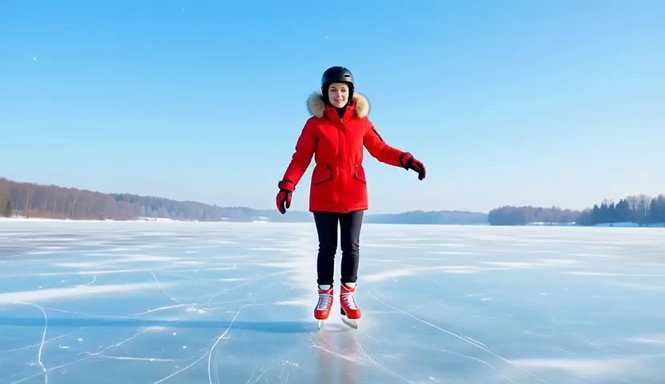}
      & \framecell{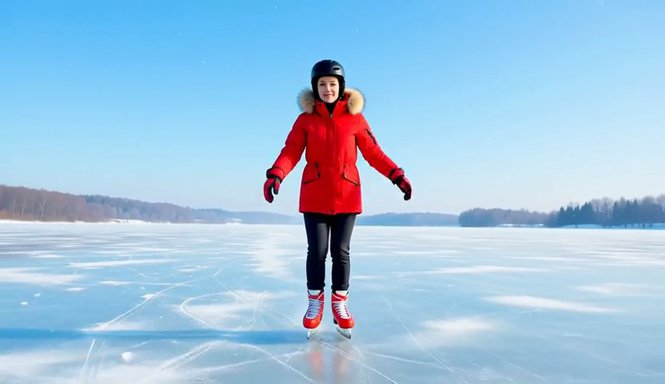}
      & \framecell{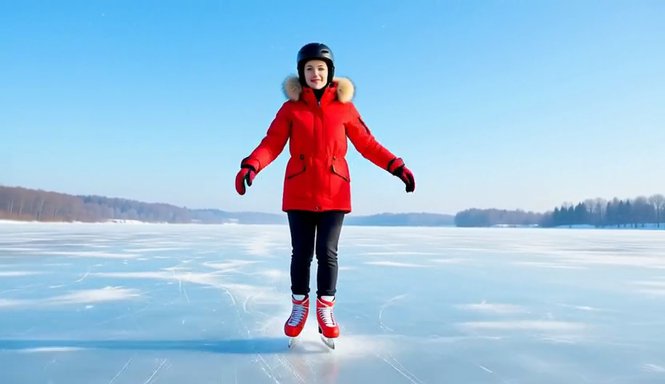} 
      & \phantom{AnyFlow} \\ \addlinespace[2pt]
      & \multicolumn{6}{@{} p{16.6cm} @{}}{\tiny\itshape ``A person is ice skating on a smooth, icy lake under a clear winter sky. They are wearing a black helmet, a bright red parka, and matching red gloves and skates. The person glides gracefully across the ice, arms extended slightly for balance. Snowflakes gently fall around them, creating a serene and wintery atmosphere. The scene is captured in a medium shot from a static camera angle, emphasizing the fluid motion of the skater.''} & \\
  \end{tabular}%
  }

  \vspace{-7.5pt}
  \label{fig:person_ice_skating}
\end{figure*}

%% file: figures/wan/diversity_tables/bar/frames_table.tex
\begin{figure*}[t]
  \centering

  \resizebox{\textwidth}{!}{%

  \begin{tabular}{@{} r @{\hspace{6pt}} c@{\hspace{2pt}}c@{\hspace{2pt}}c @{\hspace{4pt}} c@{\hspace{2pt}}c@{\hspace{2pt}}c @{\hspace{6pt}} c @{}}
                & \multicolumn{3}{c}{\large Noise 1}
                & \multicolumn{3}{c}{\large Noise 2} & \\

                \cmidrule(l{6pt}r{2pt}){2-4} \cmidrule(l{2pt}r{6pt}){5-7}

                & $t{=}0.5$s & $t{=}2.5$s & $t{=}4.5$s & $t{=}0.5$s & $t{=}2.5$s & $t{=}4.5$s & \\
    \textbf{PDD}
      & \framecell{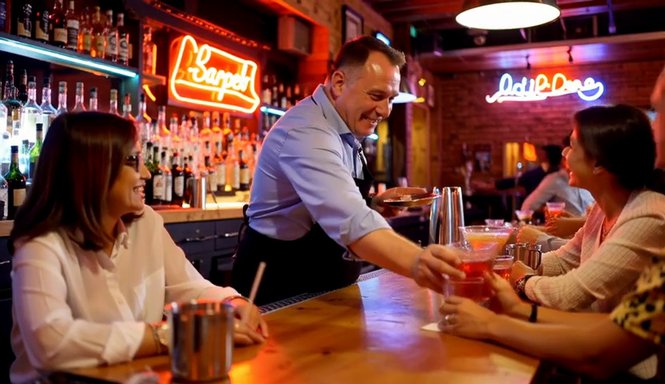}
      & \framecell{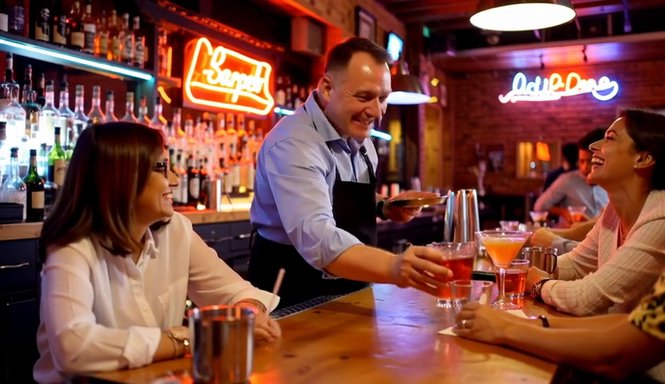}
      & \framecell{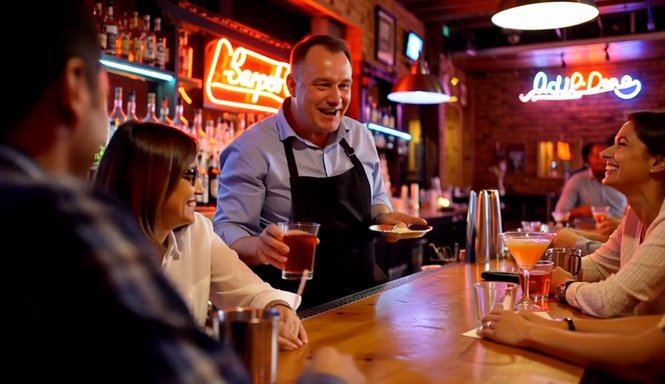}
      & \framecell{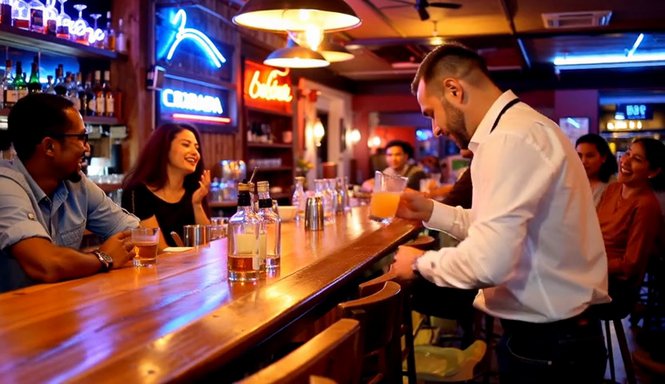}
      & \framecell{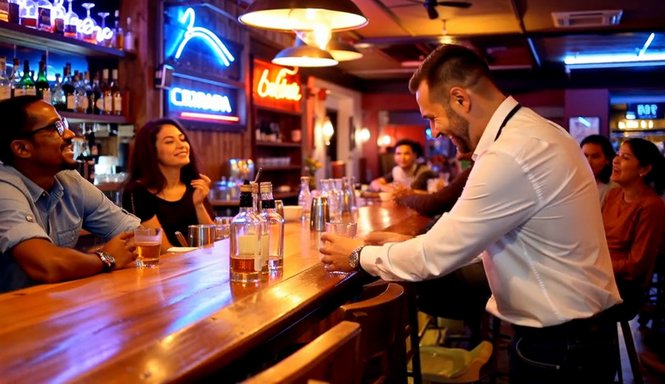}
      & \framecell{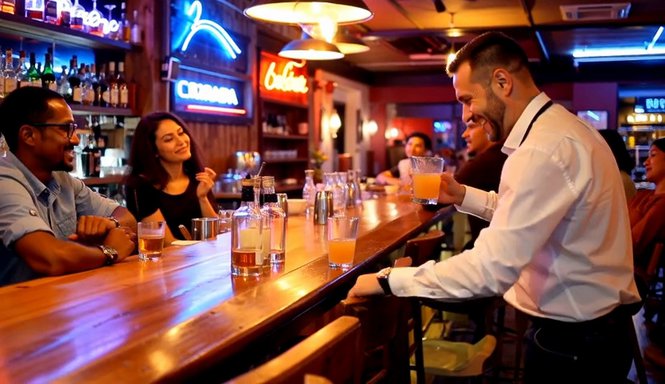} 
      & \phantom{AnyFlow} \\ \addlinespace[2pt]
    DMD2
      & \framecell{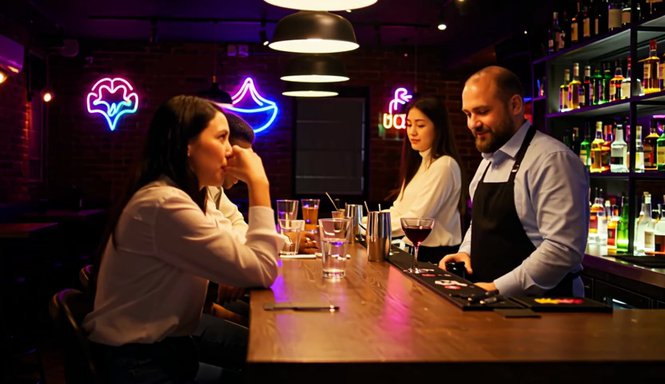}
      & \framecell{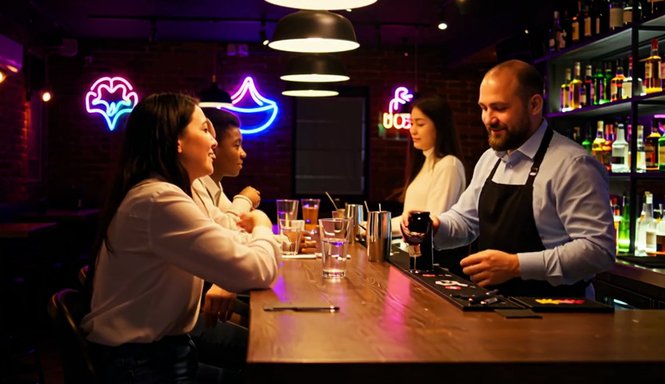}
      & \framecell{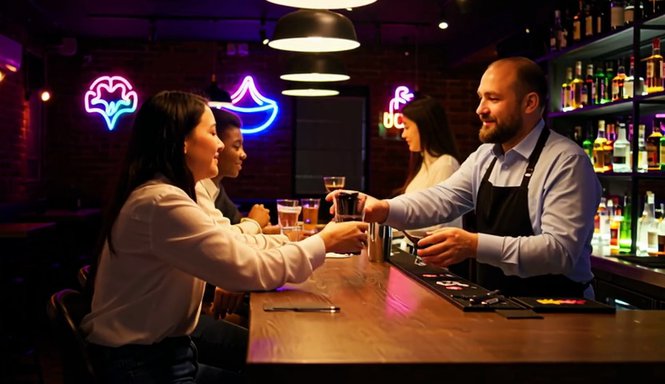}
      & \framecell{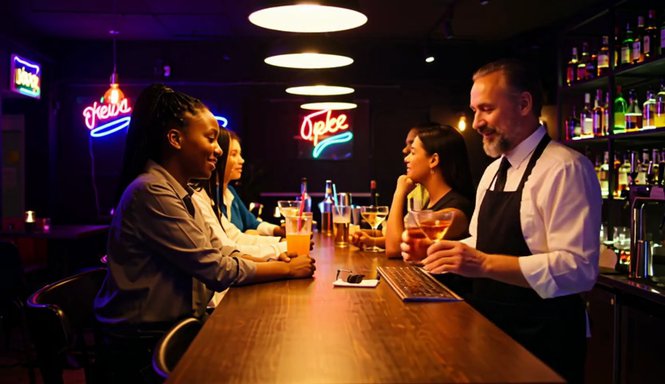}
      & \framecell{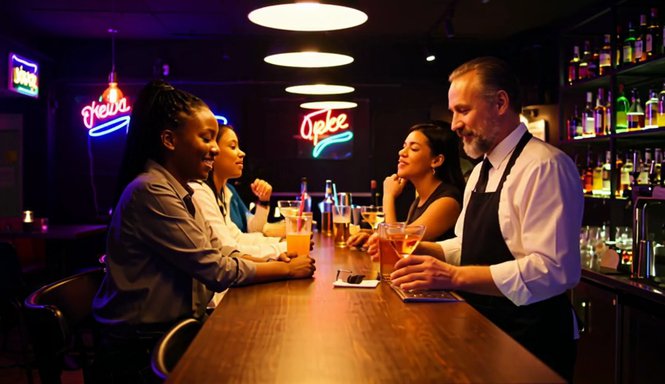}
      & \framecell{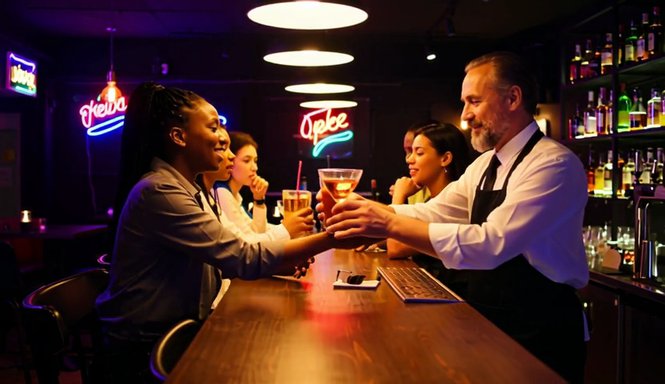} 
      & \phantom{AnyFlow} \\ \addlinespace[2pt]
    AnyFlow
      & \framecell{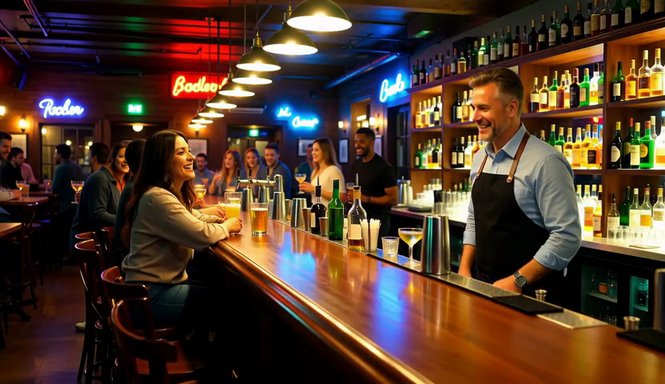}
      & \framecell{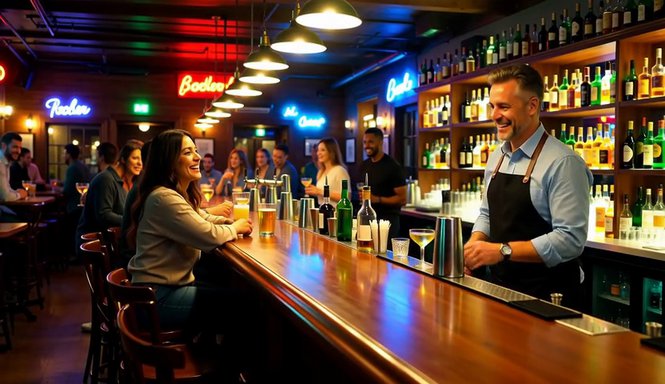}
      & \framecell{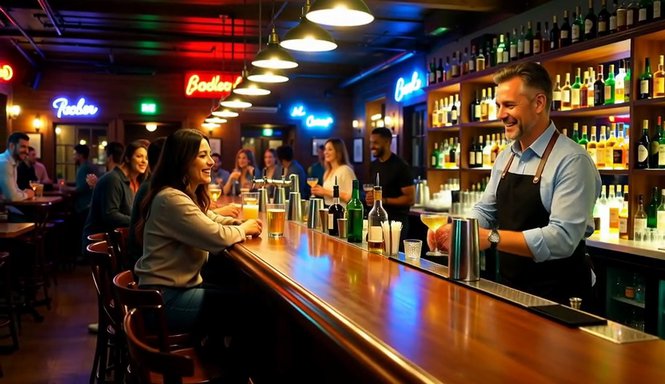}
      & \framecell{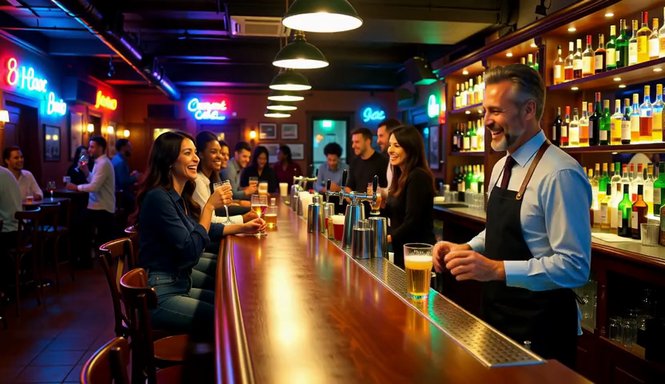}
      & \framecell{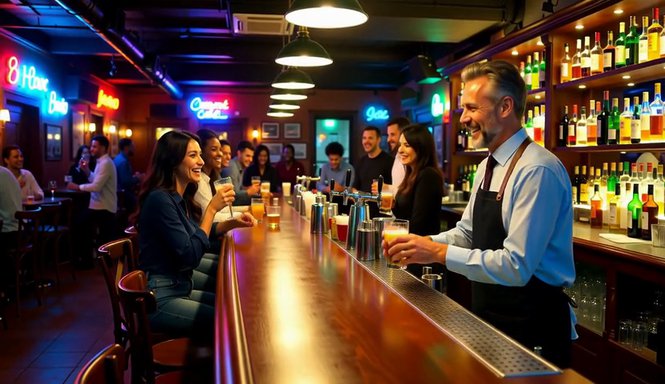}
      & \framecell{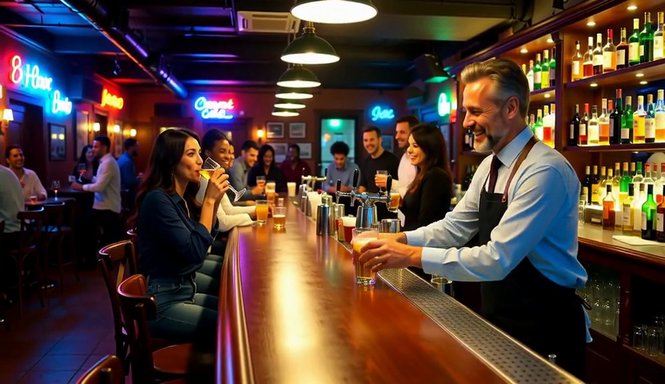} 
      & \phantom{AnyFlow} \\ \addlinespace[2pt]
      & \multicolumn{6}{@{} p{16.6cm} @{}}{\tiny\itshape ``A lively bar scene set during nighttime with warm, ambient lighting. The interior is decorated with dimly lit neon signs, wooden tables, and chairs, and a long polished bar counter with various bottles of alcohol neatly arranged behind it. Customers are engaged in conversations, some laughing, others quietly sipping their drinks. The bartender, a middle-aged man with a friendly smile, is attentively serving customers. The scene includes various races and ages of patrons, creating a diverse atmosphere. Medium shot capturing the bustling environment from a static perspective.''} & \\
  \end{tabular}%
  }

  \vspace{-7.5pt}
  \label{fig:bar}
\end{figure*}

%% file: figures/wan/diversity_tables/corgi_park_hokusai/frames_table.tex
\begin{figure*}[t]
  \centering

  \resizebox{\textwidth}{!}{%

  \begin{tabular}{@{} r @{\hspace{6pt}} c@{\hspace{2pt}}c@{\hspace{2pt}}c @{\hspace{4pt}} c@{\hspace{2pt}}c@{\hspace{2pt}}c @{\hspace{6pt}} c @{}}
                & \multicolumn{3}{c}{\large Noise 1}
                & \multicolumn{3}{c}{\large Noise 2} & \\

                \cmidrule(l{6pt}r{2pt}){2-4} \cmidrule(l{2pt}r{6pt}){5-7}

                & $t{=}0.5$s & $t{=}2.5$s & $t{=}4.5$s & $t{=}0.5$s & $t{=}2.5$s & $t{=}4.5$s & \\
    \textbf{PDD}
      & \framecell{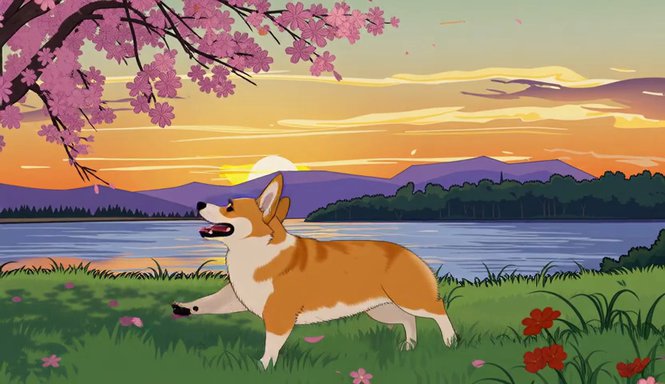}
      & \framecell{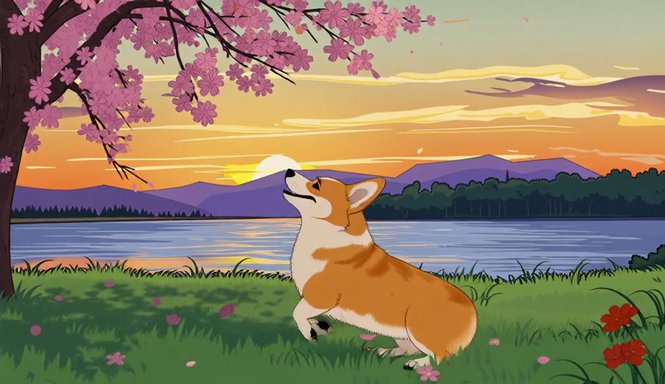}
      & \framecell{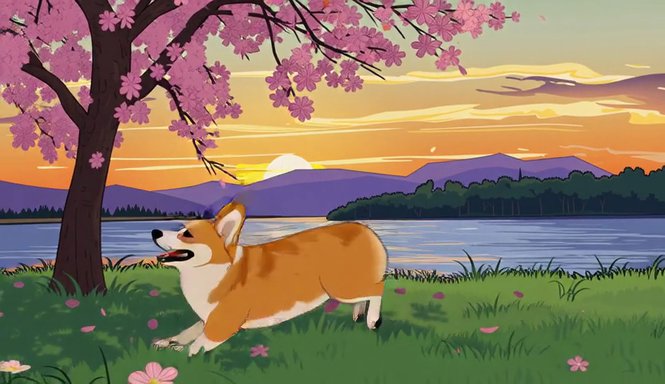}
      & \framecell{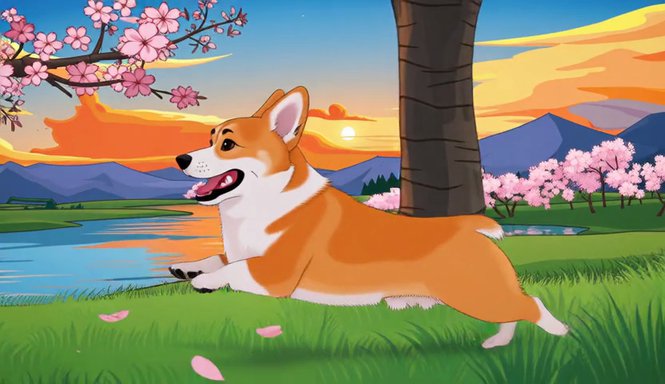}
      & \framecell{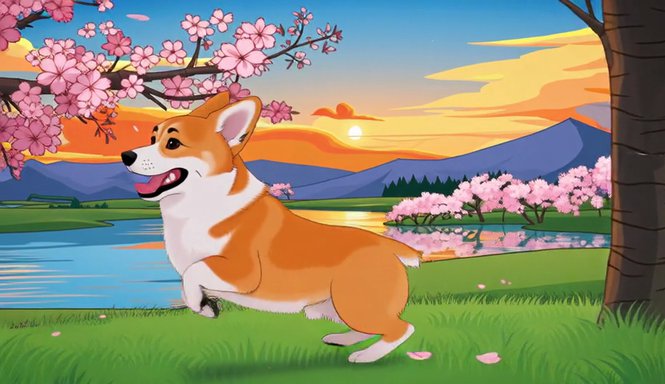}
      & \framecell{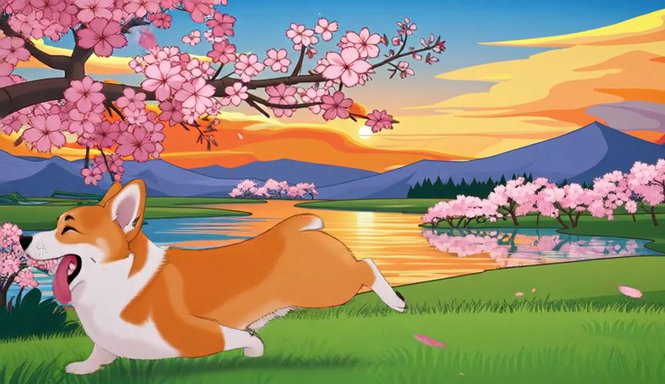} 
      & \phantom{AnyFlow} \\ \addlinespace[2pt]
    DMD2
      & \framecell{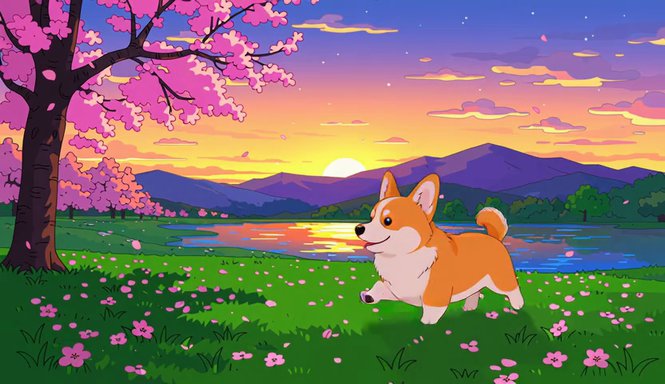}
      & \framecell{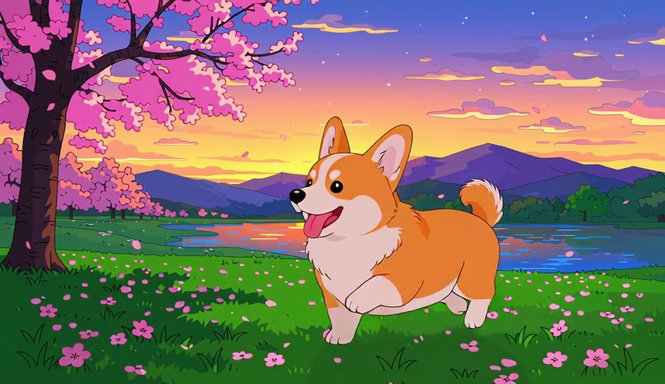}
      & \framecell{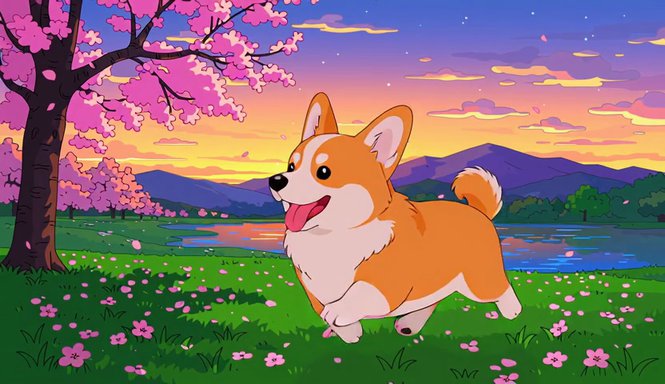}
      & \framecell{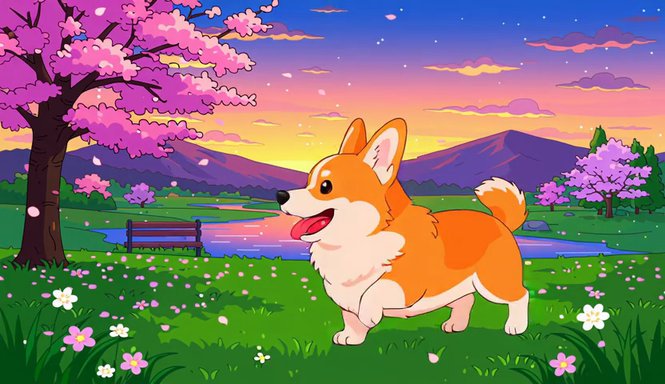}
      & \framecell{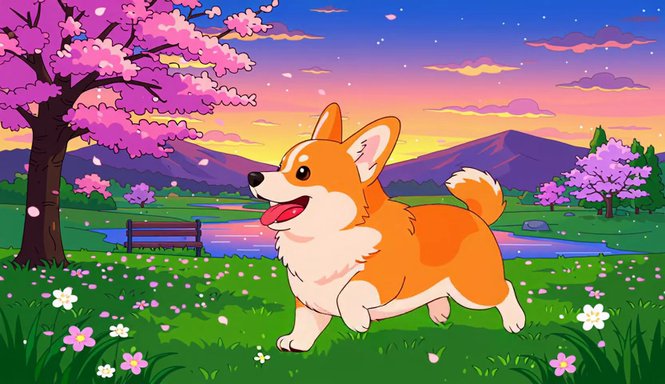}
      & \framecell{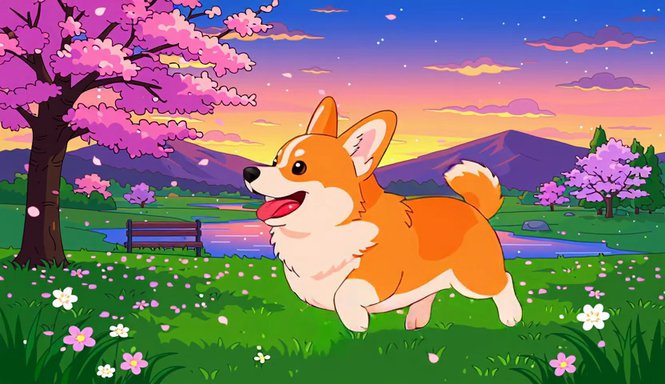} 
      & \phantom{AnyFlow} \\ \addlinespace[2pt]
    AnyFlow
      & \framecell{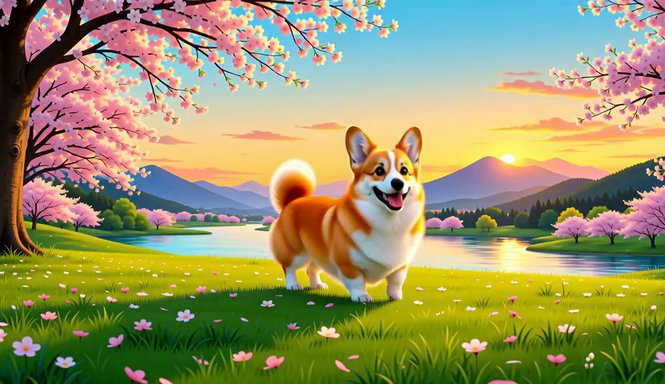}
      & \framecell{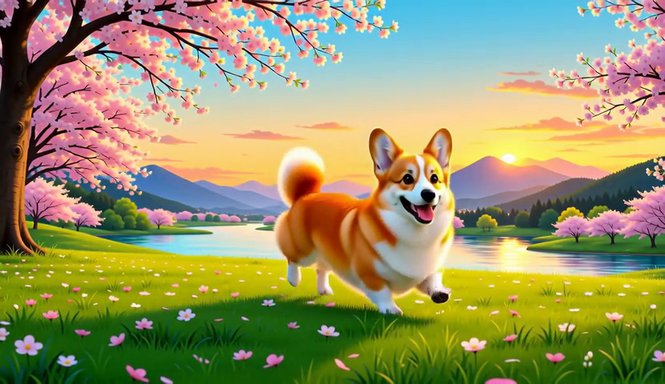}
      & \framecell{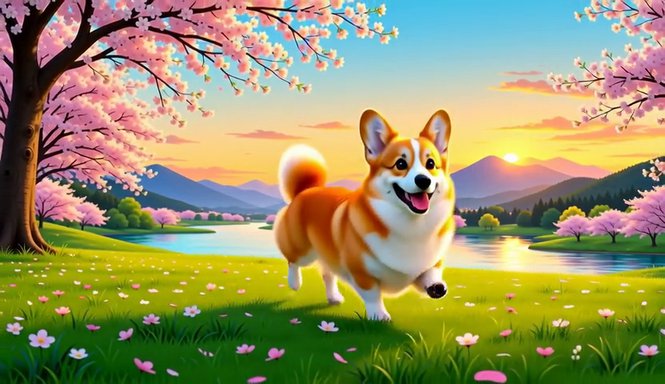}
      & \framecell{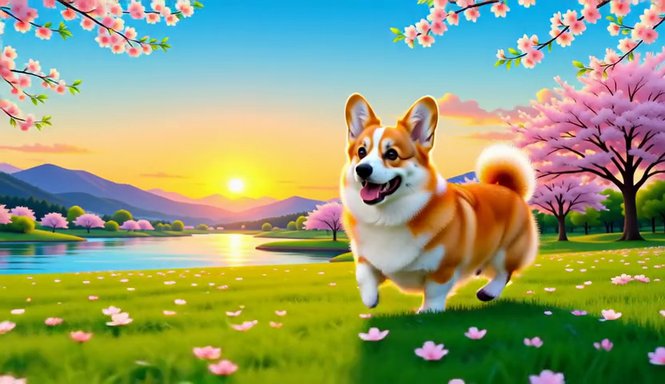}
      & \framecell{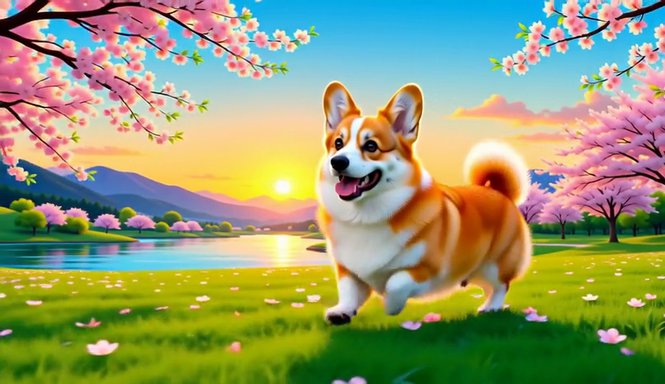}
      & \framecell{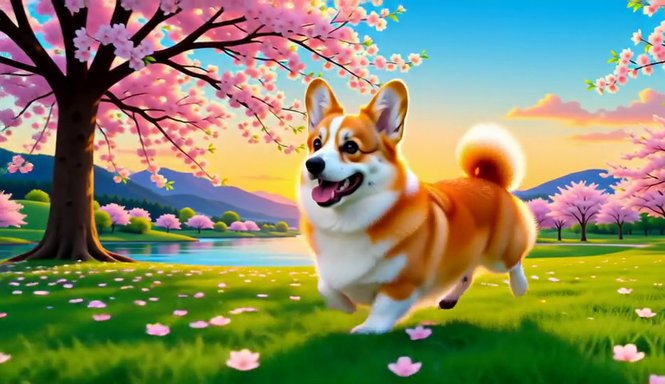} 
      & \phantom{AnyFlow} \\ \addlinespace[2pt]
      & \multicolumn{6}{@{} p{16.6cm} @{}}{\tiny\itshape ``A happy, playful Corgi running and frolicking in a park during sunset, inspired by Hokusai's style of Ukiyo-e. The Corgi has a wagging tail and joyful expression, bounding through the grass. The background showcases a serene landscape with cherry blossoms blooming, distant mountains, and a tranquil river. The sky is painted with vibrant hues of orange and pink, reflecting the beauty of a Japanese sunset. The composition follows the traditional Ukiyo-e aesthetic with bold outlines and flat colors. Medium shot capturing the lively energy of the Corgi amidst the picturesque scenery.''} & \\
  \end{tabular}%
  }

  \vspace{-7.5pt}
  \caption{Wan2.1 14B: PDD - Midpoint (ours) vs. DMD2 (FastGen) and AnyFlow with 4 NFE.}
  \label{fig:corgi_park_hokusai}
\end{figure*}

%% file: figures/wan/diversity_tables/horse_drinking_river/frames_table.tex
\begin{figure*}[t]
  \centering

  \resizebox{\textwidth}{!}{%

  \begin{tabular}{@{} r @{\hspace{6pt}} c@{\hspace{2pt}}c@{\hspace{2pt}}c @{\hspace{4pt}} c@{\hspace{2pt}}c@{\hspace{2pt}}c @{\hspace{6pt}} c @{}}
                & \multicolumn{3}{c}{\large Noise 1}
                & \multicolumn{3}{c}{\large Noise 2} & \\

                \cmidrule(l{6pt}r{2pt}){2-4} \cmidrule(l{2pt}r{6pt}){5-7}

                & $t{=}0.5$s & $t{=}2.5$s & $t{=}4.5$s & $t{=}0.5$s & $t{=}2.5$s & $t{=}4.5$s & \\
    \textbf{PDD}
      & \framecell{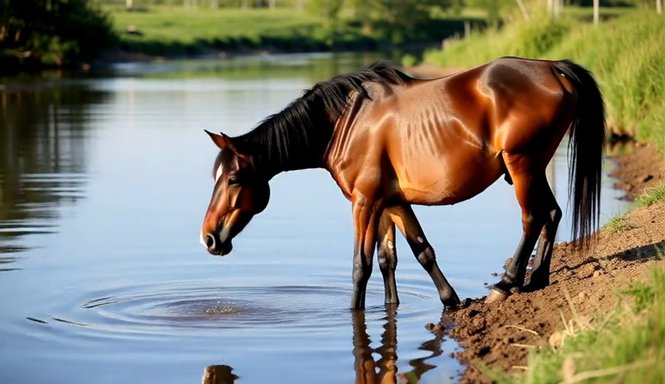}
      & \framecell{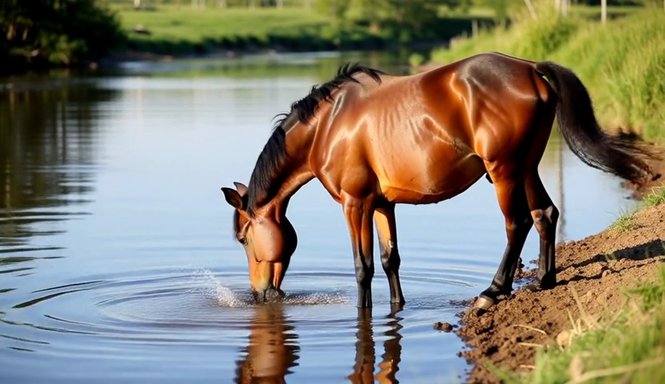}
      & \framecell{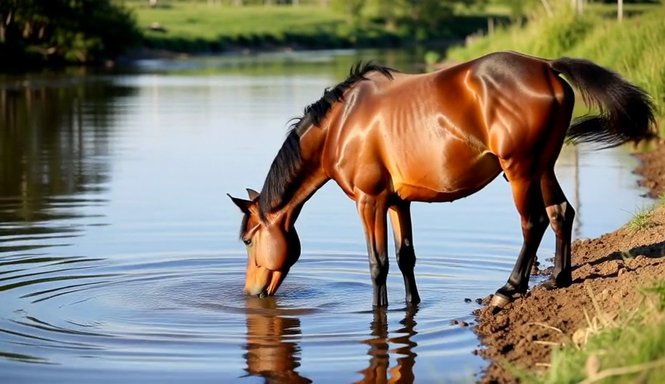}
      & \framecell{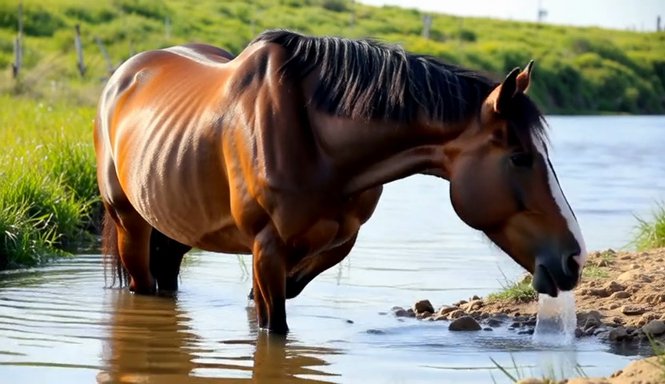}
      & \framecell{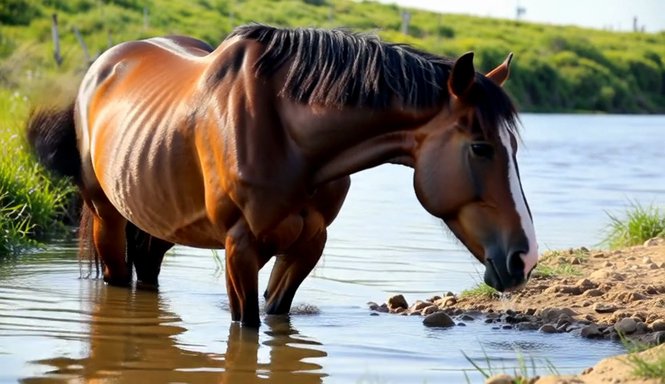}
      & \framecell{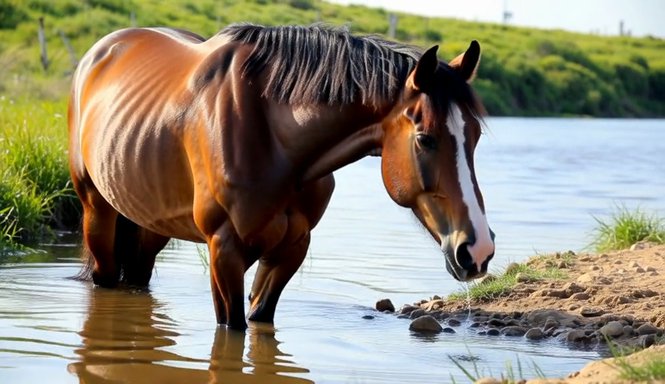} 
      & \phantom{AnyFlow} \\ \addlinespace[2pt]
    DMD2
      & \framecell{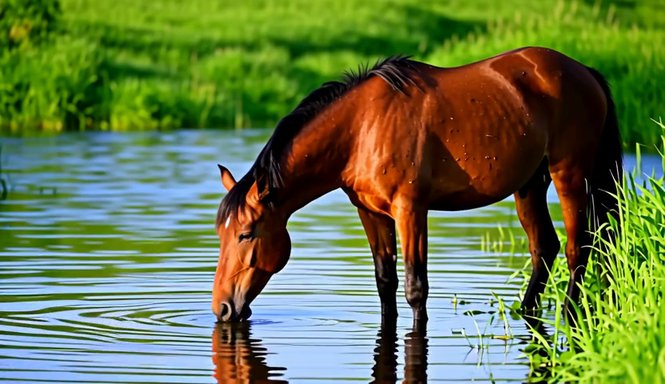}
      & \framecell{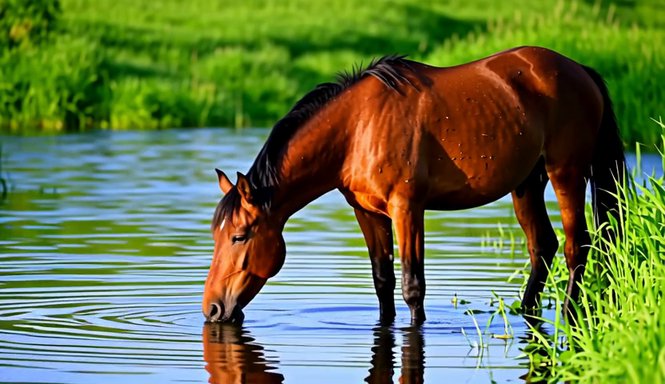}
      & \framecell{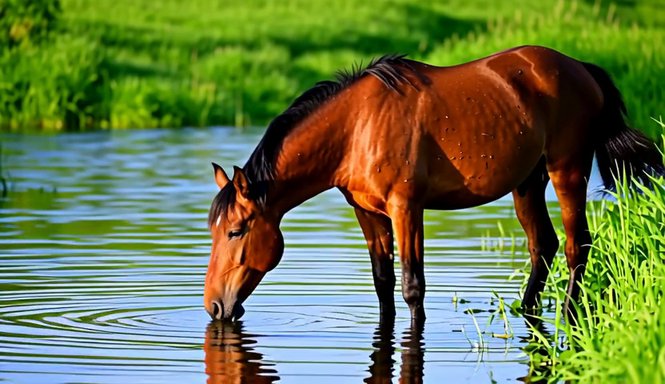}
      & \framecell{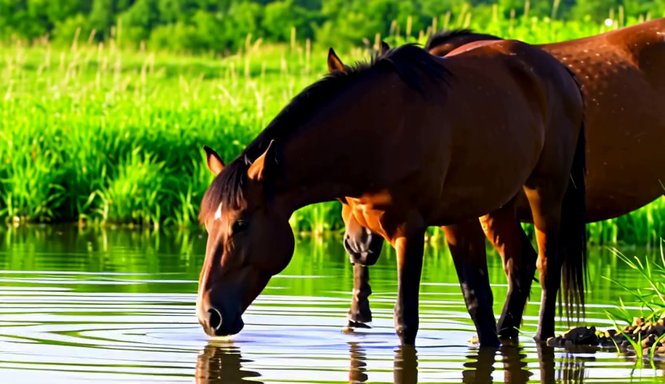}
      & \framecell{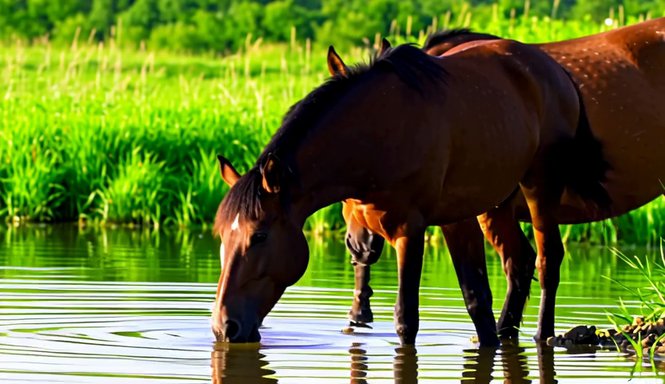}
      & \framecell{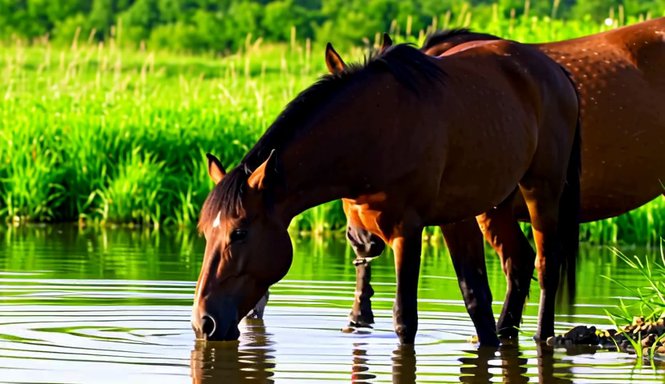} 
      & \phantom{AnyFlow} \\ \addlinespace[2pt]
    AnyFlow
      & \framecell{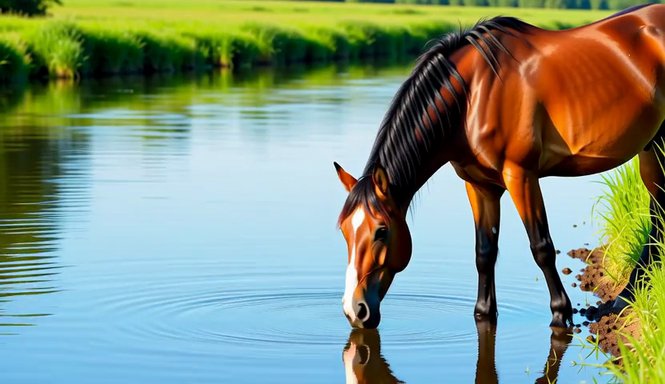}
      & \framecell{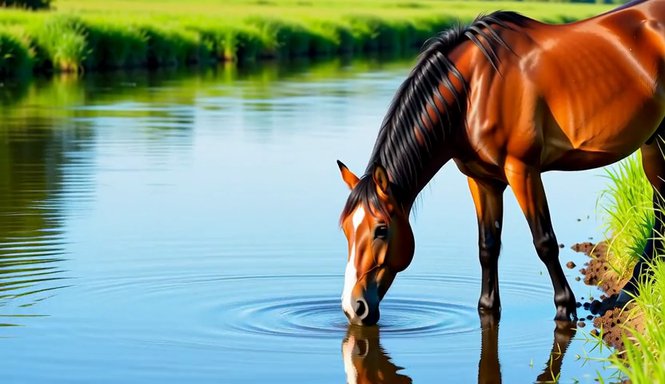}
      & \framecell{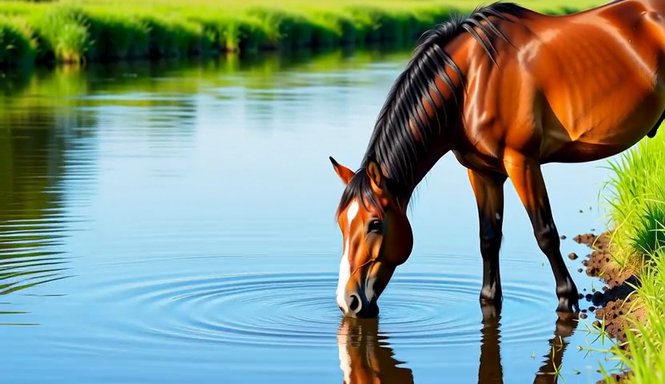}
      & \framecell{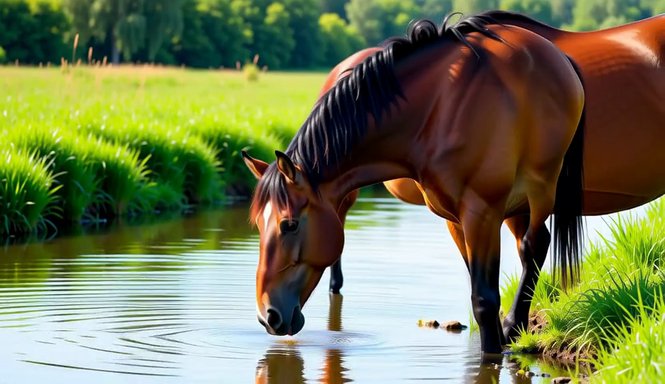}
      & \framecell{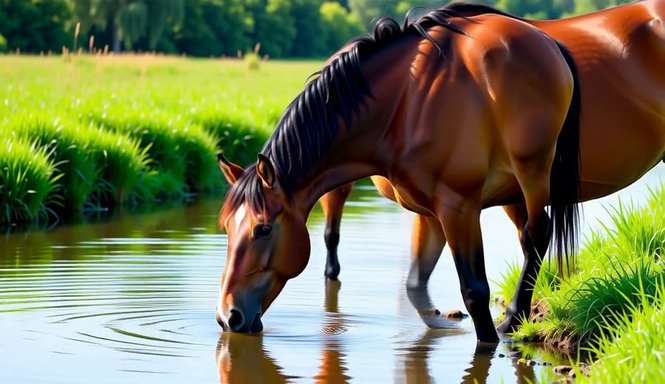}
      & \framecell{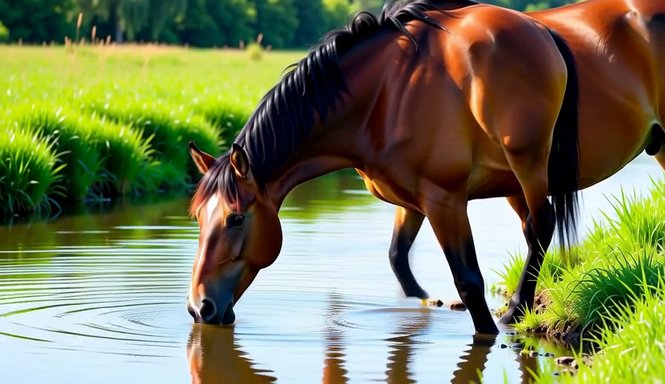} 
      & \phantom{AnyFlow} \\ \addlinespace[2pt]
      & \multicolumn{6}{@{} p{16.6cm} @{}}{\tiny\itshape ``A realistic, serene scene of a majestic brown horse bending down to drink water from a calm river. The horse has a flowing mane and tail, and its muscular body is clearly visible. It stands near the riverbank, with its front legs slightly bent and its muzzle touching the water surface. The river reflects the gentle sunlight, creating a tranquil atmosphere. The background showcases lush green grass and trees along the riverbank, adding depth to the scene. The camera focuses on a close-up of the horse’s face and neck as it drinks, capturing the serene moment in a static shot.''} & \\
  \end{tabular}%
  }

  \vspace{-7.5pt}
  \label{fig:horse_drinking_river}
\end{figure*}

%% file: figures/wan/diversity_tables/person_bungee_jumping/frames_table.tex
\begin{figure*}[t]
  \centering

  \resizebox{\textwidth}{!}{%

  \begin{tabular}{@{} r @{\hspace{6pt}} c@{\hspace{2pt}}c@{\hspace{2pt}}c @{\hspace{4pt}} c@{\hspace{2pt}}c@{\hspace{2pt}}c @{\hspace{6pt}} c @{}}
                & \multicolumn{3}{c}{\large Noise 1}
                & \multicolumn{3}{c}{\large Noise 2} & \\

                \cmidrule(l{6pt}r{2pt}){2-4} \cmidrule(l{2pt}r{6pt}){5-7}

                & $t{=}0.5$s & $t{=}2.5$s & $t{=}4.5$s & $t{=}0.5$s & $t{=}2.5$s & $t{=}4.5$s & \\
    \textbf{PDD}
      & \framecell{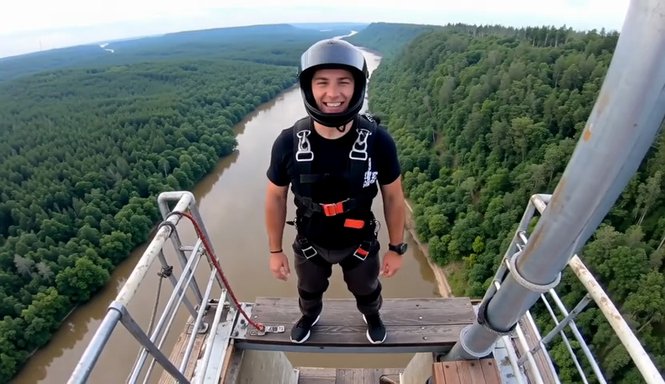}
      & \framecell{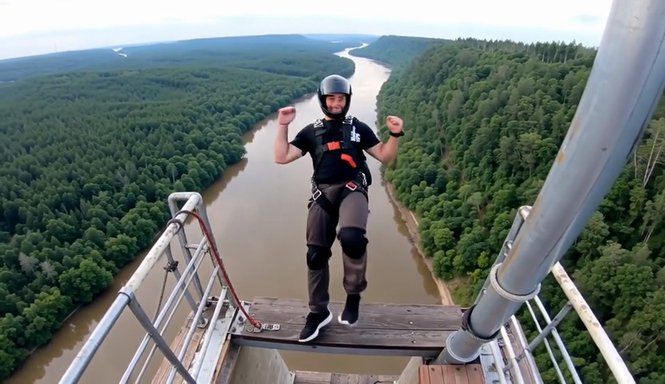}
      & \framecell{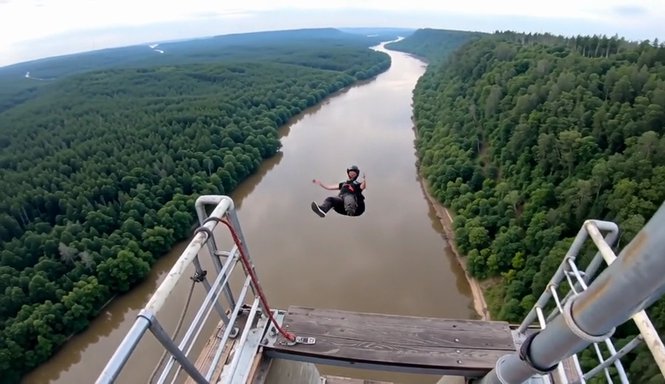}
      & \framecell{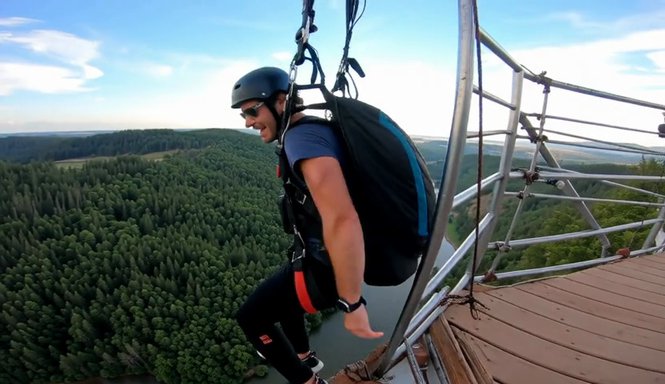}
      & \framecell{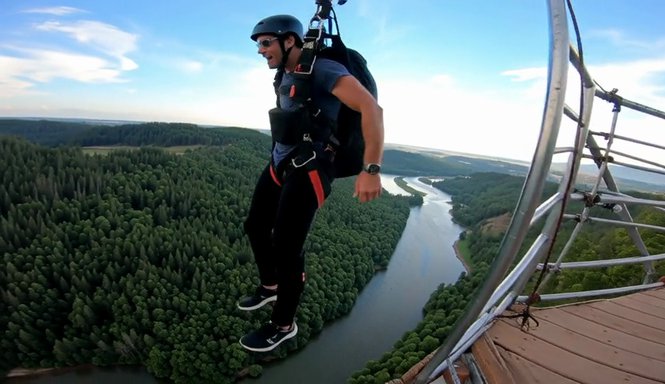}
      & \framecell{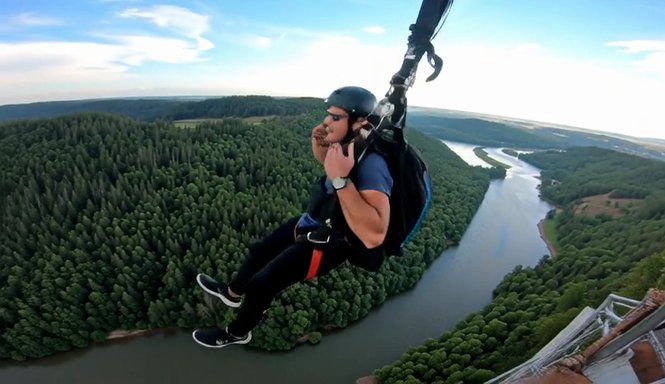} 
      & \phantom{AnyFlow} \\ \addlinespace[2pt]
    DMD2
      & \framecell{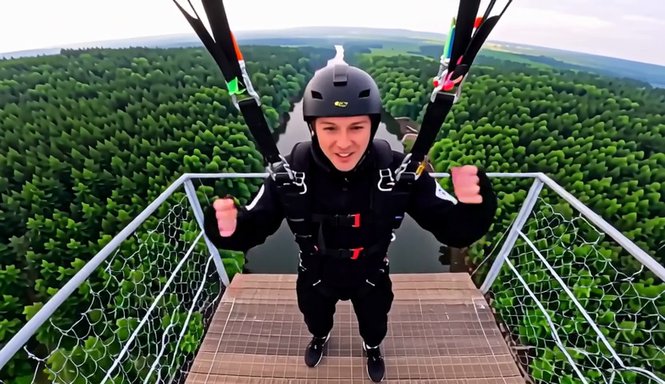}
      & \framecell{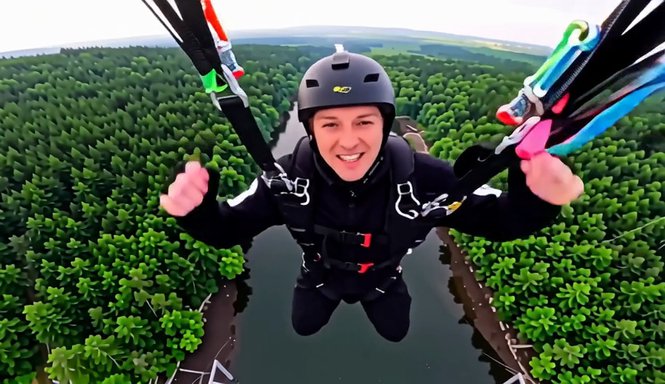}
      & \framecell{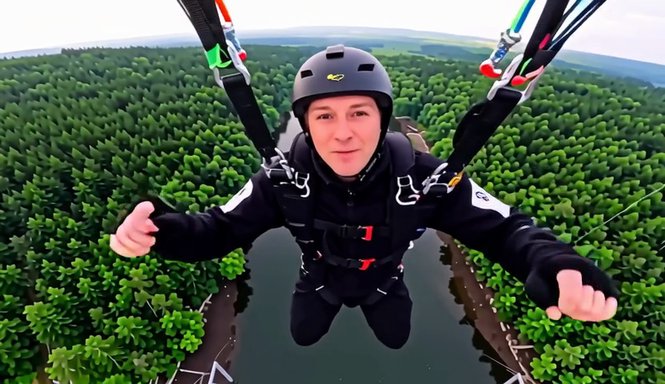}
      & \framecell{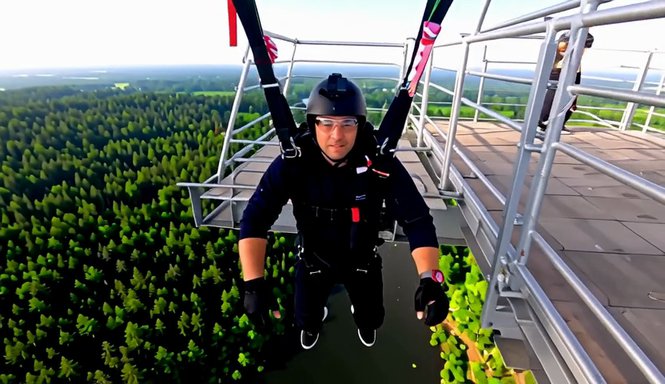}
      & \framecell{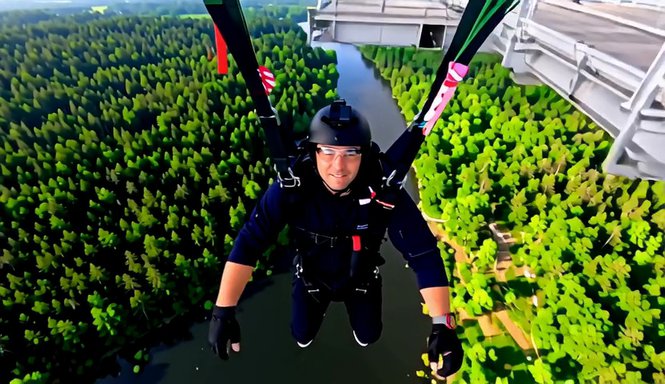}
      & \framecell{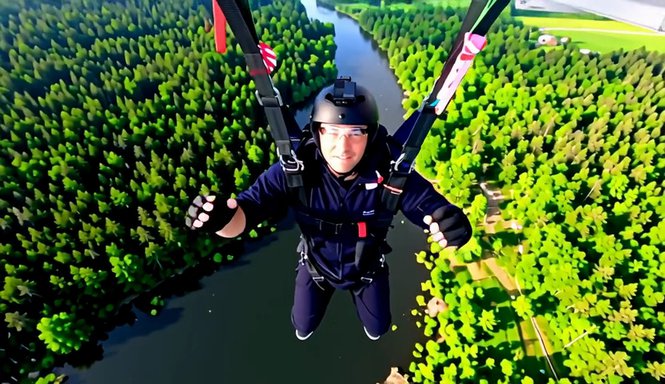} 
      & \phantom{AnyFlow} \\ \addlinespace[2pt]
    AnyFlow
      & \framecell{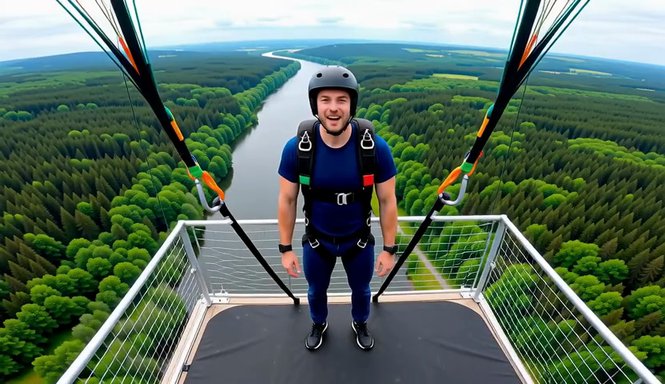}
      & \framecell{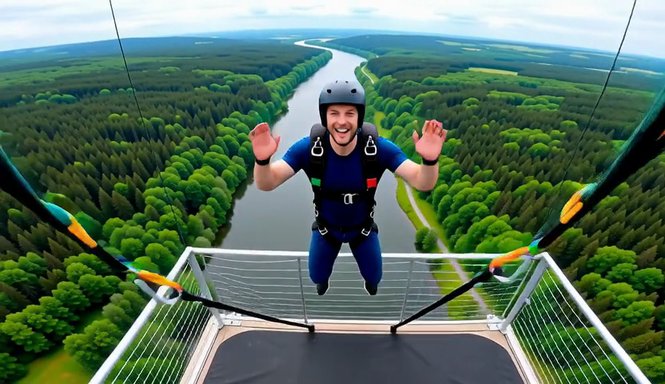}
      & \framecell{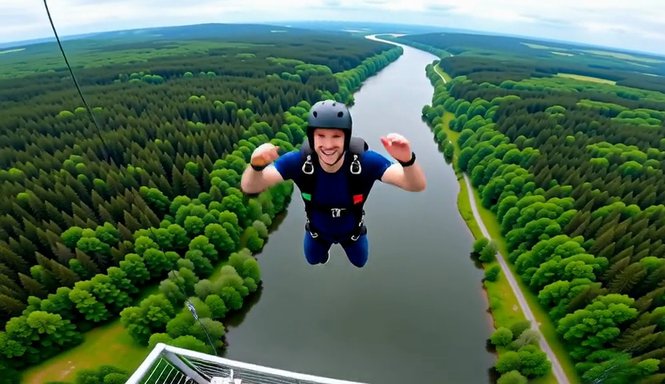}
      & \framecell{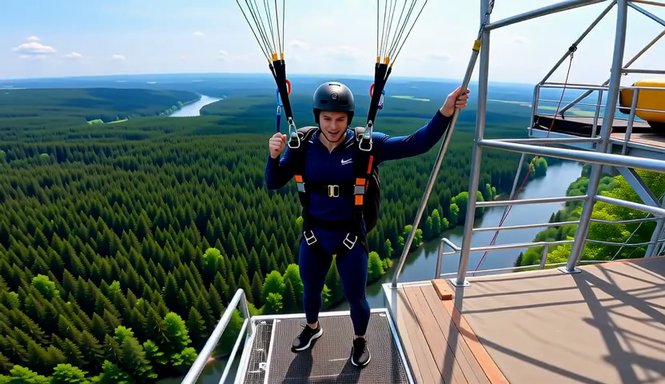}
      & \framecell{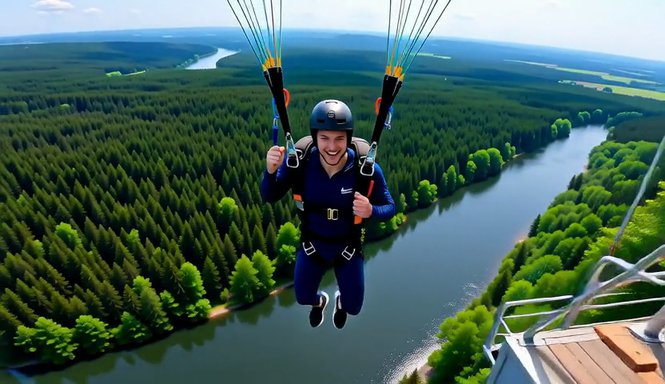}
      & \framecell{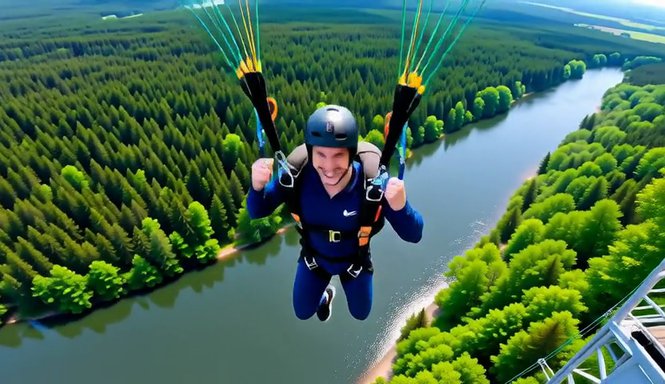} 
      & \phantom{AnyFlow} \\ \addlinespace[2pt]
      & \multicolumn{6}{@{} p{16.6cm} @{}}{\tiny\itshape ``A person is performing a thrilling bungee jump from a tall platform. The individual is standing at the edge, looking down with a mix of excitement and anticipation. They are wearing a full-body harness and helmet, with a bungee cord attached securely. The platform is high above a river surrounded by dense green forests, creating a dramatic backdrop. As the person jumps, the camera captures their descent, focusing on their facial expressions and the surrounding landscape. The scene includes the initial jump and the subsequent bounce, showcasing the exhilarating experience. Wide shot followed by a close-up during the jump.''} & \\
  \end{tabular}%
  }

  \vspace{-7.5pt}
  \label{fig:person_bungee_jumping}
\end{figure*}

%% file: figures/wan/diversity_tables/yoda_guitar_stage/frames_table.tex
\begin{figure*}[t]
  \centering

  \resizebox{\textwidth}{!}{%

  \begin{tabular}{@{} r @{\hspace{6pt}} c@{\hspace{2pt}}c@{\hspace{2pt}}c @{\hspace{4pt}} c@{\hspace{2pt}}c@{\hspace{2pt}}c @{\hspace{6pt}} c @{}}
                & \multicolumn{3}{c}{\large Noise 1}
                & \multicolumn{3}{c}{\large Noise 2} & \\

                \cmidrule(l{6pt}r{2pt}){2-4} \cmidrule(l{2pt}r{6pt}){5-7}

                & $t{=}0.5$s & $t{=}2.5$s & $t{=}4.5$s & $t{=}0.5$s & $t{=}2.5$s & $t{=}4.5$s & \\
    \textbf{PDD}
      & \framecell{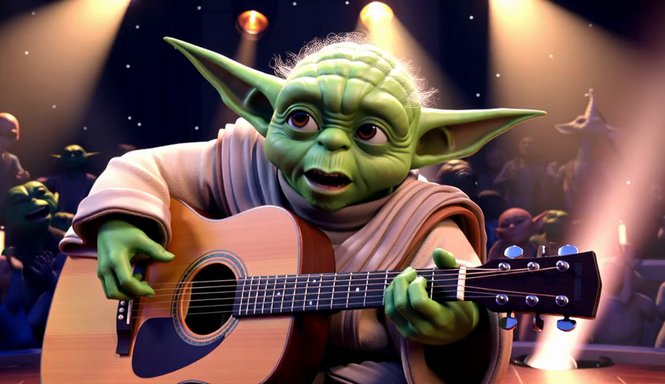}
      & \framecell{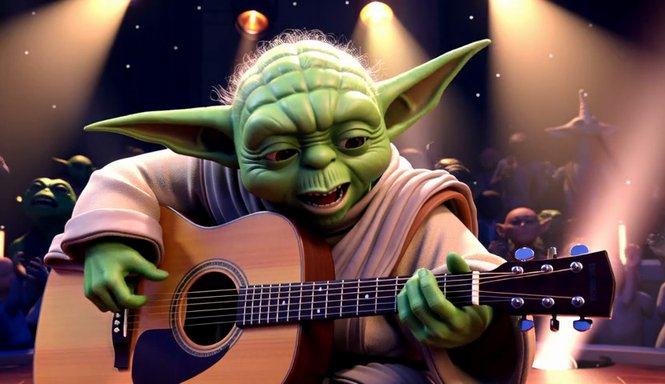}
      & \framecell{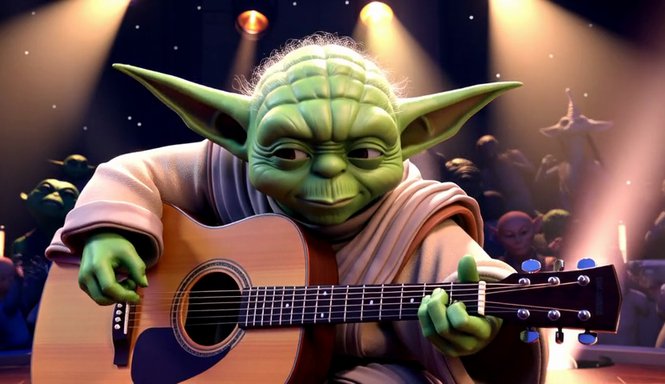}
      & \framecell{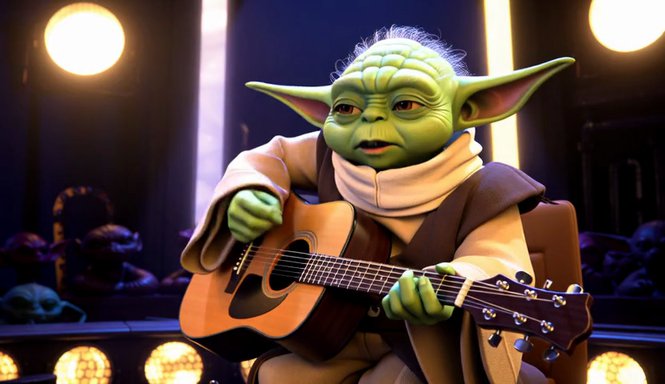}
      & \framecell{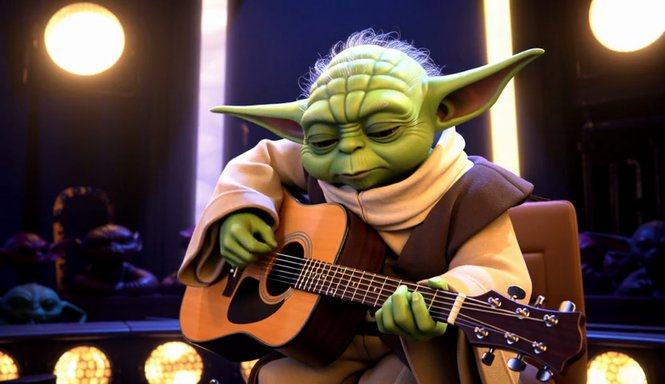}
      & \framecell{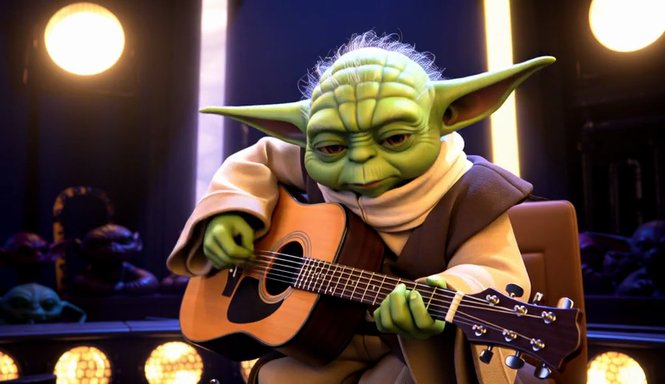} 
      & \phantom{AnyFlow} \\ \addlinespace[2pt]
    DMD2
      & \framecell{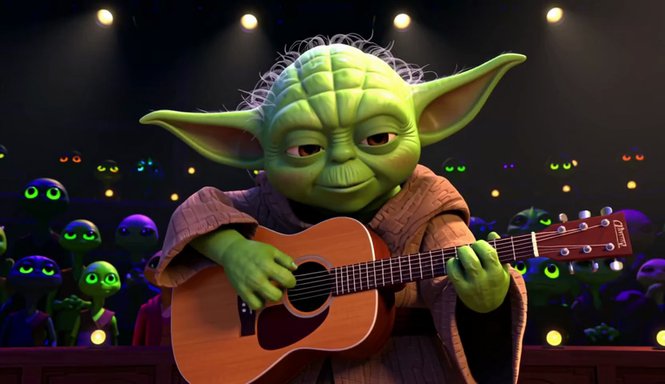}
      & \framecell{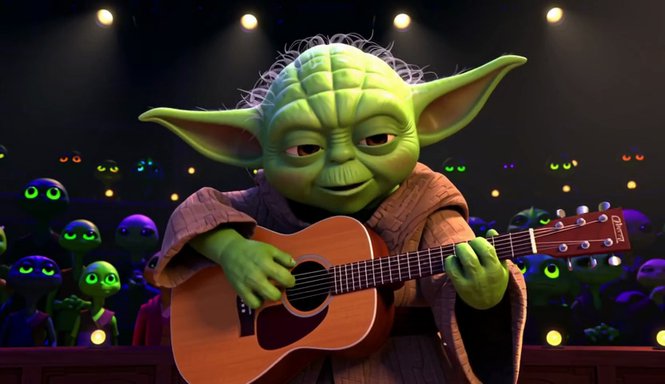}
      & \framecell{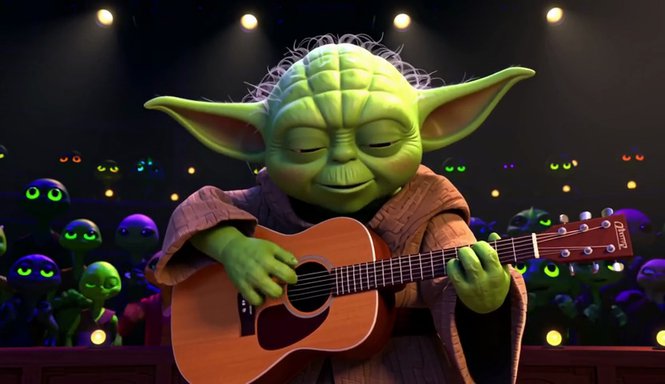}
      & \framecell{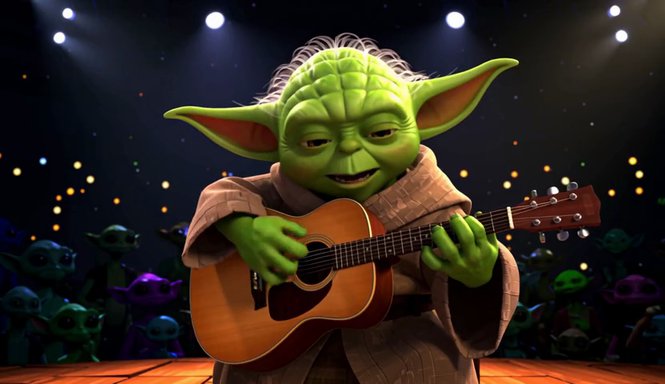}
      & \framecell{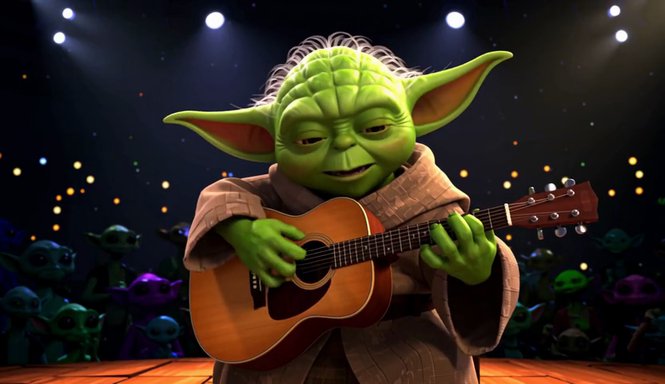}
      & \framecell{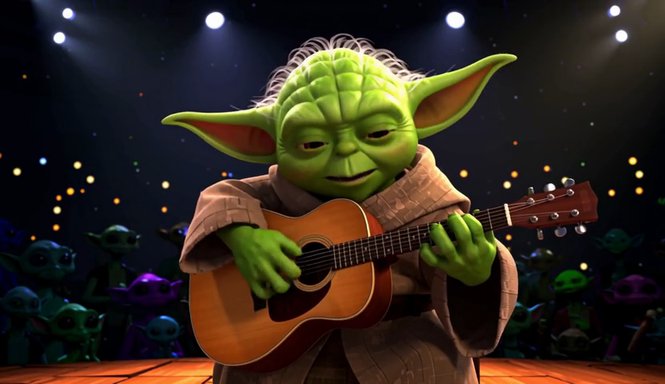} 
      & \phantom{AnyFlow} \\ \addlinespace[2pt]
    AnyFlow
      & \framecell{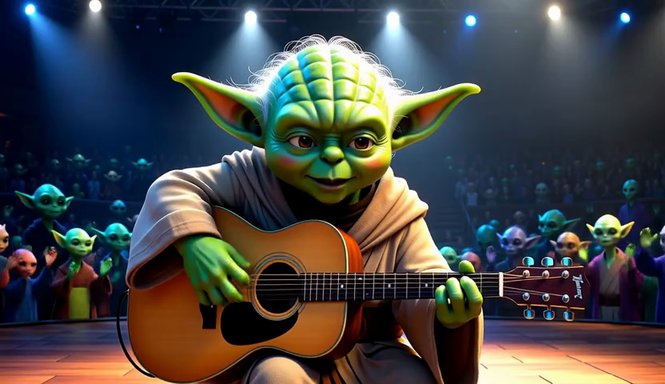}
      & \framecell{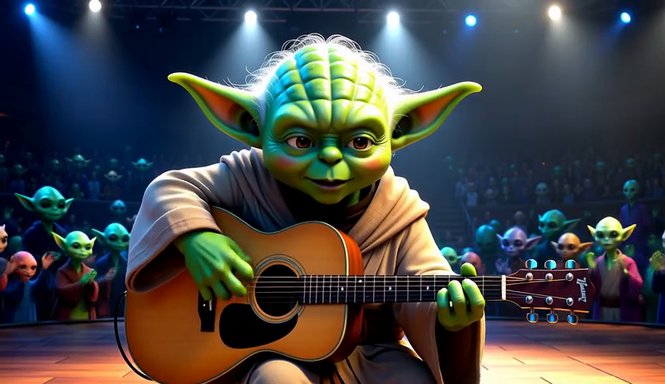}
      & \framecell{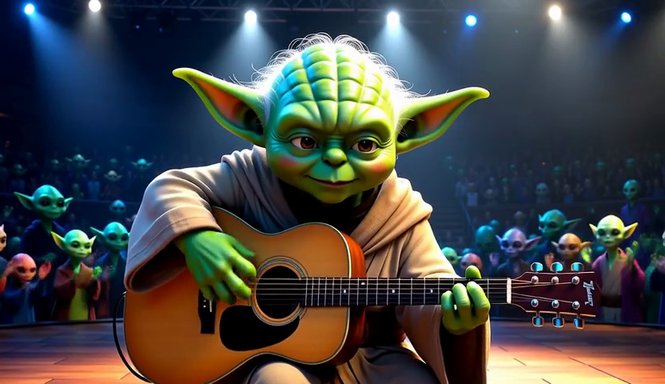}
      & \framecell{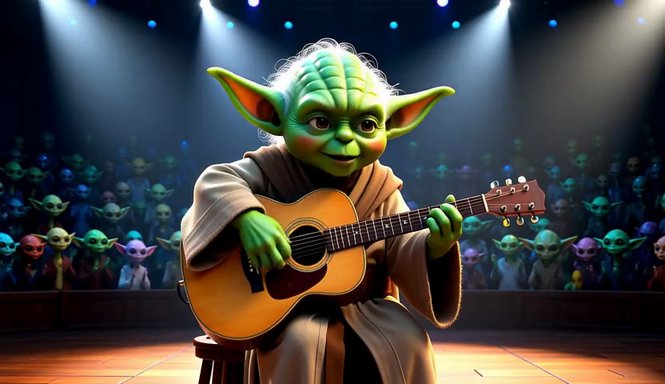}
      & \framecell{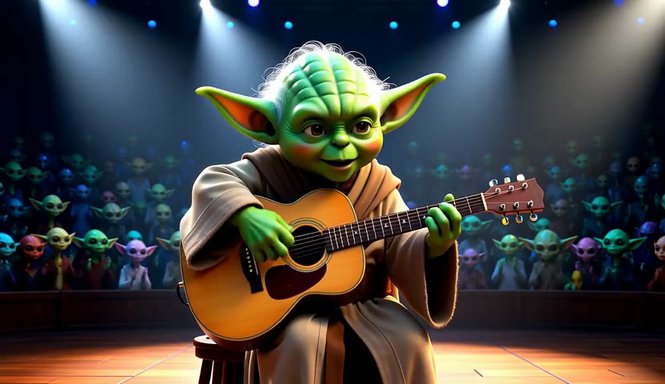}
      & \framecell{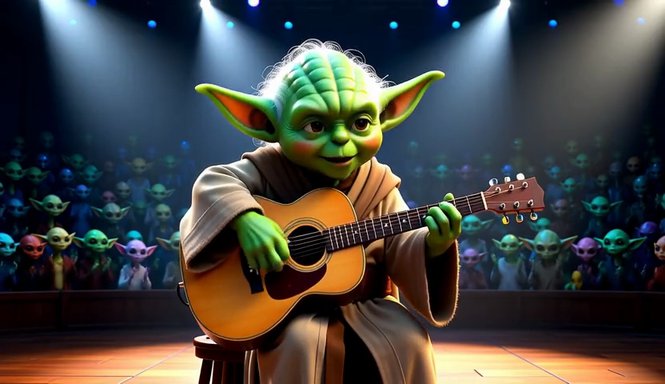} 
      & \phantom{AnyFlow} \\ \addlinespace[2pt]
      & \multicolumn{6}{@{} p{16.6cm} @{}}{\tiny\itshape ``A CGI animation of Yoda, the iconic green Jedi master, playing a guitar on a stage. Yoda is dressed in his traditional Jedi robes and has a playful, focused expression as he strums the guitar strings. The stage is dimly lit with spotlights shining down on him, creating dramatic shadows. The background features a live audience with various alien species cheering him on. Yoda is sitting on a stool, leaning slightly forward, and his fingers move nimbly over the guitar. Medium close-up shot focusing on Yoda's face and hands interacting with the guitar.''} & \\
  \end{tabular}%
  }

  \vspace{-7.5pt}
  \label{fig:yoda_guitar_stage}
\end{figure*}

%% file: figures/wan/diversity_tables/robot_dj_cyberpunk_tokyo/frames_table.tex
\begin{figure*}[t]
  \centering

  \resizebox{\textwidth}{!}{%

  \begin{tabular}{@{} r @{\hspace{6pt}} c@{\hspace{2pt}}c@{\hspace{2pt}}c @{\hspace{4pt}} c@{\hspace{2pt}}c@{\hspace{2pt}}c @{\hspace{6pt}} c @{}}
                & \multicolumn{3}{c}{\large Noise 1}
                & \multicolumn{3}{c}{\large Noise 2} & \\

                \cmidrule(l{6pt}r{2pt}){2-4} \cmidrule(l{2pt}r{6pt}){5-7}

                & $t{=}0.5$s & $t{=}2.5$s & $t{=}4.5$s & $t{=}0.5$s & $t{=}2.5$s & $t{=}4.5$s & \\
    \textbf{PDD}
      & \framecell{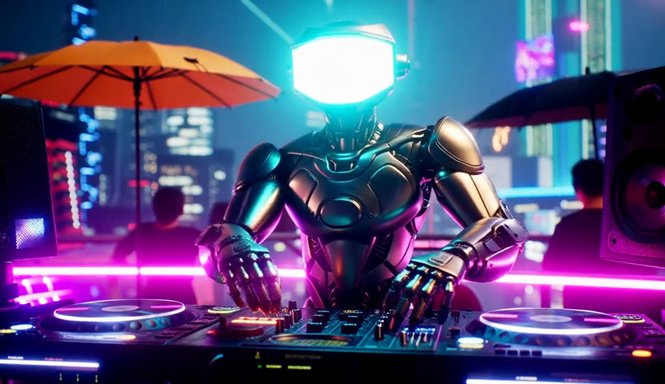}
      & \framecell{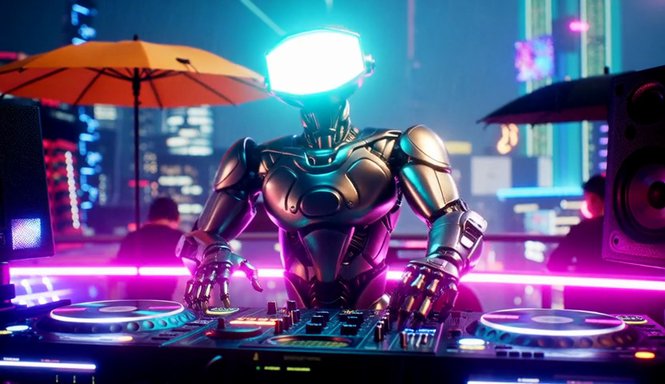}
      & \framecell{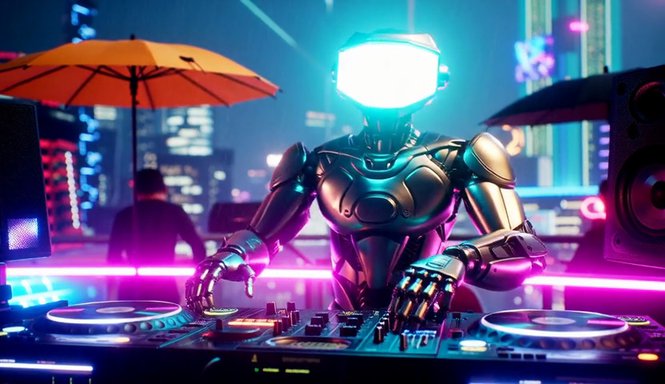}
      & \framecell{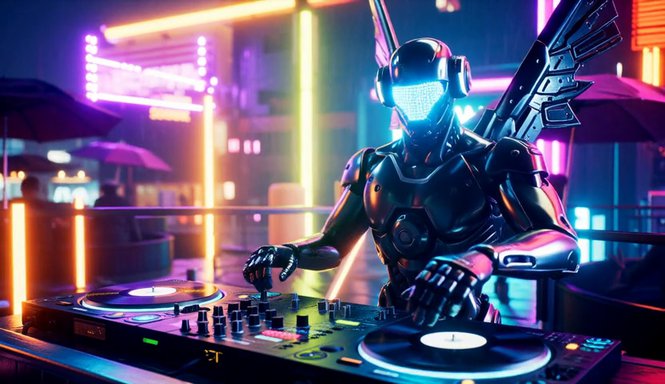}
      & \framecell{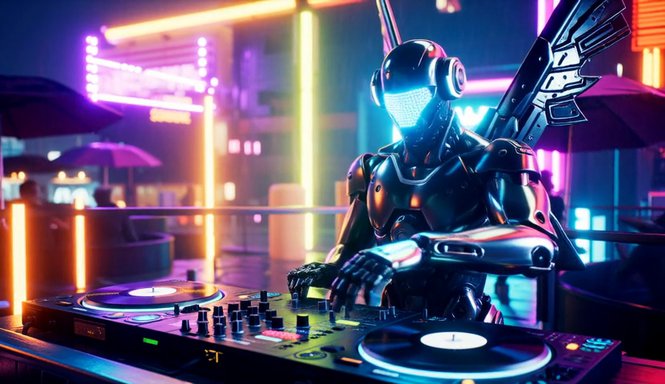}
      & \framecell{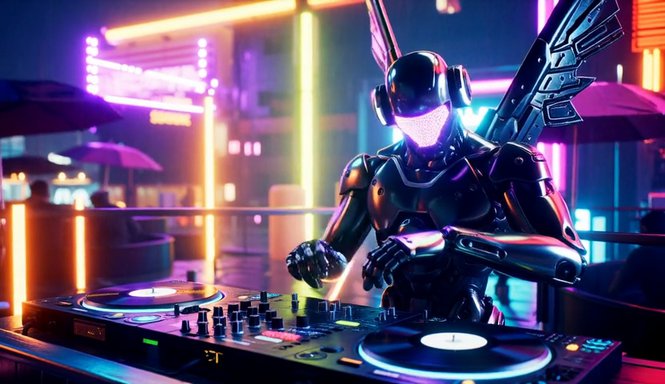} 
      & \phantom{AnyFlow} \\ \addlinespace[2pt]
    DMD2
      & \framecell{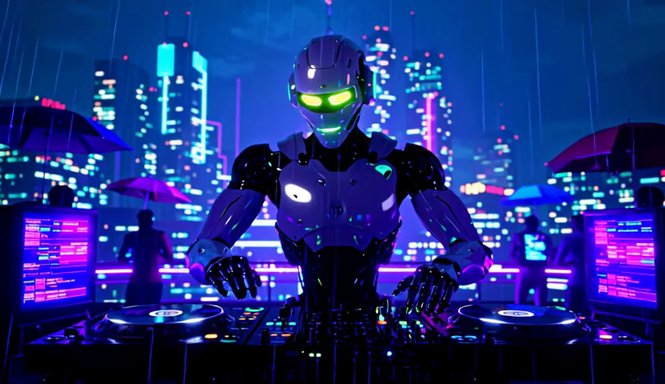}
      & \framecell{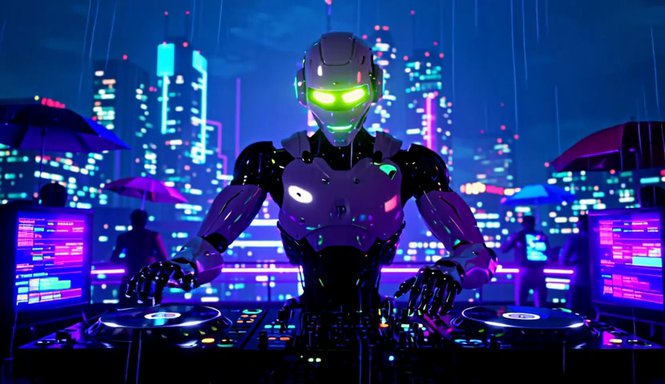}
      & \framecell{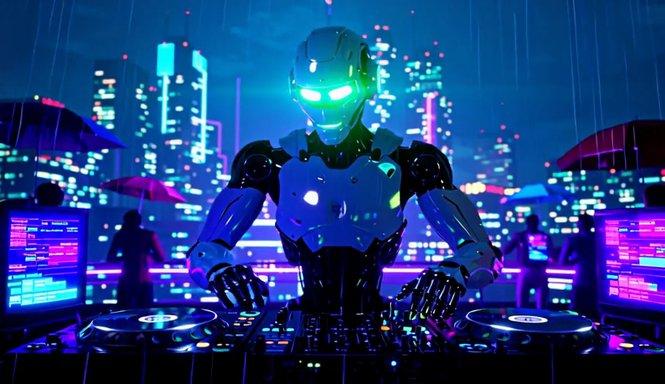}
      & \framecell{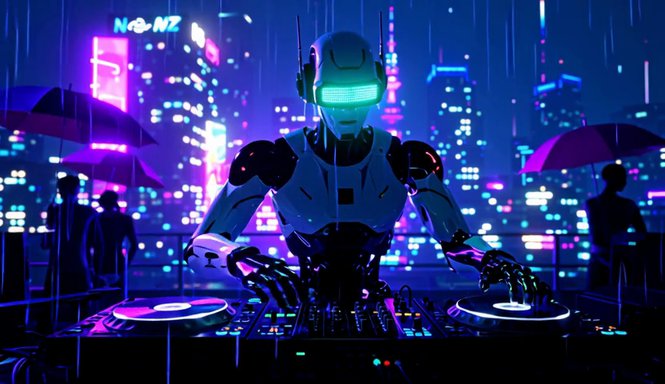}
      & \framecell{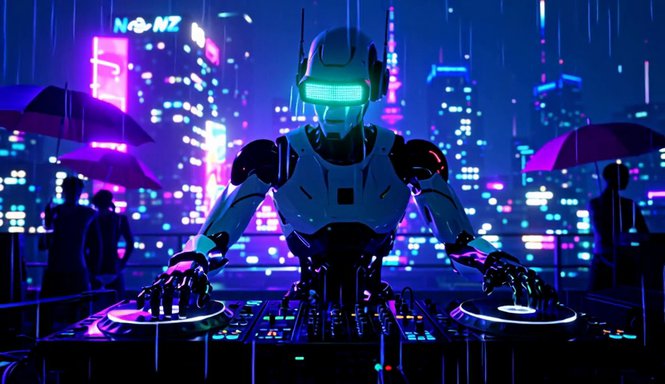}
      & \framecell{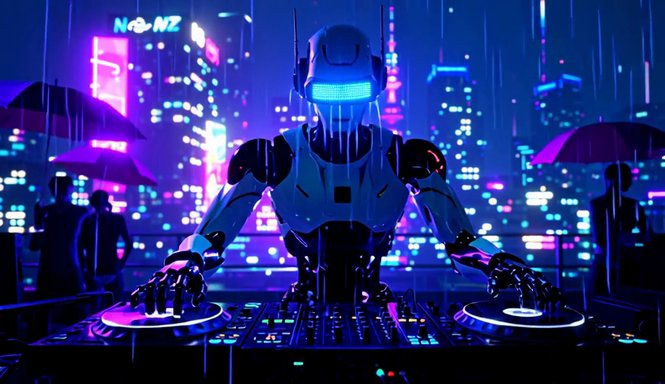} 
      & \phantom{AnyFlow} \\ \addlinespace[2pt]
    AnyFlow
      & \framecell{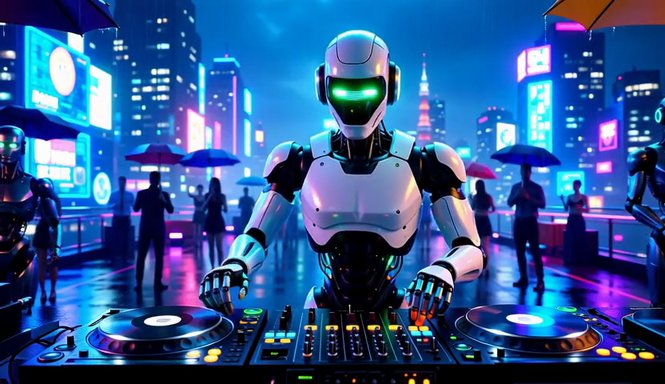}
      & \framecell{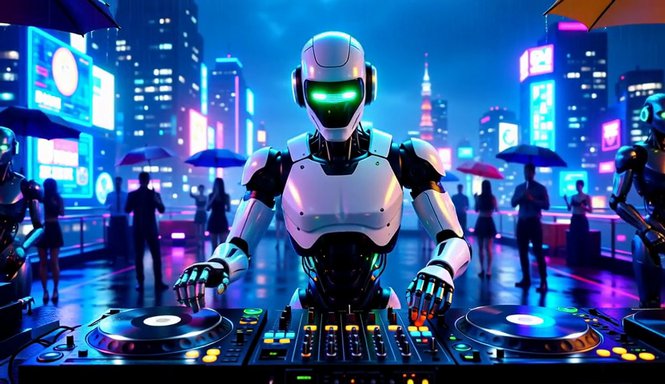}
      & \framecell{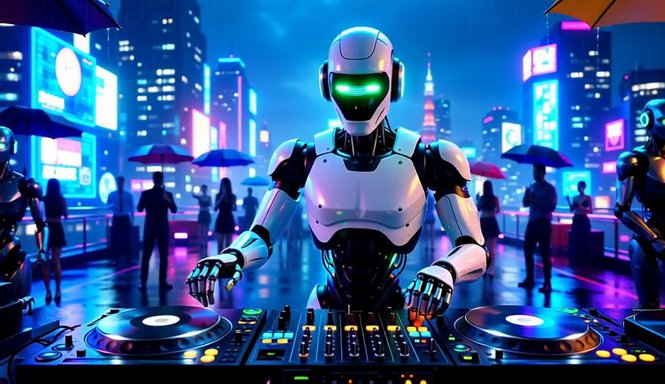}
      & \framecell{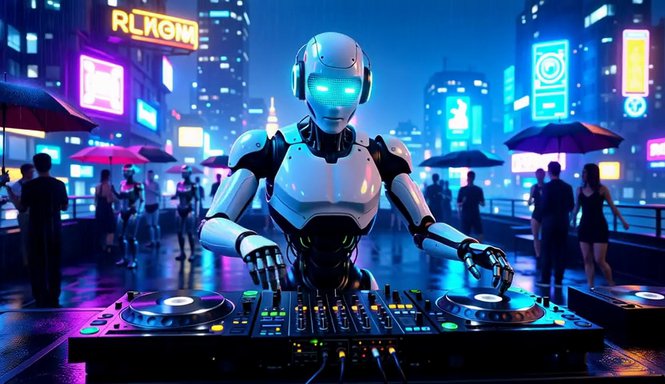}
      & \framecell{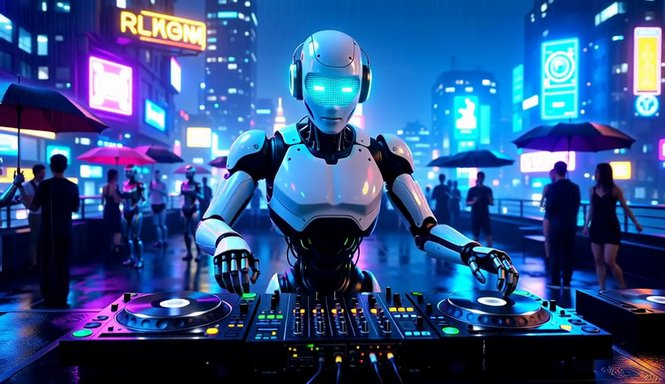}
      & \framecell{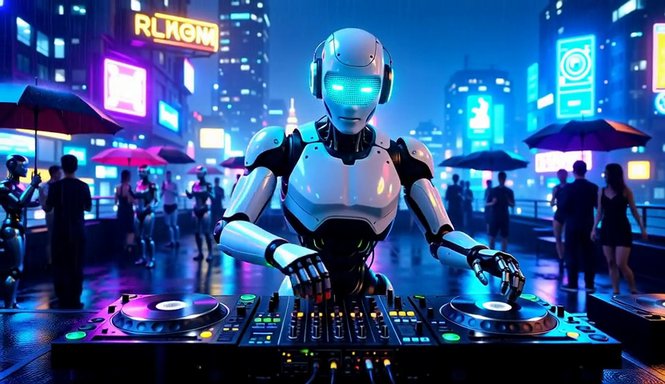} 
      & \phantom{AnyFlow} \\ \addlinespace[2pt]
      & \multicolumn{6}{@{} p{16.6cm} @{}}{\tiny\itshape ``In a futuristic Tokyo rooftop set during a heavy rainstorm, a robot DJ is spinning records at a cyberpunk nightclub. The robot has sleek metallic arms and a glowing LED screen for a face, expressing intense concentration as it mixes tracks. The DJ is surrounded by neon lights and holographic displays, casting colorful reflections off the wet surroundings. The rain creates a misty atmosphere, enhancing the dystopian feel. The scene includes other robots and humans dancing under umbrellas, adding dynamic movement. Close-up shots of the DJ's hands manipulating the turntables and a wide shot capturing the lively rooftop party. Sci-fi, fantasy style.''} & \\
  \end{tabular}%
  }

  \vspace{-7.5pt}
  \caption{Wan2.1 14B: PDD - Midpoint (ours) vs. DMD2 (FastGen) and AnyFlow with 4 NFE.}
  \label{fig:robot_dj_cyberpunk_tokyo}
\end{figure*}

%% file: figures/ltx/frames_tables.tex
\begin{figure*}[htbp]
  \centering
  \resizebox{\textwidth}{!}{%
  \begin{tabular}{@{} r @{\hspace{6pt}} c@{\hspace{2pt}}c@{\hspace{2pt}}c@{\hspace{2pt}}c@{\hspace{2pt}}c @{}}
                & $t{=}0$s & $t{=}2.5$s & $t{=}5$s & $t{=}7.5$s & $t{=}10$s \\ \addlinespace[2pt]
    Teacher
      & \framecellw{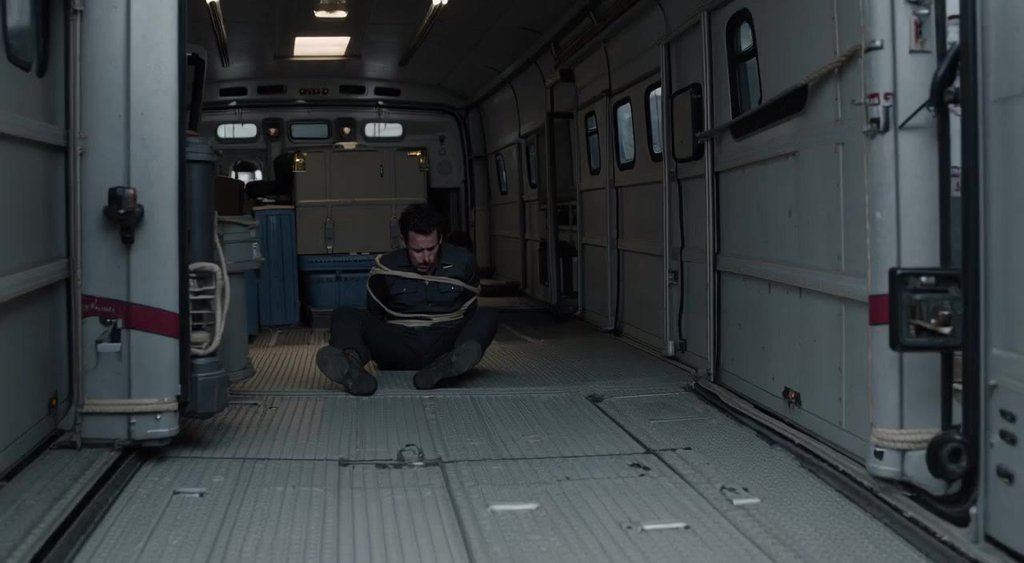} & \framecellw{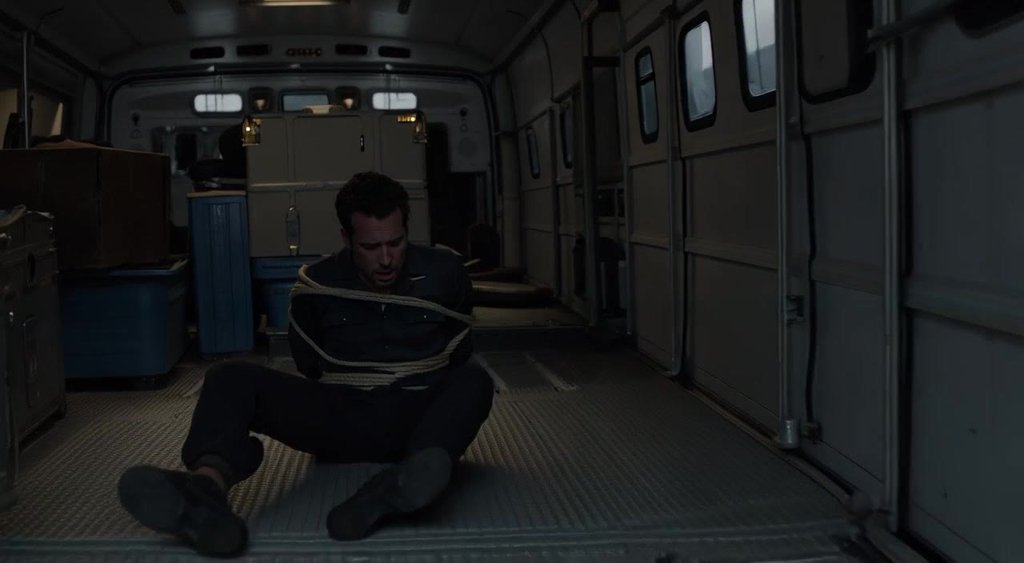} & \framecellw{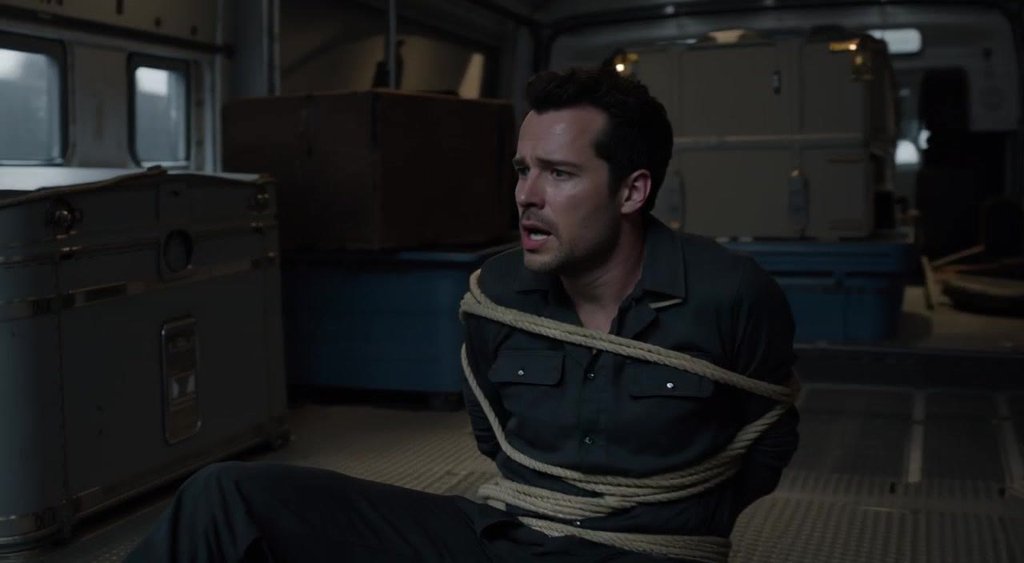} & \framecellw{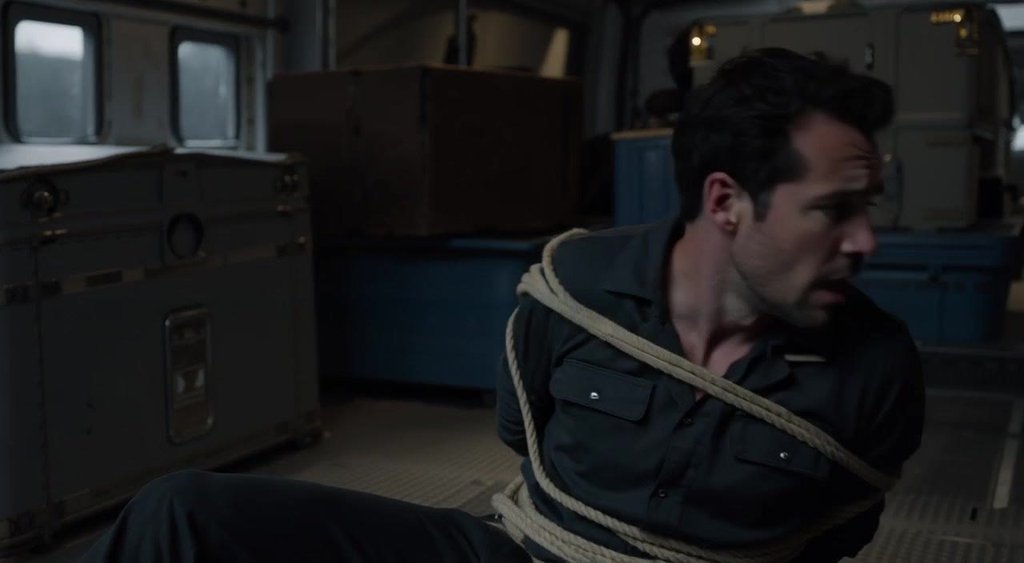} & \framecellw{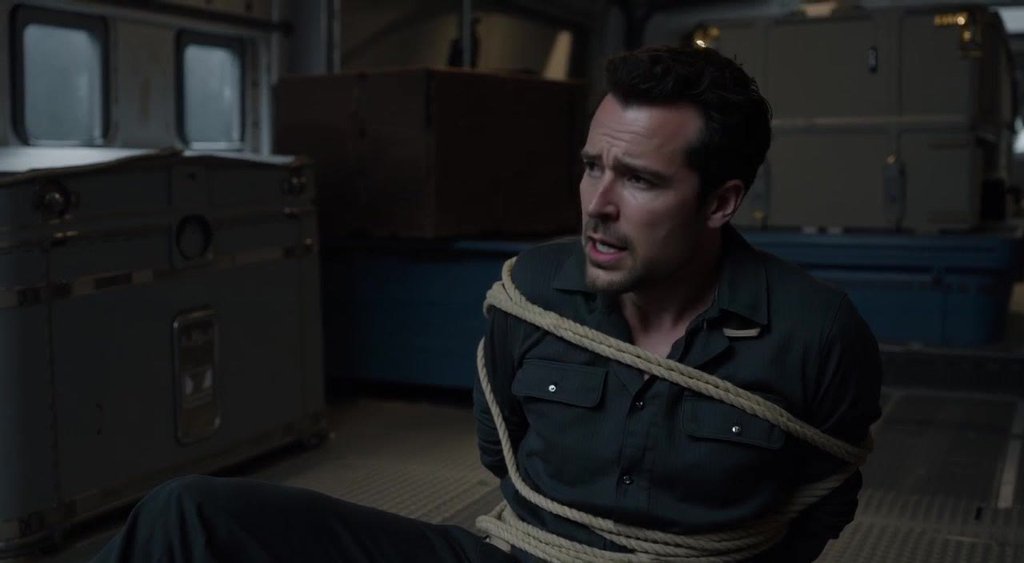} \\ \addlinespace[2pt]
    \textbf{PDD}
      & \framecellw{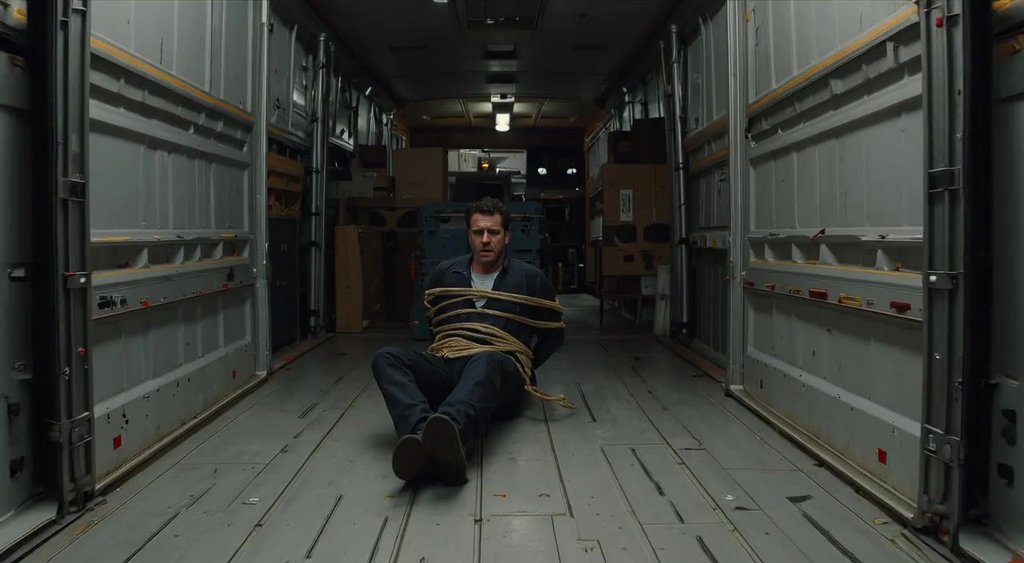} & \framecellw{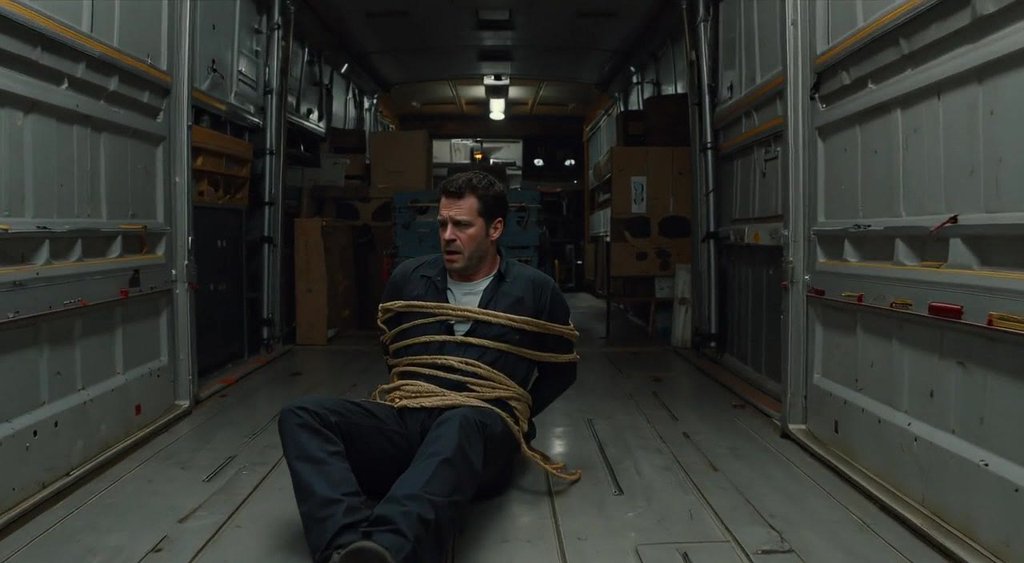} & \framecellw{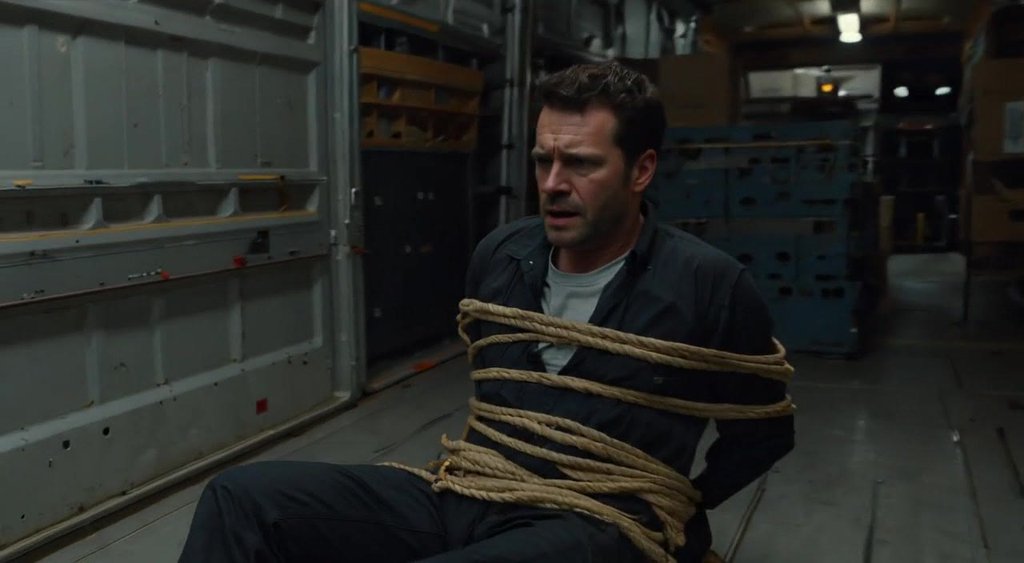} & \framecellw{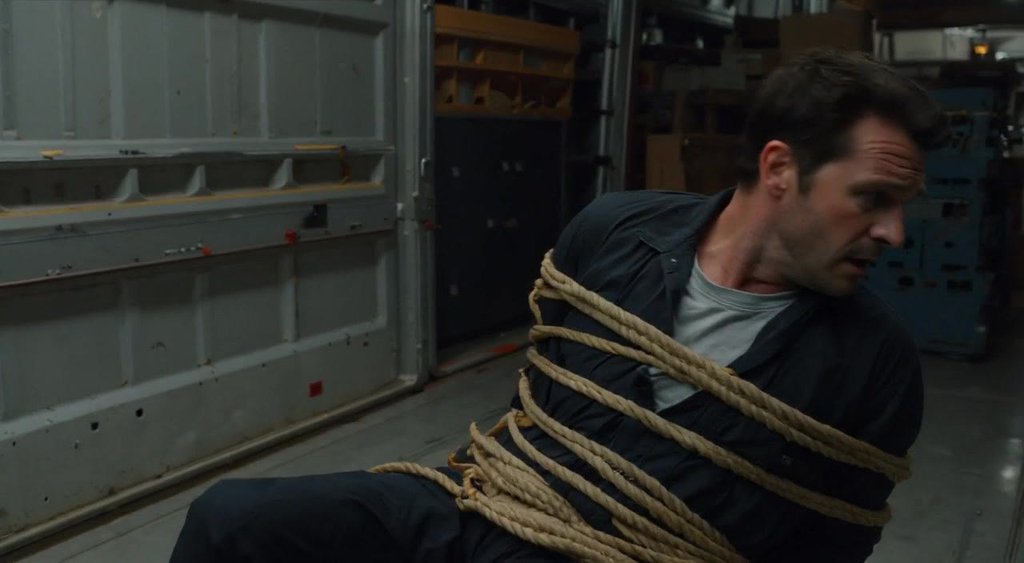} & \framecellw{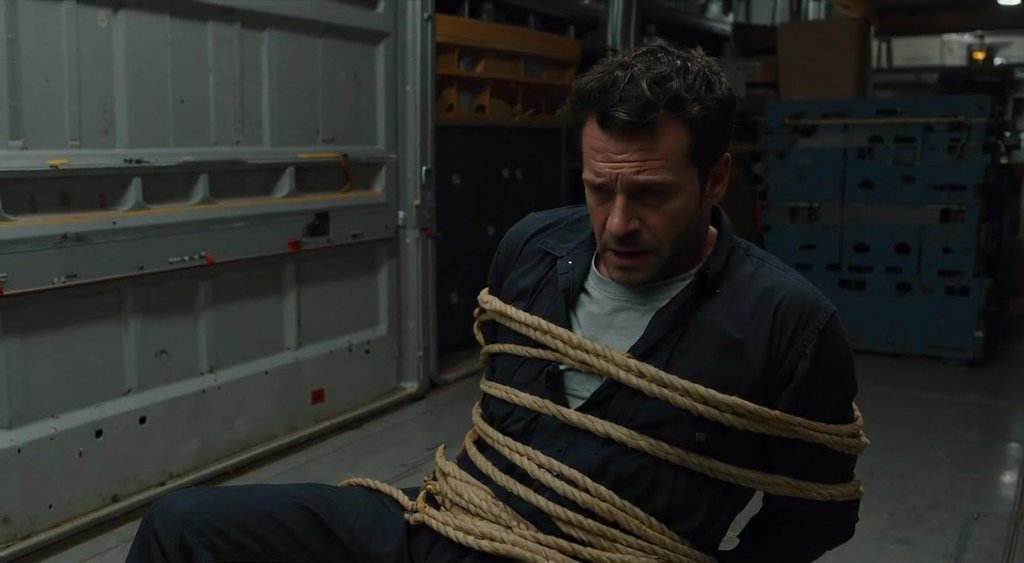} \\ \addlinespace[2pt]
    Distilled
      & \framecellw{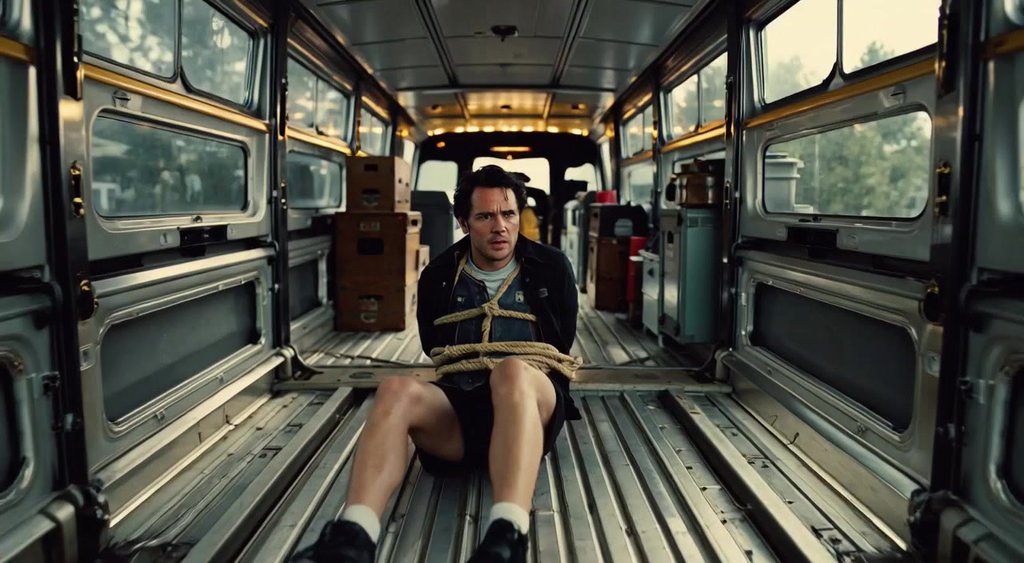} & \framecellw{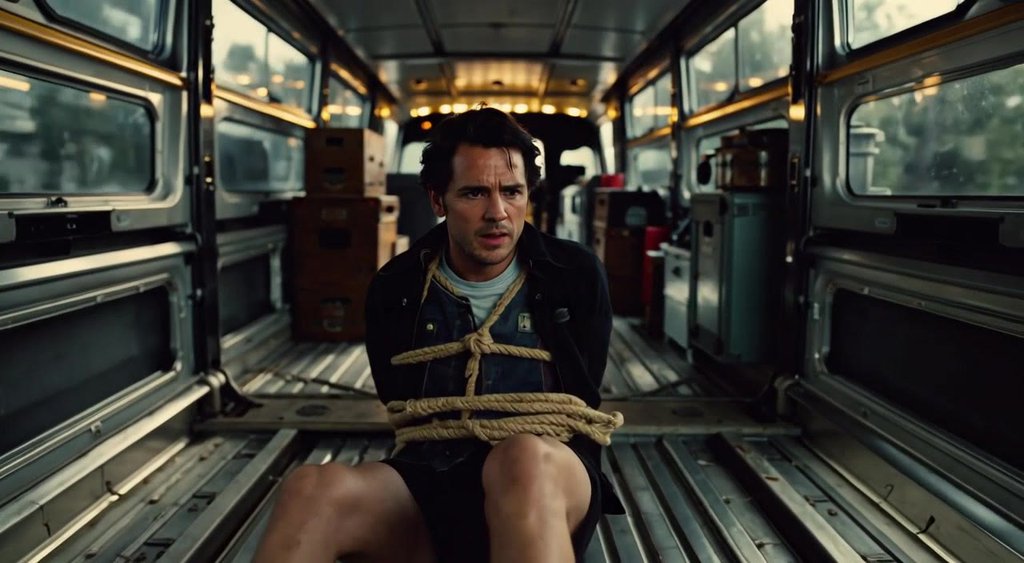} & \framecellw{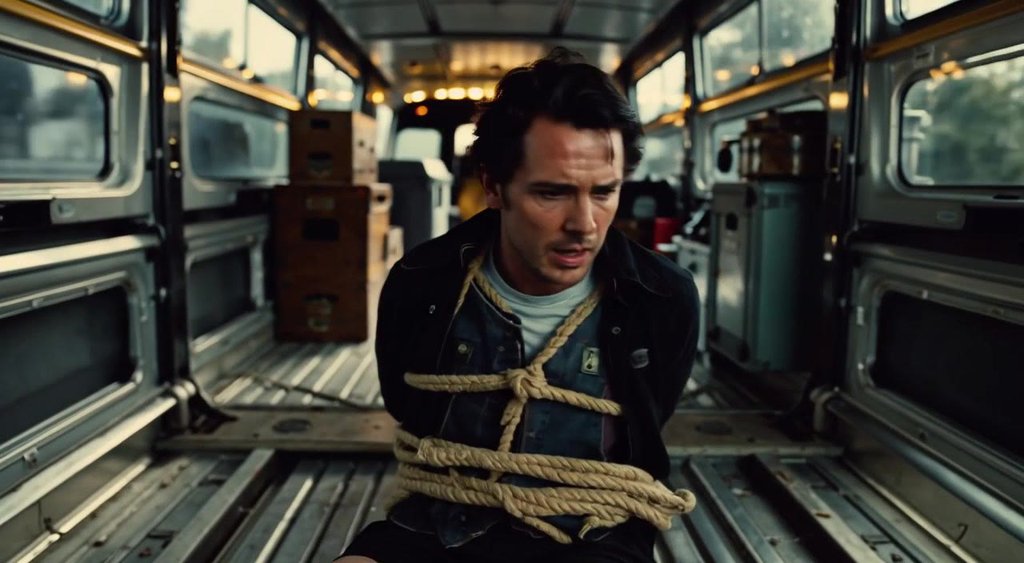} & \framecellw{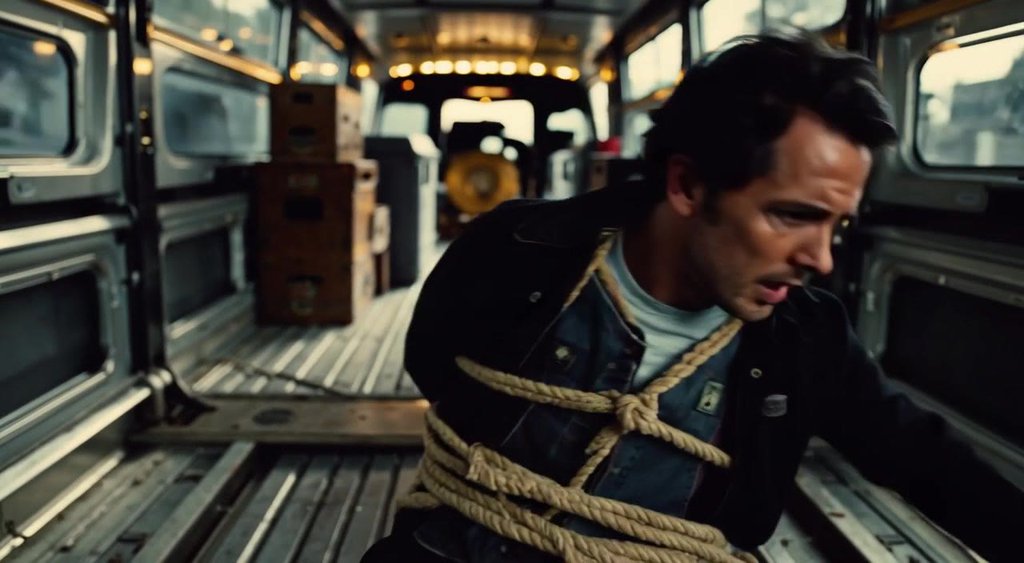} & \framecellw{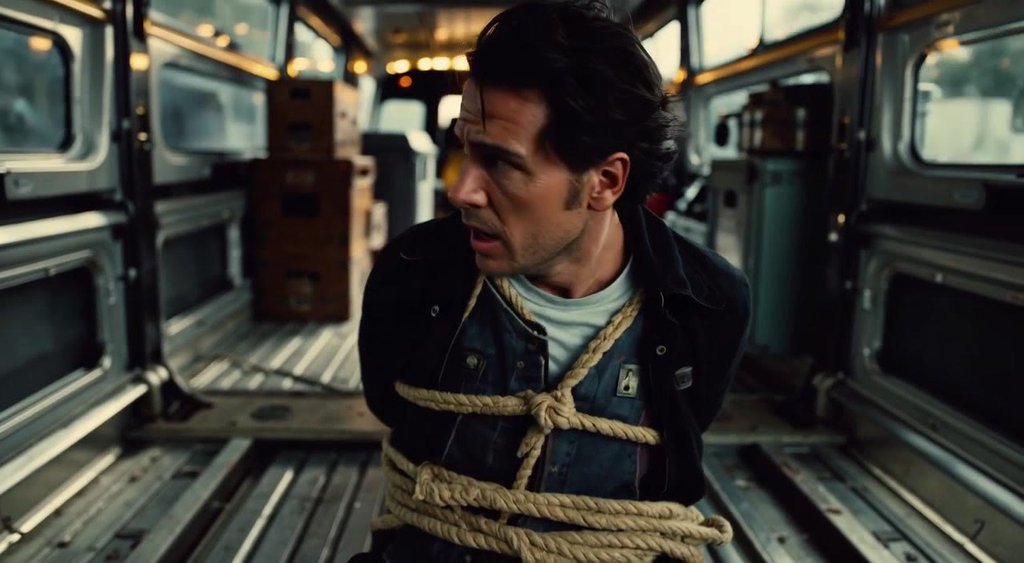} \\ \addlinespace[2pt]
                & \multicolumn{5}{@{} p{16.00cm} @{}}{\tiny\itshape ``Close-up shot inside a large semi-truck van, focusing on a man who is tied up, while the low, rhythmic rumble of a heavy diesel engine idling vibrates through the metal floor. The interior of the van is spacious, with various truck parts visible in the background, such as crates and storage containers that creak and shift slightly against the van's walls. The lighting is dim, casting shadows that enhance the tense atmosphere, accompanied by the muffled, hollow acoustic space of the cargo hold. The man, with a worried expression, is tightly bound with ropes, emphasizing his helpless state, his breathing ragged and audible in the quiet space. His hands are tied behind his back, and he is seated on the floor, the coarse hemp of the ropes rasping against his clothes as he struggles. The camera slowly pans around him, capturing his anxious glances and the details of his bindings, while he lets out a sharp, shaky exhale, then whispers, "Is anyone there?" into the cavernous, echoing silence.''} \\
  \end{tabular}%
  }
  \vspace{-10pt}
\end{figure*}

\begin{figure*}[htbp]
  \centering
  \resizebox{\textwidth}{!}{%
  \begin{tabular}{@{} r @{\hspace{6pt}} c@{\hspace{2pt}}c@{\hspace{2pt}}c@{\hspace{2pt}}c@{\hspace{2pt}}c @{}}
                & $t{=}0$s & $t{=}2.5$s & $t{=}5$s & $t{=}7.5$s & $t{=}10$s \\ \addlinespace[2pt]
    Teacher
      & \framecellw{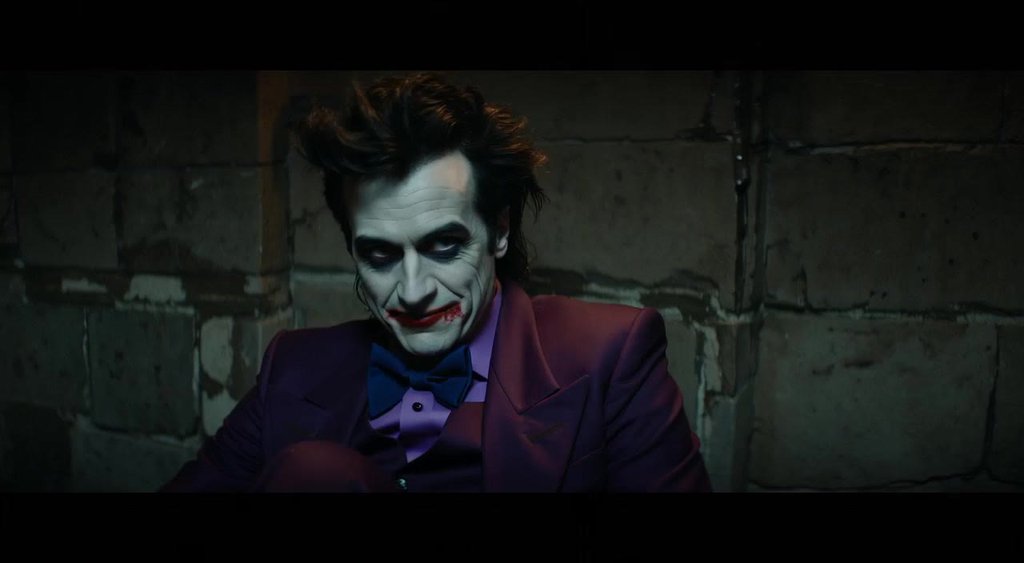} & \framecellw{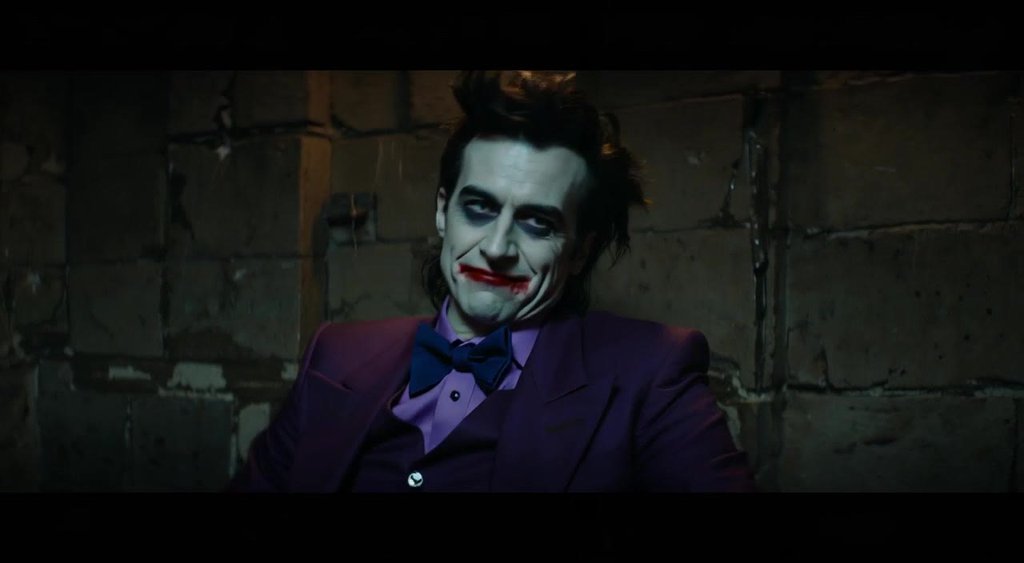} & \framecellw{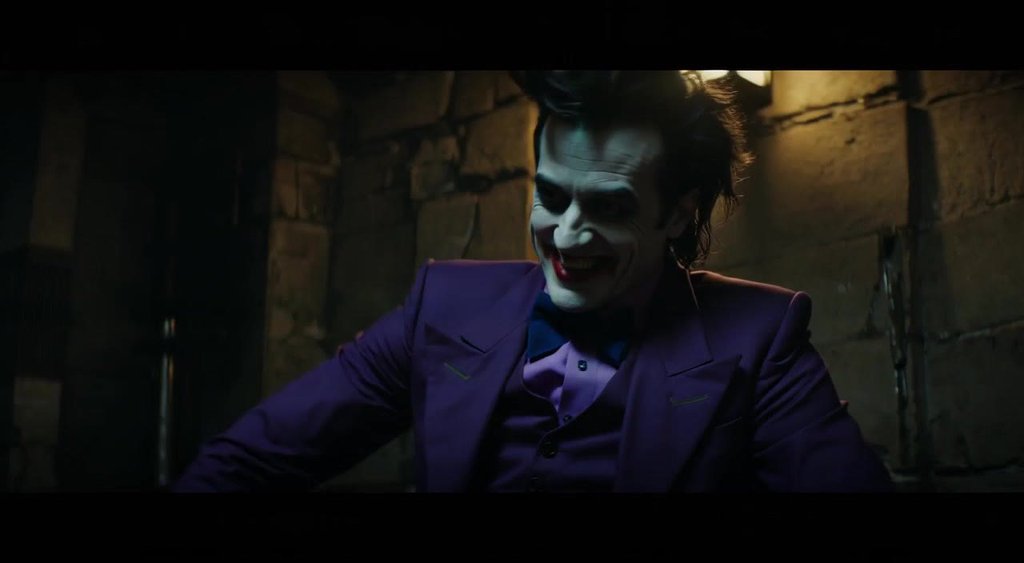} & \framecellw{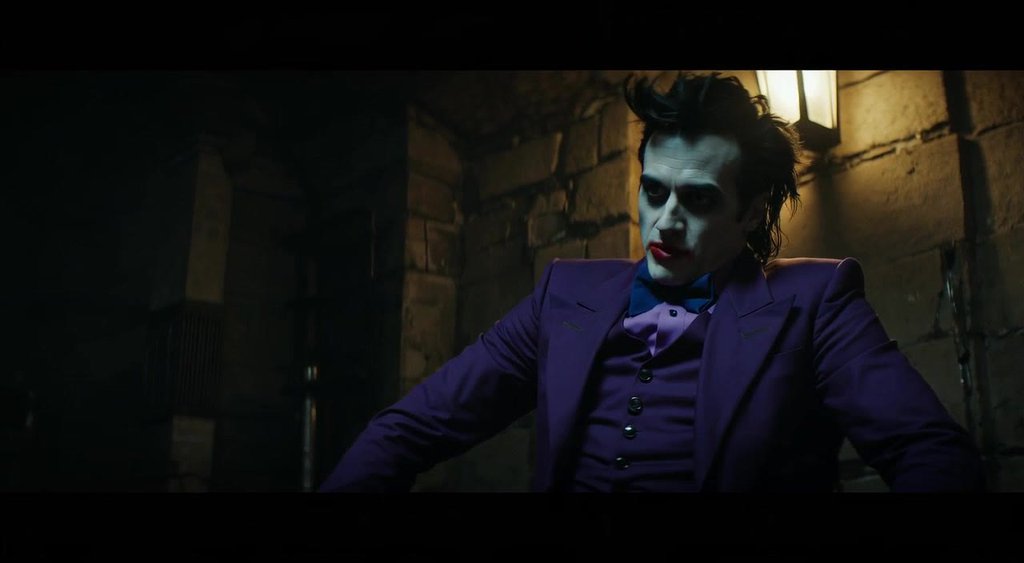} & \framecellw{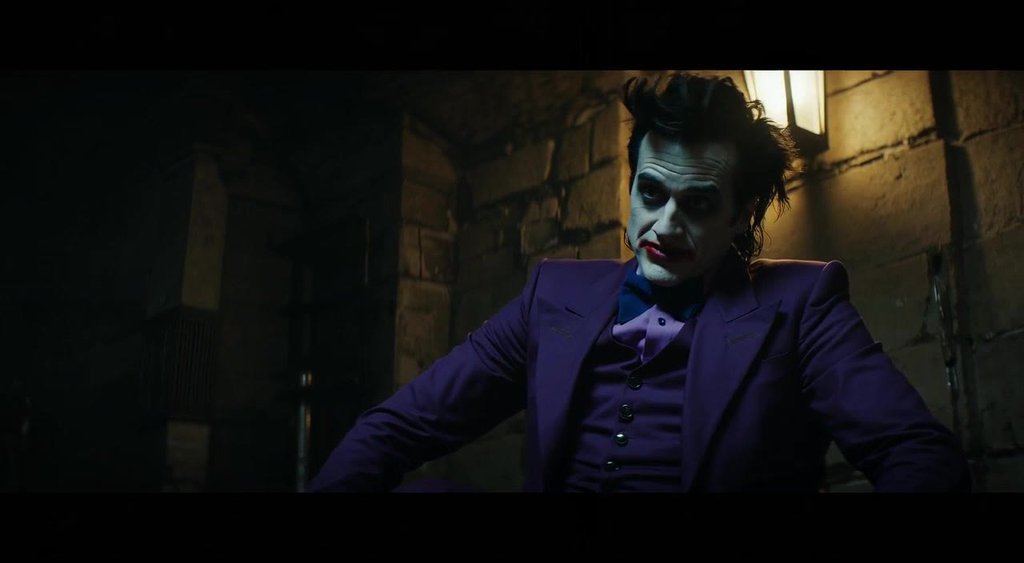} \\ \addlinespace[2pt]
    \textbf{PDD}
      & \framecellw{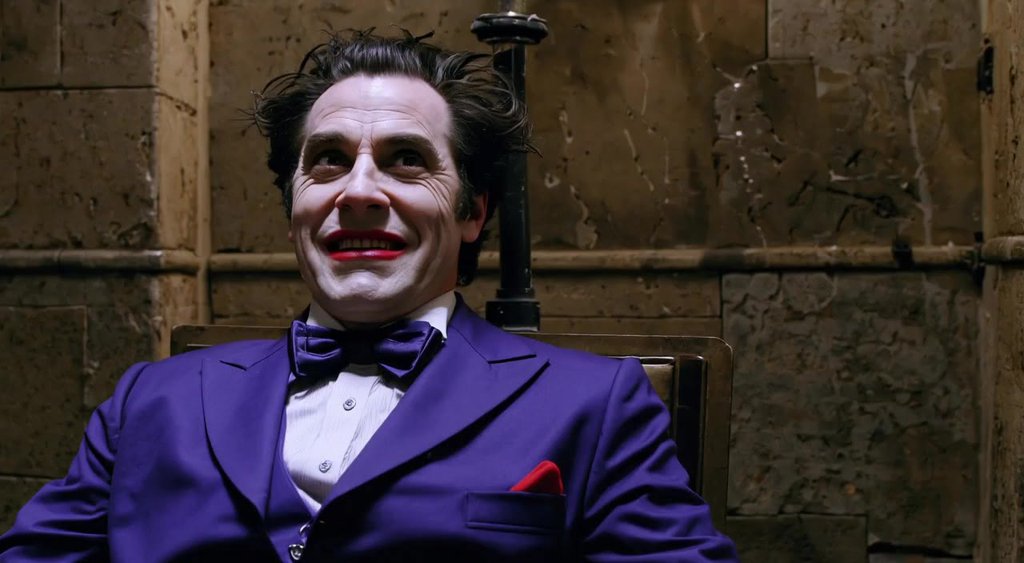} & \framecellw{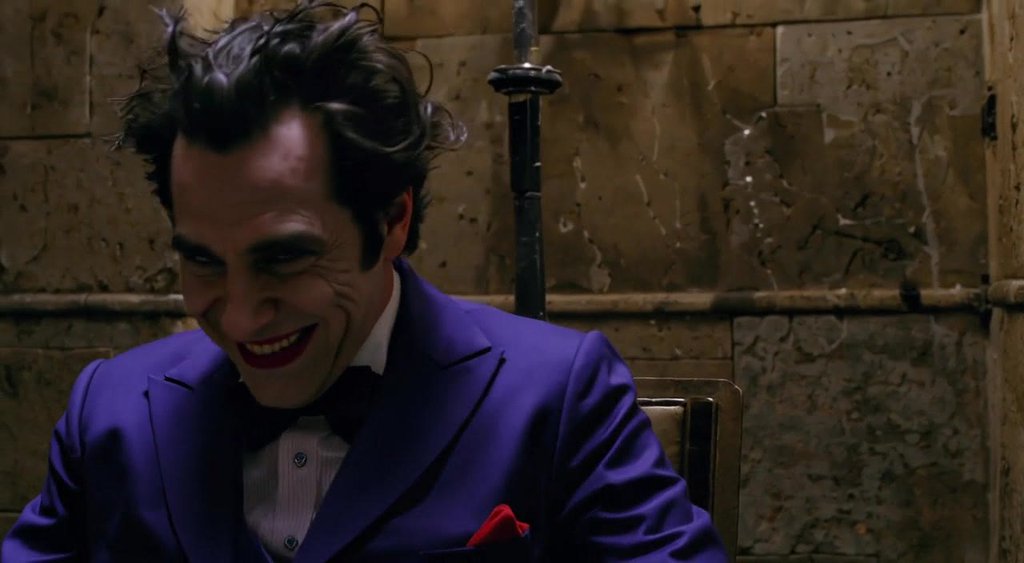} & \framecellw{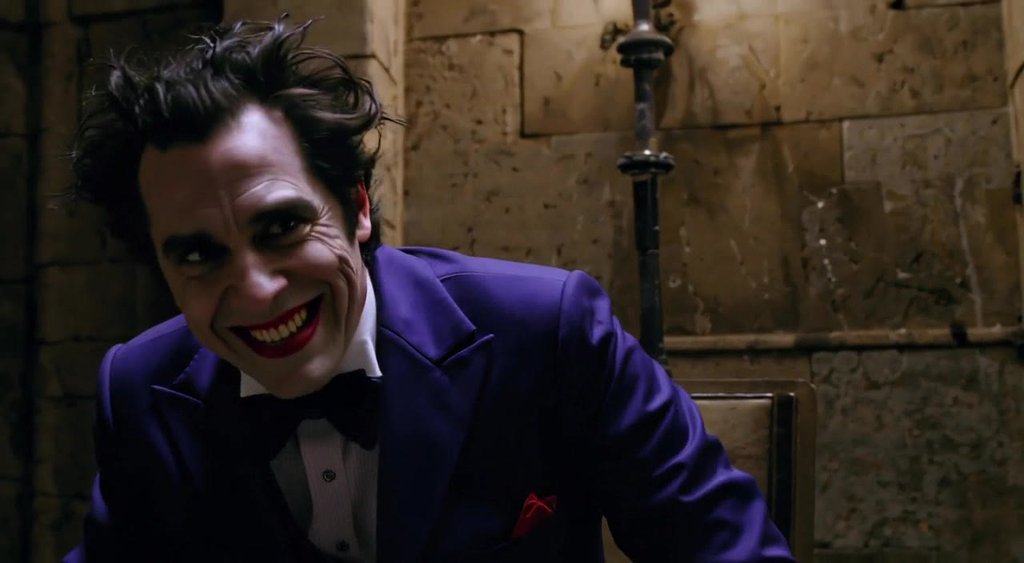} & \framecellw{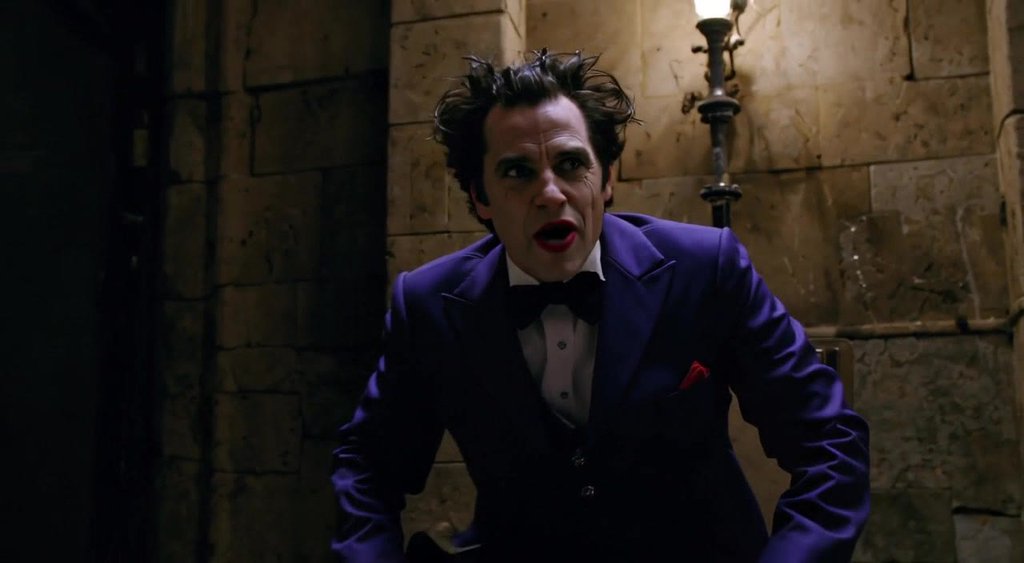} & \framecellw{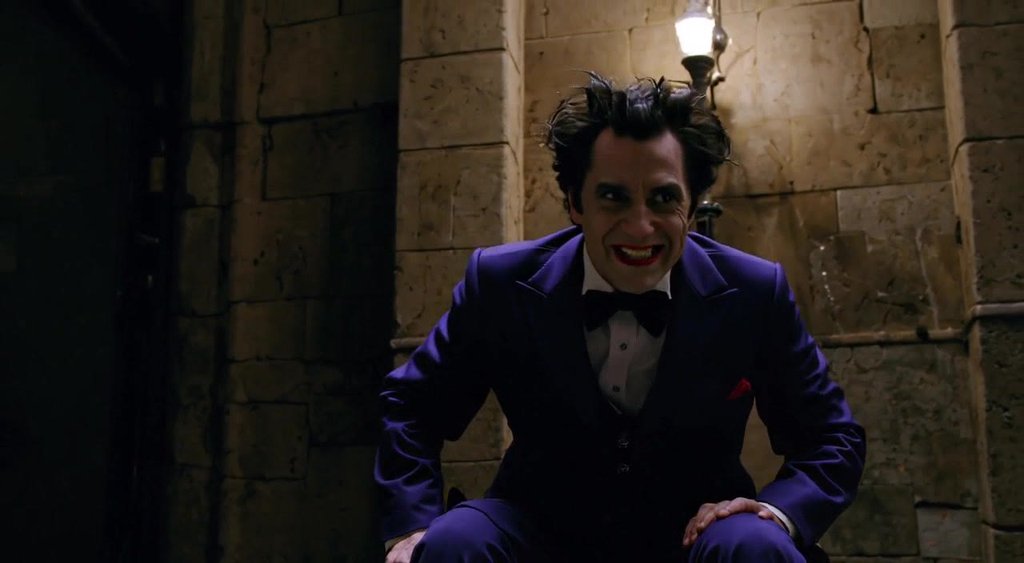} \\ \addlinespace[2pt]
    Distilled
      & \framecellw{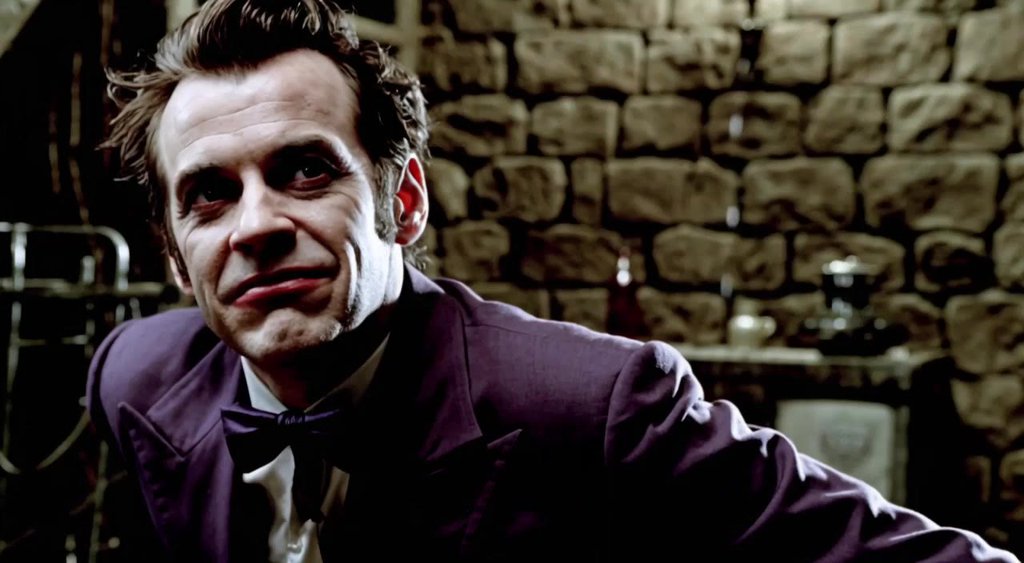} & \framecellw{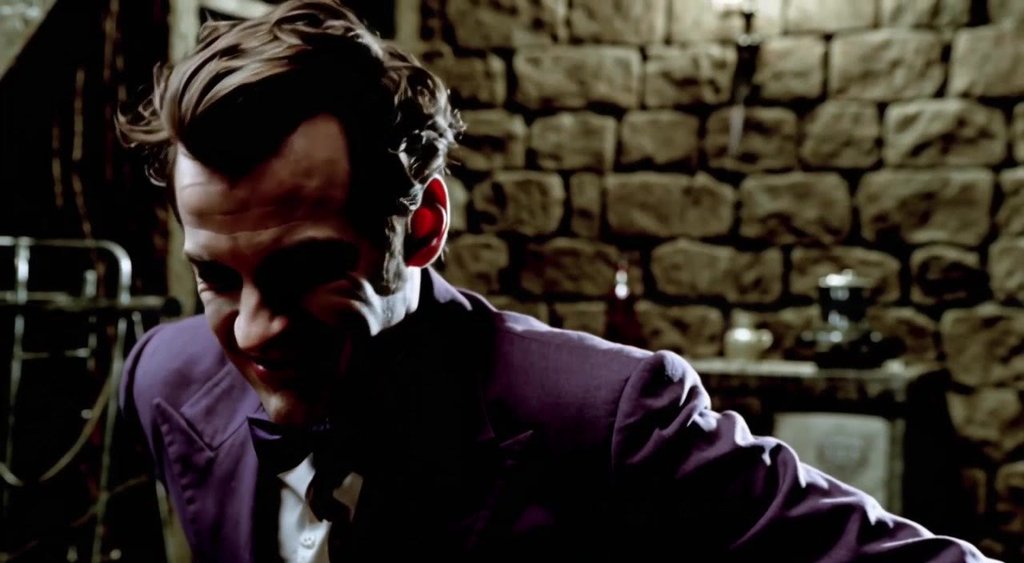} & \framecellw{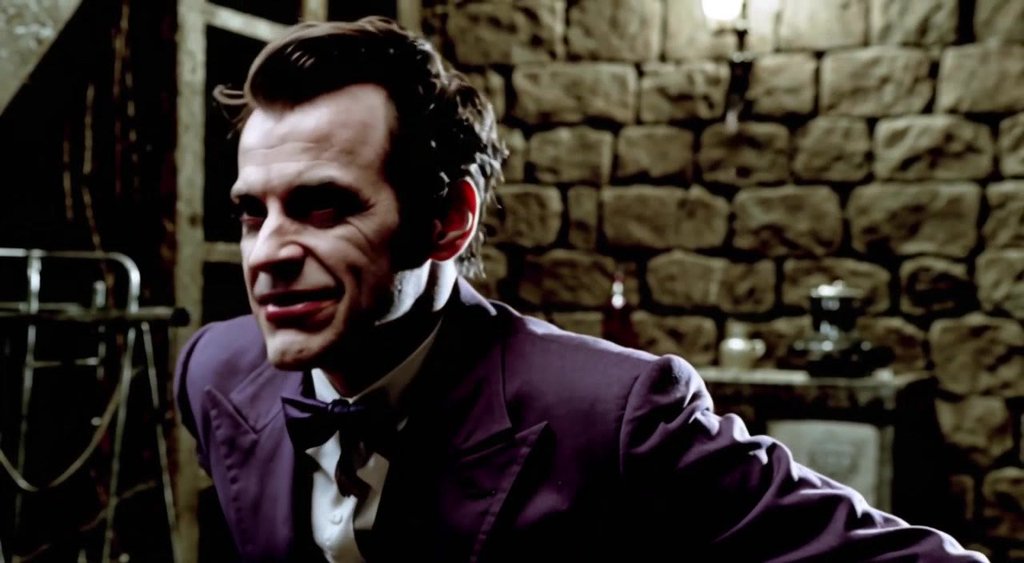} & \framecellw{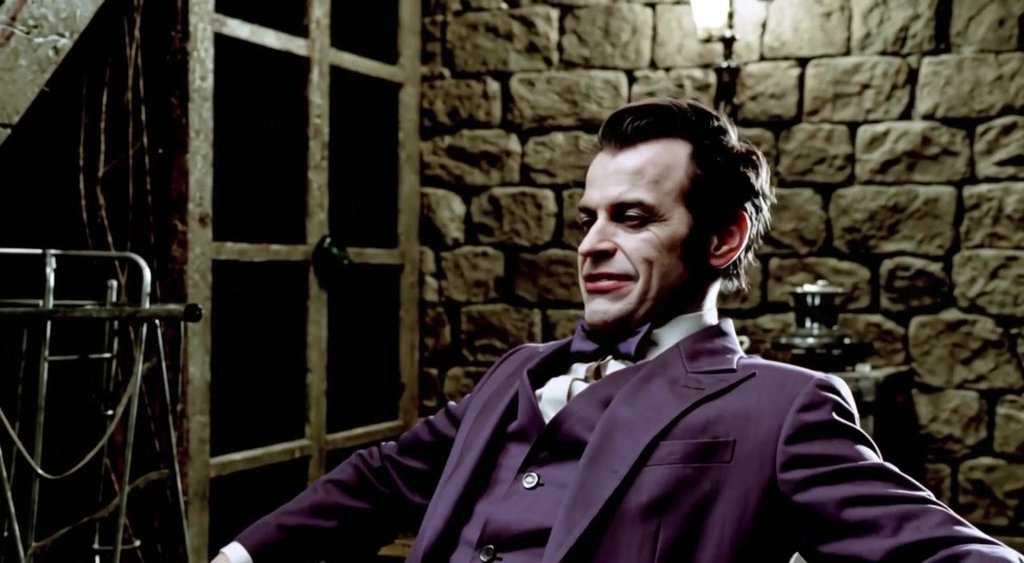} & \framecellw{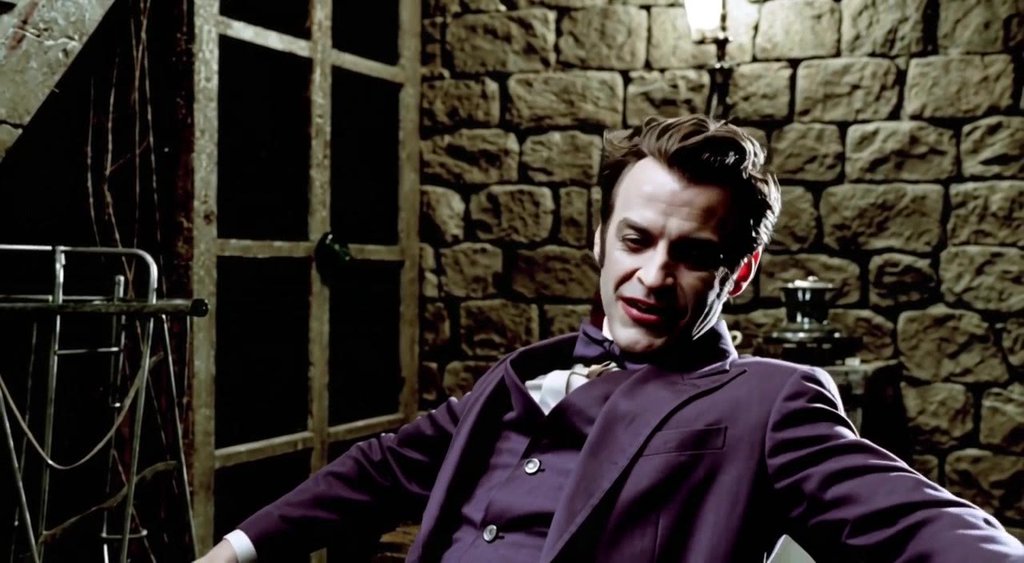} \\ \addlinespace[2pt]
                & \multicolumn{5}{@{} p{16.00cm} @{}}{\tiny\itshape ``A man resembling Brad Pitt dressed as a Joker, with a pale face, blood-red lips, and wild, disheveled dark hair, sits in a dimly lit room where the only sound is the rhythmic, hollow dripping of a leaky pipe echoing against the cold stone walls. He has dark circles under his eyes and a menacing grin, wearing a purple suit with a high-collared shirt and bow tie, the fabric rustling faintly as he shifts his posture, which is slouched to convey a mix of madness and charisma. As the camera slowly pans from his face to reveal his surroundings, he lets out a low, raspy chuckle that vibrates with a gravelly, unsettling tone, then whispers, "Do you want to know how I got these scars?" his voice thin and sharp in the stagnant air. The scene captures the intensity of his gaze and the unsettling environment in a close-up to medium shot, accompanied by the subtle, low-frequency hum of an old flickering light fixture overhead.''} \\
  \end{tabular}%
  }
  \vspace{-10pt}
\end{figure*}

\begin{figure*}[htbp]
  \centering
  \resizebox{\textwidth}{!}{%
  \begin{tabular}{@{} r @{\hspace{6pt}} c@{\hspace{2pt}}c@{\hspace{2pt}}c@{\hspace{2pt}}c@{\hspace{2pt}}c @{}}
                & $t{=}0$s & $t{=}2.5$s & $t{=}5$s & $t{=}7.5$s & $t{=}10$s \\ \addlinespace[2pt]
    Teacher
      & \framecellw{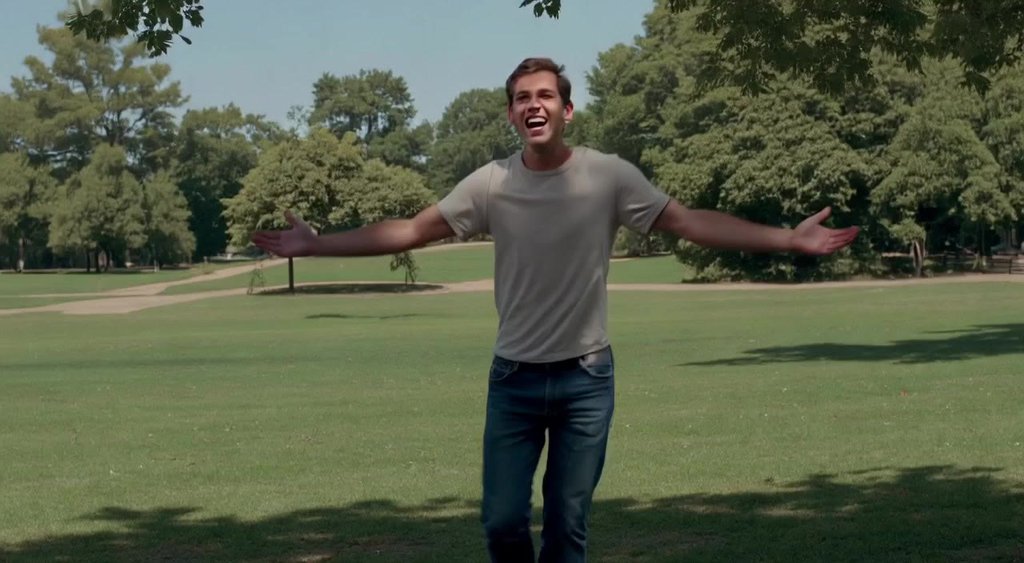} & \framecellw{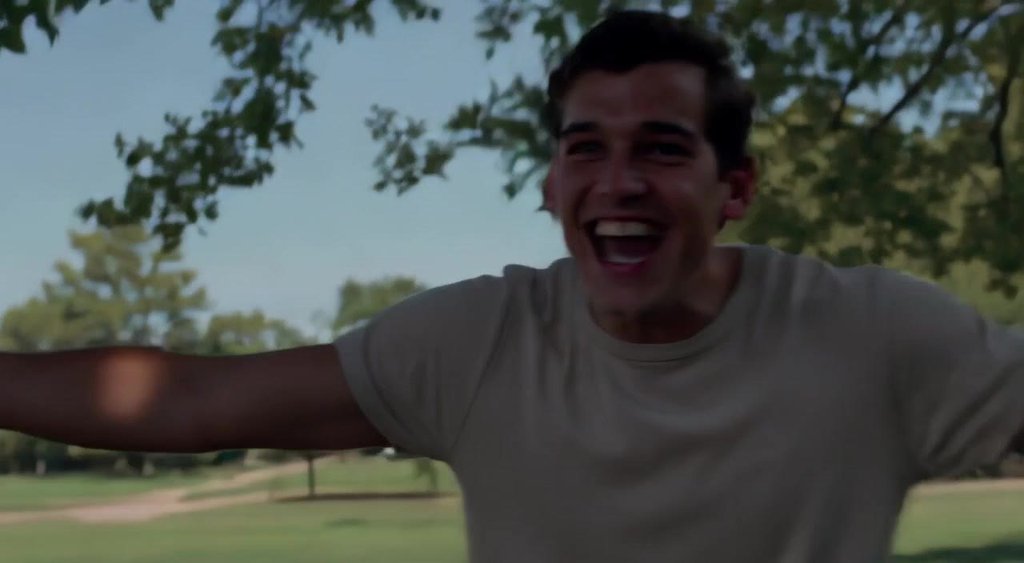} & \framecellw{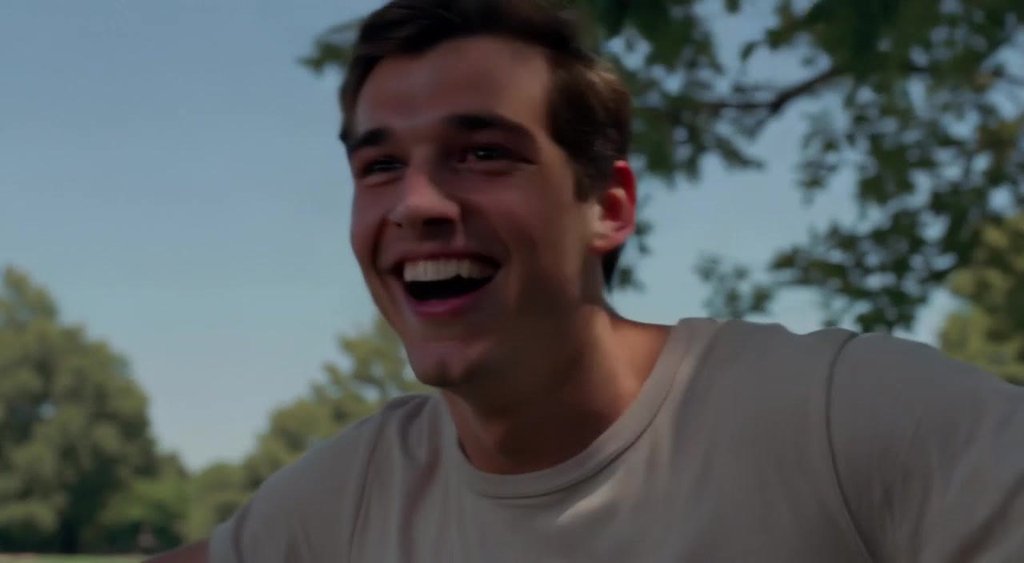} & \framecellw{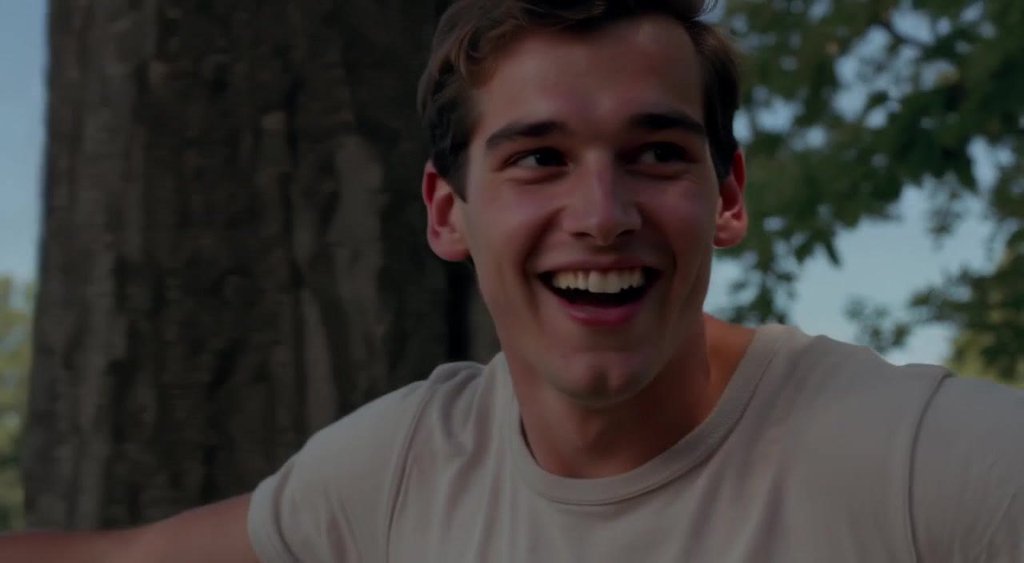} & \framecellw{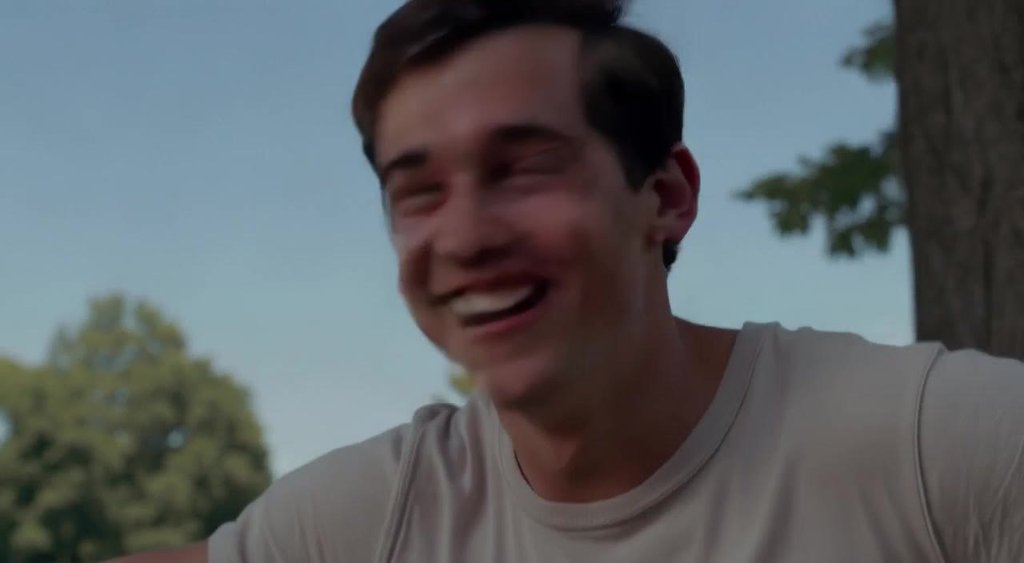} \\ \addlinespace[2pt]
    \textbf{PDD}
      & \framecellw{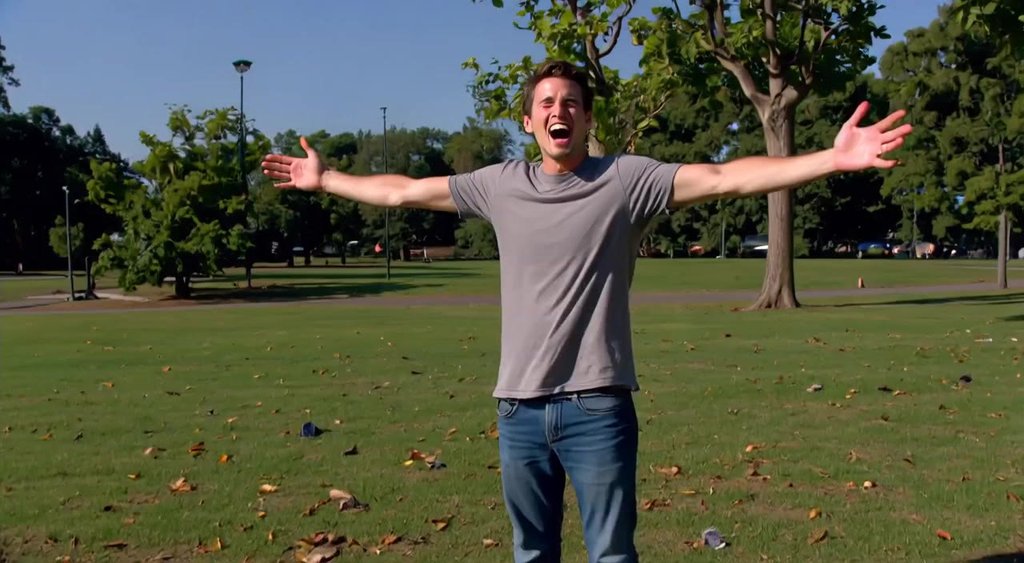} & \framecellw{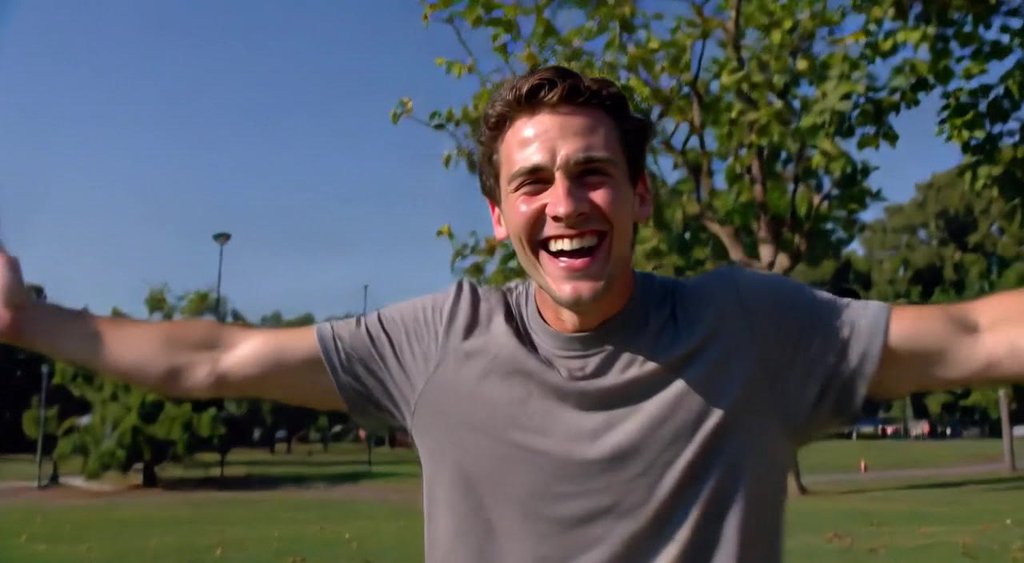} & \framecellw{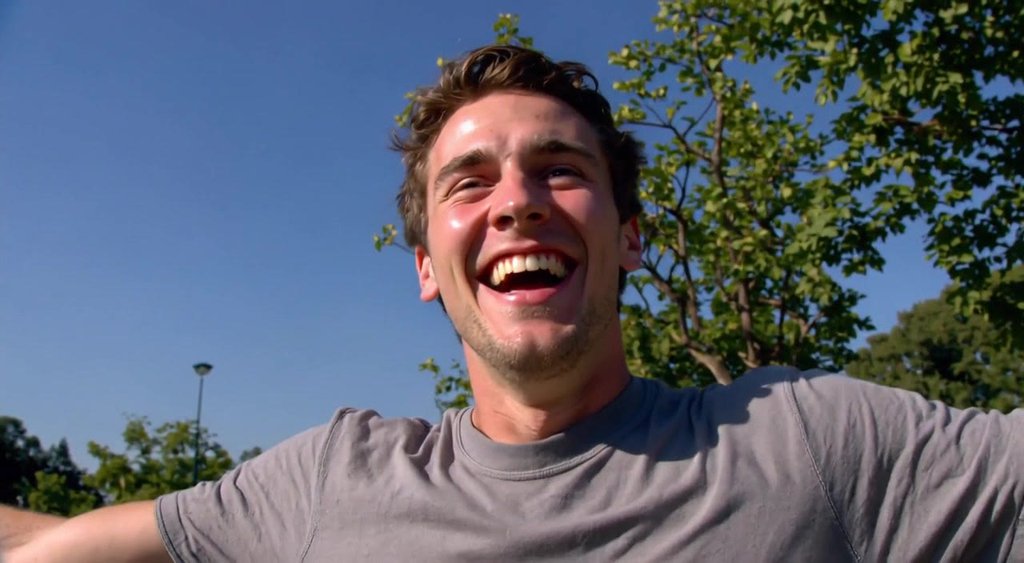} & \framecellw{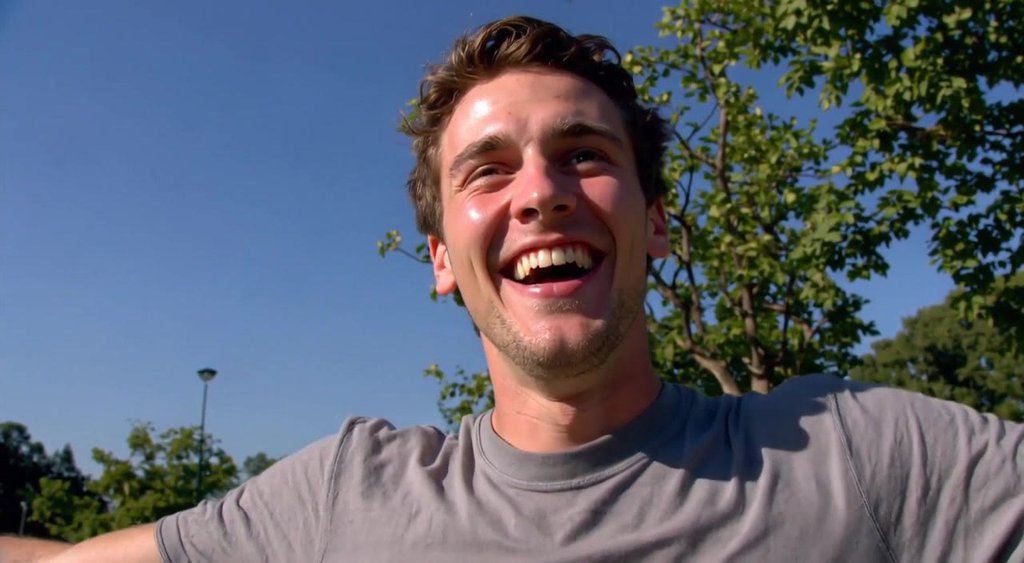} & \framecellw{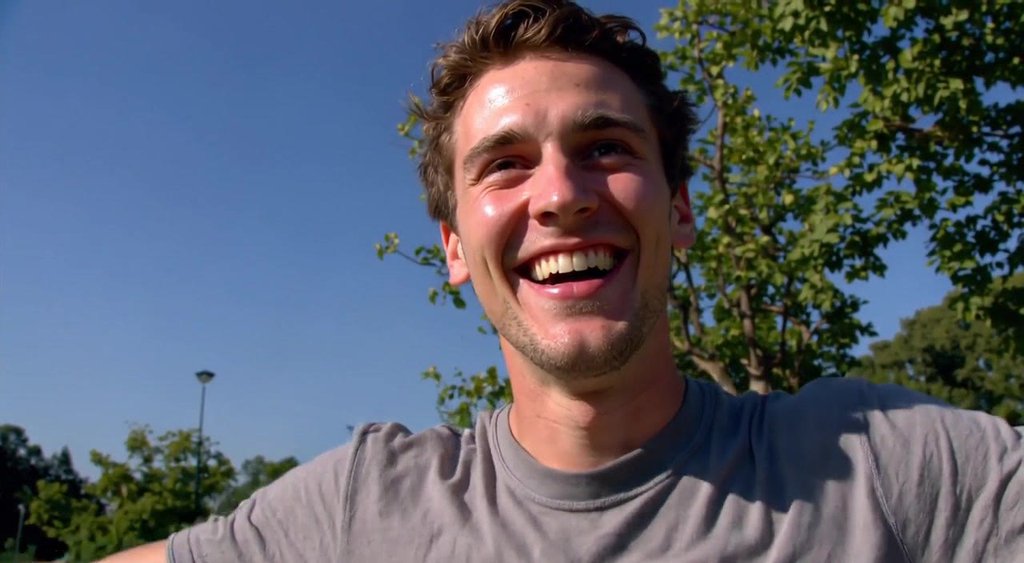} \\ \addlinespace[2pt]
    Distilled
      & \framecellw{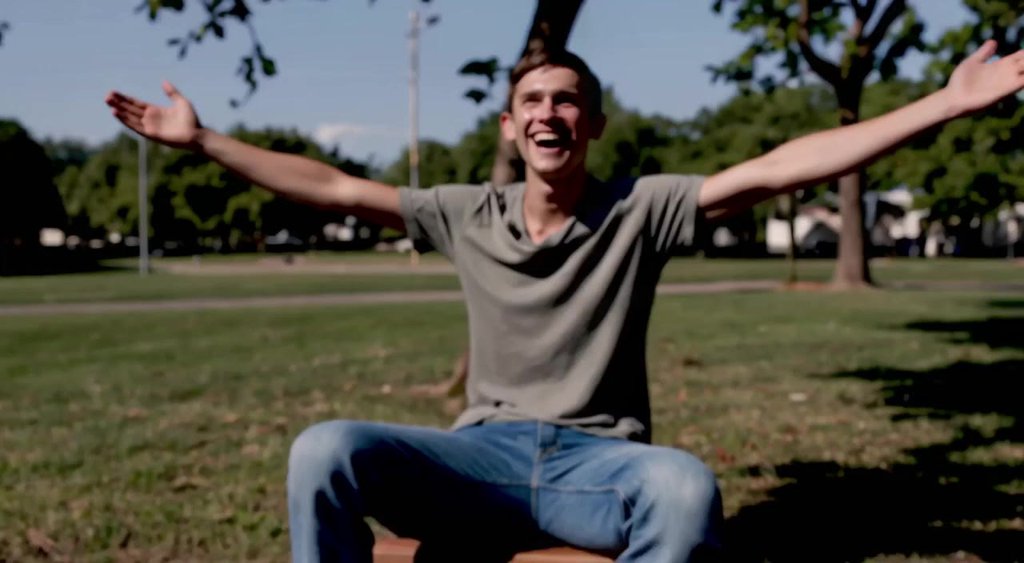} & \framecellw{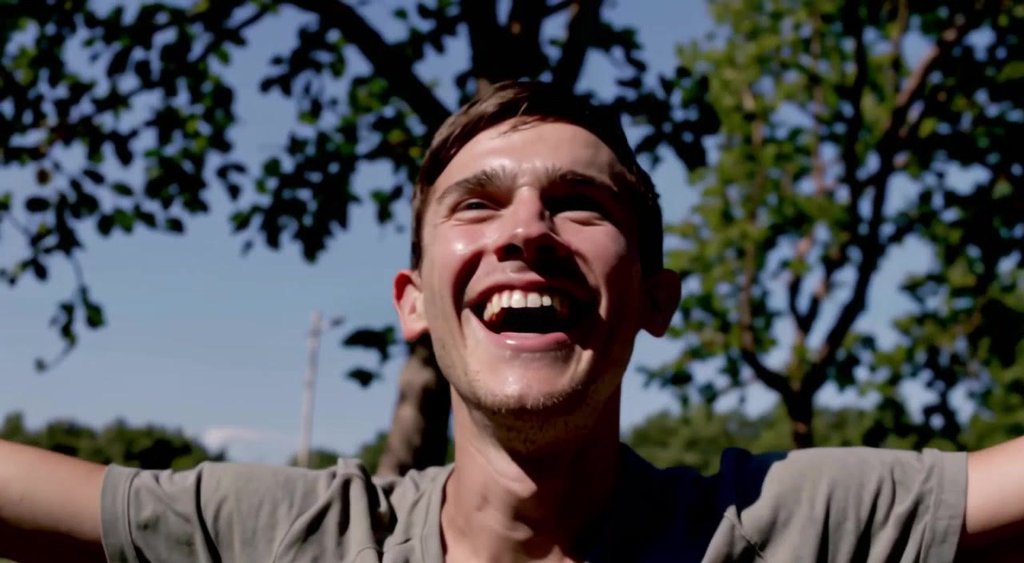} & \framecellw{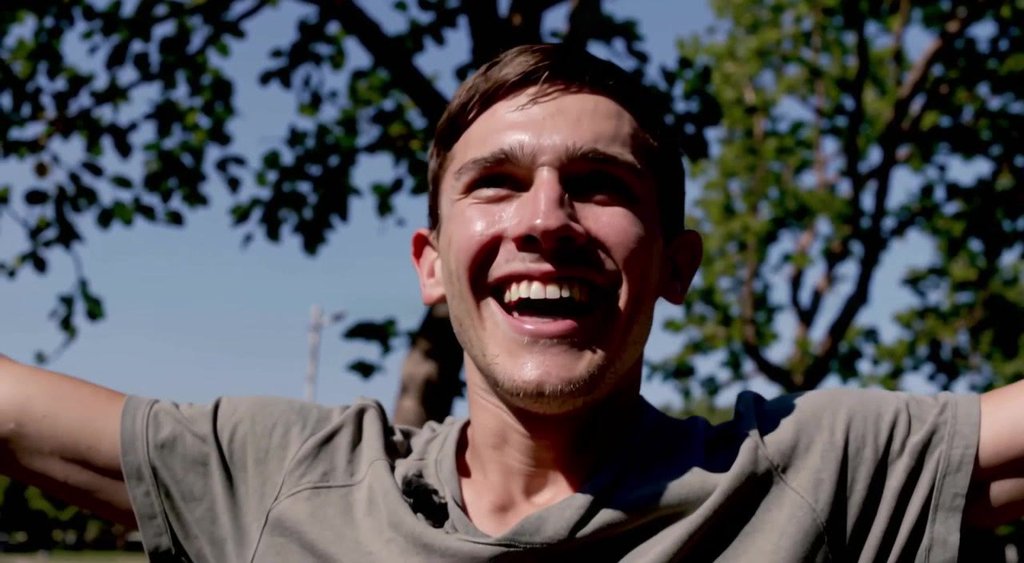} & \framecellw{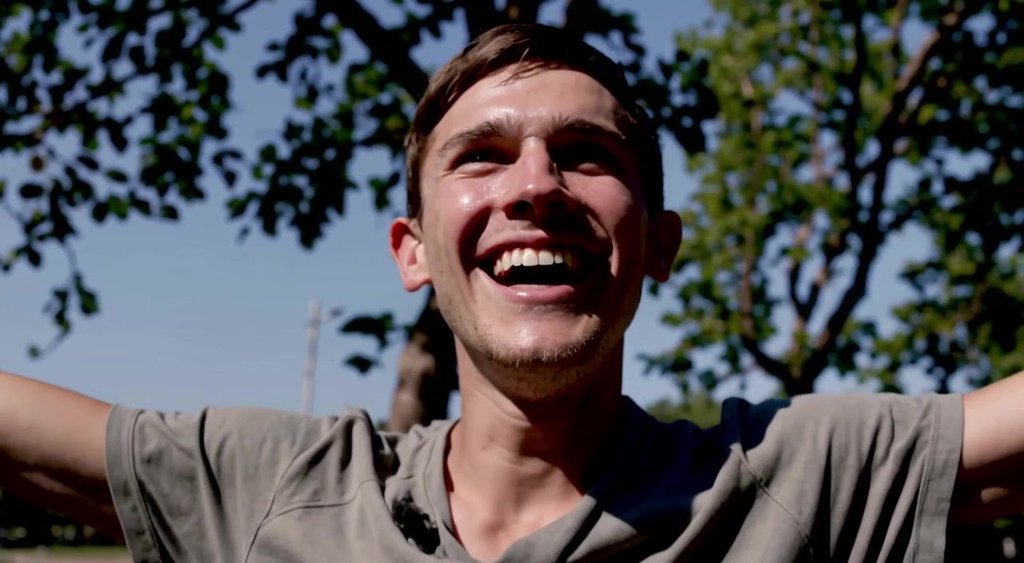} & \framecellw{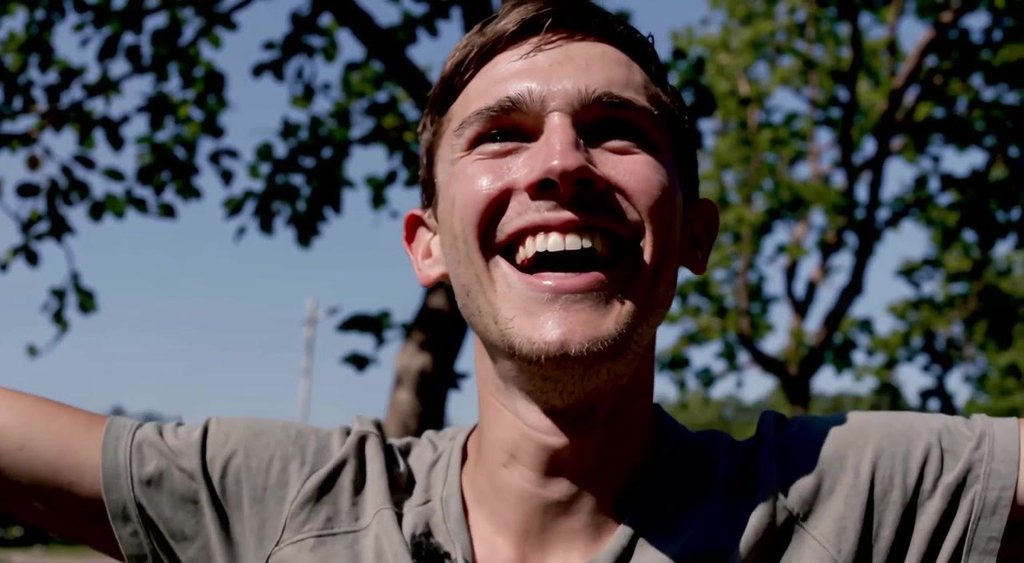} \\ \addlinespace[2pt]
                & \multicolumn{5}{@{} p{16.00cm} @{}}{\tiny\itshape ``A young Caucasian man in his early twenties, wearing a casual t-shirt and jeans, is laughing heartily outdoors in a park, his loud, genuine laughter ringing out into the open air. His arms are spread wide as he leans back, enjoying the sunlight, the soft rustle of leaves in a gentle breeze accompanying the sound of his breathy, joyful chuckles. His face is animated, with crinkled eyes and a big smile, conveying genuine joy as he lets out a bright, spontaneous "Ha! That's amazing!" The background includes green grass, trees, and a clear blue sky, with the distant, faint chirping of birds blending into the ambient park atmosphere. The camera starts from a wide shot showing the park setting, then zooms in to focus on the man's expressive face, capturing the rhythmic cadence of his laughter, mid-shot, dynamic framing catching the full range of his infectious mirth.''} \\
  \end{tabular}%
  }
  \caption{LTX-2.3 teacher ($4{\times}30$ NFE) vs.\ PDD (8 NFE) vs.\ official distilled model (8 NFE).}
  \label{fig:ltx_comparison_first}
\end{figure*}

\begin{figure*}[htbp]
  \centering
  \resizebox{\textwidth}{!}{%
  \begin{tabular}{@{} r @{\hspace{6pt}} c@{\hspace{2pt}}c@{\hspace{2pt}}c@{\hspace{2pt}}c@{\hspace{2pt}}c @{}}
                & $t{=}0$s & $t{=}2.5$s & $t{=}5$s & $t{=}7.5$s & $t{=}10$s \\ \addlinespace[2pt]
    Teacher
      & \framecellw{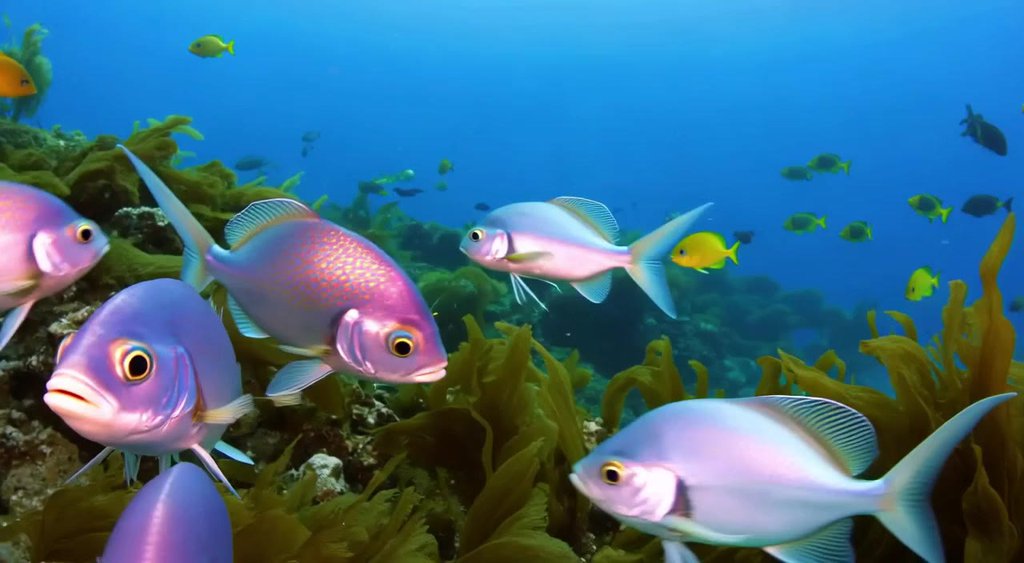} & \framecellw{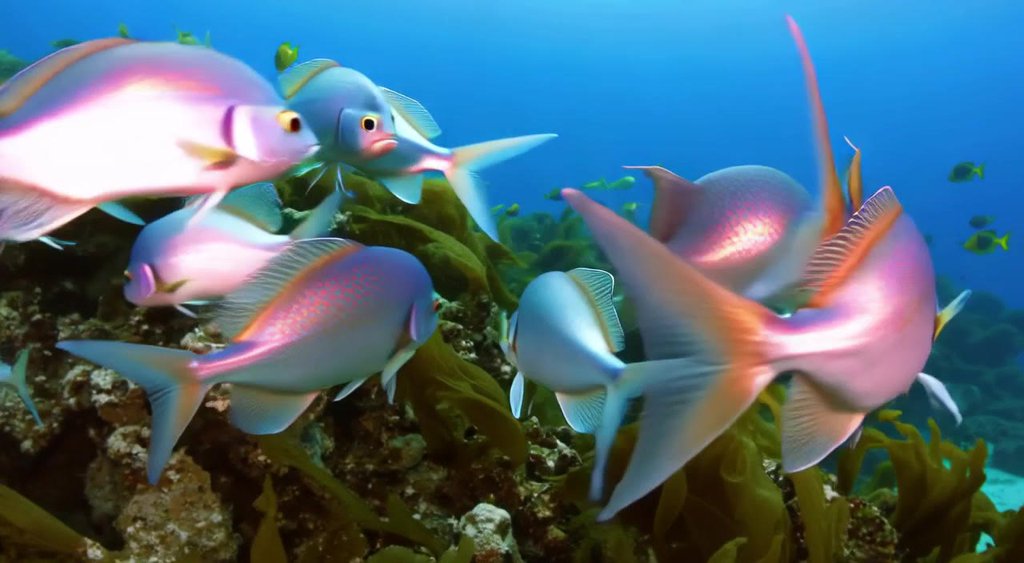} & \framecellw{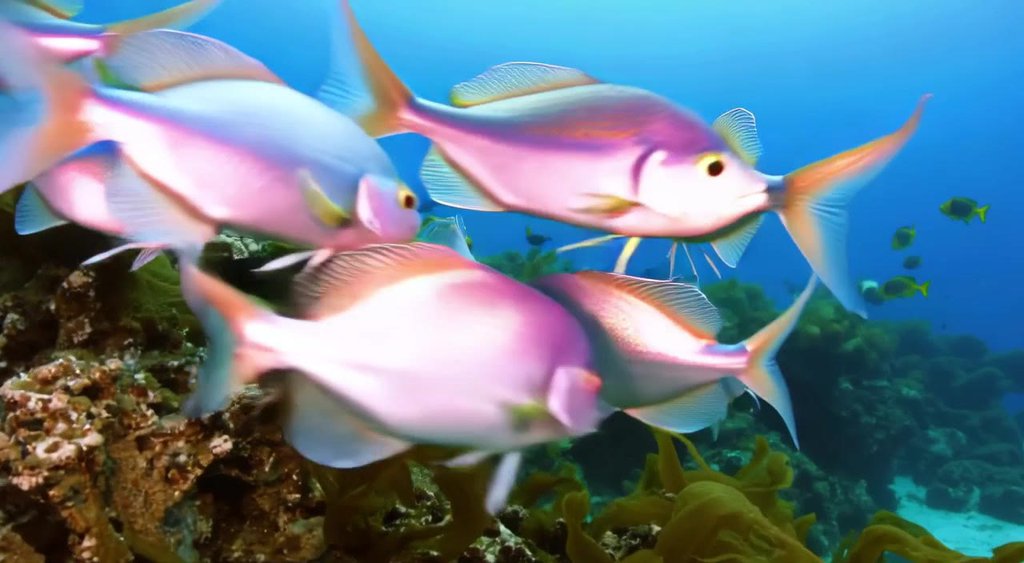} & \framecellw{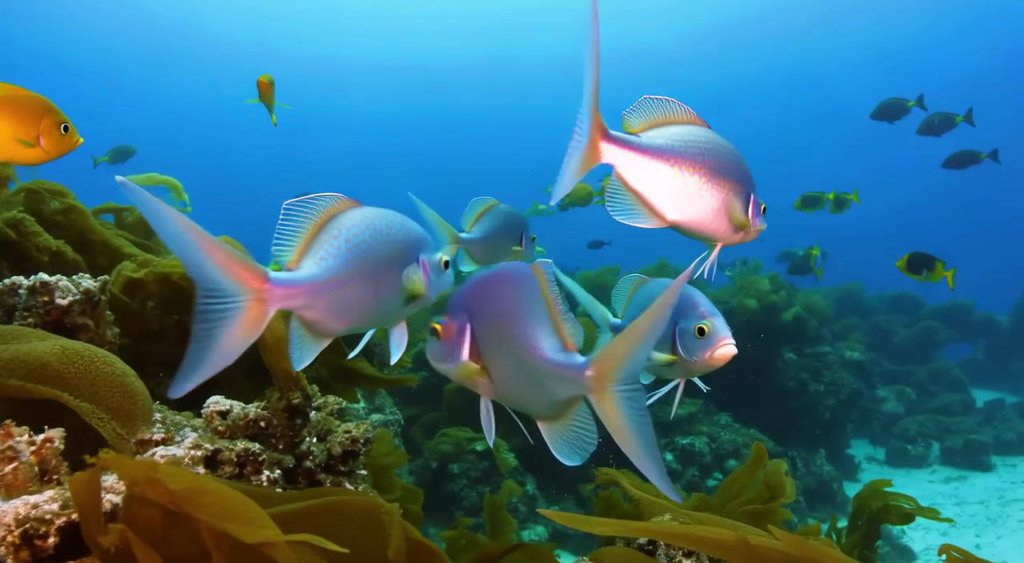} & \framecellw{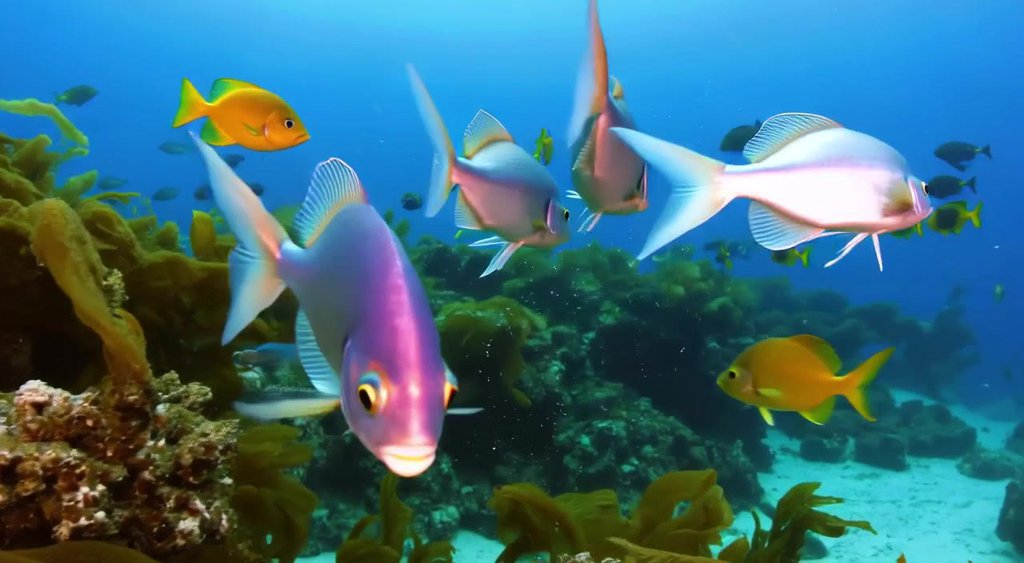} \\ \addlinespace[2pt]
    \textbf{PDD}
      & \framecellw{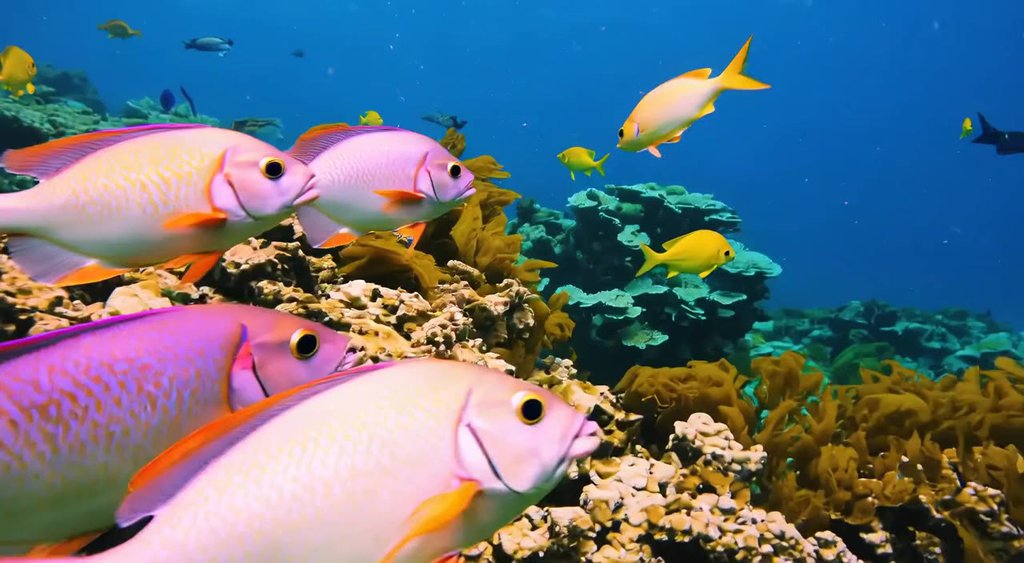} & \framecellw{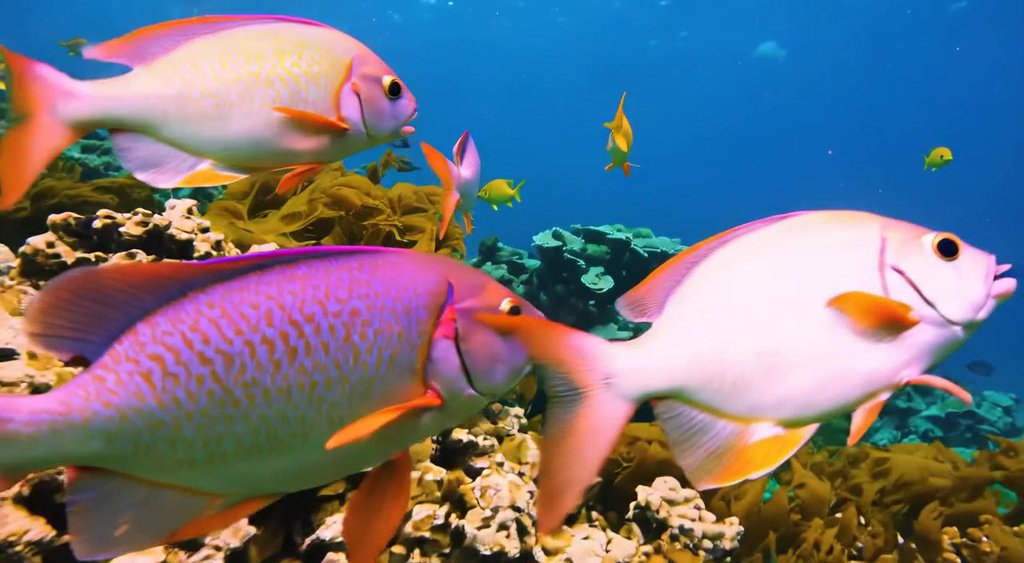} & \framecellw{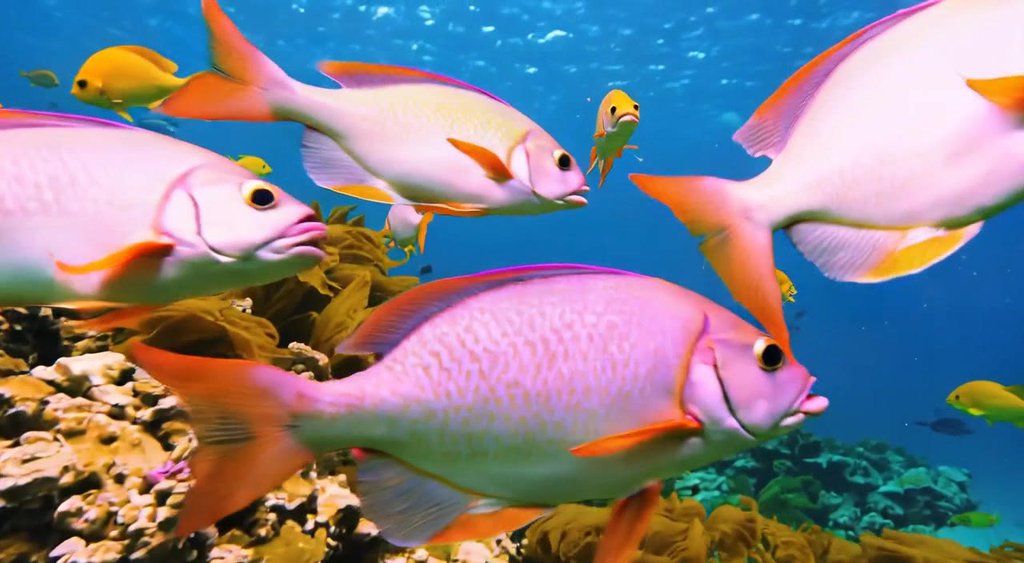} & \framecellw{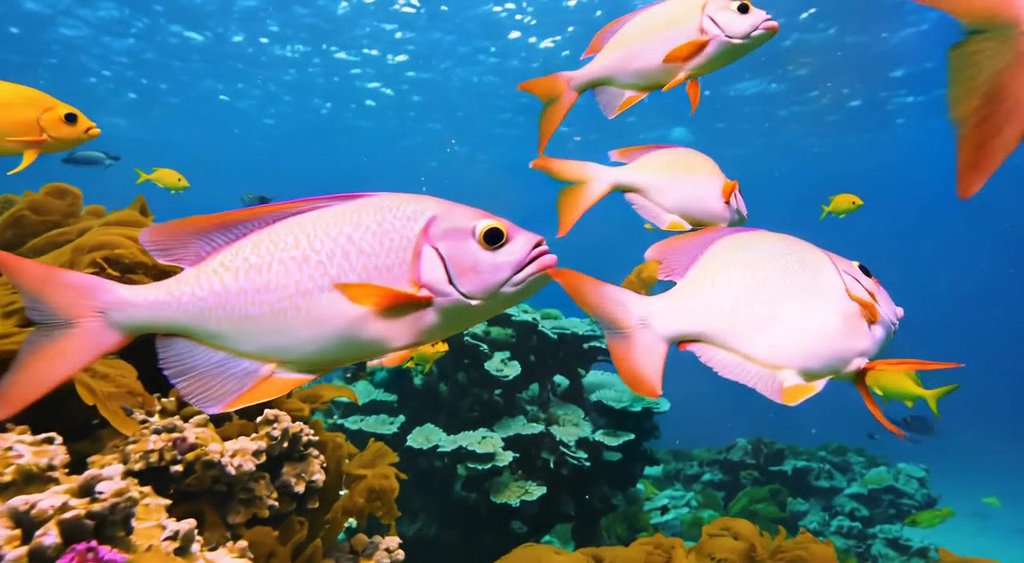} & \framecellw{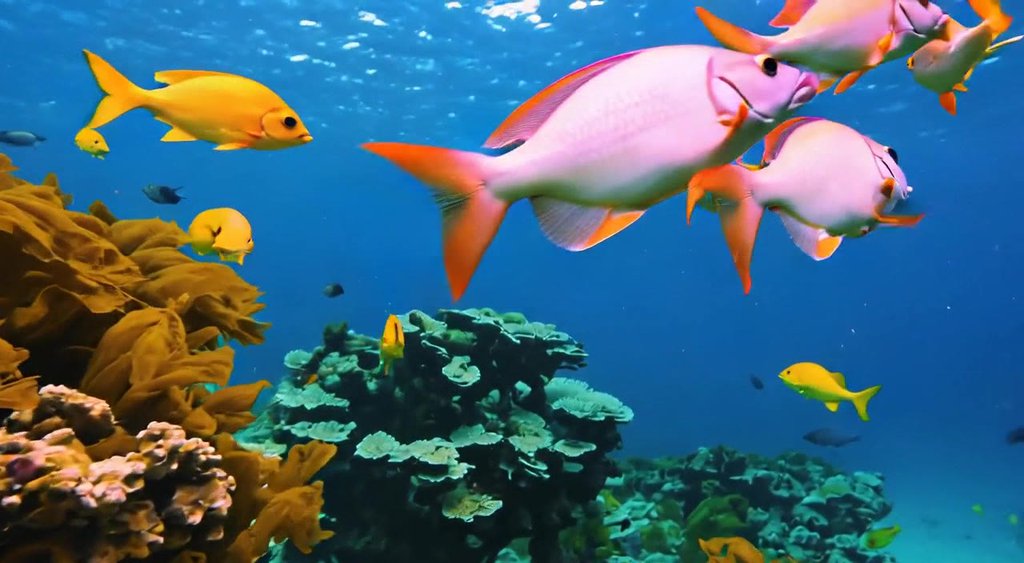} \\ \addlinespace[2pt]
    Distilled
      & \framecellw{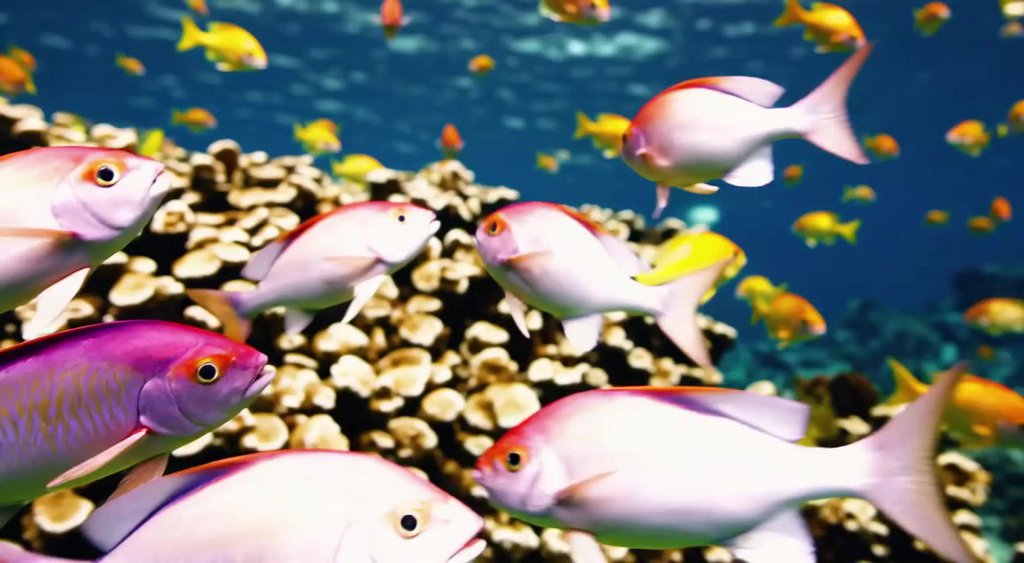} & \framecellw{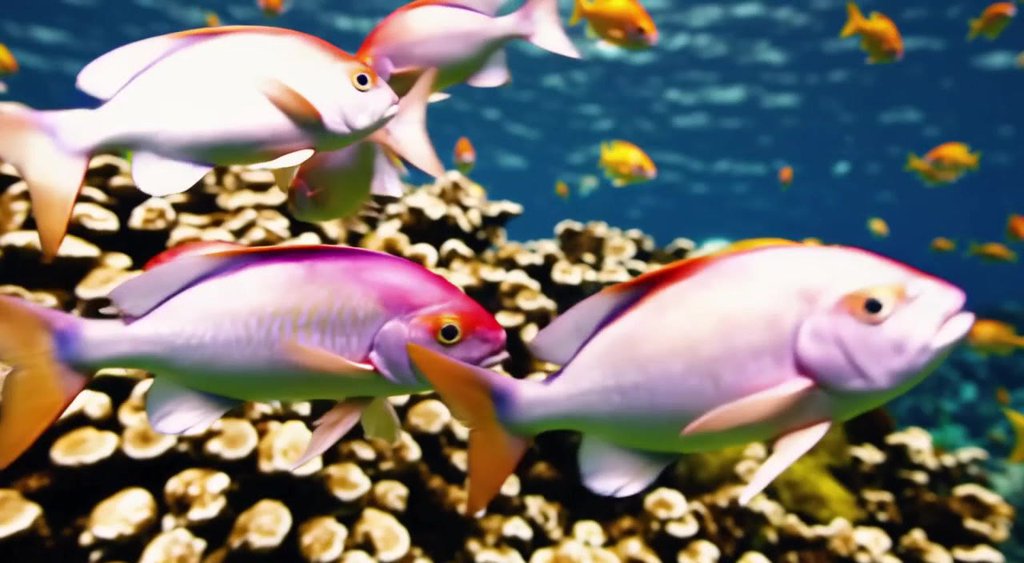} & \framecellw{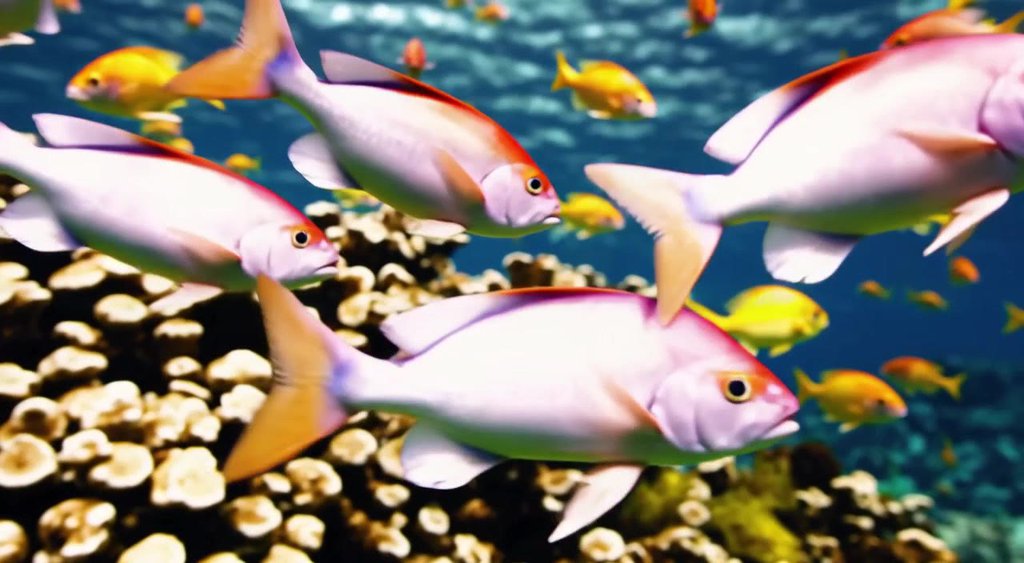} & \framecellw{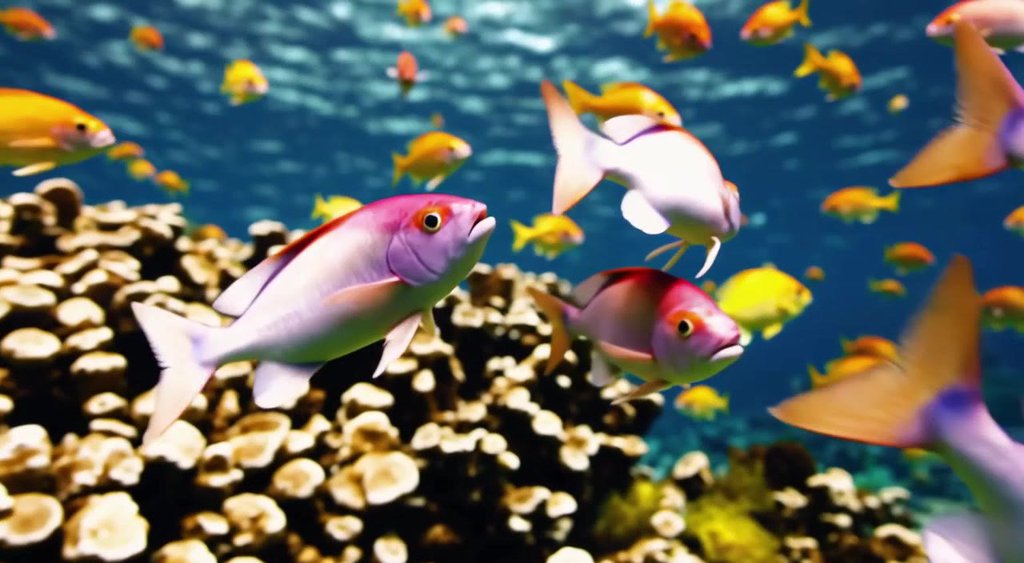} & \framecellw{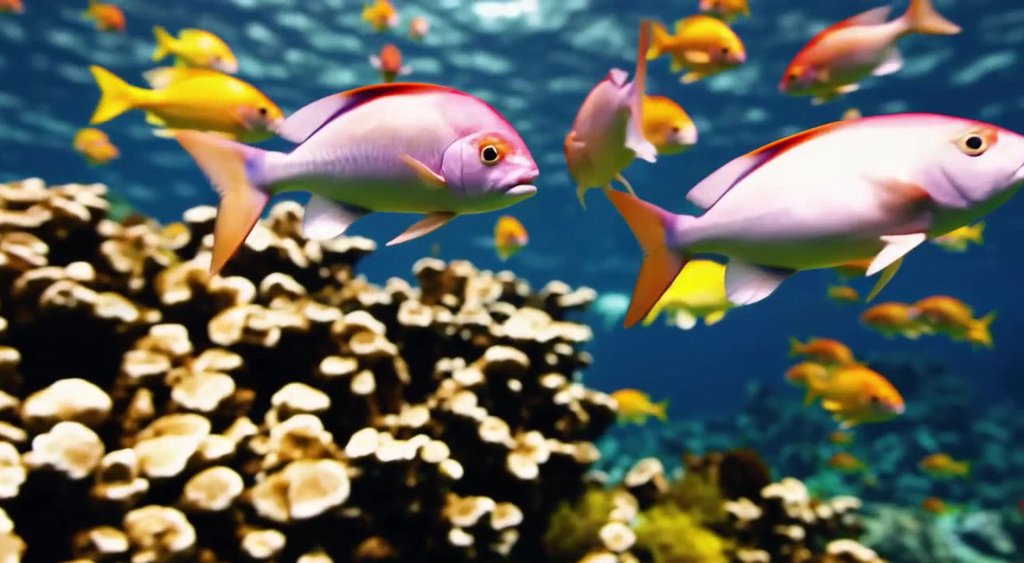} \\ \addlinespace[2pt]
                & \multicolumn{5}{@{} p{16.00cm} @{}}{\tiny\itshape ``A vibrant underwater scene features colorful fish dancing gracefully through the water, accompanied by the muffled, rhythmic whoosh of water displacement and the faint, bubbly glug of their movement. The fish have shimmering scales and expressive fins, performing synchronized dance moves that produce soft, watery swishing sounds as they swim. Their movements are fluid and dynamic, with occasional flips and spins adding to the lively atmosphere, punctuated by the gentle, tinkling chime of light refracting through the waves. The background showcases a coral reef with soft, flowing seaweed and bright tropical fish swimming in the distance, set to a backdrop of a low-frequency, tranquil aquatic hum. The camera follows the fish in a smooth tracking shot, capturing their elegant dance from various angles, including close-ups of their faces showing joy and excitement, while the acoustic space remains vast and immersive, filled with the airy, resonant silence of the deep sea.''} \\
  \end{tabular}%
  }
  \vspace{-10pt}
\end{figure*}

\begin{figure*}[htbp]
  \centering
  \resizebox{\textwidth}{!}{%
  \begin{tabular}{@{} r @{\hspace{6pt}} c@{\hspace{2pt}}c@{\hspace{2pt}}c@{\hspace{2pt}}c@{\hspace{2pt}}c @{}}
                & $t{=}0$s & $t{=}2.5$s & $t{=}5$s & $t{=}7.5$s & $t{=}10$s \\ \addlinespace[2pt]
    Teacher
      & \framecellw{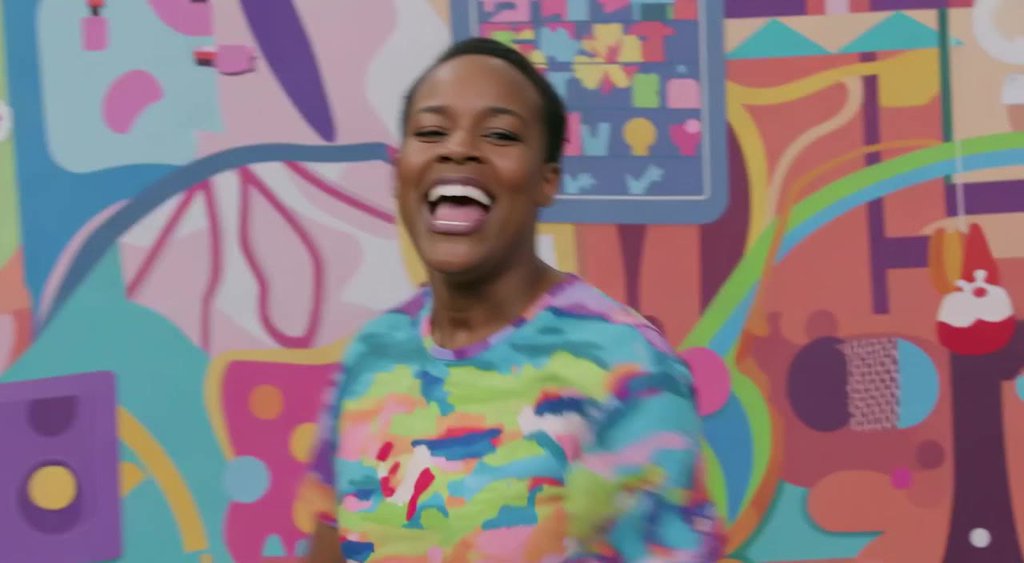} & \framecellw{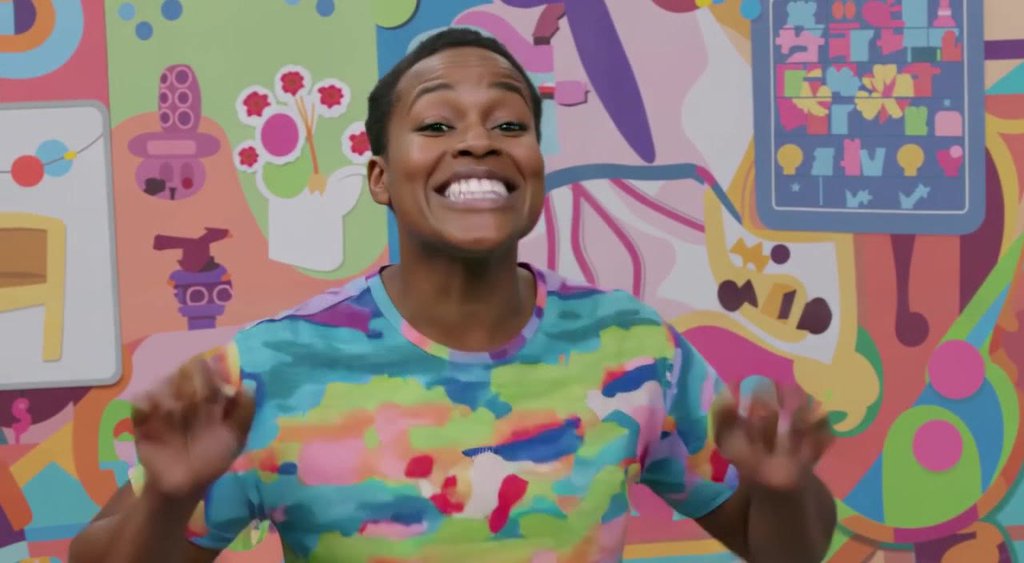} & \framecellw{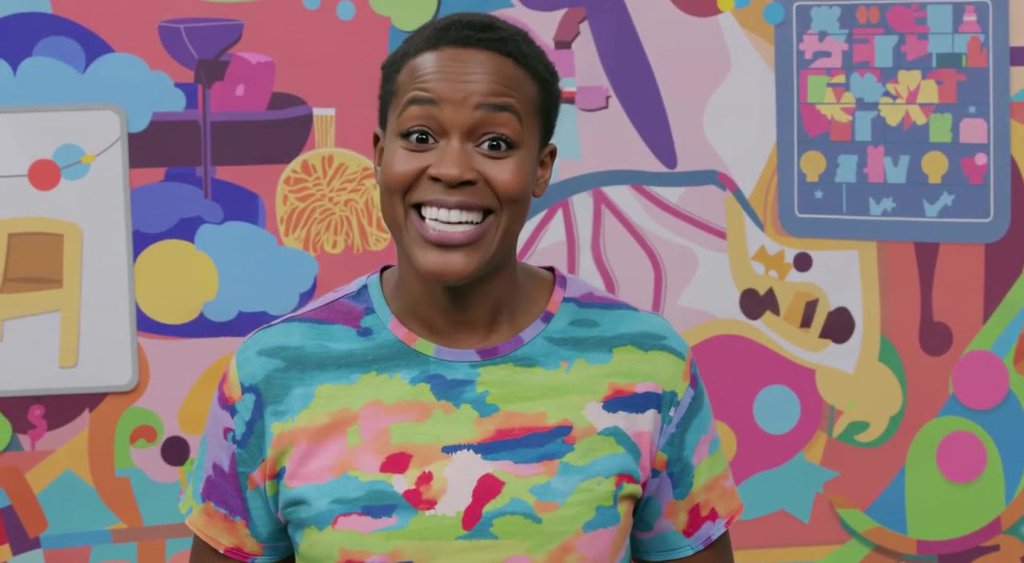} & \framecellw{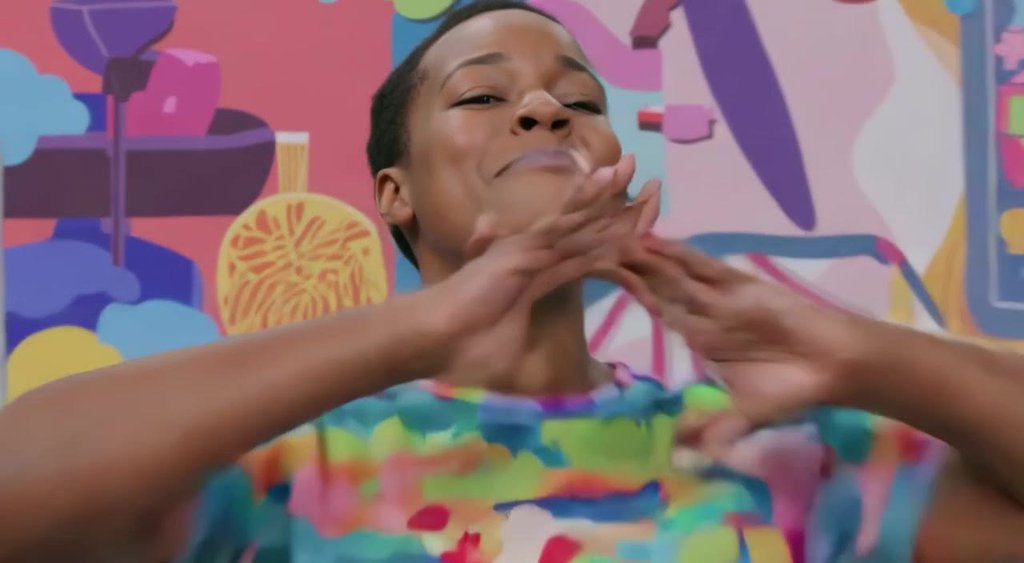} & \framecellw{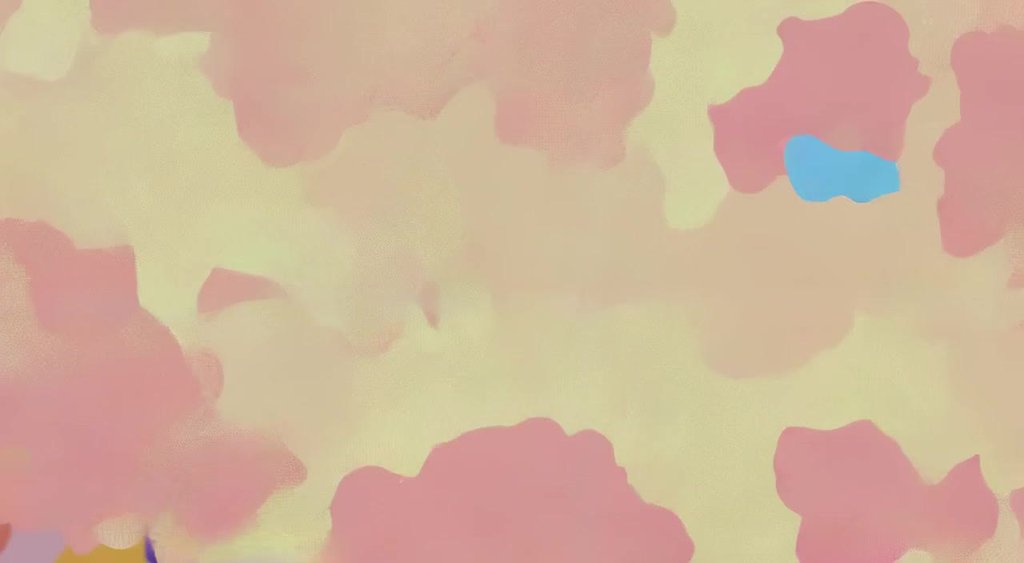} \\ \addlinespace[2pt]
    \textbf{PDD}
      & \framecellw{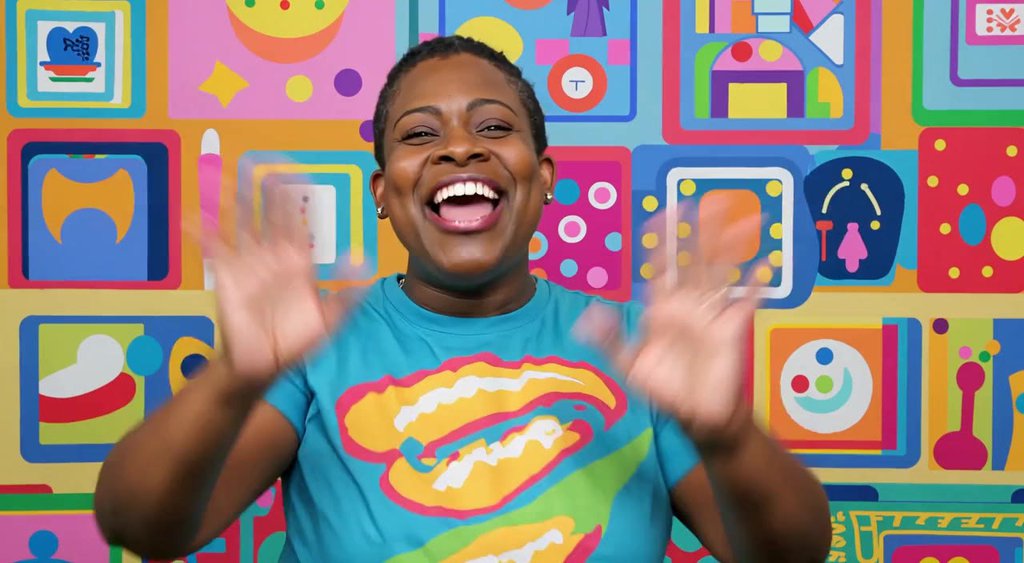} & \framecellw{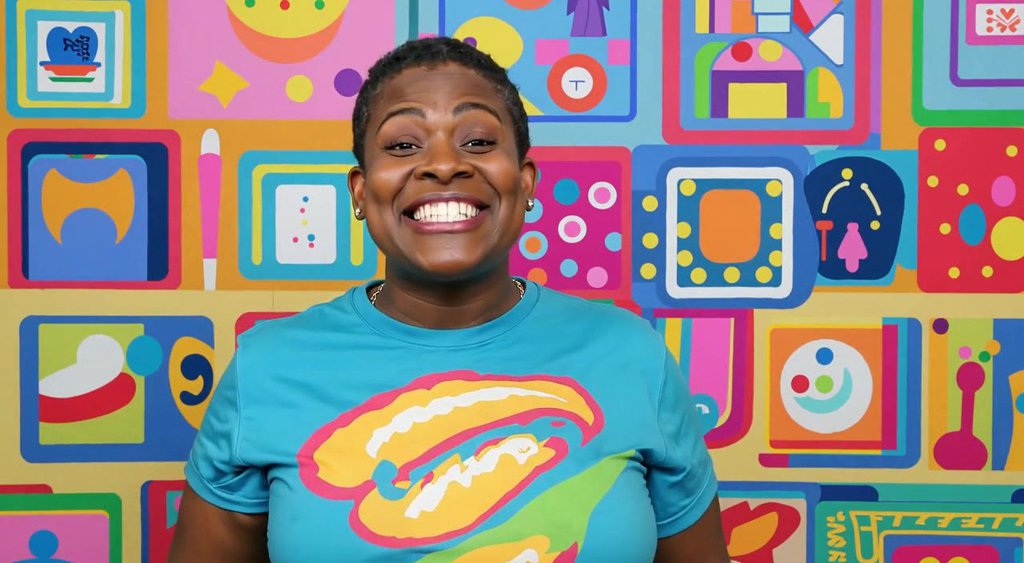} & \framecellw{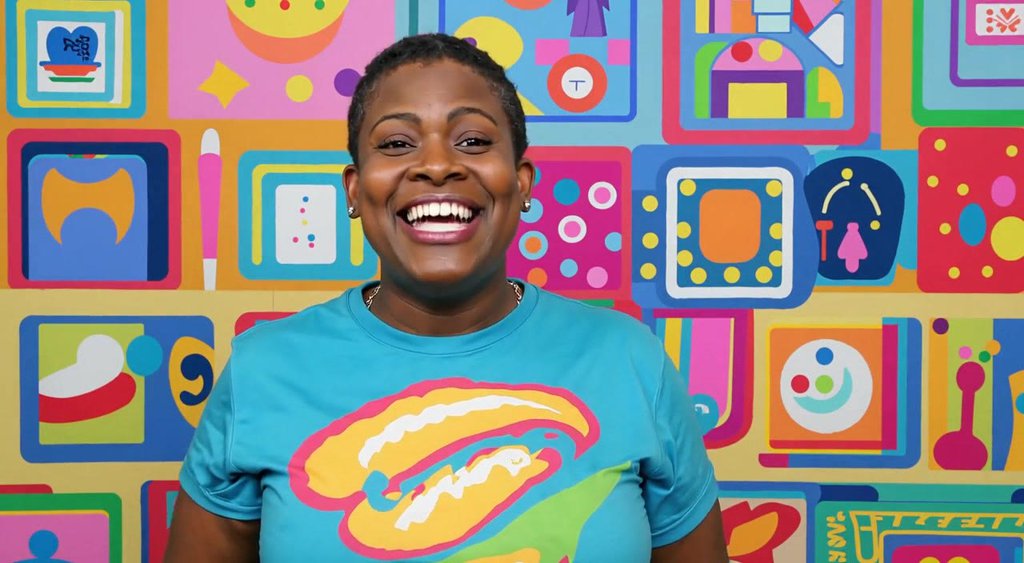} & \framecellw{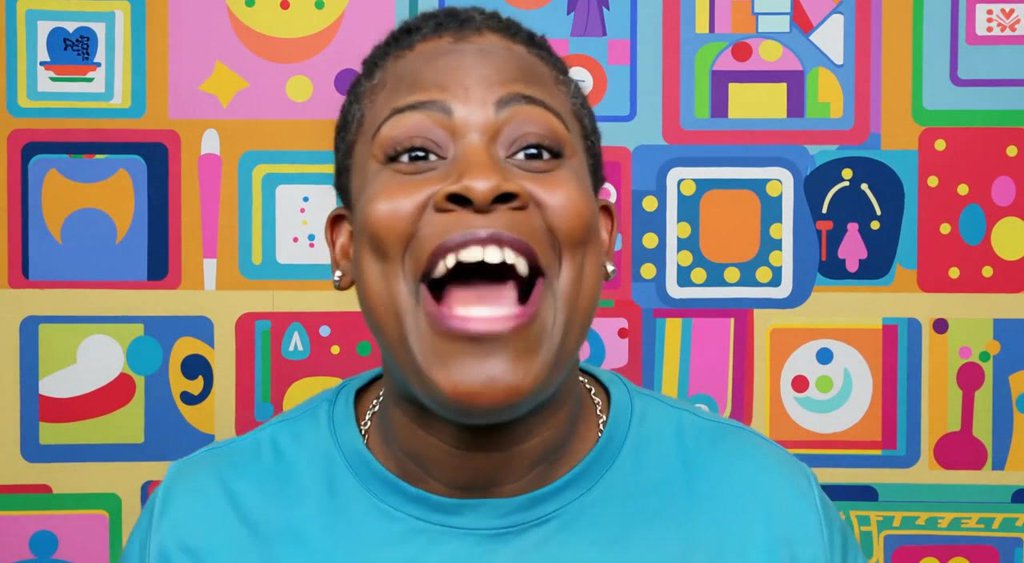} & \framecellw{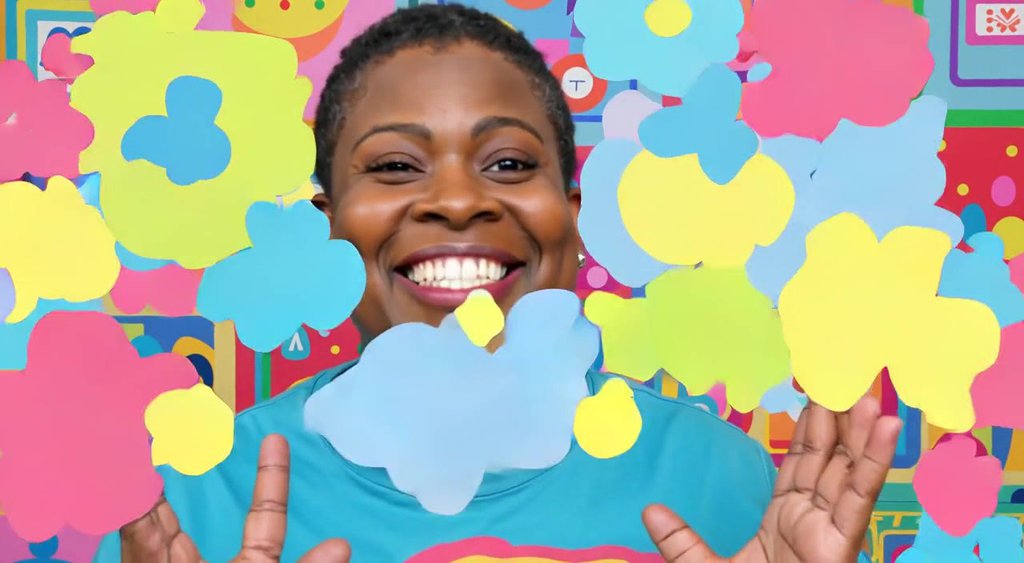} \\ \addlinespace[2pt]
    Distilled
      & \framecellw{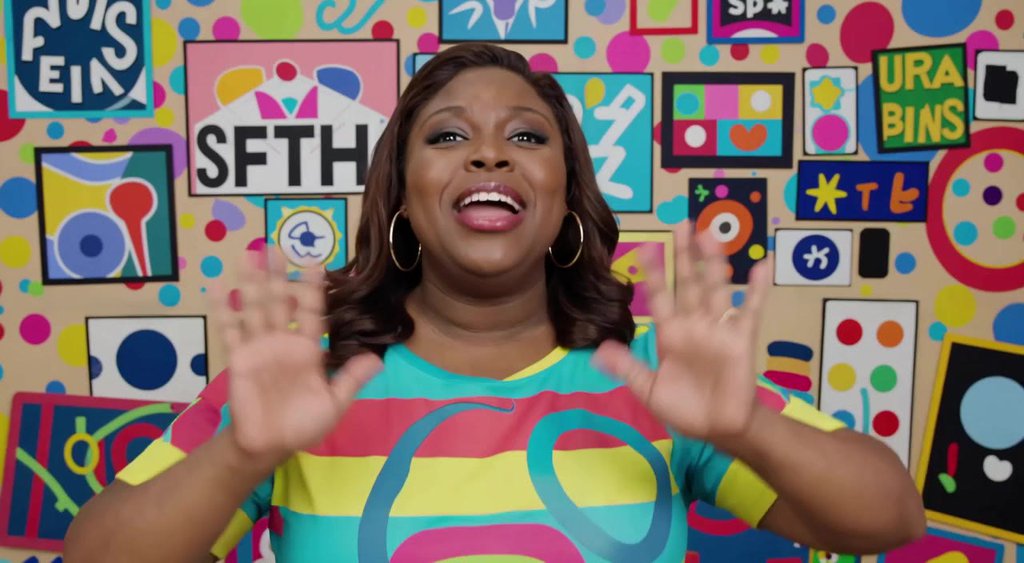} & \framecellw{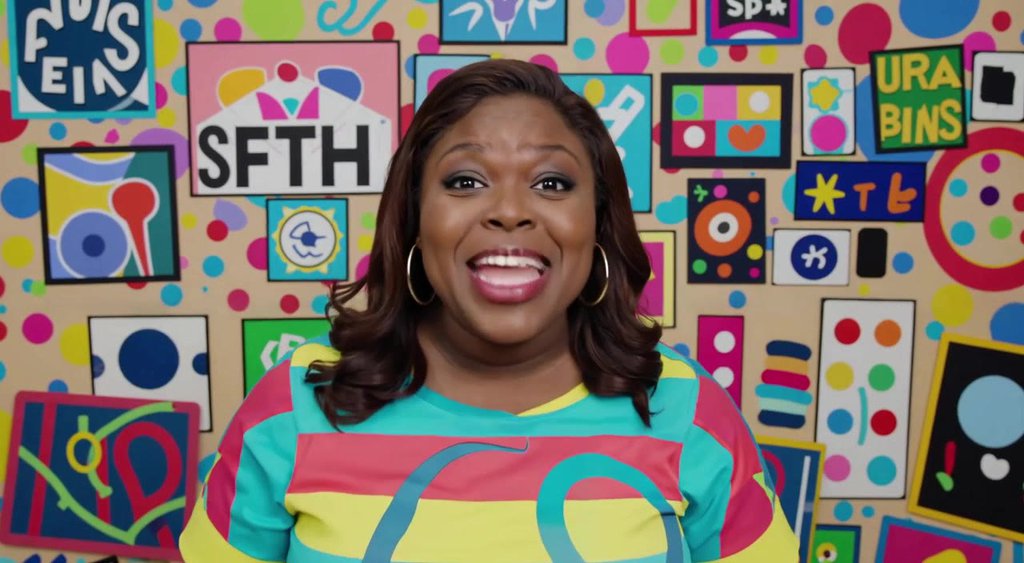} & \framecellw{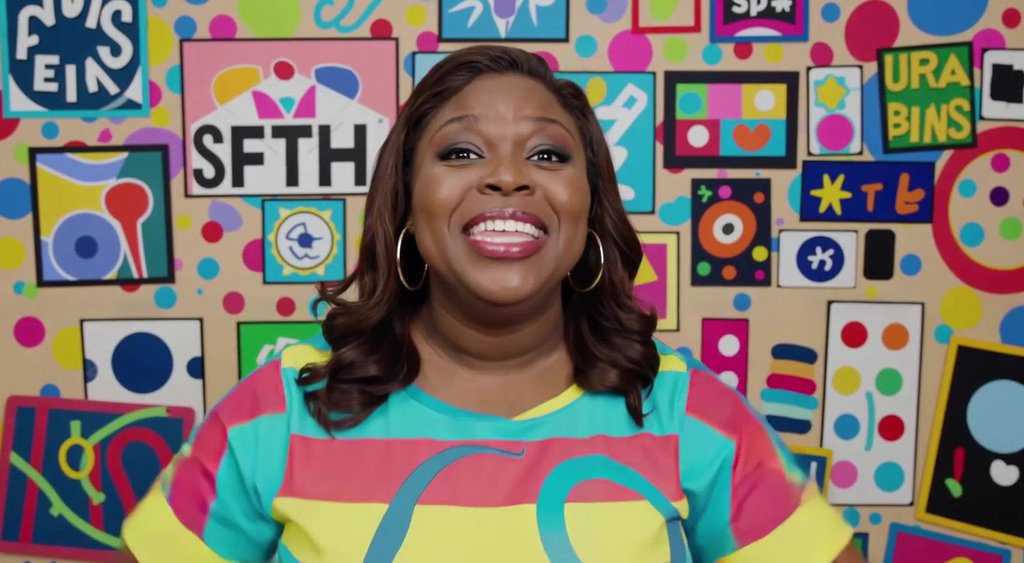} & \framecellw{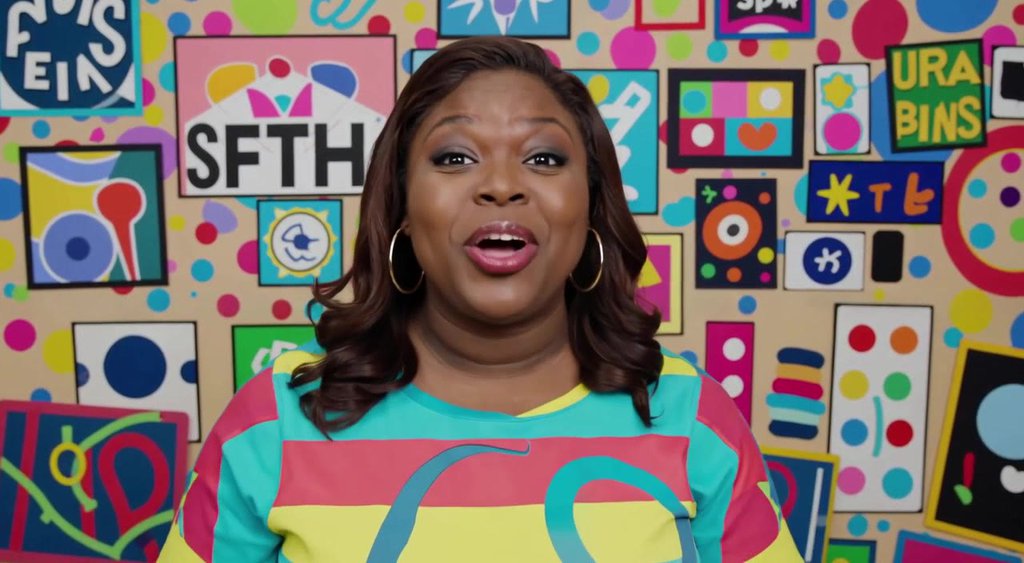} & \framecellw{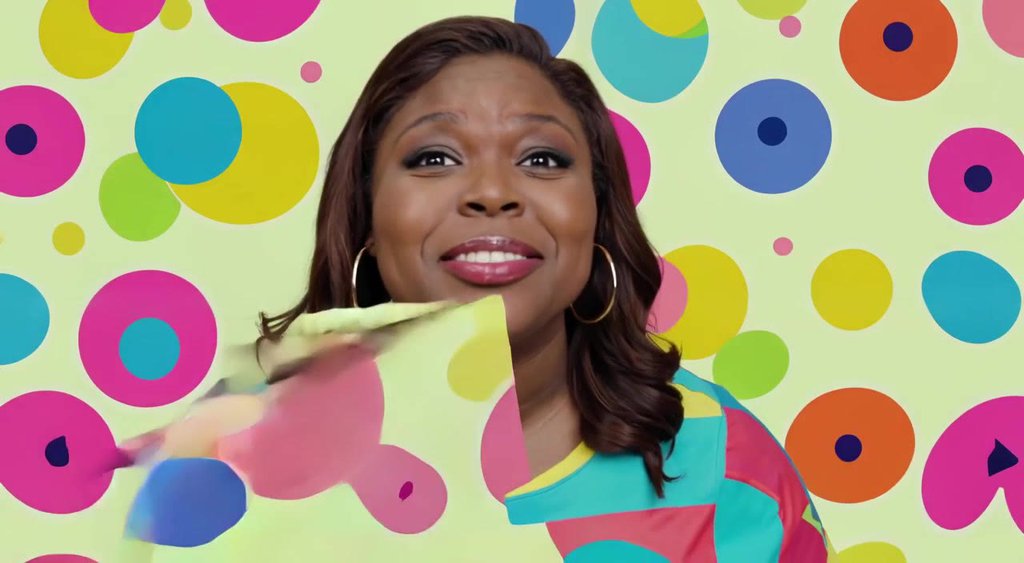} \\ \addlinespace[2pt]
                & \multicolumn{5}{@{} p{16.00cm} @{}}{\tiny\itshape ``In an animated and vibrant studio environment, the host, a cheerful and energetic adult, stands in front of a colorful backdrop featuring various shapes and objects in primary colors, accompanied by the bright, upbeat jingle of a children's show playing softly in the background. The host wears a bright, fun-themed t-shirt and has an enthusiastic smile, their fabric rustling slightly as they move. They speak directly to the camera with excitement, waving energetically, the sound of their hand-claps echoing in the bright, acoustic space of the studio. The host says with a warm, high-energy tone, "Hello, kids!" they pause briefly with a wide grin, then continue, "Welcome to another exciting episode of 'Learning with Fun.' Today, we're going to explore the amazing world of colors!" The camera zooms in on the host's face, capturing their animated expressions and gestures as they engage the audience with dynamic hand movements and playful facial expressions. As the background transitions smoothly from still images to animated elements representing different colors, a subtle, magical "pop" and "whoosh" sound effect accompanies each new graphic, enhancing the educational yet fun atmosphere.''} \\
  \end{tabular}%
  }
  \vspace{-10pt}
\end{figure*}

\begin{figure*}[htbp]
  \centering
  \resizebox{\textwidth}{!}{%
  \begin{tabular}{@{} r @{\hspace{6pt}} c@{\hspace{2pt}}c@{\hspace{2pt}}c@{\hspace{2pt}}c@{\hspace{2pt}}c @{}}
                & $t{=}0$s & $t{=}2.5$s & $t{=}5$s & $t{=}7.5$s & $t{=}10$s \\ \addlinespace[2pt]
    Teacher
      & \framecellw{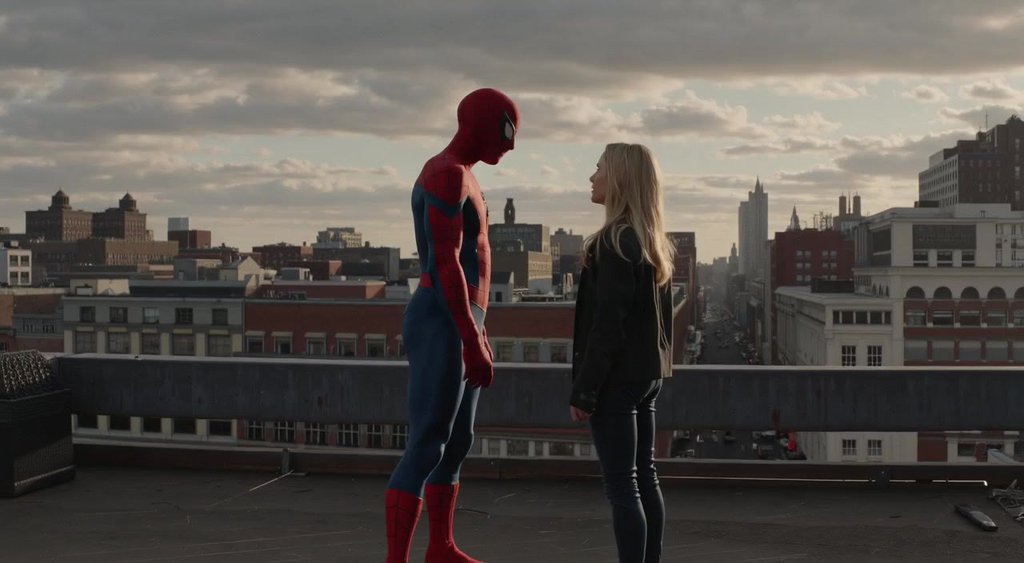} & \framecellw{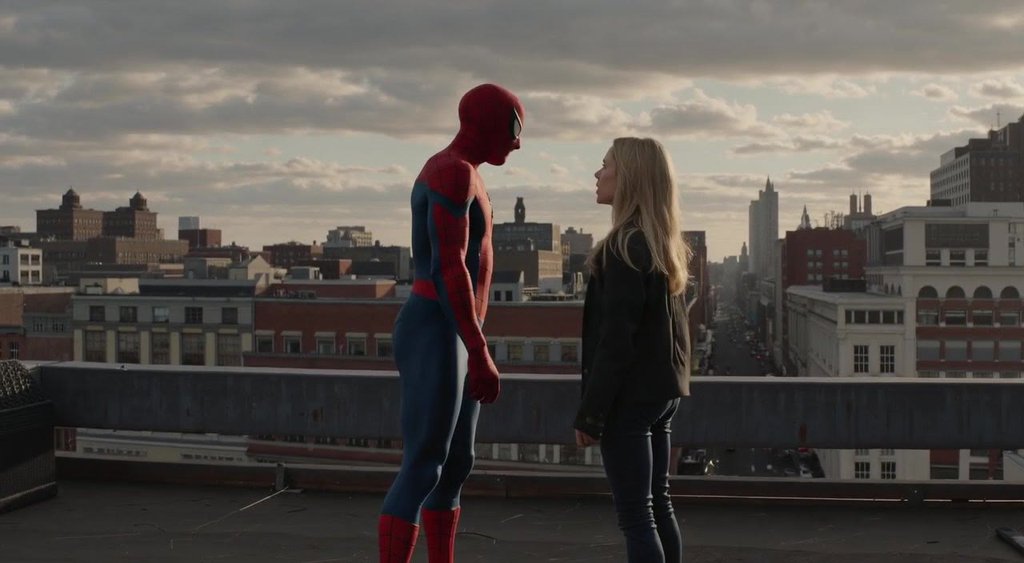} & \framecellw{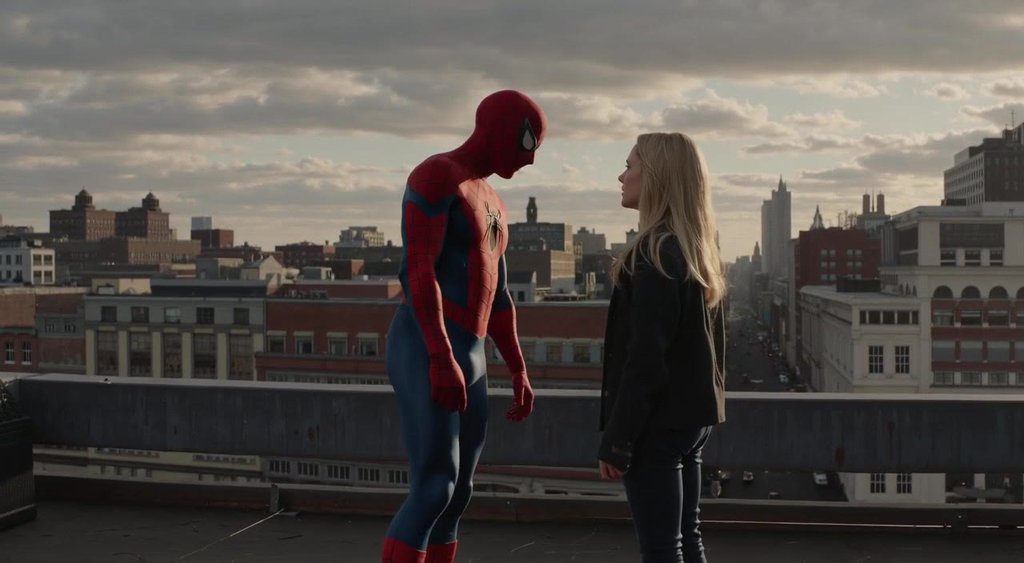} & \framecellw{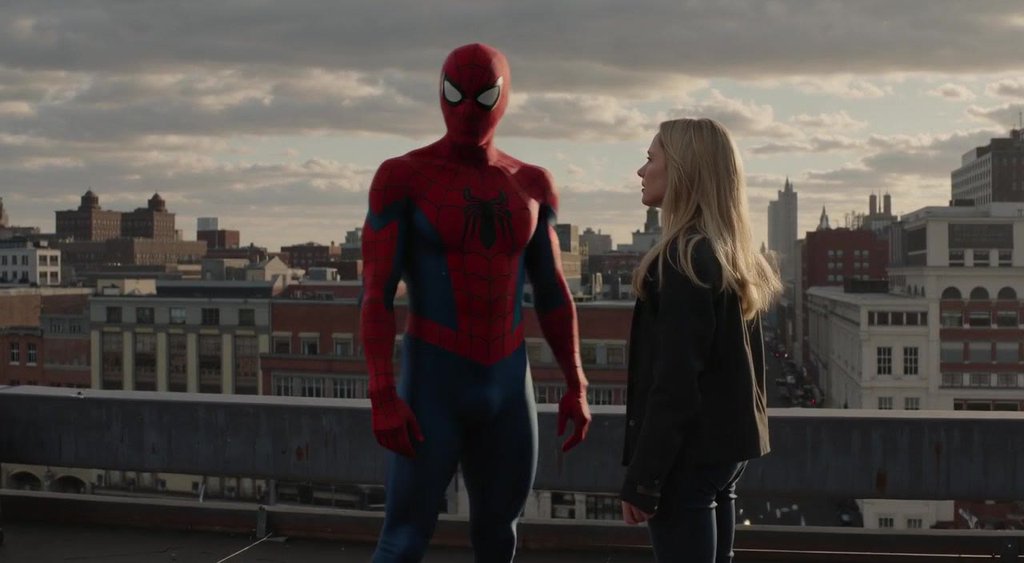} & \framecellw{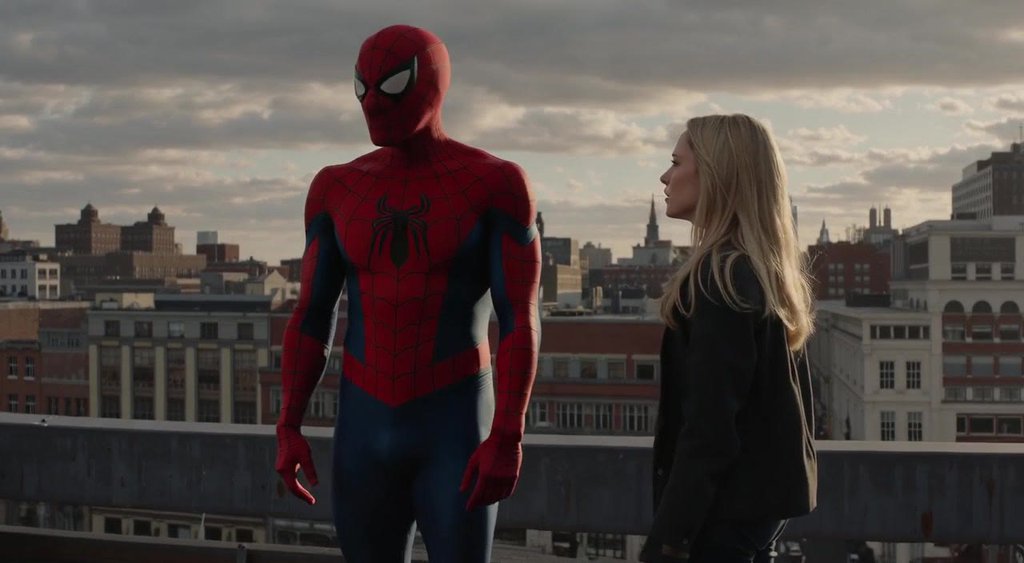} \\ \addlinespace[2pt]
    \textbf{PDD}
      & \framecellw{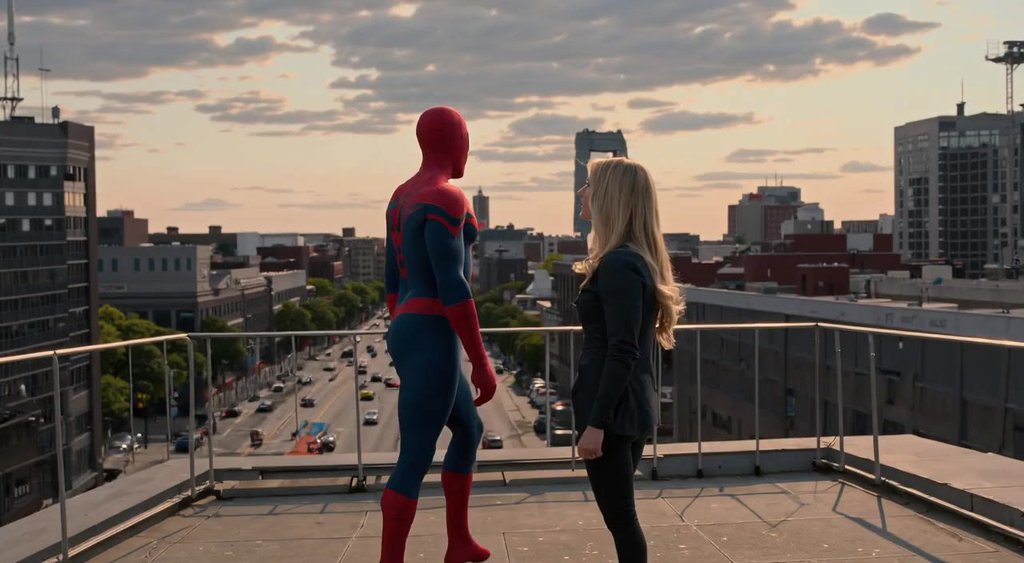} & \framecellw{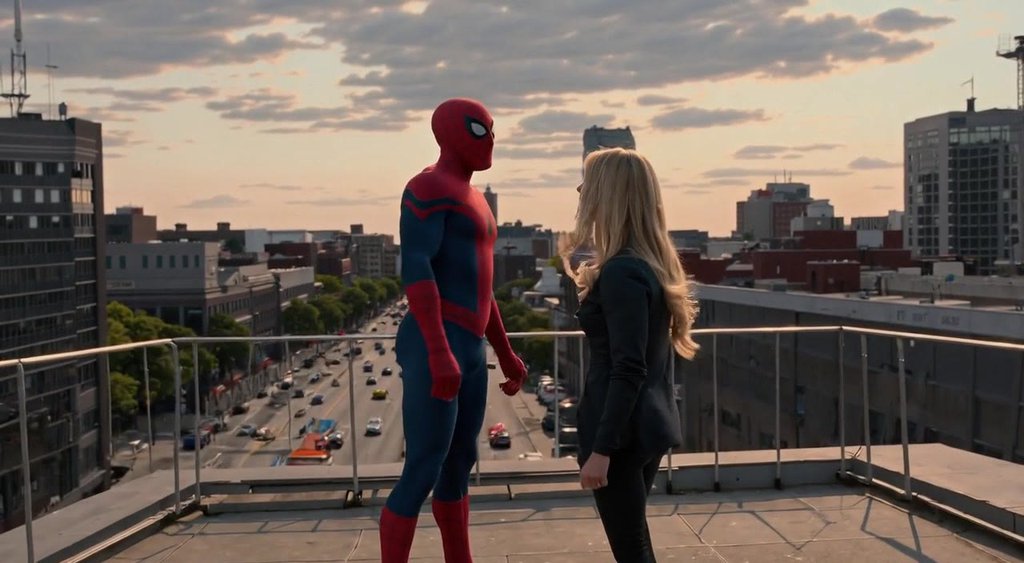} & \framecellw{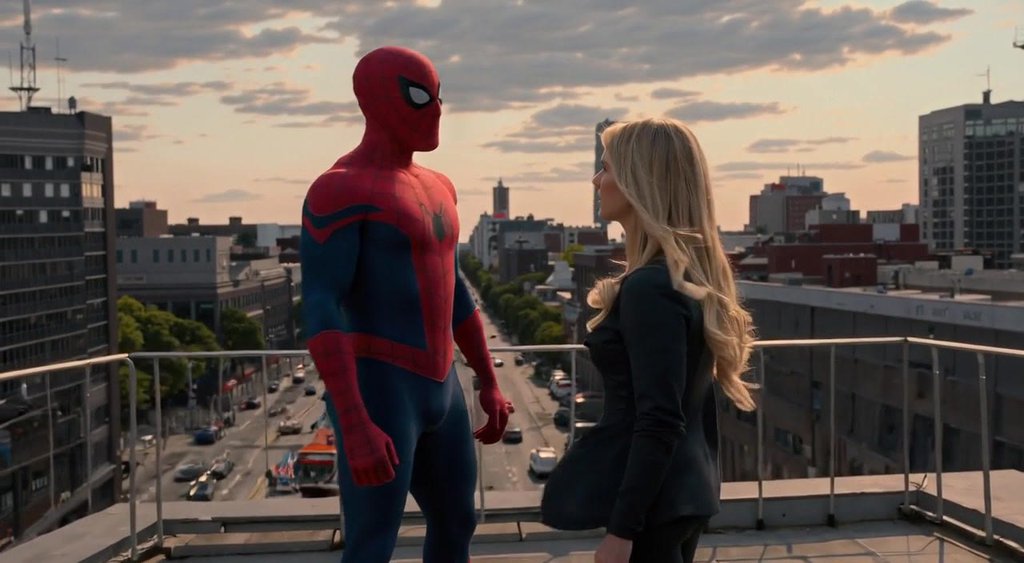} & \framecellw{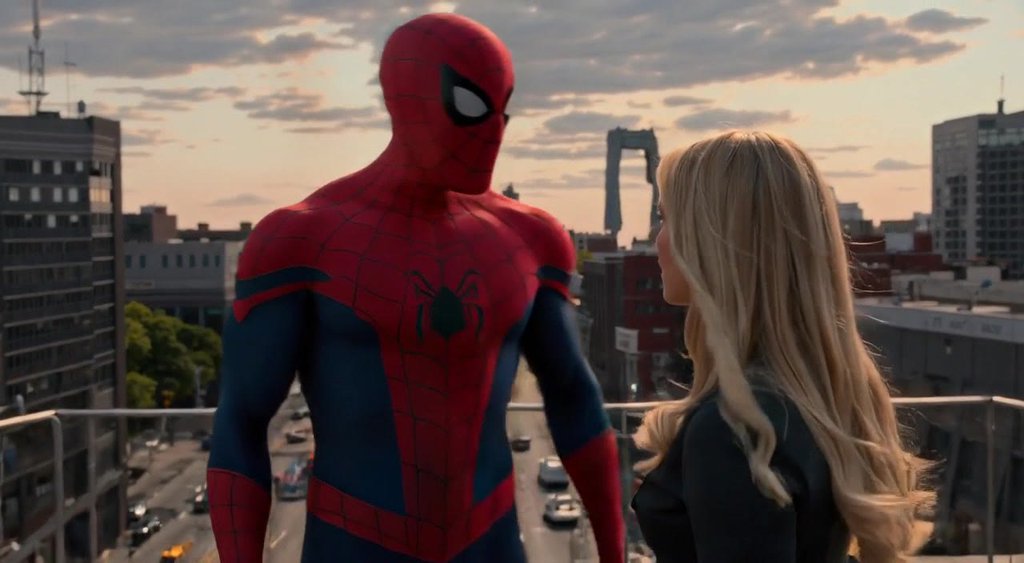} & \framecellw{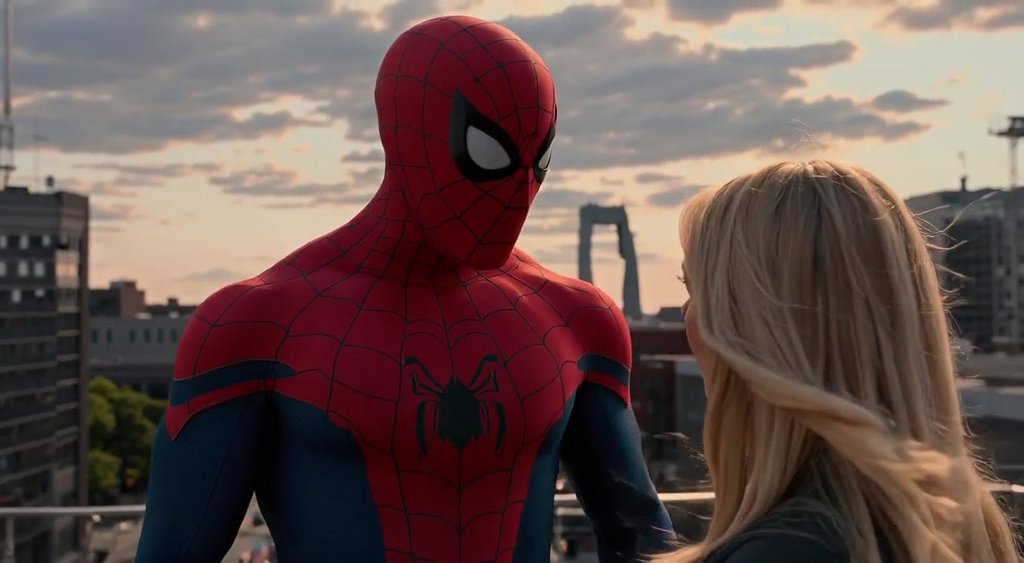} \\ \addlinespace[2pt]
    Distilled
      & \framecellw{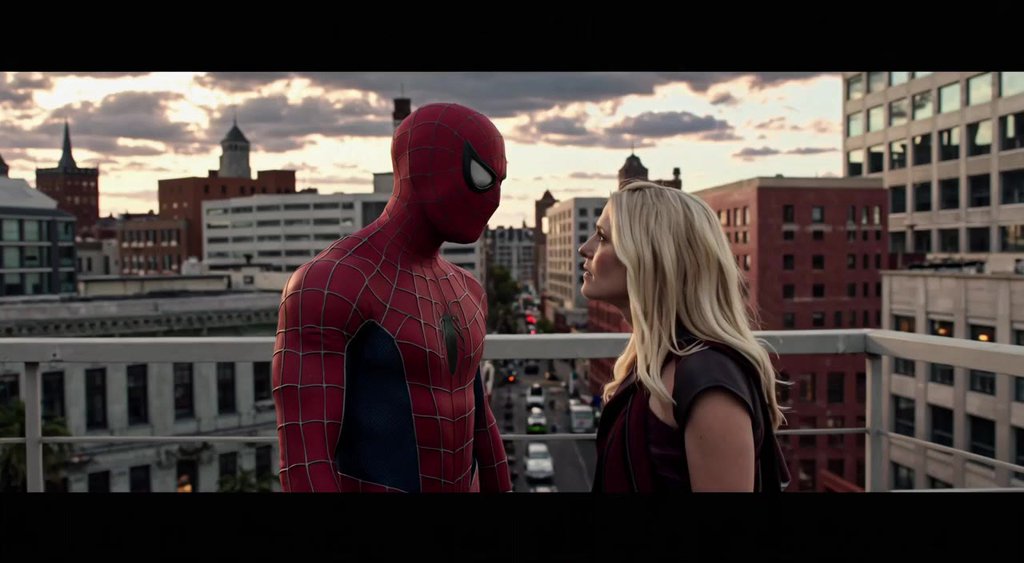} & \framecellw{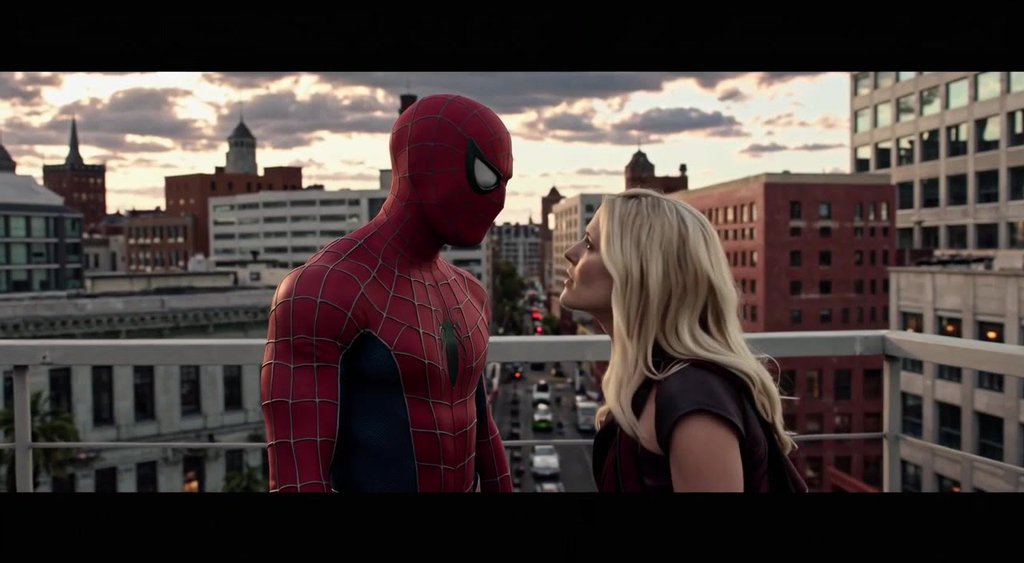} & \framecellw{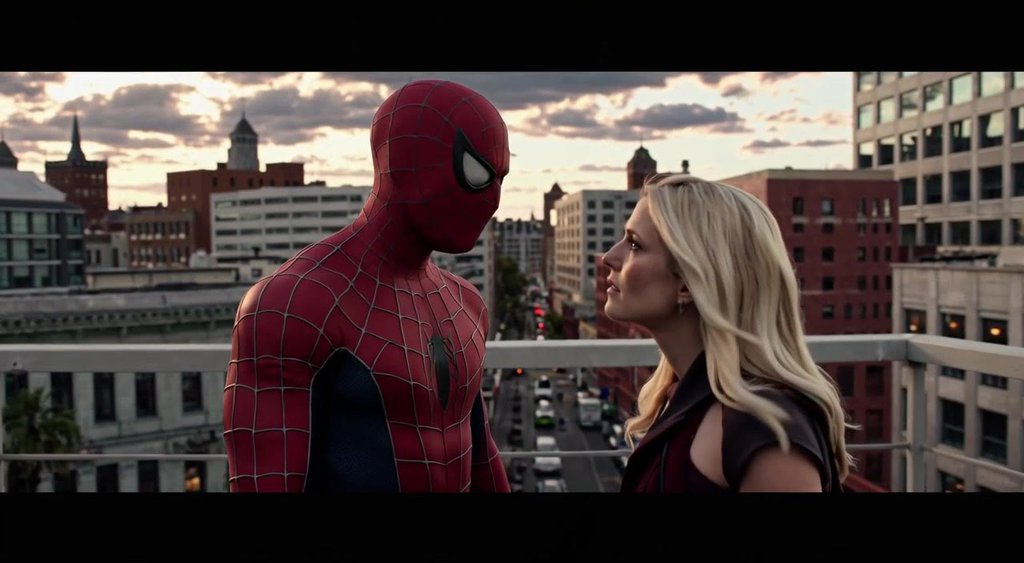} & \framecellw{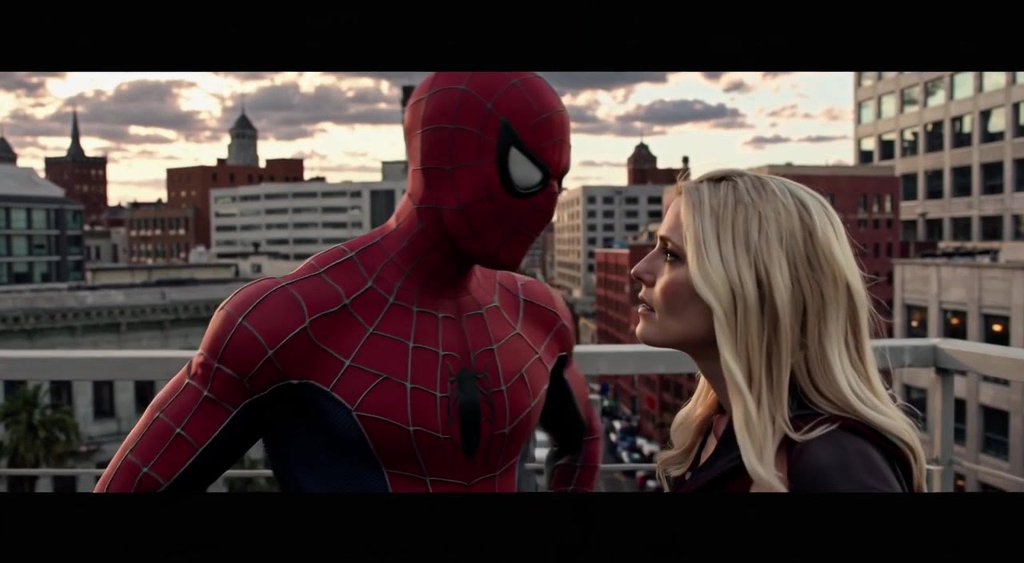} & \framecellw{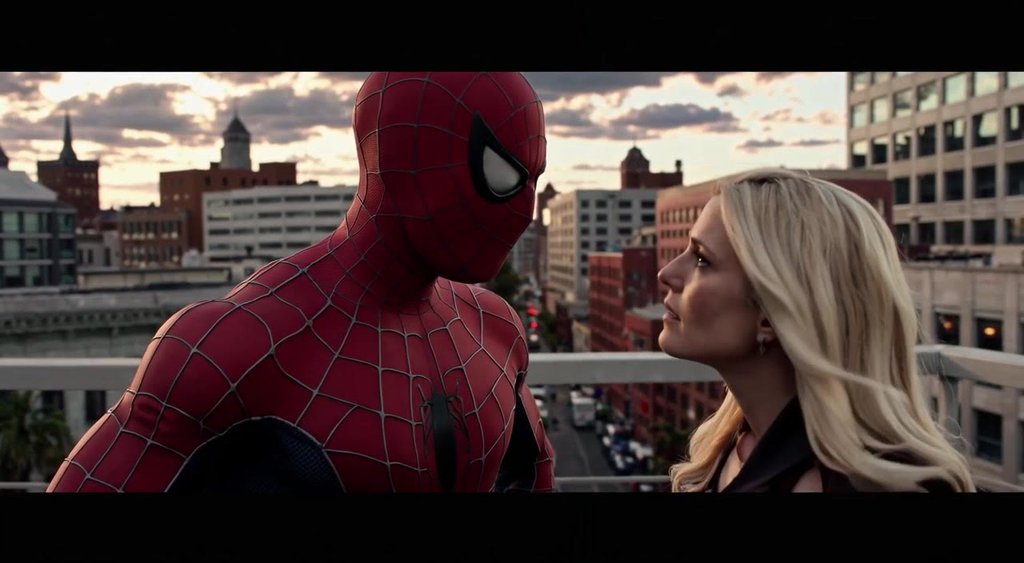} \\ \addlinespace[2pt]
                & \multicolumn{5}{@{} p{16.00cm} @{}}{\tiny\itshape ``In this scene from "The Amazing Spider-Man," the iconic superhero, Spider-Man, is engaged in a conversation with a woman. Spider-Man, dressed in his signature red and blue suit, stands facing the woman, who has long blonde hair, while the distant, low-frequency hum of city traffic rises from the streets below, muffled by the height of their rooftop vantage point. The setting is an urban rooftop with a cloudy sky, and as the wind whistles softly around them, the woman looks up at him, her voice clear and earnest as she says, "You don't have to do this alone," followed by a brief, heavy silence. Spider-Man shifts his weight, the subtle, rhythmic creak of his suit’s fabric audible as he sighs, then replies in a low, gravelly, and slightly strained tone, "...I have to." The lighting suggests late afternoon, and the acoustic space feels vast and open, carrying the faint, echoing sirens of the city far beneath them as they stand in a moment of intense, significant exchange.''} \\
  \end{tabular}%
  }
  \caption{LTX-2.3 teacher ($4{\times}30$ NFE) vs.\ PDD (8 NFE) vs.\ official distilled model (8 NFE).}
\end{figure*}

\begin{figure*}[htbp]
  \centering
  \resizebox{\textwidth}{!}{%
  \begin{tabular}{@{} r @{\hspace{6pt}} c@{\hspace{2pt}}c@{\hspace{2pt}}c@{\hspace{2pt}}c@{\hspace{2pt}}c @{}}
                & $t{=}0$s & $t{=}2.5$s & $t{=}5$s & $t{=}7.5$s & $t{=}10$s \\ \addlinespace[2pt]
    Teacher
      & \framecellw{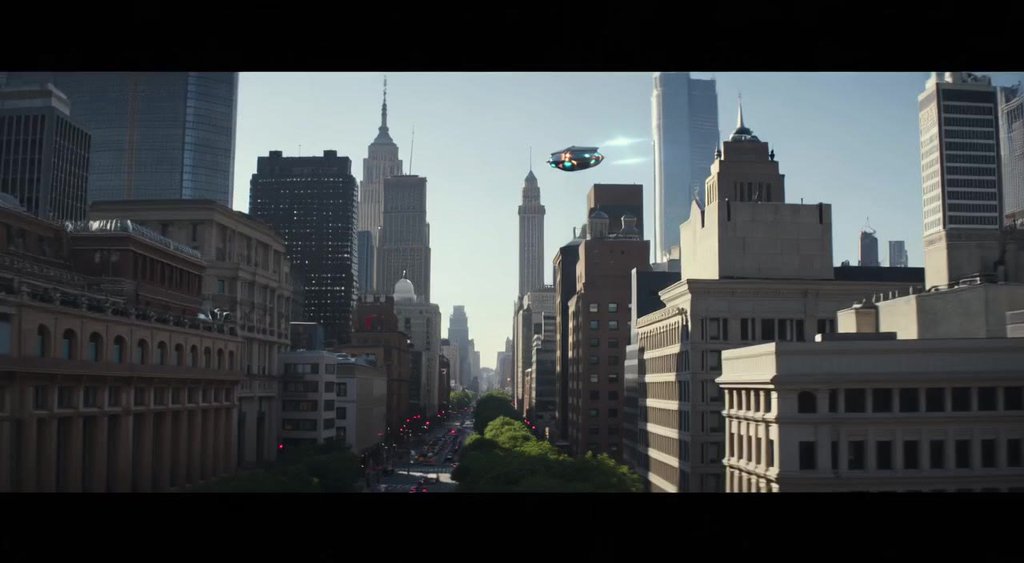} & \framecellw{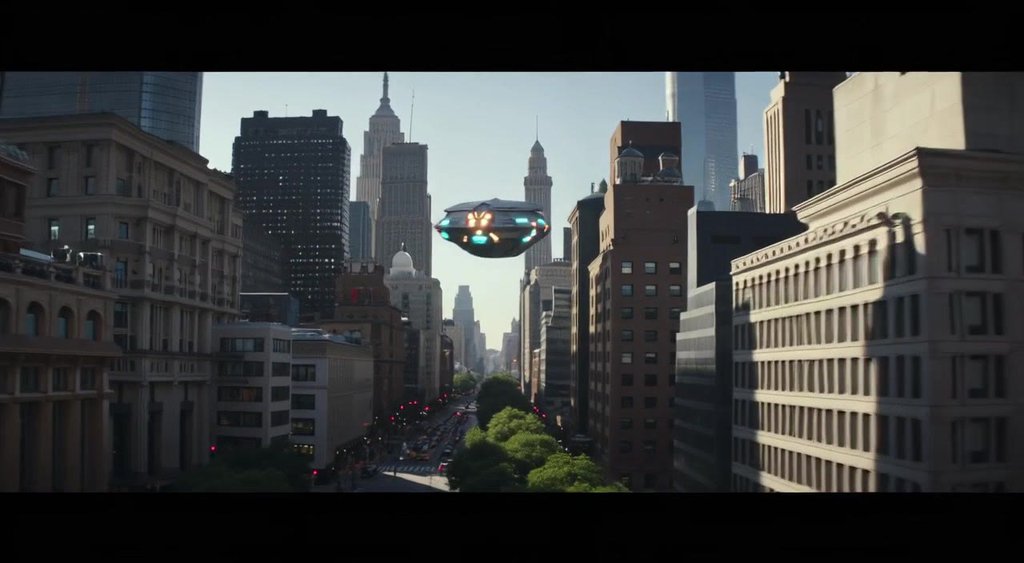} & \framecellw{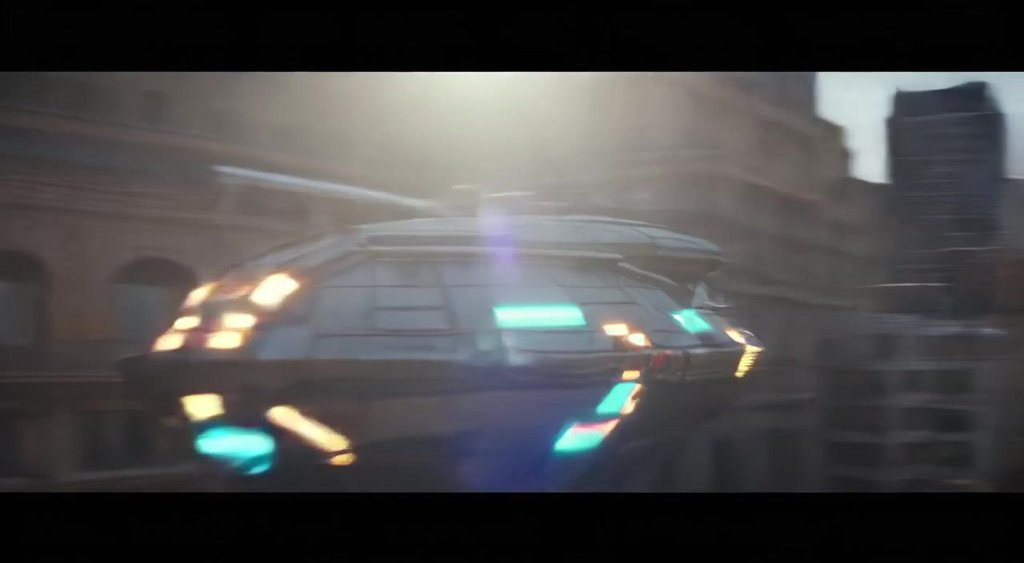} & \framecellw{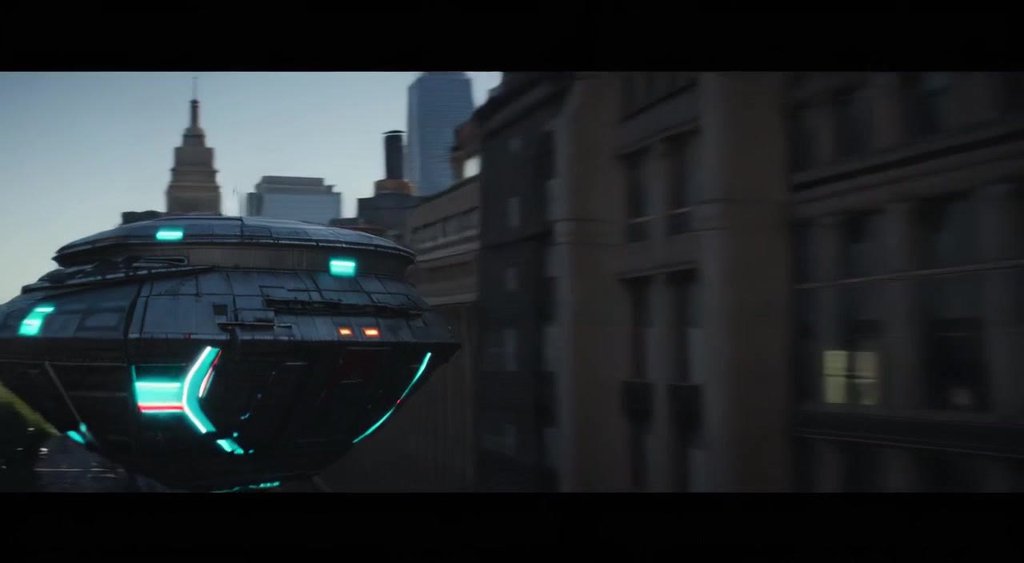} & \framecellw{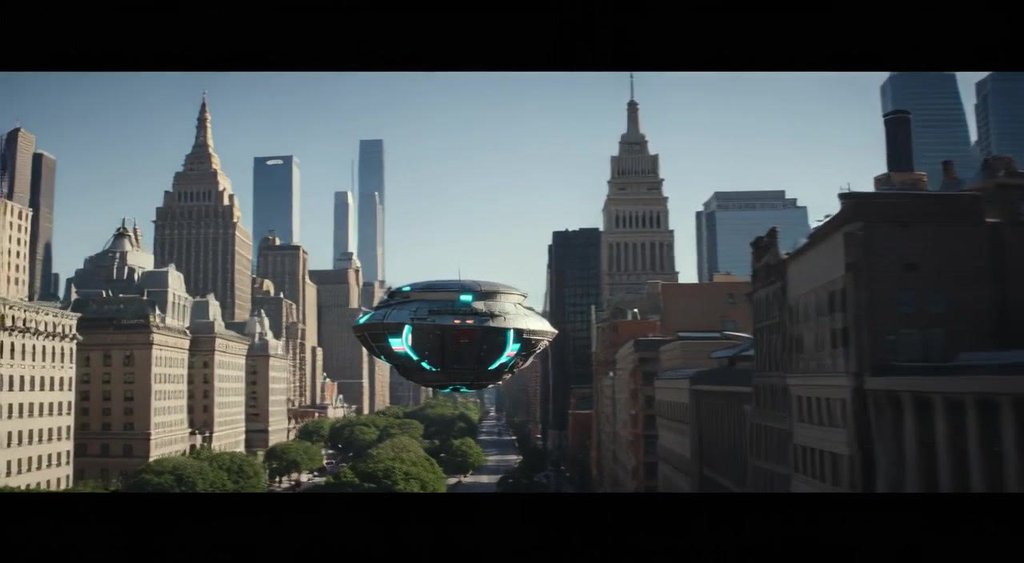} \\ \addlinespace[2pt]
    \textbf{PDD}
      & \framecellw{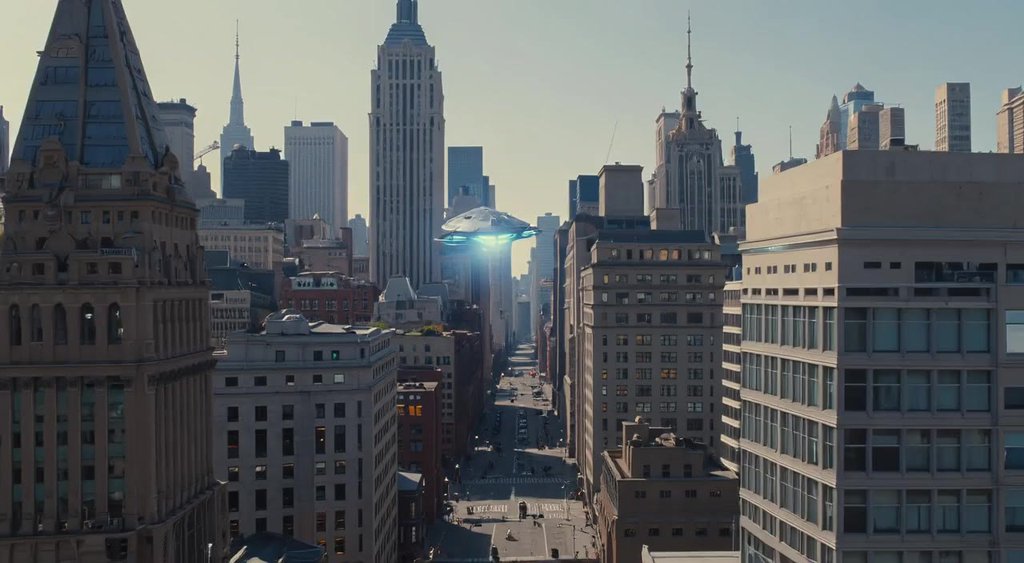} & \framecellw{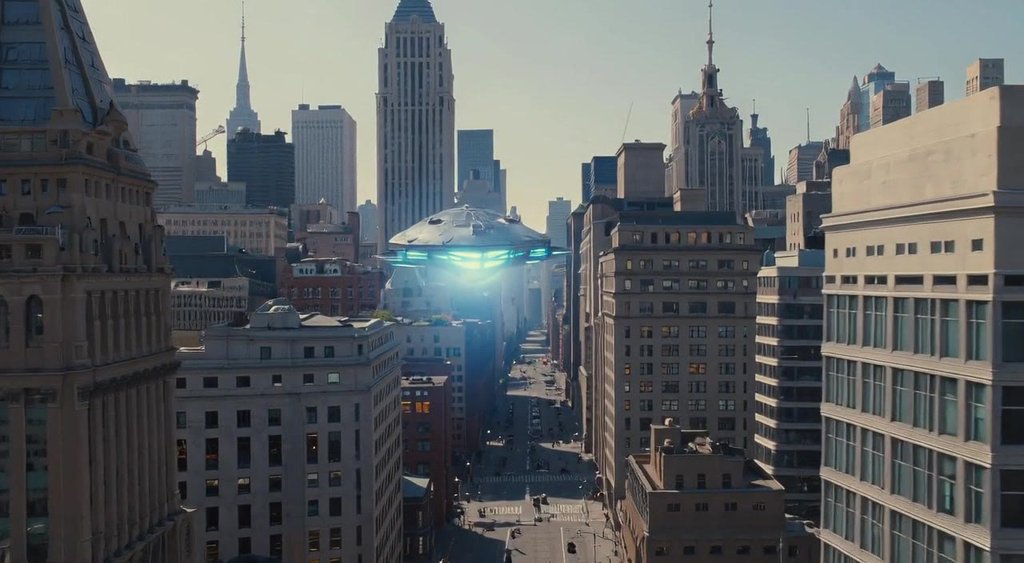} & \framecellw{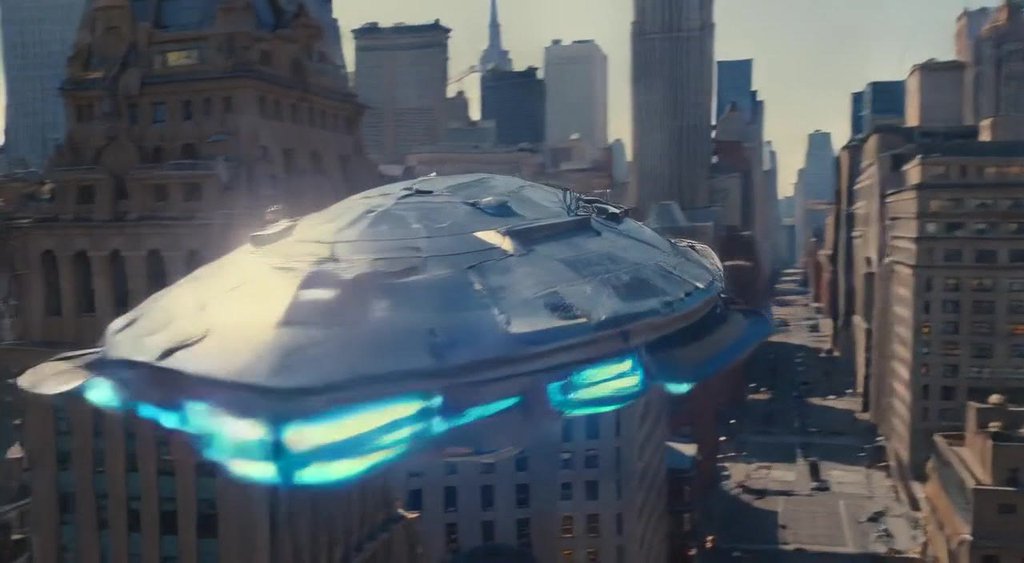} & \framecellw{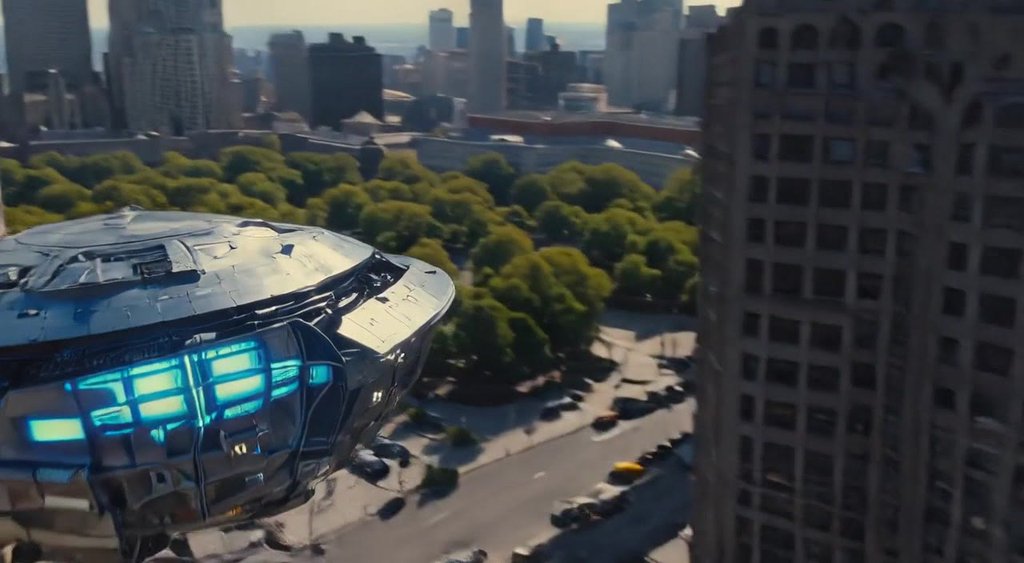} & \framecellw{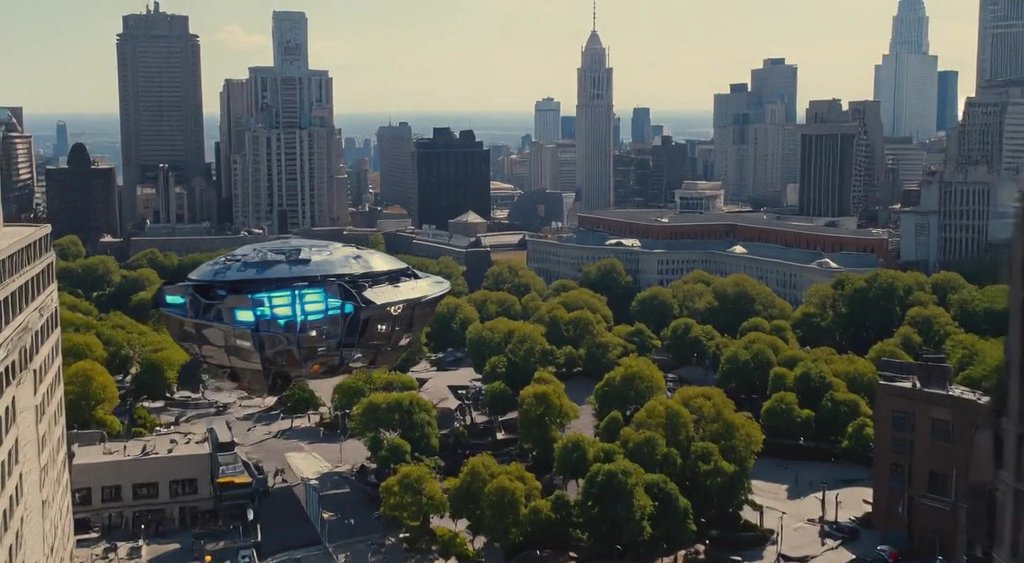} \\ \addlinespace[2pt]
    Distilled
      & \framecellw{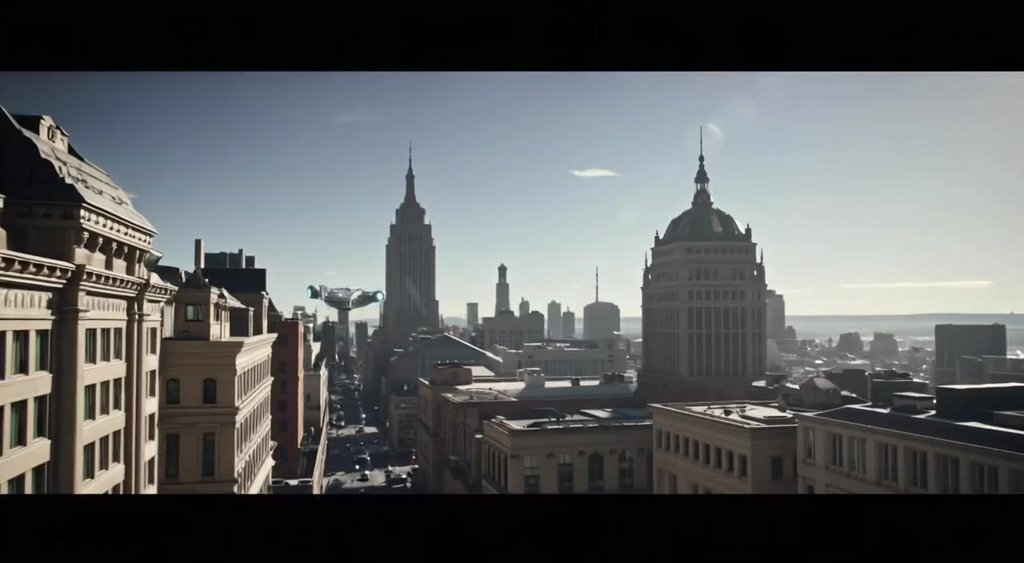} & \framecellw{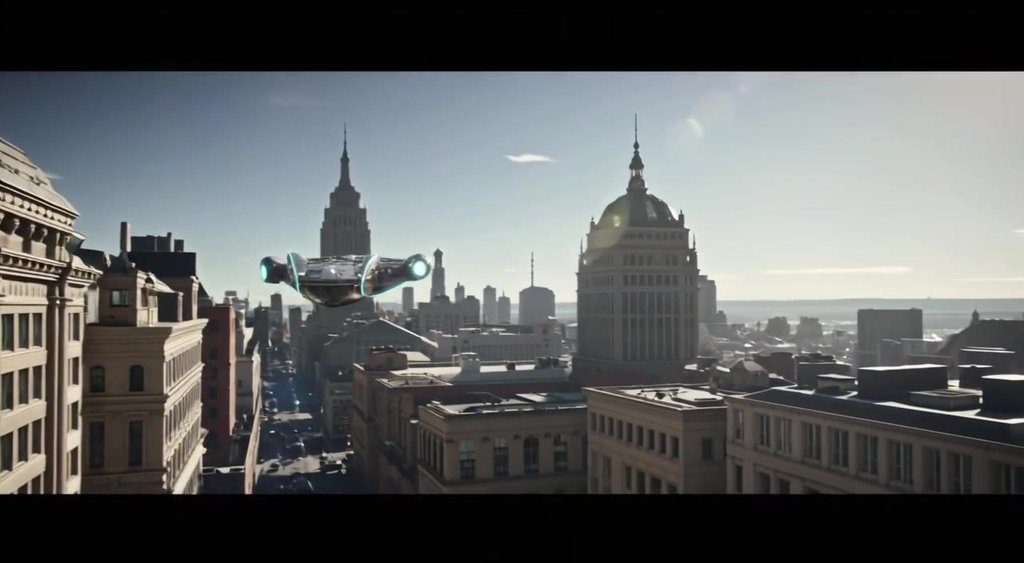} & \framecellw{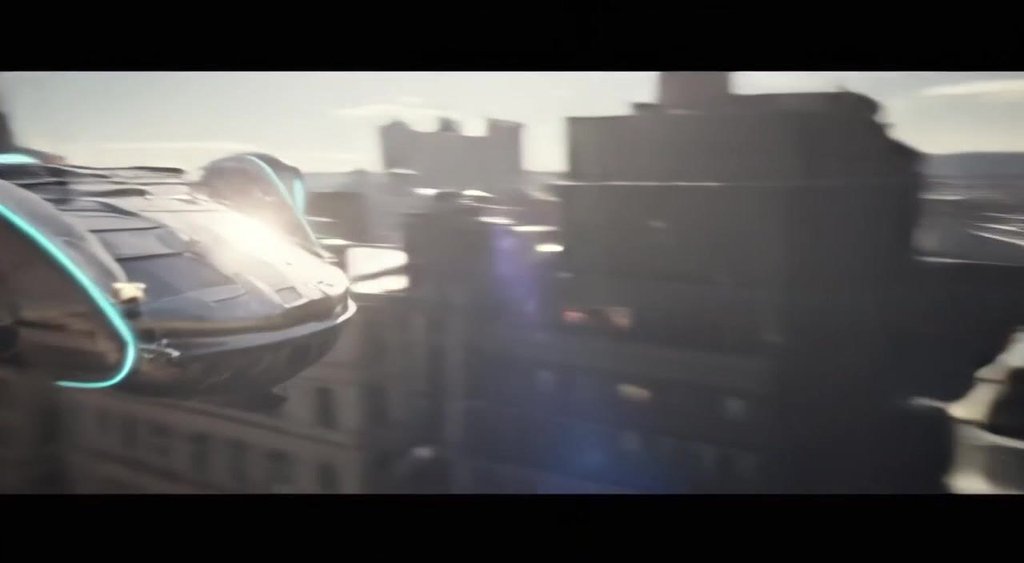} & \framecellw{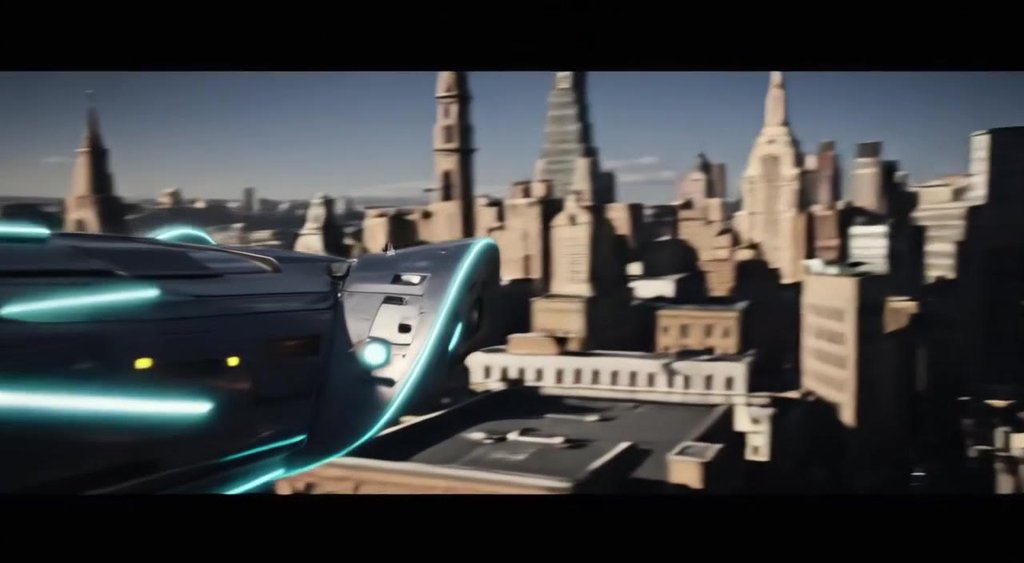} & \framecellw{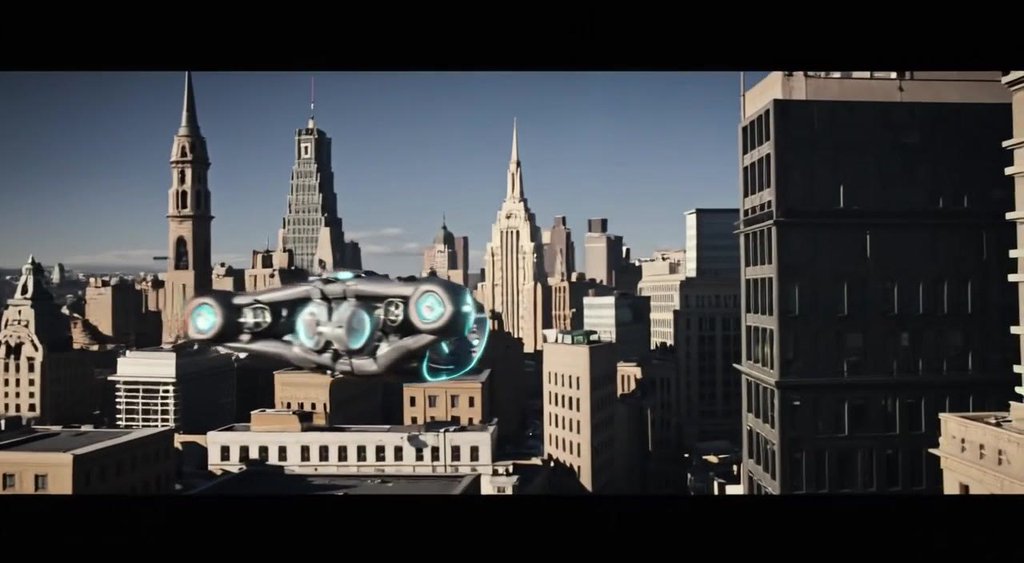} \\ \addlinespace[2pt]
                & \multicolumn{5}{@{} p{16.00cm} @{}}{\tiny\itshape ``Hyperrealistic cinematic footage of an alien UFO flying through New York City during daytime, accompanied by a low-frequency, rhythmic hum that vibrates through the air and a sharp, metallic whistling sound as it displaces the wind. The UFO emits dynamic lighting effects, casting dramatic shadows and illuminating buildings with a summilux lens effect, while the distant, muffled roar of city traffic and sirens fades beneath the craft's powerful, surging propulsion noise. The camera follows the UFO with smooth tracking shots, capturing its sleek metallic surface and vibrant neon glow as it speeds past iconic landmarks such as the Empire State Building and Central Park with a faint, high-pitched electronic whine. The video is shot at 25fps, providing a fluid and immersive experience, while close-up and mid-shot angles showcase the intricate details of the UFO as it navigates the bustling cityscape, the ambient sounds of the city echoing and shifting in the vast acoustic space as the craft passes overhead.''} \\
  \end{tabular}%
  }
  \vspace{-10pt}
\end{figure*}

\begin{figure*}[htbp]
  \centering
  \resizebox{\textwidth}{!}{%
  \begin{tabular}{@{} r @{\hspace{6pt}} c@{\hspace{2pt}}c@{\hspace{2pt}}c@{\hspace{2pt}}c@{\hspace{2pt}}c @{}}
                & $t{=}0$s & $t{=}2.5$s & $t{=}5$s & $t{=}7.5$s & $t{=}10$s \\ \addlinespace[2pt]
    Teacher
      & \framecellw{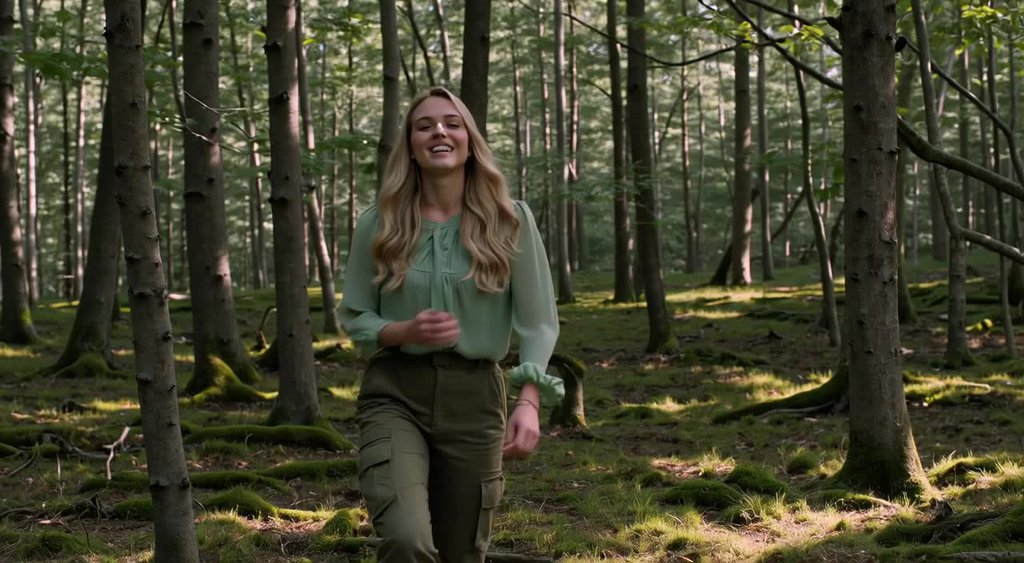} & \framecellw{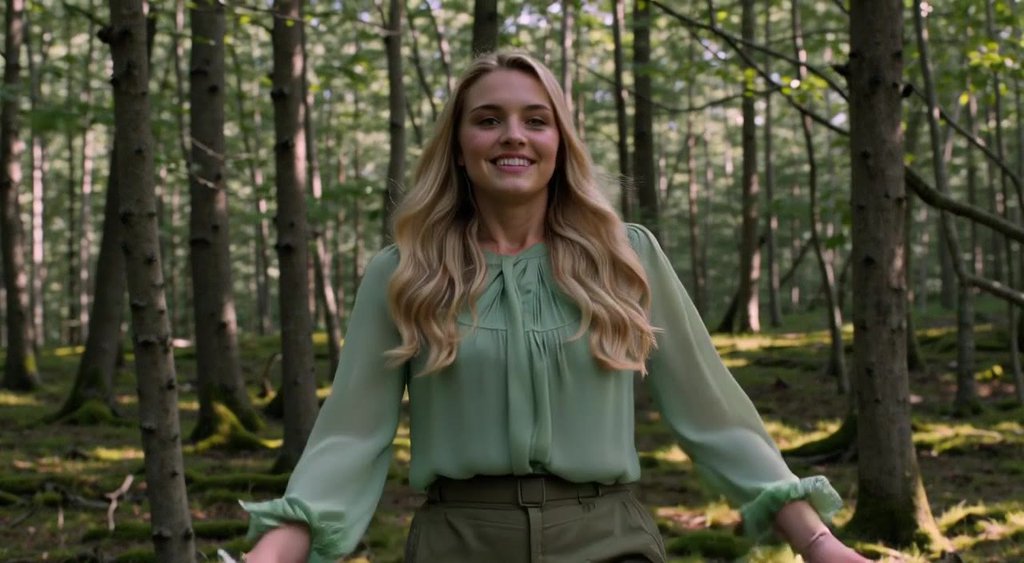} & \framecellw{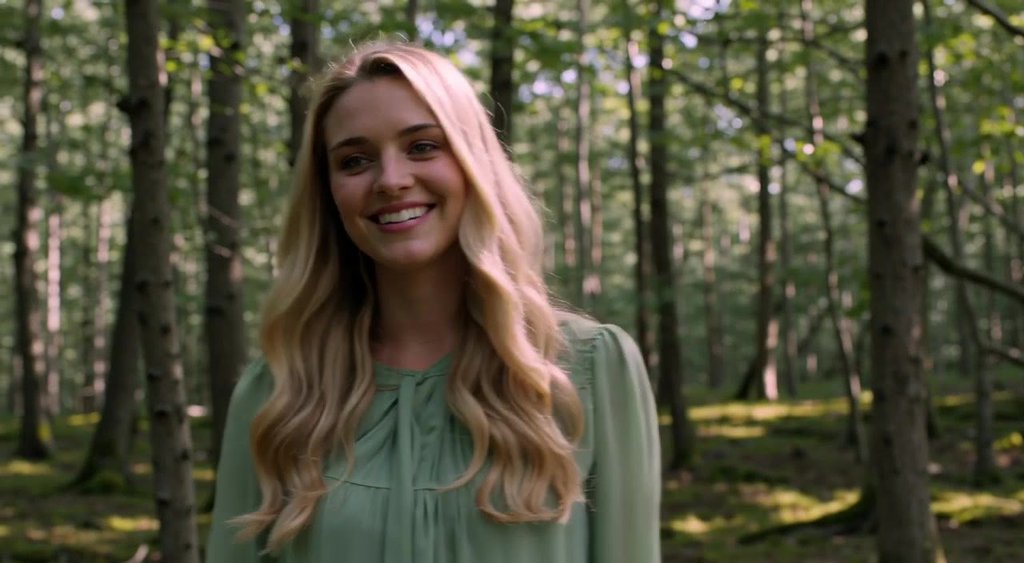} & \framecellw{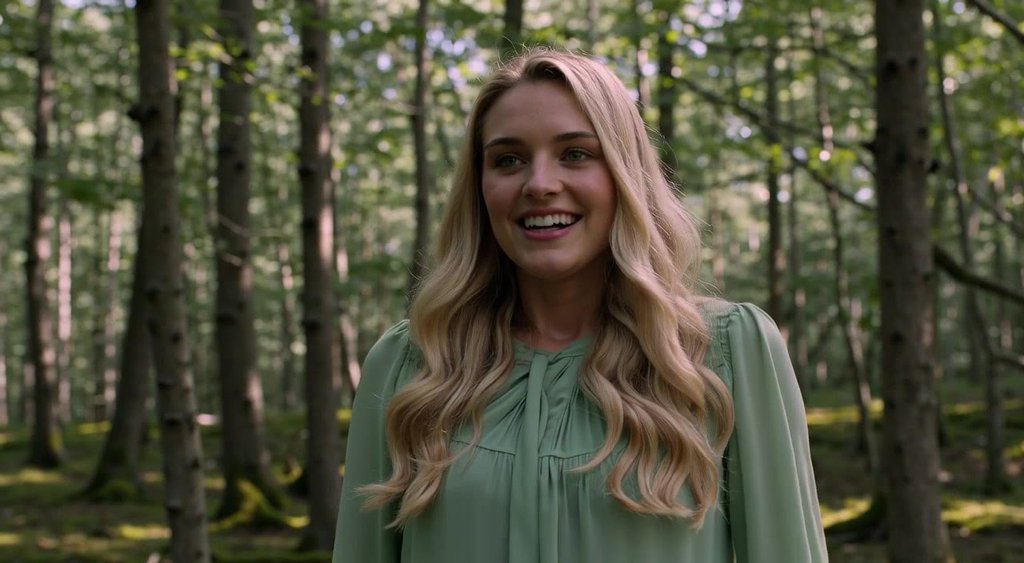} & \framecellw{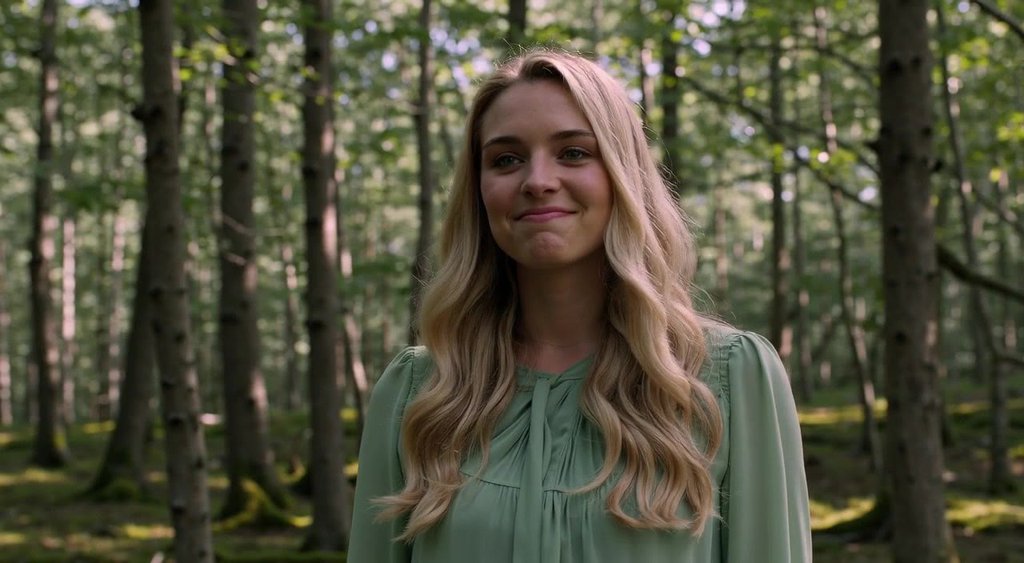} \\ \addlinespace[2pt]
    \textbf{PDD}
      & \framecellw{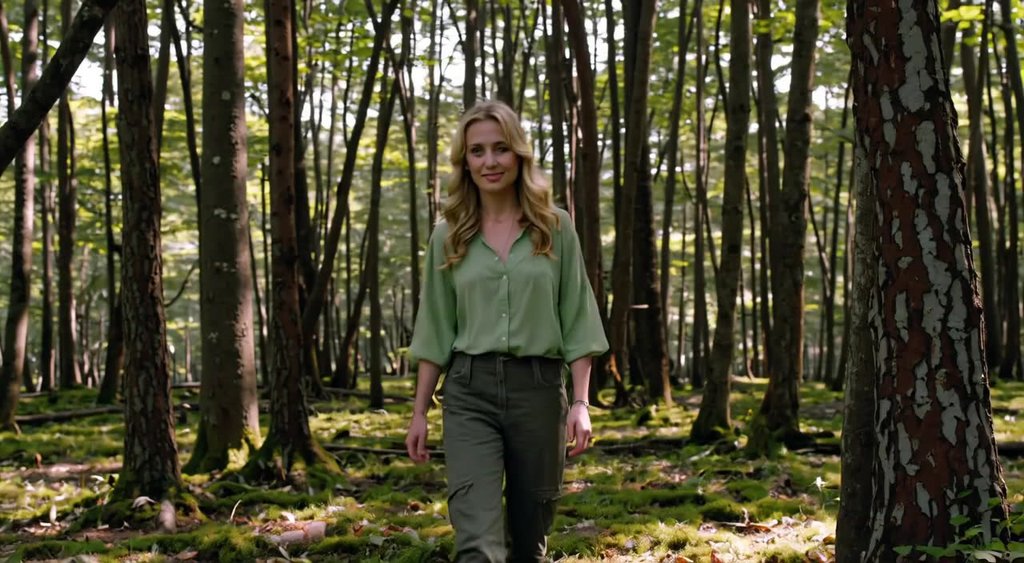} & \framecellw{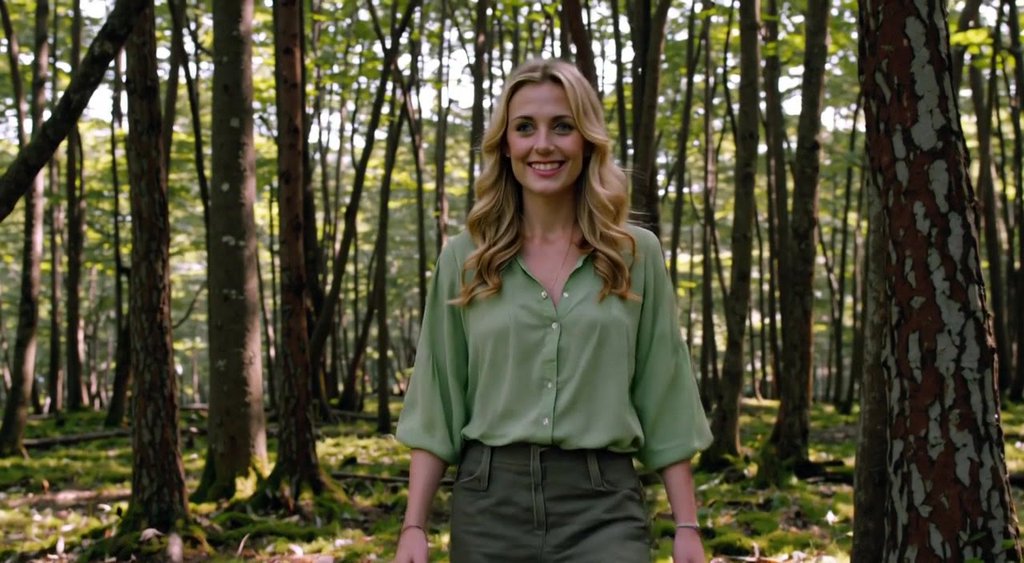} & \framecellw{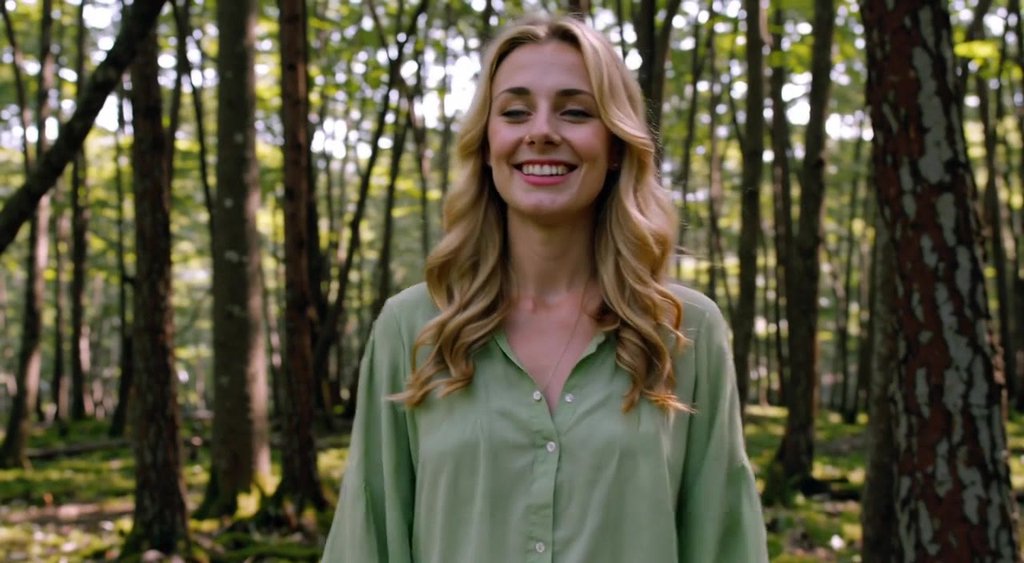} & \framecellw{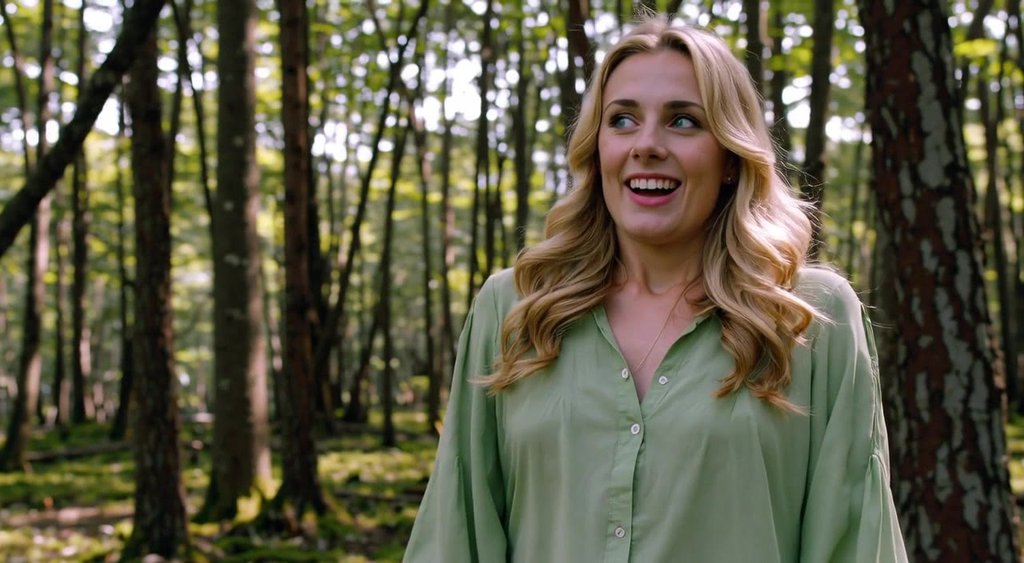} & \framecellw{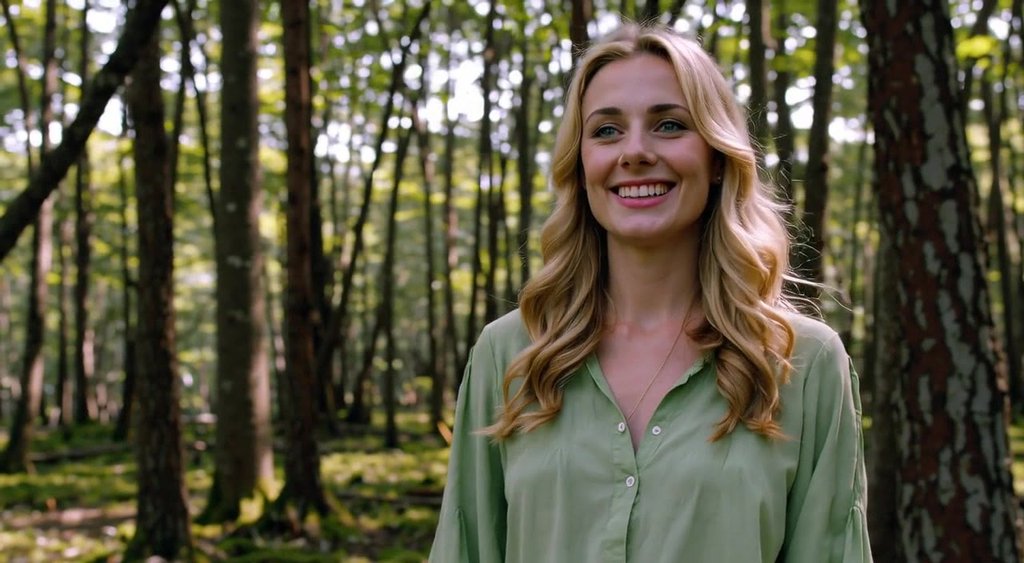} \\ \addlinespace[2pt]
    Distilled
      & \framecellw{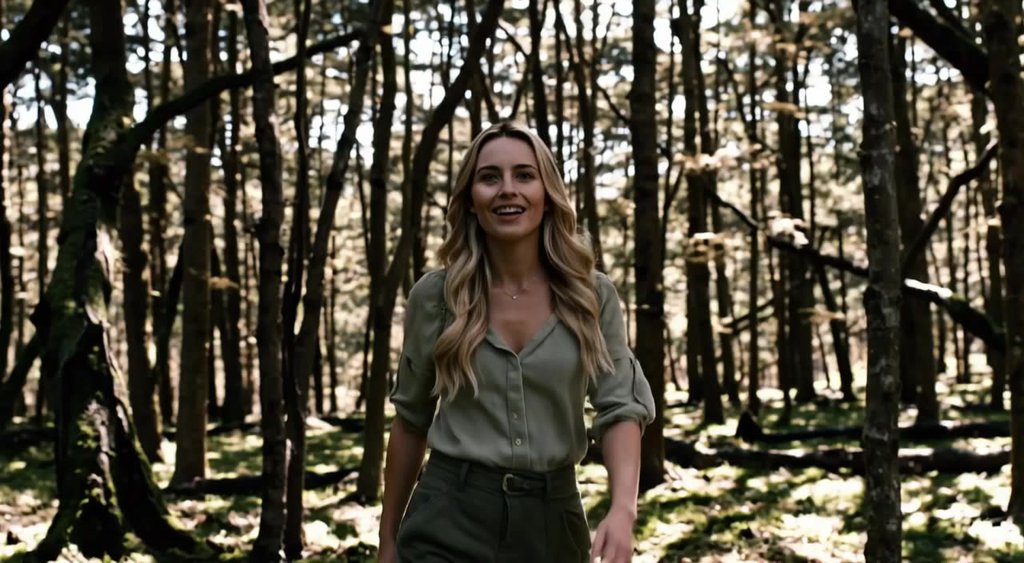} & \framecellw{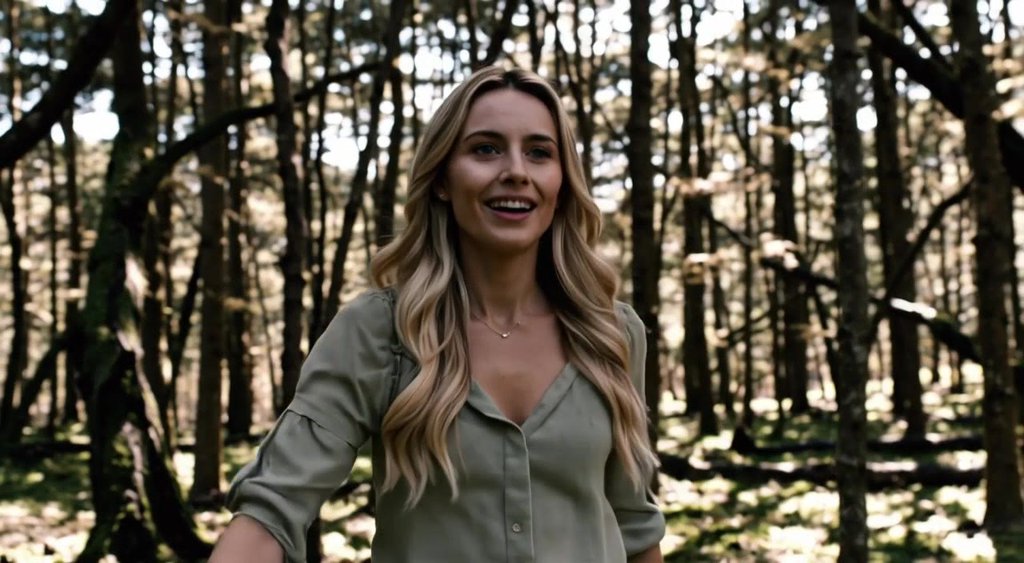} & \framecellw{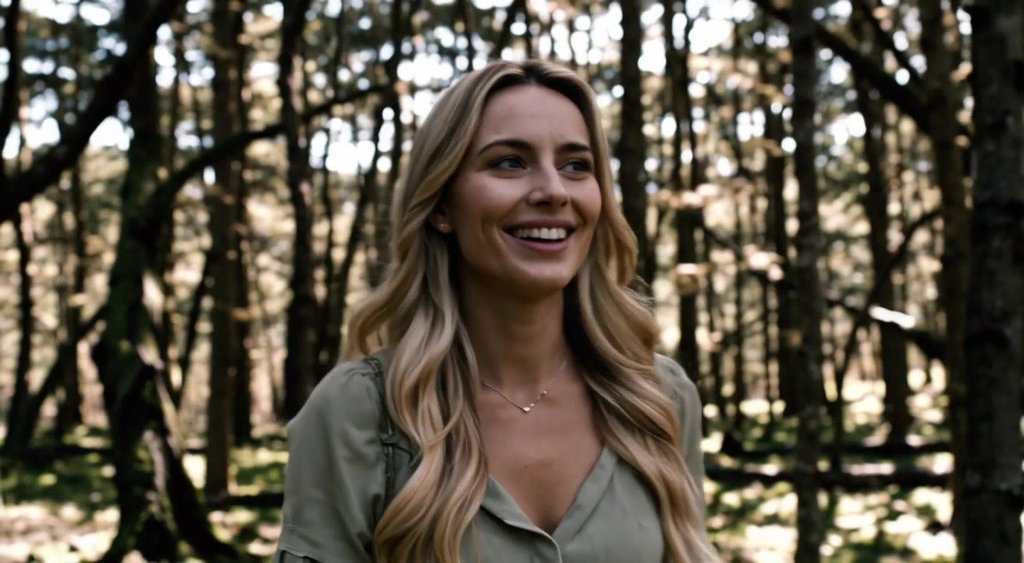} & \framecellw{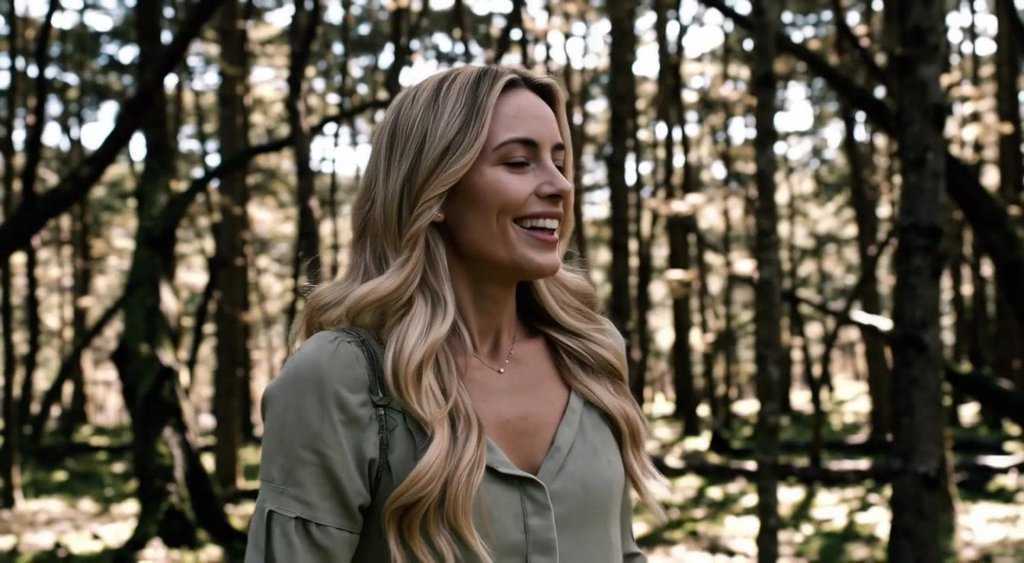} & \framecellw{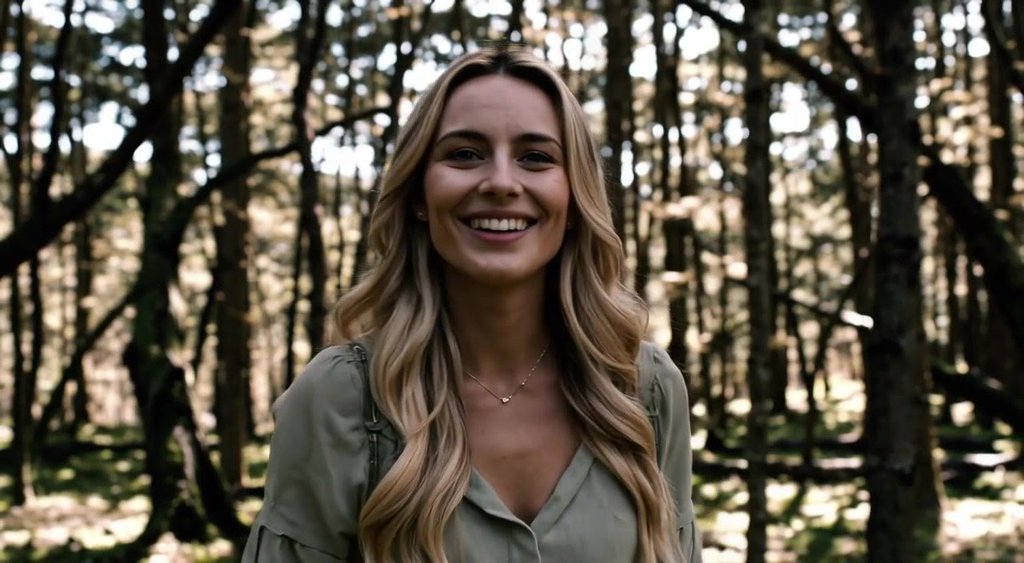} \\ \addlinespace[2pt]
                & \multicolumn{5}{@{} p{16.00cm} @{}}{\tiny\itshape ``A smiling young woman stands confidently in a dense forest, the air filled with the gentle, rhythmic rustling of leaves and the distant, melodic chirping of songbirds echoing through the open woodland. She has fair skin, green eyes, and long wavy blonde hair flowing freely, catching the light as she takes a deep, audible breath of fresh air. She is dressed in a casual outfit consisting of a light green blouse and olive-green cargo pants, the fabric softly swishing as her arms gently swing at her sides. The forest is filled with tall trees and dappled sunlight filtering through the canopy, while her footsteps make a soft, muffled crunch on the moss-covered ground. The camera captures her from a medium close-up angle, focusing on her joyful expression as she lets out a contented, airy sigh, the serene environment alive with the peaceful, immersive sound of a breeze moving through the branches.''} \\
  \end{tabular}%
  }
  \vspace{-10pt}
\end{figure*}

\begin{figure*}[htbp]
  \centering
  \resizebox{\textwidth}{!}{%
  \begin{tabular}{@{} r @{\hspace{6pt}} c@{\hspace{2pt}}c@{\hspace{2pt}}c@{\hspace{2pt}}c@{\hspace{2pt}}c @{}}
                & $t{=}0$s & $t{=}2.5$s & $t{=}5$s & $t{=}7.5$s & $t{=}10$s \\ \addlinespace[2pt]
    Teacher
      & \framecellw{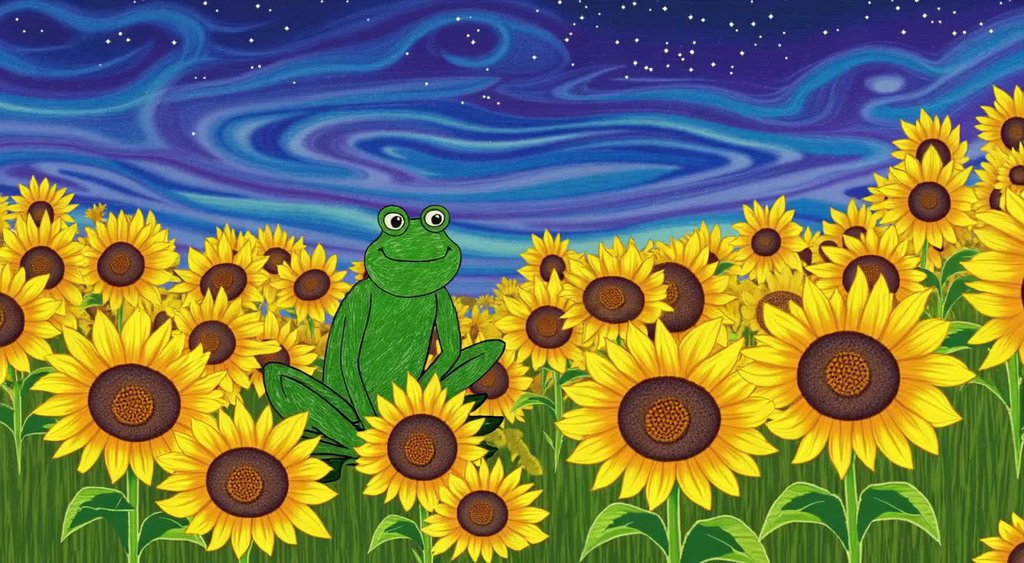} & \framecellw{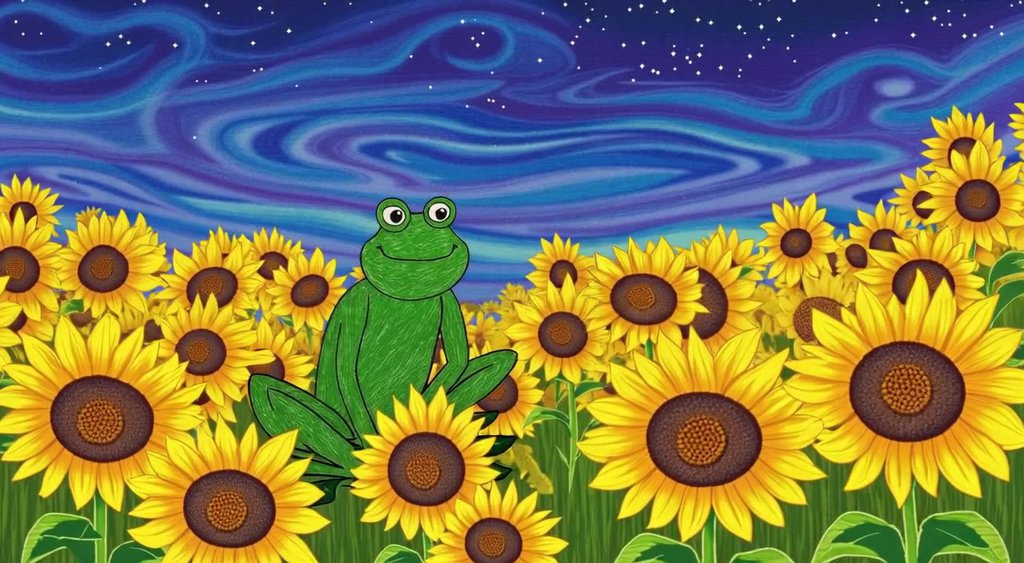} & \framecellw{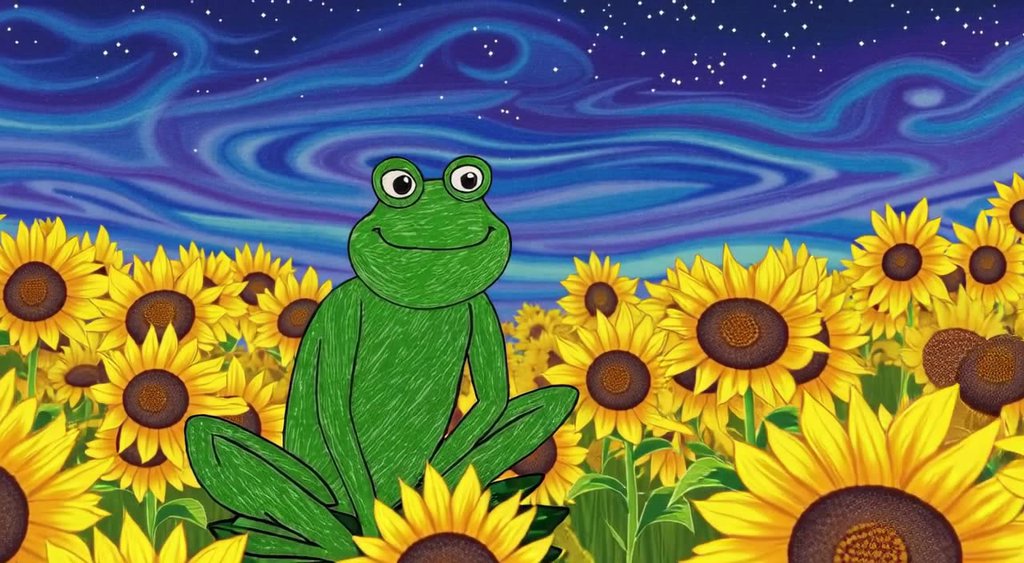} & \framecellw{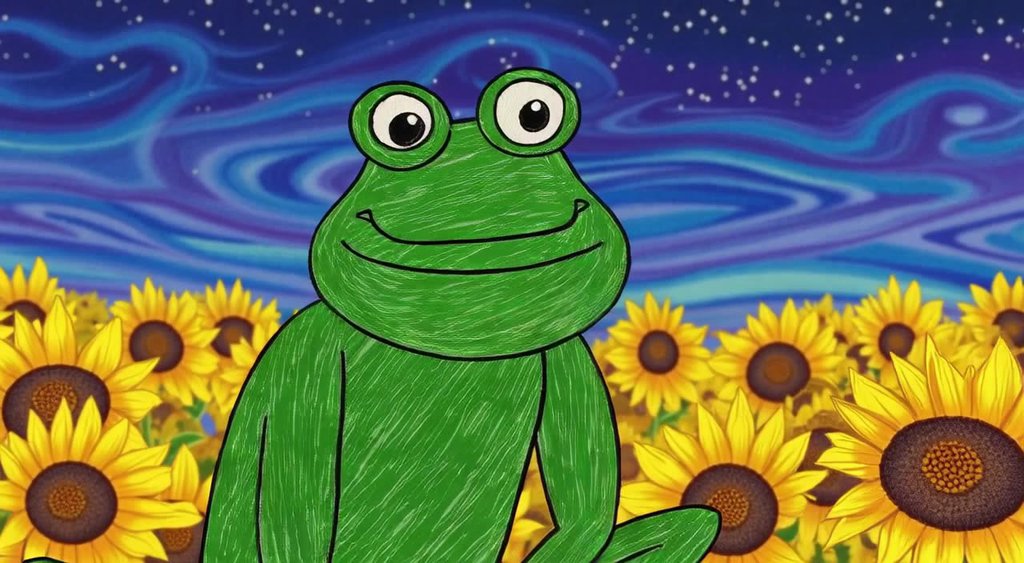} & \framecellw{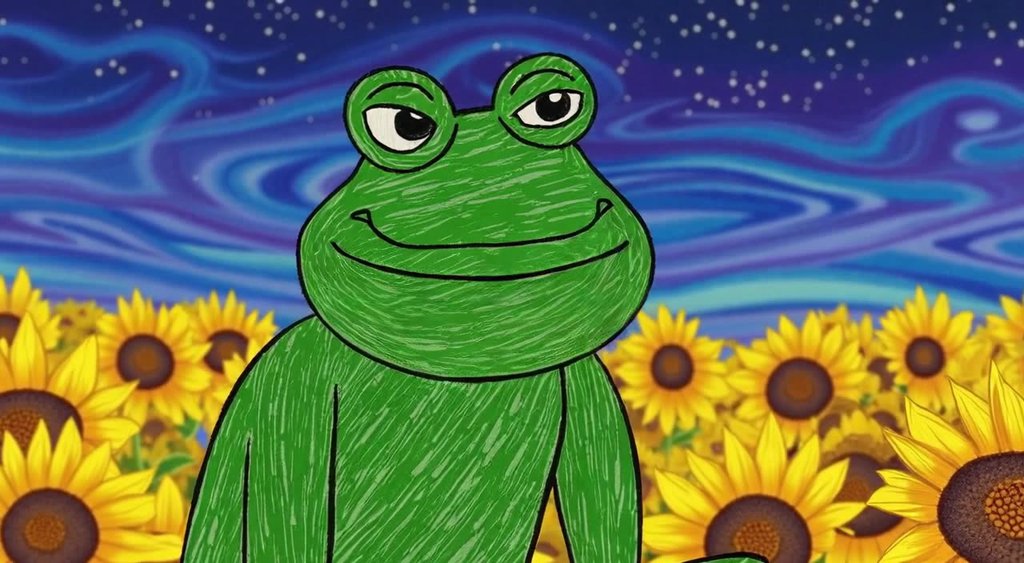} \\ \addlinespace[2pt]
    \textbf{PDD}
      & \framecellw{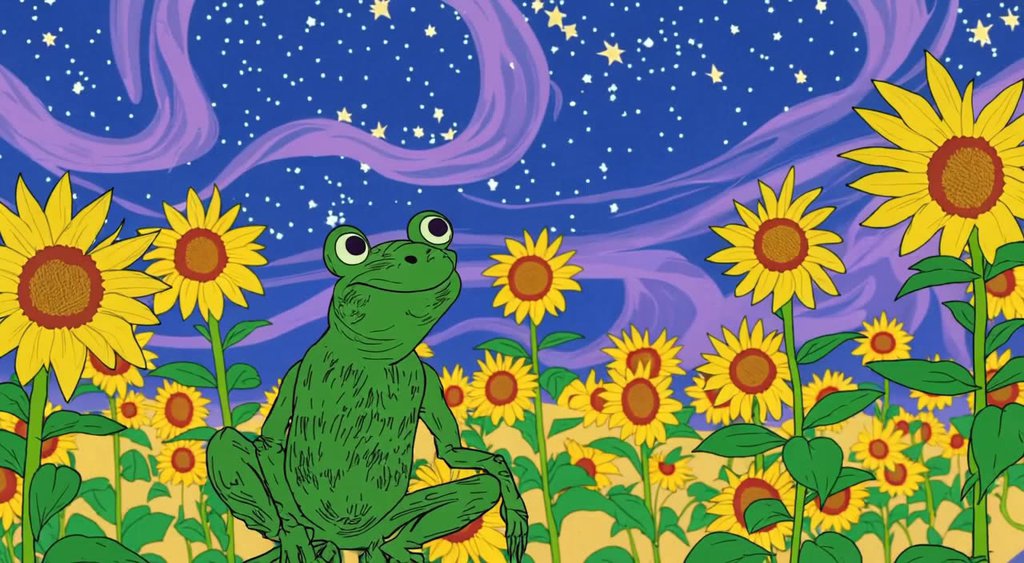} & \framecellw{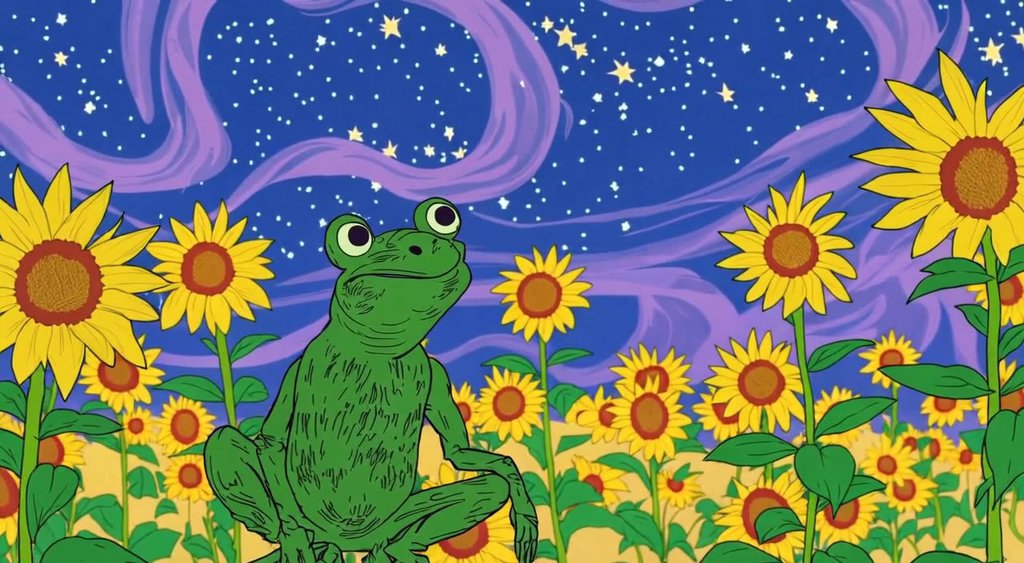} & \framecellw{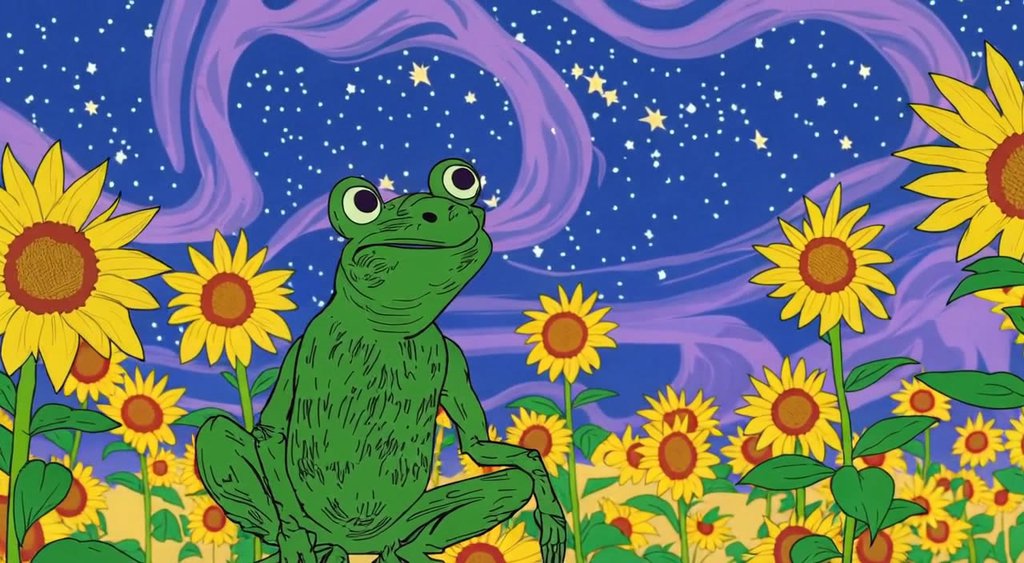} & \framecellw{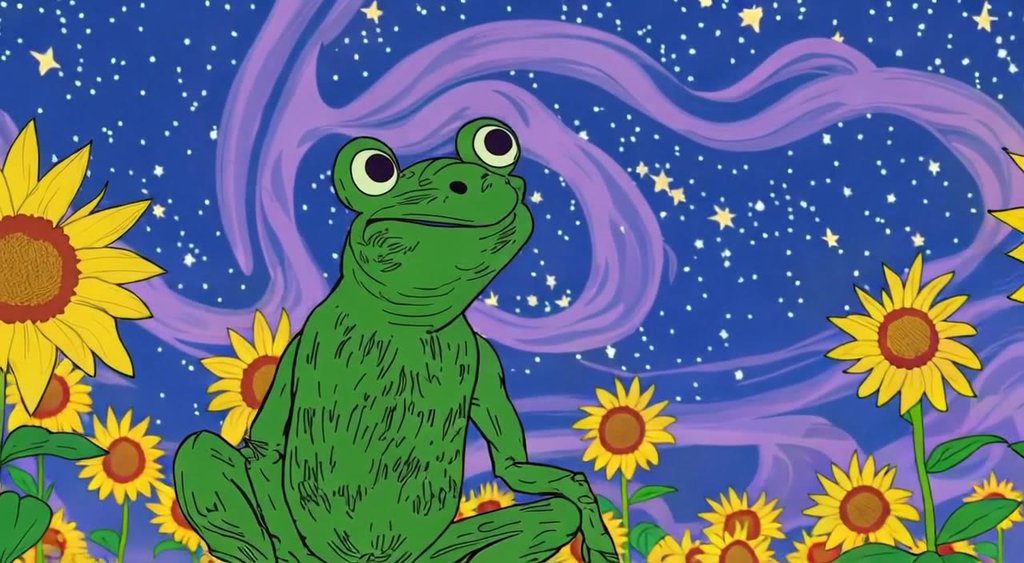} & \framecellw{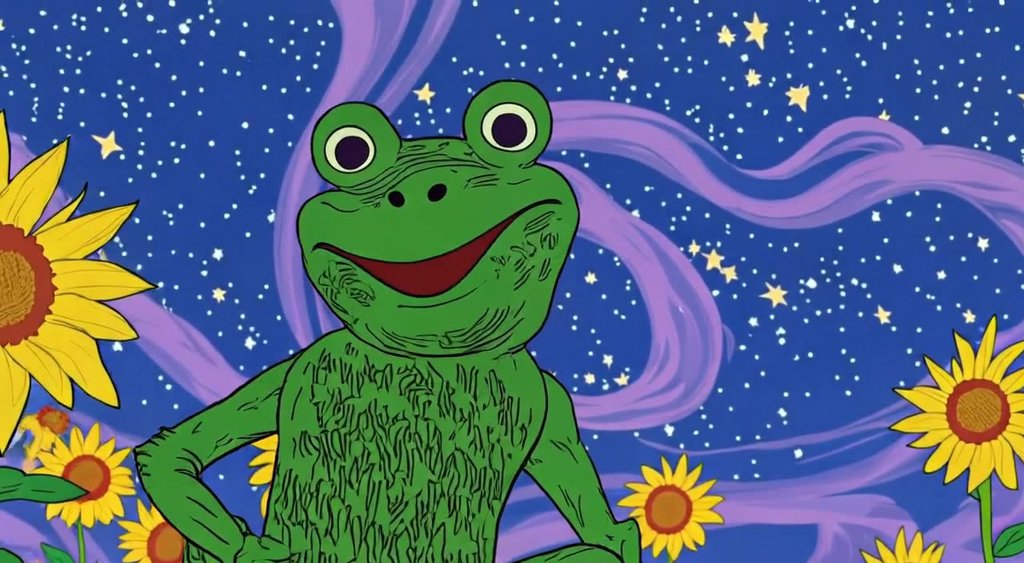} \\ \addlinespace[2pt]
    Distilled
      & \framecellw{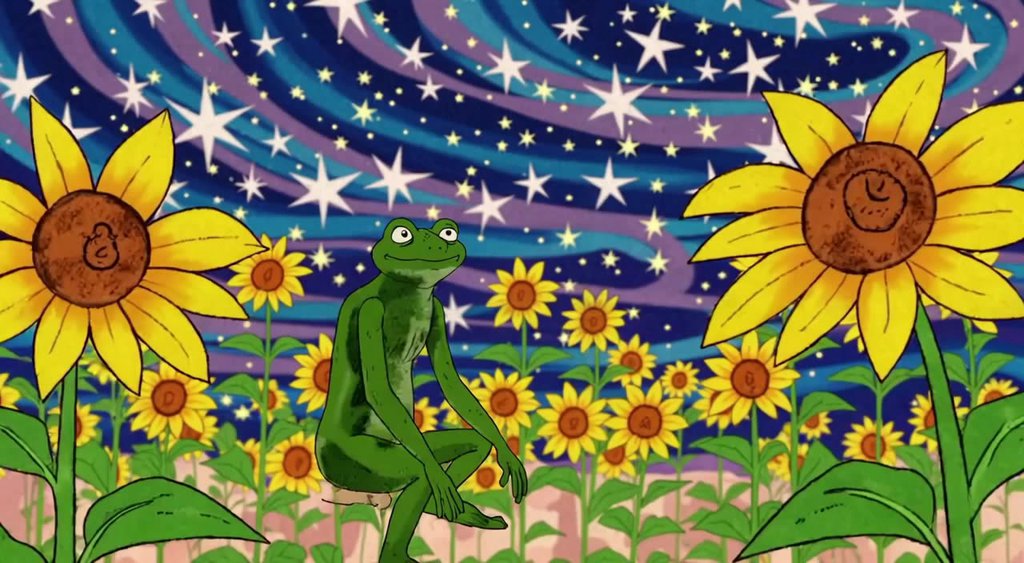} & \framecellw{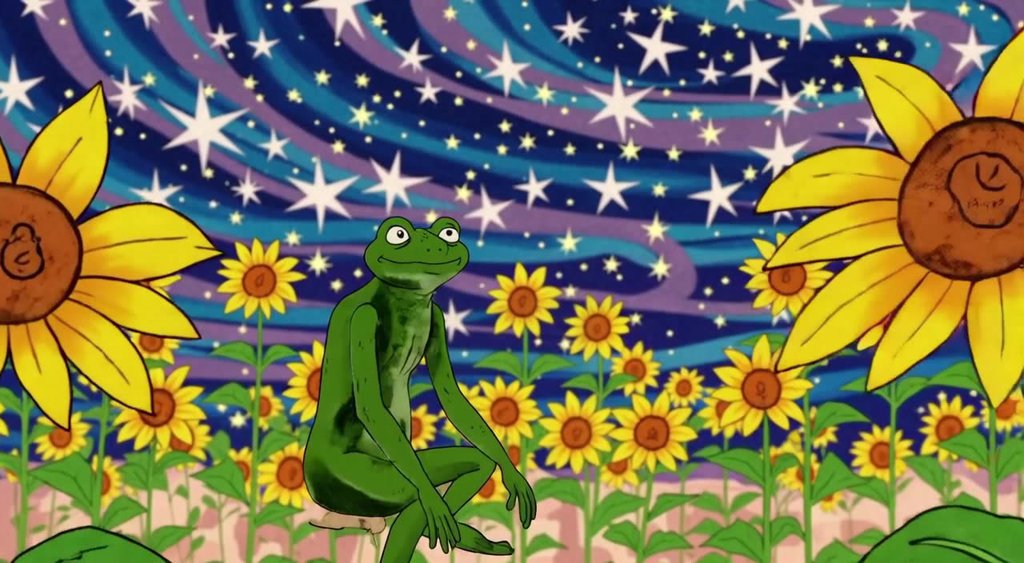} & \framecellw{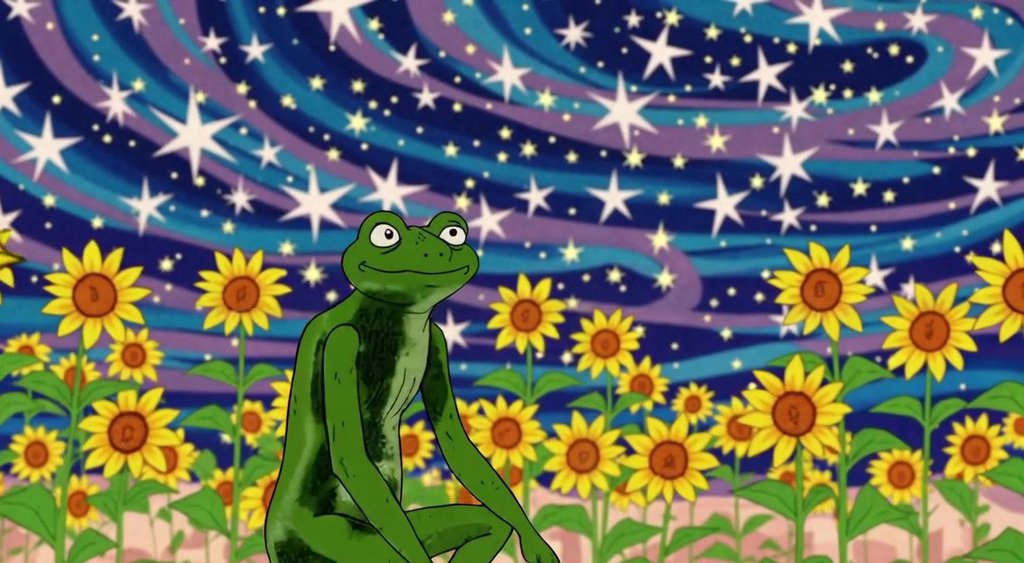} & \framecellw{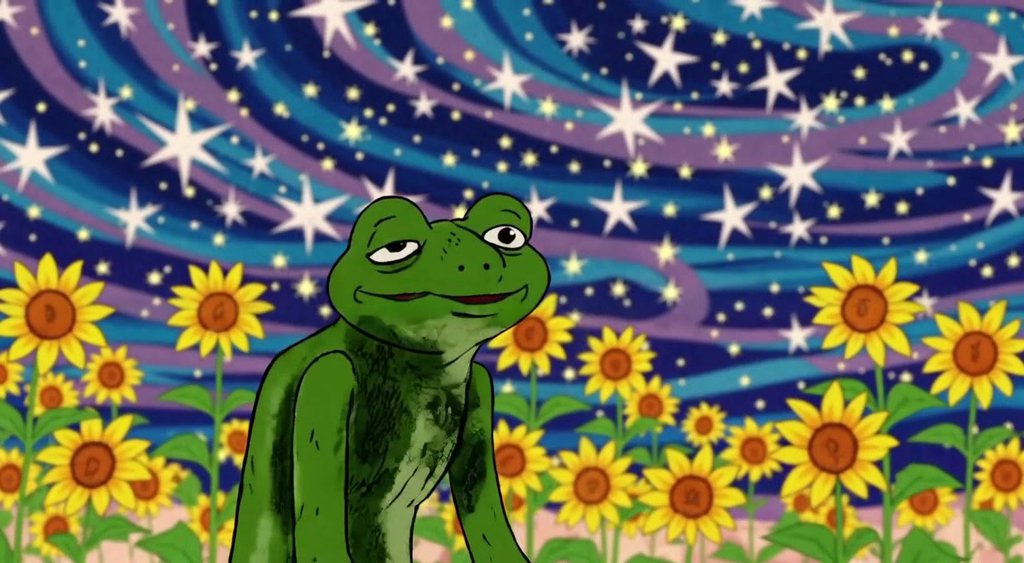} & \framecellw{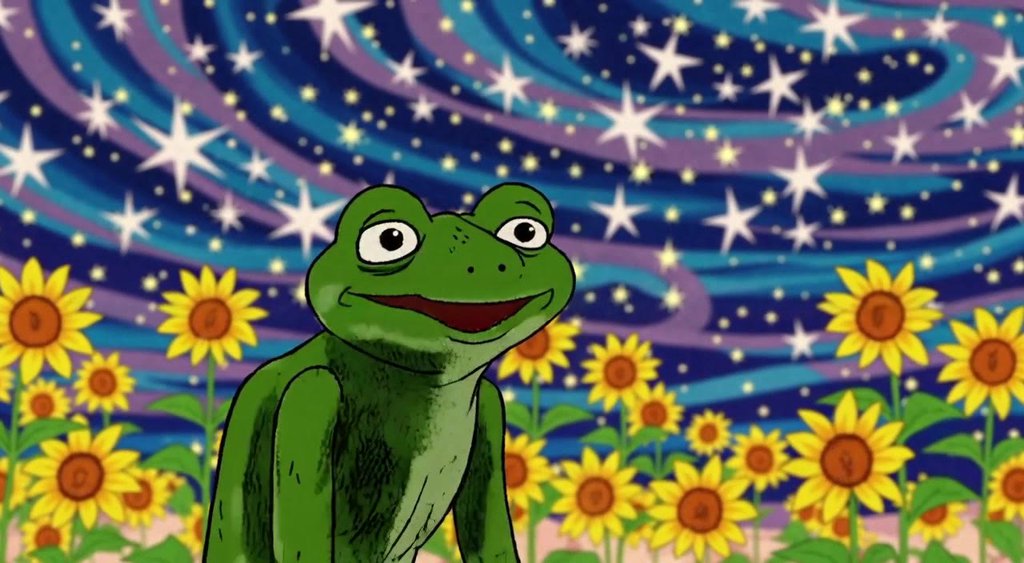} \\ \addlinespace[2pt]
                & \multicolumn{5}{@{} p{16.00cm} @{}}{\tiny\itshape ``Hand-drawn animation in a Van Gogh-inspired style, featuring a green frog with expressive eyes and a wide smile, accompanied by the gentle, rhythmic chirping of crickets and the faint, ethereal hum of a light breeze rustling through the sunflower leaves. The frog is sitting in a field of vibrant sunflowers under a swirling sky filled with stars, its skin showing a textured brushstroke effect that seems to vibrate with the soft, melodic swell of a lone, distant cello playing a melancholic yet peaceful tune. The background showcases a night sky with swirling patterns of blues and purples, blending seamlessly into the golden sunflower field, while the camera zooms in on the frog's lively face as it looks directly at the viewer and lets out a soft, low-pitched, resonant croak that echoes slightly in the vast, open acoustic space of the starry night.''} \\
  \end{tabular}%
  }
  \caption{LTX-2.3 teacher ($4{\times}30$ NFE) vs.\ PDD (8 NFE) vs.\ official distilled model (8 NFE).}
  \label{fig:ltx_comparison_last}
\end{figure*}